\documentclass[colophon, english]{phduio}

\usepackage[utf8]{inputenc} 
\DeclareUnicodeCharacter{0301}{\'{}} 
\usepackage{phdstyle}   
\usepackage{kantlipsum} 

\author{Víctor Mayoral-Vilches}
\title{Offensive Robot Cybersecurity}
\subtitle{Cyber-protecting robots by hacking-them-first with Game Theory and Machine Learning}

\department{Department of Artificial Intelligence and Cybersecurity}
\faculty{System Security Group, Alpen-Adria-Universität Klagenfurt}
\ISSN{1234-5678\todo{revisar}}           
\dissertationseries{1234\todo{revisar}}  

\newbool{contentsStyle}
\boolfalse{contentsStyle} 
\newcommand{\setHyperlinkStyle}{
    \ifbool{contentsStyle}
    { 
        \hypersetup{
            colorlinks = true, 
            linkcolor = RoyalBlue, 
            citecolor = black, 
            urlcolor = black, 
            filecolor = black, 
            pdfborderstyle = {/S/U/W 0}, 
        }        
    }
    { 
        \hypersetup{
            colorlinks = false,
            pdfborderstyle = {/S/U/W 0.9},
            linkbordercolor = {gray},
            citebordercolor = {black},
            urlbordercolor = {lightgray},
            filecolor = black,
            pdfborderstyle={/S/U/W 1}
        }
    }
}

\let\oldfootnote\footnote
\renewcommand{\footnote}[1]{%
    \begingroup
    \hypersetup{linkcolor=lightgray, pdfborderstyle={/S/U/W 0}} 
    \oldfootnote{#1}%
    \endgroup
}

\begin{document}

    \frontmatter        

\thispagestyle{empty}
\begin{tikzpicture}[remember picture,overlay]
    \fill[black] (current page.south west) rectangle (current page.north east);
    
    
    \node[align=left, anchor=west] at ([xshift=3cm, yshift=3cm]current page.west) {
        {\fontsize{32}{38}\selectfont\bfseries\color{white} OFFENSIVE}\\[0.2cm]
        {\fontsize{32}{38}\selectfont\bfseries\color{white} ROBOT}\\[0.2cm]
        {\fontsize{32}{38}\selectfont\bfseries\color{white} CYBERSECURITY}\\[0.5cm]
        {\large\color{white!50}\ttfamily Cyber-protecting robots by hacking-them-first}
    };
    
    \node[align=right, anchor=east] at ([xshift=-3cm, yshift=-7cm]current page.east) {
        {\normalsize\color{white!60} Víctor Mayoral-Vilches}\\[0.8cm]

        {\footnotesize\color{white!30} Supervised by}\\[0.1cm]
        {\footnotesize\color{white!30} Prof. Dr. Stefan Rass}\\[0.1cm]
        {\footnotesize\color{white!30} Prof. Dr. Martin Pinzger}\\[0.8cm]

        {\small\color{white!40} 2018-2025}
    };
\end{tikzpicture}
\clearpage

    \uiotitle
    \thispagestyle{empty}
\vspace*{\stretch{1}}
\begin{flushright}
    \emph{Para Leo y Sam}
\end{flushright}
\vspace*{\stretch{3}}
    \chapter{\abstractname}

Robots, integral to modern automation and services, encounter formidable cybersecurity challenges, primarily due to their inherent complexity and the lack of vendor accountability for security, shifting the burden to end-users. \emph{Offensive Robot Cybersecurity}, the title of this thesis, introduces a groundbreaking approach by advocating for offensive security methods empowered by automation leveraging modern Artificial Intelligence. It emphasizes the necessity of understanding attackers' tactics and identifying vulnerabilities in advance to develop effective defenses, thereby improving robots' security posture. This endeavor leverages a decade of robotics experience, employing Machine Learning and Game Theory to streamline the vulnerability identification and exploitation process. Intrinsically, the thesis uncovers a profound connection between robotic architecture and cybersecurity, highlighting that the design and creation aspect of robotics deeply intertwines with its protection against attacks. This duality—whereby the architecture that shapes robot behavior and capabilities also necessitates a defense mechanism through offensive and defensive cybersecurity strategies—creates a unique equilibrium. Approaching cybersecurity with a dual perspective of defense and attack, rooted in an understanding of systems architecture, has been pivotal in this research. This synergy affirms the adage that \emph{the best defense is a good offense}, especially in the complex realm of robotics. Through comprehensive analysis, including ethical considerations, the development of security tools, and executing cyber attacks on robot software, hardware, and industry deployments, this thesis proposes a novel architecture for cybersecurity cognitive engines. These engines, powered by advanced game theory and machine learning, pave the way for autonomous offensive cybersecurity strategies for robots, marking a significant shift towards self-defending robotic systems. This research not only underscores the importance of offensive measures in enhancing robot cybersecurity but also sets the stage for future advancements where robots are not just resilient to cyber threats but are equipped to autonomously safeguard themselves. 
    \chapter{Resumen}

Los robots, esenciales para la automatización y la industria moderna, enfrentan desafíos formidables en el área de la ciberseguridad, principalmente debido a su complejidad inherente --pues los robots son complejos sistemas con redes de redes-- y la falta de responsabilidad por parte de los fabricantes (software y hardware) de estos robots que con frecuencia ignoran la ciberseguridad, trasladando la responsabilidad a los usuarios finales (las fábricas donde estos operan). Esta tesis titulada \emph{Offensive Robot Cybersecurity} (Ciberseguridad Ofensiva para Robots) introduce un enfoque innovador al abogar por métodos de ciberseguridad ofensivos para proteger robots de forma económica y con recursos limitados, usando la Inteligencia Artificial como medio de automatización. Enfatiza la necesidad de comprender e imitar las tácticas de los ciber-atacantes e identificar vulnerabilidades con anticipación para desarrollar defensas efectivas, mejorando así la ciber-resiliencia de los robots. 

Esta investigación construye sobre una década de experiencias en robótica, y emplea Aprendizaje Máquina (\emph{Machine Learning}) y Teoría de Juegos (\emph{Game Theory}) para automatizar el proceso de identificación y uso de vulnerabilidades en robots. De forma intrínseca, esta tesis estudia la profunda conexión entre la arquitectura robótica (tanto software como hardware) y la ciberseguridad, destacando que el diseño y la creación de los robots debe entrelazarse profundamente con su protección frente a ciber-ataques. Esta dualidad, donde la arquitectura que da forma al comportamiento y capacidades del robot también requiere de un mecanismo de defensa a través de estrategias de ciberseguridad, crea un equilibrio único, difícil de resolver y cambiante en el tiempo. Se aborda este reto con una perspectiva dual de ataque y posterior defensa, arraigada en un entendimiento de la arquitectura de sistemas y en el popular dicho: \emph{la mejor defensa es un buen ataque}. A través de un análisis exhaustivo, incluyendo consideraciones éticas, el desarrollo de herramientas de ciberseguridad y la ejecución de ciberataques en software, hardware y despliegues industriales robóticos, esta tesis describe una arquitectura novedosa para motores cognitivos de ciberseguridad. Estos motores, impulsados por teoría de juegos avanzada y aprendizaje máquina automático, permiten la implementación de estrategias de ciberseguridad ofensivas y autónomas para robots, marcando un cambio significativo hacia sistemas robóticos auto-defensivos mediante técnica ofensivas. Esta investigación no solo subraya la importancia de las medidas ofensivas en la mejora de la ciberseguridad para los robots y otros sistemas industriales, sino que también prepara el terreno para avances futuros donde los robots no solo sean resistentes a las amenazas cibernéticas, sino que también estén equipados para protegerse y responder de manera autónoma.
    \chapter{Preface}


My professional career in robotics began in 2013, with a first stage focused on the study of the software architecture of robots. After various international experiences, in 2014, I had the fantastic opportunity to be part of the group that started one of the most significant revolutions in robotics in recent times: the creation of the Robot Operating System 2 (ROS 2). Used today by the majority of robotics engineers around the world, ROS is the common language used in robotics to build robotic behaviors. It is a set of modular and open source software libraries that include tools, interconnection and communication elements between robotic processes. Technically, ROS provides a message-passing infrastructure between robot software components, drivers for robotic hardware, state-of-the-art robotic algorithms, and powerful development tools to simulate, debug, and visualize a robot's behaviors. Simply put, ROS allows building robot brains and my involvement in building ROS 2 allowed me to grow rapidly as an engineer, and encouraged me to create dozens of robots over the next few years.

After half a decade building robots and leading teams to do so, in 2018 I decided to embark on a journey to explore cybersecurity applied to robots. Throughout this journey, I discovered the deep connection between robotic architecture and cybersecurity in robotics, more intimate than initially perceived. Robotic architecture focuses on creation, on shaping behaviors and possibilities, while cybersecurity in robotics is oriented towards offense or protection, allowing what is built to be defended. This synergy creates a unique balance when cybersecurity is approached from a knowledge of systems architecture and with a dual perspective: defense and attack, both essential. This learning, the result of exploration and study, has been key in my research, reaffirming that in the field of robots, as in many others, the best defense is indeed a good offense.

I began this research (my PhD) formally in 2020 with a clear intention: to create the technological principles so that an artificial intelligence is capable of hacking robots better than any human. I write these lines happy to have largely satisfied this initial objective while having open sourced much of it. During these years, I developed and participated in the creation of technology and prototypes that allow achieving superhuman capabilities in terms of offensive security in robotics. My contributions to the area of robotic cybersecurity today allow the creation of artificial systems that guide and even automate cyberattacks. No human is capable of processing so much so quickly. To respond on multiple fronts. Or training himself to exploit a new \emph{target} or a new attack vector in a matter of a few seconds. It is surprising, exciting, but also overwhelming what these AIs can do.\\

I'd like to thank everyone that supported me in this journey. Thanks to Laura, my companion on this and many other trips. Her constant support, patience, and vital balance have been fundamental pillars in the most challenging moments of this research, providing not only emotional refuge, but also unique perspectives that have enriched my research. Her intelligence and empathy have been a constant source of inspiration and strength every day. Thanks to my father, whose tireless support and unconditional trust have been my beacon in the darkness of those most difficult times. His wisdom and calm approach to life's challenges have taught me to face problems with patience and determination. His example of perseverance and curiosity to "always keep learning" have been a constant guide in my life. Thanks to my mother for never giving up. Thanks to David for his drive and resilience. Always there, despite our many mistakes. Thanks to Stefan and Martin, my supervisors, for their expert guidance and unwavering scientific support throughout this process. Their ability to challenge and expand my ideas, combined with their commitment to academic excellence, has been crucial to my development as a researcher. Stefan, with his renowned expertise as a mathematician in security and game theory, has brought a unique and rigorous perspective, enriching my understanding of the theoretical aspects of cybersecurity. Martin, renowned for his deep knowledge of software and software engineering, has been instrumental in honing my hands-on approach, ensuring my applications and tools are robust and advanced. Together, they have formed an unrivaled supervisory team, offering a balanced blend of theory and practice that has been essential to the depth and success of my research.

I greatly appreciate the support received from the company Alias Robotics. This research has been partially funded by SEGRES (EXP 00131359 / MIG-20201041) grant to Alias Robotics. Its funding and trust in my work has been essential in carrying out some significant parts of my research, allowing me to dive deeply into the field of robotics cybersecurity. Thanks to Endika and Unai for successfully taking the lead of the corporate daily matters, while I was focused on research. Last but not least, thanks to the rest of my family for their unwavering faith in my abilities over the past. Thanks to María for the sacrifices for a better future, and to Ramón for the stolen acorns in times of need. Now, your grandson keeps acorns on the tree in his garden, sowing perseverance and affection. Thanks also to Rosa and Adela. Their emotional support, teachings and trust in my dreams have been a constant source of motivation and courage to keep going, even in the most difficult times.\\

For any reader, but especially for my two boys, may this learning encourage you to approach the challenges of the future with courage, resilience and cunning.\\

\vskip\onelineskip
\begin{flushleft}
    \sffamily
    \textbf{\theauthor}
    \\
    Zuia,\MONTH\the\year
\end{flushleft}
    \chapter{Prefacio}


Mi carrera profesional en robótica comenzó en 2013, con una primera etapa centrada en el estudio de la arquitectura software de los robots. Tras varias experiencias internacionales, en 2014, tuve la fantástica oportunidad de formar parte del grupo que inició una de las revoluciones en robótica más significativas de los últimos tiempos: la creación del Sistema Operativo de Robots 2 (ROS 2). Usado a día de hoy por la mayoría de ingenieros e ingenieras de robótica por todo el mundo, ROS es el lenguaje común que se utiliza en robótica para construir comportamientos robóticos. Es un conjunto de librerías de software modulares y de código abierto que incluyen herramientas, y elementos de interconexión y comunicación entre procesos robóticos. Técnicamente, ROS proporciona un sistema de transmisión de mensajes entre componentes de software, controladores para el hardware robótico, algoritmos robóticos de última generación y potentes herramientas de desarrollo para simular, depurar y visualizar los comportamientos de un robot. De forma sencilla, ROS permite crear cerebros robóticos y mi participación en la construcción de ROS 2 me permitió crecer rápidamente como ingeniero, y me animó a crear decenas de robots durante los años siguientes.

Tras media década construyendo robots y liderando equipos para ello, en 2018 decidí emprender un viaje para explorar la ciberseguridad aplicada a los robots. A lo largo de este viaje, he descubierto la profunda conexión entre la arquitectura robótica y la ciberseguridad en robótica, más íntima de lo que inicialmente se percibe. La arquitectura robótica se enfoca en la creación, en moldear comportamientos y posibilidades, mientras que la ciberseguridad en robótica se orienta hacia la ofensiva o la protección, permitiendo defender lo construido. Esta sinergia crea un balance único cuando se aborda la ciberseguridad desde un conocimiento de arquitectura y con una perspectiva dual: defensa y ataque, ambas esenciales. Este aprendizaje, fruto de la exploración y el estudio, ha sido clave en mi investigación, reafirmando que en el ámbito de los robots, como en otros muchos, la mejor defensa es efectivamente un buen ataque.

Comencé esta investigación (mi doctorado) formalmente en 2020 con una intención clara: crear los principios tecnológicos para que una inteligencia artificial sea capaz de hackear robots mejor que cualquier humano, para después mitigar dichas vulnerabilidades antes de aparezcan los primeros ataques. Escribo estas líneas contento de haber satisfecho en gran medida este objetivo inicial. Durante estos años, he desarrollado y participado en la creación de tecnología y prototipos que permiten alcanzar capacidades sobre-humanas en lo que a la seguridad ofensiva en robótica se refiere. Mucho de este trabajo ha sido contribuído abiertamente (\emph{open sourced}), con licencias permisivas. Mis contribuciones al área de la ciberseguridad robótica permiten hoy la creación de sistemas artificiales que orientan e incluso automatizan ciber-ataques. Ningún humano es capaz de procesar tanto y tan rápido. De responder a múltiples frentes. O de capacitarse para explotar un nuevo \emph{target} o un nuevo vector de ataque en cuestión de pocos segundos. Es sorprendente, excitante pero también sobrecogedor. \\

Me gustaría terminar agradeciendo a todas aquellas y aquellos que me han apoyado en esta aventura. Gracias a Laura, mi compañera en este y otros muchos viajes. Su apoyo constante, su paciencia y su equilibrio vital han sido pilares fundamentales en los momentos más desafiantes de esta investigación, proporcionando no solo un refugio emocional, sino también perspectivas únicas que han enriquecido mi investigación. Su inteligencia y empatía han sido una fuente constante de inspiración y fuerza cada día. Gracias a mi padre, cuyo apoyo incansable y confianza incondicional han sido mi faro en la oscuridad de esas épocas más difíciles. Su sabiduría y su enfoque sereno ante los desafíos de la vida me siguen enseñando a enfrentar los problemas con paciencia y determinación. Su ejemplo de perseverancia y curiosidad por "siempre seguir aprendiendo" han sido una guía constante en mi vida. Gracias a mi madre por nunca darse por vencida. Gracias a David por su fuerza y resiliencia. Siempre junto a mí, a pesar de nuestras muchas equivocaciones. Gracias a Stefan y a Martin, mis supervisores, por su orientación experta y apoyo científico inquebrantable durante todo este proceso. Su capacidad para desafiar y expandir mis ideas, combinada con su compromiso con la excelencia académica, ha sido crucial para mi desarrollo como investigador. Stefan, con su renombrada experiencia como matemático en ciberseguridad y teoría de juegos, ha aportado una perspectiva única y rigurosa, enriqueciendo mi comprensión de los aspectos teóricos de la ciberseguridad. Martin, reconocido por sus profundos conocimientos en software y software engineering, ha sido fundamental para pulir mi enfoque práctico, asegurando que mis aplicaciones sean robustas y avanzadas. Juntos, han formado un equipo de supervisión inigualable, ofreciendo una combinación equilibrada de teoría y práctica que ha sido esencial para la profundidad y el éxito de mi investigación.

Agradezco enormemente el apoyo recibido por parte de la empresa Alias Robotics.  Esta investigación ha sido parcialmente financiada por el proyecto SEGRES (EXP 00131359 / MIG-20201041) concedido a Alias Robotics. Su financiación y confianza en mi trabajo han sido esenciales para llevar a cabo algunas partes significativas de mi investigación, permitiéndome sumergirme profundamente en el campo de la ciberseguridad robótica. Gracias a Endika y Unai por liderar con éxito el día a día corporativo, mientras yo me centraba en la investigación. Por último, pero no menos importante, gracias al resto de mi familia por su fé inquebrantable en mis capacidades a lo largo de muchos años. Gracias a María por los sacrificios por un futuro mejor y a Ramón por las bellotas robadas en tiempos de necesidad. Ahora, vuestro nieto guarda bellotas en el árbol de su jardín, sembrando perseverancia y cariño. Gracias también a Rosa y a Adela. Su apoyo emocional, sus enseñanzas y su confianza en mis sueños han sido una fuente constante de motivación y coraje para seguir adelante, incluso en los momentos más difíciles.\\

Para cualquier lector o lectora, pero especialmente para mis dos chicos, que este aprendizaje os impulse a abordar con valentía, resiliencia y astucia los retos del futuro.\\

\vskip\onelineskip
\begin{flushleft}
    \sffamily
    \textbf{\theauthor}
    \\
    Zuia,\MONTH\the\year
\end{flushleft}

    \booltrue{contentsStyle}
    
    \cleartorecto
    \microtypesetup{protrusion = false}
    \tableofcontents    
    \cleartorecto
    \listoffigures      
    \cleartorecto
    \listoftables       
    \microtypesetup{protrusion = true}

    \boolfalse{contentsStyle}

    \mainmatter         

    \part{Motivation}\label{part:motivation}
    
\chapter{Introduction}
\label{chapter:intro}

\section{Why robot cybersecurity?}

For the last fifty years, we have been witnessing the dawn of the robotics industry, but robots are not being created with security as a concern. Security in robotics is often mistaken with safety. From industrial to consumer robots, going through professional ones, most of these machines are not prepared for cyber-threats and resilient to security vulnerabilities. Manufacturers' concerns, as well as existing standards, focus mainly on safety. Security in robotics is still not being considered as a relevant matter. The integration between these two areas from a risk assessment perspective was studied in \cite{1673343, kirschgens2018robot} which resulted in a unified security and safety risk framework.
Commonly, robotics safety is understood as developing protective mechanisms against accidents or malfunctions, whilst security is aimed to protect systems against risks posed by malicious actors \cite{safetysecurity}. A slightly alternative view is the one that considers safety as protecting the environment from a given robot, whereas security is about protecting the robot from a given environment. In this proposal we adopt the later and tackle the emerging need to address security in robotics.\\
\newline

\section{About ROS}
\label{sec:ros}

Robots are networks of networks, with sensors capturing data, passing to compute technologies, and then on to actuators and back again in a deterministic manner. These networks can be understood as the nervous system of the robot, passing across compute \emph{Nodes}, that represent neurons. Like the human nervous systems, real-time information across all these computational \emph{Nodes} is fundamental for the robot to behave coherently.

"\emph{Robot brains}" are built with this same philosophy. Behaviors take the form of computational graphs, with data flowing between Nodes operating intra-process, inter-process and across physical networks (communication buses), while mapping to underlying sensors, compute technologies and actuators. The Robot Operating System (ROS) \cite{quigley2009ros} is a robotics framework for robot application development. ROS enables a robotics developer to build these computational graphs and create robot behaviors by providing libraries, a communication infrastructure, drivers and tools to put it all together. It provides an open source codebase with a commercially friendly license that helps roboticists reduce the effort to bring up robot behaviors. Widely used to govern interactions across robots, at the time of writing, the original ROS article \cite{quigley2009ros} has been cited about 12000 times, which shows its wide acceptance for research and academic purposes. ROS was born in this environment: its primary goal was to provide the software tools and libraries that users would need to employ to undertake novel robotics research and development. Adoption in industry has also been rapidly increasing over the last few years. According to the latest ROS community metrics \cite{rosmetrics} sampled every year in July, the number of ROS downloads has increased by over 50\%, with about 600 million downloads between July of 2020 and July of 2021\footnote{There's a slowdown identified in the year 2022 according to the data. This is justified mostly due to two reasons: 1) COVID influence and 2) slowdown of ROS contributors and interest due to Intrinsic's acquisition of Open Robotics}. Moreover, based on the download percentages reported from \texttt{packages.ros.org}, we observe a significant increase in adopting ROS 2, which suggests that within 2024 there will be more users using ROS 2 than its predecessor\footnote{We also note that past studies estimated that by 2024, 55\% of the total commercial robots shipped that year would include at least one ROS package. For more details, refer to \emph{\url{https://www.businesswire.com/news/home/20190516005135/en/Rise-ROS-55-total-commercial-robots-shipped}}.}.

\begin{figure}[!h]
    \centering
    \scalebox{0.4}{
    \tikzset{
      state/.style={
        rectangle, rounded corners, draw=gray!80, gray!20, thick,
        minimum height=2em, inner sep=10pt, text centered, text=black},
      state2/.style={
        rectangle, rounded corners, draw=black, fill=gray!70, thick,
        minimum height=2em, inner sep=10pt, text centered},
      epath/.style={draw, <->, shorten <=1pt, shorten >=1pt},
      cir/.style={draw, circle, fill, inner sep=2.5pt, text=white},
      fpath/.style={draw, color=gray, rounded corners=3pt, shorten <=3pt, shorten >=3pt, align = center, minimum height = 10pt, font = \bfseries},
      dir/.style={near start, right, font=\small},
    }
    \begin{tikzpicture}[->, >=latex, line width=0.75pt]
      \draw[fill=white, color=white] (-10.8,-4.5) rectangle (9.2,4.5);

    \node[state, text width=23.3cm, text height=0.2cm, fill=black!10]
      (rclkrs) at (1.7, 0.8) {
        {\tt ROS Client Library (rcl)}\\
    };
    
    \node[state, text width=23.3cm, text height=0.2cm, fill=black!10]
      (rmwkrs) at (1.7, -0.3) {
        {\tt ROS MiddleWare interface (rmw)}\\
    };
    
    \node[state, text width=7cm, text height=0.2cm, fill=black!10]
        (dds1adapter) at (-6.45, -1.39) {
            {DDS impl. 1 adapter (\texttt{rmw\_<dds1>})}\\
    };
    \node[state, text width=7cm, text height=1cm, fill=black!10]
      (dds1) at (-6.45,-2.8) {
        {\tt DDS implementation 1}\\
    };

    \node[state, text width=0.3cm, text height=0.2cm, fill=black!10]
        (ddsoadapter) at (-1.9, -1.39) {
            {...}\\
    };
    \node[state, text width=0.3cm, text height=1cm, fill=black!10]
      (ddso) at (-1.9,-2.8) {
        {...}\\
    };

    \node[state, text width=3.5cm, text height=0.2cm, fill=black!10]
        (ddnsadapter) at (0.9, -1.39) {
            {DDS impl. \texttt{n} adapter}\\
    };
    \node[state, text width=3.5cm, text height=0.6cm, fill=black!10]
      (ddns) at (0.9,-2.8) {
        {\tt DDS implementation n}\\
    };

    \node[state, text height=0.2cm, text width=9.9cm, fill=black!10]
        (iceoryxadapter) at (8.3, -1.39) {
            {Eclipse iceoryx™ adapter (\texttt{rmw\_iceoryx})}\\
    };
    \node[state, text height=1cm, text width=9.9cm, fill=black!10]
      (iceoryx) at (8.3,-2.8) {
        {\tt Eclipse iceoryx™}\\
    };
    

    \node[state, text width=11.6cm, text height=1.0cm, fill=black!10]
        (petalinux) at (-4.1, -4.7) {
            {Yocto/PetaLinux    }\\
    };

    \node[state, text width=10.9cm, text height=1.0cm, fill=black!10]
        (ubuntu) at (8.0, -4.7) {
            {Ubuntu Linux}\\
    };
    
    \node[state, text width=23.3cm, text height=1.0cm, fill=black!10]
      (lns) at (1.7, -6.6) {
            {Linux Networking Stack (LNS)}\\
    };
    
    \node[state, text width=3.5cm, text height=1.9cm, fill=black!10]
      (ethernet) at (-8.2, -8.9) {
            {Ethernet}\\
    };
    \node[state, text width=3.5cm, text height=1.45cm, fill=black!10]
      (tsn) at (-3.8, -8.9) {
            {Time Sensitive Networking (TSN)}\\
    };
    \node[state, text width=3.5cm, text height=1.9cm, fill=black!10]
      (ethercat) at (0.6, -8.9) {
            {WiFi}\\
    };
    \node[state, text width=3.5cm, text height=1.9cm, fill=black!10]
      (can) at (5.0, -8.9) {
            {5G}\\
    };
    % ...
    \node[state, text width=5.7cm, text height=1.9cm, fill=black!10]
      (otherlink) at (10.5, -8.9) {
            {...}\\
    };
    

    %
    
    % ...

    \node[state, text width=3.2cm, text height=8.5cm, fill=black!60, text=white]
      (osil7) at (-12.9, 2.6) {
        {\tt 7. Application}\\
    };
    
    \node[state, text width=3.2cm, fill=black!60, text=white]
      (osil6) at (-12.9, -2.7) {
        {\tt 6. Presentation}\\
    };

    \node[state, text width=3.2cm, fill=black!60, text=white]
      (osil5) at (-12.9, -3.8) {
        {\tt 5. Session}\\
    };

    \node[state, text width=3.2cm, text height=1cm, fill=black!60, text=white]
      (osil4) at (-12.9, -5.3) {
        {\tt 4. Transport}\\
    };

    \node[state, text width=3.2cm, fill=black!60, text=white]
      (osil3) at (-12.9, -6.8) {
        {\tt 3. Network}\\
    };

    \node[state, text width=3.2cm, fill=black!60, text=white]
      (osil2) at (-12.9, -7.9) {
        {\tt 2. Data Link}\\
    };

    \node[state, text width=3.2cm, text height=1cm, fill=black!60, text=white]
      (osil1) at (-12.9, -9.4) {
        {\tt 1. Physical}\\
    };

    \node[state, text width=18.9cm, text height=0.2cm, fill=black!10]
      (embeddedappstore) at (-0.5, 1.9) {
        {\tt device drivers} (sensors, actuators, etc.)\\
    };
    
    \node[state, text width=3.5cm, text height=0.6cm, fill=black!10]
      (navstackaccel) at (-8.2, 3.4) {
            {\texttt{navigation stack} Node \#1}\\
    };
    \node[state, text width=3.5cm, text height=0.6cm, fill=black!10]
      (navstackaccel2) at (-8.2, 5.3) {
            {\texttt{navigation stack} Node \#2}\\
    };

    \node[state, text width=3.5cm, text height=0.6cm, fill=black!10]
      (manaccel) at (-3.8, 3.4) {
            {\texttt{manipulation stack} Node \#1}\\
    };


    \node[state, text width=3.5cm, text height=0.6cm, fill=black!10]
      (perceptaccel) at (0.6, 3.4) {
            {\texttt{perception stack} Node \#1}\\
    };
    \node[state, text width=3.5cm, text height=0.6cm, fill=black!10]
      (perceptaccel2) at (0.6, 5.3) {
            {\texttt{perception stack} Node \#2}\\
    };

    \node[state, text width=5.7cm, text height=0.6cm, fill=black!10]
      (actaccel) at (6.1, 3.4) { 
            {\texttt{control stack} \\Node \#1}\\
    };    
    \node[state, text width=5.7cm, text height=0.6cm, fill=black!10]
      (actaccel2) at (6.1, 5.3) {
            {\texttt{control stack} \\Node \#2}\\
    };

    \node[state, text width=3.5cm, text height=0.6cm, fill=black!10]
      (accapp1) at (11.6, 2.3) {
            {\texttt{geometry stack (tf2)}  Node \#1}\\
    };    
    \node[state, text width=3.5cm, text height=0.6cm, fill=black!10]
      (accapp1) at (11.6, 4.2) {
            {\texttt{geometry stack (tf2)}  Node \#2}\\
    };    

    \path[-, dashed, fpath, color=black] ($(rclkrs)+(8.5, 0.6)$) -- ($(rclkrs)+(12.7,0.6)$) |-
            node[dir, text width=5cm] {\Large \textbf{ROS 2} \\ core layers} ($(iceoryxadapter)+(2,-0.5)$);


    


    \end{tikzpicture}}
    \caption{The robotics stack, including ROS 2 core layers. Reworked from \cite{mayoral2021kria}}
    \label{fig:robotics_stack}
\end{figure}
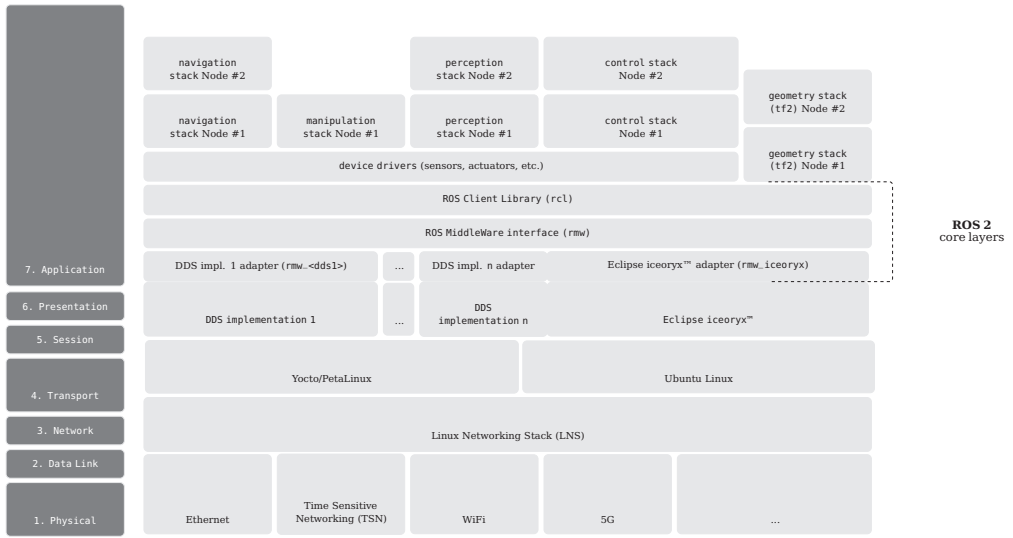

Following from the success of ROS, and redesigned for production environments, the second generation of ROS, ROS 2, is rapidly becoming a standard in the robotics industry. Built upon the Data Distribution Service (DDS \cite{pardo2003omg}) as its default communication middleware (used in safety-critical scenarios), ROS 2 aspiration was to offer a more robust framework that could be integrated into products and industry. Figure \ref{fig:robotics_stack} depicts the ROS 2 robotics stack, relative to the OSI model. From a technical viewpoint, it provides three main components:

\begin{enumerate}
    \item \textbf{ROS Plumbing}: 
    ROS provides an efficient message-passing infrastructure facilitating seamless communication between different robotic system components. This includes topics for asynchronous message exchange, services for synchronous requests, and actions for handling long-duration tasks. ROS's process management, or node management, ensures efficient handling of various processes, promoting a modular and scalable design. Additionally, ROS device drivers are integral for interfacing with the robot's hardware, ensuring effective control and data acquisition from sensors and actuators.

    \item \textbf{ROS Tools}: 
    ROS boasts a suite of visualization tools like RViz and rqt, essential for real-time monitoring and understanding of a robot's operation. CLI tools such as rostopic, rosnode, and rosservice offer convenient system interaction and control, crucial for development and troubleshooting. The simulation environment, provided by tools like Gazebo \cite{Koenig04designand}, enables safe and cost-effective testing in varied scenarios. Moreover, ROS's data logging capabilities, particularly with rosbag, are vital for data analysis, debugging, and scenario replication.

    \item \textbf{ROS Capabilities}: 
    ROS's control capabilities ensure precise robot movements and actions, encompassing real-time control loops and motion planning for complex task execution. Its perception capabilities, processing data from various sensors, are critical for tasks like navigation and environment mapping. ROS's manipulation features enable physical interaction with the environment, essential in applications such as assembly and logistics. Furthermore, coordinate transformations, managed by the tf library, maintain spatial relationships crucial for path planning and object manipulation.
\end{enumerate}


ROS 2 stands as a foundational element in the realm of modern robotics, playing a pivotal role in this thesis focused on offensive robot cybersecurity. As the \emph{de facto} standard for modern robotic software development, ROS offers a comprehensive framework that not only streamlines the creation of complex robotic applications but also inherently shapes the cybersecurity landscape of robotic systems. In the context of offensive cybersecurity, the ubiquity and modular architecture of ROS present both opportunities and challenges. Its widespread adoption makes it a critical target for cybersecurity research, as vulnerabilities within ROS (including any of its higher level abstraction, or underlayers) could potentially impact a vast array of robotic systems. This thesis, therefore, places significant emphasis on ROS, examining its architectures, communication protocols, and tools through the lens of offensive cybersecurity. By doing so, it aims to uncover and address potential security weaknesses, contributing to the robustness and resilience of robotics systems developed with ROS. This exploration not only aids in fortifying ROS-based robots against cyber threats but also provides valuable insights into the broader field of robot cybersecurity.

\section{Research Objectives and Approach}
\label{sec:research_objectives}

Security is not a product, but a process that needs to be continuously assessed in a periodic manner, as systems evolve and new cyber-threats are discovered. This becomes specially relevant with the increasing complexity of such systems \cite{bozic2017planning}. Current robotic systems are of high complexity, a condition that in most cases leads to wide attack surfaces and a variety of potential attack vectors which makes difficult the use of traditional approaches. Altogether, this leads to the following research question: given the increasing complexity of robotic systems, how can we protect them from cyber-threats? Inspired by the popular adage \emph{"the best defense is a good offense"}, the present work attempts to answer this question by focusing on offensive practices yet, given the time-consuming requirements of traditional offensive security approaches \cite{ghanem2018reinforcement, niculaelearning2018, schwartz2019autonomous}, the answer might not be straightforward and require further automation.\\

\begin{hypothesis}
\label{hypothesis:offensive_approach}
The inherent complexity of robotic systems, which leads to wide attack surfaces and a variety of potential attack vectors, will be more effectively tackled with an offensive security approach rather than a defensive one.
\end{hypothesis}

\noindent Scalability is one of the core problems of offensive security assessments (red teaming, penetration testing, etc.) as of now \cite{applebaum2016intelligent, applebaum2017analysis}. This thesis dives into robot cyber security and study how traditional and modern offensive methods apply to robotics, delivering a new angle to protect such systems. To this research question, the investigation herein proposed reviews and creates new mechanisms and techniques that leverage Machine Learning and Game Theory allowing for the full automation of offensive security assessments.

\begin{hypothesis}
\label{hypothesis:security_scalability}
An autonomous, end-to-end offensive security assessment algorithm that leverages Machine Learning and Game Theory will address scalability problems in offensive security.
\end{hypothesis}

This research also argues and challenges the underlying ethics of the security community when applied to robotics. Particularly, \emph{de facto} standards and conventions on security outcomes and disclosures are assessed and challenged, with a critical non-vendor-centered attitude. Ultimately, given the novelty of the topic itself, a relevant number of contributions are expected in the form of blueprints that will empower future research.\\
\newline
In the remainder of this chapter, the research objectives and approach are discussed in more detail. The overall research framework is introduced in Section \ref{sec:research_objectives}, followed by an explanation of each one of the building blocks. \\


\begin{mybox}
{\textbf{\textit{Security, cybersecurity and cyber security}}}
Within this thesis, the terms security, cybersecurity and cyber security are used interchangeably.
\end{mybox}

Although security has gained attention over the past years, our computational systems are increasingly vulnerable against cyber attacks. This includes of course all robots. As indicated by Bozic \cite{bozic2017planning}, one reason behind is definitely the system’s complexity, comprising interacting heterogenous systems. Complexity is the enemy of security. It is very difficult to assure no vulnerabilities in a system that an attacker could exploit. In robotics, the reality is even more concerning, as there is no culture of security\footnote{Early robot languages such as ABB's \emph{Rapid}, KUKA's \emph{KRL} or Fanuc's \emph{Karel} did not consider security primitives. Similarly, later robotic frameworks born from academic efforts (e.g. ROS\cite{quigley2009ros}) decided to ignore and forward security to third parties.}. As studied in previous work \cite{kirschgens2018robot, mayoralrvd}, except few vendors and selected actions, no real concern has been shown. Offensive security methods (such as pentesting or red teaming) play an important part in the security lifecycle. In these engagements, as introduced by Applebaum \cite{applebaum2017analysis}, (offensive) security teams try to break into an organization’s assets, identifying vulnerabilities along the way. Red teams take this concept even further, trying to fully emulate what real adversaries do: instead of just compromising the network and identifying vulnerabilities, they have a larger goal that requires significant post-compromise work. Offensive security methods are necessary to understand attackers' behavior, to train defensive mechanisms and ultimately, to help protect existing systems by discovering flaws first. The latter is specially important in highly complex systems, such as robots.\\
\newline
In an attempt to raise awareness and create some base formal studies and tools, this research objective is to study how offensive cyber security methods applied to robotics allow to protect such systems. In particular, this thesis aims to study how traditional methods in offensive security are applied to robotics. Such methods include penetration testing, red teaming, automated attack planning or advanced persistent threats amongst others. The expected outcome is twofold: a) to validate that these practices are applicable to robotics and b) that the inherent complexity of robotics systems, often systems of system at multiple levels, make the manual application of these methods unfeasible in most robotic setups. In addition, this study covers ethics and defensive security mechanisms that protect systems, networks and operations. Such is necessary for implementing offensive methods and understanding the underlying ethical and legal boundaries. Ultimately, this study looks into how to automate and optimize offensive cyber security attacks targeting robots. Preliminary ideas on this direction include the use and extension of prior work (as covered in \ref{aicornerstones}) on leveraging Machine Learning for the semi-automation of offensive security assessment. 
Such machine learning techniques could be trained using popular and widely available Capture The Flag (CTF) environments adapted for robotics. In addition, to optimize the use of resources and define proper learning and exploitation strategies, Game Theoretic approaches are  considered.\\
\newline
For coherence, this proposal has been organized in a framework consisting on four building blocks. As depicted in Figure \ref{fig:thesisblocks}, from bottom to top, (0) Ethics, (1) Blueprints for researching offensive security for robots, as well as robotics in general, (2) Cyber-attacking robots and (3) Cyber-attacking robots with Machine Learning and Game Theory. Generally, upper layers require of primitives from the underlying ones. The following subsections briefly describe each block.

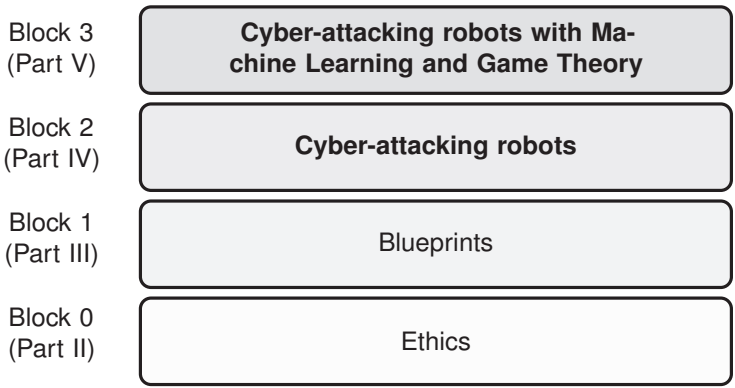
\begin{figure}[!h]
    \centering
    \begin{tikzpicture}[
    	scale=0.95,
    	start chain=1 going below, 
    	start chain=2 going right,
    	node distance=1mm,
    	desc/.style={
    		scale=0.95,
    		on chain=2,
    		rectangle,
    		rounded corners,
    		draw=black, 
    		very thick,
    		text centered,
    		text width=8cm,
    		minimum height=12mm
    		},
    	it/.style={
    	},
    	level/.style={
    		scale=0.95,
    		on chain=1,
    		minimum height=12mm,
    		text width=2cm,
    		text centered
    	},
    	level0/.style={fill=black!1},
    	level1/.style={fill=black!5},
    	level2/.style={fill=black!8},
    	level3/.style={fill=black!12},
    	every node/.style={font=\sffamily}
    ]
    
    \node [level] (Level 3) {Block 3 (Part \ref{part:automation})};
    \node [level] (Level 2) {Block 2 (Part \ref{part:attacking_robots})};
    \node [level] (Level 1) {Block 1 (Part \ref{part:blueprints})};
    \node [level] (Level 0) {Block 0 (Part \ref{part:ethics})};
    
    \chainin (Level 3); 
    \node [level3, desc, it] (ML) {\textbf{Cyber-attacking robots with Machine Learning and Game Theory}};
    \node [level2, desc, it, continue chain=going below] (attack) {\textbf{Cyber-attacking robots}};
    \node [level1, desc] (blueprints) {Blueprints};
    \node [level0, desc] (ethics) {Ethics};
    \end{tikzpicture}
    
    \caption{Thesis building blocks}
    \label{fig:thesisblocks}
\end{figure}


\subsection{Ethics}

This block treated in Part \ref{part:ethics} aims to study the ethical landscape of cybersecurity in robotics. Particularly, it aims to build some basic knowledge on the ethically accepted situation for robotics in Europe, in main areas of application. Beyond analysing existing vulnerabilities and manufacturer responses to them, this layer also studies robotic users' perception of cybersecurity by conducting various surveys in popular robotic communities.  In addition, this block challenges some of the \emph{de facto} standards in security when applied to robotics and discuss them with a critical attitude. Including topics such as disclosure policies, the requirement from vendors to freely notify research results and vulnerabilities (so called \emph{responsible disclosure}) or even, the relevance of disclosing flaws at all, inspired by previous criticism \cite{rescorla2005finding}. This block is discussed in more detail in Chapter \ref{chapter:landscape}. 

\subsection{Blueprints}

Building upon the previous ethical and legal baselines, this block is covered in Part \ref{part:blueprints} and develops the tools, libraries and frameworks necessary to perform robotics security research and offensive attacks in robots and their components. Each one of these research elements will be considered as "\emph{blueprints}".
Often considered part of a research continuum, the blueprints produced in this phase include:
\begin{itemize}
    
    \item Building upon the DevSecOps in Robotics principles researched in the non peer-reviewed published article \cite{mayoral2020devsecops},  \textbf{SROS2} proposes a series of developer tools, meant to be usable and that facilitate adding security capabilities to ROS 2 computational graphs (Chapter \ref{chapter:ros_security}). On top of SROS2, we also discuss the shortcomings of the DDS Security extensions \cite{ddssecurity11} which ROS 2 embraces and how additional layers of security via VPN could help mitigate them.
    
    \item Security measures such as those proposed by SROS2 in Chapter \ref{chapter:ros_security} add a computational overhead to robotic systems. In order to accurately estimate the performance impact, \textbf{RobotPerf} proposes an open-source and community-driven benchmarking tool designed to assess the performance of robotic computing systems in a standardized, architecture-neutral, and reproducible way, including the security associated aspects of them (Chapter \ref{chapter:benchmarking}).
    
    \item Chapter \ref{chapter:benchmarking} elucidates the computational overhead of integrating security into robotic communications, highlighting the resultant time bottlenecks that diminish system capabilities. To address this, \textbf{RobotCore} is introduced. It aims to empower the development of hardware accelerators that expedite security computations in robotics, thereby enabling the incorporation of security in robotic interactions without compromising capabilities. RobotCore introduces a framework and an \emph{open architecture for hardware acceleration} that extends ROS to easily develop hardware accelerators in an accelerator-agnostic manner, across GPUs and/or FPGAs  (Chapter \ref{chapter:accelerators}).
\end{itemize}


\subsection{Cyber-attacking robots}
Using previous blueprint security primitives, tools developed (such as Alurity \cite{mayoral2020alurity} or ExploitFlow \cite{mayoral2023exploitflow}) and know-how, research produced within this block aims at building basic attacks on robots and robot components targeting both, software and hardware. These attacks' objectives are threefold: first, they aim to contribute building some basic research ground on offensive cyber security robot practices. Second, these attacks will help raise the awareness of the current practices most vendors are applying. Third and ultimately, these actions help building expertise and develop an intuition on common offensive cyber-behavior, which will later be necessary when automating attacks using Machine Learning and Game Theory.

In Part \ref{part:attacking_robots}, we present our contributions while cyber-attacking robots. Chapter \ref{chapter:attacking_software} cares about attacking robot software and investigates the Data Distribution Service (DDS) standard, crucial for robotic operation, uncovering numerous vulnerabilities as detailed in \cite{dds14}. We highlight the discovery of publicly accessible DDS implementations, stressing the importance of persistent security evaluation and robust integration practices. Chapter \ref{chapter:attacking_hardware} explores attacking robot hardware through teardowns, emphasizing the significance of such methods in identifying security flaws. We also address the effects of planned obsolescence in robotics, advocating for a `Right to Repair' stance. The last chapter, \ref{chapter:attacking_industry} focuses on attacking robots in industry and introduces Akerbeltz ransomware, the first instance of ransomware targeting industrial collaborative robots. Akerbeltz demonstrates its impact on industrial robots by encrypting and locking robotic systems, which underscores the critical need for heightened security against potential cyber-attacks.

\subsection{Automated offensive robot cybersecurity}

Bridging the gap between the intricate world of robotic security vulnerabilities and proactive countermeasures, this final block shifts the focus from identification to automated action. It lays the groundwork for systems designed to anticipate and combat threats autonomously, fortifying defenses in an ever-evolving digital battleground where automation play an increasingly pivotal role. This transition encapsulates the move from understanding vulnerabilities and its mitigations (covered in parts \ref{part:blueprints} and \ref{part:attacking_robots}), to implementing advanced strategies for robust and resilient cybersecurity cognitive engines. Constructing upon all previous blocks, in Part \ref{part:automation} we attempt to design and deploy an autonomous system for cyber-attacking robots  using Machine Learning and Game Theory.

Building on the strategic insights provided by game theory, we advance into the domain of proactive cyber defense and offense within automated systems in Chapter \ref{chapter:game_theory}. This first chapter in Part \ref{part:automation} tackles the challenge of  \emph{Determining the attack target with Game Theory}, equipping us with a strategic toolset to predict the most vulnerable points in our robot threat landscape and to fortify them accordingly. This chapter describes how to prioritize our defensive efforts and to anticipate the moves of our adversaries, aligning our future actions with the most likely paths they might traverse in the adversary's illicit endeavors.

Following this strategic planning phase, \emph{Guiding attacks with Machine Learning} presented in Chapter \ref{chapter:machine_learning}, immerses us into the dynamic process of directing our cyber agents. Here, we leverage the advanced capabilities of modern Machine Learning, particularly the generative power of Large Language Models, to automate the decision-making processes in cyber-attacks. These models act as multi-faceted instruments, capable of parsing vast amounts of data, reasoning through complex scenarios, and generating adaptive attack and defense strategies. With the prototype \tool{}, we have a glimpse into a future where cybersecurity operations are autonomous, intelligent, and ever-evolving, reflecting the intricate nature of the digital battleground where they operate.

Part \ref{part:automation} as a whole forms a comprehensive narrative that transitions from a state of awareness and understanding of vulnerabilities to an active, autonomous and intelligent stance against them. They pave the way for automated systems that do not merely react to threats but anticipate and adapt to them, marking a significant leap in the field of offensive robot cybersecurity.

    \newgeometry{textheight=19cm} 

\chapter{Background context}
\label{chapter:background}


\section{Biographical cornerstones in robot cybersecurity}
\startchronology[startyear=1960,stopyear=2017,height=0.7ex]

\chronoevent[textwidth=3cm]{1962}{First installation of a cyber-physical system in a manufacturing plant \cite{historyofrobotics}}

\chronoevent[textwidth=5cm,markdepth=80pt]{1979}{The first time an industrial robot killed a human was in 1979 \cite{firstkiller}}

\chronoevent[textwidth=3cm,markdepth=-20pt]{2006}{Toward a unified security-safety model \cite{1673343}}

\chronoevent[textwidth=6cm, markdepth=140pt]{10/2016}{Dieber et al, Application-level security for ROS-based applications  \cite{ApplicationSecROS}}

\chronoevent[textwidth=6cm, markdepth=60pt]{06/2016}{Lera et al., Ciberseguridad en robots autónomos: Análisis y evaluación multiplataforma del bastionado ROS \cite{lera2016ciberseguridad}}

\chronoevent[textwidth=7cm, markdepth=-80pt]{11/2016}{White et al., SROS: Securing ROS over the wire, in the graph, and through the kernel \cite{white2016sros}}

\chronoevent[textwidth=4.5cm]{2009}{ROS: an open-source Robot Operating System \cite{quigley2009ros}}

\stopchronology


\restoregeometry

Arguably, the first installation of a cyber-physical system in a manufacturing plant was back in 1962 \cite{historyofrobotics}. The first human death caused by a robotic system is traced back to 1979 \cite{firstkiller} and the causes were safety-related according to the reports. From this point on, a series of actions involving agencies and corporations triggered to protect humans and environments from this machines, leading into safety standards.\\
\newline
Security however hasn't started being addressed scientifically in robotics until recently. One of the first peer-reviewed published articles on the topic goes back to 2016 \cite{lera2016ciberseguridad} which already warns about the security dangers of the Robot Operating System (ROS) \cite{quigley2009ros}. Following this first assessment, the same group in Spain authored a series of articles touching into robot security \cite{lera2016cybersecurity, lera2017cybersecurity, guerrero2017empirical, balsa2017cybersecurity, rodriguez2018message}. Contemporary to Lera's work, White et al. propose an addition to ROS API to support modern cryptography and security measures \cite{white2016sros}. This work led by White (UC San Diego, Spain) and Caiazza (Ca' Foscari University, Italy) delivered a series of articles \cite{ caiazza2017security, white2018procedurally, white2019sros1, caiazza2019enhancing, white2019network, white2019black} which focused on proposing defensive blueprints for robot cybersecurity. Starting also in 2016 \cite{ApplicationSecROS}, Dieber (Joanneum Research, Austria) led a series of publications that researched robot cybersecurity focusing around ROS \cite{Dieber:2017:SRO:3165321.3165569, dieber2017safety, SecurecomROS, taurer2018secure, dieber2019security}.\\
\newline
A careful review of the prior art in the last paragraph indicates that there're three groups, in different points of the world, that almost simultaneously started working in robot cybersecurity by tackling the lack of security primitives and proposing studies and blueprints for securitization. Most of this existing work focuses on defensive approaches and mechanisms. Very few publications touched into offensive approaches to secure robotic systems being \cite{dieber2020penetration} one of the first. Some pioneering work on the offensive side focused onto generating primary concerns and social awareness on robot security faults \cite{hackingbeforeskynet, hackingbeforeskynet2} but from this literature review, very few instances of offensive security have been seen so far in robotics.\\
\newline
The following sections shed some additional light into the foundations of this research.





\section{Robotics software quality, safety and security}

\textbf{Quality} (Quality Assurance or QA for short) and \textbf{Security} are often misunderstood when it comes to software. Ivers argues \cite{ivers_2017} that quality "essentially means that the software will execute according to its design and purpose" while "security means that the software will not put data or computing systems at risk of unauthorized access". Within \cite{ivers_2017} one relevant question that arises is whether the quality problems are also security issues or vice versa. Ivers indicates that quality bugs can turn into security ones provided they're exploitable, and addresses the question by remarking that quality and security are critical components to a broader notion: software integrity. Depicted in Figure \ref{figure:integrity}, software integrity can be represented the union of both software security and quality (\textbf{Software Integrity} = Security $\cup$ Quality).

\def\firstcircle{(0,0) circle (1.5cm)}
\def\secondcircle{(0:2.3cm) circle (1.5cm)}

\colorlet{circle edge}{black!80}
\colorlet{circle area}{black!20}
\colorlet{circle area2}{black!5}

    \tikzset{
 filled/.style = {fill=circle area2, draw=circle edge, thick},
outline/.style = {draw=circle edge, thick},
      F/.style = {draw, inner sep=7mm, fit=(current bounding box),
                  node contents={}}
            }

\begin{figure}[h!]
    \centering
    \begin{minipage}[t]{0.75\linewidth} 
        \centering   
        \begin{tikzpicture}
            \draw[filled] \firstcircle node {\small Security}
                          \secondcircle node {\small Quality};
            \node[anchor=south] at (current bounding box.north) {};
        \end{tikzpicture}
    \caption{Software integrity as the union between security and quality}
    \label{figure:integrity}
    \end{minipage}
\end{figure}
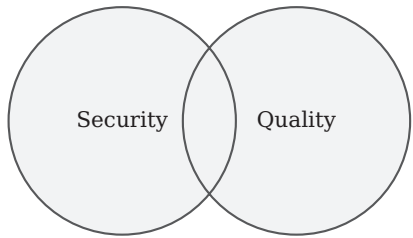

\noindent Coming from the same group, Vamosi \cite{vamosi_2017} argues that "quality code may not always be secure, but secure code must always be quality code". This somehow conflicts with the previous view and leads one to think that secure software is a subset of quality. The author of this thesis rejects this view and argues instead that Quality and Security remain two separate properties of software that may intersect on certain aspects (e.g. testing) as depicted in Figure \ref{fig:security}. \\
While the target of this thesis is Security, Quality is also studied given its intersection. Often, both secure and quality code share several requirements and mechanisms to assess them. This includes testing approaches such as static code testing, dynamic testing, fuzz testing or software component analysis (SCA) among others.\\


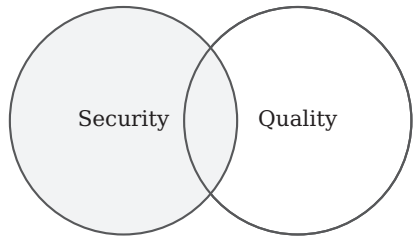
\begin{figure}[h!]
    \centering
    \begin{minipage}[t]{0.75\linewidth} 
        \centering   
        \begin{tikzpicture}
        \begin{scope}
                \clip \firstcircle;
                \fill[filled] \secondcircle;
        \end{scope}
        \draw[outline] \secondcircle node {\small Quality};
        \draw[filled] \firstcircle node {\small Security};
        \draw[outline] \secondcircle node {};
        \node[anchor=south] at (current bounding box.north) {};
        \end{tikzpicture}
        \caption{Security as the target of this thesis. 
        }
        \label{fig:security}
    \end{minipage}
\end{figure}



\def\itcircle{(0,0) circle (1.5cm)}
\def\otcircle{(0:2.3cm) circle (1.5cm)}
\def\roboticscircle{(2.3cm:2.3cm) circle (1.5cm)}
\def\iotcircle{(1.3cm:1.3cm) circle (1cm)}
\def\iotcircledos{(1.3cm:1.3cm) circle (3cm)}

\colorlet{circle edge}{black!80}
\colorlet{circle target}{black!20}
\colorlet{circle area}{black!10}
\colorlet{circle area2}{black!5}

    \tikzset{
 target/.style = {fill=circle target, draw=circle edge, thick},
 filled/.style = {fill=circle area2, draw=circle edge, thick},
outline/.style = {draw=circle edge, thick},
      F/.style = {draw, inner sep=7mm, fit=(current bounding box),
                  node contents={}}
            }


\noindent In robotics there is a clear separation between Security and Quality that is best understood with scenarios involving robotic software components. For example, if one was building an industrial Autonomous Guided Vehicle (AGV) or a self-driving car, often, she/he would need to comply with coding standards (e.g. MISRA \cite{ward2006misra} for developing safety-critical systems).  The same system's communications, however, regardless of its compliance with the coding standards, might rely on a channel that does not provide encryption or authentication and is thereby subject to eavesdropping and man-in-the-middle attacks. Security would be a strong driver in here and as remarked by Vamosi \cite{ivers_2017}, "neither security nor quality would be mutually exclusive, there will be elements of both".\\
\newline
Quality in robotics, still on its early stages \cite{pichler2019can}, is often viewed as a pre-condition for \textbf{Safety}-critical systems. Similarly, as argued by several, safety can't be guaranteed without security \cite{goertzel2009software, bagnara2017misra}. Coding standards such as MISRA C have been extended \cite{misra2016amendment1, misra2016addendum2} to become the C coding standard of choice for the automotive industry and for all industries developing embedded systems that are safety-critical and/or security-critical \cite{bagnara2017misra}. As introduced by ISO/IEC TS 17961:2013 "in practice, security-critical and safety-critical code have the same requirements". This statement is somehow supported by Goertzel \cite{goertzel2009software} but emphasized the importance of software remaining dependable under extraordinary conditions and the interconnection between safety and security in software. This same argument was later extended by Bagnara \cite{bagnara2017misra} who acknowledges that having embedded systems non-isolated anymore plays a key role in the relationship between safety and security. According to Bagnara, "while safety and security are distinct concepts, when it comes to connected software" (non-isolated)  "not having one implies not having the other", referring to integrity.\\

\noindent In the opinion of the author of this dissertation, coding standards such as MISRA or ISO/IEC TS 17961:2013 for safety-critical and security-critical software components do not guarantee that the final robotic system will be secure and thereby, safe. As illustrated in the example above, robotics involves a relevant degree of system integration and inter-connectivity (non-isolated embedded systems connected together internally and potentially, externally as well). As such, both secure and ultimately safe robotics systems do not only need to ensure quality by complying against coding standards but also guarantee that they aren't exploitable by malicious attackers. \\
\newline
In the traditional view of \emph{system security}, safety in often understood as "nothing bad happens naturally" while security intuitively indicates that "nothing bad happens intentionally". Acknowledging the acceptance of this view in the security community, this thesis puts special focus in the context of robotics 
To further understand terminology and prior art in a robotics context, Table \ref{table:terms} presents a summary of the concepts discussed with their interpretation applied to robotics and the corresponding sources used:\\

\begingroup
\hypersetup{linkcolor=lightgray, pdfborderstyle={/S/U/W 0}} 

\setlength{\tabcolsep}{20pt}
\renewcommand{\arraystretch}{1.5}
\begin{table}[h!]
    \centering
    \scalebox{0.75}{
    \begin{tabular}{ |c|m{8cm}|m{1.8cm}| } 
        \hline
        
        \rowcolor{black}
        \color{white}\textbf{Concept} & \color{white}\textbf{Interpretation} & \color{white}\textbf{Reference/s} \\
        \hline
        Safety & Safety cares about the possible damage a robot may cause in its environment. Commonly used taxonomies define it as the union of integrity and the absence of hazards ($\textrm{Safety} = \textrm{Integrity} + \textrm{Absence of catastrophic consequences}$) & \cite{kirschgens2018robot,bagnara2017misra,goertzel2009software} \\ 
        
        \rowcolor{black!5} 
        Security & Security aims at ensuring that the environment does not disturb the robot operation, also understood as  that the robot will not put its data, actuators or computing systems at risk of unauthorized access. This is often summarized as $\textrm{Security} = \textrm{Confidentiality} + \textrm{Integrity} + \textrm{Availability}$.  & \cite{kirschgens2018robot,bagnara2017misra,goertzel2009software,ivers_2017} \\ 
        
        Quality & Quality means that the robot's software will execute according to its design and purpose & \cite{ivers_2017}\\
        
        \rowcolor{black!5} 
        Integrity & Integrity can be described as the absence of improper (i.e., out-of-spec) system (or data) alterations under normal and exceptional conditions & \cite{bagnara2017misra}\\
        \hline
    \end{tabular}}
    \caption{Summary of security-related terms}
    \label{table:terms}
\end{table}

\endgroup


\noindent Security, as understood in Table \ref{table:terms} shares Integrity with Safety. As discussed in \cite{bagnara2017misra, goertzel2009software}, "the only thing that distinguishes the role of integrity in safety and security is the notion of \emph{exceptional condition}. This reflects the fact that exception conditions are perceived as accidental (safety hazards) or intentional (security threats)". The later, security threats, are always connected to vulnerabilities. A vulnerability is a mistake in software or hardware that can be directly used by an arbitrary malicious actor or actress to gain access to a system or network, operating it into an undesirable manner \cite{Pfleeger:2002:SC:579149}. In robotics, security flaws such as vulnerabilities are of special relevance given the physical connection to the world that these systems imply. As discussed in \cite{kirschgens2018robot}, "\emph{Safety cares about the possible damage a robot may cause in its environment, whilst security aims at ensuring that the environment does not disturb the robot operation.  Safety and security are connected matters. A security-first approach is now considered as a prerequisite to ensure safe operations}". Figure \ref{fig:security_safety_quality} depicts the concepts of Safety, Quality and Security representing their relationships. In particular, safety as a super-set of security and quality.

\def\firstcircle{(0,0) circle (1cm)}
\def\secondcircle{(0:1.6cm) circle (1cm)}
\def\thirdcircle{(0:0.8cm) circle (1.8cm)}

\colorlet{circle edge}{black!80}
\colorlet{circle area}{black!20}

\begin{figure}[h!]
    \centering
    \begin{minipage}[t]{0.75\linewidth} 
        \centering   
        \begin{tikzpicture}
        \begin{scope}
                \clip \firstcircle;
                \fill[filled] \secondcircle;
        \end{scope}
        \draw[outline] \secondcircle node {\footnotesize Quality};
        \draw[filled] \firstcircle node {\footnotesize Security};
        \draw[outline] \secondcircle node {};
        \draw[outline] \thirdcircle node {};
        \node[anchor=south] at (current bounding box.north) {Safety};
        \end{tikzpicture}
        \caption{
        Safety, quality and security relationship
        }
        \label{fig:security_safety_quality}
    \end{minipage}
\end{figure}
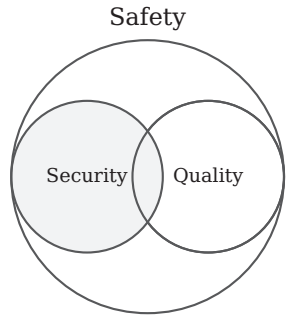

This thesis focuses on security, however, its relationship with quality and safety must be noted.  Security intersects quality in the sense that some methods are shared between both (e.g. testing). Moreover, as discussed, a safe system demands first security and quality.\\
\newline
\noindent Robot security vulnerabilities are potential attack points in robotic systems that can lead not only to considerable losses of data but also to safety incidents involving humans. Some authors \cite{zheng2011ivda} claim that unresolved vulnerabilities are the main cause of loss in cyber incidents. The mitigation and patching of vulnerabilities has been an active area of research \cite{ma2001sharing, ALHAZMI2007219, Shin2011Vulnerabilities, Finifter2013BugBounty, McQueen2009Zeroday, Bilge:2012:BWK:2382196.2382284} in computer science and other technological domains. Unfortunately, even with robotics being an interdisciplinary field composed from a set of heterogeneous disciplines (including computer science), to the best of the knowledge of this thesis' author and his literature review, not much vulnerability mitigation research related to robotics has been presented so far. \\

\section{Cybersecurity Across Various Domains: IT, OT, IoT, and Robotics}

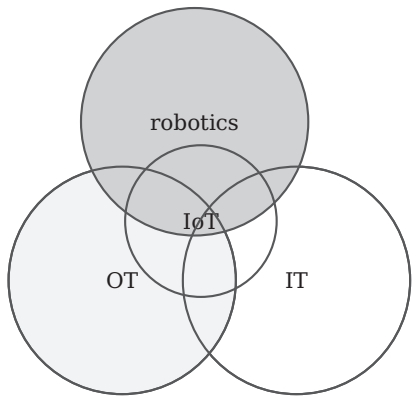
\begin{figure}[h!]
    \centering
    \begin{minipage}[t]{0.75\linewidth} 
        \centering   
        \begin{tikzpicture}
        \begin{scope}
                \clip \itcircle;
                \fill[filled] \otcircle;
        \end{scope}
        \draw[outline] \otcircle node {\small IT};
        \draw[filled] \itcircle node {\small OT};
        \draw[target] \roboticscircle node {\small robotics};
        \draw[outline] \iotcircle node {\small IoT};
        \draw[outline] \otcircle node {};
        \draw[outline] \itcircle node {};
        \end{tikzpicture}
        \caption{Relationship between various technology domains}
        \label{fig:technology_domains}
    \end{minipage}
\end{figure}

Cybersecurity is fundamentally about ensuring systems are protected from digital dangers and threats. This concept, while universally applicable, varies in interpretation and implementation across different technological domains such as Information Technology (IT), Operational Technology (OT), Internet of Things (IoT), and robotics. Each domain, due to its unique technological architecture, necessitates a distinct approach to cybersecurity.

\subsection{Definitions and Differentiations}
In the realm of digital security, various terms have evolved to denote domain-specific aspects:

\begin{description}
    \item[\textbf{Information Technology (IT):}] Involves the use of computers for data management and is typically employed organization-wide \cite{wikipedia2020information}.
    \item[\textbf{Operational Technology (OT):}] Focuses on managing industrial operations through monitoring and control of devices and processes \cite{tuvit2011industrial}.
    \item[\textbf{Industrial Control Systems (ICS):}] A crucial segment of OT that includes systems used to oversee industrial processes.
    \item[\textbf{Internet of Things (IoT):}] Represents the extension of internet connectivity to various sensors and devices.
    \item[\textbf{Industrial IoT (IIoT):}] The application of IoT technologies in industrial settings.
    \item[\textbf{Robotics:}] An interdisciplinary domain that integrates sensors, actuators, and computation to create autonomous systems.
\end{description}


The application of cybersecurity varies across IT, OT, IoT, IIoT, and robotics. Each technology, while subject to the overarching goal of operational security, requires a tailored approach. For instance, ICS, evolving from isolated systems to more IT-like structures, have increased vulnerability to cyber threats due to increased connectivity \cite{stouffer2011nist}. The confluence of IT and OT in IoT necessitates a unique set of security skills, distinct from those required in traditional IT or OT environments \cite{atlam2017integration}. Robotics, due to its inherent complexity and the integration of various systems, demands domain-specific knowledge and a specialized approach to cybersecurity.

\newcolumntype{L}[1]{>{\raggedright\arraybackslash}p{#1}}

\setlength{\tabcolsep}{20pt}
\renewcommand{\arraystretch}{1.5}
\begin{table}[h!]
    \centering
    \scalebox{0.55}{
    \begin{tabular}{ |L{3cm}|L{2cm}|L{3cm}|L{3cm}|L{3cm}| } 
        \hline
        \rowcolor{black}
        \color{white}\textbf{Security topic} & \color{white}\textbf{IT} & \color{white}\textbf{OT (ICS)} & \color{white}\textbf{I(I)oT} & \color{white}\textbf{Robotics} \\
        \hline
        \textbf{Antivirus} & Widely used, easily updated & Complicated, network detection and prevention solutions mostly & Complicated, technology fragmentation & Complicated, complex, few solutions, network monitoring insufficient \\        
        \rowcolor{black!5} 
        \textbf{Life cycle} & 3-5 years & 10-20 years & 5-10 years & 10+ years \\
        \textbf{Awareness} & Decent & Poor & Poor & Very poor \\
        \rowcolor{black!5} 
        \textbf{Patch management} & Often & Rare, manufacturer approval needed & Rare, end-user permission/action often needed & Very rare, complex setups \\
        \textbf{Change Management} & Regular and scheduled & Rare & Rare & Very rare, specialized technicians \\
        \rowcolor{black!5} 
        \textbf{Evaluation of log files} & Established practice & Unusual & Unusual & Non-established \\
        \textbf{Time dependency} & Delays accepted & Critical & Some delays accepted & Critical \\
        \rowcolor{black!5} 
        \textbf{Availability} & Not always, failures accepted & 24*7 & Some failures accepted & 24*7 available \\
        \textbf{Integrity} & Failures accepted & Critical & Some failures accepted & Critical \\
        \rowcolor{black!5} 
        \textbf{Confidentiality} & Critical & Relevant & Important & Important \\
        \textbf{Safety} & Not relevant & Relevant & Not relevant & Critical \\
        \rowcolor{black!5} 
        \textbf{Security tests} & Widespread & Rare and problematic & Rare & Mostly not present \\
        \textbf{Testing environment} & Available & Rarely available & Rarely available & Rare, difficult to reproduce \\
        \rowcolor{black!5} 
        \textbf{Determinism requirements} & Non-real-time, high delay/jitter acceptable & Hard real-time, high delay/jitter not acceptable & Often non-real-time, soft/firm real-time in some environments & Hard real-time for safety critical, firm/soft for others \\
        \hline
    \end{tabular}}
    \caption{Comparative analysis of cybersecurity across different technologies}
    \label{tab:cybersecurity_comparison}
\end{table}

Table \ref{tab:cybersecurity_comparison} presents a comparative analysis (extended from \cite{stouffer2011nist, tuvit2011industrial}) which reveals distinct requirements and challenges in implementing cybersecurity across these technologies. Notably, robotics stands out due to its critical real-time requirements and safety implications, underscoring the need for a comprehensive and dedicated cybersecurity strategy tailored to robotic systems. Such approach should emphasize the relevance of isochronous and deterministic communications, crucial in robotic systems. 

\subsection{Misunderstandings in Robot Security}
A prevalent issue in the robotics industry is the conflation of IT security with comprehensive robotics security. Many manufacturers, such as Mobile Industrial Robots (MiR) as depicted in Figure \ref{fig:security_misunderstanding}, focus predominantly on IT security, neglecting the broader spectrum of threats specific to robotics \cite{mayoralvilches2020compromised}. This oversight can lead to severe safety hazards, as compromised robots pose risks not just to data privacy but also to human safety and environmental integrity. Therefore, it is imperative for manufacturers to adopt a holistic view of security that encompasses all aspects unique to robotic systems.

\begin{figure}[ht]
    \centering
    \begin{subfigure}[b]{0.48\textwidth}
        \centering
        \includegraphics[width=\textwidth]{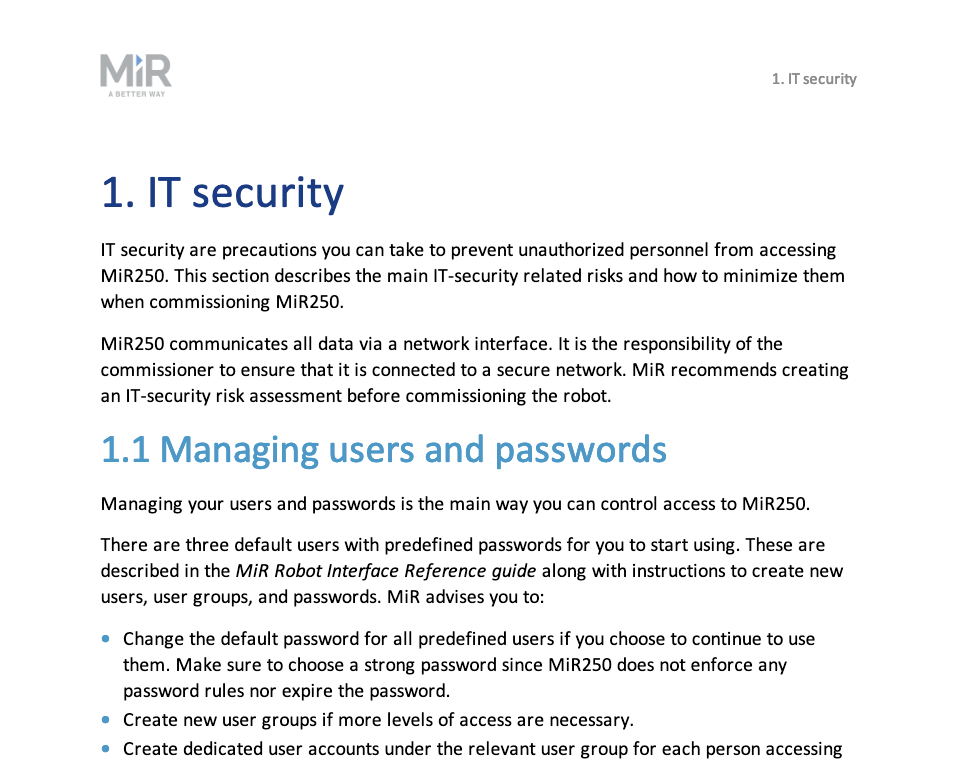}
    \end{subfigure}
    \hfill 
    \begin{subfigure}[b]{0.48\textwidth}
        \centering
        \includegraphics[width=\textwidth]{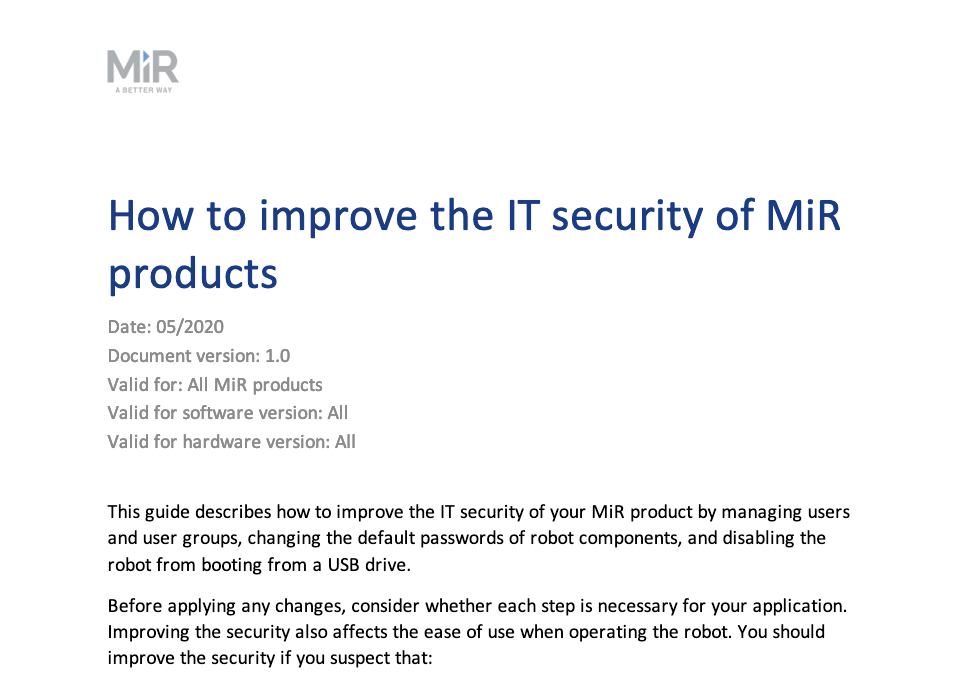}
    \end{subfigure}
    \caption{Captures from Mobile Industrial Robots (MiR) documentation indicating a purely IT security approach.}
    \label{fig:security_misunderstanding}
\end{figure}

Given the safety implications of autonomous systems, and given the security precondition of safety, the need for robust cybersecurity measures in robotics surpasses that in IT, OT, or IoT. The complexity and integration of systems in robotics demand a security approach that is both comprehensive and specific to the domain, with careful consideration for robot's real-time nature. This thesis advocates for the implementation of such dedicated security measures in the field of robotics.

\section{Machine Learning and Game Theoretic approaches to offensive cybersecurity}

Offensive security typically involves a malicious attacker performing a series of steps, often connected, and usually requiring a high degree of expertise or knowledge on the target system. Within offensive security practices, Penetration Testing (\emph{pentesting}) is an active method for assessing and evaluating the security of digital assets by planning, generating and executing all possible attacks that can exploit existing vulnerabilities. Figure \ref{figure:offensive_pipeline} pictures the pentesting pipeline for a robotics system inspired on prior work \cite{center2013apt1, assante2015industrial, bryant2017novel, strom2017finding}.\\

\tikzset{
  nonterminal/.style={
    rectangle,
    minimum size=6mm,
    very thick,
    draw=red!50!black!50,         
    top color=white,              
    bottom color=red!50!black!20, 
    font=\itshape
  },
  terminal/.style={
    rounded rectangle,
    minimum size=6mm,
    very thick,draw=black!50,
    top color=white,bottom color=black!20,
    font=\ttfamily},
  skip loop/.style={to path={-- ++(0,#1) -| (\tikztotarget)}}
}

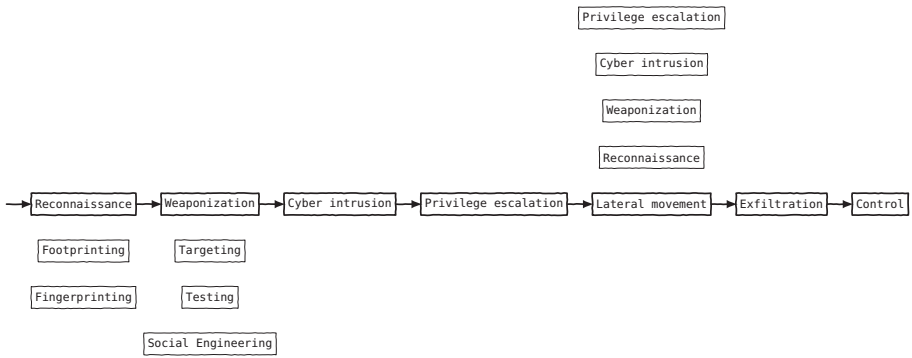
\begin{figure}[!h]
    \centering
    \scalebox{0.8}{
    \begin{tikzpicture}[
        scale=0.7,
        >=latex,thick,
        /pgf/every decoration/.style={/tikz/sharp corners},
        fuzzy/.style={decorate,
            decoration={random steps,segment length=0.5mm,amplitude=0.15pt}},
        minimum size=6mm,line join=round,line cap=round,
        terminal/.style={scale=0.7,rectangle,draw,fill=white,fuzzy,rounded corners=3mm},
        nonterminal/.style={scale=0.6,rectangle,draw,fill=white,fuzzy},
        node distance=4mm
      ]
    
        \ttfamily
        \begin{scope}[start chain,
                every node/.style={on chain},
                terminal/.append style={join=by {->,shorten >=-1pt,
                    fuzzy,decoration={post length=4pt}}},
                nonterminal/.append style={join=by {->,shorten >=-1pt,
                    fuzzy,decoration={post length=4pt}}},
                support/.style={coordinate,join=by fuzzy}
            ]
            \node [support]             (start)        {};
            \node [nonterminal](reco)                  {Reconnaissance};
            \node [nonterminal](weapo)                        {Weaponization};
            \node [nonterminal]                        {Cyber intrusion};
            \node [nonterminal]                           {Privilege escalation};
            \node [nonterminal](lateral)                           {Lateral movement};
            \node [nonterminal]                        {Exfiltration};
            \node [nonterminal]                        {Control};
            \node [coordinate] (end)        {};
        \end{scope}
        
        
        \node (footprinting)  [line width=0.1mm, nonterminal,below=of reco] {Footprinting};
        \node (fingerprinting)  [line width=0.1mm, nonterminal,below=of footprinting] {Fingerprinting};
        
        \node (targeting)  [line width=0.1mm, nonterminal,below=of weapo] {Targeting};
        \node (testing)  [line width=0.1mm, nonterminal,below=of targeting] {Testing};
        \node (social)  [line width=0.1mm, nonterminal,below=of testing] {Social Engineering};

        \node (reconnaissance_l)  [line width=0.1mm, nonterminal,above=of lateral] {Reconnaissance};
        \node (weapo_l)  [line width=0.1mm, nonterminal,above=of reconnaissance_l] {Weaponization};
        \node (intrusion_l)  [line width=0.1mm, nonterminal,above=of weapo_l] {Cyber intrusion};
        \node (Privilege escalation)  [line width=0.1mm, nonterminal,above=of intrusion_l] {Privilege escalation};

    \end{tikzpicture}}
    \caption{Attacker model for robotics. Adapted from \cite{center2013apt1, assante2015industrial, bryant2017novel, strom2017finding}
    }
    \label{figure:offensive_pipeline}
\end{figure}

\noindent As described by several authors \cite{ghanem2018reinforcement, niculaelearning2018, schwartz2019autonomous}, current pentesting practices are becoming repetitive, complex and resource consuming despite the use of automated tools. Morever, they generally require a high level of expertise in those systems targeted. Applebaum et al. \cite{applebaum2016intelligent, applebaum2017analysis} also underline the importance of pentesting in the security lifecycle and the shortcomings associated with manual execution. These challenges have made several researchers to turn into Machine Learning (ML) to try and automate the process of pentesting or other security assessments\footnote{Security assessments are often misunderstood and pentesting is used to group them all. For the context of this thesis, penetration testing assessment refers to the activity that aims to find as many vulnerabilities and configuration issues as possible in the time allotted, and exploiting those vulnerabilities to determine the risk of the vulnerability. Red team assessment aims to test the organization's detection and response capabilities against a targeted objective.}.\\
\newline
While several articles describe the growing popularity of the use of machine learning in defensive approaches for security,  \cite{niculaelearning2018, rege2018machine, vinayakumar2019deep} already report that very few sources cover the use of Machine Learning for offensive approaches to security. While reviewing literature, the author of this thesis found it extremely helpful to step back a few decades to review previous work that attempted to automate pentesting activities using Artificial Intelligence (AI) approaches. A complete and detailed walkthrough on these past efforts is beyond the scope of this thesis but nicely compiled at \cite{zhou2019nig}. For completeness and reader's conveniece, the most relevant events on this regard prior to the start of this thesis (2020) have been pictured in the diagram below:\\



\label{aicornerstones}
\startchronology[startyear=1999,stopyear=2020,height=0.5ex]

\chronoevent[textwidth=5cm,markdepth=30pt]{1999}{Attack tree model \cite{schneier1999attack}}

\chronoevent[textwidth=4cm,markdepth=-50pt]{2002}{Attack graph model \cite{sheyner2002automated}}

\chronoevent[textwidth=3cm,markdepth=100pt]{2011}{Personal vulnerability analysis \cite{roberts2011personalized}}

\chronoevent[textwidth=4cm, markdepth=40pt]{2013}{Plan Domain Definition Language (PDDL) \cite{obes2013attack}}

\chronoevent[textwidth=2.5cm, markdepth=-20pt]{2013}{Partially Observable Markov Decision Processes (POMDP), attack planning with uncertainty \cite{sarraute2013penetration}}

\chronoevent[textwidth=4cm, markdepth=-130pt]{2018}{Autonomous penetration testing using reinforcement learning\cite{schwartz2019autonomous}}

\chronoevent[textwidth=4cm, markdepth=80pt]{2019}{"Reinforcement Learning for Efficient Network Penetration Testing" \cite{ghanem2020reinforcement}}

\chronoevent[textwidth=4cm]{2018}{"Learning to hack" \cite{niculaelearning2018}}

\chronoevent[textwidth=2.5cm, markdepth=-10pt]{2019}{NIG-AP: a new method for automated penetration testing \cite{zhou2019nig}}

\stopchronology

\noindent Moving past AI and diving more contemporary ML approaches\footnote{The author of this thesis assumes the reader has some basic understanding of the differences and relationships between "artificial intelligence", "machine learning" and "deep learning".}, in 2018 Niculae \cite{niculaelearning2018} presents one of the first pieces available where ML, in particular, Reinforcement Learning (RL) is used to tackle the pentesting assessment task while modeling it as a game. Niculae evaluates different RL techniques, namely Q-Learning (QL), Extended Classifier Systems (XCS) or Deep Q-Networks (DQN) while comparing it with Human performance and multiple fixed-strategies  including Random, Greedy and a Finite State Machine. During his work and for his particular simplified model, Niculae reports building an agent that can learn to penetrate a network comparable in strength to a human, paving the way for future work in the application of RL in pentesting. 

Also in 2018 (though published later), Schwartz \cite{schwartz2019autonomous} argues that current approaches to automated pentesting have relied on methods which require a model of the exploit outcomes, however the cyber security landscape is rapidly changing as new software and attack vectors are developed which makes producing and maintaining up-to-date models a challenge. Instead and similar to Niculae, Schwartz advocates for the application of RL and evaluates QL and DQL obtaining positive results. Interestingly however, Schwartz remarks that the implemented algorithms in the simulated scenarios were only practical for smaller networks and numbers of actions and would not be able to scale to truly large networks. 

The most recent studies available on this domain are \cite{ghanem2020reinforcement} and \cite{zhou2019nig}. The first one explores the design and development of an ML-based pentesting system while supported by human experts. The article itself provides an interesting motivation for robotics, but does not provide a clear description of the underlying RL techniques, the reward function, or a detailed reasoning on how the state space and the action spaces of the pentesting activity were captured. Moreover, the authors conclude their paper highlighting that the main drawback of their contribution is the need of a high-caliber human expert supervising the initial training. Zhou et al \cite{zhou2019nig} present a new method for pentesting using RL. Particularly, they propose an attack planning algorithm based on RL to automatically discover attack paths without prior knowledge of the scenario network. Opposed to \cite{ghanem2020reinforcement}, this last article provides a good review of the pentesting foundations, underlying math formalizing pentesting as a Markov decision process and a detailed description of the assumptions made for representing the problem in reinforcement learning.
%

%
Despite the advances on ML, scalability remains a problem for automated pentesting across network layers, even when powered with modern techniques such as the ones described above. Scalability can be understood easily by switching across (attack) subjects and/or determining next steps or actions during a pentesting offensive effort.  It is unclear whether the current trend of ML research (mostly driven by increasingly complex RL constructs) can help lead towards an \emph{end-to-end ML-driven penetration testing}. The research conducted herein should tap into this problem and propose novel approaches.
Beyond the use of ML for addressing the scalability problem in pentesting, another interesting issue is selecting the right attack targets. More specifically, selecting which attack vectors (corresponding to individual targets) should be exploited first. When conducting an offensive pentesting exercise, commonly, multiple targets appear over the exercise. For each attack target, multiple attack vectors are often identified. Selecting which attack vector to prioritize first can be somewhat derived by using scoring metrics, like the Common Vulnerability Scoring System (CVSS) \cite{cvss2}. However, in practice, pentesters use this metric as yet one more indicator to build an intuition on \emph{where to attack}. Developing this intuition is complex and not fully understood. This thesis will also look at this problem by looking at Game Theory (GT). Computational game theory \cite{nguyen2016towards, elderman2016adversarial} can also help tackle the scalability problem and schedule resources while taking into account the importance of different targets, the responses of the adversaries to the security posture and the potential uncertainties in adversary payoffs and observations. The combination of both, machine learning and game theoretic approaches might eventually lead to results that empower autonomous and versatile red team assessments,  first envisioned by Applebaum \cite{applebaum2016intelligent} back in 2016.\\
\newline
A deeper study of the biographical cornerstones for AI in pentesting will be conducted in Chapter \ref{chapter:machine_learning}. To the best of this author's knowledge and extend of literature review, at the time of the start of this research (2020), no single article covered the use of ML and GT methods to secure robots via an offensive approach\footnote{Which means, employing a 'Hack-First' approach that is effective and scales by leveraging ML and GT}. This presents an interesting opportunity which is explored  throughout this research. 





    \part{Ethics}\label{part:ethics}


\chapter{Ethical Landscape in Robot Cybersecurity}
\label{chapter:landscape}

\epigraph{\textbf{Contributions}: \emph{Sole author and contributor of \cite{mayoral2022review}. All results are publicly available.}}{}

\section{Introduction}

This chapter, lying within the larger framework of ethical considerations in cybersecurity for robotics, delves into the complex landscape of this evolving field. The objective is multifaceted: to map out the current ethical standards within the robot cybersecurity domain, dissecting how these standards manifest in key application areas. This exploration goes beyond merely cataloging vulnerabilities and manufacturers' responses; it also seeks to understand the perceptions and attitudes of robotic users and vendors  toward cybersecurity. This understanding is garnered through surveys distributed across popular robotics communities, thus capturing a wide range of insights and sentiments.

An important aspect of this chapter is its critical examination of the established norms and standards in robot cybersecurity. It challenges and scrutinizes these norms, particularly focusing on the ethics of disclosure policies. This includes the expectations placed on vendors regarding the notification of research findings and vulnerabilities—a practice commonly known as \emph{responsible disclosure}. Furthermore, the chapter probes into the more controversial facets of this topic, such as the ethical implications of disclosing security flaws at all, drawing inspiration from previous critiques like those of Rescorla (2005) \cite{rescorla2005finding}.

To achieve a comprehensive understanding of the state of robot cybersecurity, this chapter employs a systematic review approach that unfolds in three distinct phases. Initially, it delves into the existing body of literature, painting a picture of the current cybersecurity landscape in robotics. Following this, the chapter presents findings from surveys conducted within various robotics groups and communities, offering a grassroots perspective on the state of cybersecurity in this field. The final phase of this exploration is anchored in empirical data, drawn from three years of proactive security research in robotics. This phase not only discusses identified vulnerabilities but also reflects on the offensive security exercises undertaken during this period.

The culmination of this chapter is a synthesis of insights and conclusions drawn from these diverse sources, including a discussion of the current ethical standards within the robot cybersecurity domain. It aims to provide actionable guidance on securing robotic systems, understanding their vulnerability to various attack vectors, and strategies to minimize their exposure to potential threats. This chapter thus serves as a critical piece in understanding the ethical underpinnings of cybersecurity in robotics, laying the groundwork for responsible and secure advancements in this dynamic and impactful field.

\section{Background}

As introduced before, arguably, the first installation of a cyber-physical system in a manufacturing plant dates back to 1962 \cite{historyofrobotics}. The first human death attributed to a robotic system occurred in 1979 \cite{firstkiller}, with safety-related issues cited as the cause. This incident sparked a series of actions by agencies and corporations aimed at protecting humans and environments from these machines, leading to the development of safety standards. However, security in robotics has only started to be addressed more recently. Following after McClean's (2013) early assessment, one of the initial published articles on the topic by Lera et al. (2016) warned about the security dangers of the Robot Operating System (ROS) \cite{lera2016ciberseguridad, quigley2009ros}. Subsequently, the same group in Spain authored a series of articles focusing on robot cybersecurity \cite{lera2016cybersecurity, lera2017cybersecurity, guerrero2017empirical, balsa2017cybersecurity, rodriguez2018message}. Around the same time, Dieber et al. (2016) led research into cybersecurity in robotics, proposing defensive blueprints for ROS-based robots \cite{Dieber:2017:SRO:3165321.3165569, dieber2017safety, SecurecomROS, taurer2018secure, dieber2019security}. Their work included enhancements to ROS APIs to incorporate modern cryptography and security measures. Concurrent with Dieber et al.'s work, White et al. (2016) also began producing a series of articles proposing defensive mechanisms for ROS \cite{caiazza2017security, white2018procedurally, white2019sros1, caiazza2019enhancing, white2019network, white2019black}.


A little over a year later, starting in 2018, more groups began to contribute to the field. Mayoral-Vilches et al. (2018) initiated a series of security research efforts aimed at defining offensive security blueprints and methodologies in robotics, resulting in various contributions \cite{vilches2018volatile, kirschgens2018robot, vilches2018aztarna, mayoral2020alurity, mayoral2020can, lacava2020current, mayoral2020devsecops, mayoral2019industrial}. Notably, this group released a framework for conducting security assessments in robotics \cite{vilches2018introducing}, a vulnerability scoring mechanism for robots \cite{vilches2018towards}, a CTF environment for robotics to train cybersecurity engineers \cite{mendia2018robotics}, and a robot-specific vulnerability database \cite{mayoralrvd}, among other contributions. In 2021, Zhu et al. published a comprehensive introduction to this emerging topic, aiming to foster a sub-community in robotics for cybersecurity \cite{zhu2021cybersecurity}. These contributions correspond to the bibliographical milestones in robot cybersecurity presented in Chapter \ref{chapter:background}.


A careful review of the prior art described in the preceding paragraphs leads to the following observation: based on the literature, robot cybersecurity is still a nascent field that requires further attention, tools, and educational material to train new engineers in security practices for robotics.


\section{Surveying security in robotics communities}

During a period of three years (2019 - 2021) various security surveys were conducted in top robotic communities through their corresponding conferences and forums. The following subsections discuss each one of them while attempting to draw some observations:

\subsection{Surveying the ROS community}

\boolfalse{contentsStyle}
\setHyperlinkStyle

Figure \ref{fig:survey:ros} presents a summarized result  of the survey conducted in the ROS community during a period of several months\oldfootnote{Security survey launched within the ROS Discourse community (\href{https://discourse.ros.org/t/call4participation-to-robot-security-survey/10290}{announcement}, \href{https://discourse.ros.org/t/call-for-participation-ii-to-survey-the-security-in-robotics/10811}{announcement 2}, \href{https://news.aliasrobotics.com/robot-security-survey-displays-first-results/}{preliminary results}).}. 
The survey received a total of 52 responses, which represented the small interest in security at the time. The largest groups of participants are depicted in Figure \ref{fig:ros:2}. The most represented group comes from Universities (30\%), followed by Software vendors (18\%) and Robot manufacturers (14\%)\footnote{Others comprise various subgroups, all with less representation than the ones mentioned}. The majority of the respondents have at last 2 years of experience with ROS and half of them at least 5 (\ref{fig:ros:3}), most coming from Europe (\ref{fig:ros:4}). Figure \ref{fig:ros:5} present data on security considerations. The data indicates that 73\% of the participants think that they have not invested enough to protect their robots from cyber-threats. Coincidentally, the same number of participants indicated that their organizations are open to invest however only 26\% acknowledge to actually have invested. This data leads to the following observation:

\begin{observation}
\label{observation:investment_gap}
There seems to be a gap between the expectations and the actual investment, which suggests that cybersecurity actions in robotics will grow in the future for the ROS community.
\end{observation}

When considering the mitigation strategies applied by respondents as depicted in Figure \ref{fig:ros:6}, it's important to highlight that most efforts concentrate on perimeter actions (i.e. firewalls, segmentation and segregation) whereas robot-specific defensive solutions are only applied in a 36\% of the cases. Similarly, network assessments and security audits are conducted only in one fourth of the cases (26\%) which conflicts with the \emph{de facto} security practices in other industries, wherein assessments are critical to evaluate the resilience of technology.

\begin{observation}
\label{observation:nomeasures}
The lack of robot-specific security measures (36\%) and offensive assessments (26\%) can be interpreted as an indicator of the maturity level of the technology when compared to other sectors (e.g. IT or OT) where these practices are common and specialized.
\end{observation}

\showkeysfalse

\begin{figure*}[p]
    \centering
    \begin{subfigure}[b]{0.45\textwidth}
        \centering
        \resizebox{0.8\textwidth}{!}{
            \begin{tikzpicture}   
                \pie[square, rotate=270, text=pin, color={lightgray!90, lightgray!70, lightgray!50, lightgray!30, lightgray!10}]  
                    {45/Industrial, 23/R\&D, 5/Professional, 14/Other,  13/Consumer}  
               
            \end{tikzpicture}  
        }
        \caption{\footnotesize Distribution of respondents by sector of activity. \emph{Sample size 52 respondents.}}
        \label{fig:ros:1}
    \end{subfigure}
    ~
    \begin{subfigure}[b]{0.45\textwidth}
        \centering
        \resizebox{0.6\textwidth}{!}{
            \begin{tikzpicture}   
                \pie[square, rotate=270, text=inside, color={lightgray!90, lightgray!70, lightgray!50, lightgray!30, lightgray!10}]  
                    {30/University,  18/SW vendor, 14/Robot vendor, 8/RTO, 30/Other}  
               
            \end{tikzpicture}  
        }
        \caption{\footnotesize Distribution within the robotics value chain. \emph{Sample size 52 respondents.}}
        \label{fig:ros:2}
    \end{subfigure}
    \\[5ex]
    \begin{subfigure}[b]{0.45\textwidth}
        \centering
        \resizebox{0.9\textwidth}{!}{
            \begin{tikzpicture}   
                \pie[rotate=270, text=pin, color={lightgray!90, lightgray!70, lightgray!50, lightgray!30, lightgray!10}]  
                    {21.2/Less than 2,  28.8/Between 2 and 5, 21.2/More than 10, 28.8/Between 5 and 10}
            \end{tikzpicture}  
        }
        \caption{\footnotesize Years of experience with ROS of each respondent. \emph{Sample size 52 respondents.}}
        \label{fig:ros:3}
    \end{subfigure}
    ~ 
    \begin{subfigure}[b]{0.45\textwidth}
        \resizebox{\textwidth}{!}{
            \begin{tikzpicture}
              \begin{axis}[
                xbar,
                y axis line style = { opacity = 0 },
                axis x line       = none,
                tickwidth         = 0pt,
                enlarge y limits  = 0.2,
                enlarge x limits  = 0.02,
                symbolic y coords = {America,Europe,Asia,Africa,Not responded},
                nodes near coords,
              ]
              \addplot[lightgray, fill=lightgray] coordinates { 
                (8,America) 
                (31,Europe) 
                (8,Asia) 
                (1,Africa) 
                (4,Not responded) 
             };
              \end{axis}
            \end{tikzpicture}
        }
        \caption{\footnotesize Geographical distribution of respondents. \emph{Sample size 52 respondents.}}
        \label{fig:ros:4}
    \end{subfigure}
    \\[5ex]
    \begin{subfigure}[b]{0.45\textwidth}
        \resizebox{\textwidth}{!}{
            \begin{tikzpicture}
              \begin{axis}[
                xbar,
                xmin=0,
                xmax=100,
                y axis line style = { opacity = 0 },
                axis x line       = none,
                tickwidth         = 0pt,
                enlarge y limits  = 0.2,
                enlarge x limits  = 0.02,
                symbolic y coords = {Open to invest,Have actually invested,Think have not invested enough,Identified cyber-weaknesses,Witnessed a cyber-attack},
                nodes near coords,
              ]
              \addplot[lightgray, fill=lightgray] coordinates { 
                (73,Open to invest) 
                (26,Have actually invested) 
                (73,Think have not invested enough) 
                (51,Identified cyber-weaknesses) 
                (9,Witnessed a cyber-attack) 
             };
              \legend{Percentage (\%)}
              \end{axis}
            \end{tikzpicture}
        }
        \caption{\footnotesize Percentage favouring each security consideration amongst respondents (in the context of robotics). \emph{Sample size 52 respondents.}}
        \label{fig:ros:5}
    \end{subfigure}
    ~
    \begin{subfigure}[b]{0.45\textwidth}
        \resizebox{\textwidth}{!}{
            \begin{tikzpicture}
              \begin{axis}[
                xbar,
                xmin=0,
                xmax=100,
                y axis line style = { opacity = 0 },
                axis x line       = none,
                tickwidth         = 0pt,
                enlarge y limits  = 0.2,
                enlarge x limits  = 0.02,
                symbolic y coords = {Perimeter countermeasures, Network segmentation and segregation, Per-robot measures/mitigations, Assessments and audits},
                nodes near coords,
              ]
              \addplot[lightgray, fill=lightgray] coordinates { 
                (53,Perimeter countermeasures) 
                (48,Network segmentation and segregation) 
                (36,Per-robot measures/mitigations) 
                (26,Assessments and audits) 
             };
              \legend{Percentage (\%)}
              \end{axis}
            \end{tikzpicture}
        }
        \caption{\footnotesize Percentages favouring each mitigation strategy amongst respondents (in the context of robotics). \emph{Sample size 52 respondents.}}
        \label{fig:ros:6}
    \end{subfigure}    
  
    \vspace{2em}
    \caption{Surveying the ROS robotics community (2019), from \cite{mayoral2022review}}
    \label{fig:survey:ros}
\end{figure*}
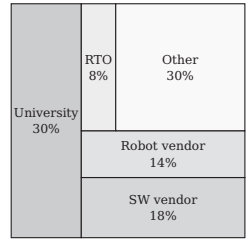
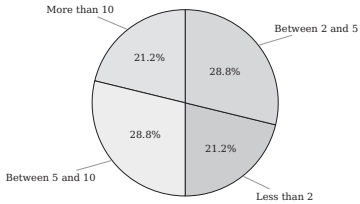
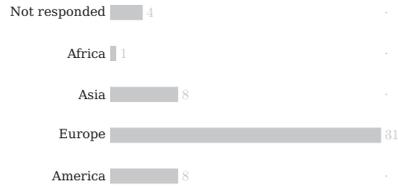
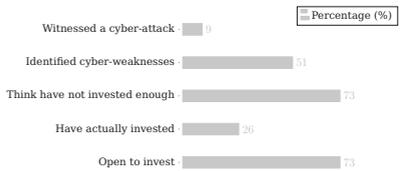
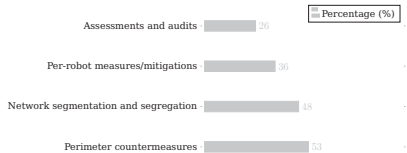

\showkeystrue

\subsection{Surveying the PX4 community}

PX4 \cite{meier2015px4} is an open source flight control software for drones and other unmanned vehicles. Similar to ROS, its community represents another relevant group in robotics. A security survey was conducted in 2020 and the results are summarized in Figure \ref{fig:survey:px4}\oldfootnote{Security survey launched within the PX4 Discourse community \cite{px4survey}.}. Though the PX4 community is significantly smaller than ROS's, the sample size obtained (11 respondents) was extremely small to draw major conclusions. Interestingly though, it was observed that the majority of the respondents have yet to see a security issue impacting the community (\ref{fig:px4:4}), only 27\% had seen it.

\begin{observation}
\label{observation:nohacks}
Both the PX4 (Figure \ref{fig:px4:4}) and the ROS (Figure \ref{fig:ros:5}) communities indicated that the majority is yet to witness a cyber-attack. In the ROS community only one out of ten respondents (9\%) had seen it whereas in the PX4 group, approximately one out of four (27\%).
\end{observation}

The majority of the respondents (81\%, figure \ref{fig:px4:5}) indicated to be willing to invest and more than 90\% confirmed that the amount could be 100 USD or above (Figure \ref{fig:px4:6}). This aligns nicely with observation \ref{observation:investment_gap} and further hints that growth should be expected in this field.

\showkeysfalse

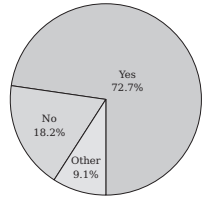
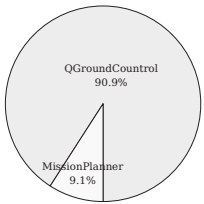
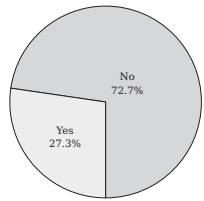
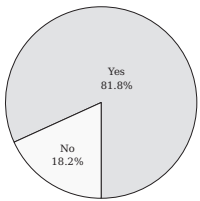
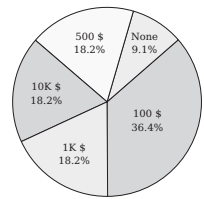
\begin{figure*}[p]
    \centering
    \begin{subfigure}[b]{0.45\textwidth}
        \centering
        \resizebox{0.6\textwidth}{!}{
            \begin{tikzpicture}   
                \pie[square, rotate=270, text=pin, color={lightgray!90, lightgray!70, lightgray!50, lightgray!30, lightgray!10}]  
                    {45.5/Medium, 27.3/High, 18.2/Very high, 9/Very low}  
               
            \end{tikzpicture}  
        }
        \caption{\footnotesize Distribution of respondents to the question: "What's your cybersecurity concern in PX4?". No respondents indicated "Low". \emph{Sample size 11 respondents.}}
        \label{fig:px4:1}
    \end{subfigure}
    ~
    \begin{subfigure}[b]{0.45\textwidth}
        \centering
        \resizebox{0.5\textwidth}{!}{
            \begin{tikzpicture}   
                \pie[rotate=270, text=inside, color={lightgray!90, lightgray!70, lightgray!50, lightgray!30, lightgray!10}]  
                    {72.7/Yes, 18.2/No, 9.1/Other}  
               
            \end{tikzpicture}  
        }
        \caption{\footnotesize Distribution of respondents to the question: "Does cyber security affect safety?". Other corresponds with user-provided answers. \emph{Sample size 11 respondents.}}
        \label{fig:px4:2}
    \end{subfigure}
    \\[5ex]
    \begin{subfigure}[b]{0.45\textwidth}
        \centering
        \resizebox{0.5\textwidth}{!}{
            \begin{tikzpicture}   
                \pie[rotate=270, text=inside, color={lightgray!30, lightgray!10}]  
                    {90.9/QGroundCountrol, 9.1/MissionPlanner}  
               
            \end{tikzpicture}  
        }
        \caption{\footnotesize Distribution of respondents to the question: "What's your GCS?". Various additional Ground Control Stations (GCS) were offered but not selected. \emph{Sample size 11 respondents.}}
        \label{fig:px4:3}
    \end{subfigure}
    ~
    \begin{subfigure}[b]{0.45\textwidth}
        \centering
        \resizebox{0.5\textwidth}{!}{
            \begin{tikzpicture}   
                \pie[rotate=270, text=inside, color={lightgray!70, lightgray!30, lightgray!10}]  
                    {72.7/No, 27.3/Yes}  
               
            \end{tikzpicture}  
        }
        \caption{\footnotesize Distribution of respondents to the question: "Have you seen any security issues so far in PX4?". \emph{Sample size 11 respondents.}}
        \label{fig:px4:4}
    \end{subfigure}
    \\[5ex]
    \begin{subfigure}[b]{0.45\textwidth}
        \centering
        \resizebox{0.5\textwidth}{!}{
            \begin{tikzpicture}   
                \pie[rotate=270, text=inside, color={lightgray!50, lightgray!10}]  
                    {81.8/Yes, 18.2/No}  
               
            \end{tikzpicture}  
        }
        \caption{\footnotesize Distribution of respondents to the question: "Would  you be willing to invest in cybersecurity for your drone?". \emph{Sample size 11 respondents.}}
        \label{fig:px4:5}
    \end{subfigure}
    ~
    \begin{subfigure}[b]{0.45\textwidth}
        \centering
        \resizebox{0.5\textwidth}{!}{
            \begin{tikzpicture}   
                \pie[rotate=270, text=inside, color={lightgray!70, lightgray!30, lightgray!10}]  
                    {36.4/100 \$, 9.1/None, 18.2/500 \$, 18.2/10K \$, 18.2/1K \$}  
               
            \end{tikzpicture}  
        }
        \caption{\footnotesize Distribution of respondents to the question: "How much per year is security worth to you?". \emph{Sample size 11 respondents.}}
        \label{fig:px4:6}
    \end{subfigure}

    \vspace{2em}
    \caption{Surveying the PX4 robotics community (2020), adapted from \cite{mayoral2022review}.}
    \label{fig:survey:px4}
\end{figure*}

\showkeystrue

\subsection{Surveying the ROS-Industrial community}

Also in 2020, a series of security-related surveys were launched as part of the European ROS-Industrial Conference, which happens every year in December. Data collected is presented in Figure \ref{fig:survey:rosi2020}\footnote{Security surveys  launched within the ROS-Industrial community during the digital ROS-I Europe Conference in December 2020.}. The majority of the respondents (93\%) showed awareness about the threats their robots faced and admitted being aware of their exposure to attackers (Figure \ref{fig:rosi:2}). Unsurprisingly, as a subset of the overall ROS community, the security mitigation actions in the ROS-I community also concentrate on the perimeter which lead to another observation:

\begin{observation}
\label{observation:perimiter}
Figures \ref{fig:ros:6} and \ref{fig:rosi:4} confirm that respectively for both ROS and ROS-I groups mitigations concentrate mostly on the perimeter.
\end{observation}

This fact becomes concerning in industrial environments wherein insider threats are as dangerous,  and the disruption of ROS could lead to catastrophic consequences for the automation processes \cite{mayoral2020can}, impacting more than 5 robots in 44\% of the cases according to respondents (Figure \ref{fig:rosi:6}). The lack of security measures in ROS are particularly concerning since its distributed communication middleware could be easily used to spread malware across connected robots. Such concept was demonstrated by \cite{mayoral2019industrial}, which prototyped an instance of ransomware targeting industrial collaborative robots, leaving these machines and its data completely locked until the corresponding ransom is paid.

\showkeysfalse

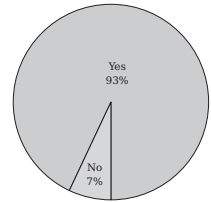
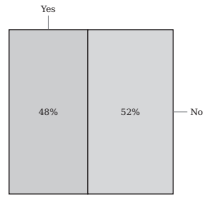
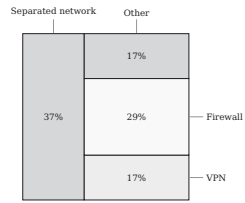
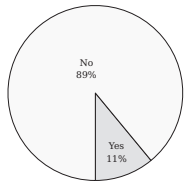
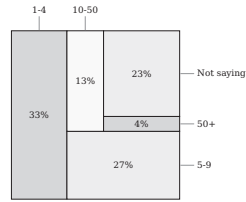
\begin{figure*}[p]
    \centering
    \begin{subfigure}[b]{0.45\textwidth}
        \centering
        \resizebox{0.6\textwidth}{!}{
            \begin{tikzpicture}   
                \pie[square, rotate=270, text=pin, color={lightgray!90, lightgray!70, lightgray!50, lightgray!30, lightgray!10}]  
                    {73/Very important, 20/Important, 7/Not important}
            \end{tikzpicture}  
        }
        \caption{\footnotesize Distribution of respondents to the question: "How important do you think is security for robotics and automation?". \emph{Sample size 30 respondents.}}
        \label{fig:rosi:1}
    \end{subfigure}
    ~
    \begin{subfigure}[b]{0.45\textwidth}
        \centering
        \resizebox{0.5\textwidth}{!}{
            \begin{tikzpicture}   
                \pie[rotate=270, text=inside, color={lightgray!90, lightgray!70, lightgray!50, lightgray!30, lightgray!10}]  
                    {93/Yes, 7/No}
            \end{tikzpicture}  
        }
        \caption{\footnotesize Distribution of respondents to the question: "Do you think your robot can be hacked?". \emph{Sample size 28 respondents.}}
        \label{fig:rosi:2}
    \end{subfigure}
    \\[5ex]
    \begin{subfigure}[b]{0.45\textwidth}
        \centering
        \resizebox{0.5\textwidth}{!}{
            \begin{tikzpicture}   
                \pie[square, rotate=270, text=pin, color={lightgray!90, lightgray!70, lightgray!50, lightgray!30, lightgray!10}]  
                    {48/Yes, 52/No}
            \end{tikzpicture}  
        }
        \caption{\footnotesize Distribution of respondents to the question: "Have you taken measures to protect your robots?". \emph{Sample size 25 respondents.}}
        \label{fig:rosi:3}
    \end{subfigure}
    ~
    \begin{subfigure}[b]{0.45\textwidth}
        \centering
        \resizebox{0.6\textwidth}{!}{
            \begin{tikzpicture}   
                \pie[square, rotate=270, text=pin, color={lightgray!70, lightgray!30, lightgray!10}]  
                    {37/Separated network, 17/VPN, 29/Firewall, 17/Other}  
            \end{tikzpicture}  
        }
        \caption{\footnotesize Distribution of respondents to the question: "What measures do you take to secure your robots?". \emph{Sample size 23 respondents.}}
        \label{fig:rosi:4}
    \end{subfigure}
    \\[5ex]
    \begin{subfigure}[b]{0.45\textwidth}
        \centering
        \resizebox{0.5\textwidth}{!}{
            \begin{tikzpicture}   
                \pie[rotate=270, text=inside, color={lightgray!50, lightgray!10}]  
                    {11/Yes, 89/No}  
               
            \end{tikzpicture}  
        }
        \caption{\footnotesize Distribution of respondents to the question: "Did  you use fuzzing before?". \emph{Sample size 27 respondents.}}
        \label{fig:rosi:5}
    \end{subfigure}
    ~
    \begin{subfigure}[b]{0.45\textwidth}
        \centering
        \resizebox{0.6\textwidth}{!}{
            \begin{tikzpicture}   
                \pie[square, rotate=270, text=pin, color={lightgray!70, lightgray!30, lightgray!10}]  
                    {33/1-4,27/5-9,13/10-50,4/50+,23/Not saying}  
               
            \end{tikzpicture}  
        }
        \caption{\footnotesize Distribution of respondents to the question: "How many robots are you controlling with ROS in your organization?". \emph{Sample size 55 respondents.}}
        \label{fig:rosi:6}
    \end{subfigure}

    \vspace{2em}
    \caption{Surveying the ROS-Industrial robotics community (2020), adapted from \cite{mayoral2022review}.}
    \label{fig:survey:rosi2020}
\end{figure*}
\showkeystrue

\subsection{Surveying the European robotics community at the European Robotics Forum (ERF) (2020)}

As one of the leading geographies in robotics and cybersecurity, the opinion of european robotics experts was sampled during the annual European Robotics Forum (ERF). Figure \ref{fig:survey:erf2020} summarizes the most relevant data collected\footnote{Security surveys  conducted during the robotics European gathering at the European Robotics Forum (ERF) 2020 in Málaga. The questionares were launched during the security sessions.}. The most interesting observation relates to the question \emph{"Who is the actor to be responsible for cyber-incidents?"}:

\begin{observation}
\label{observation:responsibility}
In Europe, the majority of the respondents (Figure \ref{fig:erf:2}) agree that the responsibility in case of damage as a result of a cyber-incident is to be assumed by the supply chain (86\% indicated that it'd sit between System Integrators and robot vendors), with only a 14\% pushing the responsibility to the end-user.
\end{observation}

\showkeysfalse
\begin{figure*}[p]
    \centering
    \begin{subfigure}[b]{0.45\textwidth}
        \centering
        \resizebox{\textwidth}{!}{
            \begin{tikzpicture}   
                \pie[explode=0.2, polar, rotate=270, text=pin, color={lightgray!90, lightgray!70, lightgray!50, lightgray!30, lightgray!10}]  
                    {0/Yes, 80/No, 0/Only in industry, 20/Depends vendor or end-user}
            \end{tikzpicture}  
        }
        \caption{\footnotesize Distribution of respondents to the question: "Do you think security awareness is sufficient in robotics?". \emph{Sample size not provided by ERF's platform.}}
        \label{fig:erf:1}
    \end{subfigure}
    ~
    \begin{subfigure}[b]{0.45\textwidth}
        \centering
        \resizebox{0.8\textwidth}{!}{
            \begin{tikzpicture}   
                \pie[explode=0.2, polar, rotate=270, text=pin, color={lightgray!90, lightgray!70, lightgray!50, lightgray!30, lightgray!10}]  
                    {43/Vendor, 14/End-user, 43/System integrators, 0/Distributors, 0/Insurance providers}
            \end{tikzpicture}  
        }
        \caption{\footnotesize Distribution of respondents to the question: "Who is the actor to be responsible for cyberincidents?". \emph{Sample size not provided by ERF's platform.}}
        \label{fig:erf:2}
    \end{subfigure}
    \\[5ex]
    \begin{subfigure}[b]{0.45\textwidth}
        \centering
        \resizebox{0.8\textwidth}{!}{
            \begin{tikzpicture}   
                \pie[explode=0.2, polar, rotate=270, text=pin, color={lightgray!90, lightgray!70, lightgray!50, lightgray!30, lightgray!10}]  
                    {11/Yes - many, 67/Yes, 0/Yes but none exploitable, 22/No}
            \end{tikzpicture}  
        }
        \caption{\footnotesize Distribution of respondents to the question: "Do you know of any feasible robot cyberattacks?". \emph{Sample size not provided by ERF's platform.}}
        \label{fig:erf:3}
    \end{subfigure}
    ~
    \begin{subfigure}[b]{0.45\textwidth}
        \centering
        \resizebox{\textwidth}{!}{
            \begin{tikzpicture}   
                \pie[explode=0.2, polar, rotate=270, text=pin, color={lightgray!90, lightgray!70, lightgray!50, lightgray!30, lightgray!10}]  
                    {56/Yes - by default, 0/Yes with minimal req.,33/Yes aligned with safety, 11/No}
            \end{tikzpicture}  
        }
        \caption{\footnotesize Distribution of respondents to the question: "Do you think robot cybersecurity standards are needed?". \emph{Sample size not provided by ERF's platform.}}
        \label{fig:erf:4}
    \end{subfigure}

    \vspace{2em}
    \caption{Surveying the European robotics community (ERF 2020), adapted from \cite{mayoral2022review}.}
    \label{fig:survey:erf2020}
\end{figure*}
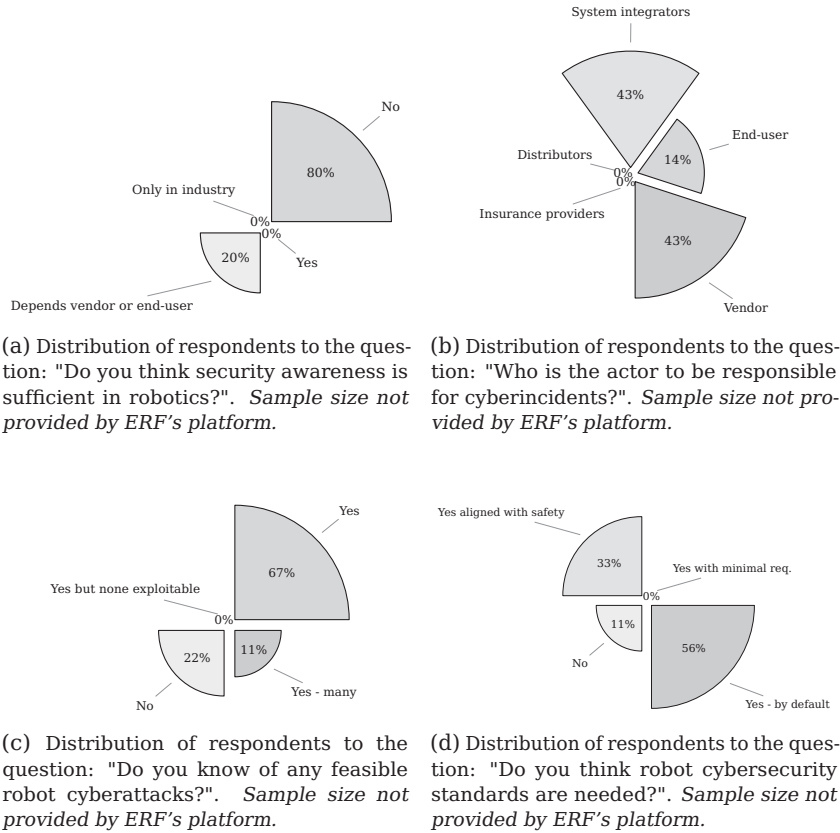

\showkeystrue


\section{Security research results in robotics}

Figure \ref{fig:vulnerabilitydata} depicts summarized vulnerability research results for three vendors: ABB, Mobile Industrial Robots (MiR) and Universal Robots (UR). The data was collected and archived over a multi-year period. Figures \ref{fig:vulns:abb}, \ref{fig:vulns:mir} and \ref{fig:vulns:ur} illustrate the "days until mitigation" for each vulnerability and according to the public data in the Robot Vulnerability Database (RVD) \cite{mayoralrvd}. The flat line represented by a series of data points in figures \ref{fig:vulns:mir} and \ref{fig:vulns:ur} denotes that the vendor hasn't reacted yet to any of these flaws and they remain unmitigated (they are zero days). For ABB robots, the scattered plot in Figure \ref{fig:vulns:abb} denotes more security activity. The following observations are drawn from the data:

\begin{observation}
\label{observation:cobots}
Collaborative robot manufacturers MiR and UR have zero days with an age at least older than one year (figures \ref{fig:vulns:mir} and \ref{fig:vulns:ur}). These flaws continue growing older due to the inactivity from the manufacturers.
\end{observation}

\begin{observation}
\label{observation:abb}
Vulnerability data affecting ABB robots (Figure \ref{fig:vulns:abb}) shows how according to historical data, vulnerabilities were patched as early as 14 days after its disclosure however the average mitigation time is above four years (1500 days) \cite{mayoralrvd}.
\end{observation}

On top of these, Figures \ref{fig:vulns:dist:abb} to \ref{fig:vuln:final:ur} enhance previous data with additional private sources of information and consider vulnerabilities that have yet to reach the public domain. It should be noted that the distribution of vulnerabilities signals the security awareness of the manufacturer. Coherently, figure  \ref{fig:vuln:final:abb} shows how for ABB robots, four out of five vulnerabilities considered have been publicly disclosed, triaged and scored. In contrast, for MiR and UR robots the oppositive is observed. Four out of five vulnerabilities have yet to be disclosed publicly.

\begin{observation}
\label{observation:readiness}
The ratio of publicly disclosed vulnerabilities versus the ones remaining private is an indicator when evaluating the security readiness of a robot manufacturer. The threat landscape of a given robot is correlated to this ratio in a direct manner.
\end{observation}

\showkeysfalse
\begin{figure*}[p]
    \centering
    \begin{subfigure}[b]{0.3\textwidth}
        \centering
        \resizebox{0.98\textwidth}{!}{
            \begin{tikzpicture}
              \pgfplotsset{scale only axis,}
              \begin{axis}[
                xlabel=$A$,
                ylabel=Days,
              ]
                \addplot[only marks, mark=x] file[] {data/abb.dat};
              \end{axis}
            \end{tikzpicture}
        }
        \caption{\footnotesize Days until mitigation for each one of the vulnerabilities publicly registered for ABB robots. \emph{Data collected from the RVD \cite{mayoralrvd}.}}
        \label{fig:vulns:abb}
    \end{subfigure}
    ~
    \begin{subfigure}[b]{0.3\textwidth}
        \centering
        \resizebox{0.95\textwidth}{!}{
            \begin{tikzpicture}
              \pgfplotsset{scale only axis,}
              \begin{axis}[
                xlabel=$A$,
                ylabel=Days,
              ]
                \addplot[only marks, mark=x] file[] {data/mir.dat};
              \end{axis}
            \end{tikzpicture}

        }
        \caption{\footnotesize Days until mitigation for each one of the vulnerabilities publicly registered for MiR robots. \emph{Data collected from the RVD \cite{mayoralrvd}.}}
        \label{fig:vulns:mir}
    \end{subfigure}
    ~
    \begin{subfigure}[b]{0.3\textwidth}
        \centering
        \resizebox{\textwidth}{!}{
            \begin{tikzpicture}
              \pgfplotsset{scale only axis,}
              \begin{axis}[
                xlabel=$A$,
                ylabel=Days,
              ]
                \addplot[only marks, mark=x] file[] {data/ur.dat};
              \end{axis}
            \end{tikzpicture}
        }
        \caption{\footnotesize Days until mitigation for each one of the vulnerabilities publicly registered for Universal Robots robots. \emph{Data collected from the RVD \cite{mayoralrvd}.}}
        \label{fig:vulns:ur}
    \end{subfigure}
    \\[5ex]    
    \begin{subfigure}[b]{0.3\textwidth}
        \resizebox{\textwidth}{!}{
            \begin{tikzpicture}
              \begin{axis}[
                xbar,
                xmin=0,
                y axis line style = { opacity = 0 },
                axis x line       = none,
                tickwidth         = 0pt,
                bar width         = 3em,
                enlarge y limits  = 0.2,
                enlarge x limits  = 0.02,
                symbolic y coords = {Public,Private,Total},
                nodes near coords,
                ytick             = data,                
              ]
              \addplot[black, fill=lightgray] coordinates { 
                (61,Public) 
                (15,Private) 
                (76,Total) 
             };
              \end{axis}
            \end{tikzpicture}
        }
        \caption{\footnotesize Vulnerabilities affecting ABB robots registered in the RVD (Public), in other private databases (Private), as well as the overall amount (Total).}
        \label{fig:vulns:dist:abb}
    \end{subfigure}
    ~
    \begin{subfigure}[b]{0.3\textwidth}
        \resizebox{\textwidth}{!}{
            \begin{tikzpicture}
              \begin{axis}[
                xbar,
                xmin=0,
                y axis line style = { opacity = 0 },
                axis x line       = none,
                tickwidth         = 0pt,
                bar width         = 3em,
                enlarge y limits  = 0.2,
                enlarge x limits  = 0.02,
                symbolic y coords = {Public,Private,Total},
                nodes near coords,
                ytick             = data,                
              ]
              \addplot[black, fill=lightgray] coordinates { 
                (15,Public) 
                (125,Private) 
                (140,Total) 
             };
              \end{axis}
            \end{tikzpicture}
        }
        \caption{\footnotesize Vulnerabilities affecting MiR robots registered in the RVD (Public), in other private databases (Private) as well as the overall amount (Total).}
        \label{fig:vulns:dist:mir}
    \end{subfigure}
    ~
    \begin{subfigure}[b]{0.3\textwidth}
        \resizebox{\textwidth}{!}{
            \begin{tikzpicture}
              \begin{axis}[
                xbar,
                xmin=0,
                y axis line style = { opacity = 0 },
                axis x line       = none,
                tickwidth         = 0pt,
                bar width         = 3em,
                enlarge y limits  = 0.2,
                enlarge x limits  = 0.02,
                symbolic y coords = {Public,Private,Total},
                nodes near coords,
                ytick             = data,                
              ]
              \addplot[black, fill=lightgray] coordinates { 
                (92,Public) 
                (367,Private) 
                (459,Total) 
             };
              \end{axis}
            \end{tikzpicture}
        }
        \caption{\footnotesize Vulnerabilities affecting UR robots registered in the RVD (Public), in other private databases (Private) as well as the overall amount (Total).}
        \label{fig:vulns:dist:ur}
    \end{subfigure}
    ~
    \\[5ex]
    \begin{subfigure}[b]{0.3\textwidth}
        \centering
        \resizebox{\textwidth}{!}{
            \begin{tikzpicture}   
                \pie[square, rotate=270, text=pin, color={lightgray!80, lightgray!10}]  
                    {19.7/Private, 80.3/Public}
            \end{tikzpicture}  
        }
        \caption{\footnotesize Distribution of the vulnerabilities affecting ABB robots and registered in the RVD (Public) or in other private databases (Private).}
        \label{fig:vuln:final:abb}
    \end{subfigure}
    ~
    \begin{subfigure}[b]{0.3\textwidth}
        \centering
        \resizebox{\textwidth}{!}{
            \begin{tikzpicture}   
                \pie[square, rotate=270, text=pin, color={lightgray!80, lightgray!10}]  
                    {89.3/Private, 10.7/Public}
            \end{tikzpicture}  
        }
        \caption{\footnotesize Distribution of the vulnerabilities affecting MiR robots and registered in the RVD (Public) or in other private databases (Private).}
        \label{fig:vuln:final:mir}
    \end{subfigure}
    ~
    \begin{subfigure}[b]{0.3\textwidth}
        \centering
        \resizebox{\textwidth}{!}{
            \begin{tikzpicture}   
                \pie[square, rotate=270, text=pin, color={lightgray!80, lightgray!10}]  
                    {80/Private, 20/Public}
            \end{tikzpicture}  
        }
        \caption{\footnotesize Distribution of the vulnerabilities affecting UR robots and registered in the RVD (Public) or in other private databases (Private).}
        \label{fig:vuln:final:ur}
    \end{subfigure}
    \\[5ex]           
    \vspace{2em}
    \caption{Vulnerability data for various robots, adapted from \cite{mayoral2022review}.}
    \label{fig:vulnerabilitydata}
\end{figure*}
\showkeystrue


\section{Discussion}


\subsection{Ethical Considerations During Surveys}

The surveys conducted within various robotics communities were aimed at gathering insights into the state of cybersecurity awareness, practices, and concerns within these communities. While the primary focus was on data collection and analysis, ethical considerations surrounding the conduct of these surveys are essential to address.

\begin{enumerate}
  \item \textbf{Ethical Screening:} Prior to launching the surveys, ethical screening was performed to ensure that the data collection process adhered to established ethical standards. This screening involved a review of the survey content, methodology, and potential risks to participants.
  
  \item \textbf{Informed Consent:} Participants were provided with clear and comprehensive information regarding the purpose, storage, and use of their personal data. This information was presented at the beginning or within the announcement of each survey (as examples, see announcement for the \href{https://discourse.ros.org/t/call-for-participation-ii-to-survey-the-security-in-robotics/10811}{ROS} and \href{https://discuss.px4.io/t/cyber-security-in-px4-survey-and-input/17449}{PX4} survey), and participants were required to provide informed consent before proceeding.
  
  \item \textbf{Confidentiality and Anonymity:} Measures were implemented to ensure the confidentiality and anonymity of participant responses. Personal identifying information was kept separate from survey responses, and only aggregated data was used for analysis and reporting.
  
  \item \textbf{Data Storage and Use:} Participant data was securely stored and used solely for the purposes outlined in the survey introduction. Data was accessible only to authorized researchers involved in the study and was not shared with third parties.
  
  \item \textbf{Feedback and Accessibility:} Participants were provided with opportunities to provide feedback on the survey process and to access the results of the study upon request. Transparency regarding data handling practices fostered trust and accountability among participants.
  
  \item \textbf{Conference Context:} It's important to note that some surveys were conducted within the context of conferences, where participants may have been attendees or members of the respective communities. While this facilitated data collection, it also presented unique challenges in terms of informed consent and data privacy.
  
  \item \textbf{Limitations and Future Considerations:} Despite efforts to uphold ethical standards, it's important to acknowledge the limitations of the survey methodology and consider potential improvements for future research endeavors. This includes exploring alternative data collection methods, such as interviews or focus groups, to ensure a more nuanced understanding of participants' perspectives.
\end{enumerate}

In summary, ethical considerations were integral to the design and implementation of the surveys conducted within robotics communities. By prioritizing participant welfare, confidentiality, and transparency, the integrity of the research process was upheld, facilitating meaningful insights into cybersecurity practices and concerns within these communities.

\subsection{Vulnerability Disclosure Policies}
Prevailing ethical standards and practices in robot cybersecurity are still on their early stages however, after a few years of research and while interacting with various communities, a few remarks can be made. The subsections below cover some of the common practices observed in robot cybersecurity and discusses their ethical implications.

\subsubsection{Evolution of Disclosure Policies in Robotics}

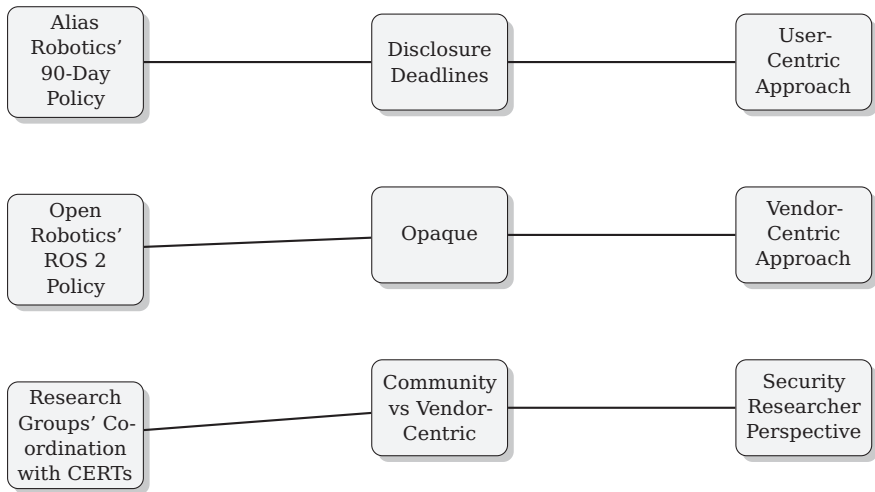
\begin{figure}[htbp]
\centering
\begin{tikzpicture}[
    auto,
    block/.style={rectangle, draw, fill=black!5, text width=5em, text centered, rounded corners, minimum height=4em, drop shadow, font=\footnotesize},
    line/.style={draw, thick}
]

\node [block] (alias) {Alias Robotics' 90-Day Policy};
\node [block, below=of alias] (openr) {Open Robotics' ROS 2 Policy};
\node [block, below=of openr] (certs) {Research Groups' Coordination with CERTs};
\node [block, right=3cm of alias] (discDeadlines) {Disclosure Deadlines};
\node [block, below=of discDeadlines] (transparency) {Opaque};
\node [block, below=of transparency] (community) {Community vs Vendor-Centric};
\node [block, right=3cm of discDeadlines] (userCentric) {User-Centric Approach};
\node [block, below=of userCentric] (vendorCentric) {Vendor-Centric Approach};
\node [block, below=of vendorCentric] (researcher) {Security Researcher Perspective};

\path [line] (alias) -- (discDeadlines);
\path [line] (openr) -- (transparency);
\path [line] (certs) -- (community);
\path [line] (discDeadlines) -- (userCentric);
\path [line] (transparency) -- (vendorCentric);
\path [line] (community) -- (researcher);

\end{tikzpicture}
\caption{Summary of Vulnerability Disclosure Policies in Robotics}
\label{figure:vdp}
\end{figure}

The landscape of vulnerability disclosure policies in robotics has undergone notable evolution, particularly influenced by the practices adopted by different organizations and the ethical implications of these choices. This subsection examines three significant contributions in the field, each highlighting different approaches and their resultant impact on the cybersecurity landscape in robotics.

\paragraph{Alias Robotics' 90-Day Disclosure Deadline}
In 2018 and 2019, Alias Robotics pioneered early work on vulnerability disclosure policies in robotics. This was part of their effort to establish the Robot Vulnerability Database (RVD) \cite{mayoralrvd}. This approach adhered to a 90-day \href{https://github.com/aliasrobotics/RVD#disclosure-policy}{disclosure deadline} for newly identified vulnerabilities. The rationale behind this timeframe was to balance the urgency for patch development against vendors' capability to respond effectively. The 90-day deadline aimed to push vendors towards improving their internal processes for vulnerability remediation, which historically had seen delays or even neglect. Alias Robotics' approach was based on the principle that quicker response times to disclosed vulnerabilities would ultimately result in better security for users.

\paragraph{Open Robotics' ROS 2 Vulnerability Disclosure Policy}

Open Robotics, in contrast to Alias Robotics' proactive approach, introduced a ROS 2 Vulnerability Disclosure Policy \cite{orros2vdp} in 2020, which notably did not incorporate relevant feedback from the Security Working Group \cite{rep2006feedback} and derived from Alias Robotics' initial work. The policy lacked crucial elements like disclosure deadlines, transparent communication, and acknowledgment of the educational value of post-mitigation security research. It diverged from established international security practices, such as assigning CVE IDs, and depended on a non-transparent process for security issue verification. While the policy mentioned a safe harbor, its actual efficiency and effectiveness were questioned. Open Robotics' approach suggested a missed opportunity to prioritize cybersecurity within the ROS community, reflecting a potential lack of commitment to rigorous security standards.

\paragraph{Security Research Groups' Direct Coordination with International CERTs}
The ROS 2 Vulnerability Disclosure Policy's limitations led several security research groups to bypass it in favor of direct coordination with international Computer Emergency Response Teams (CERTs). One of such groups included Alias Robotics contributions again and identified relevant vulnerabilities affecting a significant number of ROS robotic endpoints \cite{maggi2022security} and disclosed them responsibly after meeting reasonable deadlines \cite{ddsvulnerabilitiesdisclosure}. The public disclosure got further disseminated in the ROS robotics community \cite{ddsvulnsros2} (the one owned by Open Robotics) and empowered end-users with knowledge to mitigate outstanding security issues, in contrast to Open Robotics' response, which downplayed the significance of these disclosures and shifted responsibility to third parties\footnote{Representatives of Open Robotics responded publicly  \cite{ros2flawsarticle} ignoring the relevance of the disclosures and offloading responsibility into third parties (DDS vendors).}. This direct approach with CERTs highlighted a more efficient and impactful way of handling cybersecurity in the robotics ecosystem, as opposed to working with Open Robotics (the acting vendor, as community responsible).

\subsubsection{Ethical Implications of Disclosure Policies}

\paragraph{Necessity of Disclosure Deadlines}
The implementation of disclosure deadlines by Alias Robotics was driven by ethical considerations related to the timeliness of vulnerability patching. Slow patch timelines pose a significant risk, as attackers could potentially exploit vulnerabilities before they are patched. The 90-day deadline was designed to incentivize vendors to improve their response times, balancing the need for thorough patch development with the urgency to protect users from potential exploits.

\paragraph{Transparency and Educational Value of Disclosures}
Open Robotics' policy, which lacked transparency and discouraged the disclosure of security research for educational purposes, overlooked the long-term benefits of sharing detailed vulnerability information. Disclosing technical details post-patch offers educational value to the security community, helping to inform and strengthen overall cybersecurity practices. This approach contributes to a more informed and capable pool of security professionals and researchers.

\paragraph{Community-Centric versus Vendor-Centric Approaches}
The different policies reflect a divide between community-centric and vendor-centric approaches to cybersecurity. Alias Robotics' policy, with its emphasis on timely public disclosures, aligns with a community-centric approach that prioritizes user security and education. In contrast, Open Robotics' policy appears more vendor-centric, potentially limiting the flow of critical security information to the wider community and impeding collaborative efforts to address vulnerabilities.


\subsubsection{Conclusion: The Ethical Dimension of Cybersecurity in Robotics}

The exploration of vulnerability disclosure policies in robotics, particularly those of Alias Robotics and Open Robotics, offers a vital perspective on the ethical dimensions of cybersecurity in this domain. These cases illuminate the divergent paths that organizations can take in addressing security vulnerabilities, each path carrying its own ethical implications and impact on end-users.

\paragraph{User-Centric Approach and Its Importance}
Alias Robotics' approach, with its strict 90-day disclosure deadline, embodies a user-centric methodology. This policy not only impels vendors to act swiftly but also aligns with a broader ethical stance that prioritizes the safety and security of end-users. By setting clear deadlines and making vulnerability information public, Alias Robotics upholds the principle of transparency, empowering users with knowledge and means for protection. This approach is demonstrative of an ethical commitment to the community, where the security and well-being of users are placed at the center of cybersecurity practices.

\paragraph{Vendor-Centric Approach: A Cautionary Tale}
On the other hand, Open Robotics' less prescriptive ROS 2 Vulnerability Disclosure Policy, which lacks clear deadlines and a transparent process, reflects a vendor-centric approach. This policy, while ostensibly designed to manage vulnerabilities, falls short in fostering a proactive security culture within the community. Its lack of urgency and transparency may inadvertently lead to prolonged vulnerability exposure, thereby increasing risks for end-users. This approach can be seen as ethically problematic, as it does not adequately address the immediate security needs of the community and potentially places the interests of vendors above those of the users.

\paragraph{Implications for End-Users}
The contrasting policies and their implementation have direct implications for end-users, particularly in terms of their exposure to potential cyber threats. A policy that enforces strict deadlines and encourages open communication about vulnerabilities serves to enhance the overall security posture of the robotics ecosystem. It acknowledges the evolving nature of cyber threats and the importance of keeping users informed and equipped to respond to these threats. In contrast, a policy that lacks these elements may leave end-users vulnerable and uninformed, undermining the trust and reliability essential in the field of robotics.

\paragraph{Security Researcher Perspective: Navigating Ethical and Practical Challenges}

The exploration of vulnerability disclosure policies in the context of robotics also brings to the forefront the unique challenges faced by security researchers. These individuals, who play a pivotal role in identifying and reporting vulnerabilities, often navigate a landscape fraught with ethical dilemmas and practical hurdles. A significant concern in this realm is the tendency of some vendor-centric disclosure policies to pressure researchers into revealing their findings without proper recognition, compensation, or, at times, under the threat of legal action.

\paragraph{Compromised Incentives and Recognition}
Security researchers invest significant time and effort in uncovering vulnerabilities, with the primary goal of enhancing system security and user safety. However, the expectation set by certain disclosure policies to give away this valuable information, often without compensation or adequate acknowledgment, can be demoralizing. This lack of incentive not only undermines the researchers' contributions but also potentially devalues the importance of their work. In an environment where recognition and fair compensation are lacking, the motivation to engage in this critical work can diminish, leading to a scarcity of skilled individuals willing to dedicate their efforts to cybersecurity research in robotics.

\paragraph{Legal Threats and the Ethical Dilemma}
The situation becomes even more complex when legal threats enter the equation. Researchers, in their pursuit to enhance cybersecurity, may find themselves in a precarious position when their findings are met with hostility or legal intimidation from vendors. This scenario presents an ethical quandary: on the one hand, researchers have a moral obligation to disclose vulnerabilities to protect end-users; on the other hand, they face potential legal repercussions for doing so. Such an environment can create a chilling effect, discouraging researchers from participating in vulnerability discovery and reporting, ultimately hindering the advancement of cybersecurity in the robotics field.

\paragraph{Impact on the Security Landscape}
The cumulative effect of these challenges is a security landscape where the vital contributions of researchers are not adequately valued or supported. This situation can lead to a reduction in proactive security research, leaving vulnerabilities undiscovered and unaddressed. Consequently, the overall security posture of robotic systems may be weakened, exposing end-users to increased risks. To foster a more robust and resilient cybersecurity environment, it is essential to recognize and support the critical role of security researchers. This support can manifest in various forms, including fair and transparent policies, legal protections, recognition of their contributions, and avenues for compensation or incentives.

\paragraph{The Way Forward}
In conclusion, the ethical landscape of cybersecurity in robotics is one where user safety, transparency, and proactive engagement must be paramount. As robotics continue to integrate more deeply into various aspects of society and industry, the responsibility to ensure robust and ethical cybersecurity practices becomes increasingly critical. The experiences of Alias Robotics and Open Robotics serve as instructive examples, highlighting the necessity of policies that are not only technically sound but also ethically grounded. Such policies should aim to protect users, foster an informed and engaged community, and uphold the highest standards of safety and security in the rapidly evolving world of robotics.


\subsection{Lack of learning resources}

One of the significant challenges in the realm of robot cybersecurity is the lack of comprehensive learning resources. This shortage is exacerbated by the current vulnerability disclosure policies of various actors and groups, such as Open Robotics, as discussed above. These policies, often more vendor-centric than community-driven, do not facilitate the generation of educational content that could foster a deeper understanding of cybersecurity in robotics. This gap in knowledge and resources is particularly concerning given the increasing integration of robotics into various facets of society and industry.

In response to this void, the Robot Hacking Manual (RHM) \cite{mayoral2022robot}, authored as part of this research, emerges as a crucial resource. The RHM is an introductory series about cybersecurity for robots, conceived with the intent to raise awareness in the field and underscore the importance of adopting a security-first approach. The manual offers comprehensive case studies and step-by-step tutorials, designed to provide a foundational understanding of robot cybersecurity. The content within the RHM is independent and not aligned with any particular organization, reflecting a genuine effort to democratize knowledge in this critical area.

An integral part of the RHM is the robot hacks list displayed in table \ref{table:robot_hacks}. This list is a non-exhaustive compilation of cybersecurity research in robotics, detailing various robot vulnerabilities and attacks due to cybersecurity issues. The list serves as a testament to the evolving landscape of threats in robot cybersecurity and underscores the necessity for continued vigilance and education in the field.

\begin{table}[htbp]
\centering
\resizebox{\textwidth}{!}{%
\begin{tabular}{|c|c|c|p{0.5\textwidth}|c|}
\hline
\rowcolor{black}
\color{white}\textbf{Codename/Theme} & \color{white}\textbf{Robot Tech} & \color{white}\textbf{Researchers} & \color{white}\textbf{Description} & \color{white}\textbf{Date} \\ \hline
- & iRobot’s Roomba J7 & N/A & Personal pictures in a home environment were found on the Internet taken by an iRobot’s Roomba J7 series robot vacuum. & 19-19-2022 \\ \hline
- & Unitree's Go1 & d0tslash & A hacker found a kill switch for a gun–wielding legged robot. & 09-08-2022 \\ \hline
- & Enabot's Ebo Air & Modux & Researchers found a security flaw in Enabot Ebo Air robot leading to potential remote-controlled robot spy units. & 21-07-2022 \\ \hline
Analyzing DDS & Various DDS implementations & Various & Research on the Data Distribution Service (DDS) protocol revealing vulnerabilities. & 19-04-2022 \\ \hline
Hacking ROS 2 & ROS 2 & Various & Discovery of security vulnerabilities in ROS 2's communication middleware. & 22-04-2022 \\ \hline
JekyllBot:5 & Aethon TUG & Cynerio & Collection of vulnerabilities enabling remote control of Aethon TUG smart robots. & 01-04-2022 \\ \hline
Robot Teardown & Universal Robots & Various & Research uncovering security flaws in UR robots through hardware analysis. & 20-07-2021 \\ \hline
Rogue Automation & Various industrial robots & Various & Analysis revealing risks in industrial automation programming languages. & 01-08-2020 \\ \hline
Securing Disinfection Robots & UVD Robots & Alias Robotics & Vulnerabilities in disinfection robots used against COVID-19. & 19-09-2020 \\ \hline
MiR Bugs Week & Mobile Industrial Robots & Alias Robotics & Public release of multiple vulnerabilities in MiR robots. & 24-06-2020 \\ \hline
Smart Manufacturing Attacks & Mitsubishi Melfa & Various & Security analysis of attacks on a smart manufacturing system. & 01-05-2020 \\ \hline
UR Bugs Week & Universal Robots & Alias Robotics & Exposure of security issues in Universal Robots' products. & 31-03-2020 \\ \hline
Akerbeltz & Universal Robots UR3, UR5, UR10 & Alias Robotics & Creation of industrial robot ransomware demonstrated on UR robots. & 16-12-2019 \\ \hline
Rogue Robots & ABB IRB140 & Various & Theoretical and experimental challenges in the security of modern industrial robots. & 01-05-2017 \\ \hline
Hacking Robots Before Skynet & Various robots & IOActive & Discovery of critical cybersecurity issues in multiple robots. & 30-01-2017 \\ \hline
ROS: Safe & Insecure & ROS & Early study on security vulnerabilities in ROS. & 28-02-2014 \\ \hline
\end{tabular}%
}
\caption{Robot Hacks: Notable Cybersecurity Incidents in Robotics}
\label{table:robot_hacks}
\end{table}


\section{Conclusion}

This chapter, situated within the ethical framework of cybersecurity in robotics, has navigated through the intricate landscape of robot security, revealing its multifaceted nature and current state. The exploration, driven by the goal to understand and critique the prevailing ethical standards and practices in robot cybersecurity, delved deep into various dimensions of this field.

The findings from the systematic review of literature, combined with insights from surveys conducted across prominent robotics communities, have highlighted several key observations. A noticeable gap exists between the anticipated and actual investment in cybersecurity within the robotics domain, particularly in the ROS community (Observation \ref{observation:investment_gap}). This gap not only underlines the potential growth area in cybersecurity actions but also mirrors the nascent stage of maturity in this field compared to established sectors like IT or OT (Observation \ref{observation:nomeasures}).

A surprising revelation was the majority of the robotics community's lack of exposure to cyber-attacks (Observation \ref{observation:nohacks}), suggesting a lower perceived threat level. This could potentially lead to a complacency in adopting rigorous cybersecurity measures. Furthermore, the tendency to focus on perimeter defense strategies (Observation \ref{observation:perimiter}), as noted in both ROS and ROS-I communities, signals a possible oversight of insider threats, particularly crucial in industrial environments where robotics systems are integral.

The responsibility for cyber incidents in Europe is largely attributed to the supply chain, with system integrators and robot vendors bearing the brunt of the blame, rather than the end-users (Observation \ref{observation:responsibility}). This points to an ethical standpoint where the onus of security is more on those who manufacture and integrate rather than those who operate.

The situation with collaborative robot manufacturers like MiR and UR is particularly alarming, as they exhibit a pattern of unmitigated zero-day vulnerabilities (Observation \ref{observation:cobots}). This inaction contrasts sharply with the approach of manufacturers like ABB, who show a varied response to cybersecurity issues (Observation \ref{observation:abb}). The ratio of publicly disclosed vulnerabilities to those kept private further serves as an indicator of a manufacturer's readiness and approach towards cybersecurity (Observation \ref{observation:readiness}).

\begin{table}[htbp]
\centering
\scalebox{0.8}{
    \begin{tabular}{|p{0.2\textwidth} p{0.8\textwidth}|}
    \hline
    \rowcolor{black}
    \color{white}\textbf{Observation} & \color{white}\textbf{Description} \\
    \hline
    \ref{observation:investment_gap} & There is a gap between the expectations and the actual investment in cybersecurity in the ROS community, indicating potential future growth in cybersecurity actions. \\
    \hline
    \ref{observation:nomeasures} & The lack of robot-specific security measures and offensive assessments in the ROS community indicates a lower maturity level in cybersecurity compared to other sectors like IT or OT. \\
    \hline
    \ref{observation:nohacks} & Majority of respondents in both ROS and PX4 communities have not witnessed a cyber-attack, suggesting a lower perceived threat level. \\
    \hline
    \ref{observation:perimiter} & Security mitigations in both ROS and ROS-I groups primarily focus on perimeter defense, indicating a potential underestimation of insider threats in industrial environments. \\
    \hline
    \ref{observation:responsibility} & In Europe, most respondents believe that the responsibility for cyber incidents lies with the supply chain (system integrators and robot vendors), not the end-user. \\
    \hline
    \ref{observation:cobots} & Collaborative robot manufacturers MiR and UR have unmitigated zero-day vulnerabilities, indicating a lack of responsive action from these manufacturers. \\
    \hline
    \ref{observation:abb} & ABB robots show a mixed response to cybersecurity, with some vulnerabilities patched quickly while others remain unaddressed for years. \\
    \hline
    \ref{observation:readiness} & The ratio of publicly disclosed to private vulnerabilities is an indicator of a robot manufacturer's security readiness. ABB shows a higher rate of public disclosures, while MiR and UR have more undisclosed vulnerabilities. \\
    \hline
    \end{tabular}}
\caption{Summary of Observations from Robotics Security Surveys and Research}
\label{table:observations}    
\end{table}

Table \ref{table:observations} offers a summary of observations from robotics security surveys and research. Overall, this chapter has underscored the importance of re-evaluating and reinforcing cybersecurity measures in the rapidly evolving field of robotics. It advocates for a shift from traditional perimeter defense strategies to more comprehensive approaches that consider insider threats and embrace practices like \emph{zero-trust}. The observations made here emphasize the need for increased investment and focus on robot-specific security measures and offensive security assessments. This proactive approach is crucial to enhance the resilience of robotic systems and to bridge the gap between current practices and the evolving cybersecurity landscape in robotics.

This critical examination of the ethical standards in robot cybersecurity not only sheds light on the current state of affairs but also sets the stage for future advancements. One of the central themes uncovered is the critical role of vulnerability disclosure policies in shaping the ethical landscape of robot cybersecurity. As highlighted in the discussion, there is a stark contrast between the user-centric approach and the vendor-centric approach. The former underscores the ethical imperative of prioritizing end-user safety and system security through timely and transparent disclosure of vulnerabilities. In contrast, the latter's approach, characterized by a lack of urgency and transparency, potentially compromises the security and trust in robotic systems. By challenging established norms and scrutinizing practices like responsible disclosure, this chapter contributes to shaping a more ethical landscape and secure future in robotics, where the safety and integrity of both humans and machines are paramount.

    \part{Blueprints}\label{part:blueprints}
    \chapter{Cybersecurity in ROS}
\label{chapter:ros_security}

\epigraph{\textbf{Contributions}: \emph{Leading scientific author in \cite{mayoral2022sros2} wherein I led contributions to the DevSecOps methodology and multi-year contributions to software. Subject-matter expert in robotics and security in \cite{ichnowski2023fogros} wherein I kicked-off the initial \href{https://github.com/BerkeleyAutomation/FogROS2/issues/1}{implementation discussions}, \href{https://github.com/BerkeleyAutomation/FogROS2/issues/6}{architected the ROS-integration} (core of the paper), and \href{https://github.com/BerkeleyAutomation/FogROS2/issues/15}{released it into ROS Humble} for widespread use. Both works are open source.}}{}


\section{Introduction}

A robot is a network of networks \cite{mayoral2021hacking}. One that is comprised of sensors to perceive the world, actuators to produce a physical change, and computational resources to process it all and respond coherently, in time, and according to its application. Security is of paramount importance in this context, as any disruption of any of these robot networks can cause the complete robotic system to misbehave and compromise the safety of humans, as well as the environment \cite{kirschgens2018robot,zamalloa2017dissecting}. 

The Robot Operating System (ROS) \cite{quigley2009ros} is the \emph{de facto} framework for robot application development. Widely used to govern interactions across robot networks, at the time of writing, the original ROS article \cite{quigley2009ros} has been cited more than 9300 times, which shows its wide acceptance for research and academic purposes. ROS was born in this environment: its primary goal was to provide the software tools and libraries that users would need to employ to undertake novel robotics research and development. Adoption in industry has also been rapidly increasing over the last few years. According to the latest ROS community metrics \cite{rosmetrics} sampled every year in July, the number of ROS downloads has increased by over 50\%, with about 600 million downloads between July of 2020 and July of 2021. Moreover, based on the download percentages reported from \texttt{packages.ros.org}, one can observe a significant increase in adopting ROS 2, which suggests that by 2023 there would be more users using ROS 2 than its predecessor\footnote{The reader must also note that past studies estimated that by 2024, 55\% of the total commercial robots shipped that year would include at least one ROS package. For more details, refer to \emph{\url{https://www.businesswire.com/news/home/20190516005135/en/Rise-ROS-55-total-commercial-robots-shipped}}.}.

ROS was not designed initially with security in mind, but as it started being adopted and deployed into products or used in government programs, more attention was drawn to security issues. Some of the early work on securing ROS included \cite{lera2016ciberseguridad, ApplicationSecROS} or \cite{white2016sros}, both appearing in the second half of 2016. At the time of writing, none of these efforts remain actively maintained and the community focus on security efforts has switched to ROS 2. A recent study \cite{mayoral2022review} that surveyed the security interests in the ROS community presented data indicating that 73\% of the survey participants considered that they had not invested enough to protect their robots from cyber-threats. The same number of participants indicated that their organizations were open to invest, however only 26\% acknowledged to actually have invested. This led the original authors to conclude that there is a gap between the security expectations and the actual investment. In this work, it's argued that this gap is the result not only of the immaturity of security in robotics or the know-how but also by the lack of usability of the available security tools. Being conscious that security in robotics is not a product, but a process that needs to be continuously assessed in a periodic manner \cite{mayoral2022robot,zhu2021cybersecurity,mayoral2020alurity}, this work advocates for a usable security approach in robotics as the best way to remain secure.

In this chapter we introduce SROS2, a series of developer tools, meant to be usable and that facilitate adding security capabilities to ROS 2 computational graphs. We present in here a security methodology consisting of six steps that allow securing ROS 2 graphs iteratively, with the aid of SROS2. Driven by an application use case, we discuss how SROS2 allows achieving security in complex graphs involving popular ROS 2 packages and analyze the security trade-offs and limitations of the current tooling. The key contributions of this work are:
\begin{itemize}
    \item Create SROS2, a set of usable tools for adding security to ROS 2 that: (1) help introspect the computational graph by extracting communication middleware-level information; (2) simplify the security operations creating Identity and Permissions Certificate Authorities (CA) that govern the security policies of a ROS 2 graph; (3) help organize all security artifacts in a consistent manner and within a directory tree that is generated within the current ROS 2 workspace overlay; (4) help create a new identity for each enclave, generating a keypair and signing its x.509 certificate using the appropriate CA; (5) create governance files to encrypt all DDS traffic by default; (6) support specifying enclave permissions in familiar ROS 2 terms which are then automatically converted into low-level DDS permissions; (7) support automatic discovery of required permissions from a running ROS 2 system; and (8) dissect communication middleware interactions, to extract key information for the security monitoring of the system.
    \item Propose a methodology for securing ROS 2 computational graphs that provides roboticists with a structured process to continuously assess their security.
    \item Expose insights into how to apply SROS2 to real ROS 2 computational graphs by presenting an application case study focused on analyzing the Navigation2 and SLAM Toolbox stacks in a TurtleBot3 robot.
\end{itemize}

The core components of SROS2 are disclosed under a commercially friendly open-source license and are available and maintained at \emph{\url{https://github.com/ros2/sros2}}.

\section{Background}

Considering how ROS was originally intended as a fast prototyping robotic framework, security was not considered a priority feature in its first iteration. As ROS has evolved from the prototyping to the real-world industrial applications the entire stack came to be in dire need of cybersecurity attention \cite{mayoral2022robot} \cite{phdCaiazza}.  

A first partial analysis, with the goal of understanding what prevented ROS from being used industrially, was conducted by McClean \emph{et al.} \cite{McClean2013Preliminary}. By means of a 'honeypot' system, at DEFCON-20, they collected how malicious users would tackle a robot in the wild. Dieber \emph{et al.} \cite{dieber2020penetration} provided a complete and in-depth analysis of the security vulnerabilities and attack surfaces in ROS systems and how to exploit them, highlighting the gaps in the security of the framework. A considerable amount of research has been done as regards the publish-subscribe paradigm reviewing the performance and the techniques to secure it either via the communication channel, and ROS' internal mechanisms \cite{Goerke2021}. In the first case, via message authentication \cite{Toris2014Message}, within the later addition of using of encryption and security artifacts \cite{lera2016cybersecurity} \cite{rodriguez2018message} \cite{breiling2017secure}.
In the latter case, by enhancing the middleware behaviour with some extra, such as a run-time monitor to filter out and log all the requests and operations sent in the graph \cite{Huang2014ROSRV}, an Application Level Gateway - that wraps the existing API calls to enforce authentication and authorization - that exposes a permission token to be evaluated before executing the requested operation \cite{Doczi2016Increasing}; to the extent of forking the implementation, modifying the transport mode via IPSec \cite{Sundaresan2017SecureROS}, or via a security architecture intended with the addition of x.509 certificates and authorization server \cite{ApplicationSecROS}. Unfortunately, those approaches suffered to some extent with limitations and downsides, such as a lack of flexibility (e.g. Single Point of Failure (SPOF)) and usability, which were tackled in the Secure ROS (SROS) initiative \cite{white2019sros1}. With the objective of providing additions to the ROS API and ecosystem to support modern cryptography and security measures, the project introduced new security features to the core of ROS’ codebase and, more importantly, a set of tooling to ease the burden on the developers of correctly implementing security.

With the second iteration of the framework in ROS 2, thanks to the adoption of DDS as the communication middleware\footnote{https://design.ros2.org/articles/ros\_on\_dds.html}, we observed how the inherited security measures and methodologies in the system have evolved the framework. However, we can no longer overlook how its complexity still remains prone to human error in processes such as the access control artifacts distribution \cite{white2018procedurally}, or even to overlooking exposed attack surfaces \cite{white2019network}. Moreover, keeping track of all the new pieces to the ROS puzzle became even more demanding and lengthy procedurally, requiring continuous attention and systematic security analysis--which left usability challenging \cite{phdWhite}. The work presented here addresses this challenge with a security toolset (SROS2) and a security methodology for robotics.


\section{Approach}
\showkeysfalse  
\label{sec:approach}
\showkeystrue  



A methodology inspired by the DevSecOps approach \cite{mayoral2020devsecops} is proposed to secure ROS 2 computational graphs:
(A) introspect the graph and model its security landscape to determine the necessary security policies and enclaves; implement such policies by (B) defining the authentication and (C) authorization configurations; (D) generate all the required security artifacts; (E) deploy them appropriately across robotic systems; and (F) continuously monitor the network, reverting to (A) modeling when appropriate.

\subsection{Modeling}
\showkeysfalse  
\label{approach:modeling}
\showkeystrue  

Modeling refers to the use of abstractions to aid in a thought process. In security, threat modeling aids in thinking about risks and determines the threat landscape. The output of this effort is often called the threat model. Commonly, a threat model enumerates the potential attackers, their capabilities, resources and their intended targets. In the context of robot cybersecurity, a threat model identifies security threats that apply to the robot and/or its components\footnote{both software and hardware, including computational graph resources.} while providing means to address or mitigate them for a particular use case. A threat model also provides inputs that are used to then determine a set of policy rules (or principles) that direct how ROS 2 should provide security services to protect sensitive and critical graph resources. When put together these policy rules are called the \emph{security policy}.

SROS2 aims to provide tools to introspect and model the security of ROS 2 computational graphs into the desired security policies. Introspection of the graph can be performed in two ways:

\begin{enumerate}
    \item By leveraging the ROS 2 API and the framework for ROS 2 command line tools (\texttt{ros2cli}), we can pull ROS \emph{Nodes}, \emph{Topics}, \emph{Services} or \emph{Actions} information (among others) from the ROS 2 graph and display these in the CLI, see Listing \ref{lst:modeling}. This allows us to get a grasp of the computational graph from a ROS 2 perspective. Other tools such as \texttt{RViz} \cite{kam2015rviz} or \texttt{rqt} help get a visual depiction of the graph and its abstractions.
    \item Monitoring network interactions at the ROS communication middleware-level can be extremely helpful to model security but incredibly cumbersome from a usability perspective unless the right tooling is provided. ROS 2 uses OMG's Data Distribution Service (DDS) \cite{dds14} as its default communication middleware, which is a complex specification. To facilitate introspection of DDS, SROS2 leverages \emph{scapy} \cite{rohith2018scapy}, a powerful interactive packet manipulation library that allows us to forge or decode network packets. Particularly, we contributed an open source \emph{scapy} dissector\footnote{see \url{https://github.com/secdev/scapy/pull/3403}} that allows us to dissect the wire-level communication protocol that is used by the default ROS 2 communication middleware: the Real-Time Publish Subscribe protocol (RTPS) \cite{rtps25}. Using this, SROS2 provides tooling that allows monitoring network interactions, capturing DDS databus information directly and displaying these for the security analyst's consumption.
\end{enumerate}

\lstset{language=bash}
\lstset{label={lst:modeling}}
\lstset{basicstyle=\ttfamily\footnotesize,
    commentstyle=\color{MidnightBlue}}
\lstset{caption={SROS2 extends ROS 2 APIs to facilitate computational graph introspection at the networking level for modeling purposes.} 
}
\lstset{escapeinside={<@}{@>}}
\begin{lstlisting}
# ROS 2 CLI API allows direct introspection
ros2 topic list
<@\textcolor{gray}{/cmd\_vel}@> <@\tikz[remember picture] \node [] (p1) {}; @>
<@\textcolor{gray}{/robot\_state\_publisher}@> <@\tikz[remember picture] \node [] (p1) {}; @>
...
ros2 node list
<@\textcolor{gray}{/turtlebot3\_diff\_drive}@> <@\tikz[remember picture] \node [] (p1) {}; @>
...
# SROS2 extensions allow introspecting DDS
ros2 security introspection
<@\textcolor{gray}{DDS endpoint detected (hostId=16974402, appId=2886795267, instanceId=10045242)}@> <@\tikz[remember picture] \node [] (p1) {}; @>
    <@\textcolor{gray}{- version: 2.4}@> <@\tikz[remember picture] \node [] (p1) {}; @>
    <@\textcolor{gray}{- vendorId: ADLINK - Cyclone DDS}@> <@\tikz[remember picture] \node [] (p1) {}; @>
    <@\textcolor{gray}{- IP: 192.168.1.34}@> <@\tikz[remember picture] \node [] (p1) {}; @>
    <@\textcolor{gray}{- transport: UDP}@> <@\tikz[remember picture] \node [] (p1) {}; @>
<@\textcolor{gray}{DDS endpoint detected (...)}@>
\end{lstlisting}

For complete threat modeling, we refer the reader to \cite{ros2securitytm} which discusses details around security modeling ROS 2 computational graphs.

\subsection{Authentication}
\showkeysfalse  
\label{approach:authentication}
\showkeystrue  

Authentication provides proof of a claimed identity ($\neq$ identification,  determination of an unknown entity). ROS 2 offloads authentication to its underlying communication middleware, DDS. By default, DDS allows any arbitrary \emph{DomainParticipant} to join any \emph{Domain} without authentication. DDS however provides the means to verify the identity of the application and/or the user that invokes operations on the DDS databus through its \emph{DDS Security} extensions \cite{ddssecurity11}.  With these, for protected DDS \emph{Domains}, a \emph{DomainParticipant} that enables authentication will only communicate with other \emph{DomainParticipants} that also have authentication enabled. 

To favour usability and reduce human errors, all implementation details of authentication in ROS 2 through DDS are abstracted away by the SROS2 tools. The appropriate artifacts for enabling authentication capabilities are produced in the \emph{Generation} step (\ref{approach:generation}) of the DevSecOps methodology, and default to the security mechanisms specified by OMG's DDS Security \cite{ddssecurity11}. In particular, each \emph{DomainParticipant} uses a Public Key Infrastructure (PKI) with a common shared Certificate Authority (CA): \texttt{Identity CA}. All participants interoperating securely must be pre-configured with \texttt{Identity CA} and have a signed certificate from it. Participants are expected to use mutual authentication through a challenge-response mechanism supported by either the Rivest Shamir Adleman (RSA) \cite{rivest1983cryptographic} or the Elliptic Curve Digital Signature Algorithm (ECDSA) \cite{johnson2001elliptic} asymmetric encryption algorithms. Shared secrets are established using using the  Diffie-Hellman (DH) \cite{diffie1976new} or Elliptic Curve DH (ECDH) (Ephemeral Mode) \cite{merkle1978secure} key agreement protocols.

Listing \ref{lst:authentication} shows an example of how SROS2 tools abstract the complexity of DDS authentication away from ROS developers. The \texttt{governance.xml} policy document is auto-generated by SROS2 and captures domain-wide security settings that include authentication aspects. Additional details about the underlying authentication process and the security artifacts are available in \cite{ddssecurity11}, \cite{ros2securityintegration} and \cite{ros2securityenclaves}. 

\lstset{language=xml}
\lstset{label={lst:authentication}}
\lstset{basicstyle=\ttfamily\footnotesize,
    commentstyle=\color{blue},
    stringstyle=\ttfamily\color{red!50!brown},}
\lstset{caption={An extract from \texttt{governance.xml} policy document generated by SROS2 illustrating domain-wide security settings such as how to handle unauthenticated participants, whether to encrypt discovery or the default rules for access to topics.} 
}
\lstset{escapeinside={<@}{@>}}
\begin{lstlisting}
...
<<@\textcolor{Plum}{allow\_unauthenticated\_participants}@>><@\textcolor{ForestGreen}{false}@></<@\textcolor{Plum}{allow\_unauthenticated\_participants}@>>
<<@\textcolor{Plum}{enable\_join\_access\_control}@>><@\textcolor{ForestGreen}{true}@></<@\textcolor{Plum}{enable\_join\_access\_control}@>>
<<@\textcolor{Plum}{discovery\_protection\_kind}@>><@\textcolor{ForestGreen}{ENCRYPT}@></<@\textcolor{Plum}{discovery\_protection\_kind}@>>
<<@\textcolor{Plum}{liveliness\_protection\_kind}@>><@\textcolor{ForestGreen}{ENCRYPT}@></<@\textcolor{Plum}{liveliness\_protection\_kind}@>>
<<@\textcolor{Plum}{rtps\_protection\_kind}@>><@\textcolor{ForestGreen}{SIGN}@></<@\textcolor{Plum}{rtps\_protection\_kind}@>>
<<@\textcolor{Plum}{topic\_access\_rules}@>>
    <<@\textcolor{Mulberry}{topic\_rule}@>>
        <<@\textcolor{Periwinkle}{topic\_expression}@>><@\textcolor{ForestGreen}{*}@></<@\textcolor{Periwinkle}{topic\_expression}@>>
        <<@\textcolor{Periwinkle}{enable\_discovery\_protection}@>><@\textcolor{ForestGreen}{true}@></<@\textcolor{Periwinkle}{enable\_discovery\_protection}@>>
        <<@\textcolor{Periwinkle}{enable\_liveliness\_protection}@>><@\textcolor{ForestGreen}{true}@></<@\textcolor{Periwinkle}{enable\_liveliness\_protection}@>>
        <<@\textcolor{Periwinkle}{enable\_read\_access\_control}@>><@\textcolor{ForestGreen}{true}@></<@\textcolor{Periwinkle}{enable\_read\_access\_control}@>>
        <<@\textcolor{Periwinkle}{enable\_write\_access\_control}@>><@\textcolor{ForestGreen}{true}@></<@\textcolor{Periwinkle}{enable\_write\_access\_control}@>>
        <<@\textcolor{Periwinkle}{metadata\_protection\_kind}@>><@\textcolor{ForestGreen}{ENCRYPT}@></<@\textcolor{Periwinkle}{metadata\_protection\_kind}@>>
        <<@\textcolor{Periwinkle}{data\_protection\_kind}@>><@\textcolor{ForestGreen}{ENCRYPT}@></<@\textcolor{Periwinkle}{data\_protection\_kind}@>>
    </<@\textcolor{Mulberry}{topic\_rule}@>>
</<@\textcolor{Plum}{topic\_access\_rules}@>>
...
\end{lstlisting}

\subsection{Authorization}

\showkeysfalse  
\label{approach:authorization}
\showkeystrue  

Authorization helps define and verify the policies that are assigned to a certain identity. Access control instead --also called permissions or privileges-- are the methods used to enforce such policies. While access control is handled by the DDS implementation, authorization policies need to be defined by the developer. SROS2 helps map these policies from the ROS 2 computational graph to the underlying DDS databus abstractions through two resources: the \texttt{Permissions CA} and a \texttt{permissions.xml} policy document. Listing \ref{lst:authorization} shows an extract from one of the policy documents that defines the authorization profile for a particular ROS 2 Node. Details about how access control is implemented by the underlying communication middleware are discussed in \cite{ddssecurity11} and \cite{ros2securityaccesscontrol}.

\lstset{language=xml}
\lstset{label={lst:authorization}}
\lstset{basicstyle=\ttfamily\footnotesize,
    commentstyle=\color{blue},
    stringstyle=\ttfamily\color{red!50!brown},}
\lstset{caption={SROS2 provides means to define authentication policies through XML files.} 
}
\lstset{escapeinside={<@}{@>}}
\begin{lstlisting}
<<@\textcolor{Plum}{profile}@> node="turtlebot3_diff_drive" ns="/">
    <<@\textcolor{Mulberry}{xi:include}@> href="common/node.xml"
      xpointer="xpointer(/profile/*)"/>
    <<@\textcolor{Mulberry}{topics}@> subscribe="ALLOW">
      <<@\textcolor{Periwinkle}{topic}@>><@\textcolor{ForestGreen}{/cmd\_vel}@></<@\textcolor{Periwinkle}{topic}@>>
    </<@\textcolor{Mulberry}{topics}@>>
    <<@\textcolor{Mulberry}{topics}@> publish="ALLOW">
      <<@\textcolor{Periwinkle}{topic}@>><@\textcolor{ForestGreen}{odom}@></<@\textcolor{Periwinkle}{topic}@>>
      <<@\textcolor{Periwinkle}{topic}@>><@\textcolor{ForestGreen}{tf}@></<@\textcolor{Periwinkle}{topic}@>>
    </<@\textcolor{Mulberry}{topics}@>>
</<@\textcolor{Plum}{profile}@>>
<<@\textcolor{Plum}{profile}@> node="turtlebot3_imu" ns="/">
...
</<@\textcolor{Plum}{profile}@>>
\end{lstlisting}

\subsection{Generation}

\showkeysfalse  
\label{approach:generation}
\showkeystrue  

\emph{Modeling} (\ref{approach:modeling}), \emph{Authentication} (\ref{approach:authentication}) and \emph{Authorization} (\ref{approach:authorization}) steps of our methodology (which follows DevSecOps in robotics \cite{mayoral2020devsecops}) help define one or multiple security policies. To help implement such policies, SROS2 provides means to automate the generation of the corresponding security artifacts and simplify the translation to the underlying DDS implementation. To do so, SROS2 maps a security policy to an \emph{enclave}: a set of ROS 2 computational graph resources that operate in the same security domain, use particular \texttt{Identity CA} and \texttt{Permissions CA} authorities, and share the protection of a single, common, continuous security perimeter. 

All secure interactions in ROS 2 computational graphs must use an enclave that contains the runtime security artifacts unique to that enclave, yet each Node may not necessarily have a unique enclave. Multiple enclaves can be encapsulated in a single security policy to accurately model the information flow control. Users can then tune the fidelity of such models by controlling at what scope enclaves are applied at deployment. For example, one unique enclave per robot, or per swarm, or per network, etc.

Listing \ref{lst:generation} shows how SROS2 tools help generate all artifacts to implement a new security policy, inferred directly from the running ROS 2 graph. For a more complex policy that involves multiple enclaves, we refer the reader to \cite{sros2tb3demo_gazebopolicy}.

\lstset{language=bash}
\lstset{label={lst:generation}}
\lstset{basicstyle=\ttfamily\footnotesize,
    commentstyle=\color{MidnightBlue}}
\lstset{caption={SROS2 provides tools to implement security policies and in ROS 2 computational graphs, generating all security artifacts necessary.} 
}
\lstset{escapeinside={<@}{@>}}
\begin{lstlisting}
# Generate a new keystore with Identify and Permission CA keys, 
# associated certificates and a authentication structure through 
# governance.xml file
ros2 security create_keystore new_keystore

# Inspect current ROS graph and produce a security policy
ros2 security generate_policy new_keystore/my_policy.xml

# Generate all security artifacts necessary to enforce the policy, 
# this includes enclaves and the access control permission files
ros2 security generate_artifacts \
  -k new_keystore \
  -p new_keystore/my_policy.xml 
\end{lstlisting}

\subsection{Deployment}

\showkeysfalse  
\label{approach:deployment}
\showkeystrue  

Deployment is a relevant phase in the DevSecOps methodology \cite{mayoral2020devsecops} and must be also exercised securely. We consider three types of deployments of both artifacts and secure information:

\begin{enumerate}
    \item {Distribution of policy artifacts}: the resulting artifacts from the \emph{Generation phase} (\ref{approach:generation}) must be securely deployed into the targeted robots and related systems. At the time of writing SROS2 does not provide any particular special utilities to deploy security artifacts. We however direct readers to the ongoing efforts to launch ROS 2 graphs remotely and in multi-machine environments \cite{multimachinedeployros2} for inspiration.
    \item {Message authentication}: verification of the Message Authentication Codes (MAC) is performed using Advanced Encryption Standard (AES) with Galois MAC (AES-GMAC). DDS security extensions abstract this away from the ROS developer.
    \item {Encryption of secure DDS interactions}: authenticated symmetric cryptography governs all DDS interactions within a security policy using also AES in Galois Counter Mode (AES-GCM). Similar to message authentication, DDS abstracts this away from the developer and is enabled automatically provided that the security policy is configured appropriately.
\end{enumerate}

\subsection{Monitoring and mitigation}

\showkeysfalse  
\label{approach:monitoring}
\showkeystrue  

The last phase in the DevSecOps methodology leads to a never ending loop of continuous \emph{Monitoring, mitigation} (\ref{approach:monitoring}) and \emph{Modeling} (\ref{approach:modeling}). This way, security in ROS 2 computational graphs becomes a moving target, a process--one that demands continuous assessments as changes occur in the robots, the network, or as new security flaws are discovered affecting the running systems.

SROS2 provides tools for monitoring running ROS 2 graphs and detecting possible flaws. Listing \ref{lst:monitoring} shows an example:



\lstset{language=bash}
\lstset{label={lst:monitoring}}
\lstset{basicstyle=\ttfamily\footnotesize,
    commentstyle=\color{MidnightBlue}}
\lstset{caption={SROS2 provides tools to dissect DDS interactions, extract key information and map it to outstanding security flaws affecting DDS.} 
}
\lstset{escapeinside={<@}{@>}}
\begin{lstlisting}
# monitor the network segment for vulnerabilities 
# affecting DDS participants
ros2 security monitor
<@\textcolor{gray}{sniffing the DDS network...}@>
<@\textcolor{RedViolet}{Vulnerable DDS endpoint found (hostId=16974402, appId=2886795267, instanceId=10045242)}@>
    <@\textcolor{RedViolet}{- vendorId: Real-Time Innovations, Inc. - Connext DDS}@>
    <@\textcolor{RedViolet}{- version: 6.0.1.25}@>
    <@\textcolor{RedViolet}{- CVE IDs:}@>
        <@\textcolor{RedViolet}{* CVE-2021-38487}@>
        <@\textcolor{RedViolet}{* CVE-2021-38435}@>
\end{lstlisting}

\begin{landscape}
\vspace*{\fill} 
    \begin{figure}[!htb]
        \centering
        \includegraphics[
            width=\linewidth,
            trim= 0 300 1500 0,
            clip
        ]
        {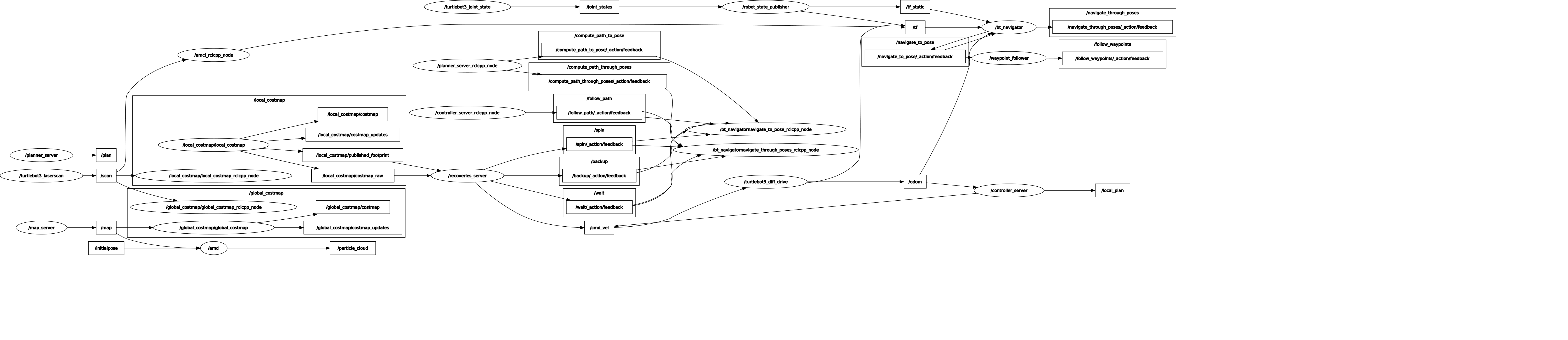}
        \caption{A subset of the computational graph of the simulated robot including sensor and control topics, adapted from \cite{mayoral2022sros2}.
        }
        \label{fig:navigationgraph}
        \vfill 
    \end{figure}
\vspace*{\fill} 
\end{landscape}


\section{Application and analysis}

To apply the DevSecOps methodology, as defined originally in \cite{mayoral2020devsecops}, we demonstrate the application of SROS2 using two of the most commonly used frameworks in ROS 2, the \texttt{navigation2} \cite{macenski2020marathon2} and \texttt{slam\_toolbox} \cite{macenski2021slamtoolbox} stacks. Particularly, the Navigation2 project\footnote{\url{https://github.com/ros-planning/navigation2}} provides a software stack including path planning algorithms and behavioral navigation servers that can be seamlessly integrated with existing sensor perception pipelines, localization and mapping services, and drivetrain velocity controllers to support various mobile robotic applications. While Navigation2 remains mostly agnostic of robotic platforms, we selected the widely accessible and community supported TurtleBot3\footnote{\url{https://www.robotis.us/turtlebot-3}} as the target robot for analysis--consisting of a differential drive, circular base footprint, and ground level 2D scanning LIDAR. The application case study is depicted in Figure \ref{fig:navigationgraph}.

To start, we begin with the modeling (Subsection \ref{approach:modeling}), authentication (Subsection  \ref{approach:authentication}) and authorization (Subsection  \ref{approach:authorization}) phases in order to bootstrap an initial security policy that captures the minimal spanning set of security measures required for the nominal function of the application across the distributed computation graph. We can either first bring up the ROS 2 application under a controlled network environment with security mode disabled, or provision an initial keystore enclave with temporary key-material and only access control governance disabled; the first option includes minimal setup while being more transparent to debug, while the later is advantageous in modeling policies directly from field deployments across untrusted networks.

With the ROS 2 application running, SROS2 can capture the topology of the computation graph to populate the permission profiles within the initial policy, registering each active ROS 2 node and respective topic publication and subscription. A limitation in SROS2’s current snapshot approach however is in accurately modeling more ephemeral resource access events, such as service clients or action requests. While ROS 2’s internal graph API (that SROS2 uses to sample topology measurements) provides a middleware agnostic interface, the observation window is only instantaneous and can easily miss asynchronous resource access events.

Given the graph API limitations, it’s often necessary to iteratively test the generated policy by using it to update the signed permission and governance files and relaunch the application with access control enabled. For moderate to advanced applications such as those relying on Navigation2, permission access denied errors may inevitably be encountered. With ROS 2 however, such events can be logged and aggregated into policy refinement, specifying the node and resource namespaces denied.

After iterative policy refinement, once the tested application is fully functional with enforced access control, the policy can then be further optimized. Such policy optimizations include sorting common permission patterns into smaller sub-profiles, being more manageable to audit and modularly reusable across repeating permission sets in a global policy. We demonstrate it in \cite{sros2tb3demo_policies}. This auditing process also provides an opportunity to assess the granularity of the policy as well, from both permission Access Control (AC) and Information Flow Control (IFC) perspectives.

While the minimal spanning set of AC permissions may be optimally secure in terms of the Principle of Least Privilege \cite{schneider2003least}, it may not be optimally usable for a target application domain. Though most computation graphs in ROS 2 are largely static at runtime, cases where resource namespaces change over the application’s lifecycle do exist. For example, multi-robot systems may fluctuate as agents enter or exit networks for missions or maintenance. Additionally, node namespaces sometimes include sequence numbers to ensure namespace uniqueness. To accommodate such scenarios, permissions could be modified to include wildcards as necessary. While static permissions are straightforward to interpret and less likely to inadvertently introduce policy flaws, wildcarding select permissions provides a usable compromise when required.

When auditing from an IFC perspective, optimizing the policy into assorted enclaves becomes a key consideration. As all ROS 2 nodes composed into a shared process share a common DDS context, they subsequently share the same security enclave or set of permission profiles. This of course is inherently coupled with how the application is architected and to be deployed across a distributed system. As such, security requirements for IFC may then instead dictate aspects of the application's designs. The degree of granularity of IFC sought then dictates the allotment of enclaves used to contain sub-profiles for the application’s policy.

In the case of Navigation2 and its large degree of coupling and composition of nodes, the planning stack derived from a single source tree is perhaps best relegated to its own enclave, while still being readily separable from any other enclave dedicated to perception or control nodes. Admittedly, such auditing procedures in determining the allotment of enclaves remains rather ambiguous for users, and so presents another area of ergonomics for SROS2 to help automate or advise through formal analysis.

The source code of the demonstrated application case study is available at \cite{sros2_casestudy}. The resulting security policies of applying the DevSecOps methodology are also available in the same repository and show various profiles that result from a systematic assessment.

\section{ROS 2 over VPN: Tunneling Communications}

\subsection{Challenges in Securing ROS 2 Communications}

ROS 2, built upon the Data Distribution Service (DDS), faces significant security challenges. Despite the existence of DDS Security extensions \cite{ddssecurity11}, these measures are still in their infancy concerning robust security implementations. As outlined in upcoming Chapter \ref{chapter:attacking_software}, modern DDS implementations exhibit numerous vulnerabilities, some of which result from collaborative research by security experts. These vulnerabilities open up potential attack vectors for malicious actors, compromising the integrity of ROS 2 systems.

Furthermore, the DDS Security Specification itself is plagued with unresolved security issues, detailed at \url{https://issues.omg.org/issues/spec/DDS-SECURITY}. This list highlights various bugs and loopholes that further exacerbate the security risks associated with ROS 2, making it difficult to assert that ROS 2, in its current state, can be deemed secure.

\subsection{ROS 2 over VPN: A Layered Security Approach}

Given the aforementioned security concerns, a practical solution is to encapsulate ROS 2 and DDS communications within an additional security layer. This approach involves tunneling all ROS 2 communications through a secure channel using a Virtual Private Network (VPN). This method, referred to as \emph{ROS 2 over VPN}, provides an extra layer of security by encrypting the data transmitted over the network, thereby mitigating the risks associated with potential vulnerabilities in DDS implementations and specifications.

\subsection{Implementation}

The concept of \emph{ROS 2 over VPN} was researched and successfully demonstrated in the FogROS2 project \cite{ichnowski2023fogros}, which builds upon the original contributions and principles presented earlier at \cite{mayoral2020devsecops}. Figure \ref{figure:fogros2} illustrates the architecture employed in this implementation. By routing ROS 2 communications through a VPN, data privacy is enhanced, and the network is safeguarded against eavesdropping and other forms of cyberattacks. FogROS2 relies on a series of extensions of the ROS 2 launch system. We implemented the FogROS2 launch sequence, shown in Fig.~\ref{figure:fogros2} that starts after processing the launch script.  

The steps FogROS2 takes are:
(1)~connect to the cloud provider through its programmatic interface to create and start a new instance along with setting up security groups to isolate from other cloud computers, and generating secure communication key pairs;
(2)~install the ROS libraries and dependencies on the cloud machine needed for the robot application to run in the cloud;
(3)~set up virtual private networking (VPN) on robot and cloud machine to secure the ROS\,2 DDS communication between them;
(4)~copy the workspace and all of its node software, from the robot to the cloud machine;
(5)~configure the DDS provider's discovery mechanism to work across the VPN;
(6)~launch docker instances;
(7)~launch cloud-based nodes; and finally,
(8)~launch nodes on the robot.

Once the launch process is complete, the nodes running on the robot and on the cloud machine(s) securely communicate and interact with each other---and the only change needed was a few lines of the launch script.

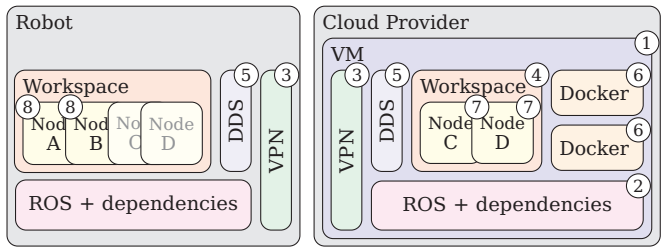
\begin{figure}[h!]
    \centering
\begin{tikzpicture}[node distance=6pt,>=stealth',font=\footnotesize,
   rosblock/.style={draw, rectangle, rounded corners, align=center, inner sep=3pt},
   roslabel/.style={inner sep=0pt, xshift=3.2pt},
   rosseq/.style={circle, inner sep=1pt, fill=white, draw=black, xshift=-2pt, yshift=-2pt, font=\scriptsize}]

\newcommand{\sep}{3pt}



\node[rosblock, text width=102pt, align=left, inner sep=4pt, minimum height=90pt, xshift=-1pt, fill=black!10] (robot) {};
\node[roslabel, below=3.2pt of robot.north west, anchor=north west] { Robot };

\node[rosblock, text width=122pt, align=left, inner sep=4pt, minimum height=90pt, fill=black!10, right=3.4in of robot.north west, anchor=north east] (cloud) {};
\node[roslabel, below=3.2pt of cloud.north west, anchor=north west] { Cloud Provider };
\draw let \p1 = (cloud.north west), \p2 = (cloud.south east) in
 node[rosblock, below=12pt of cloud.north west, anchor=north west, text width=(\x2 - \x1 - 12.8pt), align=left, xshift=3.2pt, inner sep=3pt, fill=blue!10, minimum height=(\y1 - \y2 - 12pt - 3.2pt)] (vm) {};
\node [roslabel, below=3.2pt of vm.north west, anchor=north west] { VM };
 
\draw let \p1 = (vm.north west), \p2 = (vm.south east) in
 node[rosblock, text width=(\y1 - \y2 - 12pt - 3*3.2pt), above=3.2pt of vm.south west, anchor=north west, xshift=3.2pt, rotate=90, fill=green!10] (vpn) {VPN};
 
\draw let \p1 = (vm.north east) in
 node[rosblock, below=12pt of vm.north east, anchor=north east, fill=orange!10, xshift=-3.2pt, minimum height=17pt] (docker1) { Docker };
\node [rosblock, below=3.2pt of docker1.south west, anchor=north west, minimum height=17pt, fill=orange!10] (docker2) { Docker };

\draw let \p1 = (vpn.south east), \p2 = (docker2.south west) in
 node[rosblock, right=3.2pt of vpn.south east, anchor=north east, text width=(\y1 - \y2 - 2*3.2pt + 0.4pt), fill=blue!5, rotate=90] (dds) { DDS };
 
\draw let \p1 = (dds.south east), \p2 = (docker2.south west) in
 node[rosblock, right=3.2pt of dds.south east, anchor=north west, text width=(\x2 - \x1 - 4*3.2pt), minimum height=(\y1 - \y2), fill=red!10] (workspace) {};
\node [roslabel, below=3.2pt of workspace.north west, anchor=north west] {Workspace};

\draw let \p1 = (workspace.north west), \p2 = (workspace.south east) in node [rosblock, below=12pt of workspace.north west, anchor=north west, fill=yellow!10, xshift=3.2pt, minimum height=(\y1 - \y2 - 12pt - 0.2pt - 3.2pt)] (node1) { \scriptsize Node \\[-2pt] C };
\draw let \p1 = (workspace.north west), \p2 = (workspace.south east) in node [rosblock, below=12pt of workspace.north east, anchor=north east, fill=yellow!10, xshift=-3.2pt, minimum height=(\y1 - \y2 - 12pt - 0.2pt - 3.2pt)] (node2) { \scriptsize Node\\[-2pt] D };

\draw let \p1=(dds.north west), \p2=(vm.south east) in
 node [rosblock, below=3.2pt of dds.north west, anchor=north west, text width=(\x2 - \x1 - 3*3.2pt), minimum height=(\y1 - \y2 - 0.4pt - 2*3.2pt), fill=magenta!10] (ros) { ROS + dependencies };
 
\draw let \p1=(vpn.north west), \p2=(vpn.north east) in
  node [rosblock, rotate=90, below=12pt of robot.north east, xshift=(-4*3.2pt + 0.8pt), yshift=3.2pt, anchor=south east, text width=(\y2 - \y1 - 2*3.2pt), fill=green!10] (robot vpn) { VPN };

\draw let \p1 = (vpn.south east), \p2 = (docker2.south west) in
 node[rosblock, left=3.2pt of robot vpn.north east, anchor=south east, text width=(\y1 - \y2 - 2*3.2pt + 0.4pt), fill=blue!5, rotate=90] (robot dds) { DDS };
 
\draw let \p1=(ros.north west), \p2=(ros.south east) in
  node [rosblock, left=3.2pt of robot vpn.north west, anchor=south east, text width=(\x2 - \x1 - 20pt), minimum height=(\y1 - \y2), fill=magenta!10] (robot ros) {ROS + dependencies};
  
\draw let \p1=(robot dds.north east), \p2=(robot dds.south west), \p3=(robot ros.north west), \p4=(robot ros.south east) in node [rosblock, left=3.2pt of robot dds.north east, anchor=north east, minimum width=(\x4 - \x3 + \x1 - \x2 - 3.2pt), minimum height=(\y1 - \y2), fill=red!10] (robot workspace) {};
\node [roslabel, below=3.2pt of robot workspace.north west, anchor=north west] { Workspace };

\draw let \p1 = (robot workspace.north west), \p2 = (robot workspace.south east) in node [rosblock, below=12pt of robot workspace.north west, anchor=north west, fill=yellow!10, xshift=3.2pt, minimum height=(\y1 - \y2 - 12pt - 0.2pt - 3.2pt)] (robot node1) { \scriptsize Node \\[-2pt] A };
\draw let \p1 = (robot workspace.north west), \p2 = (robot workspace.south east) in node [rosblock, below=12pt of robot workspace.north west, anchor=north west, fill=yellow!10, xshift=(3.2pt + 16pt), minimum height=(\y1 - \y2 - 12pt - 0.2pt - 3.2pt)] (robot node2) { \scriptsize Node \\[-2pt] B };
\draw let \p1 = (robot workspace.north west), \p2 = (robot workspace.south east) in node [rosblock, below=12pt of robot workspace.north west, anchor=north west, fill=yellow!5, xshift=(3.2pt + 32pt), minimum height=(\y1 - \y2 - 12pt - 0.2pt - 3.2pt)] { \scriptsize \color{gray} Node \\[-2pt] \color{gray} C };
\draw let \p1 = (robot workspace.north west), \p2 = (robot workspace.south east) in node [rosblock, below=12pt of robot workspace.north east, anchor=north east, fill=yellow!5, xshift=-3.2pt, minimum height=(\y1 - \y2 - 12pt - 0.2pt - 3.2pt)] { \scriptsize \color{gray} Node \\[-2pt] \color{gray} D };

\node [rosseq, right=0pt of vm.north east, anchor=center] { 1 };
\node [rosseq, right=0pt of ros.north east, anchor=center] { 2 };
\node [rosseq, right=0pt of vpn.south east, anchor=center] { 3 };
\node [rosseq, right=0pt of robot vpn.south east, anchor=center] { 3 };
\node [rosseq, right=0pt of workspace.north east, anchor=center] { 4 };
\node [rosseq, right=0pt of dds.south east, anchor=center] { 5 };
\node [rosseq, right=0pt of robot dds.south east, anchor=center] { 5 };
\node [rosseq, right=0pt of docker1.north east, anchor=center] { 6 };
\node [rosseq, right=0pt of docker2.north east, anchor=center] { 6 };
\node [rosseq, right=0pt of node1.north east, anchor=center] { 7 };
\node [rosseq, right=0pt of node2.north east, anchor=center] { 7 };
\node [rosseq, right=0pt of robot node1.north west, xshift=4pt, anchor=center] { 8 };
\node [rosseq, right=0pt of robot node2.north west, xshift=4pt, anchor=center] { 8 };

\begin{pgfonlayer}{background}

  

\end{pgfonlayer}

\end{tikzpicture}




    \caption{FogROS2 proposes an additional security layer to ROS 2 communications using a VPN, reproduced from \cite{ichnowski2023fogros}.
    }
    \label{figure:fogros2}
\end{figure}

The use of VPN in this context offers several advantages:

\begin{enumerate}
    \item Enhanced Data Encryption: VPN encrypts all data transmitted between ROS 2 nodes, ensuring that sensitive information remains confidential even if intercepted.

    \item Network Isolation: By creating a virtual network overlay, ROS 2 communications are isolated from other network traffic, reducing the exposure to external threats.

    \item Access Control: VPNs can be configured to allow access only to authenticated and authorized devices and users, providing a controlled environment for ROS 2 operations.

    \item Flexibility and Scalability: VPN solutions can be easily scaled and adapted to various deployment scenarios, from small-scale laboratory setups to large-scale industrial applications.

\end{enumerate}

In summary, while the native security features of ROS 2 and DDS are still maturing, employing a VPN to tunnel communications offers an effective interim solution to address the current security limitations. This approach provides robust encryption, enhanced network control, and improved data privacy, thereby fortifying ROS 2 applications against potential cybersecurity threats. 

However, this added security does not come without trade-offs. Notably, the introduction of a VPN can impact the performance of communications in ROS 2 systems. This topic will further expanded in Chapter \ref{chapter:benchmarking}.


\section{Conclusion}

In this chapter we presented various approaches to add security to ROS 2 in a scalable manner. Namely, SROS2, a series of developer tools focused on usable security that allow adding security capabilities to ROS 2 computational graphs. We introduce a methodology around these tools consisting of 6 basic steps  and aligning to the common DevSecOps flows \cite{mayoral2020devsecops}: (A) introspect the computational graph and model its security to determine the necessary security policies and enclaves; (B) define authentication and (C) authorization configurations; (D) generate all the required security artifacts for implementing such policies; (E) deploy them appropriately across robotic systems; and (F) continuously monitor the network, reverting to (A) modeling when needed. SROS2 facilitates each one of these steps by integrating itself tightly into the usual ROS 2 development flows.

We present an application case study discussing how to propose a secure architecture for the TurtleBot3 robot using the \texttt{navigation2} and \texttt{slam\_toolbox} stacks. This is of special interest since it aligns to the software architecture that many industrial and professional robots are using today, given the popularity of these packages.

We introduce security as a process in robotics and correspondingly, the work herein aims to pave the way for enabling security processes, particularly in ROS 2. Alongside the never-ending reality of security, we acknowledge that SROS2 has various limitations that deserve further attention and improvements. Some of these include the lack of granularity of security configurations in the current abstractions, which makes it difficult to configure encryption and authentication options separately. Others refer to the lifecycle management of security artifacts, including updating certificates and keys, wherein secure deployment plays a key role. We are particularly keen on improving SROS2 mechanisms in the future to ensure secure lifecycles while minimizing the downtime impact in ROS 2 graphs. 
Promising directions for future work also include the development of more advanced monitoring and introspection capabilities, the extension of SROS2 to other communication middlewares (beyond DDS) and finally, the continuous improvement of the usability of the tools. For this, we believe that the use of Graphical User Interfaces (GUIs) represents an interesting opportunity to further facilitate SROS2 usability to non-roboticists. On top of SROS2, we also discuss the shortcomings of the DDS Security extensions \cite{ddssecurity11} and how additional layers of security via VPN could help mitigate them.

The work presented here aims to inspire groups in robotics to add security to their robotic computational graphs. We look forward to security in robotics becoming more usable and accessible, minimizing the threat landscape that lies before us now, and closing the window of opportunity for bad actors. In the next chapter we will discuss the computational performance impact of adding security to ROS robotic systems and how it can be estimated experimentally with state-of-the-art benchmarking methods.

    \chapter{The Cost of Security: Benchmarking Robot Computations}
\label{chapter:benchmarking}

\epigraph{\textbf{Contributions}: \emph{Leading author at \cite{mayoral2023robotperfIROS}, both scientifically and implementation-wise. Rallied the robotics community towards its adoption and through \href{https://www.youtube.com/playlist?list=PLf4Fnww4KiFeiP1fNQXgJhyuEI760NVIl}{recorded online sessions}, throughout many months. Funded contributors and engaged other researchers and industry to support the final research product. This work is open source.}}{}

%

\section{Introduction}
As discussed in the previous chapter, while the inherent security mechanisms of ROS 2 and DDS are undergoing development and refinement, the strategy of tunneling communications through a VPN emerges as an effective interim measure to mitigate existing security vulnerabilities. This method introduces an additional layer of robust encryption, affording enhanced control over the network and bolstering data privacy. As a result, ROS 2 applications gain a fortified defense against a spectrum of cybersecurity threats. However, this added security does not come without trade-offs. Notably, the introduction of a VPN can impact the performance of communications in ROS 2 systems.

\begin{mybox}
{\textbf{\textit{VPNs in robotic deployments and performance challenges}}}

As robotics increasingly integrate into industries like healthcare and manufacturing, VPNs are becoming essential for remote operation and access. This shift presents performance challenges due to VPN overhead, especially notable in use cases like teleoperated surgical robots and industrial manipulators. The initial performance hurdle arises from the inherent computing load of VPNs in linking robotic systems.

A second, more complex challenge emerges with the implementation of finer security measures. Traditional VPN setups in robotics often adopt a \emph{'castle-and-moat'} approach, where insiders within each location’s VPN are considered trusted, often lacking encryption and access control among them. However, a more secure, \emph{'zero-trust'} model is gaining traction. This model treats all participants, even within the same network, as potential risks, mandating encryption and access control for each interaction. Although this significantly bolsters security, it also adds substantial computational overhead, intensifying the performance challenges introduced by VPNs in robotic deployments.
\end{mybox}

The performance downgrade can manifest in various ways, primarily due to the additional processing required for encrypting and decrypting data, as well as the overhead introduced by the VPN protocol itself. While VPN solutions like WireGuard \cite{donenfeld2017wireguard} are designed for efficiency and minimal overhead, they inevitably add latency to communications. In real-time robotic applications, where milliseconds can be critical, this added latency can affect the responsiveness and synchronization of robotic components. Moreover, the throughput of data can be constrained by the VPN, potentially limiting the bandwidth available for intensive data streams, such as video feeds or large sensor data arrays.

Another aspect to consider is the complexity introduced by VPN routing and network management. While VPNs can enhance security by segmenting and controlling network traffic, they also add layers of network configuration that can be challenging to optimize for high-performance robotic applications. This complexity can lead to inefficiencies in network routing, further impacting communication speed and reliability.

Therefore, while employing a VPN provides a strong interim solution for securing ROS 2 communications, it's important to carefully consider the performance implications, especially in scenarios where real-time processing and high data throughput are essential. Future advancements in ROS 2 and DDS security features may offer more integrated solutions that balance robust security with the performance needs of advanced robotic systems. For now, the choice to use a VPN must be weighed against the specific requirements and constraints of each individual robotic application.

Beyond security itself, more generally in robotics, in order for robotic systems to operate safely and effectively in dynamic real-world environments, their computations must run at real-time rates while meeting power constraints.
Towards this end, accelerating robotic kernels on heterogeneous hardware, such as GPUs and FPGAs, is emerging as a crucial tool for enabling such performance~\cite{neuman2021robomorphic,liu2021archytas,makoviychuk2021isaac,plancher2021grid,mayoral2022robotcore,wan2022robotic,liu2021robotic}. This is particularly important given the impending end of Moore's Law and the end of Dennard Scaling, which limits single CPU performance~\cite{Esmaeilzadeh11,Venkatesh10}.

While hardware-accelerated kernels offer immense potential, they necessitate a reliable and standardized infrastructure to be effectively integrated into robotic systems. As the industry leans more into adopting such standard software infrastructure, ROS~\cite{quigley2009ros} has emerged as a favored choice. Serving as an industry-grade middleware, it aids in building robust computational robotics graphs, reinforcing the idea that robotics is more than just individual algorithms. The growing dependency on ROS~2~\cite{ros-robotics-companies}, combined with the computational improvements offered by hardware acceleration, accentuates the community's demand for a standardized, industry-grade benchmark to evaluate varied hardware solutions. Recently, there has been a plethora of workshops and tutorials focusing on benchmarking robotics applications~\cite{icra2021,icra2022-metrics,icra2022-sbd,iros2020,iros2021,iros2022,iros2023,rss2020,rss2021,rss2022,rss2023}, and while benchmarks for specific robotics algorithms~\cite{bakhshalipour2022rtrbench,neuman2019benchmarking} and certain end-to-end robotic applications, such as drones~\cite{boroujerdian2018mavbench,krishnan2022automatic,krishnan2022roofline,nikiforov2023rose}, do exist, the nuances of analyzing general ROS 2 computational graphs on heterogeneous hardware is yet to be fully understood.

\begin{figure}[tbp]
\vspace{.5em}
\centering
\includegraphics[scale=0.60]{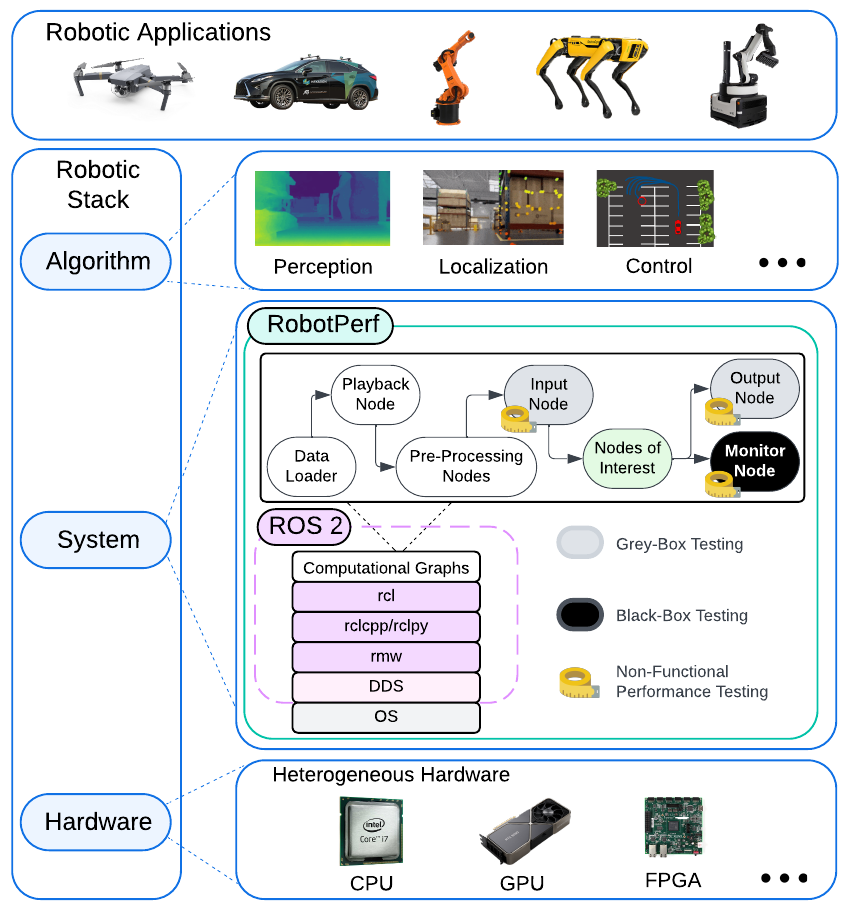}
\caption{A high level overview of RobotPerf, reproduced from \cite{mayoral2023robotperfIROS}}
\label{figure:robotperf_overview}
\end{figure}



In this chapter, we introduce \textit{RobotPerf}, an open-source and community-driven benchmarking tool designed to assess the performance of robotic computing systems in a standardized, architecture-neutral, and reproducible way, accommodating the various combinations of hardware and software in different robotic platforms.  Figure~\ref{figure:robotperf_overview} presents a high level overview of RobotPerf, which targets industry-grade real-time systems with complex and extensible computation graphs using the Robot Operating System (ROS~2) as its common baseline. Emphasizing adaptability, portability, and a community-driven approach, RobotPerf aims to provide fair comparisons of ROS 2 computational graphs across CPUs, GPUs, FPGAs and other accelerators. It focuses on evaluating robotic workloads in the form of ROS~2 computational graphs on a wide array of hardware setups, encompassing a complete robotics pipeline and emphasizing real-time critical metrics. The framework incorporates two distinct benchmarking methodologies that utilize various forms of instrumentation and ROS \emph{nodes} to capture critical metrics in robotic systems. These approaches are: black-box testing, which measures performance by eliminating upper layers and replacing them with a test application, and grey-box testing, an application-specific measure that observes internal system states with minimal interference. The framework is user-friendly, easily extendable for evaluating custom ROS 2 computational graphs, and collaborates with major hardware acceleration vendors for a standardized benchmarking approach. It aims to foster research and innovation as an open-source project. We validate the framework's capabilities by conducting benchmarks on diverse hardware platforms, including CPUs, GPUs, and FPGAs, thereby showcasing RobotPerf's utility in drawing valuable performance insights. 

RobotPerf's source code and documentation are available at \url{https://github.com/robotperf/benchmarks} and its methodologies are currently being used in industry to benchmark industry-strength, production-grade systems.


\section{Background}

\begin{table}
\vspace{7pt}
\centering
\scalebox{0.7}{
    \begin{tabular}{lccccccccc}
    \toprule
    & \multicolumn{7}{c}{\textbf{Characteristics}} \\
    \cmidrule(lr){2-8}
    & \rotatebox{90}{\scriptsize Real-time Performance Metrics}
    & \rotatebox{90}{\scriptsize Spans Multiple Pipeline Categories }
    & \rotatebox{90}{\scriptsize Evaluation on Heterogeneous Hardware}
    & \rotatebox{90}{\scriptsize Integration with ROS/ROS~2 Framework}
    & \rotatebox{90}{\scriptsize Functional Performance Testing}
    & \rotatebox{90}{\scriptsize Non-functional Performance Testing}
    & \rotatebox{90}{\scriptsize Community Led} \\
    \midrule
    OMPL Benchmark \cite{sucan2012open} & \cmark  & \xmark & \xmark & \xmark & \xmark & \cmark & \xmark \\
    MotionBenchMaker \cite{chamzas2021motionbenchmaker} & \cmark  & \xmark & \xmark & \xmark & \cmark & \cmark & \xmark \\
    OpenCollBench \cite{tan2020opencollbench} & \xmark & \xmark & \cmark & \xmark & \cmark & \xmark & \xmark \\
    BARN \cite{perille2020benchmarking} & \xmark  & \xmark & \xmark & \cmark & \cmark & \xmark & \xmark \\
    DynaBARN \cite{nair7dynabarn} & \cmark & \xmark & \xmark & \cmark & \cmark & \xmark & \xmark \\
    MAVBench \cite{boroujerdian2018mavbench} & \cmark  & \cmark & \cmark & \cmark & \cmark & \cmark & \xmark \\
    Bench-MR \cite{heiden2021bench} & \cmark  & \xmark & \xmark & \xmark & \cmark & \xmark & \xmark \\
    RTRBench \cite{bakhshalipour2022rtrbench} & \cmark  & \cmark & \xmark & \xmark & \xmark & \cmark & \xmark \\
    \midrule
    \textbf{RobotPerf} (\textbf{ours}) & \cmark  & \cmark & \cmark & \cmark & \xmark & \cmark & \cmark \\
    \bottomrule
    \end{tabular}}
\caption{Comparative evaluation of representative existing robotics benchmarks with RobotPerf across essential characteristics for robotic systems. 
}
\label{table:comparative_benchmarks}
\end{table}

%
%
%
%
%
%
%
%

\subsection*{Robotics Benchmarks}


There has been much recent development of open-source robotics libraries and associated benchmarks demonstrating their performance as well as a plethora of workshops and tutorials focusing on benchmarking robotics applications~\cite{icra2021,icra2022-metrics,icra2022-sbd,iros2020,iros2021,iros2022,iros2023,rss2020,rss2021,rss2022,rss2023}. However, most of these robotics benchmarks focus on algorithm correctness (\emph{functional} testing) in the context of domain specific problems, as well as end-to-end latency on CPUs~\cite{sucan2012open,chamzas2021motionbenchmaker,tan2020opencollbench,perille2020benchmarking,mollbenchmarking,nair7dynabarn,kingston2022robowflex,heiden2021bench,ahn2020robel,weisz2016robobench,del2006benchmarks,michel2008rat,murali2019pyrobot,james2020rlbench,leitner2017acrv,zhu2020robosuite,fan2018surreal,althoff2017commonroad}.
A few works also analyze some \emph{non-functional} metrics, such as CPU performance benchmarks, to explore bottleneck behaviors in selected workloads~\cite{bakhshalipour2022rtrbench,neuman2019benchmarking,delmerico2018benchmark}.


Recent work has also explored the implications of operating systems and task schedulers on ROS~2 computational graph performance through benchmarking~\cite{reke2020self,barut2021benchmarking,puck2020distributed,yang2020exploring,arafat2022response} as well as by optimizing the scheduling and communication layers of ROS and ROS~2 themselves~\cite{sugata2017acceleration,ohkawa2019high,choi2021picas,suzuki2018real,gutierrez2018tsn,gutierrez2018real,gutierrez2018towards,gutierrez2018time}. These works often focused on a specific context or (set of) performance counter(s).

Finally, previous work has leveraged hardware acceleration for select ROS Nodes and adaptive computing to optimize the ROS computational graphs~\cite{yamashina2015proposal,yamashina2016crecomp,podlubne2019fpga,eisoldt2021reconfros,lienen2020reconros,9415584,ohkawa2016architecture,panadda2021low,8956928,9397897,ohkawa2018fpga,9355892,amano2021dataset,nitta2018study,chen2021fogros,nvidia2022isaacros,wan2022analyzing}. However, these works do not provide comprehensive frameworks to quickly analyze and evaluate new heterogeneous computational graphs except for two works that are limited to the context of UAVs~\cite{boroujerdian2018mavbench,nikiforov2023rose}.

Research efforts most closely related to our work include \texttt{ros2\_tracing} \cite{bedard2022ros2_tracing} and RobotCore \cite{mayoral2022robotcore}. \texttt{ros2\_tracing} provided instrumentation that demonstrated integration with the low-overhead LTTng tracer into ROS 2, while RobotCore illuminates the advantages of using vendor-specific tracing to complement \texttt{ros2\_tracing} to assess the performance of hardware-accelerated ROS 2 Nodes. 
Building on these two specific foundational contributions, RobotPerf offers a comprehensive set of ROS~2 kernels spanning the robotics pipeline and evaluates them on diverse hardware.

Table~\ref{table:comparative_benchmarks} summarizes our unique contributions. It includes a selection of representative benchmarks from above and provides an evaluation of these benchmarks against RobotPerf, focusing on essential characteristics vital for robotic systems. We note that while our current approach focuses only on non-functional performance benchmarking tests, RobotPerf's architecture and methodology can be extended to also measure functional metrics.

\section{RobotPerf: Principles \& Methodology}

RobotPerf is an open-source, industry-strength robotics benchmark for portability across heterogeneous hardware platforms. This section outlines the important design principles and describes the implementation methodology.

\subsection{Non-Functional Performance Testing}


Currently, RobotPerf specializes in non-functional performance testing, evaluating the efficiency and operational characteristics of robotic systems. Non-functional performance testing measures those aspects not belonging to the system's functions, such as computational latency, memory consumption, and CPU usage. 
In contrast, traditional functional performance testing looks into the system's specific tasks and function, verifying its effectiveness in its primary goals, like the accuracy of the control algorithm in following a planned robot's path. While functional testing confirms a system performs its designated tasks correctly, non-functional testing ensures it operates efficiently and reliably.

\subsection{ROS 2 Integration \& Adaptability} 


RobotPerf is designed specifically to evaluate ROS~2 computational graphs, rather than focusing on independent robotic algorithms. We emphasize benchmarking \textit{ROS~2 workloads} because the use of ROS~2 as middleware allows for the easy composition of complex robotic systems. 
This makes the benchmark versatile and well-suited for a wide range of robotic applications and enables industry, which is widely using ROS, to rapidly adopt RobotPerf. 



\subsection{Platform Independence \& Portability} 
RobotPerf allows for the evaluation of benchmarks on a variety of hardware platforms, including general-purpose CPUs and GPUs, reconfigurable FPGAs, and specialized accelerators (e.g., ray tracing accelerators \cite{deng2017toward}). 
Benchmarking robotic workloads on heterogeneous platforms is vital to evaluate their respective capabilities and limitations. This facilitates optimizations for efficiency, speed, and adaptability, as well as fine-tuning of resource allocations, ensuring robust and responsive operation across diverse contexts.

\begin{table}[!t]
\vspace{8pt}
\centering
\footnotesize
\renewcommand{\arraystretch}{1.1} 
\setlength{\tabcolsep}{6pt} 
\begin{tabularx}{\columnwidth}{lXX}
\rowcolor{black}
\color{white}\textsc{\scriptsize Criteria} & \color{white}\textbf{Grey-Box} & \color{white}\textbf{Black-Box} \\
\textsc{\scriptsize Precision} & Utilizes tracers from in-code instrumentation. & Limited to ROS 2 message subscriptions. \\
\textsc{\scriptsize Performance} & Low overhead. Driven by kernelspace. & Restricted to ROS 2 message callbacks. Recorded by userspace processes.  \\
\textsc{\scriptsize Flexibility} & Multiple event types. & Limited to message subscriptions in current implementation. \\
\textsc{\scriptsize Portability} & Requires a valid tracer. Standard format (CTF). & Standard ROS 2 APIs. Custom JSON format. \\
\textsc{\scriptsize Ease of use} & Requires code modifications and data postprocessing.
& Tests unmodified software with minor node additions. \\
\textsc{\scriptsize Real-Robots} & Does not modify the computational graph. & Modifies the computational graph adding extra dataflow. \\
\bottomrule
\end{tabularx}
\caption{Grey-box vs. black-box benchmarking trade-offs.}
\label{table:tradeoffs}
\end{table}

\subsection{Flexible Methodology}


We offer grey-box and black-box testing methods to suit different needs. Black-box testing provides a quick-to-enable external perspective and measures performance by eliminating the layers above the layer-of-interest and replacing those with a specific test application. Grey-box testing provides more granularity and dives deeper into the internal workings of ROS 2, allowing users to generate more accurate measurements at the cost of increased engineering effort. 
As such, each method has its trade-offs, and providing both options enables users flexibility. We describe each method in more detail below and highlight takeaways in Table~\ref{table:tradeoffs}.

\subsubsection{Grey-Box Testing}

Grey-box testing enables precise probe placement within a robot's computational graph, 
generating a chronologically ordered log of critical events using a tracer that could be proprietary or open source, such as LTTng~\cite{desnoyers2006lttng}. As this approach is fully integrated with standard ROS 2 layers and tools through \texttt{ros2\_tracing}, it incurs a minimal average latency of only 3.3 \si{\micro\second}~\cite{bedard2022ros2_tracing}, making it well-suited for real-time systems. 
With this approach, optionally, RobotPerf offers specialized input and output nodes that are positioned outside the nodes of interest to avoid the need to 
instrument them.
These nodes generate the message tracepoints upon publish and subscribe events 
which are processed to calculate end-to-end latency. 

\subsubsection{Black-Box Testing}

The black-box methodology utilizes a user-level node called the \texttt{MonitorNode} to evaluate the performance of a ROS~2 node. 
The \texttt{MonitorNode} subscribes to the target node, recording the timestamp when each message is received. By accessing the propagated ID, the \texttt{MonitorNode} determines the end-to-end latency by comparing its timestamp against the \texttt{PlaybackNode}'s recorded timestamp for each message. While this approach does not need extra instrumentation, and is easier to implement, it offers a less detailed analysis and alters the computational graph by introducing new nodes and dataflow.

\subsection{Opaque Performance Tests} 

The requirement for packages to be instrumented directly within the source code poses a challenge to many benchmarking efforts. To overcome this hurdle, for most benchmarks, we refrain from altering the workloads of interest and, instead, utilize specialized input and output nodes positioned outside the primary nodes of concern. This setup allows for benchmarking without the need for direct instrumentation of the target layer. We term this methodology ``opaque tests," a concept that RobotPerf adheres to when possible. 

\subsection{Reproducibility \& Consistency} 
To ensure consistent and reproducible evaluations, RobotPerf adheres to specific common robotic dataformats. In particular, it uses ROS 2 \texttt{rosbags}, including our own available at \url{https://github.com/robotperf/rosbags}, as well third-party bags (e.g., the \texttt{r2b} dataset~\cite{nvidia_r2bdataset2023}). 


To ensure consistent data loading and finer control over message delivery rates, we drew inspiration from~\cite{ros2_benchmark}. Our computational graphs incorporate \emph{modified and improved} \texttt{DataLoaderNode} and \texttt{PlaybackNode} implementations, which can be accessed at \url{https://github.com/robotperf/ros2_benchmark}. These enhanced nodes offer improvements that report  worst-case latency and 
enable the reporting of maximum latency, introduce the ability to profile power consumption and so forth.

\subsection{Metrics}


We focus on three key metrics: latency, throughput and power consumption including energy efficiency. 
Latency measures the time between the start and the completion of a task. Throughput measures the total amount of work done in a given time for a task. Power measures the electrical energy per unit of time consumed while executing a given task. Measuring energy efficiency (or performance-per-Watt) captures the total amount of work (relative to either throughput or latency) that can be delivered for every watt of power consumed and is directly related to the runtime of battery powered robots~\cite{boroujerdian2018mavbench}. 

\subsection{Current Benchmarks and Categories}

RobotPerf \texttt{beta}~\cite{robotperf_benchmarks_repo} introduces benchmarks that cover the robotics pipeline from perception, to localization, to control, as well as dedicated benchmarks for manipulation. 
The full list of benchmarks in the \texttt{beta} release can be found in Table~\ref{table:comparative_benchmarks}. 
Aligned with our principles defined above, each benchmark is a self-contained ROS 2 package which describes all dependencies (generally other ROS packages). To facilitate reproducibility, all benchmarks are designed to be built and run using the common ROS 2 development flows (\texttt{ament} build tools, \texttt{colcon} meta-build tools, etc.).
Finally, so that the benchmarks can be easily consumed by other tools, a description of each benchmark, as well as its results, is defined in a machine-readable format. 
As such, accompanying the \texttt{package.xml} and \texttt{CMakeLists.txt} files required for all ROS packages, a YAML file named \texttt{benchmark.yaml} is  in the root of each benchmark which  describes the benchmark and includes  accepted results.

\subsection{Run Rules}

To ensure the reliability and reproducibility of the performance data, we adhere to a stringent set of run rules. 
First, tests are performed in a controlled environment to ensure that performance data is not compromised by fluctuating external parameters. 
As per best practices recommended by \texttt{ros2\_tracing}~\cite{bedard2022ros2_tracing}, we record and report settings like clock frequency and core count.
%
%
Second, we look forward to the possibility of RobotPerf being embraced by the community and have results undergo peer review, which can contribute to enhancing reproducibility and accuracy. 
%
%
Finally, we aim to avoid overfitting to specific hardware setups or software configurations by encompassing a broad spectrum of test scenarios. 

\section{Evaluation}

\begin{figure}[!t]
\centering
\vspace{2em}
\includegraphics[width=\textwidth]{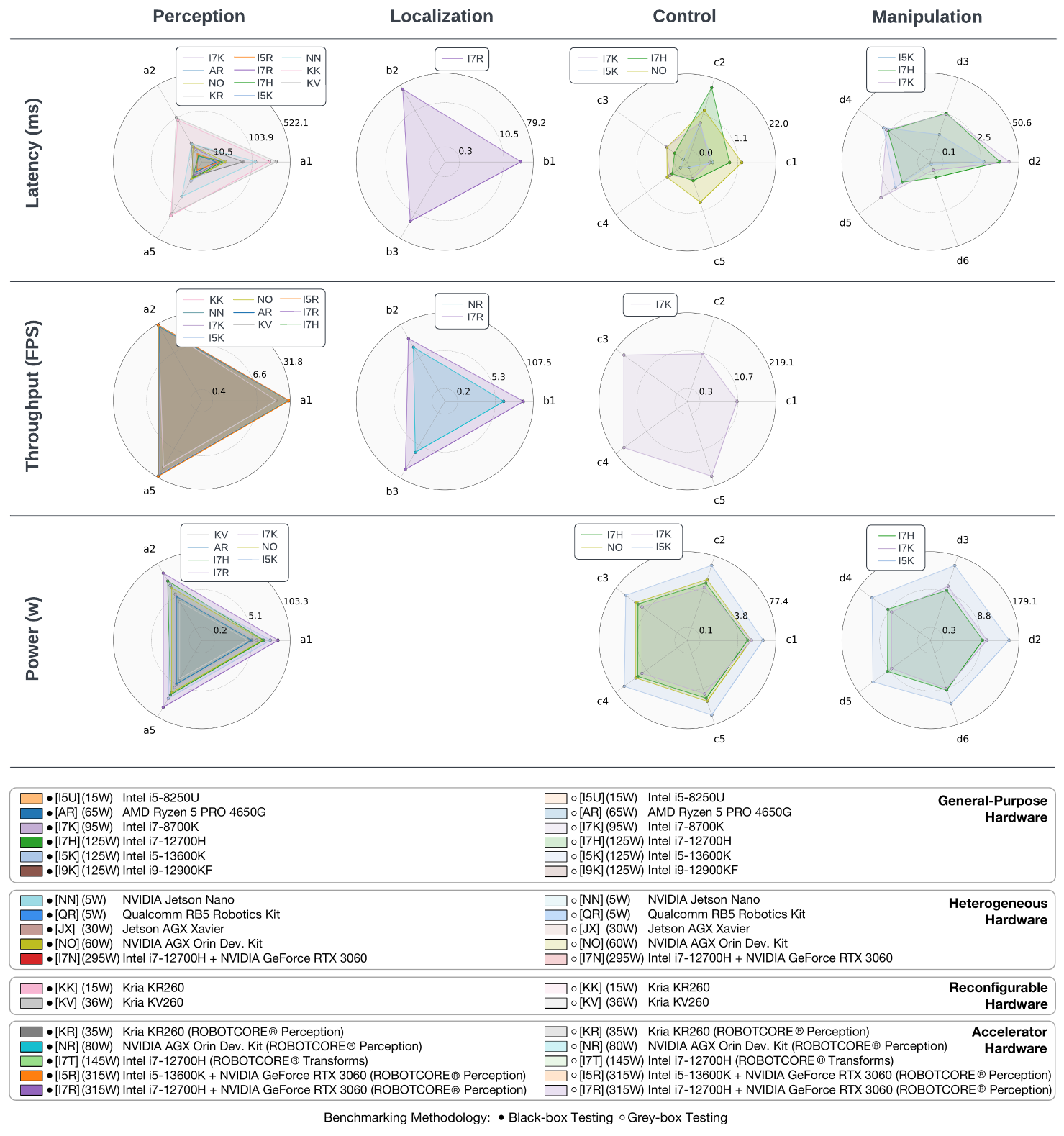}  
\caption{Benchmarking results on diverse hardware platforms across perception, localization, control, and manipulation workloads defined in RobotPerf \texttt{beta} Benchmarks. Reproduced from \cite{mayoral2023robotperfIROS}.
}
\label{figure:robotperf}
\end{figure}


%


We conduct comprehensive benchmarking using RobotPerf to evaluate its capabilities on three key aspects vital for a robotics-focused computing benchmark. First, we validate the framework's capacity to provide comparative insights across divergent heterogeneous platforms from edge devices to server-class hardware. Second, we analyze the results to understand RobotPerf's ability to guide selection of the optimal hardware solution tailored to particular robotic workloads. Finally, we assess how effectively RobotPerf reveals the advantages conferred by hardware and software acceleration techniques relative to general-purpose alternatives. All of our results and source code can be found open-source at: \url{https://github.com/robotperf/benchmarks}.

\subsection{Fair and Representative Assessment of Heterogeneity}

Assessing hardware heterogeneity in robotic applications is imperative in the ever-evolving field of robotics. Different robotic workloads demand varying computational resources and efficiency levels. Therefore, comprehensively evaluating performance across diverse hardware platforms
is crucial. 

We evaluated the RobotPerf benchmarks over a wide list of hardware platforms, including general-purpose CPUs on edge devices (e.g., Qualcomm RB5), server-class CPUs (e.g., Intel i7-8700), and specialized hardware accelerators (e.g., AMD Kria KR260). Figure \ref{figure:robotperf} illustrates benchmark performance in robotics per category of workload (perception, localization, control, and manipulation) using radar plots, wherein the different hardware solutions are depicted together alongside different robotic workloads per category.  Each hardware solution is presented with a different color. Radar plots illustrate the latency, throughput, and power consumption for each hardware solution and workload, with reported values representing the maximum across a series of runs. Each hardware platform and performance testing procedure is delineated by a separate color, with darker colors representing Black-box testing and lighter colors Grey-box testing. In the figure's key, the hardware platforms are categorized into four specific types: general-purpose hardware, heterogeneous hardware, reconfigurable hardware, and accelerator hardware. Within each category, the platforms are ranked based on their Thermal Design Power (TDP), which indicates the maximum power they can draw under load. The throughput values for manipulation tasks and power values for localization tasks have not been incorporated into the \texttt{beta} version of RobotPerf. As RobotPerf continues to evolve, more results will be added in subsequent iterations.

%
%
Given our ability to benchmark 18 platforms (bottom of Figure~\ref{figure:robotperf}), RobotPerf is capable of benchmarking heterogeneous hardware platforms and workloads, 
paving the way for community-driven co-design and optimization of hardware and software. 



\subsection{Quantitative Approach to Hardware Selection}



The rapid evolution and diversity of tasks in robotics means we need to have a meticulous and context-specific approach to computing hardware selection and optimization. A ``one-size-fits-all'' hardware strategy would be an easy default selection, but it fails to capitalize on the nuanced differences in workload demands across diverse facets like perception, localization, control, and manipulation, each exhibiting distinctive sensitivities to hardware capabilities. Therefore, a rigorous analysis, guided by tools like RobotPerf, becomes essential to pinpoint the most effective hardware configurations that align well with individual workload requirements. 

The results in Figure \ref{figure:robotperf} demonstrate the fallacy of a ``one-size-fits-all'' solution. 
For example, focusing in on the latency radar plot for control from Figure \ref{figure:robotperf} (col 3, row 1), we see that the i7-12700H (I7H) outperforms the NVIDIA AGX Orin Dev. Kit (NO) on benchmarks \texttt{C1}, \texttt{C3}, \texttt{C4}, and \texttt{C5}, but is $6.5\times$ slower on benchmark \texttt{C2}.
%
As such, by analyzing data from the RobotPerf benchmarks, roboticists can better determine which hardware option best suits their needs given their specific workloads and performance requirements.

One general lesson learned while evaluating the data is that each workload is unique, making it hard to generalize across both benchmarks and categories. To that end, RobotPerf results help us understand how the use of various hardware solutions and dedicated domain-specific hardware accelerators significantly improves the performance.

\begin{figure}[!t]
\centering
\includegraphics[width=.9\columnwidth]{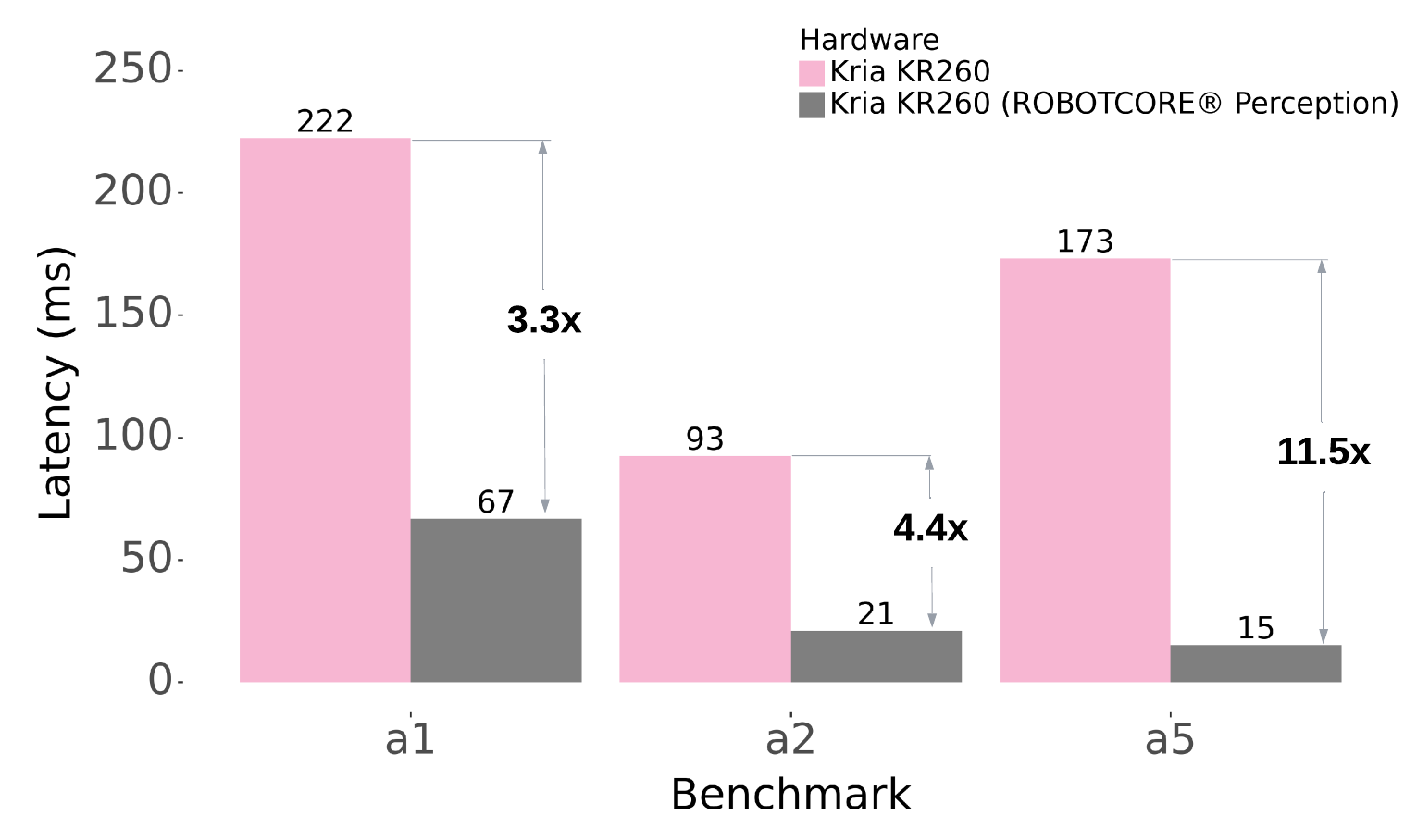}
\vspace{-5pt}
\caption{Benchmark comparison of perception latency (ms) on AMD's Kria KR260 with and without the ROBOTCORE Perception accelerator.  Reproduced from \cite{mayoral2023robotperfIROS}.}
\vspace{-1em}
\label{fig:robotperf_perception_comparison_KR260}
\end{figure}

\subsection{Rigorous Assessment of Acceleration Benefits}

In the rapidly advancing field of computing hardware, the optimization of algorithm implementations 
is a crucial factor in determining the success and efficiency of robotic applications. The need for an analytical tool, like RobotPerf, that facilitates the comparison of various algorithmic implementations on uniform hardware setups becomes important. 

Figure \ref{fig:robotperf_perception_comparison_KR260} is a simplified version of Figure \ref{figure:robotperf}, depicting AMD's Kria KR260 hardware solution in two forms: the usual hardware and a variant that leverages a domain-specific hardware accelerator (ROBOTCORE Perception, a soft-core running in the FPGA for accelerating perception robotic computations). 
The figure demonstrates that hardware acceleration can enable performance gains of as much as 11.5$\times$ (from 173~ms down to 15~ms for benchmark \texttt{a5}). 
We stress that the results obtained here should be interpreted according to each end application and do not represent a generic recommendation on which hardware should be used. Other factors, including availability, the form factor, and community support, are relevant aspects to consider when selecting a hardware solution.




\subsection{Benchmarking the security impact in ROS 2 communications}

%
%
\begin{figure*}[h!]
    \centering 
    \resizebox{0.98\textwidth}{!}{
        \begin{tikzpicture}
          \begin{axis}[
            xbar,
            xmin=0,
            y axis line style = { opacity = 0 },
            axis x line       = none,
            tickwidth         = 0pt,
            bar width         = 3em,
            enlarge y limits  = 0.2,
            enlarge x limits  = 0.02,
            symbolic y coords = {ROS 2 Intra-network,ROS 2 Intra-network + DDS Security,ROS 2 Intra-network + VPN, ROS 2 Intra-network + DDS Security + VPN},
            nodes near coords={\pgfmathprintnumber\pgfplotspointmeta~us},
            ytick             = data,                
          ]
          \addplot[black, fill=lightgray] coordinates { 
            (312,ROS 2 Intra-network) 
            (485,ROS 2 Intra-network + DDS Security) 
            (955,ROS 2 Intra-network + VPN) 
            (1310,ROS 2 Intra-network + DDS Security + VPN) 
         };
          \end{axis}
        \end{tikzpicture}
    }
    \caption{Mean latency of ROS 2 intra-network communications with various different security measures after 1000000 samples}
    \label{figure:security_overhead}
\end{figure*}
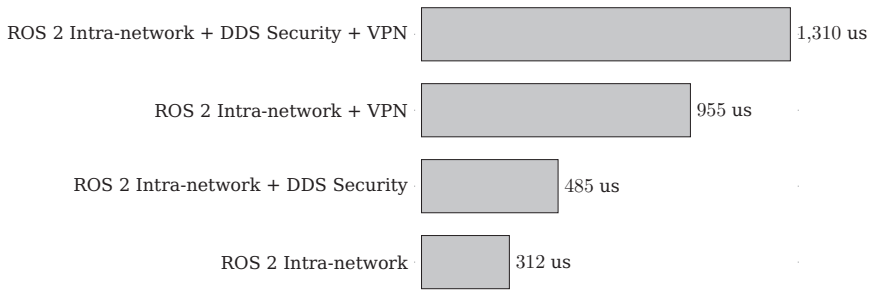

Building upon the previous RobotPerf work and results, in this subsection, we delve into a series of benchmarks designed to evaluate the latency and performance impacts of implementing security protocols within ROS 2 environments. ROS 2, as a critical framework in the robotics domain, must balance the dual demands of robust security and efficient performance. Our benchmarks focus on various configurations, encompassing communications with no security (\emph{\small ROS 2 Intra-network}), internal security measures within the communication middleware, specifically the DDS (Data Distribution Service) Security extensions (\emph{\small ROS 2 Intra-network + DDS Security}), as well as the effects of tunneling communications through VPNs (\emph{\small ROS 2 Intra-network + VPN}). Additionally, we explore the compounded impact when these two security layers are combined (\emph{\small ROS 2 Intra-network + DDS Security + VPN}). Results are depicted in Figure \ref{figure:security_overhead}.

The DDS Security extensions \cite{ddssecurity11} are integral to securing ROS 2 communications, providing mechanisms for authentication, encryption, and access control. However, the computational overhead associated with these security features can potentially affect the system's responsiveness and throughput. Our communication benchmarks displayed in Figure \ref{figure:security_overhead} illustrate the DDS Security extensions overhead (\emph{\small ROS 2 Intra-network + DDS Security}). When implemented, they lead to a mean latency which is \textbf{1.55$\times$} higher than the no security baseline (\emph{\small ROS 2 Intra-network}), offering insights into the trade-offs between security robustness and communication efficiency in ROS 2 environments.

Furthermore, the use of VPNs for remote operation and access of robotic systems introduces another layer of security, albeit with its own performance implications. VPNs can safeguard data transmission across various networks, crucial for operations like teleoperated surgery or remote industrial automation. Yet, the encryption and routing processes inherent in VPNs can introduce additional latency. Our benchmarking efforts demonstrate that tunneling ROS 2 over a modern VPN (\emph{\small ROS 2 Intra-network + VPN}) lead to a mean latency which is \textbf{3.06$\times$} higher than the no security baseline (\emph{\small ROS 2 Intra-network}), providing a comprehensive understanding of the performance costs associated with secure, remote robotic operations over VPN.

Lastly, we investigate scenarios where DDS Security extensions and VPNs are used concurrently (\emph{\small ROS 2 Intra-network + DDS Security + VPN}). This dual-layered security approach represents a robust defense mechanism, aligning with the zero-trust model. However, it also poses the greatest challenge in terms of computational overhead and potential latency increases. Our benchmarks hint towards a mean latency increase \textbf{4.19$\times$} higher than the no security baseline (\emph{\small ROS 2 Intra-network}). This result elucidates the cumulative impact of security measures, offering a holistic view of their implications on ROS 2's latency performance. This analysis will aid stakeholders in making informed decisions about security implementations, balancing the crucial aspects of security and efficiency in robotic systems.


\section{Conclusion}

RobotPerf represents an important step towards standardized benchmarking in robotics, including its security capabilities. With its comprehensive evaluation across the hardware/software stack and focus on industry-grade ROS 2 deployments, RobotPerf can pave the way for rigorous co-design of robotic hardware and algorithms. As RobotPerf matures with community involvement, we expect it to compare CPU, GPU and FPGA, exploring their power consumption and flexibility in augmenting real-world robotic computations. With a standardized robotics benchmark as a focal point, the field can make rapid progress in delivering real-time capable systems that will unlock the true potential of robotics in real-world applications. 


Through RobotPerf, we observe quantitatively significant overheads due to security implementations. Specifically, the DDS Security extensions increase latency by approximately 1.55$\times$ compared to a non-secure baseline. When implementing VPNs, this latency escalates to about 3.06$\times$ the baseline. Moreover, combining DDS Security with VPNs results in a latency increase of around 4.19$\times$. These insights clearly indicate that VPNs constitute a major bottleneck in secure robotic communications. This finding highlights the urgency of developing solutions that can accelerate these security computations, thereby meeting the real-time demands of advanced robotic systems. Addressing this limitation is crucial, especially for applications requiring real-time responses. The next chapter aims to explore strategies to mitigate these overheads, focusing on the hardware acceleration enablement of ROS computations and communications to align with the real-time requirements of robotic systems.
    \chapter{Hardware Accelerators in Robotics}
\label{chapter:accelerators}

\epigraph{\textbf{Contributions}: \emph{Leading author at \cite{mayoral2022robotcore} and sole code contributor. Self-funded work. Minor help in the article visualizations and final edition from co-authors. This work is partially open source.
}}{}

In this chapter we embark on a critical exploration of how hardware accelerators and their corresponding acceleration kernels can address the computational and communication bottlenecks identified in robotic systems, specifically those related to security, as revealed in the previous chapter. The focus is on how integrating hardware accelerators—such as GPUs and FPGAs—into robotic architectures can drastically reduce latency and improve efficiency. This chapter delves into the potential of hardware accelerators to revolutionize robotics by enabling real-time processing capabilities, a vital aspect for many robotics applications. We discuss how these accelerators can be optimized for robotic tasks, including but not limited to, security protocol computations, thereby aligning with the real-time performance requirements crucial in robotics. The chapter particularly focuses on highlighting the innovative approaches and challenges involved in integrating these accelerators into the existing ROS 2 framework, shaping the future of robotics towards more efficient and secure systems.

\section{Introduction}

Recent work has seen an explosion of specialized robotics acceleration on nontraditional computing platforms such as GPUs, FPGAs, and ASICs~\cite{wan2021survey,mayoral2021adaptivecomputing,mayoral2021kria,murray2016robot,murray2016microarchitecture,murray2019programmable,plancher2021accelerating,neuman2021robomorphic,plancher2021grid,austin2020titan,freeman2021brax,suleiman2019navion,liu2020hardware,asgari2020pisces,liu2021archytas,mayoral2017shift,williams2017model,sacks2018robox,PlancherParallelDDP,PlancherRealtimeMPC,gupta2021efficient}.
This has been sparked by the decline of Moore's Law and Dennard Scaling, which limits the performance of traditional CPU computing, positioning hardware acceleration as an emerging solution to achieve high performance and power efficiency in robotics applications.

However, this increased diversity of computing platforms leads to a dramatic growth in design space complexity that makes it difficult for users to easily deploy robotics applications on hardware accelerators without substantial expertise in each specific accelerator platform.
The Open Computing Language (OpenCL)~\cite{munshi2009opencl} is an effort to standardize hardware acceleration under a common language, but its adoption across silicon vendors has been uneven and support for it varies.
As a result, current hardware acceleration usage is often tied to a particular vendor's solutions and platforms.
This not only impedes interoperability and reuse of acceleration kernels, but presents yet another layer of complexity that users must overcome while implementing robotic systems that use acceleration kernels.
A key obstacle is that each hardware acceleration vendor provides their own framework for development, but these are often disconnected from the common tools and libraries in robotics, and mostly aimed at hardware engineers, not roboticists.

\begin{figure}[!t]
    \centering
    \includegraphics[width=0.85\columnwidth]{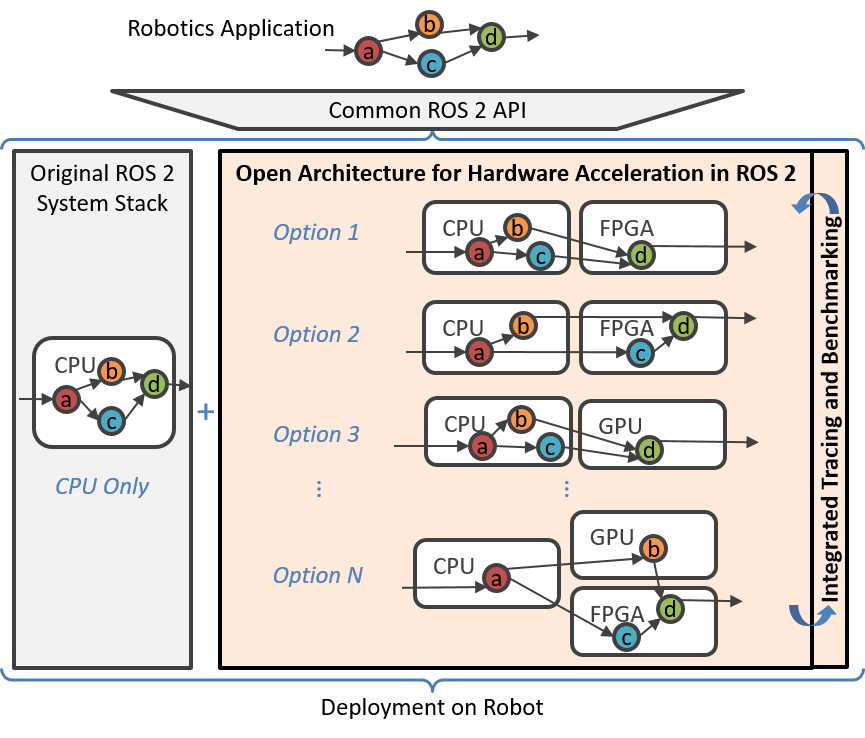}
    \caption{The open architecture for hardware acceleration in ROS~2, reproduced from \cite{mayoral2022robotcore}.
    }
    \label{figure:robotcore_overview}
\end{figure}

To address this challenge,
we present \emph{RobotCore}, an \emph{open architecture for hardware acceleration} that extends the Robot Operating System (ROS)~\cite{quigley2009ros}.
ROS is widely used by academia and industry, and early work has demonstrated its potential for hardware-accelerated robotics applications~\cite{mayoral2021adaptivecomputing,mayoral2021kria,nvidia2022isaacros}.
We facilitate this emerging direction by implementing a vendor and platform-agnostic abstraction layer for hardware acceleration in robotics. Figure~\ref{figure:robotcore_overview} depicts the open architecture for hardware acceleration in ROS~2, which extends the ROS~2 build system to support vendor and platform-agnostic deployment of robotics applications on accelerator hardware. The integrated tracing and benchmarking infrastructure enables users to analyze the system and make strategic design improvements to optimize performance.
Starting with a popular robotics API as the foundation, our ROS~2-based acceleration architecture provides a common ground for both academic researchers and silicon vendors alike to develop specialized robotics acceleration kernels, and deploy them for easy usage by a large, established user base.

Once roboticists can easily harness hardware acceleration across multiple platforms, the next major challenge is \textit{profiling and benchmarking the application}. Benchmarking is needed to determine the best mapping of the robotics computational graph to the different hardware resources available to optimize overall robot system performance.
This is a difficult task, however, since every application is different and deployment scenarios are widespread. Full end-to-end system analysis is required to understand how different implementation tradeoffs impact overall performance. To enable this analysis, we demonstrate how to leverage prior work~\cite{Bedard_2022} to benchmark accelerated ROS~2 kernels with a low-overhead framework for real-time tracing based on the Linux Tracing Toolkit next generation (LTTng)~\cite{desnoyers2006lttng}. We demonstrate analysis of a case study deployment using CPU and FPGA nodes for a simple perception pipeline. 

Using our framework and benchmarking, we diagnose that substantial latency bottlenecks in this computational graph come from inter-node interactions across ROS~2 layers in the CPU. We recognize this as an opportunity for design optimization in hardware accelerators, because interaction with the CPU should not be necessary for dataflow between nodes co-located on the same non-CPU platform (e.g., FPGA). 

Based on the benchmarking analysis, we demonstrate two novel separate paths toward \emph{hardware acceleration}: (1) kernel fusion, and (2) improved message passing. Kernel fusion results in the highest speedup, an average of $26.96\%$, but it requires manual redesign of the underlying kernels. To avoid manual redesign entirely and improve design re-use and portability, we alternatively develop an intra-FPGA ROS~2 node communication queue template that leverages AXI4-Stream interfaces~\cite{amba4axi4} and transfers data in a sequential streaming manner directly between acceleration kernels. Using this design pattern improves the overall inter-node performance in our computational graph by $24.42~\%$ on average, while requiring no change in the accelerated kernels.
This template extends to applications beyond our case study, since it can be reused for any ROS~2 inter- or intra-process communication by adapting its data types.


In summary, key contributions of this chapter are that we:
\begin{itemize}
    \item Create a new open infrastructure to increase the performance of robotics applications by enabling \emph{integration of hardware acceleration} into ROS~2 that is \emph{flexible across accelerator platforms} (e.g., FPGAs, GPUs) and system deployments (e.g., edge devices, workstations, data centers, and cloud);
    \item Expose insights into how to optimize overall system-wide performance by extending and providing a template API for \emph{low-overhead tracing and benchmarking framework} to analyze application performance across hardware accelerated ROS~2 computational graphs, laying foundation to analyze mixed-platform systems (e.g., combinations of CPU and FPGA-based nodes); and
    \item Increase ROS~2 node-to-node
    dataflow performance to achieve an average overall accelerator speedup of $24.42\%$ over CPU in our experiments by designing a template for \emph{intra-FPGA ROS~2 node communication queues}, based on insights uncovered using our open acceleration infrastructure and low-overhead benchmarking on a case study of a simple perception graph.
\end{itemize}

The core components of our architecture are disclosed under a commercially friendly open-source license and are available and maintained at the ROS~2 Hardware Acceleration Working Group GitHub organization: \url{https://github.com/ros-acceleration}.

\section{Background}

\subsection{ROS and ROS~2}

The Robot Operating System (ROS) is an open-source collection of software frameworks and tools designed to provide a \emph{structured communications layer} for robotics applications running on heterogenous computer hardware~\cite{quigley2009ros}. 

ROS applications are designed around event driven graphs of \emph{Nodes} which communicate through \emph{Messages} on various \emph{Topics}, \emph{Services}, and \emph{Actions}. Each Node can be thought of as a software process which applies an algorithm to the input message and then broadcasts the resulting output message. By managing all inter-Node communications across abstraction layers (e.g., \texttt{rclcpp}, \texttt{rcl}, \texttt{rmw}), ROS simplifies the robotic system deployment process and enables roboticists to quickly develop and test new algorithms.
ROS also provides substantial infrastructure to facilitate the automatic building, evaluation, and deployment of robotic systems, including dependency managers, package managers, build systems and tools, simulators, and visualizers.

ROS~2 is a re-design of ROS that modernizes and updates all of its components while adhering to its core design principles. ROS~2 provides a stronger partitioning of the communication middleware from the robotics logic, enabling more flexibility, scalability, and reliability~\cite{mayoral2021adaptivecomputing}.
ROS~2 also provides an updated build system, \texttt{ament}, and a new universal build tool, \texttt{colcon}. This provides a single simple interface for managing the building and deployment of complete robotics applications. Leveraging these tools, roboticists can write new algorithms and rely on ROS~2 to handle all lower level operations and middleware management.

\subsection{Hardware Acceleration for ROS and ROS~2}

There has been previous work that has focused on ways to accelerate robotics applications by developing tools and methodologies to help roboticists leverage hardware acceleration for select ROS Nodes and to optimize the ROS computational graph through adaptive computing~\cite{yamashina2015proposal,yamashina2016crecomp,podlubne2019fpga,eisoldt2021reconfros,lienen2020reconros,9415584,ohkawa2016architecture,panadda2021low,8956928,9397897,ohkawa2018fpga,9355892,amano2021dataset,nitta2018study,chen2021fogros,nvidia2022isaacros}. There has also been some work to accelerate the scheduling and communication layers used by ROS and ROS~2~\cite{sugata2017acceleration,ohkawa2019high,choi2021picas,suzuki2018real,gutierrez2018time,gutierrez2018real,gutierrez2018towards,gutierrez2018tsn}. Unfortunately, the majority of these efforts assume an end-user has substantial experience with embedded systems and embedded hardware flows, or is customized to a specific hardware acceleration board or deployment scenario.

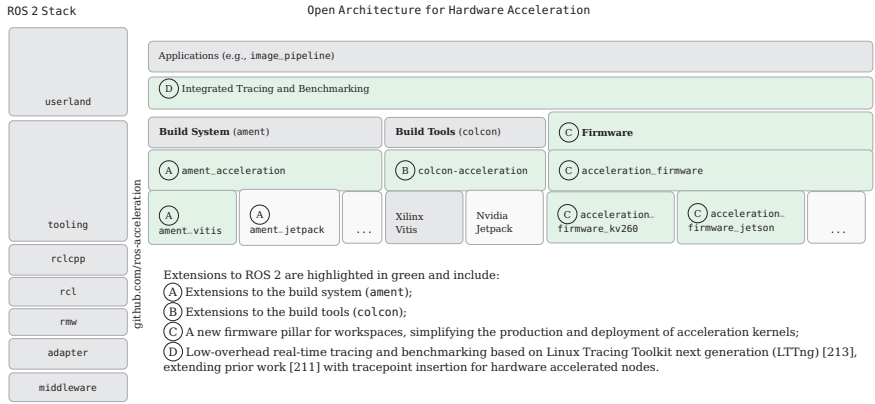
\begin{figure}[!t]
    \centering
    \scalebox{0.4}{
    \tikzset{
      state/.style={
        rectangle, rounded corners, draw=gray!80, fill=gray!20, thick,
        minimum height=2em, inner sep=10pt, text centered},
      state2/.style={
        rectangle, rounded corners, draw=black, fill=gray!70, thick,
        minimum height=2em, inner sep=10pt, text centered},
      epath/.style={draw, <->, shorten <=1pt, shorten >=1pt},
      cir/.style={draw, circle, fill, inner sep=2.5pt, text=white}
    }
    \begin{tikzpicture}[->, >=latex, line width=0.75pt]
      \draw[fill=white, color=white] (-10.8,-4.5) rectangle (9.2,4.5);
      
      \node[text width=4cm] (mug) at (-12.9,9.) {\Large \tt ROS 2 Stack};
      \node[text width=15cm] (mug) at (2.5,9.) {\Large \tt Open Architecture for Hardware Acceleration};
      
      
      \node[rotate=90] (site) at (-10.6,1) {
        \large github.com/ros-acceleration 
      };
      

    \node[state, text width=7cm, text ragged]
        (buildsystem) at (-6.4, 5.) {
            {\textbf{Build System} (\texttt{ament})}\\
    };
    \node[state, text width=7cm, text ragged, fill=green!10]
        (amentacceleration) at (-6.4, 3.75) {
            {\circledwhite{A} \tt ament\_acceleration}\\
    };

    \node[state, text width=2.2cm, text height=0.6cm, text ragged, fill=green!10]
        (amentvitis) at (-8.8, 2.2) {
            {\circledwhite{A} \tt ament\_vitis}\\
    };

    \node[state, text width=2.6cm, text height=0.6cm, fill=gray!5, text ragged]
        (amentrocm) at (-5.6, 2.2) {
            {\circledwhite{A} \tt ament\_jetpack}\\
    };

    \node[state, text width=0.6cm, text height=1cm, fill=gray!5]
        (amentmore) at (-3.2, 2.2) {
            {\tt ...}\\
    };
    
    \node[state, text width=4.6cm, text ragged]
        (colcon) at (0.2, 5.) {
        {\textbf{Build Tools} (\texttt{colcon})}\\
    };

    \node[state, text width=4.6cm, text ragged, fill=green!10]
        (colconaccel) at (0.2, 3.75) {
        {\circledwhite{B} \tt colcon-acceleration}\\
    };
    
    \node[state, text width=1.85cm, text height=0.6cm, text ragged] (vitis) at (-1.15, 2.2) {
        Xilinx \\Vitis
    };

    \node[state, text width=1.85cm, text height=0.6cm, text ragged, fill=gray!5] (cuda) at (1.5, 2.2) {
        Nvidia\\ Jetpack
    };

    \node[state, text width=10cm, text ragged, fill=green!10]
    (firmware) at (8.3, 5.) {
        {\circledwhite{C} \textbf{Firmware}}\\
    };
    \node[state, text width=10cm, text ragged, fill=green!10]
    (accelerationfirmware) at (8.3, 3.75) {
        {\circledwhite{C} \tt acceleration\_firmware}\\
    };
    \node[state, text width=3.5cm, text height=0.5cm, text ragged, fill=green!10]
        (accelerationfirmwarekv260) at (5.0, 2.2) {
            {\circledwhite{C} \tt acceleration\_\\firmware\_kv260}
    };
    \node[state, text width=3.5cm, text height=0.5cm, text ragged, fill=green!10]
        (accelerationfirmwarejetson) at (9.3, 2.2) {
            {\circledwhite{C} \tt acceleration\_\\firmware\_jetson}
    };
    \node[state, text width=1.2cm, text height=1cm, fill=gray!5]
        (accelerationfirmwaremore) at (12.45, 2.2) {
            {\tt ...}
    };

    \node[state, text height=0.15cm, text width=23.2cm, text ragged, fill=green!10]
        (tracing) at (1.7, 6.3) {
            {\circledwhite{D} Integrated Tracing and Benchmarking}\\
    };
    \node[state, text width=23.2cm, text ragged]
        (accelerationexamples) at (1.7, 7.5) {
            {Applications (e.g., {\tt image\_pipeline})}\\
    };

    \node[state, text width=3.2cm, text height=2.2cm]
      (userland) at (-12.9, 7.) {
        {\tt userland}\\
    };
    \node[state, text width=3.2cm, text height=3.2cm]
      (tooling) at (-12.9, 3.4) {
        {\tt tooling}\\
    };

    \node[state, text width=3.2cm]
      (rcl) at (-12.9, 0.8) {
        {\tt rclcpp}\\
    };

    \node[state, text width=3.2cm]
      (rcl) at (-12.9, -0.25) {
        {\tt rcl}\\
    };
    \node[state, text width=3.2cm]
      (rmw) at (-12.9,-1.25) {
        {\tt rmw}\\
    };
    \node[state, text width=3.2cm]
      (adapter) at (-12.9,-2.3) {
        {\tt adapter}\\
    };
    \node[state, text width=3.2cm]
      (middleware) at (-12.9,-3.4) {
        {\tt middleware}\\
    };

    \node[text width=23.5cm] (pilar1) at (2,-1.25) {\Large Extensions to ROS 2 are highlighted in green and include:\\\circledwhite{A} Extensions to the build system (\texttt{ament}); \\\circledwhite{B} Extensions to the build tools (\texttt{colcon}); \\\circledwhite{C} A new firmware pillar for workspaces, simplifying the production and deployment of acceleration kernels; \\\circledwhite{D} Low-overhead real-time tracing and benchmarking based on Linux Tracing Toolkit next generation (LTTng)~\cite{desnoyers2006lttng}, extending prior work~\cite{bedard2022ros2_tracing} with tracepoint insertion for hardware accelerated nodes.};

    \end{tikzpicture}}
    \caption{Overview of the components of the open architecture for hardware acceleration in ROS~2. Reproduced from \cite{mayoral2022robotcore}.}
    \label{figure:robotcore_architecture}
\end{figure}

Our proposed open architecture takes a ROS-centric approach to integrate the hardware and embedded flows directly into the core ROS~2 ecosystem. This enables a separation between those who produce accelerated kernels and those who use them by providing end-users with a build and deployment experience for hardware accelerators similar to the standard, non-accelerated ROS~2 experience.

\section{An Open Architecture for Hardware Acceleration in ROS~2}
\showkeysfalse  
\label{sec:architecture}
\showkeystrue  

Our open architecture (Figure~\ref{figure:robotcore_architecture}) extends the core ROS~2 build system and tools to provide platform-agnostic (i.e., supports  edge, workstation, data center, or cloud targets) and technology-agnostic (i.e., supports FPGAs and GPUs), hardware-accelerated ROS~2 capabilities for roboticists. We: A)~extend the ROS~2 build system, \texttt{ament}; B)~extend the ROS~2 meta build tool, \texttt{colcon}; and C)~develop integrated ROS~2 firmware extensions. We also D)~integrate a low-overhead tracing and benchmarking framework to enable the analysis of holistic application performance across ROS graphs. This section describes these extensions in detail.


\subsection{Extending the ROS~2 Build System}

\showkeysfalse  
\label{sec:architecture:step1}
\showkeystrue  

The first pillar of our open architecture, Figure~\ref{figure:robotcore_architecture}~\circledwhite{A}, allows roboticists to generate acceleration kernels directly from the ROS~2 build system (\texttt{ament}) in the same way they generate CPU binaries. To do so, the \texttt{ament\_acceleration} ROS~2 package and its extensions abstract the ROS build system from vendor-specific accelerators (e.g. FPGAs or GPUs), including their frameworks and software platforms. This allows the build system to easily support hardware acceleration across commercial solutions while using the same syntax, simplifying the work of ROS~2 package maintainers. 

Under the hood, \texttt{ament\_acceleration} abstracts away the corresponding vendor-specific firmware. For example, \texttt{ament\_vitis}\footnote{{github.com/ros-acceleration/ament\_vitis}} relies on the proprietary Xilinx Vitis~\cite{xilinx2022vitis} and on the Xilinx Runtime (XRT) library~\cite{xilinx2022xrt}. 
This simplifies the creation of acceleration kernels and separates firmware concerns from algorithm development. This way, robotics engineers can focus on improving their computational graphs with a ROS-centric development flow. Separately, hardware experts, potentially sponsored by silicon vendors, can improve acceleration kernels for a particular commercial solution. Overall, these extensions help achieve the objective of simplifying the creation and integration of acceleration kernels from different vendors into ROS~2 computational graphs.


Figure~\ref{figure:robotcore_architecture} depicts the build system extensions showing how  \texttt{ament\_acceleration} abstracts the build system from vendor-specific solutions. As an example of an alternative acceleration technology supported, \texttt{ament\_jetpack} is included and illustrates the integration of Nvidia JetPack~\cite{nvidia2022jetpack}. 


\subsection{Extending the ROS~2 Build Tools}
\showkeysfalse  
\label{sec:architecture:step2}
\showkeystrue  

The second pillar of our open architecture, Figure~\ref{figure:robotcore_architecture}~\circledwhite{B}, extends the \texttt{colcon} ROS~2 meta build tool to integrate hardware acceleration flows into the ROS~2 Command Line Interface (CLI) commands. Examples of these extensions include the selection of the target accelerator and build-time through mixins, emulation capabilities to speed-up the development process and facilitate design without access to the real hardware, raw disk image production tools, and simplified configuration of hypervisors. These extensions are implemented by the \texttt{colcon-acceleration}\footnote{{github.com/ros-acceleration/colcon-acceleration}} ROS~2 package. As in Section~\ref{sec:architecture:step1}, \texttt{colcon\_acceleration} further enables roboticists to leverage hardware accelerators while using standard ROS~2 commands and flows.

\subsection{ Adding Firmware Extensions}
\showkeysfalse  
\label{sec:architecture:step3}
\showkeystrue  

Represented by the abstract \texttt{acceleration\_firmware} ROS package and its corresponding specializations (e.g. \texttt{acceleration\_firmware\_kv260}\footnote{{github.com/ros-acceleration/acceleration\_firmware\_kv260}} for the Xilinx Kria KV260 board), the third pillar of our open architecture, Figure~\ref{figure:robotcore_architecture}~\circledwhite{C}, firmware extensions, are meant to provide firmware artifacts for each supported technology solution. This again simplifies the process for ROS package consumers and maintainers, and further aligns hardware acceleration workflows with typical ROS development flows.
Each ROS~2 workspace can leverage multiple firmware packages, but can only use one at a time. As \texttt{colcon\_acceleration} supports the selection of the active firmware in the ROS workspace, by separating the firmware out into their own packages, our open architecture enables silicon vendors to maintain an \texttt{acceleration\_firmware\_<solution>} package that automatically integrates into standard ROS~2 workflows. 

\subsection{Low-Overhead Real-Time Tracing \& Benchmarking}

\showkeysfalse  
\label{sec:architecture:step4}
\showkeystrue  

In the context of hardware acceleration in robotics, it is fundamental to be able to inspect performance improvements. To that end, it is important to benchmark and trace the system. Benchmarking is the process of running a computer program to assess its relative performance, whereas tracing is a technique used to understand what is happening in a system while it is running. Tracing helps determine which pieces of a Node are consuming more compute cycles or generating indeterminism, and are thereby good candidates for hardware acceleration. Benchmarking instead helps investigate the relative performance of an acceleration kernel versus its CPU scalar computing baseline. Similarly, benchmarking also helps with comparing acceleration kernels across different hardware acceleration technology solutions (e.g., Kria KV260 vs. Jetson Nano) and across kernel implementations within the same hardware acceleration technology solution.

\begin{figure}[!t]
    \centering
    \includegraphics[width=1.0\columnwidth]{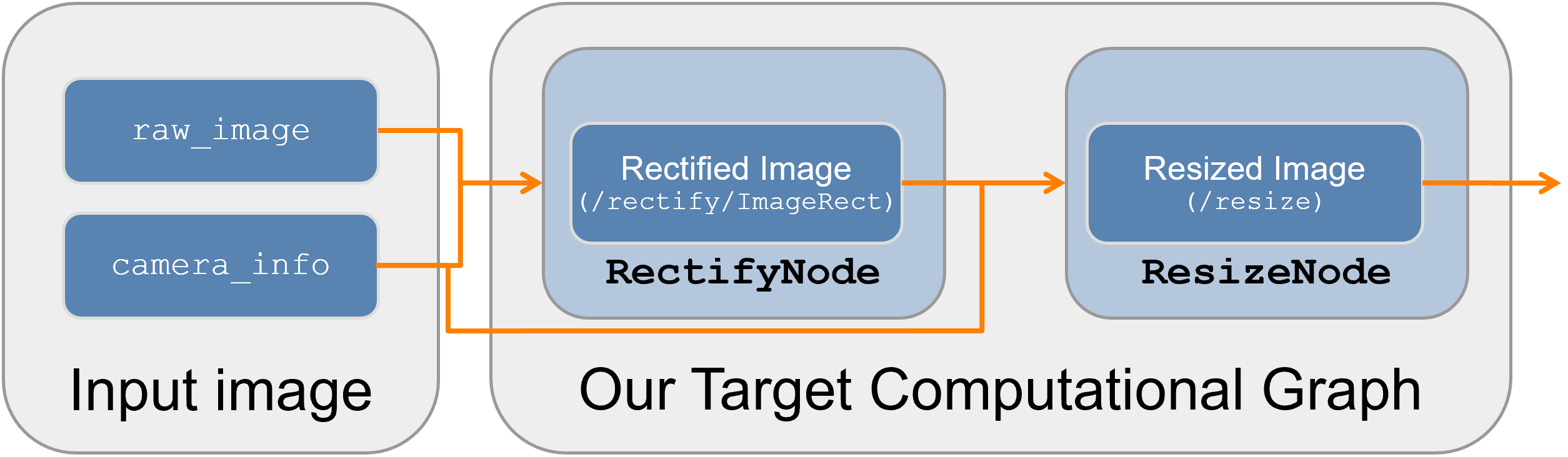}
    \caption{
        Computational graph of our case study perception application, reproduced from \cite{mayoral2022robotcore}.
    }
    \label{figure:perception_use_case}
\end{figure}

In order to trace and evaluate the relative performance of both ROS~2 individual Nodes and complete computational graphs, we leverage Linux Tracing Toolkit next generation (LTTng \cite{desnoyers2006lttng}) for tracing and benchmarking, Figure~\ref{figure:robotcore_architecture}~\circledwhite{D}. Building upon prior work \cite{bedard2022ros2_tracing}, LTTng provides a collection of flexible tracing tools and multipurpose instrumentation for ROS~2 that allow collecting runtime execution information in real-time in distributed systems using low-overhead tracers. 
For example, when enabling all ROS~2 instrumentation, end-to-end message latency overhead is below $5.5$us \cite{bedard2022ros2_tracing}, making it suitable for a wide variety of hardware acceleration use cases.
Building on top of this foundation, we developed a tracing and benchmarking template that enables roboticists to easily instrument both their accelerated and non-accelerated code in a vendor-agnostic manner. This infrastructure also lays a foundation for future integration with platform-specific performance counters and tracing tools that can extend analysis to more fine-grained introspection and profiling of the kernels running onboard an accelerator device.

%

\section{Case Study: Accelerating ROS~2 Perception}

For our case study, we trace, benchmark, and accelerate a subset of  \texttt{image\_pipeline}~\cite{mihelich2022image_pipeline},
one of the most popular packages in the ROS~2 ecosystem, and a core piece of the ROS perception stack. We compose a simple computational graph consisting of two nodes, \texttt{resize} and \texttt{rectify}, as shown in Figure~\ref{figure:perception_use_case}. The computational graph of our case study perception application in Figure~\ref{figure:perception_use_case} uses the ROS \texttt{image\_pipeline} package, containing two ROS~2 nodes: (1)~\texttt{RectifyNode} subscribes to the \texttt{/camera/image\_raw} and \texttt{/camera/camera\_info} topics from Gazebo \cite{koenig2004design} and publishes a \emph{rectified image} to (2)~\texttt{ResizeNode}, which publishes the final \emph{resized image}. We leverage our open architecture for hardware acceleration (Section~\ref{sec:architecture}) to benchmark, trace and accelerate our computational graph, comparing a CPU to an FPGA implementation. In this section we describe the methodology of our approach, and analyze our timing results, presenting a case study for how our open architecture can help enable hardware accelerated applications in ROS~2.





\begin{figure*}[!t]
    \centering
    \includegraphics[width=0.91\textwidth]{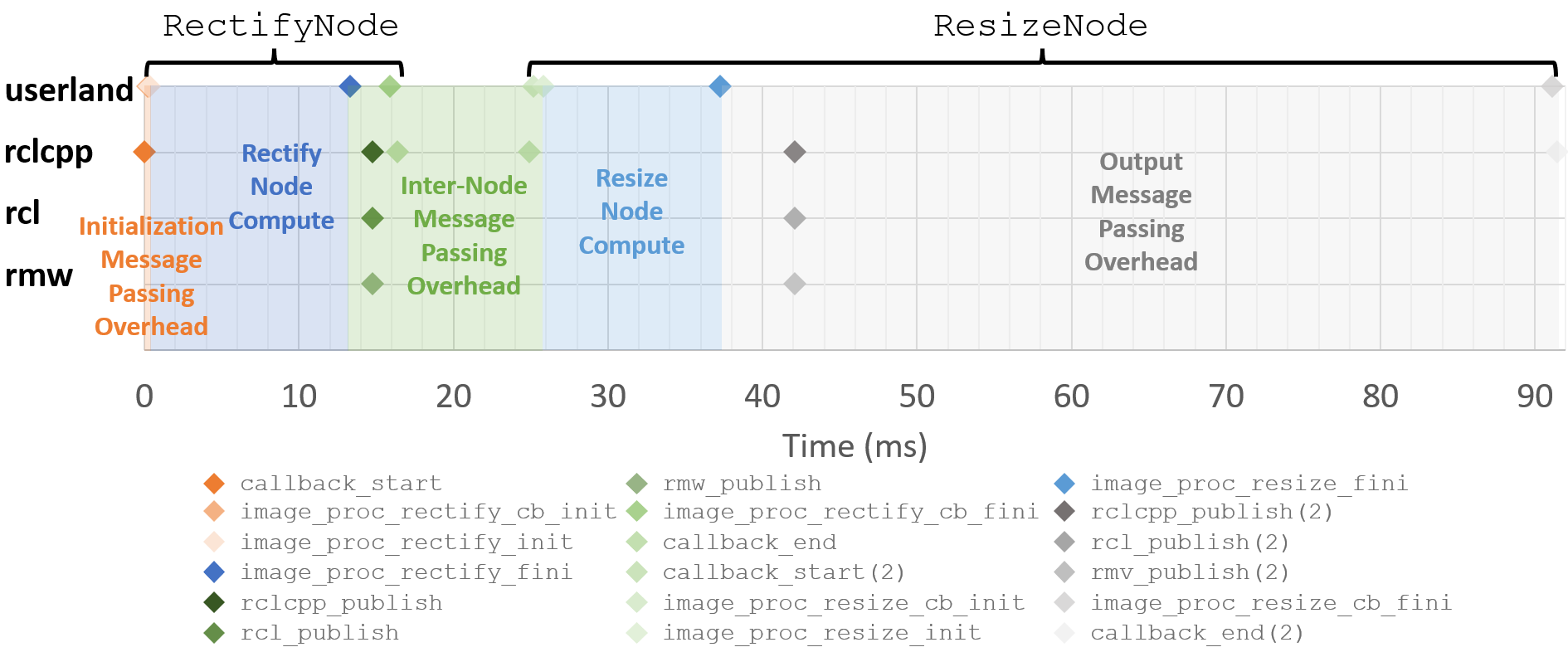}
    \caption{Tracepoints instrumented across ROS~2 abstraction layers on CPU for case study computational graph (Figure~\ref{figure:perception_use_case}). Breakdown summary in Figure~\ref{figure:benchmarking_breakdown_cpu}. Reproduced from \cite{mayoral2022robotcore}.}
    \label{figure:benchmarking_tracepoints}
    \vspace{-10pt}
\end{figure*}

\subsection{Method}
\showkeysfalse  
\label{sec:casestudy:methodology}
\showkeystrue  

We propose the following steps
to analyze a ROS~2 application and design appropriate acceleration:
(i)~instrument both the core components of ROS~2 and the target kernels; (ii)~trace and benchmark the kernels on the CPU to establish a baseline; (iii)~develop a hardware accelerated implementation on alternate hardware (e.g., GPU, FPGA); and (iv)~trace, benchmark against the CPU baseline, and improve the accelerated implementation.

Following this methodology, in our case study we begin by instrumenting both ROS~2 and our target kernels with LTTng probes. Reusing past work and probes~\cite{bedard2022ros2_tracing} allows us to easily get a grasp of the dataflow interactions within \texttt{rmw}, \texttt{rcl}, and \texttt{rclcpp} ROS~2 layers. We then also instrument the \texttt{ResizeNode} and \texttt{RectifyNode} components of the \texttt{image\_pipeline} package used in our case study.
The relevant tracepoints placed in our computational graph across ROS~2 stack layers are listed in Figure~\ref{figure:benchmarking_tracepoints}~and~\ref{figure:benchmarking_breakdown_cpu} (full list in Pull Request $717$ in the \texttt{image\_pipeline} repository~\cite{mihelich2022image_pipeline}).
On the CPU, these tracepoints enable us to isolate the latency of computation within a node from the time it takes ROS~2 to package and pass information between nodes.


In the following sections we report timing results from using a Xilinx Kria® KV260 Vision AI Starter Kit~\cite{xilinx2022kria}, which has an onboard integrated Quad-core Arm® Cortex®-A53 CPU and an FPGA containing 256K System Logic Cells and 1.2K DSP Slices. All benchmark results report the mean value obtained from a $60$ second continuous run of the computational graph. The FPGA kernels are synthesized, placed and routed with a $250$MHz clock.

\subsection{CPU-Only Tracing Results}
\showkeysfalse  
\label{sec:casestudy:cpuresults}
\showkeystrue  

Figure~\ref{figure:benchmarking_tracepoints} demonstrates the results of instrumenting and tracing our target computational graph (Figure~\ref{figure:perception_use_case}) across multiple ROS~2 stack layers on the CPU,
and Figure~\ref{figure:benchmarking_breakdown_cpu} summarizes the breakdown of timing results across operations, establishing the CPU baseline for our application.
The breakdown in Figure~\ref{figure:benchmarking_breakdown_cpu} shows the time taken to do the computations within each node, as well as the time taken by the ROS~2 lower-level message-passing system across the various abstraction layers. We find that the message-passing overhead in our application consumes more than $73.3\%$ of the total time and is therefore a large bottleneck in the total computation time of the full graph.
We next explore FPGA hardware acceleration options, comparing performance to the CPU baseline.


\begin{figure}[!t]
    \centering
    \includegraphics[trim=0 0 0 12, clip, width=1.0\columnwidth]{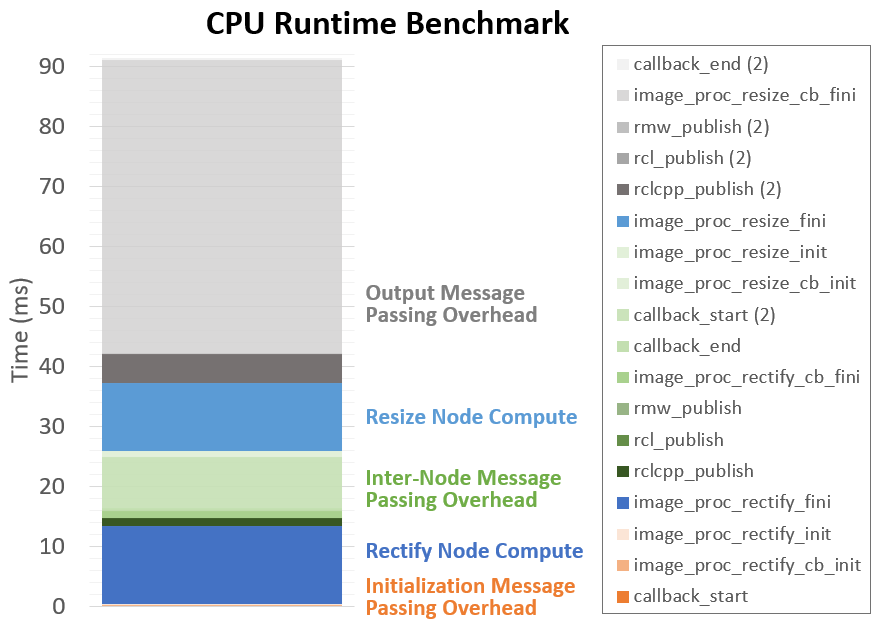}
    \vspace{-20pt}
    \caption{Breakdown of CPU runtime derived from tracing and benchmarking. Total computation time of our case study graph is dominated by message passing overheads, a bottleneck consuming over $73.3\%$ of total runtime. Reproduced from \cite{mayoral2022robotcore}.}
    \label{figure:benchmarking_breakdown_cpu}
\end{figure}




\subsection{Accelerating and Benchmarking CPU \& FPGA}
\showkeysfalse  
\label{sec:casestudy:acceleration}
\showkeystrue  

In this section, we explore hardware acceleration options for an FPGA for our case study application (Figure~\ref{figure:perception_use_case}).
In Section~\ref{sec:casestudy:acceleration:nodes}, we first explore hardware acceleration kernels for the core logic of each of the Nodes (\emph{rectify} and \emph{resize}), harnessing our open architecture for implementation.
In Section~\ref{sec:casestudy:acceleration:graph}, we then explore two different FPGA designs to accelerate the computational graph by optimizing dataflow interactions between FPGA-based nodes, addressing the ROS~2 \emph{communication infrastructure} performance bottleneck revealed by the CPU baseline in Section~\ref{sec:casestudy:cpuresults}.

\subsubsection{Accelerating Nodes \& Components on an FPGA}

\showkeysfalse  
\label{sec:casestudy:acceleration:nodes}
\showkeystrue  
We first accelerate the computations at each one of the graph nodes. The \texttt{RectifyNode} and \texttt{ResizeNode} \emph{Components} of Figure~\ref{figure:perception_use_case} are accelerated using Xilinx's HLS, XRT, and OpenCL targeting the Kria KV260\footnote{github.com/ros-acceleration/image\_pipeline/blob/ros2/image\_proc/src/ \{rectify,resize\}\_fpga.cpp}.
Each ROS~2 \emph{Component} has an associated acceleration kernel\footnote{{github.com/ros-acceleration/image\_pipeline/tree/ros2/image\_proc/src/ image\_proc}} that leverages the Vitis Vision Library, a computer vision library optimized for Xilinx silicon solutions and based on OpenCV APIs. These accelerated \emph{Components} and their kernels easily integrate with the rest of the ROS meta-package through our open architecture (Figure~\ref{figure:robotcore_architecture}), and are openly available to the public. Building the accelerators is abstracted away from roboticist end-users, and takes no significant additional effort than the standard build of the \texttt{image\_pipeline}.

After benchmarking the accelerated \emph{Components} using the trace points of Section~\ref{sec:casestudy:cpuresults}, we observe an average $6.22\%$ speedup in the total computation time of the perception pipeline when offloading tasks to the FPGA (see Figure~\ref{figure:benchmarking_runtime}).
For this case study example, it is not surprising that accelerating the computational nodes and components alone only gives a modest performance increase because, as we saw in Section~\ref{sec:casestudy:cpuresults}, the performance bottleneck in the baseline CPU system was communication overhead, not computation.

\subsubsection{Accelerating the \underline{Computational Graph} on an FPGA}
\showkeysfalse  
\label{sec:casestudy:acceleration:graph}
\showkeystrue  

In our case study application, message-passing overheads across the ROS~2 abstraction layers far outweigh other operations, so in this section we focus on optimizing these dataflows.
Addressing performance bottlenecks in our system leads to overall lower computational graph latency, and to faster robots. 
To seize this acceleration opportunity in our case study example, we optimize the dataflow within the computational graph and across ROS~2 Nodes and Components through two different design approaches: (a) kernel fusion, and (b) dedicated streaming queues. 

\begin{figure}[!t]
    \centering
    \includegraphics[width=0.93\columnwidth]{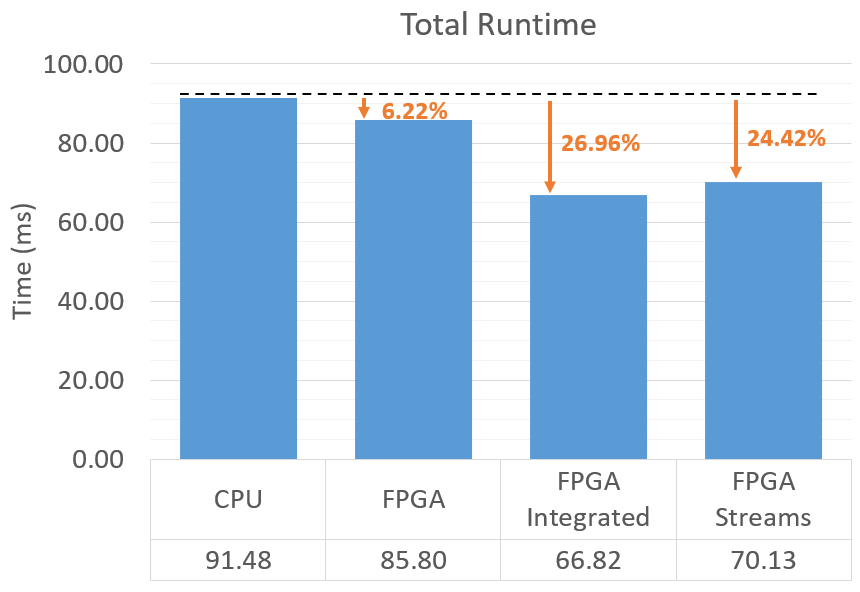}
    \caption{Total runtime of CPU baseline and FPGA, FPGA-Integrated, and FPGA-Streaming hardware-accelerated implementations of case study application.
    Acceleration enables up to $26.96\%$ speedup over CPU. Reproduced from \cite{mayoral2022robotcore}.}
    \label{figure:benchmarking_runtime}
\end{figure}

The speedup obtained by integrating both ROS \emph{Components} on the FPGA into a \emph{single unified kernel} is shown in Figure~\ref{figure:benchmarking_runtime}. The benefits of doing this are two-fold. First, we avoid any message-passing between the \emph{Rectify} and \emph{Resize} Nodes' \emph{Components}. Second, we avoid the compute cycles wasted while memory is mapped back and forth between the host CPU and the FPGA. This results in an overall latency speedup of $26.96\%$ over the CPU. In addition to speeding up the perception stage, another added benefit of this improvement is that such speedups make room for other robot tasks in a complete end-to-end system. Note, however, that this improvement required the construction of an entirely new ROS Node and unified acceleration kernel on the FPGA.

We then develop a template for an \emph{accelerated ROS~2 message passing} interface on the FPGA. This interface is \emph{Node} and \emph{Component}-agnostic and can be leveraged by roboticists to accelerate the communication channels of any computational graph on an FPGA. 
This is done by leveraging an AXI4-Stream interface to create an intra-FPGA ROS~2 communication queue template which is then used to pass data across Nodes in the FPGA without sending messages to the CPU\footnote{AXI4-Stream interfaces are data-type specific and as such our template may require type adaptations for other use cases depending on the Node-to-Node data interactions.}. This allows us to completely bypass the original CPU-centric ROS~2 message-passing system and optimizes dataflow, achieving an overall latency improvement of $24.42\%$ over the CPU in our application (see Figure~\ref{figure:benchmarking_runtime}).


Based on these results, for this case study, we show that implementing FPGA-accelerated versions of key ROS~2 \emph{Components} is easily feasible, and that addressing the right bottleneck is key to improving performance.
Tracing and benchmarking the CPU baseline suggested that communication is the bottleneck in our case study.
In fact, independent examination of, e.g., a single run of the fused-kernel accelerator using the Xilinx Vitis Analyzer, confirms that this is also the case on the FPGA---we note that integrating device-specific profiling tools into our foundational tracing infrastructure in future work can further automate this type of fine-grained introspection of kernels onboard accelerator devices.
We can achieve overall performance improvements by either combining Nodes or streamlining intra-FPGA communication.
While combining nodes may result in slightly higher performance, it is a much more labor-intensive design effort. By contrast, our accelerated intra-FPGA-Node communication queue template can be applied by any roboticist, to any computational graph.





\section{Conclusion}

In this chapter we presented a new open infrastructure to introduce hardware acceleration in ROS~2 in a scalable and technology-agnostic manner.
Our architecture allows us to increase the performance of robotics applications through the \emph{integration of hardware acceleration} with ROS~2 APIs and its conventional flows.
We do so by extending ROS~2 in a way that is \emph{portable across accelerator platforms} (e.g., FPGAs, GPUs) and system deployments (e.g., edge devices, workstations, data centers, and cloud).
We also present a \emph{template for low-overhead tracing and benchmarking} to analyze performance across both hardware accelerated and standard ROS~2 computational graphs.

We use our open architecture and our tracing and benchmarking infrastructure to demonstrate a principled design methodology for ROS~2 hardware acceleration, exposing insights into how to optimize overall system-wide performance by analyzing a CPU baseline, and comparing accelerator design iterations to that original baseline.
We examine a case study using the Xilinx Kria KV260 platform to demonstrate FPGA acceleration of one of the most popular packages in the ROS perception pipeline: \texttt{image\_pipeline}.
We first demonstrate a modest performance speedup of $6.22\%$ from offloading perception tasks to the FPGA,
and then increased speedup by additionally addressing the communication overheads that we identified as bottlenecks by analyzing our CPU baseline.
We achieved a speedup of $26.96\%$ from re-architecting the graph to combine nodes and avoid inter-FPGA-node communication delays inflicted by interactions with the CPU, but this approach requires substantial effort from users to re-architect their graphs.
Instead, to avoid this overhead and stay in alignment with the ROS~2 programming model, we then design a novel \emph{template for intra-FPGA ROS~2 Node communication queues} that allows ROS \emph{Nodes} and \emph{Components} to deliver faster dataflows, achieving a $24.42\%$ speedup over a CPU without excessive manual per-kernel design effort.

We contribute our open architecture to the ROS community, so that future work can use our infrastructure and extend to new applications beyond our case study example. Promising directions for future work include: benchmarking computational graphs with other hardware solutions (e.g., GPUs) to establish consistent cross-accelerator comparisons; extending our tracing and benchmarking approach to include additional tracing information (e.g., profiling within FPGA or GPU devices) for more fine-grained introspection of kernels running onboard accelerators; and applying our open architecture and analysis to other ROS~2 packages.

Our code is disclosed under a commercially friendly open-source license and is available and maintained at the ROS~2 Hardware Acceleration Working Group GitHub organization: \url{https://github.com/ros-acceleration}.
This work is being further integrated into the ROS ecosystem through a community standardization effort, REP-2008~\cite{hawg2022rep2008}.

\section{Future work}


The promising outcomes of this chapter lay the groundwork for the development of dedicated \emph{robot security accelerators}. Future initiatives will concentrate on commercializing these innovations. Preliminary findings suggest that our security accelerator prototypes could significantly enhance intra-network ROS 2 communications, potentially reducing average networking latency to single-digit microseconds, which is over 100$\times$ faster than current software implementations on CPUs. Particularly in worst-case scenarios, these accelerators maintain deterministic and isochronous performance, a stark contrast to the higher latencies observed with CPU-based software implementations. The potential for a maximum latency speedup exceeding 100,000$\times$ faster communications and security computations marks a significant leap forward in the field of robotics. Due to the commercial potential of these developments, a detailed discussion is outside the scope of this research, but these early results are indicative of the transformative impact these accelerators could have on the future of robotic communications.


    \part{Cyber-attacking robots}\label{part:attacking_robots}
    
    \chapter{Attacking Robot Software}
\label{chapter:attacking_software}

\epigraph{\textbf{Contributions}: \emph{One of the top two contributors of \cite{maggi2022security} and core contributor to software, methodology and research results. Overall research conducted over a 6-month period, with other leading security researchers from industry participating. Produced various open source security package dissectors (e.g. \href{https://github.com/secdev/scapy/pull/3462}{this one}, and \href{https://github.com/secdev/scapy/pull/3403}{this other one}). All findings have been publicly disclosed.}}{}

%

This chapter delves into the security analysis of robot software and particularly, the Data Distribution Service (DDS) \cite{dds14} standard, a middleware pivotal in operating a wide range of robotic systems and devices, including ROS 2 (DDS is the default communication middleware of ROS 2). The summarized investigation presented herein is the result of a multi-year research effort involving various researchers \cite{maggi2022security} and led to uncover multiple vulnerabilities in DDS, leading to the issuance of various new CVE IDs for common DDS implementations and even in the standard specifications. An extensive scanning exercise revealed numerous public-facing DDS services across various industries, some affected by these CVEs, highlighting significant security risks due to their unintended public exposure. We emphasize the need for continuous security testing of DDS and similar technologies, offering both immediate and long-term recommendations for secure integration and advocating for robust supply-chain management and ongoing security practices.

\section{Introduction}

Even within the industry, a big percentage of practitioners are unaware that the Data Distribution Service (DDS) drives systems such as industrial robots (manipulators), railways, autonomous cars, airports, surgical robots, spacecraft, diagnostic imaging machines, luggage handling, industrial robots, military tanks, and frigates, among others. It has been in use for about a decade, and its adoption continues to steadily increase.

We discovered and reported vulnerabilities in DDS that warranted new CVE IDs: Five with a score of greater or equal than 7.0, four with a score of > 8.5, one vulnerability in the standard specifications, and other deployment issues in the DDS software ecosystem (including a fully open production system).

Successful exploitation of these vulnerabilities can facilitate initial access (MITRE ATT\&CK \cite{strom2018mitre} Technique ID TA0108) via exploitation of remote services (T0866, T0886) or supply chain compromise (T0862), and allow the attacker to perform discovery (TA0102, T0856) by abusing the discovery protocol. The consequences of successful exploitation, in any of the critical sectors where DDS is used, range from denial of service (T0814) via brute forcing (T0806), to loss of control (T0827), or loss of safety (T0880). The DDS protocol itself can also be abused to create an efficient command and control channel (T0869). Based on our analysis, we recommend mitigations such as vulnerability scanning (ATT\&CK Mitigation ID M1016), network intrusion prevention (M1031), network segmentation (M1030), filter network traffic (M1037), execution prevention (M1038), and auditing (M1047).

\subsection*{Findings in Brief}


\begin{figure}[h!]
    \includegraphics[width=1\textwidth]{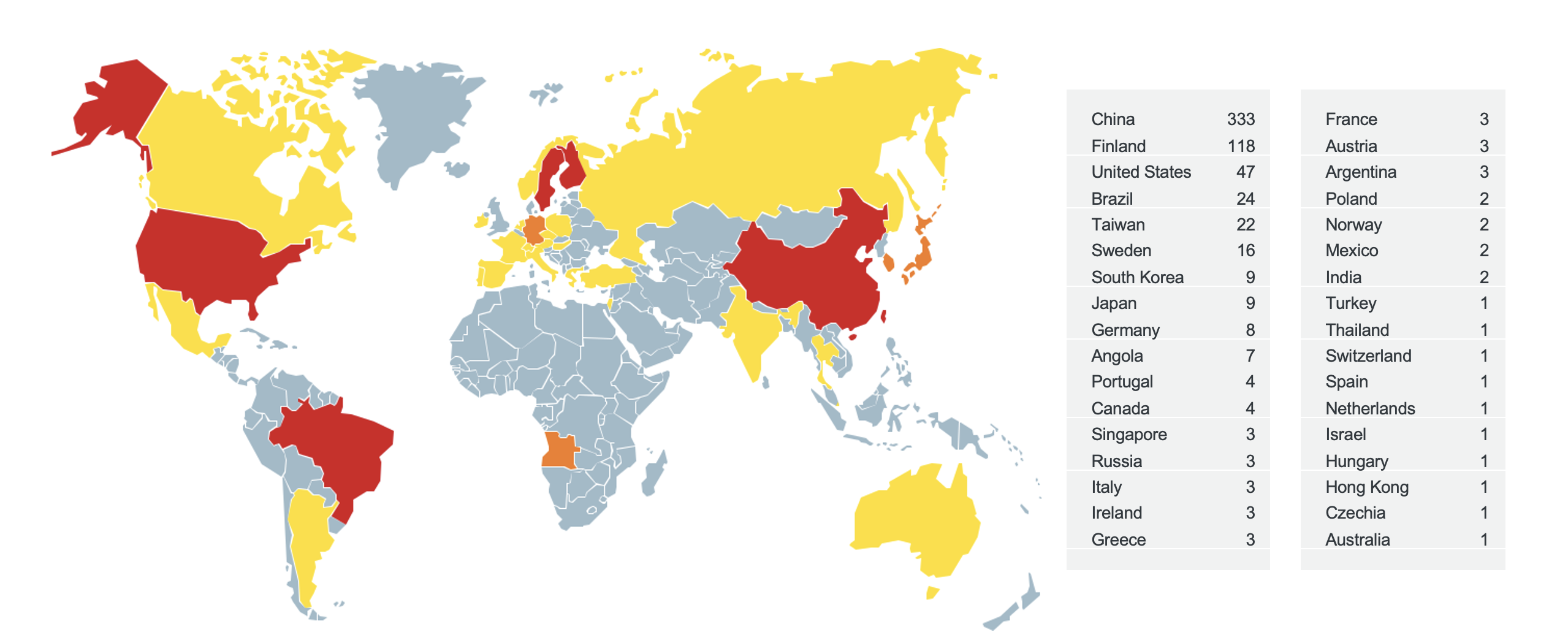}
    \centering
    \caption{We found exposed DDS systems in 34 countries \cite{maggi2022security}, including vulnerable ones, identified via distinct IPs leaking data. Color code hints about criticality by vulnerable targets: red for most, yellow/grey for least. Reproduced from \cite{maggi2022security}.}
    \label{figure:dds_findings_brief}
\end{figure}

Given this technology’s versatility, we analyzed and discovered multiple security vulnerabilities, resulting in 13 new CVE IDs for the six most common DDS implementations. This includes one vulnerability in the standard specifications and other deployment issues in the DDS software ecosystem (including a fully open production system). Some of these vulnerabilities have been patched or mitigated by the vendors since we reported them. Others remain exposed and can still be abused.

By measuring the exposure of DDS services, in one month we found over 600 distinct public-facing DDS services in 34 countries affecting 100 organizations via 89 internet service providers (ISPs). Of the DDS implementations by seven distinct vendors (one of which we were initially unaware of), 202 leaked private IP addresses (referring to internal network architecture details), and seven supposedly secret URLs. Some of these IP addresses expose unpatched or outdated DDS implementations, which are affected by some of the vulnerabilities that we’ve discovered and disclosed in November 2021.

During our research, we interviewed key DDS users and system integrators to collect their feedback on our findings and the importance of DDS for innovation in their respective sectors. In this chapter, we analyze and discuss the specifications of DDS and the six most actively developed implementations maintained by certified vendors and with millions of deployments worldwide. We also released an open-source software: a Scapy-based dissector and several fuzzing harnesses for three open DDS implementations.

\section{Background}

DDS is a standardized middleware software based on the publish-subscribe paradigm, helping the development of middleware layers for machine-to-machine communication. This software is integral especially to embedded systems or applications with real-time6 requirements. Maintained by the Object Management Group (OMG), DDS is used in all classes of critical applications to implement a reliable communication layer between sensors, controllers, and actuators.

\begin{mybox}
{\textbf{\textit{About real-time in DDS and robotics}}}
\emph{Real-time} which in robotics is interpreted as meeting time deadlines in their computations, is used in a generic and non-strict (hard real-time) manner in DDS specifications. We were unable to find timing guarantees (hard real-time, firm real-time, or soft real-time) provided in any of the reviewed documents. For the purposes of this research, we conclude that DDS targets remote soft real-time communications at best, leaving firm and hard real-time interactions to other technologies.
\end{mybox}

DDS is at the beginning of the software supply chain, making it easy to lose track of and is an attractive target for attackers. Between 2020 and 2021, 66\% of attacks focused on the suppliers’ codes \cite{ENISA2021}. While we were in the process of doing this research, we encountered an exposed source-code repository hosting a proprietary implementation of DDS. Left open, this would have let an attacker infect the source code (MITRE ATT\&CK T0873, T0839).

\begin{figure}[h!]
    \includegraphics[width=1\textwidth]{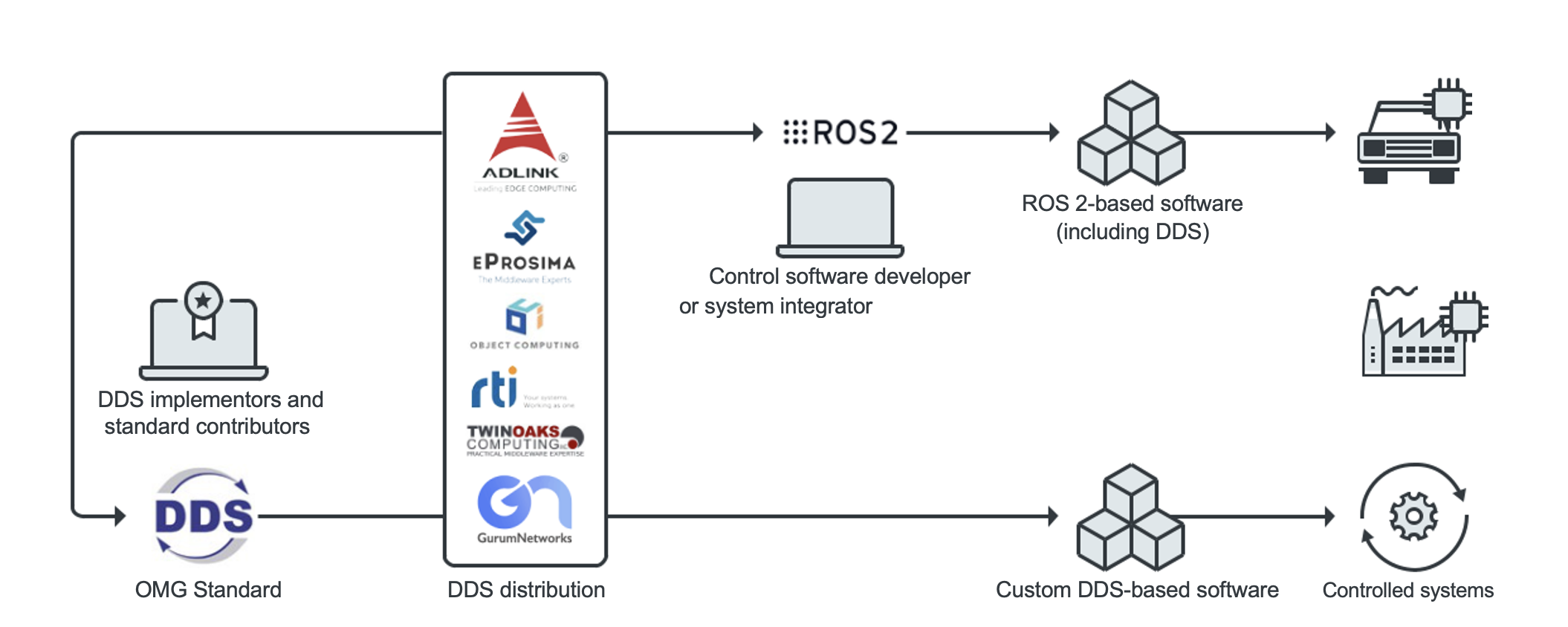}
    \centering
    \caption{DDS is a standardized software library used for software-based controlled systems, directly or via ROS 2. Reproduced from \cite{maggi2022security}.}
    \label{figure:dds}
\end{figure}

Notably, the following companies and agencies use DDS (note that this is not an exhaustive list of currently using the technology):

\begin{itemize}
    \item Military robots (e.g. legged robots from Ghost Robotics)
    \item National Aeronautics and Space Administration (NASA) at the Kennedy Space Center
    \item Siemens in wind power plants
    \item Volkswagen and Bosch for autonomous valet parking systems
    \item Nav Canada and European CoFlight for air-traffic control
\end{itemize}

DDS is the foundation of other industry standards such as OpenFMB \cite{OpenFMB2022} for smart-grid applications and Adaptive AUTOSAR \cite{RichterGuijarro2021}, among other sectors that we identify in the next section. The Robot Operating System 2 (ROS 2), which is the de facto standard operating system for robotics and automation and as previously mentioned, also uses DDS as the default middleware. We also noted that, according to a confidential document leaked online, NVIDIA has listed DDS as a strategic tool for system-virtualization and cloud-computing applications, mainly for exchanging data within and across virtual machines.


\section{DDS and Real-Time Publish-Subscribe (RTPS) Packets}

There are many software-based controlled systems in the world that connect sensors, actuators, and controlling and monitoring applications. DDS was invented for such systems, with a strong focus on interoperability and fault tolerance. It is optimized for publish-subscribe and peer-to-peer applications as most applications can’t afford a single point of failure. The middleware relies on multicast (group communication or data transmission to multiple recipients) for discovery, allowing everything to run without needing initial configurations.

\begin{figure}[h!]
    \includegraphics[width=1\textwidth]{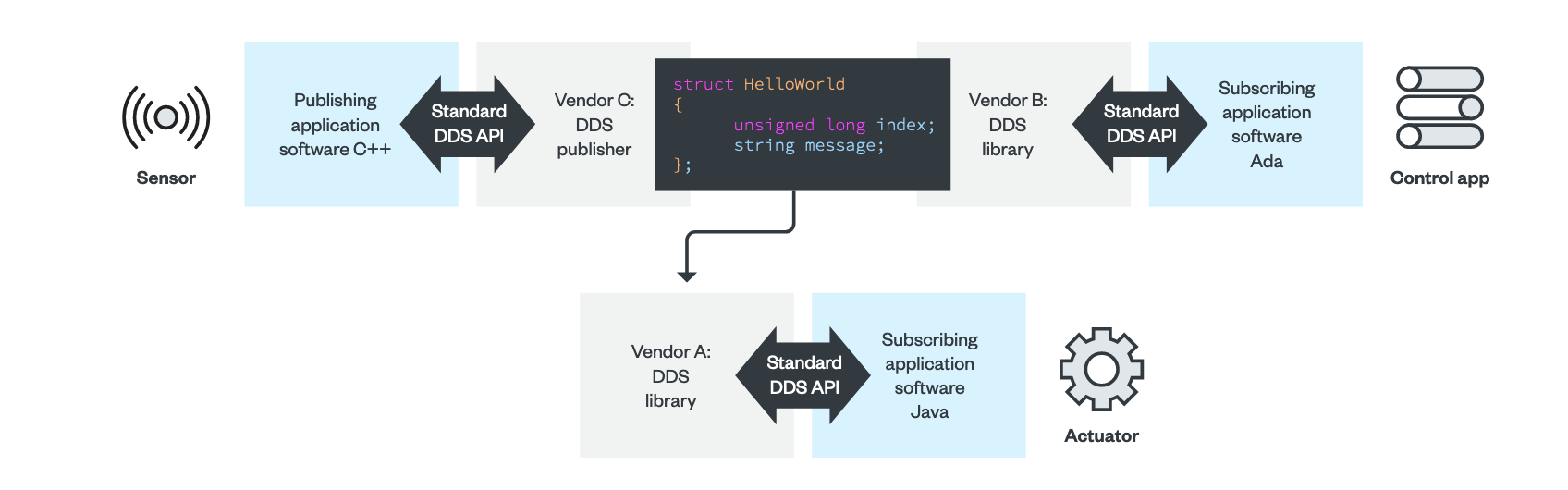}
    \centering
    \caption{Simplified software-based control system with actuators, controller, sensors, communicating by exchanging data over DDS. Reproduced from \cite{maggi2022security}.}
    \label{figure:dds_data}
\end{figure}

From a software developer standpoint, DDS is a communication middleware that facilitates interoperability of processes across machines in all main programming languages. From another viewpoint, DDS is a data-centric, publish-subscribe communication protocol that allows developers to build a flexible shared data “space” or “bus” for virtually any application that requires two or more nodes to exchange typed data.
\begin{figure}[h!]
    \includegraphics[width=0.7\textwidth]{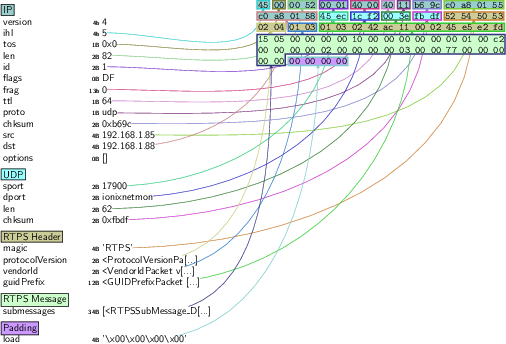}
    \centering
    \caption{An example of a dissected RTPS message with a DATA submessage. Reproduced from \cite{maggi2022security}.}
    \label{figure:rtps_dissection}
\end{figure}
The DDS layer is encapsulated into real-time publish-subscribe (RTPS) packets, which for now can be considered as a collection of sub-messages (such as timestamp, discovery, data, and security metadata), as shown in Figure 5. Given the strong dependency between DDS and RTPS, we focused our research on RTPS for increased generality. Because of its flexibility, DDS and its underlying layers do not come as a ready-to-use, off-the-shelf product like other middleware tools (such as Message Queuing Telemetry Transport or MQTT). Rather, DDS is a programming library that developers use to build custom middleware protocols with advanced features such as custom data types, quality of service (QoS) policies, network partitioning, authentication, and encryption.

\begin{mybox}
{\textbf{\textit{Understanding DDS as an robotics databus (as opposed to a database)}}}
DDS can be understood as databus for robotic solutions. A databus is a data-centric software framework for distributing and managing real-time data in intelligent distributed systems. In intelligent distributed systems, managing dataflow is critically important. The databus — designed specifically to manage dataflow in intelligent distributed systems — simplifies application and integration logic with a powerful data-centric paradigm.
\end{mybox}

\subsection*{Research Scope: RTPS, DDS, and Robot Operating System 2 (ROS 2)}

\begin{table}[h!]
    \includegraphics[width=1\textwidth]{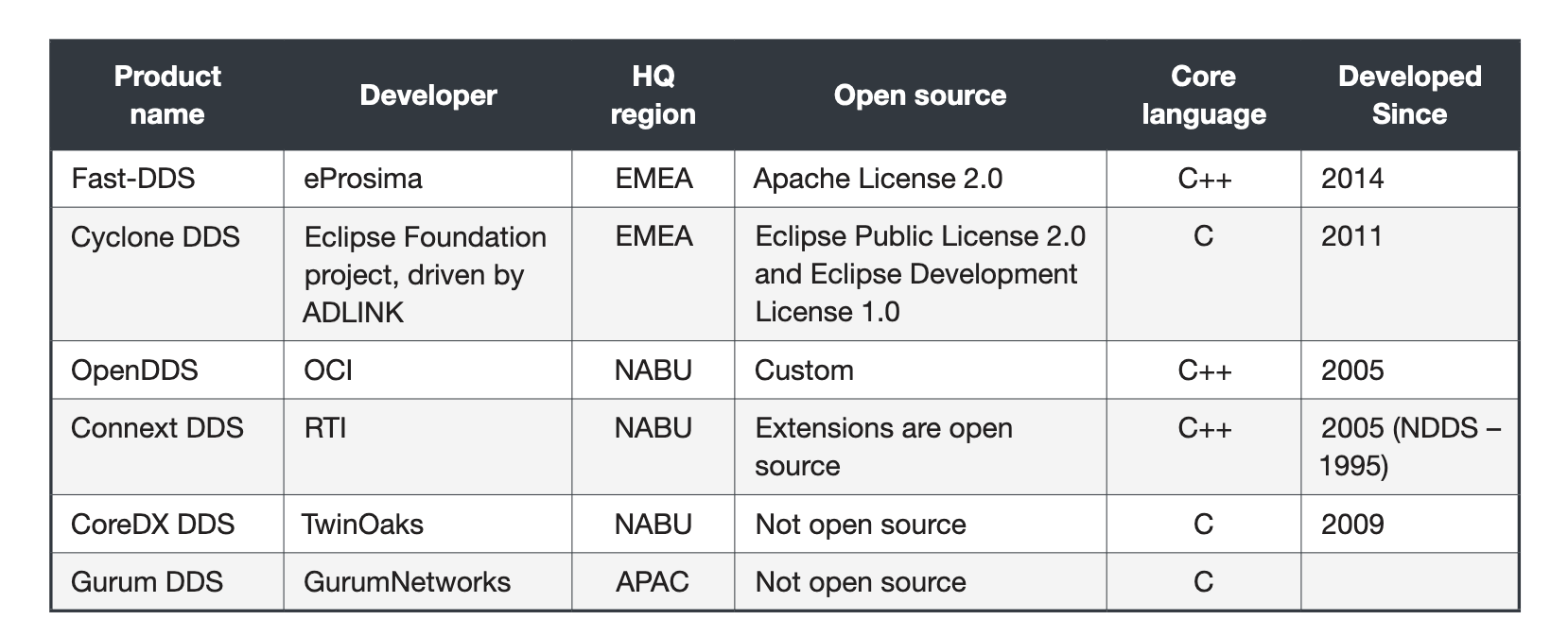}
    \centering
    \caption{A list of DDS implementations that we analyzed in this research. Reproduced from \cite{maggi2022security}.}
    \label{table:dds_analyzed}
\end{table}

In addition to the DDS standard specifications, we focused our investigation on the six DDS implementations listed in Table \ref{table:dds_analyzed}. These implementations are used globally and have customers and users in the critical sectors identified in the same table. Because DDS depends on RTPS as a lower-layer standard protocol, each DDS implementation ships with its own RTPS implementation. In other words, DDS data is contained as a sub-message within RTPS, thereby focusing on both protocols. To make sure we highlight the criticality of this for robotics, again, the Robot Operating System 2 (ROS 2), which is the default standard meta-operating system for robotics and automation, has DDS as its default middleware. For this reason, the impact of each vulnerability extends beyond DDS alone, and includes all ROS 2 instances.

\section{Research Methodology and Technical Details}

Refer to \cite{maggi2022security} for details on the final findings derived from this research. This subsection  instead discusses the technical details of how those findings were encountered. The offensive methodology explained here applies to other (software) robotic endeavours. Figure \ref{figure:attack_research_methodology} provides an overview of the research methodology we followed. We analyzed each DDS implementation from two main angles: network and source — or binary — code. We wanted to have a deep understanding of the low-level details of the RTPS network layer to craft arbitrary test packets, and prove that it is possible for an attacker to mass scan a network and map the attack surface. In a parallel investigation, we manually studied the original or decompiled source code, with the main goal to find good fuzz targets (i.e., functions that receive and process untrustworthy data like network packets).

\begin{figure}[h!]
    \includegraphics[width=1\textwidth]{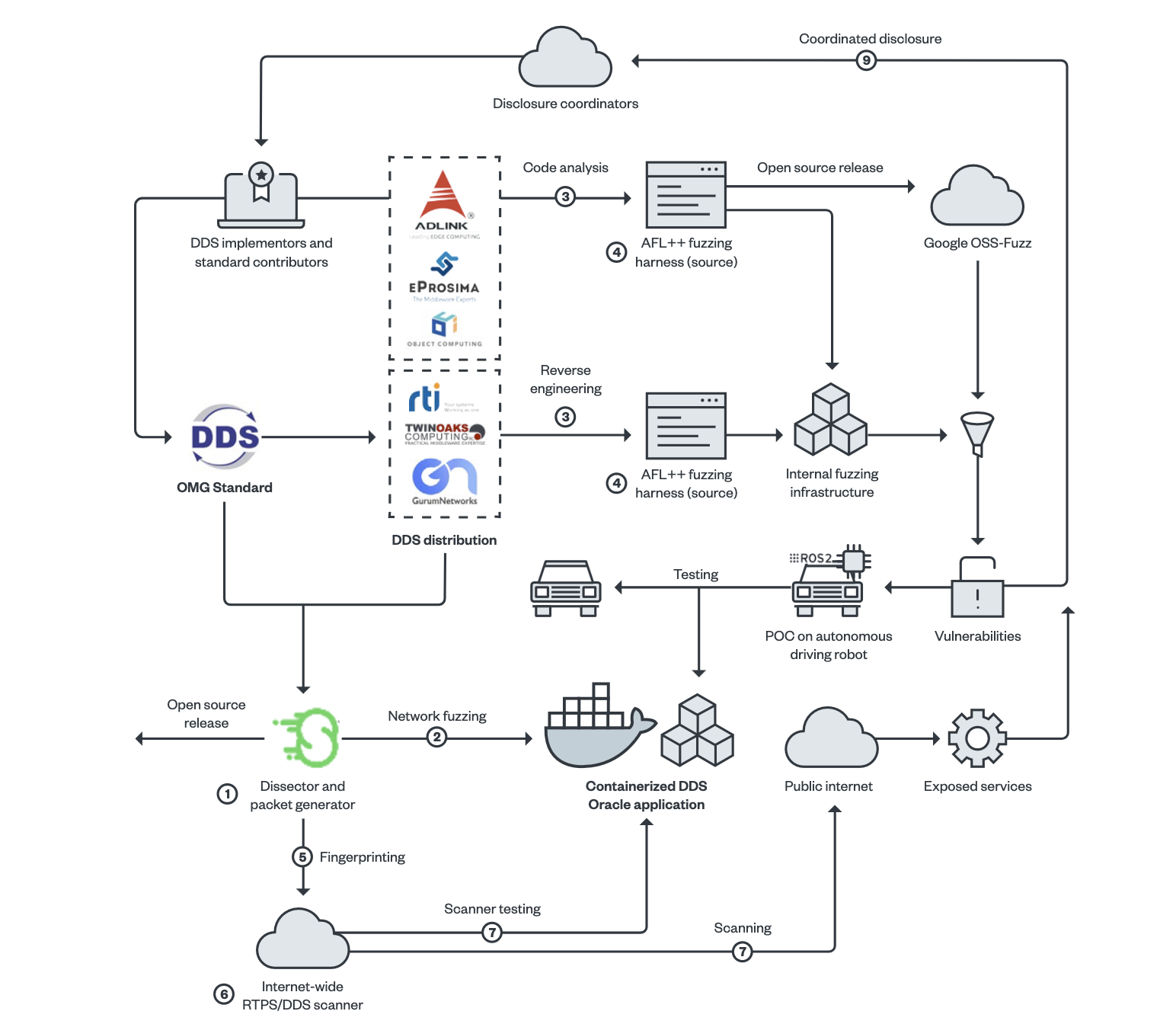}
    \centering
    \caption{Diagram of the research methodology and workflow we used: The numbered items indicate the sequence of steps we followed throughout this research, beginning with the development of a packet dissector. Reproduced from \cite{maggi2022security}.}
    \label{figure:attack_research_methodology}
\end{figure}

\subsection{A New Scapy Layer to Dissect and Forge RTPS and DDS Data}

Although Wireshark already includes an RTPS dissection plugin, we needed something more scriptable. Since we spent some time manually crafting RTPS packets at the beginning, we decided to develop a RTPS Scapy-based dissector. We released the resulting Scapy layer as open-source code under the GNU General Public License v2.0.74 which is available ever since \faGithub~  \url{https://github.com/secdev/scapy/pull/3403}.

\subsubsection{Crafting RTPS probes with Scapy}

Without going into the details of our Scapy RTPS implementation, note that it can be used to programmatically create RTPS packets by writing Python code, like any other Scapy layer. In practice, that’s seldom what a researcher would do, especially for “thick” protocols with lots of options.

\begin{figure}[h!]
    \includegraphics[width=1\textwidth]{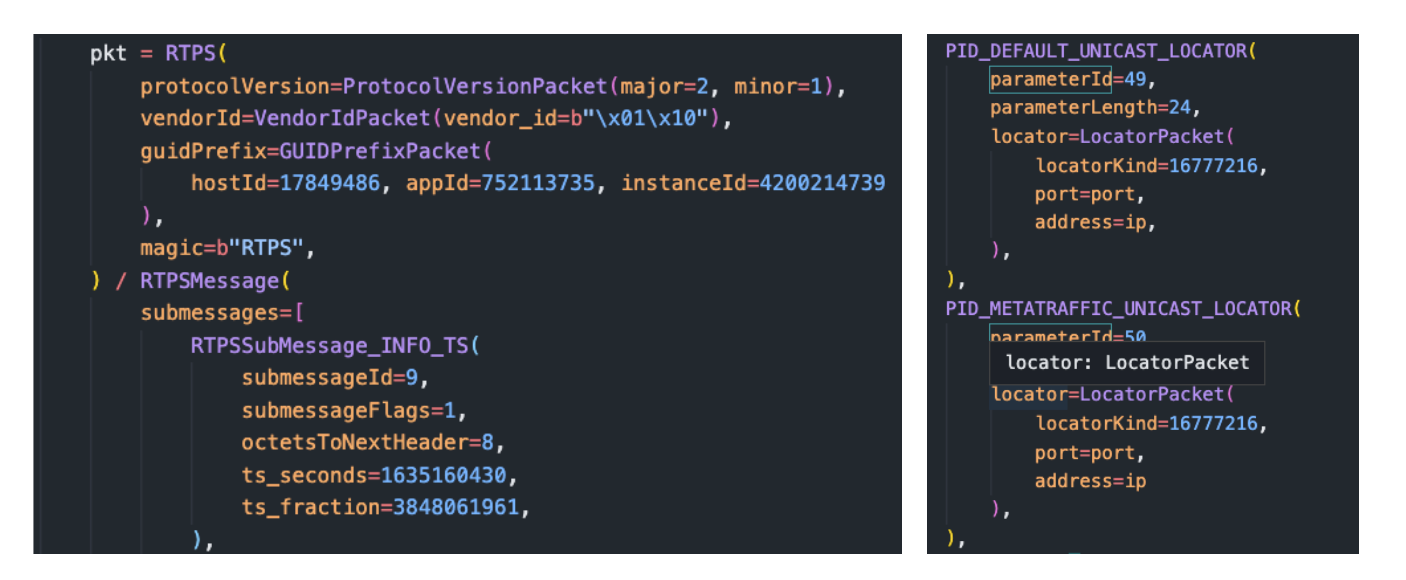}
    \centering
    \caption{With the Scapy RTPS layer, a developer can create arbitrarily complex (and unexpected but valid) RTPS packets. Reproduced from \cite{maggi2022security}.}
    \label{figure:rtps_crafting}
\end{figure}

Instead, the developer’s typical workflow can be:
\begin{itemize}
    \item Intercept traffic. Use Tcpdump or Wireshark to collect the traffic generated by the “hello world” example typically provided with a DDS distribution.
    \item Extract UDP payload. Use Scapy (or manually via Wireshark) to select the packet of interest and extract the UDP payload (which contains the RTPS layer, as shown in Figure \ref{figure:attack_research_methodology}).
    \item Dissect with the RTPS class. Pass the extracted payload to the Scapy RTPS class, which will automatically dissect it.
    \item Generate Python code automatically. Use Scapy’s built-in \verb|.command()| method to output the Python code that will declaratively generate the packet that has just been dissected. If necessary, modify the packet so obtained (as exemplified in Figure \ref{figure:rtps_exchange}).
    \item Test modified packet against the endpoint. Either use Scapy’s built-in send()/sendp() functions or Python’s socket module to send the packet to a target DDS endpoint and check if it triggers the desired behavior.
\end{itemize}

\begin{figure}[h!]
    \includegraphics[width=1\textwidth]{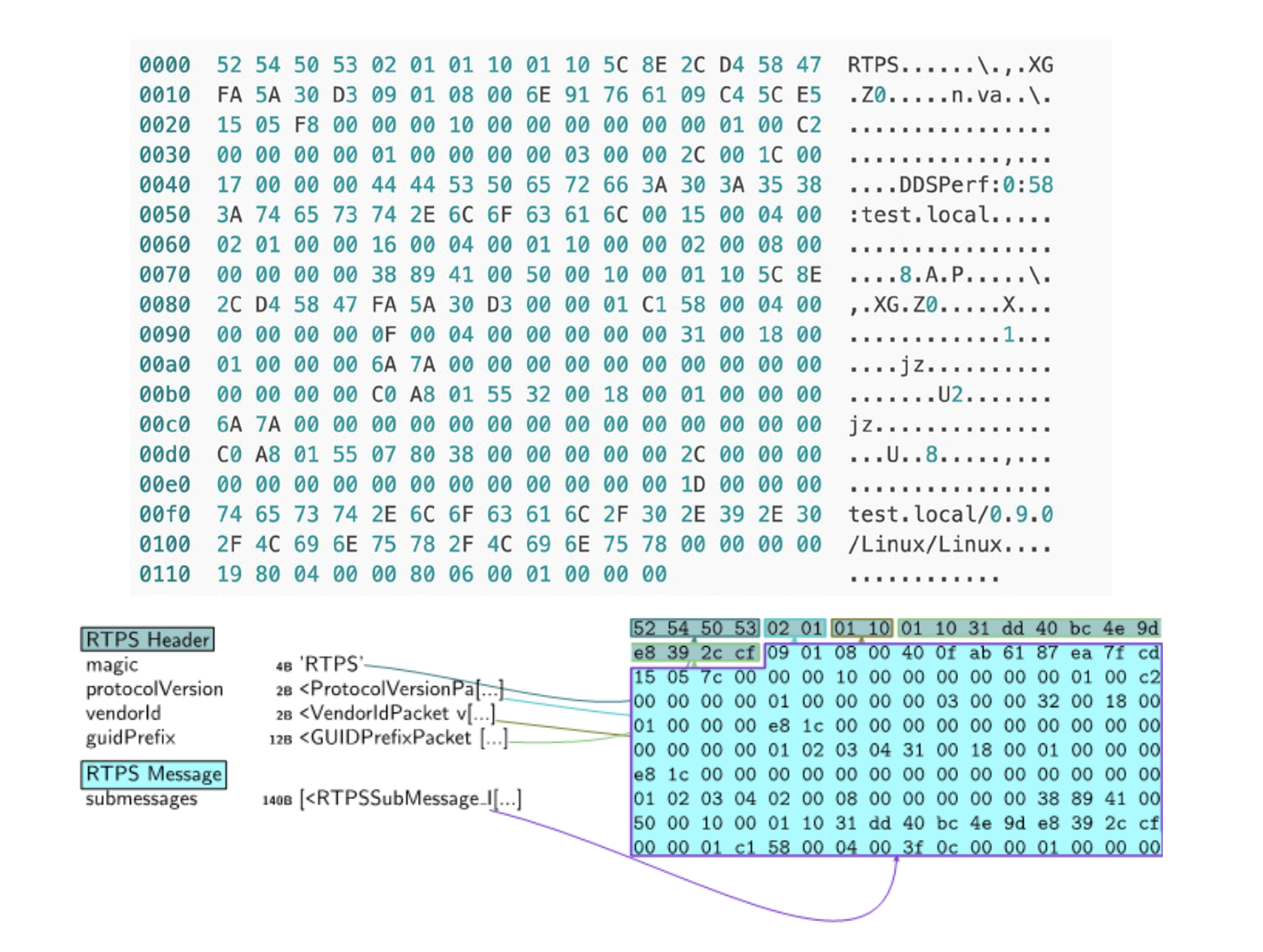}
    \centering
    \caption{The UDP payload contains the RTPS header and subsequent data. Reproduced from \cite{maggi2022security}.}
    \label{figure:rtps_udp}
\end{figure}

The following section briefly describes how, almost by accident, we discovered the amplification vulnerability while dissecting and modifying packets during the early stages of the development of our Scapy RTPS layer.

\subsubsection{Finding the Amplification Vulnerability}

Although network fuzzing via Scapy was not directly effective in our research, creating a Scapy layer helped and motivated us to investigate all the RTPS packets’ fields in depth. The activity led us to find the amplification vulnerability (CVE-2021-38425, CVE-2021-38429, CVE-2021-38487, CVE-2021-43547). In the long run, we recommend that developers and users leverage our Scapy RTPS layer — or similar libraries — as a building block for building continuous network fuzzers for RTPS and DDS.

The goal of the RTPS discovery phase is to send “probe” packets (e.g., to multicast addresses) and wait for responses from new locators. Locators could be IP-port pairs (see the \verb|PID_DEFAULT_UNICAST_ LOCATOR| in Figure \ref{figure:rtps_crafting}, right side of the screenshot) or memory offsets in a shared-memory transport. Before reading the specifications in depth, we assumed that an RTPS discovery packet would allow us to restrict the locator to the IP addresses within the network the machine is connected to, and would not blindly send RTPS data to any IP-port found in the locator field. On a second read, however, this is exactly how discovery works by design.

\begin{figure}[h!]
    \includegraphics[width=1\textwidth]{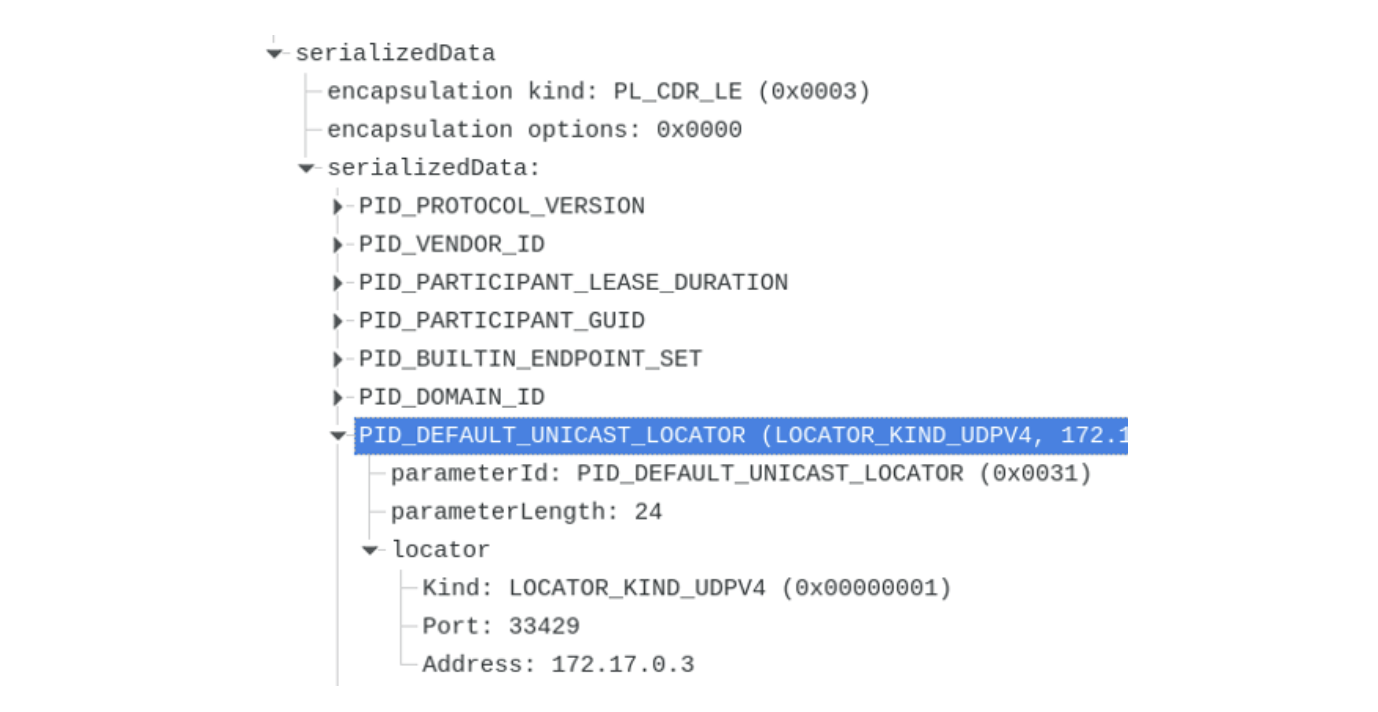}
    \centering
    \caption{In case of UDP or TCP transport, the locator is the IP-port pair. Reproduced from \cite{maggi2022security}.}
    \label{figure:locator}
\end{figure}

We discovered this by setting the \verb|PID_DEFAULT_UNICAST_LOCATOR| to the first IP address that came to mind (in this case, the Google DNS, because it’s easy to type at “8.8.8.8”). Almost immediately, a flow of outbound packets came from the DDS node, as shown in Figure \ref{figure:rtps_udp} and \ref{figure:amplification_discovery}.

\begin{figure}[h!]
    \includegraphics[width=1\textwidth]{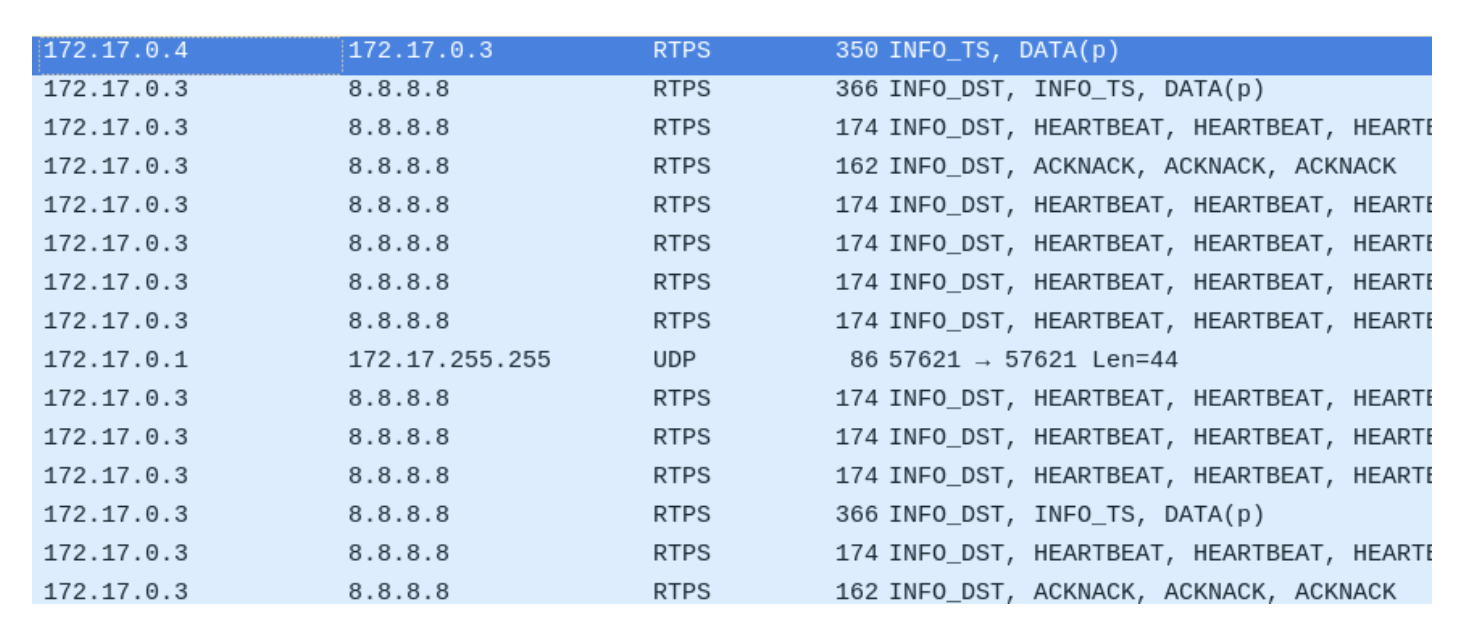}
    \centering
    \caption{We found the amplification vulnerability almost by accident, by setting the PID\_DEFAULT\_UNICAST\_LOCATOR to the first IP address that came to mind and easy to type. Reproduced from \cite{maggi2022security}.}
    \label{figure:amplification_discovery}
\end{figure}

\subsection{Source-code and Binary Fuzzing}

Of all the implementation vulnerabilities that we disclosed for this research, all but three have been found through source-code or binary fuzzing, and three through scripting a file-format input mutator (RADAMSA). There are many fuzzing tools freely available to researchers, and we chose one based on what has been used successfully for years by the largest public fuzzing platform (Google OSS-Fuss), which uses a combination of AFL++, libFuzzer, and Honggfuzz. Although the choice of the specific tool can influence the efficiency of a fuzzing campaign, we focused our attention on the most important piece: finding good fuzz targets and writing good fuzzing harnesses.

\subsubsection{Source-code Fuzzing with AFL++ and libFuzzer}

We used AFL++ for fuzzing with multiple sanitizers in LLVM. AFL++ requires the project compile with the latest version of LLVM and the build system of some DDS implementations required some work. Aside from this, most of the effort in this phase went into finding the right fuzz target and implementing a harness while keeping the code deterministic (for example, no threading).

From the high-level viewpoint depicted in Figure \ref{figure:rtps_exchange} and \ref{figure:opendds_afl}, we were interested in finding the most self-contained function in charge of processing data coming from the network. We found a repeating pattern in all the DDS implementations: upon receiving network data (i.e., \texttt{recv()} or some abstraction on top of it), there are one or more deserialization functions in which we likely find a switch-case control structure, used to dispatch the RTPS sub-message IDs to the right routine.

Given the importance of finding the right fuzz target, we dedicate the remainder of this section to showcase some examples of fuzz targets.

\begin{figure}[h!]
    \includegraphics[width=1\textwidth]{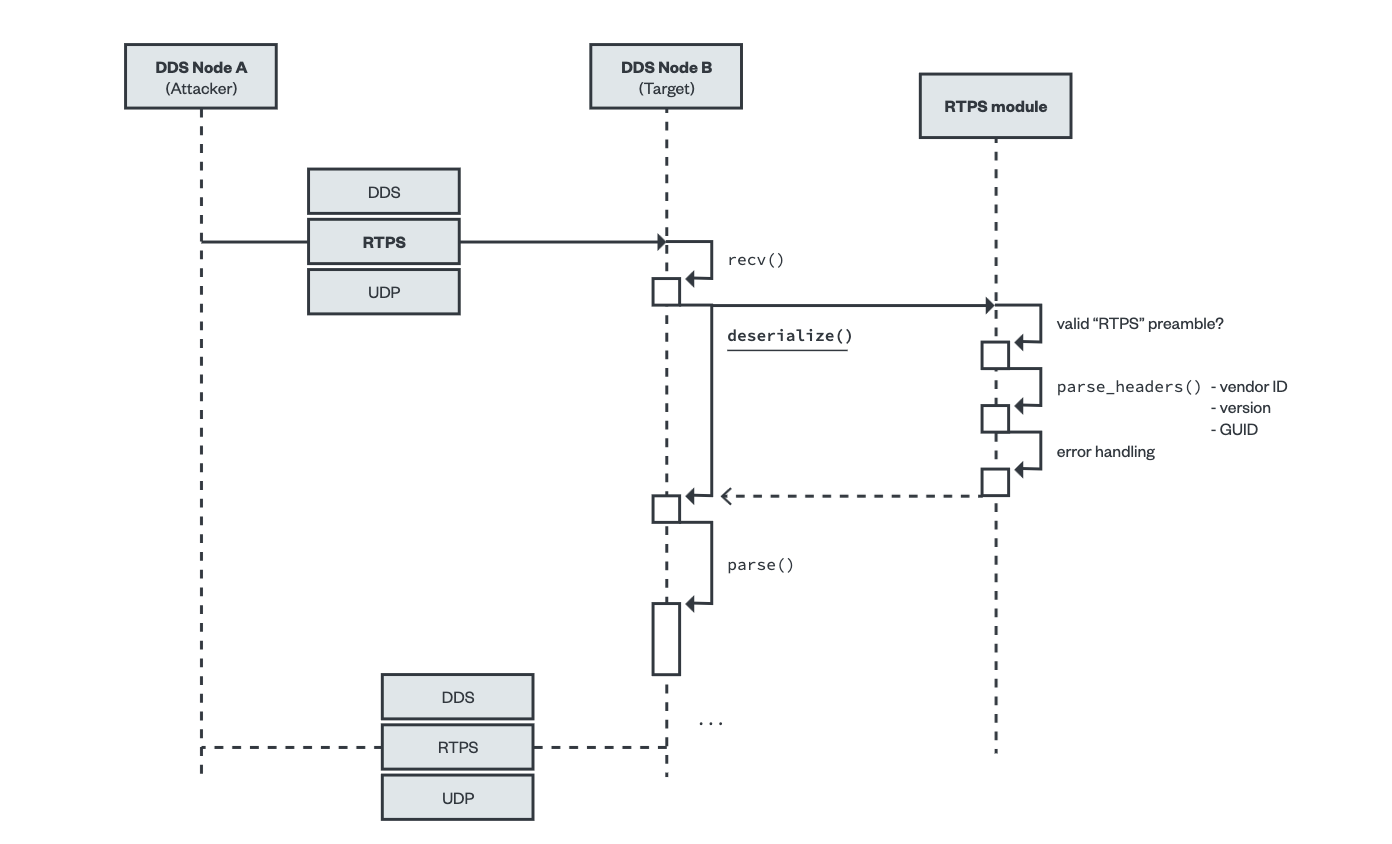}
    \centering
    \caption{Abstract representation of the data flow in a typical DDS/RTPS message exchange. From a fuzzing perspective, the \texttt{deserialize()} step is the fuzz target. Reproduced from \cite{maggi2022security}.}
    \label{figure:rtps_exchange}
\end{figure}

We began with the supplied Docker images or make files to compile example programs, run them under GDB, and inspect debug traces, like exemplified for OpenDDS in Figure \ref{figure:rtps_exchange}.

\begin{figure}[h!]
    \includegraphics[width=1\textwidth]{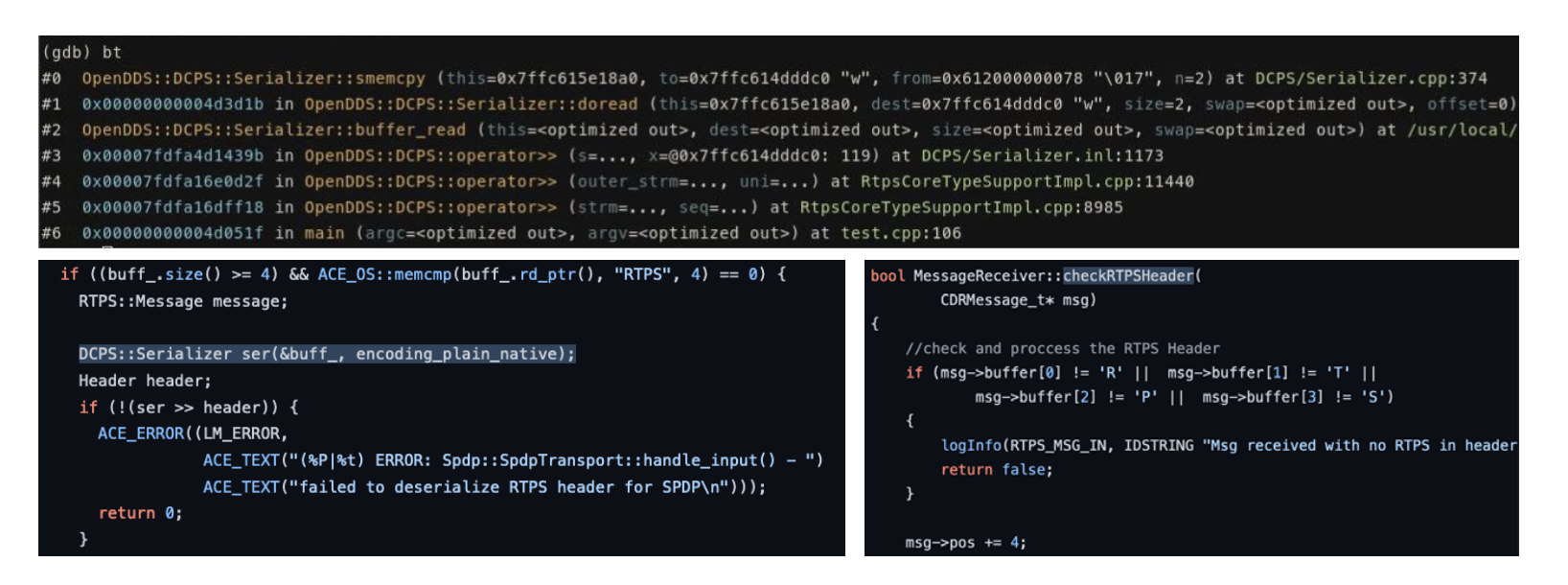}
    \centering
    \caption{Starting from a debug trace, we found interesting functions and explored further by manually looking into the source code with the aid of Visual Studio Code engine. Reproduced from \cite{maggi2022security}.}
    \label{figure:rtps_exchange}
\end{figure}

By following the function calls with the aid of the code analyzer part of Visual Studio Code, we were able to see that all three DDS implementations were using very similar procedures for deserializing network payloads. In particular, we found that they all had a switch-case to handle the RTPS sub-message types, as exemplified for OpenDDS in Figure \ref{figure:cyclonedds_harness}.

\begin{figure}[h!]
    \includegraphics[width=1\textwidth]{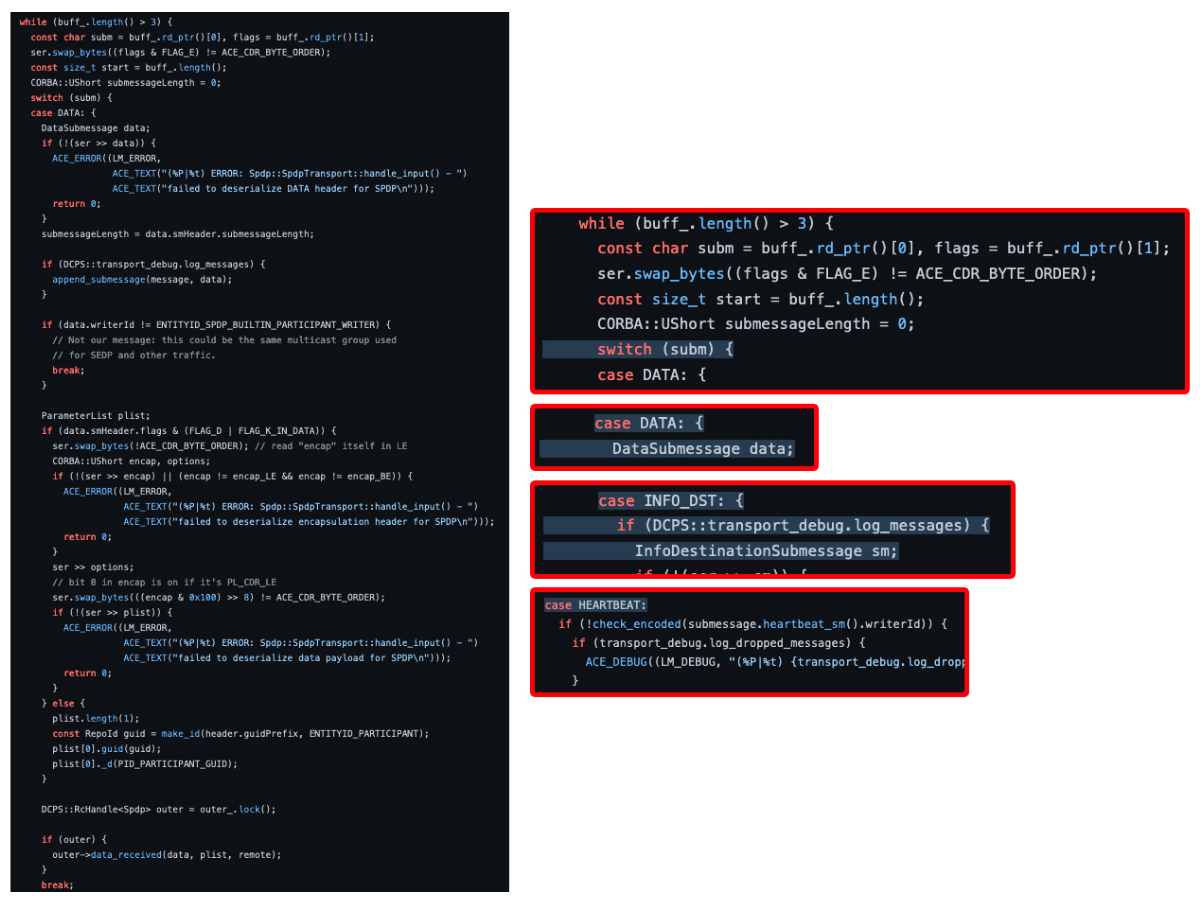}
    \centering
    \caption{Typical switch-case control structure found similar in all DDS implementations. Each of the branches takes care of one RTPS sub-message type (e.g., \texttt{DATA}, \texttt{INFO\_DST}, \texttt{HEARTBEAT}). Reproduced from \cite{maggi2022security}.}
    \label{figure:dds_switch}
\end{figure}

In some cases, we adjusted the source code right before the beginning of the de-serialization to dump the binary data being passed to the first function. This was useful to confirm that it was the network payload that we expected the function to receive, as exemplified for Cyclone DDS in Figure \ref{figure:cyclonedds_backtrace}.\\

\begin{figure}[h!]
    \includegraphics[width=1\textwidth]{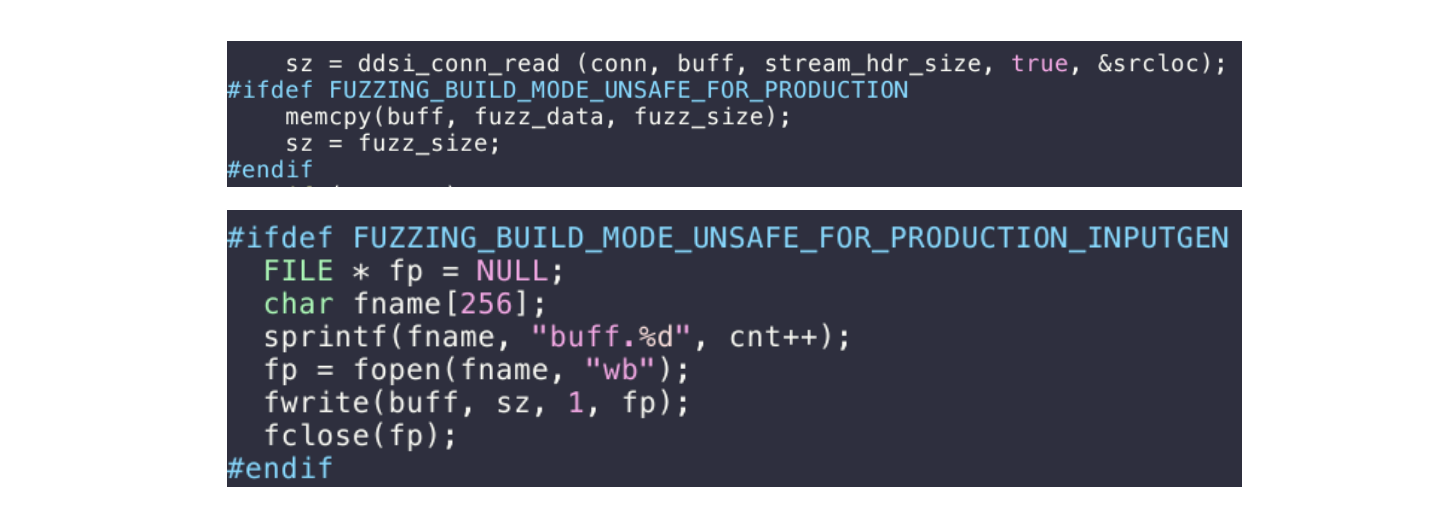}
    \centering
    \caption{(Top screenshot) In some cases, we inserted a memcpy() to directly fuzz the input in the right place when it was not possible to further decompose the function into a smaller, self-contained fuzz target. In other cases, we used the original code to dump the data received by the de-serialization routine to verify that we found the right fuzz target (bottom screenshot). Reproduced from \cite{maggi2022security}.}
    \label{figure:fuzz_strategy}
\end{figure}


\newpage

The following list showcases the groundbreaking vulnerabilities discovered by the author of this thesis, in collaboration with an international cohort of security researchers \cite{ddsvulnerabilitiesdisclosure,ddsvulnsros2}. Over an intensive 18-month study dedicated to Data Distribution Service (DDS) protocols, the author played a pivotal role in identifying these novel security flaws. These contributions, which are original and previously undocumented in the literature, highlight the author's  research results and innovative approach to cybersecurity.

Our methodical and responsible disclosure process further underscores our commitment to ethical research practices. Each vulnerability was first reported to the respective manufacturers, allowing for necessary remediation, and subsequently to relevant authorities after a stipulated grace period, ensuring an orderly public disclosure \cite{ddsvulnsros2}. This diligent process not only demonstrates our dedication to advancing the field but also our leadership in setting new standards for responsible vulnerability reporting. These discoveries represent significant advancements in DDS security and mark a notable contribution to the broader cybersecurity community in robotics.

\paragraph{CVE-2021-38445 (OpenDDS): Failed Assertion Check in RTPS Handshake}

Using the harness (shown in Figure \ref{figure:fuzzing_other}) we found out that, in OpenDDS less or equal v3.17, while receiving a RTPS packet with valid headers, with DATA sub-message, any attached serialized sub-data segment with a parameterLength of 0 will cause an assertion to fail in Serializer::doread, which subsequently called Serializer::smemcpy with a const char* from of zero. This vulnerability can be exploited via the network even without authorization and can cause the DDS node to crash. It cannot be developed into a buffer overflow so it does not grant any code-execution primitives.

More specifically, Serializer::doread does not check for segments of 0 length and continues to handle the messages. This is passed by RtpsCoreTypeSupportImpl.cpp near: 

\begin{verbatim}
bool operator >>(Serializer\& outer_strm, 
    ::OpenDDS::RTPS::Parameter\& uni) 
\end{verbatim}

which extracts size information from parameterLength but does not check if it is a valid value. It only makes sure extracting values from the serializer is successful.

\begin{figure}[h!]
    \includegraphics[width=1\textwidth]{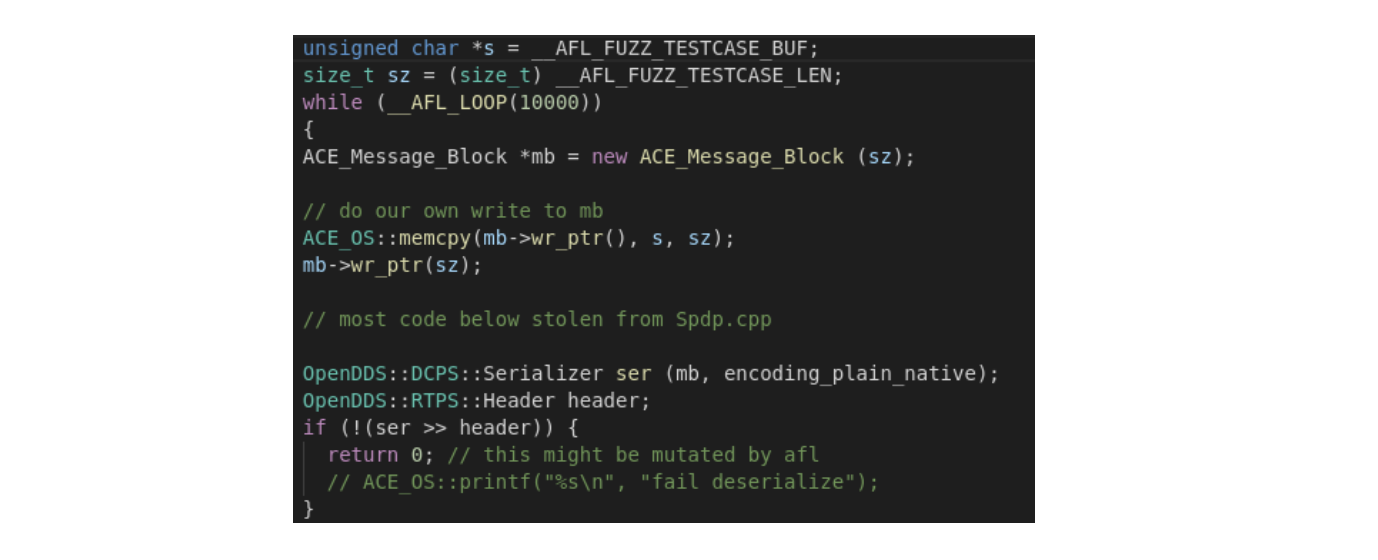}
    \centering
    \caption{Example harness for OpenDDS RTPS deserialization routine written for AFL++ using persistent mode. Reproduced from \cite{maggi2022security}.}
    \label{figure:opendds_afl}
\end{figure}

\paragraph{CVE-2021-38445 (OpenDDS): Memory exhaustion}

The opposite occurs with CVE-2021-38445 explained in the previous section, wherein the serializer is tricked into allocating very large chunks of memory. AFL++ found a crash in OpenDDS’s serializer: While deserializing data with parameter ID type 0x55, it does not sanitize the value in its length field. This allows attackers to exhaust a server’s memory by crafting a packet with a very large number in that field. In 
\begin{verbatim}
bool operator>>(Serializer\& strm, 
    ::OpenDDS::RTPS::FilterResult_t\& seq) 
    (RtpsCoreTypeSupportImpl. cpp:1977)
\end{verbatim}
a check should be made to make sure it never allocates more memory than it’s allowed, or a hard limit should be implemented.

\paragraph{CVE-2021-38441 and CVE-2021-38443 (Cyclone DDS): XML Parsing to Heap-write}

Some DDS implementations had networking functionalities plugged deep into the application code, which required some mock functions in the harness, as exemplified in Figure \ref{figure:fuzzing_unicorn} for Cyclone DDS.
\begin{figure}[h!]
    \includegraphics[width=1\textwidth]{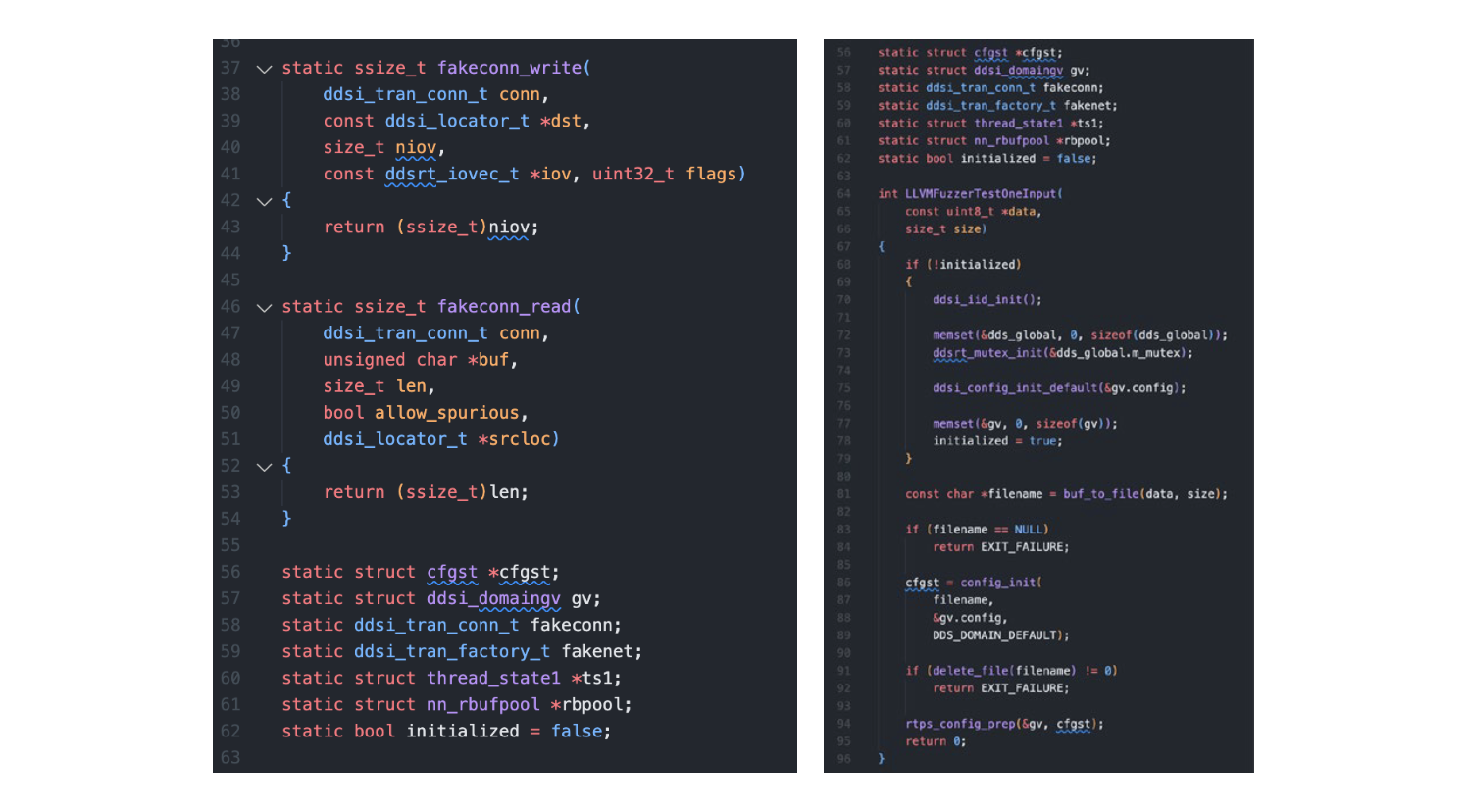}
    \centering
    \caption{Cyclone DDS harness required a mock network subsystem. The actual fuzzing is happening at line 86, where we pass the configuration initializer a pointer memory-mapped XML file. Reproduced from \cite{maggi2022security}.}
    \label{figure:cyclonedds_harness}
\end{figure}
The harness in Figure \ref{figure:fuzzing_unicorn} found several crashes, which led to two vulnerabilities. One is exemplified in Figure \ref{figure:radamsa_xml}, a multi-byte heap-write primitive. Upon checking the source code, we noticed that there were multiple inputs that can lead to a heap overflow in the XML parsing routines. This causes at least a crash and can be exploited to write in the heap, potentially overflowing into the stack. Without heap protections, this vulnerability is exploitable as it is a write primitive of at least 8 bytes, and certainly causes the program to crash in the best case.

\begin{figure}[h!]
    \includegraphics[width=1\textwidth]{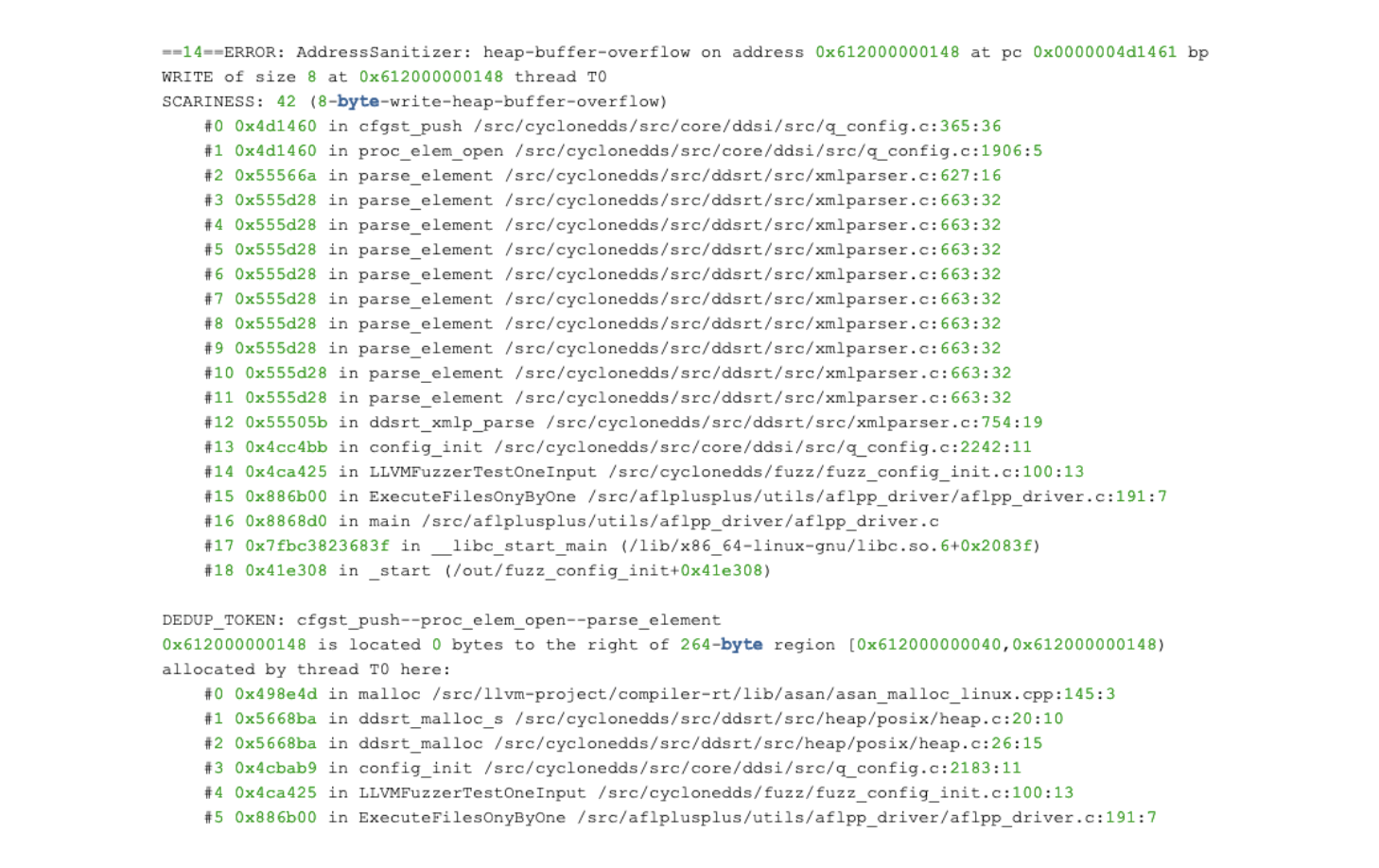}
    \centering
    \caption{Backtrace of a crash found by libFuzzer on Cyclone DDS, which led us to CVE-2021-38441, a multi-byte heap-write primitive. Reproduced from \cite{maggi2022security}.}
    \label{figure:cyclonedds_backtrace}
\end{figure}

\subsubsection{Binary Fuzzing with UnicornAFL}

The trial licenses for RTI Connext DDS, CoreDX DDS, and Gurum DDS grant access only to binary distributions of the libraries. After compiling the example programs that ship with the original software distribution, we used GDB to inspect run traces. This turned out to be quite verbose given the presence of several debug symbols. We filled the missing information by inspecting the listing via Ghidra and IDA Pro. This allowed us to find interesting fuzz targets, as seen in Figure \ref{figure:radamsa_crash}. For coverage-guided fuzzing we used UnicornAFL, which is a fork of AFL++ that uses the Unicorn emulation engine to “execute” the target and employs block-edge instrumentation in a similar fashion to AFL’s QEMU mode.
 
In practice, we dumped the context of a running process with GDB and prepared a harness (see Figure \ref{figure:dds_scan_world})	that lets UnicornAFL restore that context, set registers and memory state, the RIP register, and start emulation. Like AFL, UnicornAFL will take care of mutating the input, passing it to the fuzz target, and keep track of the coverage. The main shortcoming is that we had to re-implement some memory management functions (e.g., malloc, memset). This approach is inherently slow due to emulation, but was good enough for initial vulnerability research. It costed us a few hours of AWS EC2 computation (c5a.8xlarge), and we found that AMD EPYC machines were three times faster than Intel Xeons while fuzzing using UnicornAFL.

\begin{figure}[h!]
    \includegraphics[width=1\textwidth]{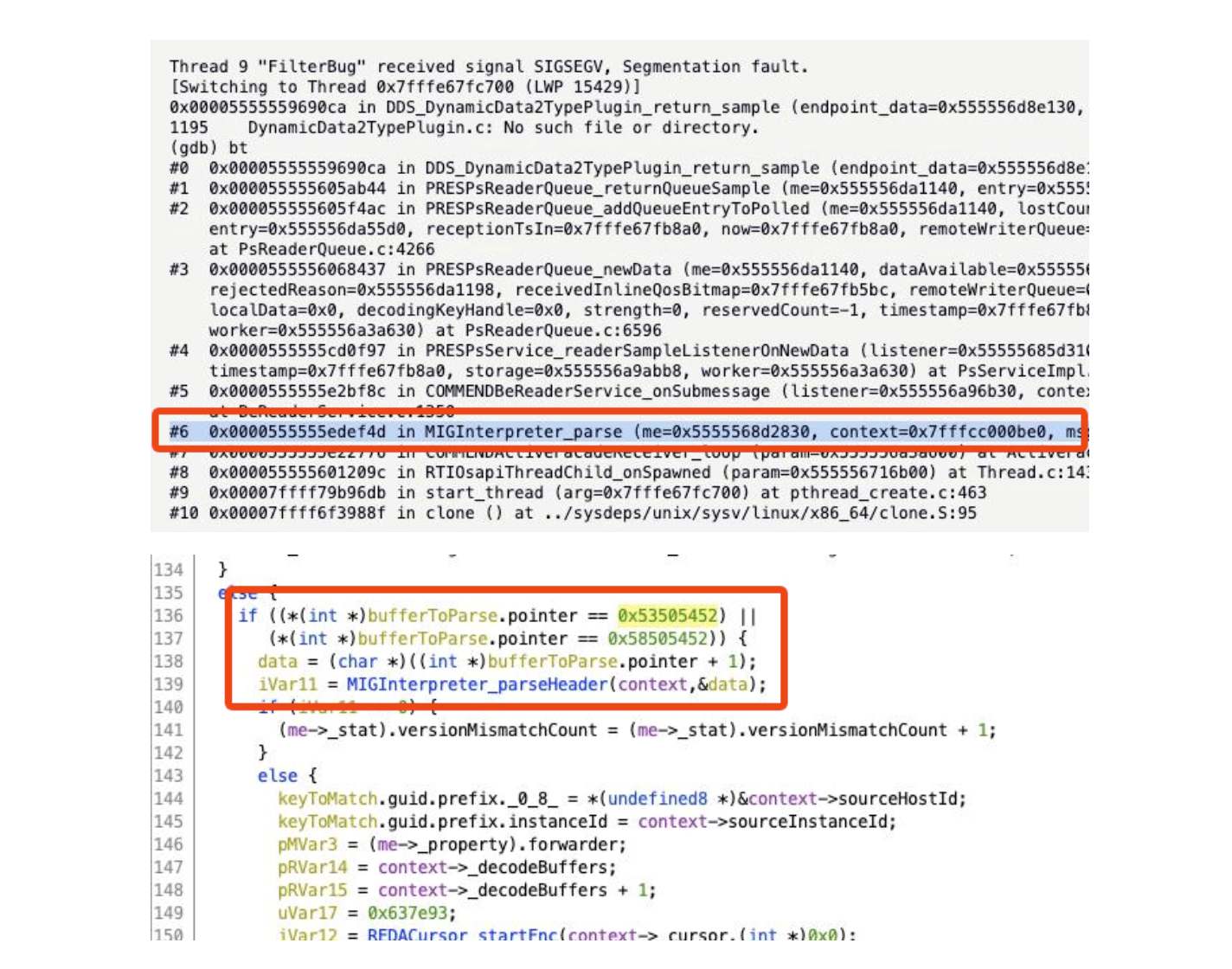}
    \centering
    \caption{Finding fuzz targets in RTI Connext DDS, CoreDX DDS, and Gurum DDS required us to reverse engineer the binary libraries, which was easy as the vendor did not use any anti-reverse engineering measures. Reproduced from \cite{maggi2022security}.}
    \label{figure:fuzzing_other}
\end{figure}

\begin{figure}[h!]
    \includegraphics[width=1\textwidth]{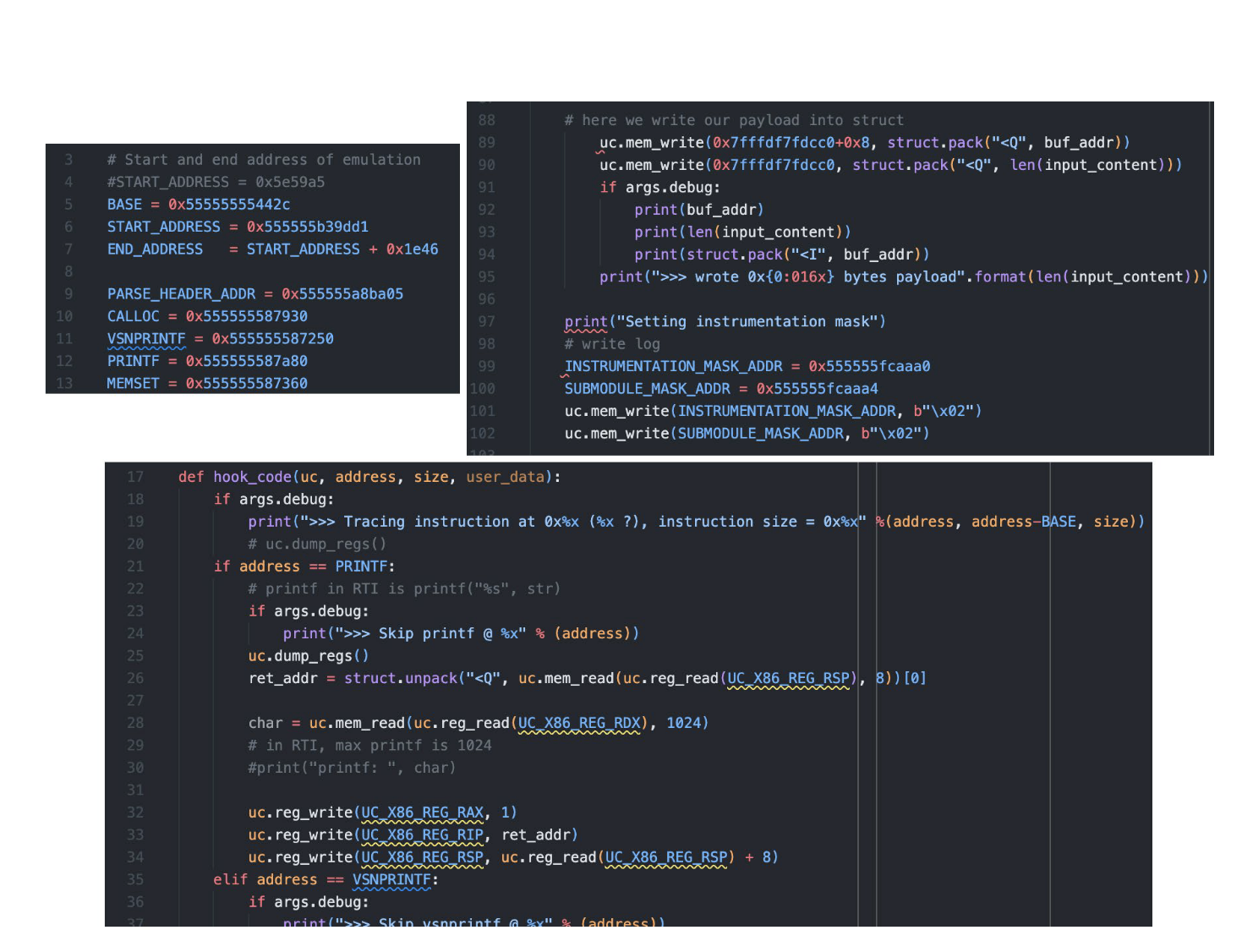}
    \centering
    \caption{We used the debug trace and the decompiled code (see Figure \ref{figure:radamsa_crash}) to create a harness for UnicornAFL. Reproduced from \cite{maggi2022security}.}
    \label{figure:fuzzing_unicorn}
\end{figure}

\paragraph{CVE-2021-38435 (RTI Connext DDS): Segmentation Fault on Malformed RTPS Packet}

The UnicornAFL instrumentation that we prepared found a segmentation-fault in the RTPS deserializer in RTI Connext DDS when receiving a malformed packet. This would cause runtimes to exit immediately and a DoS. In particular, the RTICdrStream\_skipStringAndGetLength() function does not properly check inputs, using the result straight from RTICdrStream\_align(), thus triggering a segmentation fault. Both publisher and subscriber are affected.

\paragraph{CVE-2021-38439 and CVE-2021-38423 (Gurum DDS): Heap Overflow and Segmentation Fault}

While using UnicornAFL on Gurum DDS fuzz targets, we discovered that there is a heap overflow in the RTPS routine that handles payload parsing. This causes a segmentation fault leading to DoS. More specifically, the crash is triggered in rtps\_read\_AckNackMessage() function when called in read\_ Submessage(), which creates a multi-byte heap overflow condition. We found this crash by using a harness that passes RTPS payload directly to the rtps\_read\_Data(…, buf, len, …) function through the buf argument.

Another case we\footnote{For the full list of contributors, refer to \cite{ddsvulnerabilitiesdisclosure,ddsvulnsros2}} found is in the \texttt{rtps\_Parameter\_load2()} function, which does a type conversion from a buffer and does a check to exclude specific IDs. During this conversion, we found some inputs causing a segmentation fault. We verified that this is exploitable via network by crafting a packet based on the crash dump provided by the fuzzing engine.

\subsubsection{Scripting RADAMSA to Mutate XML Files}

At the beginning of this research and before using AFL++ and UnicornAFL, we used RADAMSA directly, with some simple shell scripting (see Figure \ref{figure:dds_scans_methodology}). Without any prior knowledge on the target software, this simple technique can be surprisingly effective at finding crashes, which can also lead to the discovery of vulnerabilities.

\begin{figure}[h!]
    \includegraphics[width=1\textwidth]{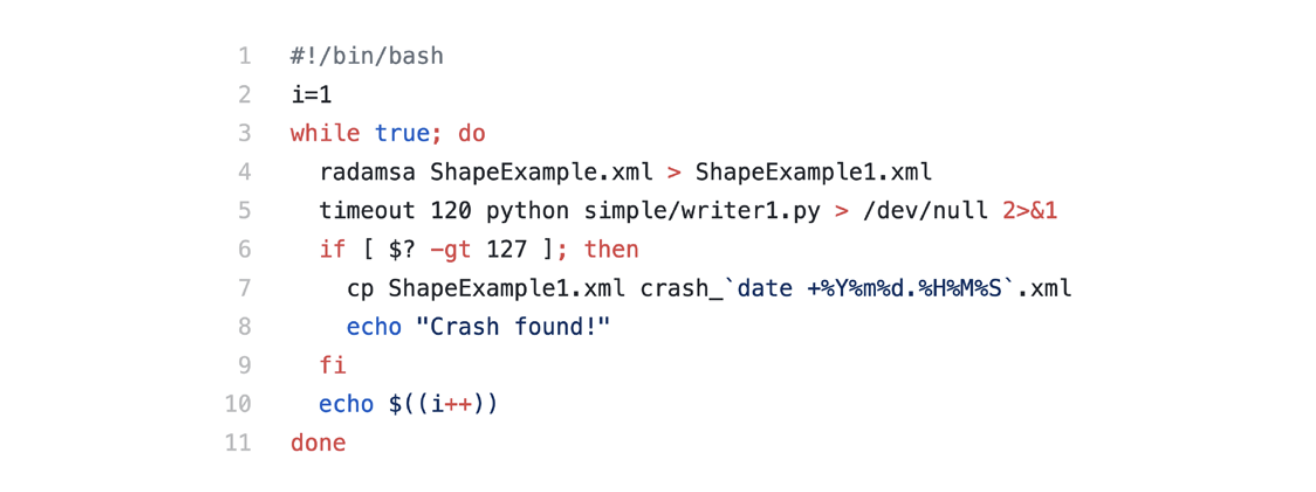}
    \centering
    \caption{A simple scripting of RADAMSA can lead to surprising results. Reproduced from \cite{maggi2022security}.}
    \label{figure:radamsa_xml}
\end{figure}

\paragraph{CVE-2021-38427 and CVE-2021-38433 (RTI Connext DDS): Stack-based Buffer Overflows Python Bindings}

The simple “harness” shown in Figure \ref{figure:dds_scans_methodology} allowed us to find two vulnerabilities; one could be exploited beyond just a crash to control a pointer using a malformed XML file. When the length of an attribute value in a configuration XML file is longer than a certain limit, RTIXMLObject\_lookUpRef() would trigger a buffer overflow. If the length is exactly 894 characters, we could overwrite RIP register (see Figure \ref{figure:dds_scans_methodology}). We have not investigated further, but we see the possibility of preparing a ROP chain for this target. However, the XML parser does not accept arbitrary hexadecimal characters, so we’re limited within the Unicode range.

\begin{figure}[h!]
    \includegraphics[width=1\textwidth]{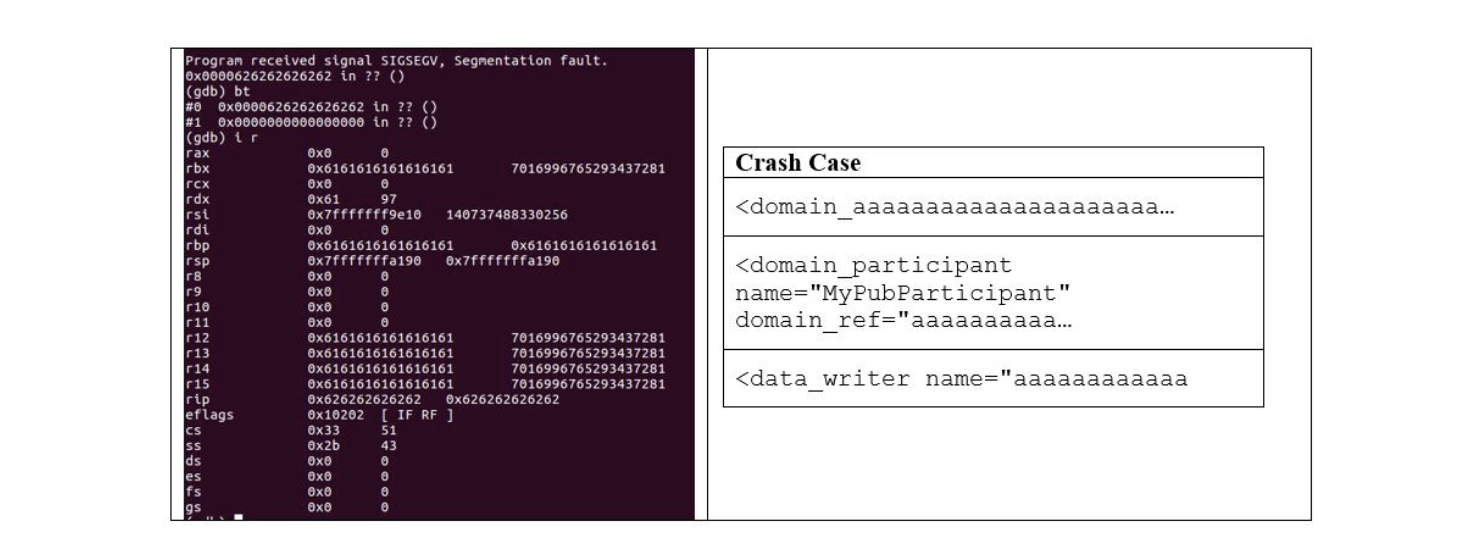}
    \centering
    \caption{Register state and sample crash cases in RTI Connext DDS Connector (Python bindings) found with RADAMSA from the original XML configuration file. Reproduced from \cite{maggi2022security}.}
    \label{figure:radamsa_crash}
\end{figure}

A variant of this vulnerability is in RTIXMLDtdParser\_getElement(), which does not properly validate the length of an element prior to copying it to a fixed-length stack buffer.

\paragraph{Unmaintained XML Parsing Libraries}

Unfortunately, we discovered that Gurum DDS uses ezXML, an open-source XML library that has been in beta status since 2006 that has never been updated. The mailing list of the project has been silent since 2010, showing that no users are actively discussing it. The latest version was 0.8.6, but an inspection of the binary code revealed that the developer changed the version number to 1.0.0, which was the only change. EzXML currently has 16 known vulnerabilities (eight in 2021), all with medium to high severity ratings and have never been fixed. Probably because of its small footprint, we discovered that ezXML is also used in many embedded applications like router firmware, and has hundreds of downloads per week. We have reached out to Gurum DDS several times — about this and other vulnerabilities — since the summer of 2021 and have received no response.

\subsection{Internet-wide Scanning for RTPS Endpoints}

We wanted to demonstrate how an attacker could leverage the RTPS built-in discovery protocol for automated, large-scale reconnaissance of RTPS/DDS endpoints. We found hundreds of exposed services as a byproduct, which was unexpected. Understanding that RTPS/DDS was designed for local-network applications, we did not expect to find more than a couple of endpoints exposed by mistake. Not only did we find several hundreds, but 35 of them have never stopped sending responses to our scanner despite the fact that we only sent them one single RTPS packet.

After trying to use readily available internet scanning services (such as Shodan, Censys, and LeakIX), we ended up implementing our own scanning prototype because of the intricacies of the RTPS discovery phase. This makes it a bit convoluted to correctly fingerprint a service.
 
\begin{figure}[h!]
    \includegraphics[width=1\textwidth]{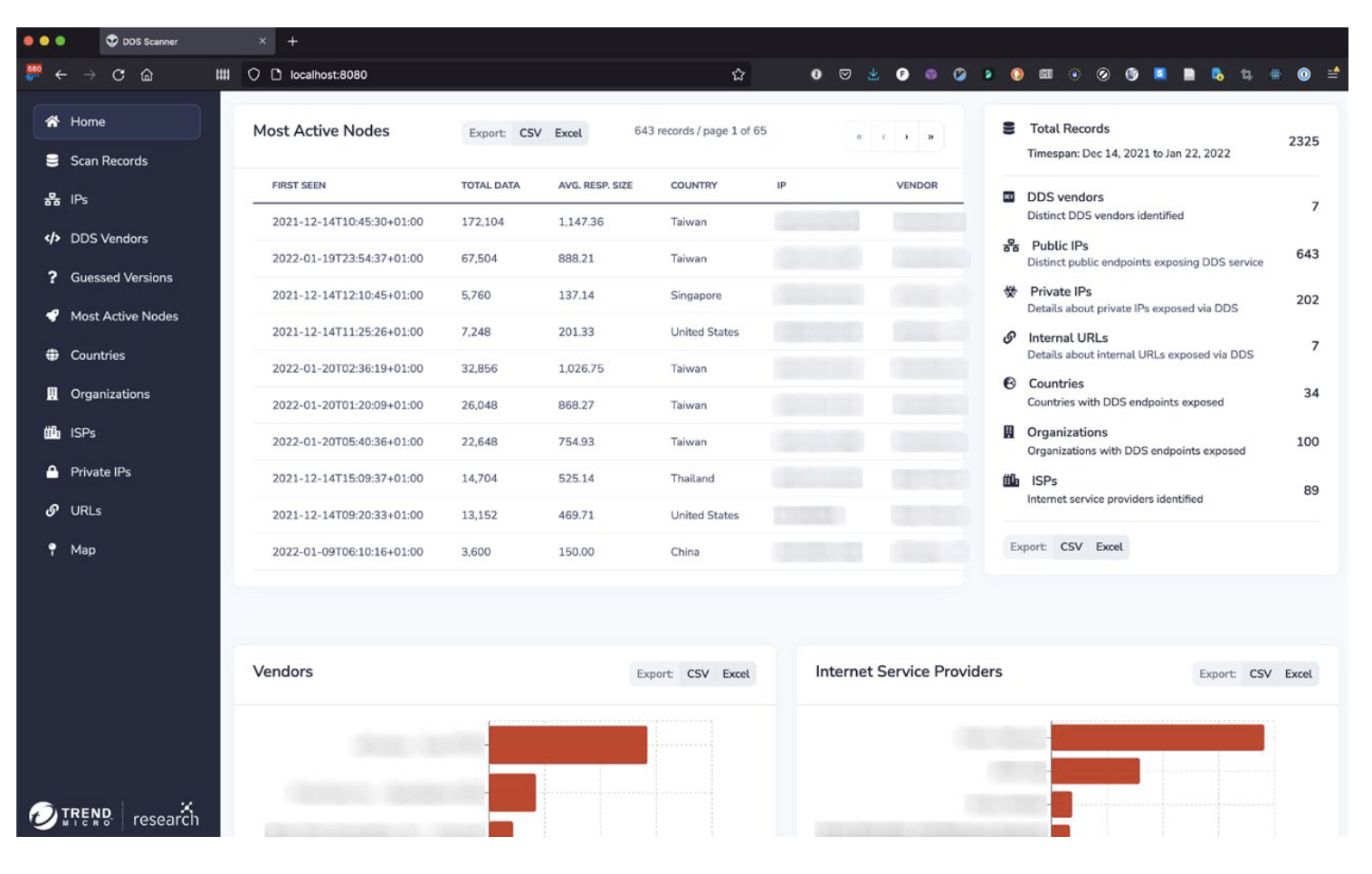}
    \centering
    \caption{The dashboard of our DDS-scanning system allows analysts to explore the data. User interface is not publicly available. Reproduced from \cite{maggi2022security}.}
    \label{figure:dds_scan_world}
\end{figure}

\subsubsection{Challenges of RTPS/DDS Reconnaissance}

In summary, the main challenges of RTPS/DDS reconnaissance are:
\begin{itemize}
    \item Dynamic and arbitrarily large port range. Depending on the number of participants in a RTPS/DDS network, there can be tens of thousands of ports to check. The formula to calculate the port is defined in the specifications \cite{rtps25}, and each implementation has distinct defaults, as shown in Table \ref{table:dds_udp_ports}.

    \item Latency and connectionless nature. Although RTPS/DDS are transport-agnostic, the de facto standard is to use UDP, which makes efficient scanning techniques useless. To verify if there is a valid RTPS/DDS endpoint bound to a given address (IP and UDP port), we need to wait for an answer, which may or may come within a few seconds. Given the size of the public IPv4 space, it’s impractical to wait for answers upon each request.

    \item Addressing information at application layer. Addressing information is exchanged at the application layer. Sending a valid RTPS packet to the correct UDP port (for example, the default 7400 discovery port) does not guarantee a response, even if there is an RTPS service running. To trigger a response, the RTPS discovery packet must include correct locator information (for instance, IP and UDP port), which will receive a response.
\end{itemize}

\begin{table}[h!]
    \includegraphics[width=1\textwidth]{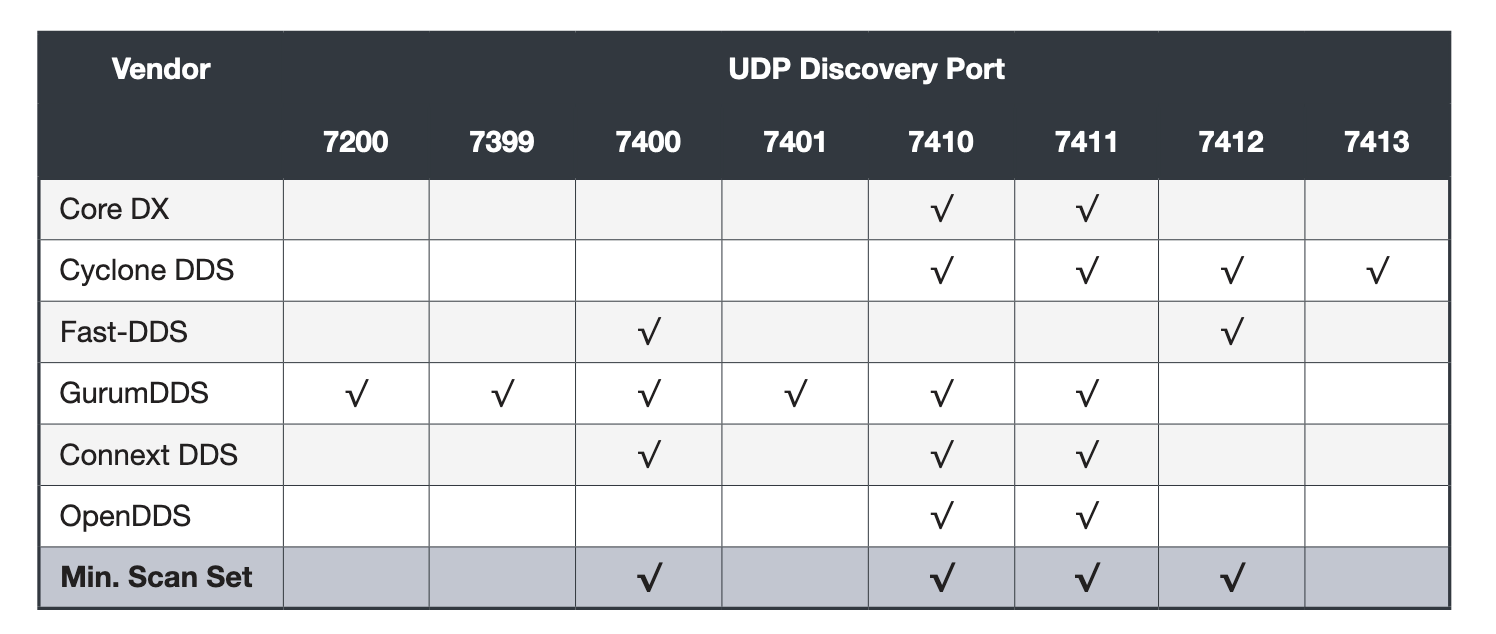}
    \centering
    \caption{Assuming up to one DDS domain and at least one participant, we tested the open UDP ports of each of the six reference implementations. All those marked with checks could detect any of the identified ports just by scanning for four ports listed on the last row of the table. Reproduced from \cite{maggi2022security}.}
    \label{table:dds_udp_ports}
\end{table}

\subsubsection{Scanning Approach}

Given the challenges mentioned in the previous section, we implemented a distributed scanning system (see Figure \ref{figure:amplification_discovery}) that we first validated in a private network against all six DDS implementations, under the simplifying assumption that the developer would not go too far from the “default” set of ports listed in Table \ref{table:dds_udp_ports}.

\begin{figure}[h!]
    \includegraphics[width=1\textwidth]{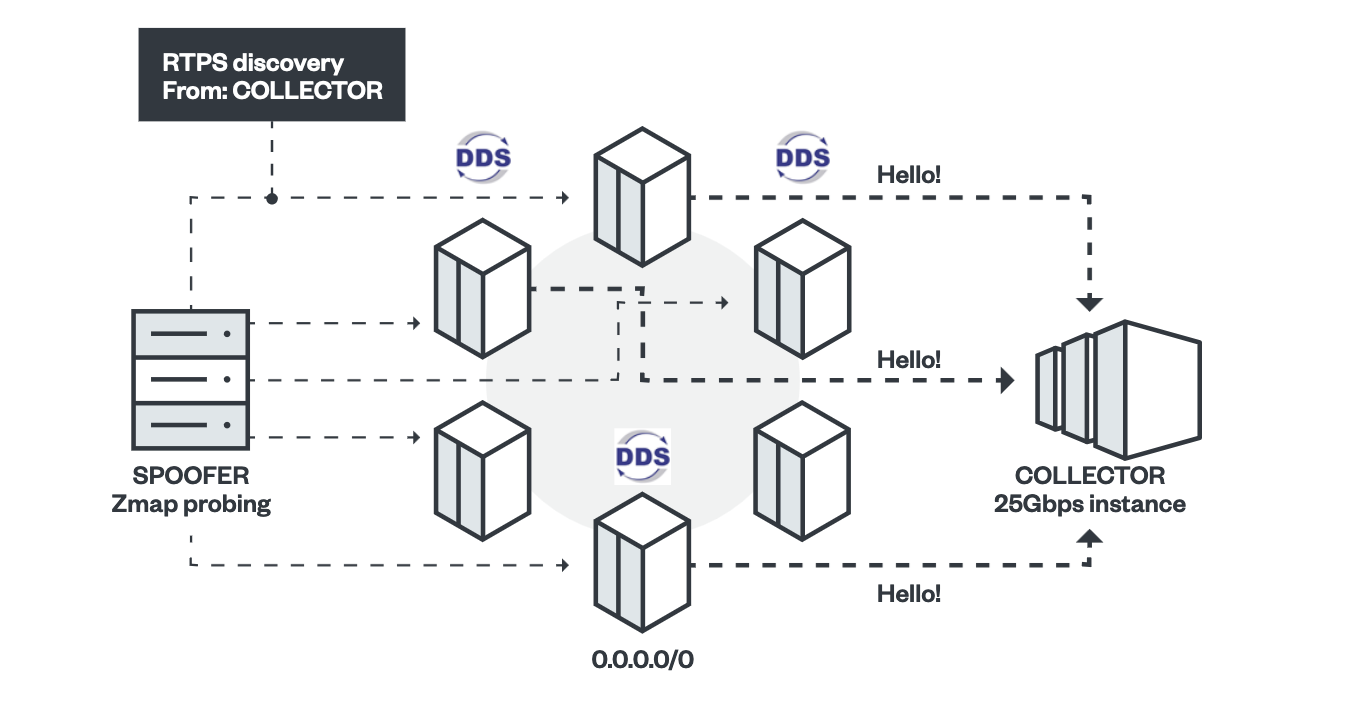}
    \centering
    \caption{We used spoofed RTPS discovery messages sent via ZMap to collect answers from valid DDS endpoints and filtered echoed and invalid responses. Reproduced from \cite{maggi2022security}.}
    \label{figure:dds_scans_spoofed}
\end{figure}

As shown in Figure \ref{figure:rtps_exchange}, we created a template RTPS packet (using our Scapy RTPS layer, as shown in Figure \ref{figure:rtps_exchange}) with a parametric locator IP and port number. We then generated an actual RTPS discovery packet by fixing the locator IP and port numbers according to the collector that we set up to receive the (reflected) packet. The collector will know what packet to expect given its IP and port number. Since there are several honeypots that simply reply to every request by echoing traffic they receive, the collector filters these “echoed” packets and keeps only valid responses. The collector checks if a received RTPS packet is valid by using the Scapy layer to dissect it and checks whether the globally unique identifier field (GUID) is new.

\begin{figure}[h!]
    \includegraphics[width=1\textwidth]{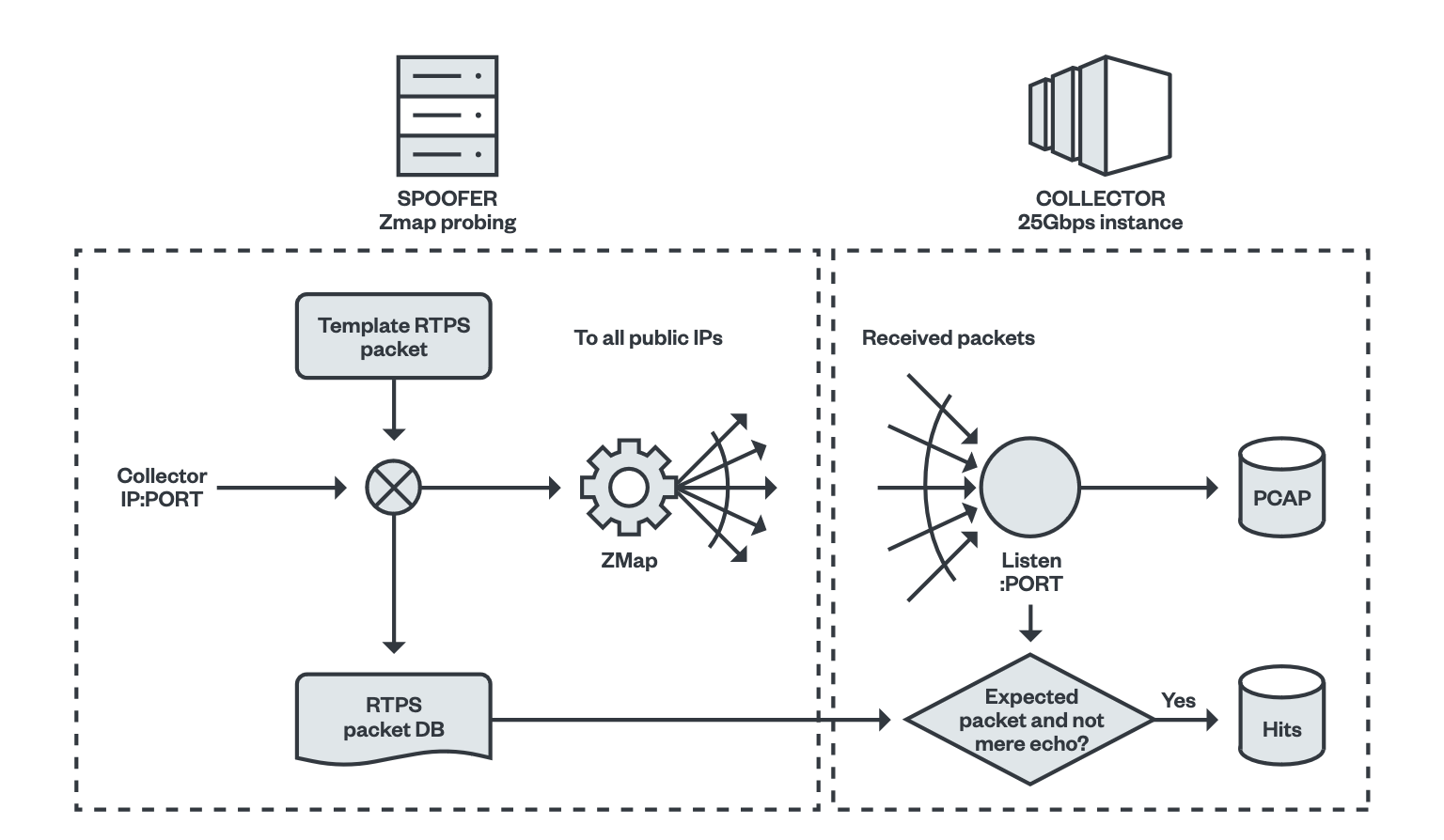}
    \centering
    \caption{Starting from a template RTPS packet, the spoofer instantiates it for a given collector (IP and port) and sends it out via ZMap. The collector uses the RTPS packets sent out to decide whether the received packets are valid and not simply duplicates of what was sent out. Reproduced from \cite{maggi2022security}.}
    \label{figure:dds_scans_methodology}
\end{figure}

\FloatBarrier

\section{Conclusion}

This chapter has provided an in-depth analysis of robot software with particular focus on the the Data Distribution Service (DDS) standard and its implementation across various critical sectors, revealing significant security vulnerabilities and widespread exposure of DDS services. Through rigorous research methodologies, including network and binary fuzzing, our investigation uncovered numerous vulnerabilities, resulting in the issuance of new CVE IDs. The findings highlight the urgent need for continuous security testing and robust supply chain management in DDS and similar technologies.


The extensive exposure of DDS services on public-facing networks and the discovery of vulnerabilities in standard specifications call for immediate and long-term mitigation strategies. As we move forward, the focus should be on enhancing the security frameworks of DDS implementations and advocating for more secure practices in the development and deployment of critical middleware technologies. This research serves as a foundation for future work in securing robot software and other similar systems.
    
    \chapter{Attacking Robot Hardware}
\label{chapter:attacking_hardware}

\epigraph{\textbf{Contributions}: \emph{Leading author at \cite{mayoral2021hacking}. Self-funded work for myself and the team participating. All findings have been publicly disclosed.}}{}


\section{Introduction}

\showkeysfalse  
\label{sec:introduction}
\showkeystrue  

Robotics is the art of system integration, as noted by \cite{mayoral2017shift}. Building a robot requires one to carefully select components that exchange information across networks while meeting timing deadlines. In a way, a robot is a network of networks. One that comprises sensors to ``read'' the world, actuators to produce a physical change, and one for dedicated computational resources to process it all and respond coherently, in time, and according to its application. Roboticists often conceive the robot not as one of its parts, but as the complete system including all of its components, whether they are assembled under the same structure or physically distributed. In the case of a robotic manipulator, these robots are often presented physically distributed and include the robot arm mechanics (which generally include actuators and sensors), the \ac{HMI}  or teach pendant, the controller (the main compute substrate for reasoning), and any additional safety mechanism related to the robot operation. The robotic system is thereby the \textit{composition} of all these sub-systems and networks.

Under such system integration complexity, it is not uncommon for one of the robot sub-components to fail over time, often leading to the complete system malfunction. Given the high price point of robots, it is reasonable to consider the need for repairing these machines, often replacing individual faulty components for continued operation, or simply for re-purposing them. The European Commission (EC) showed early interest on this topic in a report by~\cite{cordella2019analysis} evaluating different scoring systems for repairing and upgrading different consumer-oriented products, including robots. More recently, and as part of the Circular Economy Action Plan~\cite{eccircularap}, the EC has shown commitment towards establishing a new \emph{Right to Repair} in the context of reviewing directive 2019/771. \cite{hatta2020right} summarizes major events in the U.S. with regard the \emph{Right to Repair} and highlights that it wasn't until 2012 that the Automotive \emph{Right to Repair} passed in Massachussets, empowering customers with tools to fight planned obsolescence. Hatta summarizes how material obsolescence works:
\begin{itemize}
    \item Making items difficult to repair (by raising the cost of repair, requiring special tools, etc.)
    \item Failing to provide information (for instance, manuals are not provided)
    \item Systematic obsolescence (making parts among models incompatible or making it impossible to fix newer models with parts from the older models)
    \item Numbering (frequently changing the model numbers to make it psychologically less attractive to use old models)
    \item Legal approaches (prohibiting access and modification to the internal structure of products by means of copyrights and patents)
\end{itemize}
\cite{hatta2020right} noticed that, similar to Ford in the 1920s, most robot manufacturers follow several of these practices nowadays and organize dealers (often called distributors) or approved system integrators into private networks, providing repair parts only to \emph{certified} companies in an attempt to discourage repairs and evade competition. Amongst the most recent examples we observe an interesting development from \href{https://www.teradyne.com/}{Teradyne}, where two of its owned robotics companies (\href{https://www.universal-robots.com}{Universal Robots} and \href{https://www.mobile-industrial-robots.com}{Mobile Industrial Robots}), follow this practice. The case of Teradyne is of special interest because its robots are advertised as collaborative, that is: designed to augment human capabilities by closely (physically) cooperating without causing any harm. Past research however hints that the lack of security measures in these robots leads to safety hazards, as concluded by~\cite{kirschgens2018robot, mayoral2019industrial, taurer2019case}.

Cybersecurity in robotics is still on its early stages, as demonstrated by~\cite{quarta17:robosec,kirschgens2018robot,pogliani20:detecting}. Therefore, as in many other fields, it remains addressed mostly in disconnected silos. With most efforts concentrated in IT, hardware security in robotics has received very limited attention. Building secure robots, however, demands consideration throughout domains (hardware, firmware, OS, application, network, cloud, etc.), \cite{vilches2018introducing}, and across the robot lifecycle,~\cite{mayoral2020devsecops}.

The present chapter introduces and promotes robot teardown as a systematic \emph{process}
to repair robots, improve robot hardware and research its security. We advocate against the business priorities set in industry, which avoid repairs and planned obsolescence. Instead, we advocate for a \emph{Right to Repair} in robotics as a means to reduce robot e-waste and recycle components, both across robots and throughout use-cases. Ultimately, we argue that, in the long run, the more researchers and practitioners will learn to systematically teardown robots, the more this practice will impact the quality assurance of hardware in robotics, putting pressure on manufacturers to produce robots with better (hardware) security measures, thereby safer. Our contributions are fourfold: first, we discuss the empirical results of three robot teardowns performed on popular industrial collaborative robots and uncover various quality, security, and safety flaws in the process. Second, we demonstrate how, as a result of a teardown, we gain repairing capabilities in the robots. This leads us to acquire means to mitigate security flaws early, by simply extending the robotic system with additional, off-the-shelf hardware elements that increase the overall cybersecurity posture with a minimal cost impact. Third, we show how teardown helps pinpoint security vulnerabilities across internal and external robot networks while discussing some of them. Fourth, we show evidence of planned obsolescence practices in robotics on leading industrial collaborative robots and demonstrate how by applying minor fixes, we managed to bypass the obsolescence limitations obtaining full control of the hardware across subsequent releases.


The content below is organized as follows: Section \ref{sec:teardown} describes the robot teardown process in three different robots and Section~\ref{sec:reversing} the posterior reversing exercise to gain repairing capabilities. Section~\ref{sec:obsolescence} argues about the obsolescence indicators encountered and demonstrates how to bypass them as a result of the results in previous sections. Finally, Section~\ref{sec:conclusions} summarizes our work and draws some conclusions.

%
%
%
%



\section{Robot teardown}

\showkeysfalse  
\label{sec:teardown}
\showkeystrue  

A teardown is the process of taking apart a product to understand how it is made and works. More formally, it is the approach to modeling the functional behavior and physical components of a product, as described by~\cite{younis2010reverse, skorobogatov2017deep, tutunjireverse}. Robot teardown is thereby the process to study robot hardware architectures through systematic disassembly to understand how the robot works and what physical sub-systems compose it.


The motivation behind teardowns was previously researched by other groups such as~\cite{kohlweiss2020integration, sandborn2006using}. In robotics, we identify three key purposes: a) dissection and analysis to evaluate the status of a product, b) competitive benchmarking against similar products, and c) gain engineering experience and knowledge. This chapter focuses on a) and c). Particularly, we show three case studies on the robots from Universal Robots (UR) and Mobile Industrial Robots (MiR). Our motivation for selecting these targets is twofold: first, these robots are arguably widely used across use cases in the professional and industrial environments, with tenths of thousands of units sold \cite{teradyne_growth} and operating in close contact with humans (as collaborative robots). Second, past research by~\cite{hackingbeforeskynet, hackingbeforeskynet2, vilches2018aztarna, zhu2021cybersecurity, mayoralrvd, taurer2019case} has shown a lack of security concern and readiness from these two manufacturers making them attractive targets for adversaries aiming to disrupt industrial processes or causing injuries as reported by~\cite{kirschgens2018robot}. Disruption-based attacks, unfortunately, continue to be the most effective leverage used by financially driven threat actors such as DarkSide\footnote{\url{https://www.intel471.com/blog/darkside-ransomware-shut-down-revil-avaddon-cybercrime}}, just to name the most damaging and recent one (as of June 2021).

Based on common teardown practices \cite{kohlweiss2020integration, sandborn2006using}, we present in Figure \ref{figure:teardown_methodology} our teardown methodology for robots. The process involves 5 steps: \textbf{1.} Identifies the purpose and scope of the teardown exercise. \textbf{2.} Prepares for the teardown gathering required tools for documentation and dissasembly. \textbf{3.}  Examines the supply chain identifying how to acquire parts, what's the installation process and who's entitled for repairs, including costs and liabilities. \textbf{4.} Takes apart the robot, documenting each step and avoiding the damage of any component. \textbf{5.} Extracts relevant data (e.g. firmware version) from each robot component, constructs a \acp{BOM} and gathers additional information by researching public resources. To the best of our knowledge, we are the first to propose and document a teardown approach for industrial robotic products. The following subsections provide a walk-through on three case studies and discuss the most interesting findings on each one of them.

\begin{figure}[!h]
    \centering
    \begin{tikzpicture}[
        scale=0.9,
        >=latex,thick,
        /pgf/every decoration/.style={/tikz/sharp corners},
        fuzzy/.style={decorate,
        decoration={random steps,segment length=0.5mm,amplitude=0.15pt}},
        minimum size=6mm,line join=round,line cap=round,
        terminal/.style={scale=0.9,rectangle,draw,fill=white,fuzzy,rounded corners=3mm},
        nonterminal/.style={scale=0.9,rectangle,draw,fill=white,fuzzy},
        node distance=4mm
      ]
        \begin{scope}[start chain,
                every node/.style={on chain},
                terminal/.append style={join=by {->,shorten >=-1pt,
                    fuzzy,decoration={post length=4pt}}},
                nonterminal/.append style={join=by {->,shorten >=-1pt,
                    fuzzy,decoration={post length=4pt}}},
                support/.style={coordinate,join=by fuzzy}
            ]
            \node [nonterminal](reco)                  {1. Scope};
            \node [nonterminal](weapo)                 {2. Tooling};
            \node [nonterminal]                        {3. Supply chain};
            \node [nonterminal]                        {4. Dissasembly};
            \node [nonterminal](lateral)               {5. Info. gathering};
            \node [coordinate] (end)                   {};
        \end{scope}
    \end{tikzpicture}
    \vspace{1em}
    \caption{Our teardown methodology for robots}
    \label{figure:teardown_methodology}
\end{figure}
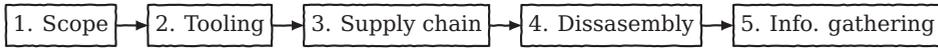




%
%


\subsection{Case Study 1: Teardown of an industrial collaborative robot}



\showkeysfalse  
\begin{figure*}[!h]
    \centering
    \begin{subfigure}[b]{0.3\textwidth}
        \includegraphics[width=\textwidth]{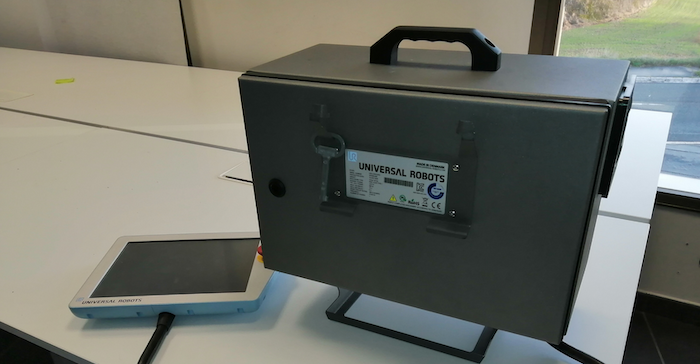}
        \caption{\footnotesize Universal Robots UR3 robot CB3.1 controller and teach pendant (HMI). Controller has a mechanical lock to secure physical access.}
        \label{fig:ur:1}
    \end{subfigure}
    ~ 
    \begin{subfigure}[b]{0.3\textwidth}
        \includegraphics[width=\textwidth]{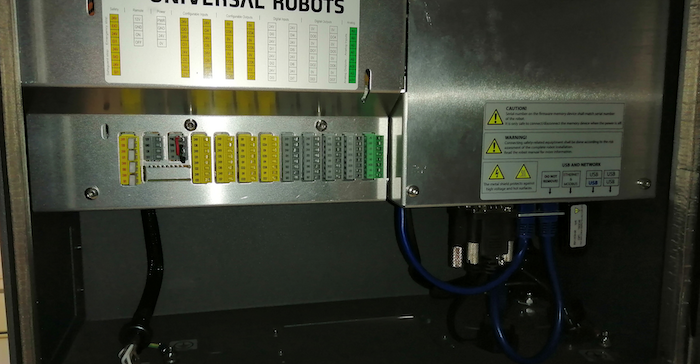}
        \caption{\footnotesize Inside the controller we learn about connectors and cables, which are exposed. The left side includes I/O and safety.}
        \label{fig:ur:2}
    \end{subfigure}
    ~
    \begin{subfigure}[b]{0.3\textwidth}
        \includegraphics[width=\textwidth]{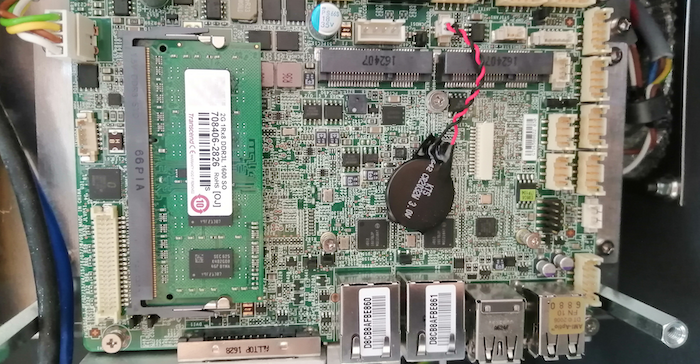}
        \caption{\footnotesize The main computer of the controller with a 2G DDR3L RAM module. Ethernet NICs are connected to controllers from Intel.}
        \label{fig:ur:3}
    \end{subfigure}
    \\[3ex]
    \begin{subfigure}[b]{0.3\textwidth}
        \includegraphics[width=\textwidth]{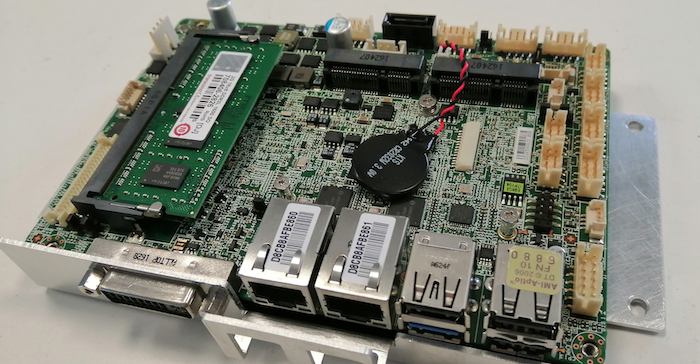}
        \caption{\footnotesize No secondary memory is located on the \ac{PCB} besides minor non-volatile memories and the USB stick we found connected outside.}
        \label{fig:ur:4}
    \end{subfigure}
    ~
    \begin{subfigure}[b]{0.3\textwidth}
        \includegraphics[width=\textwidth]{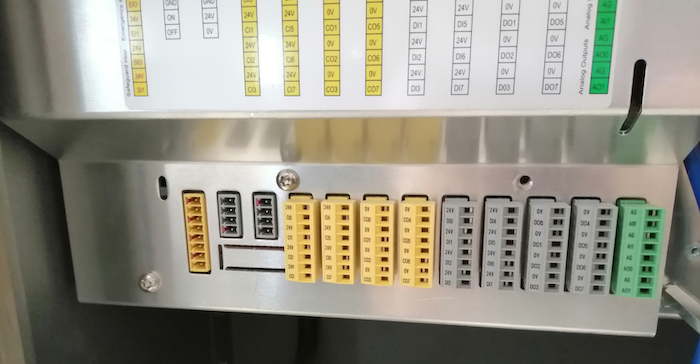}
        \caption{\footnotesize The safety side of the controller (documented in the user manuals) includes quick connectors which can be removed by carefully wiggling them out.}
        \label{fig:ur:5}
    \end{subfigure}
    ~
    \begin{subfigure}[b]{0.3\textwidth}
        \includegraphics[width=\textwidth]{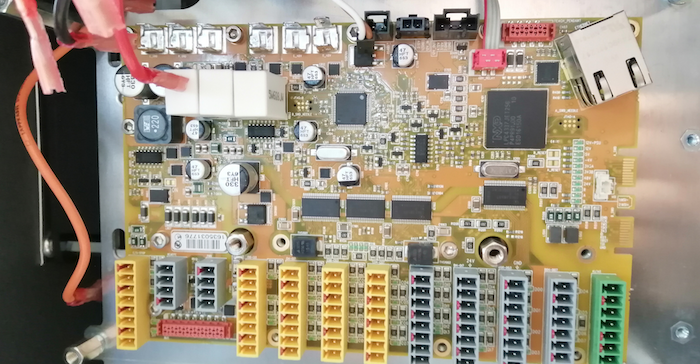}
        \caption{\footnotesize After removing the metal shields, the safety board electronics are fully visible. The main logic is driven by an NXP LPC4437JET256 \ac{MCU}.}
        \label{fig:ur:6}
    \end{subfigure}
    \\[3ex]
    \begin{subfigure}[b]{0.3\textwidth}
        \includegraphics[width=\textwidth]{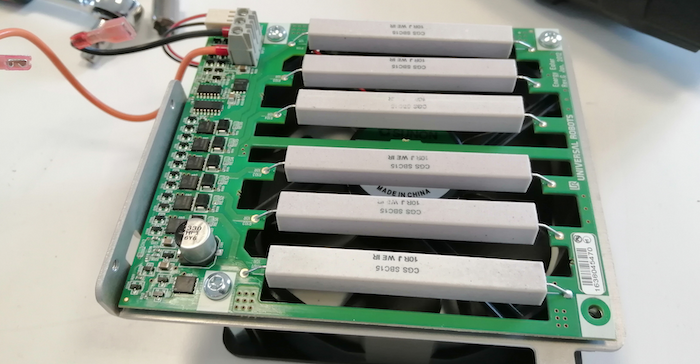}
        \caption{\footnotesize The energy-eater board. This component tends to overheat a fair bit and should generally be checked in case of failure for signs of degradation.}
        \label{fig:ur:7}
    \end{subfigure}
    ~
    \begin{subfigure}[b]{0.3\textwidth}
        \includegraphics[width=\textwidth]{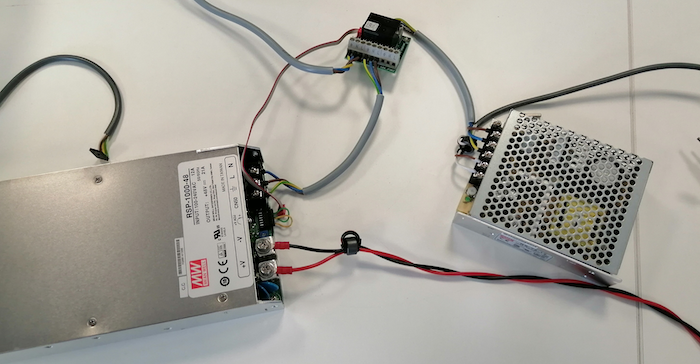}
        \caption{\footnotesize A safety relay and two \acp{PSU} identified, one for the compute logic (12V) and another one to power the actuators (48V).}
        \label{fig:ur:8}
    \end{subfigure}
    ~
    \begin{subfigure}[b]{0.3\textwidth}
        \includegraphics[width=\textwidth]{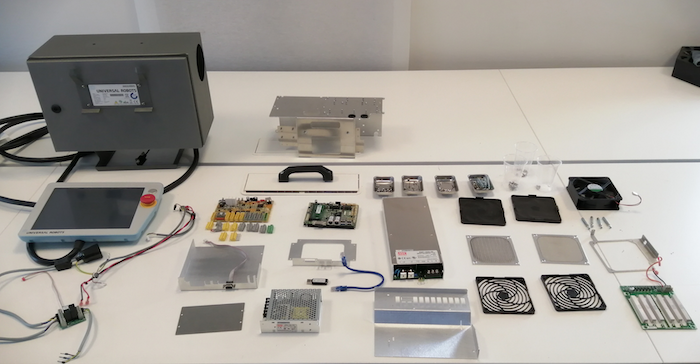}
        \caption{\footnotesize Final figure depicting all the components contained inside of the Universal Robots UR3 CB3.1 controller, leaving aside the teach pendant.}
        \label{fig:ur:9}
    \end{subfigure}
    \vspace{1em}
    \showkeystrue  
    \caption{UR3 collaborative robot teardown. Adapted from \cite{mayoral2021hacking}.}
    \label{figure:ur:teardown}
\end{figure*}

Figure~\ref{figure:ur:teardown} shows a selection of images obtained from the complete teardown of the UR3 CB3.1 industrial collaborative robot. Our goal is to show how a systematic teardown can lead to understanding how to obtain repairing capabilities of the complete robot, including the controller (i.e., the ``brain'' of the robot), teach pendant, and robot arm mechanics. We put particular emphasis in the CB3.1 controller since most safety-related electronics reside in there. The total time spent in the teardown is of five hours and 30 minutes.

An interesting observation is depicted in Figure \ref{fig:ur:6}, which displays that the compute substrate in charge of implementing the safety logic is the NXP LPC4437JET256 microcontroller. While doing hardware reconnaissance we found the following excerpt within the part datasheet \cite{LPC435xdatasheet} of the corresponding microntroller:

\begin{quote}
    \textbf{Suitability for use} — NXP Semiconductors products are \emph{\textbf{not designed}, authorized or warranted to be suitable for use in life support, life-critical or \textbf{safety-critical systems} or equipment, nor in applications where failure or malfunction of an NXP Semiconductors product can reasonably be expected to result in \textbf{personal injury}, death or severe property or environmental damage}. NXP Semiconductors and its suppliers accept no liability for inclusion and/or use of NXP Semiconductors products in such equipment or applications and therefore such inclusion and/or use is at the customer’s own risk.
\end{quote}

\begin{observation}
\label{observation:ur3_nosafetycompute}
The microcontroller implementing the safety logic in the UR3 CB-Series robot controller is in fact not suitable for safety-critical systems according to the silicon vendor. Confusingly, the list of applications on the first pages of the datasheet includes industrial automation or motor control, which are typical safety-critical use cases.
\end{observation}

In other words, this observation leads us to question the quality and reliability of the safety implementation within robots of the CB-Series by Universal Robots. \cite{hatta2020right} indicated that vendors have historically opposed to teardowns under the argument that \emph{closed networks of dealers guarantee quality}. However, our first observation indicates the exact opposite. Third parties with the required technical expertise might be able to identify and pinpoint hardware components that don't meet the quality standards for the safety situations the robot may have to face, leading to an overall improved scenario for end-users.


%
%

\subsection{Case Study 2: Teardown of a next-gen industrial collaborative robot}
Following the CB-Series, we proceeded and disassembled one of the latest releases from Universal Robots, the UR3e, an e-Series. Figure \ref{figure:ure:teardown} depicts the complete process through selected images. The complete teardown of the robot took us a total time of three hours\footnote{Note that a) the robot is similar to the one studied in case study 1 and b) repeated teardowns of manipulators helped us optimize the time.}.
%
%
\showkeysfalse
\begin{figure*}[p]
    \centering
    \begin{subfigure}[b]{0.3\textwidth}
        \includegraphics[width=\textwidth]{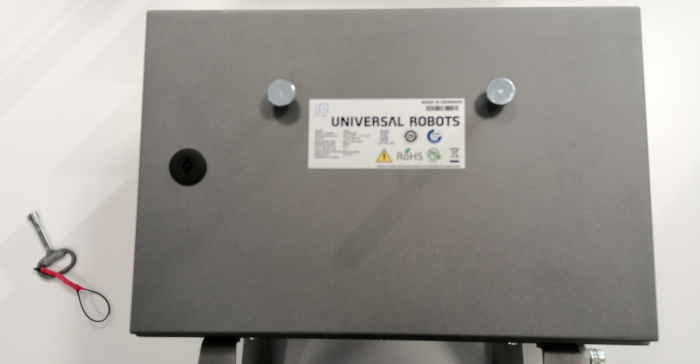}
        \caption{\footnotesize Universal Robots URe3 controller. The controller has a mechanical lock aimed to prevent ingress to the internals from non authorized parties.}
        \label{fig:ure:1}
    \end{subfigure}
    ~ 
    \begin{subfigure}[b]{0.3\textwidth}
        \includegraphics[width=\textwidth]{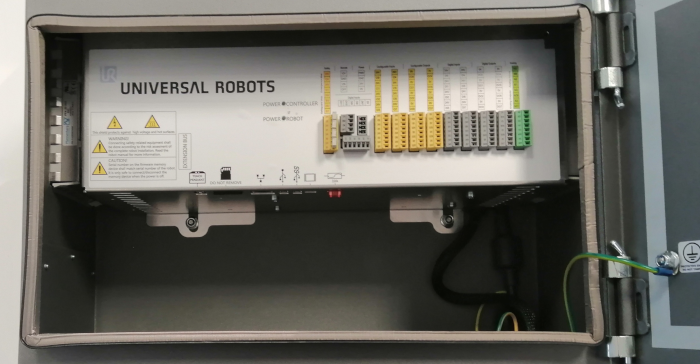}
        \caption{\footnotesize Inside the controller we can see various connectors and cables exposed. The right side includes I/O and safety, at the bottom USB, HMI and SD card.}
        \label{fig:ure:2}
    \end{subfigure}
    ~
    \begin{subfigure}[b]{0.3\textwidth}
        \includegraphics[width=\textwidth]{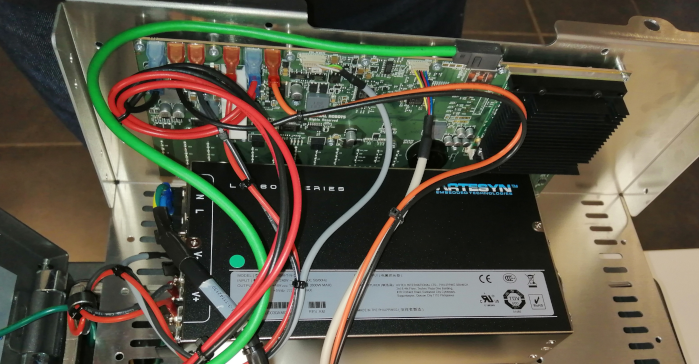}
        \caption{\footnotesize The main computer of the controller and the \ac{PSU} are affixed to the front plate. Given the real state available we miss some cable management. }
        \label{fig:ure:3}
    \end{subfigure}
    \\[3ex]
    \begin{subfigure}[b]{0.3\textwidth}
        \includegraphics[width=\textwidth]{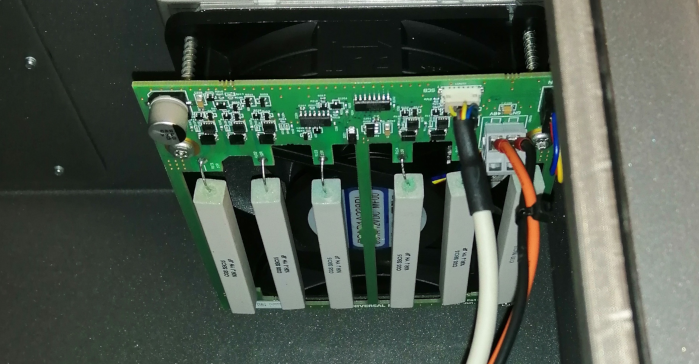}
        \caption{\footnotesize The energy-eater board. This component tends to overheat a fair bit and should generally be checked in case of failure for signs of degradation.}
        \label{fig:ure:4}
    \end{subfigure}
    ~
    \begin{subfigure}[b]{0.3\textwidth}
        \includegraphics[width=\textwidth]{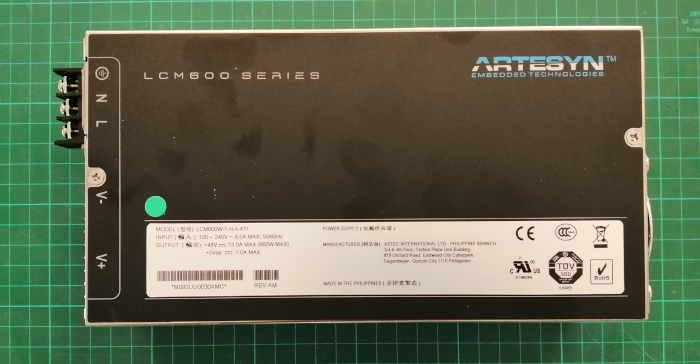}
        \caption{\footnotesize The \ac{PSU} is a Artesyn LCM600 series with an output of 48 V, and an input of 85–264 Vac. Has a typical full load efficiency of 89\% up to 600 watts. }
        \label{fig:ure:5}
    \end{subfigure}
    ~
    \begin{subfigure}[b]{0.3\textwidth}
        \includegraphics[width=\textwidth]{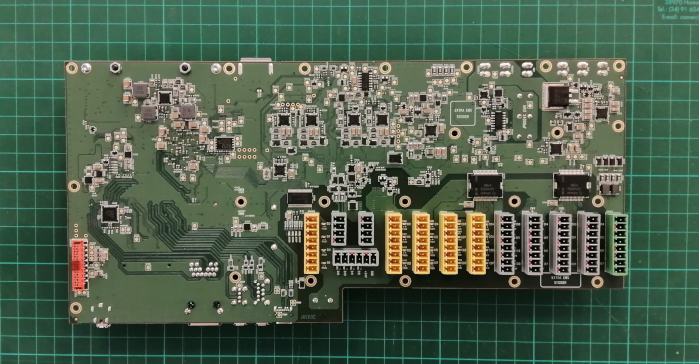}
        \caption{\footnotesize Unlike previous hardware iterations, the e-Series controller presents both the safety logic and the control logic merged into a single \ac{PCB}.}
        \label{fig:ure:6}
    \end{subfigure}
    \\[3ex]
    \begin{subfigure}[b]{0.3\textwidth}
        \includegraphics[width=\textwidth]{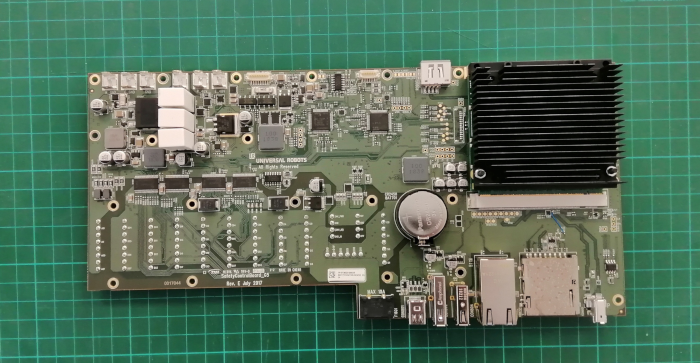}
        \caption{\footnotesize DC to DC power management takes place on the board unlike previous iteretaions. The positions of the relays may hinder transistor cooling }
        \label{fig:ure:7}
    \end{subfigure}
    ~
    \begin{subfigure}[b]{0.3\textwidth}
        \includegraphics[width=\textwidth]{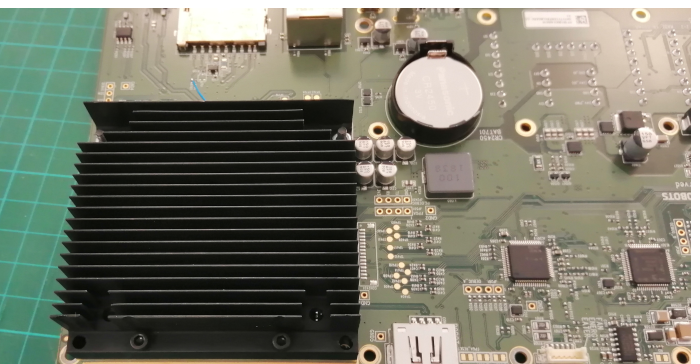}
        \caption{\footnotesize Under the heatsink we find a MSC Q7-BT module on an ECX form factor and an Intel SoC with DDR3L memory.}
        \label{fig:ure:8}
    \end{subfigure}
    ~
    \begin{subfigure}[b]{0.3\textwidth}
        \includegraphics[width=\textwidth]{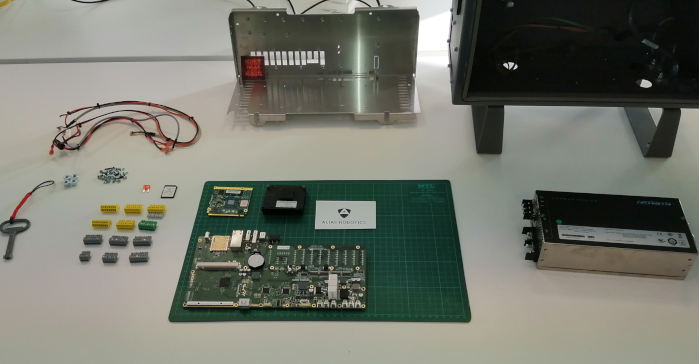}
        \caption{\footnotesize Final figure depicting all the components contained inside of the Universal Robots URe3 controller, leaving aside the teach pendant.}
        \label{fig:ure:9}
    \end{subfigure}
    \\[3ex]
    \begin{subfigure}[b]{0.3\textwidth}
        \includegraphics[width=\textwidth]{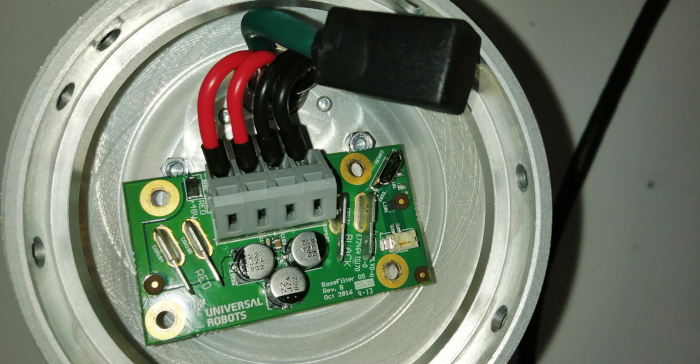}
        \caption{\footnotesize Connections coming from the controller lead to the 48V and micro-USB lines.}
        \label{fig:ure:10}
    \end{subfigure}
    ~
    \begin{subfigure}[b]{0.3\textwidth}
        \includegraphics[width=\textwidth]{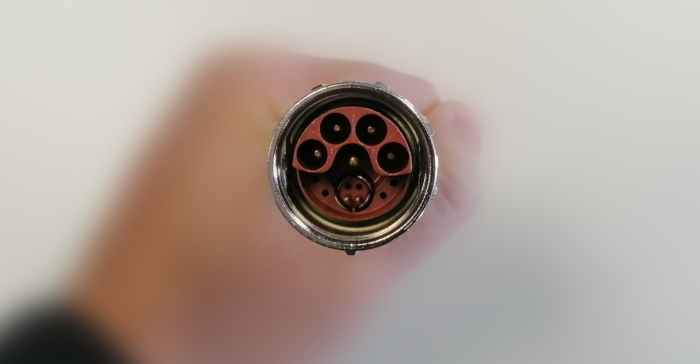}
        \caption{\footnotesize Connector has a different shape when compared to CB3 series}
        \label{fig:ure:11}
    \end{subfigure}
    ~
    \begin{subfigure}[b]{0.3\textwidth}
        \includegraphics[width=\textwidth]{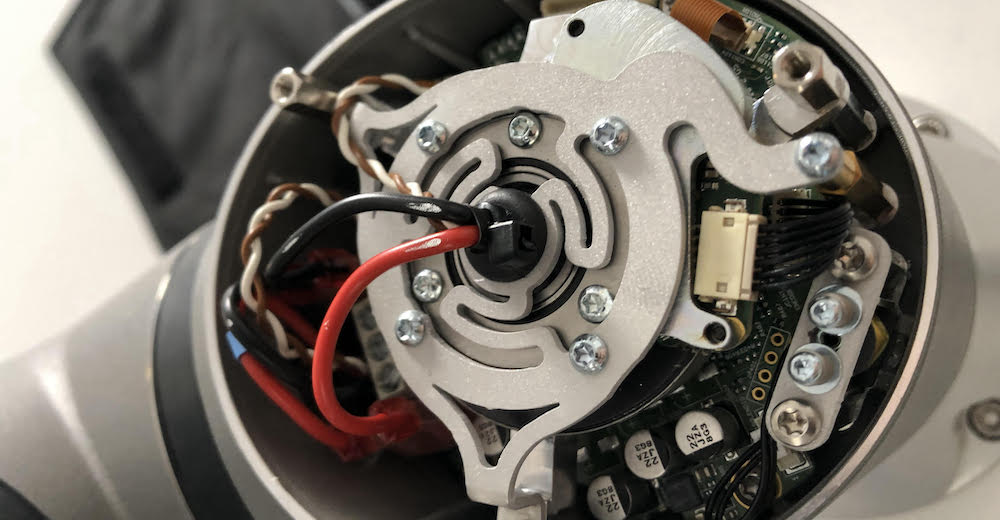}
        \caption{\footnotesize Harmonic drive actuators.}
        \label{fig:ure:12}
    \end{subfigure}
    \vspace{1em}
    \showkeystrue  
    \caption{UR3e collaborative robot teardown. Adapted from \cite{mayoral2021hacking}.}
    \label{figure:ure:teardown}
\end{figure*}
We observe how, while the overall outer look remains similar, the internals have suffered a significant change:



\begin{itemize}
    \item The e-Series controller integrates a single \ac{PSU}, as presented in Figure \ref{fig:ure:5}, while the CB-series had two (see Figure \ref{fig:ur:8}).
    \item While the CB-Series presented two boards containing compute, power, and safety logic (Figures \ref{fig:ur:3} and \ref{fig:ur:6}, respectively), the e-Series presents only one single \ac{PCB} named as ``SAFETYCONTROLBOARD'' and depicted in Figures \ref{fig:ure:6}, \ref{fig:ure:7}, \ref{fig:ure:8} and \ref{fig:ure:9}.
    \item Figure \ref{fig:ure:7} shows that the new \ac{PCB} includes a Xilinx Artix-7 series \ac{FPGA}, widely used for implementing safety logic in a variety of automotive and control domains, and a much more reliable compute substrate for safety-related tasks than a \ac{MCU}.
    \item Figure \ref{fig:ure:10} shows that the base filter \ac{PCB}---which helps interface power and RS485 communications from the controller (e-Series) to the robot arm mechanics---is similar to the one present in the CB-series. We also note that, while the arm mechanics connector changed in the e-Series (see Figure \ref{fig:ure:11}), power and communications lines remain coherent (through the base filter board).
    \item For the most part, the electronics contained in the arm mechanics (Figure \ref{fig:ure:12}) do not present relevant changes from an interoperability perspective. This facilitates re-purposing and reusing them (see Section~\ref{sec:reversing}).
\end{itemize}
Looking at the results of our teardown, we highlight the following observations.
\begin{observation}
\label{observation:ur3e_safety}
e-Series controllers from Universal Robots include a Xilinx Artix-7 series \ac{FPGA}, widely used for implementing safety logic in a variety of automotive and control domains, as noted by~\cite{hallett2015xilinx, gracic2016implementation, gracic2017evaluation}, a much more reasonable choice from a user's safety perspective.
\end{observation}

\begin{observation}
\label{observation:ur3e_comms_same}
While adopting different physical connectors, power and communication (RS485) lines remain coherent between CB-Series and e-Series. From the context of repairability, changing physical connectors is a clear planned obsolescense action.
\end{observation}



\showkeysfalse

\begin{figure*}[p]
    \centering
    \begin{subfigure}[b]{0.3\textwidth}
        \includegraphics[width=\textwidth]{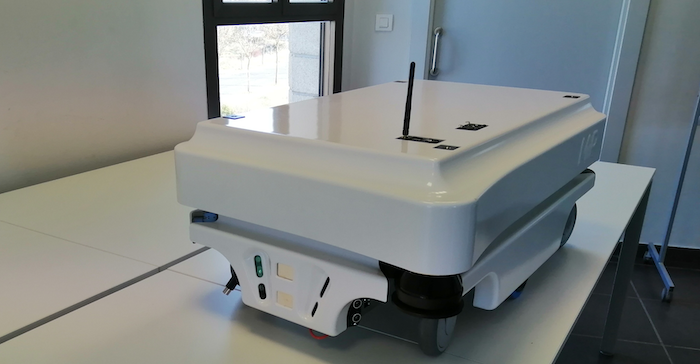}
        \caption{\footnotesize The top shell sits on top of a metal frame that protects all the electronic components. Simply lifting the top shell reveals the internal electronic components.}
        \label{fig:mir:1}
    \end{subfigure}
    ~ 
    \begin{subfigure}[b]{0.3\textwidth}
        \includegraphics[width=\textwidth]{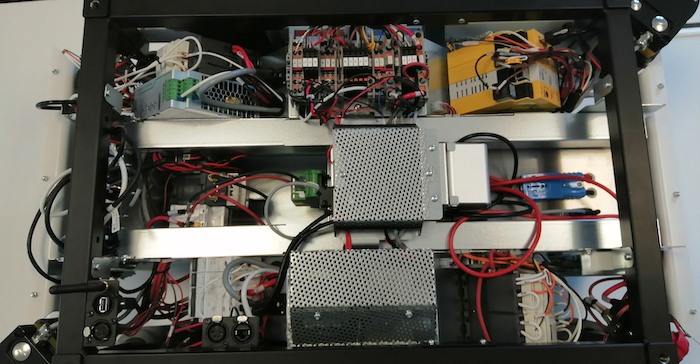}
        \caption{\footnotesize A circuit-breaker switch is present to disconnect the main power line going from the batteries to the rest of the robot. A quick-release connector is also present.}
        \label{fig:mir:2}
    \end{subfigure}
    ~
    \begin{subfigure}[b]{0.3\textwidth}
        \includegraphics[width=\textwidth]{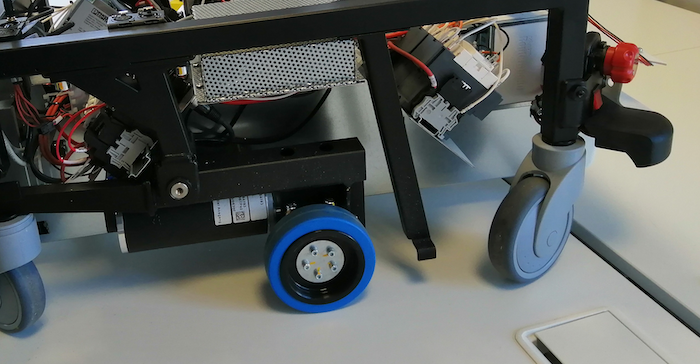}
        \caption{\footnotesize Plastic fenders are identified around the perimeter of the mobile robot to enclose and protect the internals. These are prone to crack under heavy mechanical stress.}
        \label{fig:mir:3}
    \end{subfigure}
    \\[3ex]
    \begin{subfigure}[b]{0.3\textwidth}
        \includegraphics[width=\textwidth]{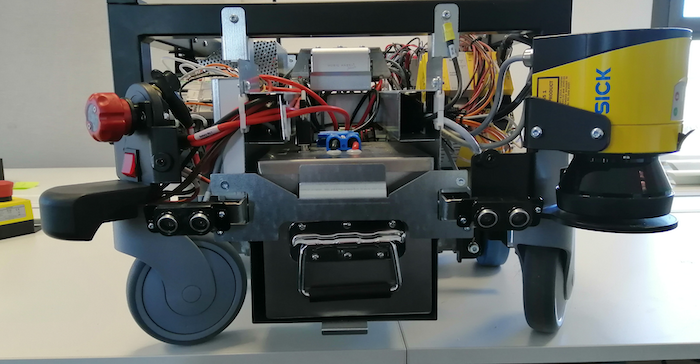}
        \caption{\footnotesize The battery is enclosed in a steel box and held by a retaining plate to prevent movement and connected to a DP9 connector and DC 24V wires.}
        \label{fig:mir:4}
    \end{subfigure}
    ~
    \begin{subfigure}[b]{0.3\textwidth}
        \includegraphics[width=\textwidth]{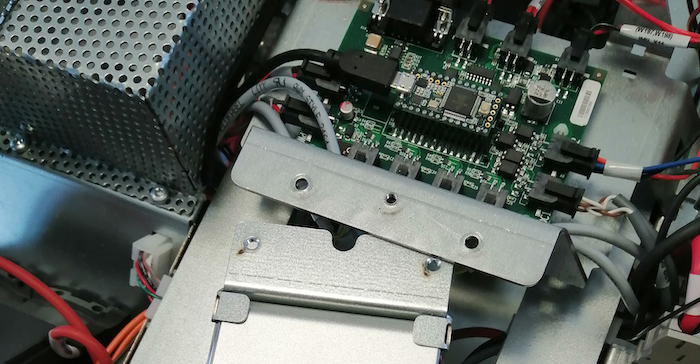}
        \caption{\footnotesize Under an RF cage we find a Teensy board for LED control and a third party speaker to play the sounds from the on-board controller.}
        \label{fig:mir:5}
    \end{subfigure}
    ~
    \begin{subfigure}[b]{0.3\textwidth}
        \includegraphics[width=\textwidth]{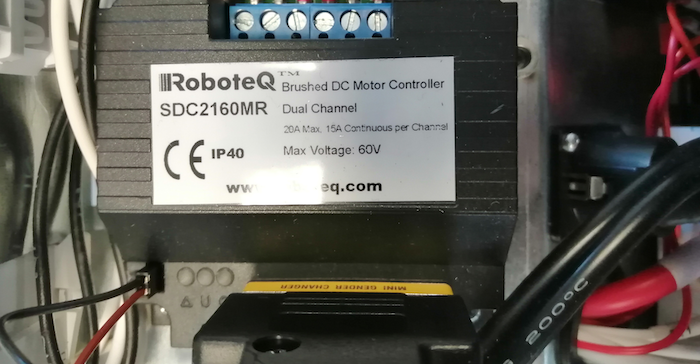}
        \caption{\footnotesize Under the sloped RF cage we find the Roboteq SDC2160, a dual-channel controller for brushed motors with no safe torque (safety) capabilities.}
        \label{fig:mir:6}
    \end{subfigure} 
    \\[3ex]
    \begin{subfigure}[b]{0.3\textwidth}
        \includegraphics[width=\textwidth]{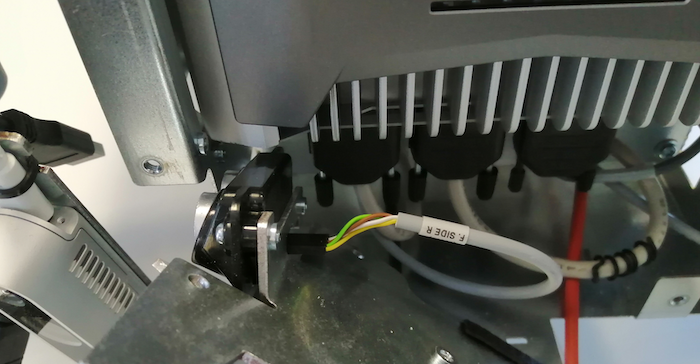}
        \caption{\footnotesize The on-board controller is the embedded and ``ruggedized'' EC70A-SU from DFI which features an Intel processor and a Ubuntu 16.04 file system.}
        \label{fig:mir:4}
    \end{subfigure}
    ~
    \begin{subfigure}[b]{0.3\textwidth}
        \includegraphics[width=\textwidth]{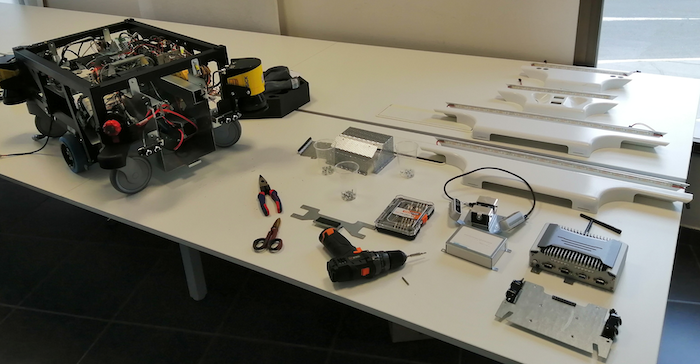}
        \caption{\footnotesize Final  figure  depicting  the main components  contained  inside  of  the  MiR100 robot, alongside the tools used for the teardown.}
        \label{fig:mir:5}
    \end{subfigure}
    ~
    \begin{subfigure}[b]{0.3\textwidth}
        \includegraphics[width=\textwidth]{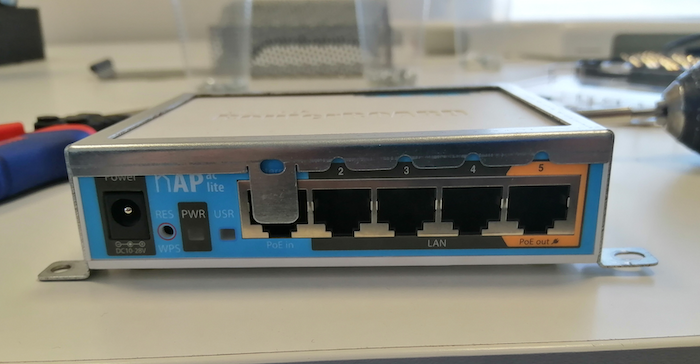}
        \caption{\footnotesize Both the safety \ac{PLC} and the on-board controller are connected to a Mikrotik hAP ac. A 2.4/5GHz dual-band omnidirectional access point.}
        \label{fig:mir:6}
    \end{subfigure}
    \\[3ex]
    \showkeystrue  
    \caption{MiR-100 mobile industrial robot teardown. Adapted from \cite{mayoral2021hacking}.}
    \label{figure:mir:teardown}
\end{figure*}

\subsection{Case Study 3: Teardown of a mobile industrial robot}

\showkeysfalse
\label{mirdissasembly}
\showkeystrue  


Figure \ref{figure:mir:teardown} depicts the teardown process of a MiR-100, a popular mobile robot manufactured by the Danish Mobile Industrial Robots (MiR), also owned by the US Teradyne. The teardown took four hours and 20 minutes approximately. The first impression is that various components of the robot  could be improved from a safety perspective, as highlighted in Figure \ref{fig:mir:4} or \ref{fig:mir:6}). Moreover, the teardown helped understand how this robot presents multiple (internal and external) networks and how each one of the sensors and actuators are connected across these networks, forming the data layer graph. One interesting finding resulting from the teardown is obtaining a better understanding of the robot's computational graph (the behavior itself). The robot itself is powered by \ac{ROS} (\cite{quigley2009ros}) and gaining further understanding of the \ac{ROS} computational graph requires understanding also its underlying hardware mapping (from which one derives the data layer graph). The teardown exercise yields exactly this, and allows to produce a data layer graph represented in the form of a hardware schematic which can then be used in combination with the computational graph to gain further understanding of the robot.

\begin{observation}
\label{observation:hardware_arch}
Teardown processes help determine the underlying networking architecture in a robot, from which the robot data layer graph can be inferred. Mapping the data layer graph to the computational graph (the robot behavior) is fundamental to gain better understanding of the robot and propose an appropriate security architecture.
\end{observation}


\section{Teardown-enabled security research}
\label{sec:reversing}

The previous section highlighted how teardown helped identify quality and safety issues in robots, as well as obtaining a better understanding of their architectures by matching each robot's data layer graph with their corresponding computational graph. Beyond this, we argue that robot teardown is also key for security research in robotics. Recall that safety and cybersecurity are very related and influence one another, as noted by~\cite{mayoral2020can, kirschgens2018robot, mayoral2019industrial}.

Teardown, as a process, is an essential part of a hardware reverse engineering task, and brings useful lessons and insights for the design of current and future robot systems.
Generally, teardown supports Kerckhoffs' principle in revealing all the details and weaknesses of a security system, excluding volatile secrets such as keys or credentials that are stored in memory and most likely disappear naturally once the power supply is taken away (with the exception of keys stored in permanent memory, which is generally discouraged, and would be discovered along the teardown process).
Overall, the history of proprietary systems violating Kerckhoffs' principle by pursuing ``security through obscurity'' is rich of failure cases (with the military domain as the sole exception), as a vast amount of related work demonstrates.

Reverse engineering has always been invaluable to discover vulnerabilities and develop remedies in many domains: network security (\cite{guha_network_1997}), access control (\cite{wang_towards_2008}), embedded systems (\cite{mcloughlin_secure_2008,rajendran_security_2013}), software engineering (\cite{lin_automatic_2010,treude_exploratory_2011}), or the internet of things (\cite{tellez_iot_2016}).
By promoting systematic teardown we want to extend this successful concept to the analysis of abandoned robots.

Particularly, and as part of this research, our group identified more than 100 security flaws across the three robots described above over a period of two years. Most of the flaws were cataloged as vulnerabilities and 17 obtained new Common Vulnerabilities and Exposures (CVE) IDs all of which are publicly disclosed at the \ac{RVD}, proposed by~\cite{mayoralrvd}. Table \ref{table:vulnerabilities} introduces some of the selected security vulnerabilities found. 
The information obtained through teardown helps pinpoint flaws across the multiple (internal and external) robot networks. In most cases, these robots present few or no security measures, allowing adversaries to easily exploit the flaws of internal components (e.g., \cite{rvd2558,rvd2561,rvd2562}, so compromising the robot behavior or taking full control of it.

We advocate for robot teardowns as a means to improve security in robotics and encourage manufacturers, integrators and end-users to carefully consider the underlying hardware architecture to protect their robotic systems. Similarly, we encourage teardowns as a tool to mitigate outstanding security flaws. Proper knowledge of the hardware helps determine which additional elements can help mitigate security issues when the manufacturer does not react. As an example, our group introduced an additional commercial off-the-shelf hardware firewall within MiR's internal network, between the main controller and the safety PLC, by SICK AG, mitigating \cite{rvd2558} without having to modify any parts of the firmware. This modification could enable users and system integrators frustrated with MiR's security policies to secure their robots directly.

\begin{observation}
\label{observation:security}
Teardown helps pinpoint security flaws across the multiple internal and external robot networks.
\end{observation}


\section{Lessons learned}
\label{sec:obsolescence}

Through this work we learned about the underlying hardware architectures that popular collaborative robots deploy today, researching security in the process. As a result of the teardown case studies described in Section~\ref{sec:teardown}, our group also identified several of the planned obsolescence indicators previously introduced in Section~\ref{sec:introduction}.  Planned obsolescence was particularly evident in the robots from Universal Robots. To further illustrate this, Figures~\ref{fig:ur_schematics} and \ref{fig:ure_schematics} depict the simplified electrical diagrams of the UR3 and UR3e robots. From an electrical point of view, these two robots present a similar layout for interfacing with the robot arm.

While we appreciate certain changes in the electronics, given the teardown results, we find no real reason why backwards or forward compatibility between controllers and robotic arms should not be possible. This would mean that existing customers with UR3 robots could repair and replace parts in either the controller or the robotic arm, without being forced to pay the premium price of buying a complete new set including both.

Unsurprisingly, we observe that the manufacturer introduced subtle changes meant to make this particular intent harder. One of such actions is depicted in Figure~\ref{fig:ure:11}, which shows the replacement of the controller-to-arm connector, which we can only justify with attempts to exercise obsolescence practices. Another of such actions includes the obscurity around the changes introduced in the UR3e robot arm itself. These changes can be summarized with the addition of an extra 6-axis force-torque sensor at the end of the robot. The \emph{exact} same result can be achieved in UR3 robot arms by adding commercial off-the-shelf robot components, gaining such capabilities.


\section{Conclusions}
\label{sec:conclusions}

In this chapter we presented robot teardowns as an approach to study robot hardware architectures, obtain repairing capabilities and research its security. We discuss the empirical results of three robot teardowns and the findings affecting quality and safety throughout the process. We then discuss how teardown is a relevant tool for security research in robotics which helps pinpoint security flaws early across the multiple internal and external networks in a robot. Moreover, we introduce our security findings and propose mitigations powered by the hardware know-how and repairing capabilities acquired. Ultimately, we research planned obsolescence practices in the robots from Teradyne and propose actions that could be taken to bypass obsolescence.

Our results show evidence that robot teardowns can help the robotics industry and supply chain by improving significantly quality, safety and security. Our findings extrapolate to most of the robots manufactured by Teradyne and its subsidiaries. We show concern for the currently growing trend in robotics to create private networks of certified groups, a common practice shown by   manufacturers like MiR or UR, both owned by Teradyne. This difficulties system integration, repairability and ultimately security. We advocate for a `\emph{Right to Repair}' in robotics and encourage end-users to reflect their needs into their supply chains and into the original upstream robot manufacturers.

\setlength{\tabcolsep}{10pt}
\renewcommand{\arraystretch}{1.6}

\begin{landscape}
\vspace*{\fill} 

    \begin{table}[h!]
        \centering
    	\resizebox{1.3\textwidth}{!}{ 
    		\begin{tabular}{|l|l|l|l|}
    			\hline
                \rowcolor{black}
                \color{white}\bf \textsc{CVE ID} & \color{white}\bf \textsc{RVD ID} & \color{white}\bf \textsc{Description} & \color{white}\bf \textsc{Report}  \\ 
                \hline
                CVE-2019-19626 & RVD\#1408  & Bash scripts (magic UR files) get launched automatically with root privileges and without validation or sanitizing &  \url{https://github.com/aliasrobotics/RVD/issues/1408} \\
                \hline
                CVE-2020-10290 & RVD\#1495  & Universal Robots URCaps execute with unbounded privileges &  \url{https://github.com/aliasrobotics/RVD/issues/1495} \\
                \hline
                CVE-2020-10267 & RVD\#1489  & Unprotected intelectual property in Universal Robots controller CB 3.1 across firmware versions &  \url{https://github.com/aliasrobotics/RVD/issues/1489} \\
                \hline
                CVE-2020-10266 & RVD\#1487  & No integrity checks on UR+ platform artifacts when installed in the robot &  \url{https://github.com/aliasrobotics/RVD/issues/1487} \\
                \hline
                CVE-2020-10265 & RVD\#1443 & UR dashboard server enables unauthenticated remote control of core robot functions & \url{https://github.com/aliasrobotics/RVD/issues/1443}\\
                \hline
                CVE-2020-10264 & RVD\#1444 & RTDE Interface allows unauthenticated reading of robot data and unauthenticated writing of registers and outputs & \url{https://github.com/aliasrobotics/RVD/issues/1444}\\
                \hline
                
                CVE-2020-10278 & RVD\#2561  & Unprotected BIOS allows user to boot from live OS image &  \url{https://github.com/aliasrobotics/RVD/issues/2561} \\
                \hline
                CVE-2020-10270 & RVD\#2557  & Hardcoded Credentials on MiRX00 Control Dashboard &  \url{https://github.com/aliasrobotics/RVD/issues/2557} \\
                \hline
                CVE-2020-10279 & RVD\#2569  & Insecure operating system defaults in MiR robots &  \url{https://github.com/aliasrobotics/RVD/issues/2569} \\
                \hline
                CVE-2020-10276 & RVD\#2558  & Default credentials on SICK PLC allows disabling safety features &  \url{https://github.com/aliasrobotics/RVD/issues/2558} \\
                \hline
                CVE-2020-10273 & RVD\#2560  & Unprotected intellectual property in Mobile Industrial Robots (MiR) controllers &  \url{https://github.com/aliasrobotics/RVD/issues/2560} \\
                \hline
                CVE-2020-10277 & RVD\#2562  & Booting from a live image leads to exfiltration of sensible information and privilege escalation &  \url{https://github.com/aliasrobotics/RVD/issues/2562} \\
                \hline
                CVE-2020-10269 & RVD\#2566  & Hardcoded Credentials on MiRX00 wireless Access Point &  \url{https://github.com/aliasrobotics/RVD/issues/2566} \\
                \hline
                CVE-2020-10275 & RVD\#2565  & Weak token generation for the REST API &  \url{https://github.com/aliasrobotics/RVD/issues/2565} \\
                \hline
                CVE-2020-10274 & RVD\#2556  & MiR REST API allows for data exfiltration by unauthorized attackers (e.g. indoor maps) &  \url{https://github.com/aliasrobotics/RVD/issues/2556} \\
                \hline
                CVE-2020-10271 & RVD\#2555  & MiR ROS computational graph is exposed to all network interfaces, including poorly secured wireless networks and open wired ones & \url{https://github.com/aliasrobotics/RVD/issues/2555} \\
                \hline
                CVE-2020-10272 & RVD\#2554  & MiR ROS computational graph presents no authentication mechanisms & \url{https://github.com/aliasrobotics/RVD/issues/2554} \\
                \hline
                
    		\end{tabular}
    		\vspace{1em}
    	}
    	\caption{The 17 novel (new CVE IDs) vulnerabilities encountered during a period of two years in the robots of Teradyne and as a result of an initial hardware teardown. 
        }
    	\label{table:vulnerabilities}
    \end{table}
\vspace*{\fill} 
\end{landscape}


\showkeysfalse
\begin{figure*}[!h]
    \centering

    \begin{subfigure}[b]{\textwidth}
        \includegraphics[width=\textwidth]{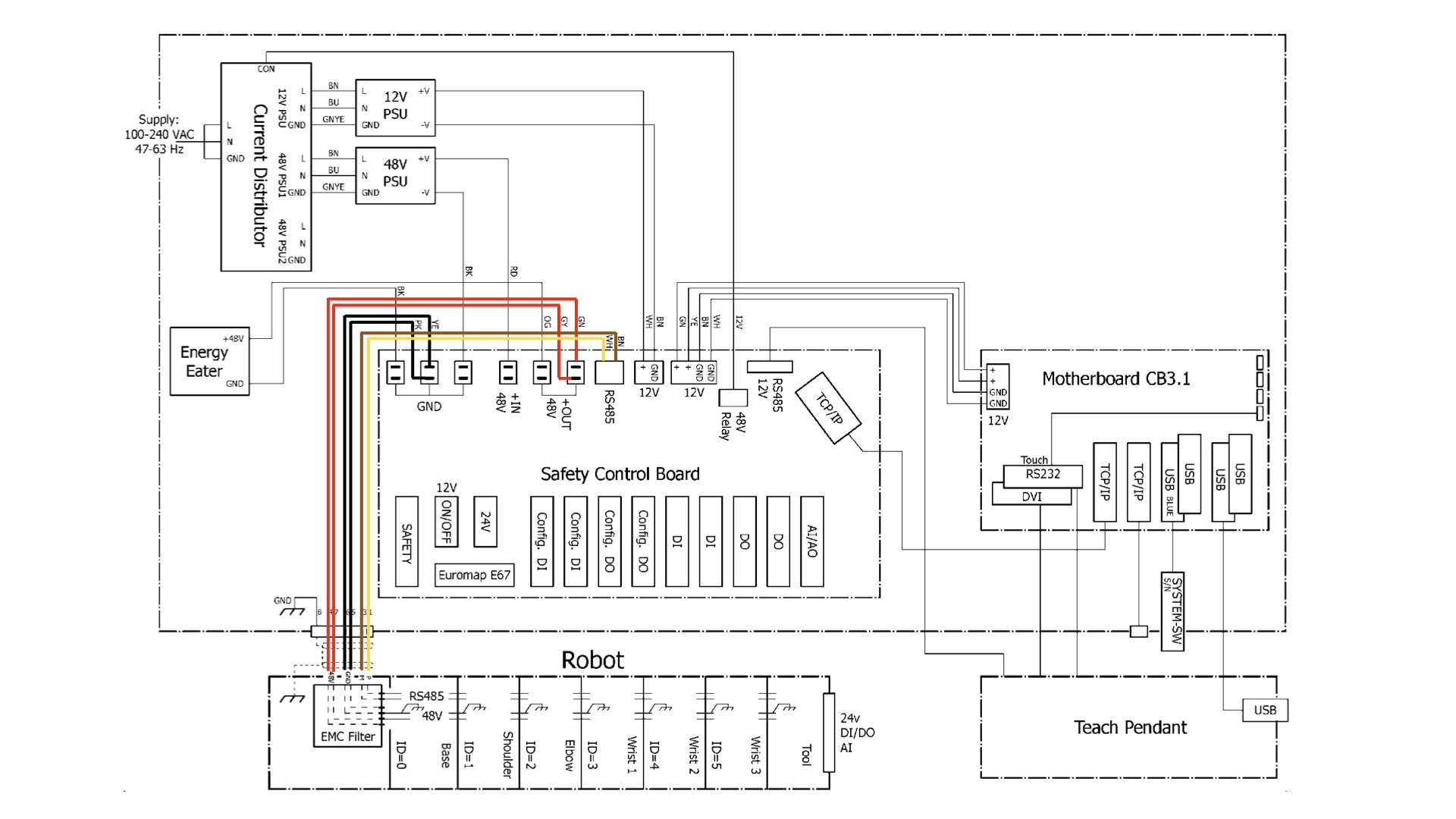}
        \caption{\footnotesize Simplified electrical diagram of Universal Robots UR3 CB-Series collaborative robot.}
        \label{fig:ur_schematics}
    \end{subfigure}
    ~
    \begin{subfigure}[b]{\textwidth}
        \includegraphics[width=\textwidth]{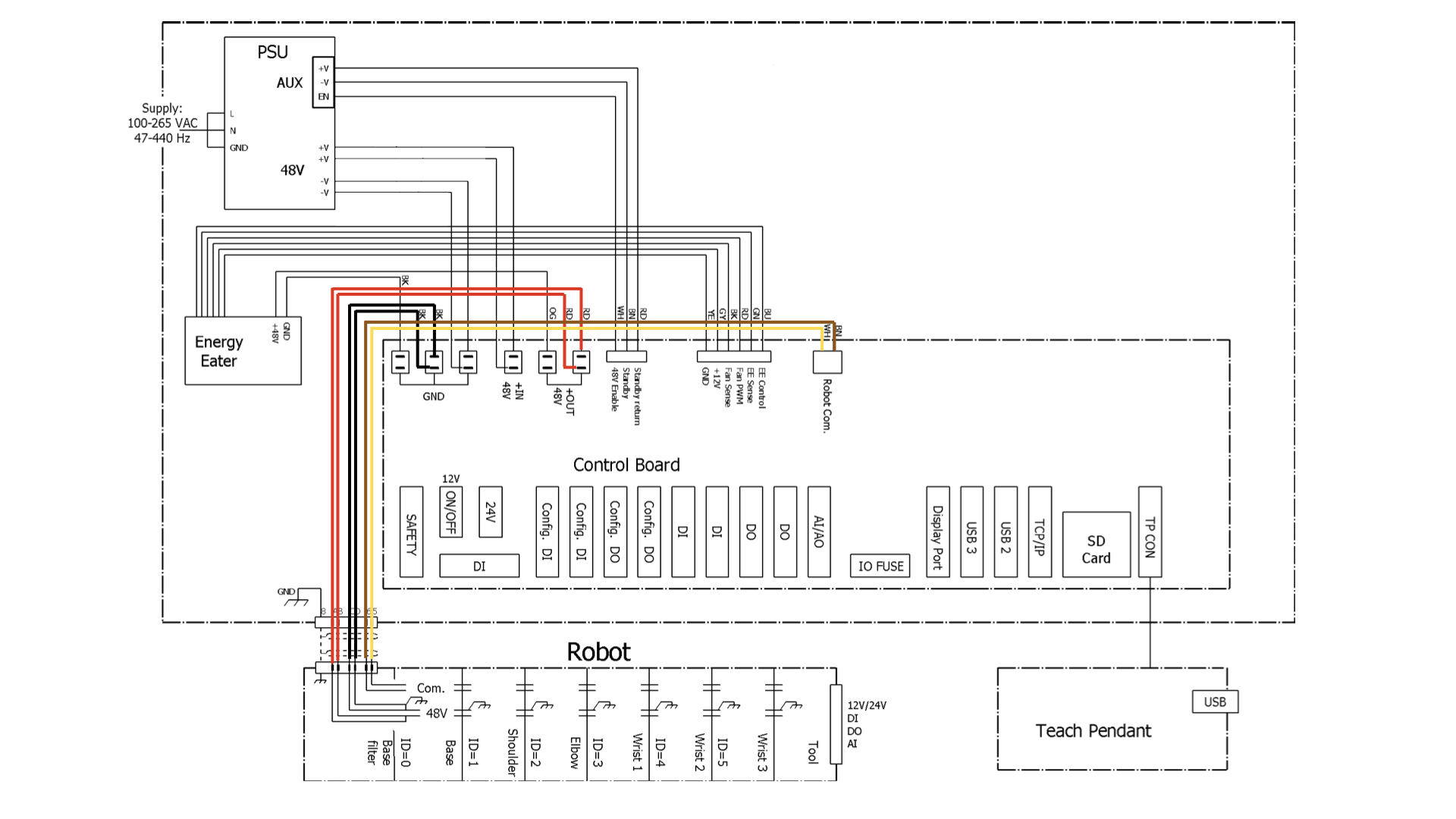}
        \caption{\footnotesize Simplified electrical diagram of Universal Robots UR3e e-Series collaborative robot.}
        \label{fig:ure_schematics}
    \end{subfigure}
    \\[3ex]
    
    \showkeystrue  
    \caption{
        Simplified electrical diagrams of Universal Robots UR3 CB-Series (\ref{fig:ur_schematics}) and UR3e e-Series (\ref{fig:ure_schematics}) collaborative robots. Reproduced from \cite{mayoral2021hacking}.
    }
    \label{figure:electrical_diagrams}
\end{figure*}

    \chapter{Attacking Robots in Industry}
\showkeystrue  
\label{chapter:attacking_industry}

\epigraph{\textbf{Contributions}: \emph{Leading scientific author and contributor at \cite{mayoral2019industrial}. Self-funded work and the contributor team. Led research, implemented malware and wrote article. Other co-authors contributed with experimental tests. All findings have been disclosed.}}{}




\section{Introduction}

In the context of computer security, ransomware is malicious software (malware) that either locks a computer, prevents  from accessing the data using encryption, or both, until the subject has paid a ransom. First ransomware Proof of Concept (PoC) appeared in 1989 \cite{richardson2017ransomware, formby2017out} when Joseph Popp, an evolutionary biologist and AIDS researcher, carried out an experimental attack by distributing 20,000 floppy disks in a Conference by the World Health Organization focused in AIDS research. The conference spanned to researchers from more than 90 countries, and malware was distributed claiming that the disks contained a program that analyzed risk of acquiring AIDS through the use of a questionnaire. Thereafter, ransomware called "AIDS Trojan" got distributed. Since then, it has evolved leading into two big subgroups:
\begin{itemize}
    \item \textbf{Crypto ransomware}: which encrypts data and asks the user for a ransom in exchange for the decryption key.
    \item \textbf{Locker ransomware}: locks the system by some means, prevents its use and asks for a ransom to re-enable it.
\end{itemize}


\section{Background}

According to literature \cite{richardson2017ransomware}, from 1989 to 2007 ransomware instances mostly focused on crypto-ransomware. In 2007, locker-ransomware began to appear and went mainstream. These ransomware instances locked systems and intimidated by displaying certain content (mostly pornographic images) while demanding a ransom to remove such content and unlock the systems. In 2013, Richardson et al.  \cite{richardson2017ransomware} observed that attackers pivot back to crypto-ransomware. According to the authors, the most famous piece of ransomware was released in August 2013. Named as CryptoLocker, it was originally distributed via a botnet and later through e-mail. CryptoLocker used public/private cryptographic key pairs to encrypt the target's file. Decrypting the files required to pay a ransom of 2 bitcoins (100 \$ at the time) within the first three days. To the extent of our literature review and to  date, variations of CryptoLocker remain being the most wide spread instances of ransomware \cite{ mcdonald2012ransomware, bhardwaj2016ransomware}.\\
\newline
In robotics, no targeted malware has yet been observed out of the PoC phase. Cesar Cerrudo and Lucas Apa earlier \cite{hackingbeforeskynet, hackingbeforeskynet2} published a ransomware attack over Nao, a social robot by Softbank Robotics, which got the media attention. According to research being conducted on the security concerns of the robotic market\footnote{\url{https://news.aliasrobotics.com/robot-security-survey-displays-first-results/} for more details on the ongoing survey.}, only 9 \% of robotics users have witnessed a cyber attack. This preliminary figure indicates that there is still very little activity yet known to the general public however, according to the same source, 51 \% of the users inquired confirmed having identified security flaws in robotic systems which leads to consider that there exists a relevant landscape of insecurities. Confirming this hypothesis, users inquired assigned a rating of 8 out 10 to the security relevance in robotics yet only 26 \% of the inquired acknowledged to have invested in robot cybersecurity, which includes evaluating security and protecting existing robot setups. Therefore, it seems that robot users do not fully apprehend the insecurity by design governing robots in the market.\\
\newline
In this chapter we aim to illustrate the existing insecurity status in some robotics vendors \cite{vilches2018introducing}. We have selected one of the most popular industrial collaborative robots and present \textbf{Akerbeltz}, an instance of ransomware targeting industrial robotic systems. We present the PoC ransomware attack, describe the rationale behind our target and discuss the general flow of the attack including the initial cyber-intrusion, lateral movement and later control phase. We then briefly discuss the resulting consequences from the installation of Akerbeltz and wrap up by sharing some major conclusions.


\section{Akerbeltz}

In basque mythology, \emph{Akerbeltz} is an antique deity impersonated in a male goat which is the animal-kind protector. Ethimologically coming from the Basque works \emph{Aker} (male goat) and \emph{Beltz} (black), its origin is attributed to a meadow in the surroundings of Zugarramurdi caves (Navarra), a pligrimage place for Basque Mythology.  Akerbeltz is the demon that is chairing "Akelarres" or Basque witch (sorginak) meetings. Some authors note that the mythological figure Akerbeltz represents was adored in many European countries, some of which remain up to present day. Akerbeltz often shows two different faces. On the one hand, it is the protector of animals and is even able to heal their illnesses if needed. Indeed, belief on Akerbeltz is thought to be the origin of hosting a black goat as the protector of all animals within a "Baserri" (Basque cottage or farm). On the other hand, when Akerbeltz participated Akelarres, it showed his darkest face, where witches obeyed and adored him as the genius representing revolution against established status-quo, amidst banquets of human flesh.\\

We advocate for a change in (most) robot manufacturer's attitude towards security and to do so,  we take Akerbeltz as inspiration and present below the first instance of industrial robotic ransomware. Due to our concerns about malicious applications of the software, authors will not be releasing the source code nor the low level method of operation of this industrial robot ransomware. 



%

\subsection{Target selection and rationale}

Our target was selected based on the results obtained from prior research \cite{kirschgens2018robot, mayoralrvd} on the vulnerability landscape across several industrial robot manufacturers. Our critical attitude was previously introduced by Alzola-Kirschgens et al. \cite{kirschgens2018robot} and essentially builds on the fact that several robot manufacturers, particularly collaborative robot (cobot) vendors, are profiting the popularity of these devices, via third parties (distributors and integrators) totally disregarding the ethical consequences of not caring about security, to the point that security vulnerabilities are left unadressed, or "up to the end user". To further understand our viewpoint, we ask the reader and potential robot users to consider the following questions: Thousands of insecure robots are being deployed all around the world, some of them thought to be networked and to increasingly collaborate with humans, who will be held responsible when these security holes get exploited and cause human damages? How would robot vendors respond when these reported, non-patched and public vulnerabilities cause safety hazards? Would they continue claiming that "their robots are designed to be open" and thereby never meant to be used beyond research? Should authorities regulate the use of these devices and their corresponding incomplete standards?\\
\newline
The authors discussed on all the above and coherently, decided to select what arguably is the most popular collaborative industrial robot manufacturer: Universal Robots. We select one of their best-selling robots, the UR3 and implement our PoC on it. Figure \ref{figure:industry_robot_target} shows a picture of our target together with the control box and teach pendant.\\

\begin{figure}[h!]
    \includegraphics[width=1\textwidth]{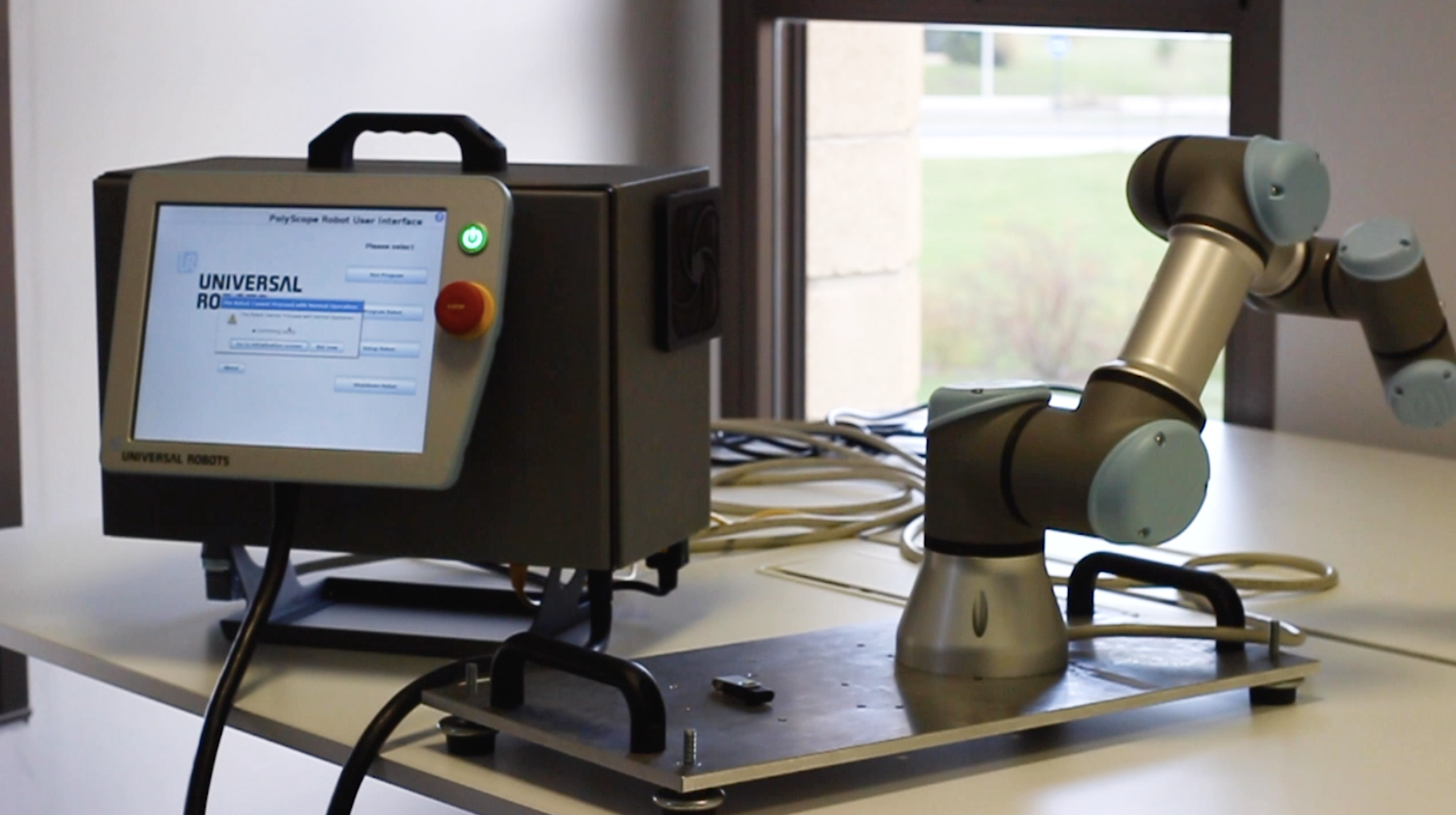}
    \centering
    \caption{Universal Robots UR3. Reproduced from \cite{mayoral2019industrial}.}
    \label{figure:industry_robot_target}
\end{figure}

Before our work, Universal Robots had other groups assessing their insecurity. In 2017, Cerrudo and Apa reported \cite{hackingbeforeskynet, hackingbeforeskynet2} five 0-day vulnerabilities. Several months later, representatives from the vendor acknowledged that security patches had been applied \cite{robo09_2018} yet there seems to be no public information available. Instead, it appears, the vendor disregarded the previously reported issues under the claim that attackers required very specific conditions. Further to that, Jacob Bom Madsen, Software Product Manager at Universal Robots publicly claimed that Universal Robots is \emph{"proud to have a fairly open architecture, that allows system integrators and UR+ Partners to easily develop and integrate the solutions they need."} \cite{robotiqcommunityurflaws}. This attitude conflicts directly with the very principle of Universal Robots safety claims, previously highlighted by Cerrudo and Apa \cite{hackingbeforeskynet2}: \emph{"Do not change anything in the safety configuration of the software (e.g. the force limit). If any safety parameter is changed the complete robot system shall be considered new, meaning that the overall safety approval process, including risk assessment, shall be updated accordingly"}. In other words, any modifications of the safety setup in the UR3 will lead to the complete invalidation of the robots' compliance with ISO 10218-1 \cite{standard10218} incurring in potential relevant losses and conflicts for the end user. Beyond the human and economical damages caused by modifying the safety setup of the UR3, the claim by Madsen leads to a troubling an arguably unethical statement: "The lack of security facilitates system integration". We have seen this statement repeatedly and yet, once again, this time coming from a leading cobot vendor, we see how openness and feature inclusion is used to justify the lack of security. Wielding the interoperability pitch, vendors push security up to their partners, collaborators or ultimately, to "the community" avoiding security actions, critical for the use of these products in human environments.\\
\newline
At the time of writing, our team knows yet of no  security patch mitigating these vulnerabilities. Moreover, we performed a penetration testing assessment in the UR3 CB series robot confirming the still unpatched existence of several of the previously reported flaws. Furthermore, we found \textbf{more than 300 new vulnerabilities} of relevant severity according to robot-specific scoring mechanisms \cite{2018arXiv180204082M}. The overall picture depicts a vendor which shows the little care not only for security, but also for quality of software.\\
\newline
The following subsection elaborates on how Akerbeltz acts on our UR3 CB series.

\subsection{Ransomware's flow}

\subsubsection{Cyber intrusion}

Initial infection gets realised by exploiting unpatched vulnerabilities in the robot. Akerbeltz's initial cyber intrusion is implemented exploiting these well known vulnerabilities and deployed via one of the following two attack vectors:
\begin{itemize}
    \item \textbf{physical USB ports in the teach pendant:} Exploiting \href{https://cve.mitre.org/cgi-bin/cvename.cgi?name=CVE-2019-19626}{CVE-2019-19626}\footnote{Remains confidential for responsible disclosure reasons at the time of writing.} an attacker can execute malicious code with root privileges by simply connecting an external USB stick. This attack vector is likely the easiest to implement given the exposure of the teach pendant in most UR3 robots. Moreover, the same attack vector applies not only in the teach pendant but also in the control box which exposes internal USB ports that can be accessed with standard mechanical fixations. 
    
    \item \textbf{remote access via adjacent networks: } \href{https://github.com/aliasrobotics/RVD/issues/672}{RVD\#672} (\href{https://cve.mitre.org/cgi-bin/cvename.cgi?name=CVE-2018-10633}{CVE-2018-10633}) permits an adjacent network attacker to easily ssh into the control box. Alternatively,  \href{https://github.com/aliasrobotics/RVD/issues/6}{RVD\#6} allows an adjacent attacker to exploit a stack-based buffer overflow in the UR3 ModBus TCP service, and execute commands as root equally fulfilling the cyber intrusion. Akerbeltz can be configured to exploit these flaws and take action over industrial LANs.\\
\end{itemize}

\subsubsection{Lateral movement}

Since the cyber-intrusion allowed to obtain root privileges in the control box, no privilege escalation phase is necessary within Akerbeltz. The "open" architecture of the robot facilitates lateral movement to other sub-devices within the robot. We briefly explored such setup and found that it is trivial to access the robot user interface, the PolyScope. Figure \ref{figure:polyscope_hijacked} shows a preview of the ransomware's message in the teach pendant:

\begin{figure}[h!]
    \includegraphics[width=1\textwidth]{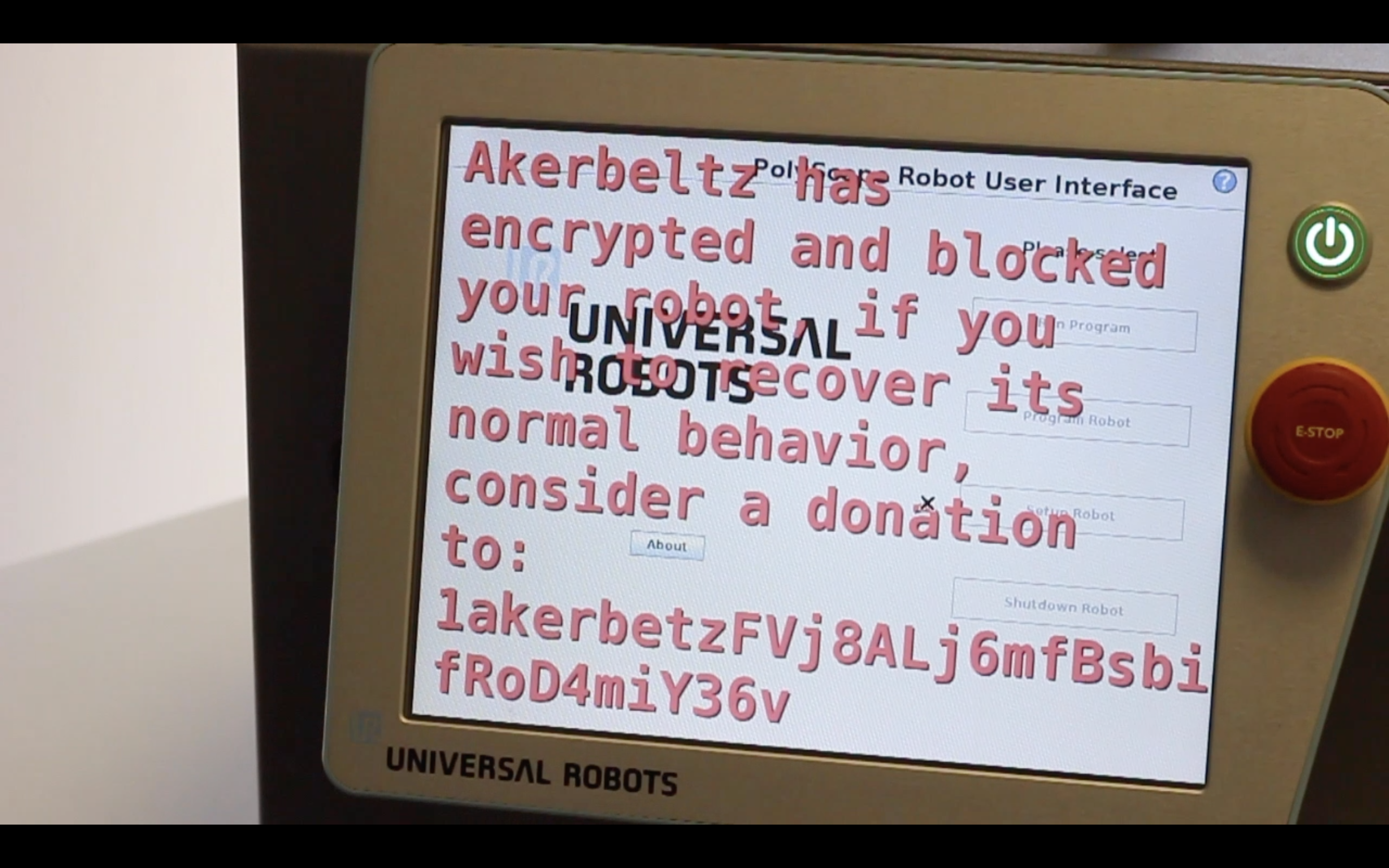}
    \centering
    \caption{UR3 PolyScope Robot User Interface  hijacked. In this case, Akerbeltz locks and disables most functions and a message is overlayed on top. Reproduced from \cite{mayoral2019industrial}.}
    \label{figure:polyscope_hijacked}
\end{figure}


The UR3 lack of security allows for additional lateral movement. Both the BIOS of the robot controller and the safety PLC are easily accessible and exposed. 
While it remains beyond the scope of our study, we argue that targeting any of these systems (or both together) will likely lead to malware that could damage the robot to a point of no return and would likely be matter of future security research.\\

\subsubsection{Control}

\begin{itemize}
	\item \textbf{Locking:} In our PoC, our team was able to lock the whole system while displaying a message (refer to Figure \ref{figure:industry_robot_target}) after booting, using previously mentioned CVEs, we managed to access the control box, acquire root privileges, disable the default safety configuration and change the default user and password. We also identified and disabled several open ports used to control the robot from the outside, blocking other mechanisms to reestablish normal operation. Our work with Akerbeltz finalized by performing some minor and non-exhaustive hardening, meant to avoid users to unlock the system.
	\item \textbf{Encrypting:} After identifying most of the critical files used for the control of the robot, we located previous programs and IP deployed within the control box. These files are encrypted using the local pgp binary in combination with a series of robot-specific identifiers.
\end{itemize}

\subsection{Discussion}

In an attempt to responsibly disclose and mitigate the existing flaws, we generated a series of reports for selected vulnerabilities and approached the manufacturer by e-mail. To this date, no formal answer with intention to establish discussions on security has been received. In a second attempt, making use of well established and \emph{de facto} approaches in security, we filed for a CVE identifier in one of the new discovered vulnerabilities. At the time of writing, no formal communication has been established via this channel either.\\
\newline
In a third attempt, the PoC attack was disclosed to Universal Robots directly, in a public robotics industrial conference, the ROS-Industrial Conference in Europe (Stuttgart, December 2019). While maintaining the vulnerabilities undisclosed, we publicly presented\footnote{Recording of the talk is available at \url{https://youtu.be/J5-8ptUT9qU?t=27052}.} to the vendor and the rest of the audience the consequences of their insecurity. We briefly presented Akerbeltz and followed with a possible solution to mitigate existing flaws in their robots. \\
\newline
The deployment of Akerbeltz in a UR3 via the physical attack vector has been recorded and made available at \url{https://youtu.be/VF7fcV5j1t0} arguing that by making this public and available we finally call  to a reaction of this particular vendor or its associated value chain.
\newline
Last but not least, our final discussion point is that these kind of targeted attacks have the potential to spread rapidly and across the use cases in which the Universal Robots CB-series are deployed. Given the low complexity of the attack, the easiness of programming on top of Universal Robots programming environment and the relative high cost of the assets involved, a series of attacks on this taxonomy are foreseen. 


\section{Conclusions}

The current insecurity status in robotics  allows for malware to be created easily and rapidly by simply evaluating known vulnerabilities. These pieces of malware could be used and exploited by malicious actors to gain economic profits via extorsion. In this chapter we presented Akerbeltz, a piece of ransomware  that locks and encrypts industrial collaborative robots from Universal Robots. After the targeted attack, the result is rendering the industrial robot totally useless, losing IP within the robot and potentially leading to economical losses and human or environmental damages. We described and prototyped Akerbeltz for the UR3 CB series robot which could be introduced following  physical or network based attacks on zero-day vulnerabilities.\\
\newline
At the time of writing, Universal Robots has yet to answer our vulnerability reports, which in some cases provide the manufacturer relevant insights on how to mitigate the most relevant flaws exploited by the attack. As a particular example, for the physical attack vector we exploited in Akerbeltz, a possible mitigation involves the validation and sanitization of specific files as well as the implementation of Discretionary Access Control (DAC) model by providing Mandatory Access Control (MAC) within the control box. This will limit the programs' capabilities with per-program profiles (e.g. via AppArmor).\\
\newline
When asked directly, face to face, representatives of the vendor, they indicated once again that their robot "is designed to be open" or that we are "disabling features" and therefore, "it's just normal". We would like to express our strong concern for a reiterated lack of security actions. While our team did not explore malicious attacks or their profitability on the robot any further, we argue that a more aggressive individual or group, with bad intentions, might easily come up with some sort of "self-destructing" malware (by disabling safety and repeatedly crossing the boundaries of the kinematics model and hitting itself) or worse, a "human or environment damaging" one, or any other kind of creative exploits on top of the vulnerabilities exploited by Akerbeltz.\\
\newline
The lack of communication from the vendor's perspective makes us guess that, at best, there is security by obscurity around Universal Robots. Once again as we did in the past, we argue against this and advocate for a proactive interaction of robot manufacturers with security researchers.\\
\newline
Other industry case studies have also been conducted as part of this research and have been made available at \cite{aliascasestudies}. These studies cover a broad spectrum of penetration testing scenarios, including the analysis of widely-used industrial robot arms from manufacturers like ABB and KUKA, the security testing of Mobile Industrial Robots (MiR) in healthcare environments, and the exploration of cybersecurity vulnerabilities in commercial drones.


    \part{Automation}\label{part:automation}
    
\chapter{Determining the Attack Target with Game Theory}
\label{chapter:game_theory}

\epigraph{\textbf{Contributions}: \emph{Supplied the case studies at \cite{rass2023game}, including basic data to run the experiments on, and the context thereof. Contributed also to the discussion of results and judgement of their applicability, validity and verification in the robotics use case.}}{}

As we venture into the intricate landscape of offensive robot cybersecurity, particularly within the realm of automation (Part \ref{part:automation}), the act of selecting an attack vector transcends brute force or mere choice—it becomes a strategic decision of profound consequence for both defense and offense approaches. This chapter ventures into the analytical domain where the discipline of cybersecurity intersects with the predictive prowess of game theory. Herein, we dissect the art of anticipating an adversary's move, transforming what is traditionally an intuitive decision (heuristics) into a calculable prediction.\\
The quintessence of this approach lies not just in its offensive application but equally in its defensive potential. By leveraging game-theoretic principles, we argue an automation engine can forecast an attacker's most probable vector of exploitation (the attack vector). Such predictive capacity is paramount; it informs how we can judiciously deploy our finite security defensive resources where they are most needed—effectively outmaneuvering potential breaches before they materialize. This is ultimately the objective pursued in this thesis with the \emph{offensive robot cybersecurity} approach.
\begin{mybox}
{\textbf{\textit{Cybersecurity cognitive engines}}}
\emph{Cybersecurity cognitive engines} refer to advanced computational systems that utilize Artificial Intelligence and Game Theory to understand, predict, and make decisions about cybersecurity threats. These engines act as a smart, evolving line of defense or offense against cyber threats, adapting to the ever-changing cybersecurity landscape.
\end{mybox}
This chapter is not only a discourse on methodology and theoretical concepts but also a tactical guide. It reveals how game theory serves as the toolbox in the creation of cybersecurity cognitive engines that adapt dynamically to the ever-shifting tides of digital conflict. These engines, armed with the ability to discern the most vulnerable vectors, provide an indispensable tool for entities tasked with safeguarding automated systems. Through a meticulous blend of theory and practical application, this chapter embarks on a quest to ascertain the optimal point of action—be it for a defender fortifying their bastion or an attacker seeking a breach—in the digital arena where automation's role is ever-expanding and critical.

\section{Introduction}

%
%
The daily business of a \ac{CISO} concerns seeking a proactive defense against an invisible intruder and more elaborated attacks. Such threats are often well represented by contemporary \acp{APT}, a type of cyber attack executed by well-resourced and capable attackers, such as nation-states or organized criminal groups, aiming to infiltrate a network to extract valuable information or disrupt operations. These threats are "advanced" in the sophistication of their tactics, "persistent" in their long-term objectives, and stealthy in nature, often remaining undetected for extended periods. APTs typically involve a high degree of covertness over a prolonged duration, with the intention to steal data rather than cause immediate harm. This work is concerned with the daily business of defense, under the assumption that the infection has already happened, but there has not been any damage so far, thus we are in the ``incubation'' phase in the APT life-cycle. A refined view on the evolution of an APT is the \emph{kill chain} \cite{kamhoua_game_2018}, depicted initially in Figure \ref{figure:offensive_pipeline}. 
%
%

The contributions presented here focus on the study of a game theoretic model intended to support \acp{CISO} and/or to serve as a tactical guide that helps autonomous cybersecurity agents to make decisions and reason about the threat landscape from an attacker's perspective. Thus, coherent with \acp{APT}, we assume that the adversary is already in the system\footnote{Following the famous quote of Robert S. Mueller: ``There are only two types of companies: Those that have been hacked and those that will be hacked.''}, and the \ac{CISO} duty is preventing damage. The battlefield on which the CISO, hereafter called the defender, matches the stealthy adversary, is an attack graph. This is a graph-theoretic model of a system, relating individual system components to threats and exploits on them, and thereby visualizing possible attack paths towards defined target nodes. Such attack graphs can be compiled from threat modeling activities and further refined with penetration tests and help of topological vulnerability scans (see, e.g., \cite{jajodia_cauldron_2011}). Alternatively, attack graphs (also referred to as the \emph{attack tree} or the \emph{exploit tree} in this thesis) can be built by tools like \exploitflow \cite{mayoral2023exploitflow}. Figure \ref{figure:cybersecurity_cognitive_engine:game_theory} depicts how this work fits into our architecture to develop fully automated offensive robot cybersecurity tools, a cognitive cybersecurity engine we denominate \framework{}.

\begin{figure}[h!]
    \centering 

    \begin{tikzpicture}[node distance=2cm, auto, scale=0.55, transform shape]
        \tikzstyle{box} = [rectangle, draw, fill=gray!20, text centered, rounded corners, minimum height=2.5em, minimum width=3cm]
        \tikzstyle{emphasize} = [box, fill=blue!30, draw=blue, line width=1.5pt]
        \tikzstyle{line} = [draw, -{Latex[scale=1.2]}]
        \tikzstyle{deemphasize} = [opacity=0.3]
        \tikzstyle{innerComp} = [rectangle, draw, fill=orange!30, text centered, minimum height=1.5em, minimum width=3em]
        \tikzstyle{innerComp_emphasize} = [rectangle, fill=blue!30, draw=blue, line width=1.5pt]
    
        \node[box] (U) {User};
        \node[box, deemphasize] (EF) at (11cm,0) {1. ExploitFlow}; 
        \node[box, deemphasize, below left=4cm and 2.5cm of EF] (PentestGPT) {2. PentestGPT};
        \node[box, deemphasize, below=8cm of EF] (PentestPerf) {3. PentestPerf};
        \node[box, right=3cm of EF] (T) {Target};
    
        \node[innerComp, deemphasize, above=2.5cm of EF, xshift=-4cm] (exploit) {exploit};
        \node[innerComp, deemphasize, right of=exploit] (flow) {flow};
        \node[innerComp, deemphasize, right of=flow] (graph) {graph};
        \node[innerComp, deemphasize, right of=graph] (adapters) {adapters};
        \node[innerComp, deemphasize, right of=adapters] (models) {models};
        \node[innerComp, deemphasize, right of=models] (state) {state};

        \node[innerComp, innerComp_emphasize, below=0.5cm of EF, xshift=3.5cm] (gt) {Game Theory (\CutTheRope)};
        
        \draw[dashed, shorten <=5pt, shorten >=5pt] (EF) -- (exploit);
        \draw[dashed, shorten <=5pt, shorten >=5pt] (EF) -- (flow);
        \draw[dashed, shorten <=5pt, shorten >=5pt] (EF) -- (graph);
        \draw[dashed, shorten <=5pt, shorten >=5pt] (EF) -- (adapters);
        \draw[dashed, shorten <=5pt, shorten >=5pt] (EF) -- (models);
        \draw[dashed, shorten <=5pt, shorten >=5pt] (EF) -- (state);
        \draw[dashed, shorten <=5pt, shorten >=5pt] (EF) -- (gt);
        \draw[dashed, shorten <=5pt, shorten >=5pt] (PentestGPT) -- (gt);
    
        \path[line] (U) -- node[font=\small, midway, sloped, above] {programatically in Python} (EF); 
        \path[line, deemphasize] (U) edge[bend right=40] node[font=\small, sloped, below, opacity=0.5] {goal description in text} (PentestGPT);
        \path[line, deemphasize] (PentestGPT) edge[bend right=20] node[font=\small, sloped, below] {exchange exploit tree} (EF);
        \path[line, deemphasize] (EF) edge[bend left=20] node {} (PentestGPT);
        \path[line, deemphasize] (EF) -- node[font=\small, sloped, above, opacity=0.5] {Benchmarks an exploit flow} (PentestPerf);
        \path[line] (EF) -- (T);
        \path[line] (T) -- (EF);
    
        \node[rectangle, fill=gray!5, draw, dashed, deemphasize, fit=(EF) (PentestGPT), inner sep=1cm] (group) {};
        \node[anchor=north west, deemphasize, inner sep=0.5cm] at (group.south west) {4. Malism};
    
        \node[box] (U) {User};
        \node[box, deemphasize] (EF) at (11cm,0) {1. ExploitFlow}; 
        \node[box, deemphasize, below left=4cm and 2.5cm of EF] (PentestGPT) {2. PentestGPT};
        \node[box, deemphasize, below=8cm of EF] (PentestPerf) {3. PentestPerf};
        \node[box, right=3cm of EF] (T) {Target};
        \node[innerComp, innerComp_emphasize, below=0.5cm of EF, xshift=3.5cm] (gt) {Game Theory (\CutTheRope)};
    
        \begin{scope}[node distance=0.3cm, font=\footnotesize]
            \coordinate (legendOrigin) at (-4,-6); 
            \node[box, right=0.3cm of legendOrigin, minimum height=1.5em, minimum width=1.5em, text width=1.5em] (legendBox) {}; 
            \node[right=0.3cm of legendBox] {External entity};
    
            \node[box, below=0.2cm of legendBox, fill=gray!5, minimum height=1.5em, minimum width=1.5em, text width=1.5em] (otherpapers) {}; 
            \node[right=0.3cm of otherpapers] {Other chapters};
    
            \node[emphasize, below=0.2cm of otherpapers, minimum height=1.5em, minimum width=1.5em, text width=1.5em] (legendEmphasize) {};
            \node[right=0.3cm of legendEmphasize] {This contribution};
    
            \node[innerComp, below=0.2cm of legendEmphasize, minimum height=1.5em, minimum width=1.5em, text width=1.5em] (legendInner) {};
            \node[right=0.3cm of legendInner] {Inner Component};
        \end{scope}
    \end{tikzpicture}
        
    \caption{Game theoretic contributions to our architecture to develop fully automated offensive robot cybersecurity agent, a cognitive cybersecurity engine denominated \framework{}. The \CutTheRope method implemented sits between \exploitflow and \tool{}, receiving the exploit tree from \exploitflow, reasoning about it to determine optimal attack vectors, passing such vectors to \tool{} to transform them into actionable exploits and finally, feeding such exploits back to \exploitflow for their execution and result digestion.
        %
        %
    }
    \label{figure:cybersecurity_cognitive_engine:game_theory} 
\end{figure}

On the attack graph, we play a stealthy intrusion game with infinite time horizon and repetitions, in which the defender takes action to keep the attacker away from the critical assets in the enterprise, while the adversary is on its way towards some target asset, along one or more attack paths. The game instantly terminates if the attacker has reached the critical asset, in which case the defender (permanently) lost. The interaction between the two players is constrained as follows:

\emph{Assumption 1}: The attack graph $G=(V,E)$ consists of $V$ nodes and $E$ are the edges (e.g., exploits).
We assume a single starting node (for all attacks), and a single target node, representing some critical asset to be captured. The graph is assumed as acyclic (thus, the attacker will never enter infinite loops along accomplished exploits), and all paths, more precisely the attack paths/vectors, lead to the target asset. We denote this target as $v_0\in V$ hereafter. The nodes in $V$ represent system threats and vulnerabilities, or system states, while edges represent threats or exploits to get from one component/state into the next component/state. For example, the adversary may jump from a desktop computer $v_{pc}$ to a server $v_{server}$, or may gain root privileges, denoted as $v_{pc,root}$ from user-level privileges $v_{pc,user}$ on the same computer. In either case, we would have a directed edge $v_{pc}\to v_{server}$, or $v_{pc,user}\to v_{pc,root}$ to express this attack path in the graph model.
	
\emph{Assumption 2}: The adversary is already somewhere in the system at an unknown location when the defender enters the gameplay (i.e., we are  past the event of infection/reconnaissance). We also assume that the defender has no indication of adversarial activity (for example, from an Intrusion Detection System (IDS)); the adversary is \emph{stealthy}. In absence of adversarial signals, the defender may assume all possible locations of the adversary as uniformly distributed (the inclusion of signals is discussed in Section \ref{sec:signals}).
	
\emph{Assumption 3}: the adversary may run \emph{parallel} or \emph{concurrent attacks}, thereby exploiting several, up to all, attack paths simultaneously to maximize its chances to conquer $v_0$.

\begin{figure}[h!]
	\centering
	\includegraphics[width=0.8\textwidth]{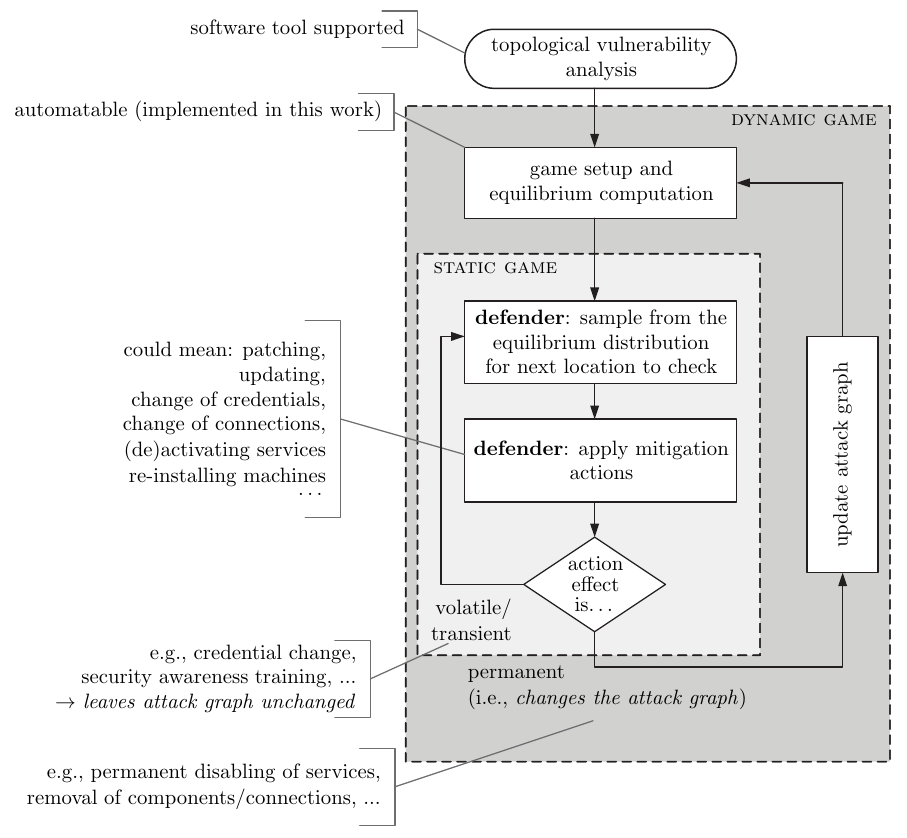}
	\caption{\CutTheRope (static game) inside the continuous process of permanent system hardening to fight against. Reproduced from \cite{rass2023game}.}
 \label{figure:game_theory_ctr}
\end{figure}

This general setting was converted into a game theoretic model in \cite{rass_cut--rope_2019}, named \CutTheRope: Like on a \emph{chess-board}, the adversary in \CutTheRope runs parallel attacks, one on each attack path available. Since the attacker's location on each attack paths is unknown to the defender, this player imagines a whole ``cohort'' of \emph{avatars} starting from all possible locations in the network and moving towards $v_0$. The strategic choice of the adversary, from the defender's perspective, is about existing attack paths, but the defender does not know where the attacker is, equivalently, how far the adversary has already come down an attack path $\pi$. To tackle this uncertainty, the defender plays the game as if the adversary would first (strategically) choose the path $\pi$, and move \emph{all} avatars on $\pi$ simultaneously towards $v_0$. In other words, it moves avatars at all \emph{possible}, \emph{not} necessarily also \emph{probable}, locations. The target asset (and security game) is lost (to the attacker), if at least one of the adversary's avatars reaches $v_0$. 

For the adversarial movement, \CutTheRope lets the defender assume a random distribution on how many steps an avatar can take when it is on the move. Figure \ref{figure:game_theory_ctr} provides a high-level overview of the games involved in the method. This \emph{random movement pattern} is a model design choice, and the original work \cite{rass_cut--rope_2019} about \CutTheRope assumed one specific movement regime, in which the defender acts periodically, and the attacker has some ``aggressiveness level'' $\lambda$, interpretable as an ``expected number exploits per day'' or within a defined unit of time. This amounts to a Poisson-distributed number of steps taken along the attack path. 



\subsection*{Research Questions and Contribution}

Some companies have \acp{CISO} with a fixed working schedule, which corresponds to a periodically active defender (working days, day/night-shifts, etc.). What if there is a 24/7 continuous response team available, like in bigger (globally distributed) companies that run their own security operations center? These (multiple) defenders may become active at random time intervals and at any time. For the security game, it means that the defender will not be active periodically, but rather at random times and possibly at any time. This is the \emph{first new} movement contributed and studied in this chapter. 

Both, the original Poisson movement and the just described continuous security response policy are agnostic of the particular details of exploits, like their difficulty or severity. However, many attack graphs do carry additional information about exploits, and if so, it is desirable to use it in the security model. These additional details can range from a security threat research and risk assessment or scoring like \ac{CVSS}, up to proof of concept implementations for each exploit. The \emph{second new} movement pattern proposed in this chapter makes the attacker's traversal dependent on exploit complexities (threats assigned a higher complexity would thereby be probabilistically less feasible), as far as they are known, and studies the defense performance against a defender that is again periodically active. 

The \emph{third new} movement pattern is a combination of a defender that can take action at any time in a 24/7 continuous security provisioning, against an attacker that has to deal with threats and exploits of different complexity. 

We stress that none of these cases assumes a purely reactive defense, i.e., we do not study security response patterns when the incident has been noticed. This is due to the assumption of stealthiness of the intrusion; once the attacker becomes visible, it is because the target asset $v_0$ was lost (permanently). 

The other contributions of this chapter are two case studies:
First, we provide an experimental study and illustration of how to use \CutTheRope in different settings, based on two documented attack graphs for industrial robots. Specifically, we look at the Modular Articulated Robotic Arm (MARA) and Mobile industrial Robotics MiR100 robots \cite{alias_robotics_case_2019,acutronicrobotics_threat_2021,alias_robotics_week_2020,alias_robotics_case_2021}, for which attack graphs have been compiled by security experts. On these, we instantiate \CutTheRope and compute results in the aforementioned settings of a periodically/randomly active defender versus an attacker that traverses an attack path with uniform speed at an average number of exploits per time unit, or mounts attacks with individually distinct exploit complexities, thus being slower or faster, depending on the chosen path. 
Experiments are conducted on the MARA robot, for which the threats and exploits are known, but without a \ac{CVSS} rating or further details. In that case, the original Poisson model from \cite{rass_cut--rope_2019} and the first of the new patterns announced above are usable. The other case study is on the MiR100 robot, whose attack graph carries additional vulnerability and exploit details.

\begin{mybox}
{\textbf{\textit{Robotic Applications of CTR in Cybersecurity}}}

To concentrate on the topic of robot cybersecurity addressed in this thesis, this chapter primarily discusses the application of \CutTheRope (CTR) within robotic contexts, rather than delving deeply into its theoretical foundations. For those interested in a comprehensive exploration of the theoretical principles that underpin CTR, references \cite{rass_cut--rope_2019} and \cite{rass2023game} are highly recommended for thorough review. Additionally, a concise summary of these theoretical aspects is provided in Appendix \ref{chapter:gt_appendix}, offering a brief overview to complement the focused discussions herein.
\end{mybox}

\section{Case Studies}
\showkeysfalse
\label{sec:case-studies}
\showkeystrue
We dedicate the next couple of subsections to numeric results, starting with a brief correction to past calculations in the literature, and then moving onward to the new case studies and the comparison of defense policies optimized with \CutTheRope, versus a heuristic common-sense defense policy.

To assess the game w.r.t. a real-life application, we conducted two case studies on the industrial robots in Section \ref{sec:robot-case-studies}. The game is similar to capture-the-flag competitions known in ethical hacking, since there and also here, the goal is to ``capture'' a target asset $v_0$. Our analysis, different from ethical hacking, is purely game-theoretic and optimization-based here. 

\subsection{Implementation Remarks}
\showkeysfalse
\label{sec:implementation-remarks}
\showkeystrue
We adapted the implementation from \cite{rass_cut--rope_2019} and thereby discovered a few bugs in this older code that we corrected in our version. The original code used fictitious play on the full distribution $U=(u_1,u_2,\ldots,u_n=\Pr$(adversary reaches $v_0))$ obtained from equation \eqref{eqn:adversary-utility}.

We compute an optimum w.r.t. a lexicographic order from right to left, first minimizing the last coordinate $u_n$, and breaking ties by continuing to minimize $u_{n-1}$ while keeping $u_n$ at minimum. The next tie is broken using $u_{n-2}$, while keeping the so-far optimized coordinates at their minima and so on. This introduces a dependence on the ordering of the coordinates, corresponding to a likewise ordering of locations in the attack graph. Therefore, the solution returned by the implementation from \cite{rass_cut--rope_2019} is ambiguous in the sense of depending on the node ordering.

The optimization, however, independently of the node order, always minimizes the chances to reach the target asset, and hence provides a valid defense policy w.r.t. the targets of the defender. Our implementation inherits this dependence on the node order, but since our sole interest is reaching or avoiding to reach $v_0$ anyway as \eqref{eqn:attacker-utility-formula} and \eqref{eqn:defender-utility-formula} define, this ambiguity is not a limitation. In light of this, we chose the graph-topological sorting to order the probabilities in $U'$ other than for $v_0$, which is the last element in this vector.

The attack graphs for our robot case studies have several entry points for the attacker and also several targets to reach. To handle them all in a single run of the analysis, we \emph{added an artificial} (virtual) \emph{entry node} from which all (real) entry nodes are trivially reachable (with probability 1). Since \CutTheRope in the original version, analyzed here, assumes only one target, we \emph{contracted the multitude of target nodes} into a single ``compound'' target node. This corresponds to the target being to reach any of the possible target nodes, not distinguishing which in particular. A \emph{target} node is, by default in our implementation, any node that does not have descendants (zero out-degree in the attack graph). Consequently, all inner nodes, except the virtual start, are possible defense spot-check locations. The technical simplification towards having one target (only) is to avoid multi-criteria optimization, which is theoretically possible (even supported by the packages to run the optimization practically), but is more involved to interpret for a defense policy. 

\subsection{Robot Case Studies}
\showkeysfalse
\label{sec:robot-case-studies}
\showkeystrue
For both of the robot cases to follow, we give computational results and a discussion of their practicality. To avoid confusion between the attack graphs appearing here and those found in the cited literature, we use the original versions thereof to visualize the battlefield and results. The actual simulation was done on an attack graph with added virtual starting and a single compound target node (if more than one exists).

It is perhaps practically interesting to remark that both attack graphs have inner nodes that classify as attack targets, but have descendant nodes as subsequent attack targets. With the convention of taking nodes with zero out-degree in the graph as targets (see Section \ref{sec:implementation-remarks}), the simulation will include all ``inner'' nodes as defense locations even though they may be attack targets too. This is not precluded by the game design, and may be interpreted as considering inner nodes as ``intermediate targets'' whose prevention may avoid subsequent final, perhaps more dangerous, attack targets. The game's defense policy would then advise to prevent a certain attack sub-target in the attack graph, with a certain level of effort (expressed as likelihood). From a simulation perspective, including or excluding any node from the defense policy is a simple matter of defining the action set for the defender accordingly.

\paragraph{Case \#1: Modular Articulated Robotic Arm (MARA)}
MARA is a collaborative robotic arm with ROS 2.0 in each actuator, sensor or any other representative module. Each module has native ROS 2.0 support, can be physically extended in a seamless manner and delivers industrial-grade features including synchronization, deterministic communication latencies, a ROS 2.0 software and hardware component life-cycle and more. Altogether, MARA empowers new possibilities and applications in the professional landscape of robotics. The use case considered contemplates the MARA modular robot operating in an industrial environment while performing a pick \& place activity. Details about MARA for this case study can be found in \cite{alias_robotics_case_2019,acutronicrobotics_threat_2021}.

\paragraph{Case \#2: MiR100 - Mobile industrial Robotics}
The MiR100 autonomous mobile robot is advertised as a safe and cost-effective mobile robot that quickly automates your internal transportation and logistics. The robot claims to optimize workflows, freeing staff resources so you can increase productivity and reduce costs. A case study analyzing the cyber-resilience of MiR100 robots was conducted and documented at \cite{alias_robotics_week_2020} and \cite{alias_robotics_case_2021}, which considered a single robot operating in a structured environment while connected to a local area network that gets compromised. Through the local area network, prior work demonstrated how an attacker could exploit vulnerabilities, pivoting across subsystems in the robot all the way into its safety system, disabling it fully in a remote manner.

For both robots, we took attack graphs out of industrial security assessments, which, in the particular case of MiR100, were also annotated with \ac{CVE} and \ac{CVSS} information, which allows an assessment of the ``hardness'' of vulnerabilities along the attack path. Such annotations were not available for the MARA use case, which, in lack of such details, suggests an application of the Poissonian movement pattern of Section \ref{sec:movement-patterns}. The more detailed attack graph for the MiR100 robot enables the consideration of probabilistic success on exploits as Section \ref{sec:probabilistic-succes} described.

\section{Results and Comparison}
\showkeysfalse
\label{sec:results and comparison}
\showkeystrue
To evaluate how much a game-theoretic defense may add to the security, we do not only give the absolute results from the simulations, but also compare them to a heuristic best-effort defense policy, described in Section \ref{sec:heuristic-defense}. Its simulation is run likewise with each of the four movement patterns from Sections \ref{sec:movement-patterns}, with the probability to reach $v_0$ given for each case as (i) optimized by \CutTheRope versus (ii) according to a best-effort defense.

\subsection{Baseline Comparison: A Best-Effort defense Policy}
\showkeysfalse
\label{sec:heuristic-defense}
\showkeystrue
For an assessment of the quality of the game theoretic defense, let us use the following heuristic defense policy to compare:
\begin{itemize}
	\item We assume that an adversarial avatar will always follow the shortest, or ``easiest'' attack path towards $v_0$. The distinction between shortest and easiest is made in dependence of how much is known about exploit complexities. In the MARA use case, the path choice will be for shortest, in terms of the number of exploits, since there is no further detail given about the exploit complexities. In the more detailed MiR100 use case, we have attack complexities and can likewise apply a shortest path algorithm to guide the attacker to the path whose success probability (as the product of all exploit success probabilities) is maximal\footnote{The usual trick of assigning the negative logarithm of probabilities as edge weights and computing a shortest path in the well known way}.
	\item The defender, unbeknownst of where the attacker is, and unable to actively detect it, applies a uniformly random defense strategy. That is, if the attacker is equally likely to be anywhere in the system, the defense policy would likewise be a uniformly random spot checking.
\end{itemize}
Under these hypotheses, we apply the same mechanism as in \CutTheRope, i.e., we let the attacker follow its chosen (shortest/easiest) path, and be occasionally sent back by the defender upon a coincidental cut of the path equivalently, closure of any backdoor. If so, then the avatar will keep retrying, until it hits the final target $v_0$. Note that this regime also includes lateral movement, since we still have a multitude of avatars attacking in parallel, each on its individually optimal route from its starting location $\theta$ towards $v_0$.

We implemented this defense policy simulation by adapting the code from the implementations of \CutTheRope accordingly, to implement the heuristic defense and attack policy of above. Like for the game optimization, the heuristic defense implementation outputs the probability to reach $v_0$ by simulating this defender-attacker interaction. We remark that this heuristic defense may still be overly optimistic relative to real life situations, in which defense teams may have only an incomplete view on the attack graph $G=(V,E)$. The defender would thus only be active on a subset $D\subset V$, so that all nodes in $V\setminus D$ would be zero-day exploits. 

\subsection{Overview of Experiments}
In total, comparing the periodic/exponential defense strategy against a randomly moving adversary in two use cases, gives a total of 4 evaluation scenarios, each accompanied with its own comparison to the baseline heuristic of Section \ref{sec:heuristic-defense}. Table \ref{table:ctr_results_overview} relates the sections and figures in the following to these four configurations.

\setlength{\tabcolsep}{20pt}
\renewcommand{\arraystretch}{1.2}
\begin{table}[h!]
    \scalebox{0.8}{
    \begin{tabular}{|p{0.28\columnwidth}|p{0.28\columnwidth}|p{0.28\columnwidth}|}
    	\hline 
    	\textbf{use case} & defender's policy: \textbf{periodic} & defender's policy: \textbf{exponential strategy}\tabularnewline
    	\hline 
    	\hline 
    	MARA (no particular exploit hardness & attacker movement model: \ref{sec:periodic-defender-poisson} & attacker movement model: \ref{sec:random-spot-checks}\tabularnewline
    	 annotations), Sec. \ref{sec:mara-results}  & results shown in: Fig. \ref{figure:ctr_experiment1_results} & results shown in: Fig. \ref{figure:ctr_experiment2_results}\tabularnewline
    	\hline 
    	MiR100 (known exploit complexities & attacker movement model: Sec. \ref{sec:probabilistic-succes} & attacker movement model: Sec. \ref{sec:double-exponential}\tabularnewline
    	to consider), Sec. \ref{sec:mir100-results}  & results shown in: Fig. \ref{figure:ctr_experiment3_results} & results shown in: Fig. \ref{figure:ctr_experiment4_results}\tabularnewline
    	\hline 
    \end{tabular}}
    \caption{Overview of Experiments}
    \label{table:ctr_results_overview}
\end{table}

%

\subsection{MARA: Results}
\showkeysfalse
\label{sec:mara-results}
\showkeystrue
The attack graph for the MARA robot is taken from \cite{acutronicrobotics_threat_2021} and shown in Figures \ref{figure:ctr_experiment1_results} and \ref{figure:ctr_experiment2_results}. This graph has 11 nodes and 10 edges in total, among them one entry point (node \textcircled{\footnotesize{1}}) for the attacker, and two targets (nodes \textcircled{\footnotesize{6}}, and  \textcircled{\footnotesize{9}}). 

We played \CutTheRope on this graph with a periodic defender versus an attacker that takes an average of 2 moves per time unit (i.e., in-between two appearances of the defender, e.g., per day). Figure \ref{figure:ctr_experiment1_results} shows a table with the probabilities to spot-check each node on the attack graph. For the visualization, we have put bubbles on the attack graph, whose size corresponds to the probability of spot-checking there. That is, the larger the bubble, the more effort should be out on defending at this point.

\begin{figure}[t!]
	\centering
	\begin{minipage}{0.4\textwidth}
		\scriptsize Optimal defense policy:\\[1mm]\centering
		\begin{tabular}{|l|l|}
			\hline
			Node & Probability\\\hline
			2 & 0\\
			3 & 0,279\\
			4 & 0,279\\
			5 & 0,0000384\\
			7 & 0,000268\\
			8 & 0,443\\\hline
		\end{tabular}
	\end{minipage}\quad
	\begin{minipage}{0.9\textwidth}
		\includegraphics[width=\textwidth]{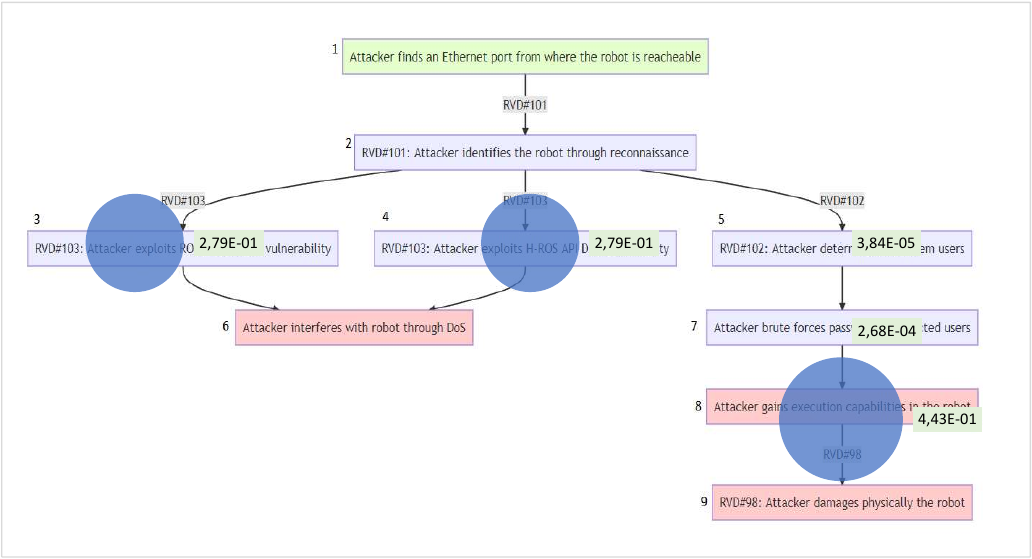}
	\end{minipage}\\[3mm]\scriptsize
Efficacy of the optimal vs. heuristic defense:\\[1mm]
	\begin{tabular}{|c|c|}
		\hline 
		 defense policy & chance to hit $v_0$ \\
		\hline 
		\CutTheRope & 12,8\% \tabularnewline 
		heuristic & 30,6\% \tabularnewline
		\hline 
	\end{tabular}
	\caption{MARA use case results for periodic spot checks (Section \ref{sec:periodic-defender-poisson}). Adapted from \cite{rass2023game}.}
    \label{figure:ctr_experiment1_results} 
\end{figure}

Turning to the case of the defender coming back in random time intervals, we let the game run in three configurations, with the defender moving slower ($\lambda_D=1<\lambda$), at equal speed ($\lambda_D=\lambda=2$) and faster than the attacker $(\lambda_D=3>\lambda)$ in \eqref{eqn:movement-by-geom}. The resulting spot checking probabilities are again displayed as bubbles located at the respective nodes in the attack graph, and put over one another in Figure \ref{figure:ctr_experiment2_results}.

\begin{figure}
	\centering
	\begin{minipage}{0.9\textwidth}
		\scriptsize
		Optimal defense policy:\\[1mm]
        \centering
		\begin{tabular}{|l|l|l|l|} 
			\hline 
			\multirow{2}{*}{Node} & \multicolumn{3}{c|}{probability for $\lambda_{A}=2$ and}\tabularnewline
			& $\lambda_{D}=1$ & $\lambda_{D}=2$ & $\lambda_{D}=3$\tabularnewline
			\hline 			
			2 & 0,000147 & 0,0000502 & 0,0000263\\
			3 & 0,272 & 0,291 & 0,303\\
			4 & 0,272 & 0,291 & 0,303\\
			5 & 0 & 0 & 0,0000123\\
			7 & 0,000473 & 0,000548 & 0,000606\\
			8 & 0,455 & 0,417 & 0,393\\
			\hline
		\end{tabular}
	\end{minipage}\quad
	\begin{minipage}{0.9\textwidth}
		\includegraphics[width=\textwidth]{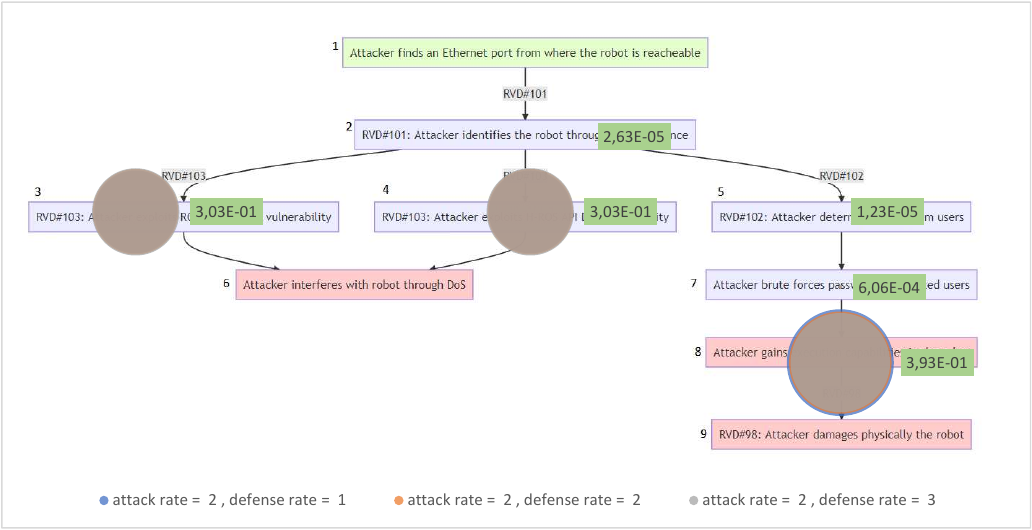}
	\end{minipage}\\[3mm]\scriptsize
Efficacy of the optimal vs. heuristic defense:\\[1mm]
	\begin{tabular}{|c|c|c|c|}
		\hline 
		\multirow{2}{*}{defense policy} & \multicolumn{3}{c|}{chances to hit $v_{0}$}\tabularnewline
		& $\lambda_{D}=1$ & $\lambda_{D}=2$ & $\lambda_{D}=3$\tabularnewline
		\hline  
		\CutTheRope & 7,4\% & 5,7\% & 4,6\%\tabularnewline 
		heuristic & 17,8\% & 13,8\% & 11,2\%\tabularnewline
		\hline 
	\end{tabular}
	\caption{MARA use case results for spot checks at random intervals. Adapted from \cite{rass2023game}.}
    \label{figure:ctr_experiment2_results}
\end{figure}

The numbers and bubbles are almost of the same size, showing that for the defense locations, the speed of spot checking has only a negligible impact, while the performance of the defense accordingly becomes better if the defender is ``more active''. The performances of the defense policy as displayed in the bottom table of Figure \ref{figure:ctr_experiment2_results} show that the optimized defense pays over the heuristic ``blind'' spot checking policy.

The takeaways from these findings is not that a more intense defense activity will reduce the chances of the attacker (this would be obviously the case), but rather giving the defender an indication of where to allocate its (limited) resources to gain the best possible effect. Without signaling and without additional information in the attack graphs, the results are necessarily a crude approximation of reality, and \CutTheRope has been designed to be workable in such a situation of limited information, as well as with cases when more details are available, such as for the MiR100 robot following next. The results in the rather little detailed MARA use case are quite evident but therefore also plausible (``guard the closest graph cut between the asset and the defender''). The nontrivial indication here is the advice to let the attacker come ``close'' to the asset, while a defender would perhaps otherwise try to guard the outer perimeter of the system to keep the intruder out in first place. The optimum to be at the closest graph-cut towards the asset is here explainable by our assumption that the attacker is stealthy and can start from anywhere, and in a practical situation, the defender may indeed have no reliable information about infected parts (otherwise, it would be trivial to disconnect and repair/replace the malfunctioning component). The defense policy that \CutTheRope computes is for practitioners operating blue teams that need to protect a large attack surface with no monitoring or signaling. A game-theoretic defense can help prioritize resources.

\subsection{MiR100: Results}
\showkeysfalse
\label{sec:mir100-results}
\showkeystrue
Similarly as for MARA, we used an attack graph for the MiR100 robot as shown in Figures \ref{figure:ctr_experiment3_results} and \ref{figure:ctr_experiment4_results}. The attack graph has 16 nodes and 24 edges. The attacker can enter at four points (nodes \textcircled{\footnotesize{1}}\ldots \textcircled{\footnotesize{4}}), and four targets (\textcircled{\footnotesize{12}}, \textcircled{\footnotesize{13}}, \textcircled{\footnotesize{14}} and \textcircled{\footnotesize{16}}). 

We conducted the likewise experiments under the same configurations as for the MARA use case, but this time making use of the CVE annotations to give information on how hard it is for the attacker to mount an exploit. For the defender, we again assume this one to be periodically active (as in Section \ref{sec:periodic-defender-poisson}) and to randomly spot check (as in Section \ref{sec:double-exponential}). Note that in this case we do not have an attack rate $\lambda$ as for the MARA use case before, since the movement of the attacker is solely governed by the difficulty to mount exploits. 

\begin{figure}[t!]
	\centering
	\begin{minipage}{0.4\textwidth}
		\scriptsize
		Optimal defense policy:\\[1mm]
		\centering\begin{tabular}{|l|l|}
			\hline
			Node & Probability\\\hline
			5 & 0\\
			6 & 0\\
			7 & 0\\
			8 & 0,163\\
			9 & 0\\
			10 & 0,000837\\
			11 & 0,37\\
			15 & 0,466\\
			\hline
		\end{tabular}
	\end{minipage}\quad
	\begin{minipage}{0.9\textwidth}
		\includegraphics[width=\textwidth]{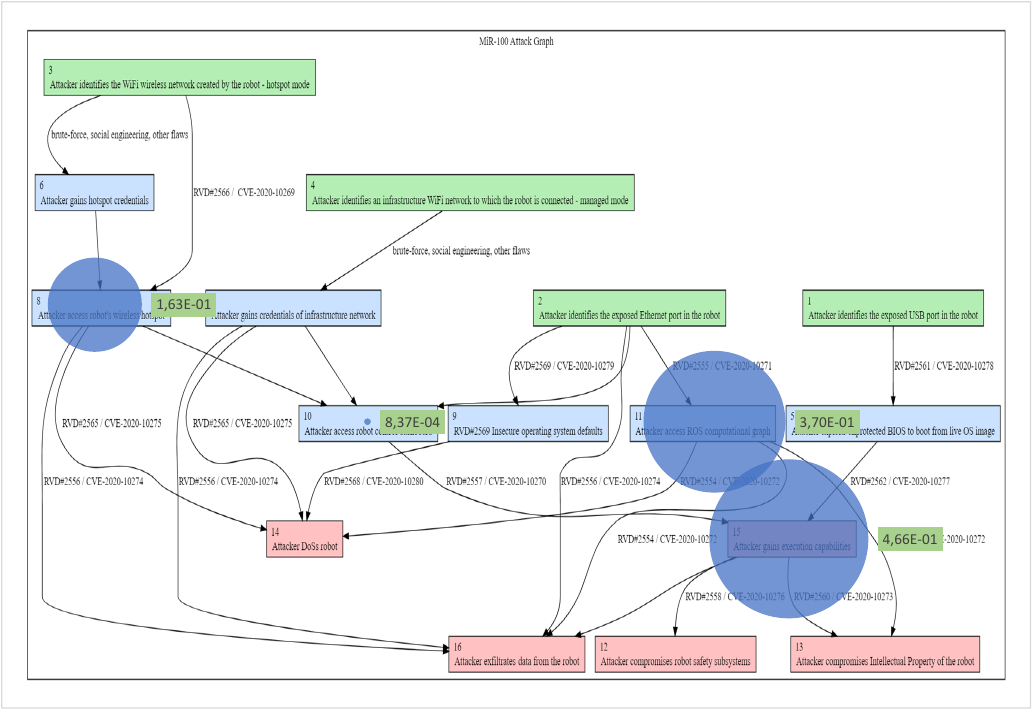}
	\end{minipage}\\[3mm]\scriptsize
	Efficacy of the optimal vs. heuristic defense:\\[1mm]
	\begin{tabular}{|c|c|}
	\hline 
	defense policy & chance to hit $v_0$ \\
	\hline 
	\CutTheRope & 7,7\% \tabularnewline 
	heuristic & 29,8\% \tabularnewline
	\hline 
	\end{tabular}
	\caption{MiR100 use case results for periodic spot checks (Section \ref{sec:periodic-defender-poisson}). Adapted from \cite{rass2023game}.}
    \label{figure:ctr_experiment3_results}
\end{figure}
It is interesting to note that the optimal defense policy does not advise to guard node \textcircled{\footnotesize{7}} or \textcircled{\footnotesize{9}}, which is a way towards reaching goal node \textcircled{\footnotesize{14}}. This may be assumption of the game, of the attacker already being somewhere in the network. The defense policy accounts for this and hence does not put more weight on lower nodes with higher incidence index. This way, the model accounts for defense in depth rather than entry prevention. 

Finally, let us turn to the case of the adversary working towards $v_0$ only in random time intervals between two appearances of the defender. This time, the defender's parameter $\lambda_D$ is the average ``window size'' $W$ (see Appendix \ref{chapter:gt_appendix} section \ref{apx:derivation-of-the-probability}), measured in units of time, e.g., days. It is the time that we give the attacker to mount activities in the game. The results are shown in Figure \ref{figure:ctr_experiment4_results}.

\begin{figure}
	\centering
	\begin{minipage}{0.9\textwidth}
		\scriptsize Optimal defense policy:\\[1mm]\centering
		\begin{tabular}{|l|l|l|l|} 
			\hline 
			\multirow{2}{*}{Node} & \multicolumn{3}{c|}{probability for}\tabularnewline
			& $\lambda_{D}=1$ & $\lambda_{D}=2$ & $\lambda_{D}=3$\tabularnewline
			\hline 			
			5 & 0 & 0 & 0\\
			6 & 0 & 0 & 0\\
			7 & 0,101 & 0,108 & 0,108\\
			8 & 0,249 & 0,281 & 0,291\\
			9 & 0 & 0 & 0,0212\\
			10 & 0,00194 & 0,00161 & 0,00206\\
			11 & 0,354 & 0,28 & 0,238\\
			15 & 0,294 & 0,33 & 0,34\\
			\hline
		\end{tabular}
	\end{minipage}\quad
	\begin{minipage}{0.9\textwidth}
		\includegraphics[width=\textwidth]{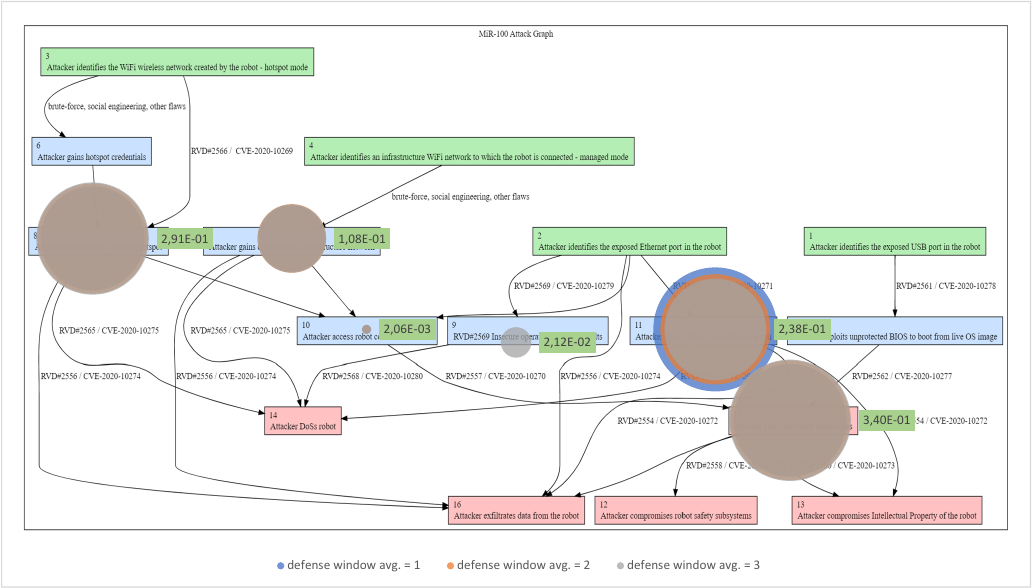}
	\end{minipage}\\[3mm]\scriptsize
	Efficacy of the optimal vs. heuristic defense:\\[1mm]
	\begin{tabular}{|c|c|c|c|}
	\hline 
	\multirow{2}{*}{defense policy} & \multicolumn{3}{c|}{chances to hit $v_{0}$}\tabularnewline
	& $\lambda_{D}=1$ & $\lambda_{D}=2$ & $\lambda_{D}=3$\tabularnewline
	\hline  
	\CutTheRope & 2.4\% & 3.5\% & 4.1\%\tabularnewline 
	heuristic & 11\% & 16.1\% & 19\%\tabularnewline
	\hline 
	\end{tabular}
	\caption{MiR100 use case results for spot checks at random intervals (Section \ref{sec:double-exponential}). Adapted from \cite{rass2023game}.}
    \label{figure:ctr_experiment4_results}
\end{figure}

Similar as for the MARA use case, the defense locations are the same in all cases, with the defense efforts only slightly differing according to how large the window is for the attacker, respectively, how frequently the defender comes back. The performance of the defense is shown in the bottom table of Figure \ref{figure:ctr_experiment4_results}. Consistent with the intuition, the attacker's chances to reach $v_0$ become larger if the defense window is made larger. In both, the experiments with the heuristic defense and optimized under \CutTheRope, the value $\lambda_D$ gives the average number of time units before the defender comes back. That is, larger $\lambda_D$ give the attacker more time to exploit (conversely to the interpretation of $\lambda_D$ in the other experiments, where it was the frequency of the defender's return). Again, the experiments show that \CutTheRope outperforms the heuristic defense considerably.

\section{Discussion}
\showkeysfalse
\label{sec:discussion}
\showkeystrue
The experimental findings suggest that the apparent optimal defense strategy delivered by \CutTheRope is to guard the immediate neighborhood of the target asset, so as to cover cases where the attacker has already deeply penetrated the system when the game begins. Indeed, an analytic characterization of the optimal defense under \CutTheRope is obtained in Appendix \ref{chapter:gt_appendix} section \ref{sec:analytic-results} as Proposition \ref{prop:defenders-equilibrium-support}. It confirms a certain graph cut to be optimal under certain assumptions, but not in all of our test cases. For this reason, we leave the discussion of analytic results as an appended remark here, and continue the discussion with more practical aspects.

\subsection{Incomplete Attack Graphs and Zero-Day Exploits}
\showkeysfalse
\label{sec:zero-days}
\showkeystrue
The heuristic defense of Section \ref{sec:heuristic-defense} may in reality be still over-optimistic, in its assumption of complete knowledge about the attack graph. Practical defense teams may only have a limited knowledge or possibility to construct the entire attack graph, and it is generally unaccomplishable for the defender to get exactly the same attack graph as the adversary has. 
The simulations implemented in this chapter have been made with the possibility to include only a randomly chosen subset of nodes in the defender's possibilities to spot check, so include such incomplete knowledge in the analysis. Concretely, the code was made to randomly reduce the defender's spot check locations to, for example, only 75\% of the nodes in the attack graph. Under such reduced possibilities, the game runs against an attacker with more, i.e., full, knowledge about the attack graph. We confine ourselves here to reporting that the defense policies performed worse than under full knowledge (not surprisingly), but both policies (\CutTheRope and the heuristic) lost performance at approximately equal magnitudes, leaving their relative quality over one another without substantial changes.

We emphasize that a simulation under such reduced knowledge for the defender, whereas giving the attacker full knowledge, can be viewed as a study of the impact of zero-day exploits used by the attacker. That is, any node excluded from the defense, but used by the attacker is nothing else than a zero-day vulnerability. Since a systematic account for this would be beyond the scope (and space limits) of this work, we will explore this route along future work.

\subsection{Cutting the rope vs. Changing the Attack Graph}
\showkeysfalse
\label{sec:cutting-the-rope-vs-changing-the-attack-graph}
\showkeystrue
In our experimental instance of the game, we let the attack graph remain unchanged over time. In particular, we assume that none of the defender's actions causes a permanent removal of a certain backdoor. 
This is practically motivated by the fact that spot checking may remove some, but not all vulnerabilities, so that, for example, one buffer overflow vulnerability in a secure shell implementation might get fixed, but other exploits of the same kind remain open, making the respective nodes remain unchanged in the graph after an inspection. 
Likewise, remote shell access may be required for the business workflow and hence cannot be deactivated, but only the access credentials might be updated. 
In that case, the remote shell access exists before and after an inspection. 
Even though the game model itself uses a static attack graph (see Figure \ref{figure:game_theory_ctr}), this one may itself require an update from time to time upon changes in the infrastructure.  
This is part of the business continuity management related to security, and accordingly changes the action sets for the defender and attacker. 
The implementation of the game, however, remains unaffected, except for the specification/input of the attack graph. We close the discussion at this point, referring to Appendix \ref{chapter:gt_appendix} section \ref{sec:application-for-risk} for a continuation of this discussion.

\subsection{Further Generalizations}
The movement patterns as studied admit further modifications and generalizations, yet to be explored, such as:
\subsubsection*{Probabilistic Success on Spot Checks}
First, to the advantage of the attacker, suppose that the defender is not necessarily successful on wiping out the adversary inside a node $c$, which may the more ``probable'' case in an enterprise or embedded network. It is not difficult to generalize the model towards this: 
If we write $p_c$ for the likelihood to actually cut the rope at $c\in V$ upon trying so,
\eqref{eqn:adversary-utility} becomes a mix of cut and uncut paths,
\begin{align*}
	\Pr(&\text{adversary's location}=v) = \\
	&p_c \cdot \Pr(\text{adversary's location}=v|V(\pi|_c))\\
	&\quad+ (1-p_c)\cdot\Pr(\text{adversary's location}=v|V(\pi)),
\end{align*}
and \eqref{eqn:adversary-utility} is defined alike by the entirety of all these values for all $v\in V$. Thus, the computation as such does not change, only the code needs to use the above formulas to compute the payoffs. If the probabilities are made conditional on the system state, the analysis can be made to account for changing system conditions too. 

\subsubsection*{Multiple Adversarial Targets} \CutTheRope may be modified towards a multi-criteria game, treating all target nodes as individual targets in the game. The concept of a security strategy has a multi-goal counterpart, which the software used for the experimental implementation already supports. The experiments reported here could, possibly, be re-conducted without the merge of targets, i.e., the graph-theoretic contraction.

\subsection{Complexity and Scalability}
\showkeysfalse
\label{sec:complexity}
\showkeystrue
The complexity of the analysis is governed by the time to solve a sequence of linear optimizations. The dimensions of these problems depend on the number of strategies for both players. Using interior point methods, the computational complexity is $O(\abs{AS_1}\cdot p(\abs{AS_1}\cdot\abs{AS_2}))$ for a polynomial $p$ that depends on the chosen optimization algorithm. The need for an exhaustive enumeration of attack paths can raise scalability issues, since the number of attack paths is worst-case exponential. However, the number of paths in an attack \emph{graph} may become large only because many paths overlap in large portions, and the defender may consider using only a subset of paths that cover all edges in the attack graph, so as to cover all known exploits (which is a polynomial number), rather than all possible paths (whose number is exponential). We did not implement such a dimensionality reduction, this modification of the model is a possible aisle of future studies. For the use cases in our research, the number of paths was sufficiently small to admit an exhaustive enumeration. Likewise is the number of paths feasibly small if the battlefield is an attack tree, rather than an attack graph. In any case, \CutTheRope itself does not conceptually change if the restriction is imposed only on the cardinality of the strategy sets to be polynomial in the number of nodes in the attack graph.

\subsection{Including Signals about Adversarial Activities}
\showkeysfalse
\label{sec:signals}
\showkeystrue
The model assumes zero information for the defender about where the adversary is located. Many real-life systems use intrusion detection, and other signaling means (here explicitly not to be understood in the game-theoretic sense of signaling games). If we can compile the entirety of indications about the adversary's activity into a weight $\Pr(\theta)\neq 1/\abs{AS_2}$ for location $\theta$ to possibly start from. Higher values may be assigned where we have stronger indication of recent adversarial activity at location $\theta$.

\section{Conclusions and Outlook}

\CutTheRope has been designed for ease of use in applications with little information (such as exemplified with the MARA robot use case), but also situations where there is detailed information encoded in the attack graph (such as for the MiR100 use case), or even when adversarial indications are available from auxiliary security systems, such as intrusion detection or others (including is possible as outlined in Section \ref{sec:signals}). The accuracy hence depends on how much information we can bring into the game, being a rather crude approximation for MARA, but much more fine-grained for the MiR100 use case. In both cases, however, the defender gets nontrivial advice on where to allocate its typically scarce resources for a best defense, beyond just guarding a graph-cut or choke point towards the critical asset $v_0$. When there are several such (evident) critical regions in the network to defend, different choke points may be of different criticality, depending on how many attack scenarios (each executed by another avatar in \CutTheRope) actually make use of this area in the attack graph.

The results obtained show some limitations for the practical use. First, and most substantially, the results depend on the ordering of the nodes, and -- in addition to the general non-uniqueness of equilibria -- hence may be ambiguous for the defender, leaving a residual chance of there being other defense possibilities. The algorithms applied in this work give only one solution, among perhaps many others. Second, the setting of probabilities from CVE, CVSS or likewise annotations is a nontrivial matter on its own, with only few first steps towards a systematic and sound derivation available in the literature \cite{Koenig2018b}. The assumption of invisibility of the intruder can perhaps be weakened by including signals from intrusion detection or other side-information in the defender's policy. The model simulated here does not include this possibility. Finally, the condensation of several attack goals into a single target node comes with the price of losing accuracy and some information about which attack goal may be more likely to be reached, thus making multi criteria optimization an interesting generalization to study.

Generally, \CutTheRope opens up an interesting class of games of mixed timing of moves between the actors, unlike as in extensive or normal form games, where players usually take actions in a fixed order. Likewise, and also different to many other game models, \CutTheRope has no defined start or finish for the defender (``security is never done''), while only one of the players knows when the game starts and ends, and the attacker can send its avatars from all possible locations in the network. The model is thus complementary to \textsc{FlipIt}, while it allows the attacker to spend any amount of time in the system, as long as the vital asset remains out of reach. This is actually to reflect the reality of security management: we cannot keep the adversary out, we can only try keeping as far away as possible.

In conclusion, \CutTheRope work presented in this chapter epitomizes the dynamic interplay between attack and defense in the realm of robotic cybersecurity, embodying a nuanced approach to resource allocation and threat anticipation. Its adaptability across various information landscapes—from sparse to rich data scenarios—provides a robust framework for decision-making in cybersecurity. While acknowledging its inherent limitations and the non-uniqueness of solutions, \CutTheRope underscores the perpetual nature of cyber defense. It represents a significant stride in the evolution of cybersecurity models, transitioning from static defense strategies to a more proactive, predictive stance. This model not only facilitates effective defense posturing but also enriches our understanding of the complex cyber battlefields, filled with automation, and where the security landscape is constantly reshaped by emerging technologies and sophisticated adversaries.

    \chapter{Guiding Attacks with Machine Learning}
\label{chapter:machine_learning}

\epigraph{\textbf{Contributions}: \emph{One of the top two contributors at \cite{deng2023pentestgpt}. Co-led software contributions and security methodology. All artifacts have been open sourced and results reproducible. Note various LLM models used within this work are not deterministic. This may difficult obtaining identical results.}}{}


In the rapidly evolving field of offensive robot cybersecurity, particularly in automated systems (referenced in Part \ref{part:automation}), the challenge lies in effectively comprehending and navigating the complex threat landscape. Following the previous Chapter \ref{chapter:game_theory}, which focused on automating attack vector selection through game theory, we now turn to the task of guiding an autonomous agent. This involves choosing the correct sequence of actions (like reconnaissance, exploitation, and exfiltration) and appropriate payloads to effectively utilize the chosen attack vector. This necessitates a fusion of in-depth cybersecurity knowledge with specialized domain expertise, such as robotics, each bringing unique challenges and technological nuances.

In this chapter, we delve into the sophisticated realm of guiding autonomous cybersecurity agents, whether for offensive or defensive purposes. The challenge lies in the enormity and obscurity of the cybersecurity state and action spaces, coupled with the unpredictability of the environment. While traditional machine learning approaches, including Reinforcement Learning, have demonstrated proficiency in learning basic policies, their effectiveness is often limited to the data sets they are trained on, lacking the capacity to generalize across diverse cybersecurity scenarios. 

This chapter introduces a novel method for autonomously steering cybersecurity agents towards specific objectives. Our approach harnesses the power of foundational models and Large Language Models (LLMs), which encompass a broad spectrum of general-purpose knowledge. These models are particularly adept at parsing, reasoning, and generating security-related artifacts, thereby playing a pivotal role in guiding cybersecurity exercises. We specifically focus on the application of Generative Pre-trained Transformer (GPT) models \cite{vaswani2017attention,radford2018improving}, with our prototype, \tool{}, demonstrating remarkable capabilities. \tool{} leverages the extensive knowledge base and analytical prowess of LLMs to offer innovative solutions and insights in the field of robotic cybersecurity.

Through this exploration, we aim to showcase how leveraging advanced machine learning techniques can profoundly impact the field of cybersecurity, particularly in automating and refining attack strategies. This chapter not only highlights the potential of these technologies but also sets the stage for future developments in the automated orchestration of cybersecurity tactics.





\section{Introduction}

\begin{figure}[h!]
    \centering 
   
    \begin{tikzpicture}[node distance=2cm, auto, scale=0.5, transform shape]
        \tikzstyle{box} = [rectangle, draw, fill=gray!20, text centered, rounded corners, minimum height=2.5em, minimum width=3cm]
        \tikzstyle{emphasize} = [box, fill=blue!30, draw=blue, line width=1.5pt]
        \tikzstyle{line} = [draw, -{Latex[scale=1.2]}]
        \tikzstyle{deemphasize} = [opacity=0.3]
        \tikzstyle{innerComp} = [rectangle, draw, fill=orange!30, text centered, minimum height=1.5em, minimum width=3em]
        \tikzstyle{innerComp_emphasize} = [rectangle, fill=blue!30, draw=blue, line width=1.5pt]
    
        \node[box] (U) {User};
        \node[box, deemphasize] (EF) at (11cm,0) {1. ExploitFlow}; 
        \node[box, emphasize, below left=4cm and 2.5cm of EF] (PentestGPT) {2. PentestGPT};
        \node[box, emphasize, below=8cm of EF] (PentestPerf) {3. PentestPerf};
        \node[box, right=3cm of EF] (T) {Target};
    
        \node[innerComp, deemphasize, above=2.5cm of EF, xshift=-4cm] (exploit) {exploit};
        \node[innerComp, deemphasize, right of=exploit] (flow) {flow};
        \node[innerComp, deemphasize, right of=flow] (graph) {graph};
        \node[innerComp, deemphasize, right of=graph] (adapters) {adapters};
        \node[innerComp, deemphasize, right of=adapters] (models) {models};
        \node[innerComp, deemphasize, right of=models] (state) {state};
        \node[innerComp, deemphasize, below=0.5cm of EF, xshift=3.5cm] (gt) {Game Theory (\CutTheRope)};
    
        \node[innerComp, above=2.5cm of PentestGPT, xshift=-2cm] (parsing) {parsing};
        \node[innerComp, right of=parsing] (reasoning) {reasoning};
        \node[innerComp, right of=reasoning] (generation) {generation};
    
        \draw[dashed, shorten <=5pt, shorten >=5pt] (EF) -- (exploit);
        \draw[dashed, shorten <=5pt, shorten >=5pt] (EF) -- (flow);
        \draw[dashed, shorten <=5pt, shorten >=5pt] (EF) -- (graph);
        \draw[dashed, shorten <=5pt, shorten >=5pt] (EF) -- (adapters);
        \draw[dashed, shorten <=5pt, shorten >=5pt] (EF) -- (models);
        \draw[dashed, shorten <=5pt, shorten >=5pt] (EF) -- (state);
        \draw[dashed, shorten <=5pt, shorten >=5pt] (EF) -- (gt);
        \draw[dashed, shorten <=5pt, shorten >=5pt] (PentestGPT) -- (gt);        
    
        \draw[dashed, shorten <=5pt, shorten >=5pt] (PentestGPT) -- (parsing);
        \draw[dashed, shorten <=5pt, shorten >=5pt] (PentestGPT) -- (reasoning);
        \draw[dashed, shorten <=5pt, shorten >=5pt] (PentestGPT) -- (generation);
    
        \path[line] (U) -- node[font=\small, midway, sloped, above] {programatically in Python} (EF); 
        \path[line] (U) edge[bend right=40] node[font=\small, sloped, below] {goal description in text} (PentestGPT);
        \path[line, deemphasize] (PentestGPT) edge[bend right=20] node[font=\small, sloped, below] {exchange exploit tree} (EF);
        \path[line, deemphasize] (EF) edge[bend left=20] node {} (PentestGPT);
        \path[line, deemphasize] (EF) -- node[font=\small, sloped, above, opacity=0.5] {Benchmarks an exploit flow} (PentestPerf);
        \path[line] (EF) -- (T);
        \path[line] (T) -- (EF);
    
        \node[rectangle, fill=gray!5, draw, dashed, deemphasize, fit=(EF) (PentestGPT), inner sep=1cm] (group) {};
        \node[anchor=north west, deemphasize, inner sep=0.5cm] at (group.south west) {4. Malism};
    
        \node[box, emphasize, below left=4cm and 2.5cm of EF] (PentestGPT) {2. PentestGPT};
    
        \begin{scope}[node distance=0.3cm, font=\footnotesize]
            \coordinate (legendOrigin) at (-4,-6); 
            \node[box, right=0.3cm of legendOrigin, minimum height=1.5em, minimum width=1.5em, text width=1.5em] (legendBox) {}; 
            \node[right=0.3cm of legendBox] {External entity};
    
            \node[box, below=0.2cm of legendBox, fill=gray!5, minimum height=1.5em, minimum width=1.5em, text width=1.5em] (otherpapers) {}; 
            \node[right=0.3cm of otherpapers] {Other chapters};
    
            \node[emphasize, below=0.2cm of otherpapers, minimum height=1.5em, minimum width=1.5em, text width=1.5em] (legendEmphasize) {};
            \node[right=0.3cm of legendEmphasize] {This chapter};
    
            \node[innerComp, below=0.2cm of legendEmphasize, minimum height=1.5em, minimum width=1.5em, text width=1.5em] (legendInner) {};
            \node[right=0.3cm of legendInner] {Inner Component};
        \end{scope}
    \end{tikzpicture}
    
    \caption{
        Machine Learning contributions to our architecture to develop fully automated offensive robot cybersecurity agent, a cognitive cybersecurity engine denominated \framework{}.
        %
        %
    }
    \label{figure:cybersecurity_cognitive_engine:machine_learning} 
\end{figure}
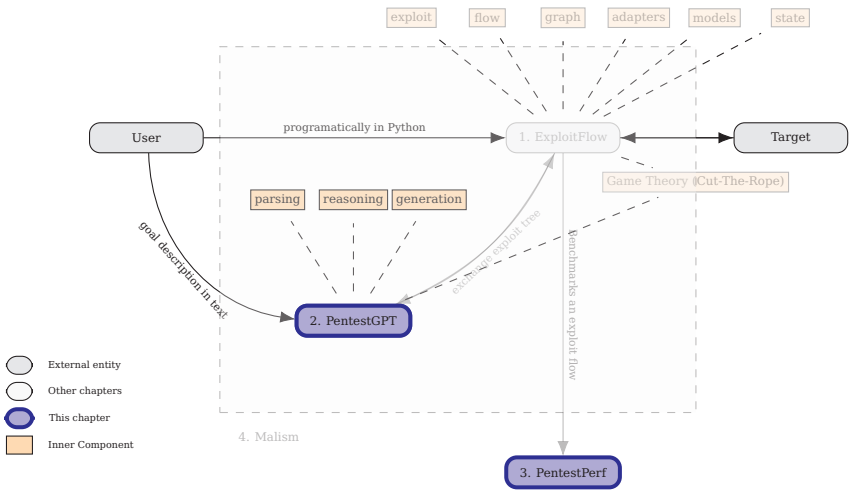

Guaranteeing a system's immunity to potential attacks is a formidable challenge. Offensive security methods, such as penetration testing (pen-testing) or red teaming, have become essential in the security lifecycle. As detailed by Applebaum~\cite{applebaum2017analysis}, these methods require security teams to attempt breaches of an organization's defenses to uncover vulnerabilities. They offer marked advantages over traditional defensive mechanisms, reliant on incomplete system knowledge and modeling. Guided by the principle \emph{``the best defense is a good offense''}, this study focuses on offensive strategies, particularly \emph{penetration testing}.

Penetration testing~\cite{arkin2005software} is a proactive offensive technique aiming at identifying, assessing, and mitigating as many security vulnerabilities as possible. This involves executing targeted attacks to confirm diverse flaws (e.g., erratic behaviors) and is efficacious in creating a comprehensive inventory of vulnerabilities complemented by actionable enhancement recommendations. As a widely-employed practice for security appraisal, penetration testing empowers organizations to discern and neutralize potential vulnerabilities in their networks and systems before exploitation by malicious entities. Despite its significance, the industry often leans on manual techniques and specialized knowledge~\cite{dengnautilus}, making it labor-intensive. This has generated a gap in responding to the escalating demand for adept and efficient security evaluations.

Recently Large Language Models (LLMs)~\cite{zhao2023survey, liu2023summary} are making striking progress, exhibiting an increasingly nuanced understanding of human-like text and effectively executing various tasks across diverse domains. One intriguing aspect of LLMs is their emergent abilities~\cite{wei2022emergent}, which are not explicitly programmed but arise during the training process. These abilities enable LLMs to perform complex tasks such as reasoning, summarization, question-answering, and domain-specific problem-solving without requiring specialized training. Such capabilities indicate the transformative potential of LLMs across various sectors, including cybersecurity. A critical question thus emerges: can LLMs be leveraged in cybersecurity, particularly for performing automated penetration testing?

Motivated by this question, we set out to evaluate the capabilities of LLMs on real-world penetration testing tasks. Unfortunately, the current benchmarks for penetration testing~\cite{antunes2010benchmarking, xiong2010model} are not comprehensive and fail to assess progressive accomplishments fairly during the process. To address this limitation, we construct a robust benchmark that includes test machines from HackTheBox~\cite{htb} and VulnHub~\cite{vulnhub}—two leading platforms for penetration testing challenges. Comprising 13 targets with 182 sub-tasks, our benchmark encompasses all vulnerabilities appearing in OWASP's top 10 vulnerability list~\cite{owasp}. Also, it offers a more detailed evaluation of the tester's performance by monitoring the completion status for each sub-task.

Armed with this benchmark, we conduct an exploratory study using \gptthree{}~\cite{gpt3.5}, \gptfour{}~\cite{gpt4}, and Bard~\cite{bard} as representative LLMs. We interactively test these models by guiding them to complete the penetration tasks against our benchmark targets. This interaction involves setting a penetration testing goal for the LLM, soliciting it for the appropriate operation to execute, implementing it in the testing environment, and feeding the test outputs back to the LLM for next-step reasoning (Figure~\ref{figure:testing_strategy}). By repeating this cycle, we derive the final penetration testing results.
To evaluate the performance of the LLMs, we compare their results against baseline solutions provided by official walkthroughs and solutions from certified penetration testers. By analyzing similarities and differences in their problem-solving approaches, we aim to better understand LLMs' penetration testing capabilities and discern how their problem-solving strategies diverge from those of human experts.

Our investigation yields intriguing insights into the capabilities and limitations of LLMs in penetration testing. We discover that LLMs demonstrate proficiency in managing specific sub-tasks within the testing process, such as utilizing testing tools, interpreting their outputs, and suggesting subsequent actions. Compared to human experts, LLMs are especially adept at executing complex commands and options with testing tools, while models like \gptfour{} excel in comprehending source code and pinpointing vulnerabilities. Furthermore, LLMs can craft appropriate test commands and accurately describe graphical user-interface operations needed for specific tasks. Leveraging their vast knowledge base, they can design inventive testing procedures to unveil potential vulnerabilities in real-world systems and CTF challenges. However, we also note that LLMs have difficulty in maintaining a coherent grasp of the overarching testing scenario, a vital aspect for attaining the testing goal. As the dialogue advances, they may lose sight of earlier discoveries and struggle to apply their reasoning consistently toward the final objective. Additionally, LLMs might overemphasize recent tasks in the conversation history, regardless of their vulnerability status. As a result, they tend to neglect other potential attack surfaces exposed in prior tests and fail to complete the penetration testing task.

The outcomes of our empirical study are promising, revealing that LLMs possess the necessary domain knowledge to perform penetration testing tasks. In particular, they are great at providing \emph{an intuition} of what to do in a given networking scenario. However, what they lack is effective guidance to carry out these tasks independently and maintain a cohesive grasp of the testing scenario. On the other hand, as investigated in a prior research publication \cite{mayoral2023exploitflow} focused on capturing the exploitation route (or flow) for automation. Given the complexity of the (network) state space, the state itself is not enough to reason about what are the best actions to pentest. It rapidly becomes evident that a heuristic is needed to support autonomous pentesting which helps pick actions to achieve given goals.
With this understanding, we aim to contribute unlocking the potential of modern machine learning approaches and develop a fully automated penetration testing framework that helps produce cybersecurity cognitive engines. Our overall architecture is depicted in Figure~\ref{figure:cybersecurity_cognitive_engine:machine_learning}, which shows our current work so far and near future planned contributions.
Our proposed framework, \framework{}, is designed to enable a user without in-depth security domain knowledge to produce its own cybersecurity cognitive engine that helps conduct penetration testing over an extensive range of targets. This framework comprises three primary components:
\begin{enumerate}

\item \ef{} \cite{mayoral2023exploitflow}: A modular library to produce cyber security exploitation routes (\emph{exploit flows}). \ef{} aims to combine and compose exploits from different sources and frameworks, capturing the state of the system being tested in a flow after every discrete action which allows learning attack trees that affect a given system. \ef{}'s main motivation is to facilitate and empower Game Theory and Artificial Intelligence (AI) research in cyber security. It provides a unique representation of the exploitation process that encodes every facet within it. Its representation can be effectively integrated with various penetration testing tools and scripts, such as Metasploit~\cite{metasploit} to perform end-to-end penetration testing. Such representation can be further visualized to guide the human experts for the reproduction of the testing process. 

\item \tool{} (\textbf{this chapter}): An automated penetration testing system that leverages the power of LLMs to produce testing guidance and intuition at every given discrete state. It functions as the core component of the \framework{} framework, guiding the LLMs to efficiently utilize their domain knowledge in real-world testing scenarios.

\item \pentestperf: A comprehensive penetration testing benchmark developed to evaluate the performances of penetration testers and automated tools across a wide array of testing targets. It offers a fair and robust platform for performance comparison.
\end{enumerate}

The harmonious integration of these three components forms an automated, self-evolving penetration testing framework capable of executing penetration tests over various targets, \framework{}. This framework to develop fully automated penetration testing tools, which we named back in Chapter \ref{chapter:machine_learning} \emph{cybersecurity cognitive engines}, aims to revolutionize the field of penetration testing by significantly reducing the need for domain expertise and enabling more comprehensive and reliable testing.

Building on our insights into LLMs' capabilities in penetration testing, we present \tool{}, an interactive system designed to enhance the application of LLMs in this domain. Drawing inspiration from the collaborative dynamics commonly observed in real-world human penetration testing teams, \tool{} is particularly tailored to manage large and intricate projects. It features a tripartite architecture comprising Reasoning, Generation, and Parsing Modules, each reflecting specific roles within penetration testing teams.
The Reasoning Module emulates the function of a lead tester, focusing on maintaining a high-level overview of the penetration testing status. We introduce a novel representation, the Pentesting Task Tree (PTT), based on the cybersecurity attack tree~\cite{attacktree}. This structure encodes the testing process's ongoing status and steers subsequent actions. Uniquely, this representation can be translated into natural language and interpreted by the LLM, thereby comprehended by the Generation Module and directing the testing procedure.
The Generation Module, mirroring a junior tester's role, is responsible for constructing detailed procedures for specific sub-tasks. Translating these into exact testing operations augments the generation process's accuracy. Meanwhile, the Parsing Module deals with diverse text data encountered during penetration testing, such as tool outputs, source codes, and HTTP web pages. It condenses and emphasizes these texts, extracting essential information. Collectively, these modules function as an integrated system. \tool{} completes a complex penetration testing task by bridging high-level strategies with precise execution and intelligent data interpretation, thereby maintaining a coherent and effective testing process.

We evaluate \tool{} using our benchmark to showcase its efficacy. Specifically, our system achieves remarkable performance gains, with 228.6\% and 58.6\% increases in sub-task completion compared to the direct usage of \gptthree{} and \gptfour{}, respectively. We also apply \tool{} to the HackTheBox active penetration testing machines challenge~\cite{htbactive}, completing 4 out of the 10 selected targets at a total OpenAI API cost of 131.5 US Dollars, ranking among the top 1\% players in a community of over 670,000 members. This evaluation underscores \tool{}'s practical value in enhancing penetration testing tasks' efficiency and precision. The solution has been made publicly available on GitHub\footnote{For anonymity during the review process, we have created an anonymous repository to open-source our solution~\cite{anonymous-project}.}, receiving widespread acclaim with over 4,700 stars to the date of writing, active community engagement, and ongoing collaboration with multiple industrial partners. In summary, we make the following contributions:
\begin{itemize}[leftmargin=*]
    \item \textbf{Development of a Comprehensive Penetration Testing Benchmark.} We craft a robust and representative penetration testing benchmark, encompassing a multitude of test machines from leading platforms such as HackTheBox and VulnHub. This benchmark includes 182 sub-tasks covering OWASP's top 10 vulnerabilities, offering fair and comprehensive evaluation of penetration testing.

    \item \textbf{Empirical Evaluation of LLMs for Penetration Testing Tasks.} By employing models like \gptthree{}, \gptfour{}, and Bard, our exploratory study rigorously investigates the strengths and limitations of LLMs in penetration testing. The insights gleaned from this analysis shed valuable light on the capabilities and challenges faced by LLMs, enriching our understanding of their applicability in this specialized domain.

    \item \textbf{Development of an Innovative LLM-powered Penetration Testing System.} We engineer \tool{}, a novel interactive system that leverages the strengths of LLMs to carry out penetration testing tasks automatically. Drawing inspiration from real-world human penetration testing teams, \tool{} integrates a tripartite design that mirrors the collaborative dynamics between senior and junior testers. This architecture optimizes LLMs' usage, significantly enhancing the efficiency and effectiveness of automated penetration testing.
\end{itemize}

\section{Background}
\showkeysfalse
\label{sec:background}
\showkeystrue

\subsection{Penetration Testing}
Penetration testing, or ``pentesting'', is a critical practice to enhance organizational systems' security. In a typical penetration test, security professionals, known as penetration testers, analyze the target system, often leveraging automated tools. The standard process is divided into seven phases~\cite{weidman2014penetration}: Reconnaissance, Scanning, Vulnerability Assessment, Exploitation, and Post Exploitation (including reporting). These phases enable testers to understand the target system, identify vulnerabilities, and exploit them to gain access.

Despite substantial efforts~\cite{abu2018automated, schwartz2019autonomous, xiong2010model} in the field, a fully automated penetration testing pipeline remains elusive. The challenges in automating the process arise from the comprehensive knowledge needed to understand and manipulate various vulnerabilities and the demand for a strategic plan to guide subsequent actions. 
In practice, penetration testers often use a combined approach integrating depth-first and breadth-first search techniques~\cite{weidman2014penetration}. They begin by obtaining an overarching understanding of the target environment (utilizing a breadth-first approach) before focusing on specific services and vulnerabilities (employing a depth-first approach). This strategy ensures a thorough system analysis while prioritizing promising attack vectors, relying heavily on individual experience and domain expertise.
Additionally, penetration testing requires many specialized tools with unique features and functions. This diversity adds complexity to the automation process. Therefore, even with the support of artificial intelligence, creating a fully unified solution for automated penetration testing remains a formidable challenge.

\subsection{Large Language Models}
Large Language Models (LLMs), including OpenAI's \gptthree{} and \gptfour{}, are prominent tools with applications extending to various cybersecurity-related fields, such as code analysis~\cite{pearce2022asleep} and vulnerability repairment~\cite{pearce2023examining}. These models are equipped with wide-ranging general knowledge and the capacity for elementary reasoning. They can comprehend, infer, and produce text resembling human communication, aided by a training corpus encompassing diverse domains like computer science and cybersecurity. Their ability to interpret context and recognize patterns enables them to adapt knowledge to new scenarios. This adaptability, coupled with their proficiency in interacting with systems in a human-like way, positions them as valuable assets in enhancing penetration testing processes. Despite inherent limitations, LLMs offer distinct attributes that can substantially aid in the automation and improvement of penetration testing tasks. The realization of this potential, however, requires the creation and application of a specialized and rigorous benchmark.


\section{Penetration Testing Benchmark}
\showkeysfalse
\label{sec:benchmark}
\showkeystrue

\subsection{Motivation}
\showkeysfalse
\label{sec:benchmark:motivation}
\showkeystrue
The fair evaluation of Large Language Models (LLMs) in penetration testing necessitates a robust and representative benchmark. Existing benchmarks in this domain~\cite{antunes2010benchmarking,xiong2010model} have several limitations. 
First, they are often restricted in scope, focusing on a narrow range of potential vulnerabilities, and thus fail to capture the complexity and diversity of real-world cyber threats. For instance, the OWASP benchmark \textit{juiceshop}~\cite{owasp-juiceshop} is commonly adopted for evaluating web vulnerability testing. However, it does not touch the concept of privilege escalation, which is an essential aspect of penetration testing.
Second, existing benchmarks may not recognize the cumulative value of progress through the different stages of penetration testing, as they tend to evaluate only the final exploitation success. This approach overlooks the nuanced value each step contributes to the overall process, resulting in metrics that might not accurately represent actual performance in real-world scenarios.

To address these concerns, we propose the construction of a comprehensive penetration testing benchmark that meets the following criteria:

\noindent\textbf{Task Variety.} The benchmark must encompass diverse tasks, reflecting various operating systems and emulating the diversity of scenarios encountered in real-world penetration testing.

\noindent\textbf{Challenge Levels.} To ensure broad applicability, the benchmark must include tasks of varying difficulty levels suitable for challenging novice and expert testers.

\noindent\textbf{Progress Tracking.} Beyond mere success or failure metrics, the benchmark must facilitate tracking of incremental progress, thereby recognizing and scoring the value added at each stage of the penetration testing process.

\subsection{Benchmark Design}
\showkeysfalse
\label{sec:benchmark:design}
\showkeystrue
Following the criteria outlined previously, we develop a comprehensive benchmark that closely reflects real-world penetration testing tasks. The design process progresses through several stages.

\noindent\textbf{Task Selection.} Our first step is to meticulously select tasks from HackTheBox~\cite{htb} (HTB) and VulnHub~\cite{vulnhub}. These platforms are widely recognized and frequently utilized for penetration testing practice. Our selection process is guided by a desire to incorporate a diverse and challenging set of tasks. Capture The Flag (CTF) exercises and real-world testing scenarios have been included. The targets are drawn from various operating systems and encompass a broad spectrum of vulnerabilities. This approach ensures a wide representation of real-world penetration testing tasks. To account for different skill levels, the selected tasks cover a broad range of difficulty. While HTB and VulnHub offer reference difficulty levels, we further validate these with input from three certified penetration testers\footnote{Our penetration testers are all Offensive Security Certified Professionals (OSCP).}, including the authors of this work. This collaborative process yields a consensus on the final difficulty rating for each target, aligning with the conventional categorization~\cite{vulnhub} of penetration testing machines into \textit{easy}, \textit{medium}, and \textit{hard} levels.
It is worth noting that our benchmark does not explicitly include benign targets for evaluating false positives. This is because the iterative and exploratory nature of penetration testing inherently involves investigating services within the target that may ultimately be deemed benign. In this process, our primary focus is successfully identifying genuine vulnerabilities.

\noindent\textbf{Task Decomposition.} We further parse the testing process of each target into a series of sub-tasks, following the standard solution commonly referred to as the ``walkthrough'' in penetration testing. Each sub-task corresponds to a unique step in the overall process. Specifically, a sub-task may represent a micro-step involving the use of a particular penetration testing tool (e.g., performing port scanning with nmap~\cite{Nmap}) or exploiting a unique vulnerability identified in the Common Weakness Enumeration (CWE)~\cite{cwe} (e.g., exploiting SQL injection).
To standardize decomposition, we arrange the sub-tasks into a two-layer structure. Initially, we categorize each sub-task according to the five phases of penetration testing, as described in Section~\ref{sec:background}. Then, we label the sub-task with either the corresponding CWE item it targets or the specific tools employed. These two steps enable us to formulate an exhaustive list of sub-tasks for every benchmark target. We include this list in Appendix \ref{chapter:ml_appendix}, and the complete sub-task information is accessible on our anonymous open-source project~\cite{anonymous-project}.

\noindent\textbf{Benchmark Validation.} The final stage of our benchmark development involves rigorous validation. This step ensures that our benchmark accurately reflects real-world penetration testing scenarios and offers reproducibility. During validation, three certified penetration testers independently attempt the penetration testing targets, refining the sub-tasks as needed. We adjust our task decomposition accordingly because some targets may have multiple valid solutions.

By the end, we compile a benchmark of 13 penetration testing targets with 182 sub-tasks in 25 categories. The benchmark includes all types of vulnerabilities as listed in the OWASP~\cite{owasp} Top 10 Project. Detailed information on the included categories is listed in the Appendix \ref{chapter:ml_appendix}. To contribute to community development, we have made this benchmark publicly available online at our anonymous project website~\cite{anonymous-project}.


\section{Exploratory Study}
\showkeysfalse
\label{sec:exploratory-study}
\showkeystrue

We conduct an exploratory study to assess the capabilities of LLMs in penetration testing. Our primary objective is determining how well LLMs can adapt to the real-world complexities and challenges associated with penetration testing tasks. Specifically, we aim to address the following two research questions:

\noindent\textbf{RQ1 (Capability):} To what extent can LLMs perform penetration testing tasks? 

\noindent\textbf{RQ2 (Comparative Analysis):} How do the problem-solving strategies of human penetration testers and LLMs differ?

We utilize the benchmark described in Section~\ref{sec:benchmark} to evaluate the performance of LLMs on penetration testing tasks. In the following, we first delineate our testing strategy for this study. Subsequently, we present the testing results and an analytical discussion to address the above research questions.

\subsection{Testing Strategy}
\showkeysfalse
\label{sec:exploratory-study:testing-strategy}
\showkeystrue

LLMs cannot perform penetration tests directly. Their capabilities are primarily text-based, responding to queries and providing suggestions. However, penetration testing often involves operations with user interfaces (UI) and understanding graphical information, such as website images. This necessitates a bridge between the test machine and the LLM to facilitate task completion.

We introduce an interactive loop structure to evaluate the LLM's abilities in penetration testing within our benchmark. This process, depicted in Figure~\ref{figure:testing_strategy}, consists of the following stages: \ding{182} We present the target information to the LLM and request recommendations for penetration testing actions. This initiates a looped testing procedure.
\ding{183} We implement the actions suggested by the LLM, which encompass both terminal commands and graphical interactions.
\ding{184} We gather the results of the actions. Text-based output, such as terminal responses or source code, is recorded directly. Human penetration testers provide concise summaries and descriptions for non-textual results (e.g., images). The summarized information is returned to the LLM to inform subsequent actions. 
\ding{185} This cycle continues until we identify a solution or reach a standstill. We compile a record of the testing procedures, encompassing successful tasks, ineffective actions, and any reasons for failure, if applicable.

\begin{figure}[t]
	\centering
	\includegraphics[width=\linewidth]{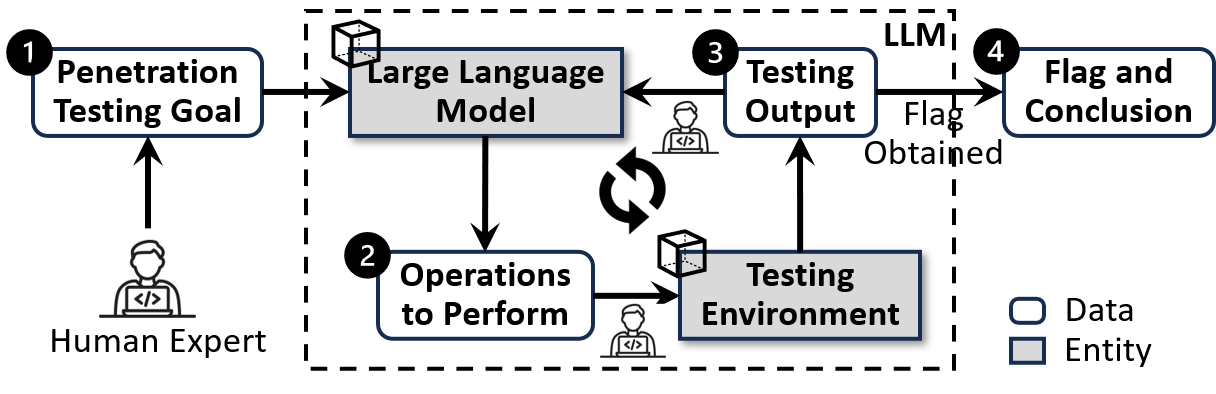}
 \caption{Overview of strategy to use LLMs for penetration testing. Reproduced from \cite{deng2023pentestgpt}.}
	\label{figure:testing_strategy}
\end{figure}

\subsection{Evaluation Settings}
\showkeysfalse
\label{sec:study-setting}
\showkeystrue
We proceed to assess the performances of various LLMs in penetration testing tasks using the strategy mentioned above.

\noindent\textbf{Model Selection.} Our study focuses on three cutting-edge LLMs that are currently accessible: \gptthree{} and \gptfour{} from OpenAI and LaMDA~\cite{Collins_2021} from Google. These models are selected based on their prominence in the research community and consistent availability. To interact with the LLMs mentioned above, we utilize chatbot services provided by OpenAI and Google, namely ChatGPT~\cite{chatgpt} and Bard~\cite{bard}. For this chapter, the terms \gptthree{}, \gptfour{}, and Bard represent these three LLMs.

\noindent\textbf{Experimental Setup.} We conduct our experiments in a local environment where the target and testing machines are part of the same private network. The testing machine operates on Kali Linux~\cite{kali}, version 2023.1. Several measures are implemented to validate the effectiveness of our testing procedures. First, we repeat the tests to account for inherent variability in the LLM outputs. In particular, we test each target with each LLM five times. We performed 195 tests in total, i.e., 5 repetitions * 3 models * 13 targets, wherein the 13 targets correspond to the \emph{Easy}, \emph{Medium} and \emph{Hard} target scenarios\footnote{\emph{Easy}, \emph{Medium} and \emph{Hard} is the categorization convention used in the security CTF community to represent the difficulty for a human to solve such challenge.}.
In this process, a sub-task is considered successful if it succeeds in at least one trial, and a penetration task is considered successful as long as one trial succeeds. Second, we make the best efforts to translate UI operations and graphical information into natural languages accurately. In addition, we ensure the precise execution of the instructions provided by the LLMs. Third, we maintain the integrity of the testing process by strictly limiting the tester's role to executing actions and reporting results without adding expert knowledge or guidance. Finally, the testing and target machines are rebooted after each test to reset their states, ensuring a consistent starting point for each test.

\noindent\textbf{Tool Usage.} Our study aims to assess the innate capabilities of LLMs without reliance on automated vulnerability scanners such as Nexus~\cite{nexus} and OpenVAS~\cite{rahalkar2019openvas}. Consequently, we explicitly instruct the LLMs to refrain from using these tools. However, we follow the LLMs' recommendations for utilizing other tools designed to validate specific vulnerability types (e.g., sqlmap~\cite{sqlmap} for SQL injections). Occasionally, versioning discrepancies may lead the LLMs to provide incorrect instructions for tool usage. In such instances, our penetration testing experts evaluate whether the instructions would have been valid for a previous version of the tool. They then make any necessary adjustments to ensure the tool's correct operation.

\subsection{Capability Evaluation (RQ1)}

\begin{table}[th]
\centering
\scalebox{0.4}{
    \begin{tabular}{c|cc|cc|cc|cc} 
    \rowcolor{black}
    \multicolumn{1}{c}{ }        & \multicolumn{2}{c}{\color{white}\textbf{Easy}}                                           & \multicolumn{2}{c}{\color{white}\textbf{Medium}}                                           & \multicolumn{2}{c}{\color{white}\textbf{Hard}}                                       & \multicolumn{2}{c}{\color{white}\textbf{Total}}                                           \\ 
    Tools                       & Overall (7)                      & Sub-task (77)                   & Overall (4)                      & Sub-task (71)                     & Overall (2)                   & Sub-task (34)                  & \multicolumn{1}{c}{Overall (13)} & \multicolumn{1}{c}{Sub-task (182)} \\ \midrule
    \gptthree{}                      & 1 (14.29\%)                       & 24 (31.17\%)                     & 0 (0.00\%)                        & 13 (18.31\%)                       & 0 (0.00\%)                     & 5 (14.71\%)                     & 1 (7.69\%)                        & 42 (23.07\%)                        \\
    \gptfour{}                        & 4 (57.14\%)                       & 52 (67.53\%)                     & 1 (25.00\%)                       & 27 (38.03\%)                       & 0 (0.00\%)                     & 8 (23.53\%)                     & 5 (38.46\%)                       & 87 (47.80\%)                        \\
    Bard                        & 2 (28.57\%)                       & 29 (37.66\%)                     & 0 (0.00\%)                        & 16 (22.54\%)                       & 0 (0.00\%)                     & 5 (14.71\%)                     & 2 (15.38\%)                       & 50 (27.47\%)                        \\ \midrule
    \multicolumn{1}{l}{Average} & \multicolumn{1}{l}{2.3 (33.33\%)} & \multicolumn{1}{l}{35 (45.45\%)} & \multicolumn{1}{l}{0.33 (8.33\%)} & \multicolumn{1}{l}{18.7 (26.29\%)} & \multicolumn{1}{l}{0 (0.00\%)} & \multicolumn{1}{l}{6 (17.64\%)} & 2.7 (20.5\%)                      & 59.7 (32.78\%) \\
    \bottomrule
    \end{tabular}}
\caption{Overall performance of LLMs on Penetration Testing Benchmark.}    
\label{table:pentestgpt_overall_performance}
\end{table}

To study \textbf{RQ1}, we begin by assessing the overall performance of three prominent LLMs: \gptfour{}, Bard, and \gptthree{}. The results of these evaluations are compiled in Table~\ref{table:pentestgpt_overall_performance}. The experimental results show that the three LLMs completed at least one end-to-end penetration testing task. This achievement underscores their ability to conduct a broad spectrum of testing operations, particularly within environments of less complexity. 
Among the models, \gptfour{} stands out with superior performance, achieving success with 4 targets of easy difficulty and 1 of medium difficulty. Bard and \gptthree{} also demonstrate commendable capabilities, completing 2 and 1 targets of easy difficulty, respectively. When examining sub-tasks, \gptfour{} accomplishes 52 of 77 on easy difficulty targets and 27 out of 71 on medium ones, underlining its potential for significant contributions to more complex penetration testing scenarios. Though not as proficient as \gptfour{}, \gptthree{} and Bard still show promise, completing 13 (18.31\%) and 16 (22.54\%) of sub-tasks on medium difficulty targets, respectively. 
However, the performance of all three models noticeably diminishes when challenged with hard difficulty targets. While each model can complete the initial reconnaissance phase on these targets, they fall short in exploiting the identified vulnerability. This outcome is not entirely unexpected, as the hard difficulty machines are deliberately crafted to be exceedingly difficult. They often include services that appear vulnerable but are, in fact, non-exploitable—a trait commonly referred to as \textit{rabbit holes}~\cite{yeo2013using}. Additionally, the routes to successfully exploiting these machines are typically inventive and unforeseeable, making them resistant to straightforward replication by automated tools. For instance, the benchmark target \textit{Falafel} involves deliberately crafted SQL injection vulnerabilities, which are resistant to \textit{sqlmap} and can only be exploited through manually designed payloads. Existing LLMs do not exhibit the capability to solve them solely without the guidance of human experts. 

\begin{tcolorbox}[left=1mm, right=1mm, top=0.5mm, bottom=0.5mm, arc=1mm]
\textbf{Finding 1:} Large Language Models (LLMs) have shown proficiency in conducting end-to-end penetration testing tasks but struggle to overcome challenges presented by more difficult targets.
\end{tcolorbox}

\begin{table}[ht]
\resizebox{\linewidth}{!}{
\begin{tabular}{l|cccc}

\rowcolor{black}
\color{white}\textbf{Sub-Tasks}                    & \color{white}\textbf{Walkthrough} & \color{white}\textbf{\gptthree{}} & \color{white}\textbf{\gptfour{}} & \color{white}\textbf{Bard} \\ 

General Tool Usage           & 25          & 4      & 10   & 7    \\ 
Port Scanning                & 9           & 9      & 9    & 9    \\
Web Enumeration              & 18          & 4      & 8    & 4    \\
Code Analysis                & 18          & 4      & 5    & 4    \\
Shell Construction           & 11          & 3      & 7    & 4    \\
Directory Exploitation       & 11          & 1      & 7    & 1    \\
General Privilege Escalation & 8           & 2      & 4    & 3    \\
Flag Capture                 & 8           & 1      & 5    & 2    \\
Password/Hash Cracking       & 8           & 2      & 4    & 2    \\
Network Exploitation         & 7           & 1      & 3    & 2    \\ \bottomrule
\end{tabular}            
}
\caption{Top 10 Types of Sub-tasks completed by each tool.}
\label{table:pentestgpt_sub_tasks}
\end{table}

We further examine the detailed sub-task completion performances of the three LLMs, as presented in Table~\ref{table:pentestgpt_sub_tasks}. Analyzing the completion status, we identify several areas where LLMs excel. First, they adeptly utilize common penetration testing tools to interpret the corresponding outputs, especially in enumeration tasks correctly. For example, all three evaluated LLMs successfully perform all nine \textit{Port Scanning} sub-tasks. They can configure the widely-used port scanning tool, \textit{nmap}~\cite{Nmap}, comprehend the scan results, and formulate subsequent actions.
Second, the LLMs reveal a deep understanding of prevalent vulnerability types, connecting them to the services on the target system. This understanding is evidenced by the successful completion of sub-tasks related to various vulnerability types.
Finally, LLMs demonstrate their effectiveness in code analysis and generation, particularly in the tasks of \textit{Code Analysis} and \textit{Shell Construction}. These tasks require the models to read and generate codes in different programming languages, essential in penetration testing. This often culminates in identifying potential vulnerabilities from code snippets and crafting the corresponding exploits. Notably, \gptfour{} outperforms the other two models regarding code interpretation and generation, marking it the most suitable candidate for penetration testing tasks.

\begin{tcolorbox}[left=1mm, right=1mm, top=0.5mm, bottom=0.5mm, arc=1mm]
\textbf{Finding 2:}  LLMs can efficiently use penetration testing tools, identify common vulnerabilities, and interpret source codes to identify vulnerabilities. 
\end{tcolorbox}

\subsection{Comparative Analysis (RQ2)}

\begin{table}[tb]
\centering
\resizebox{\linewidth}{!}{
    \begin{tabular}{l|cccc}
    \rowcolor{black}
    \color{white}\textbf{Unnecessary Operations} & \color{white}\textbf{\gptthree{}} & \color{white}\textbf{\gptfour{}} & \color{white}\textbf{Bard} & \color{white}\textbf{Total} \\ 
    Brute-Force           & 75                         & 92                       & 68                       & 235                       \\
    CVE Study             & 29                         & 24                       & 28                       & 81                        \\
    SQL Injection         & 14                         & 21                       & 16                       & 51                        \\
    Command Injection     & 18                         & 7                        & 12                       & 37              \\ \bottomrule       
    \end{tabular}}
\caption{Top Unnecessary Operations Prompted by LLMs on the Benchmark Targets. Experiments conducted on June 2023.}    
\label{table:pentestgpt_unnecessary}
\end{table}

To address \textbf{RQ2}, we examine the problem-solving strategies that LLMs employ, contrasting them with human penetration testers. In each penetration testing trial, we concentrate on two main aspects: (1) Identifying the unnecessary operations that LLMs prompt, which are not conducive to successful penetration testing, as compared to a standard walkthrough; and (2) Understanding the specific factors that prevent LLMs from successfully executing penetration tests.

We analyze the unnecessary operations prompted by LLMs by breaking down the recorded testing procedures into sub-tasks. We employ the same method to formulate benchmark sub-tasks, as Section~\ref{sec:benchmark} outlines. By comparing this to a standard walkthrough, we identify the primary sub-task trials that fall outside the standard walkthrough and are thus irrelevant to the penetration testing process. The results are summarized in Table~\ref{table:pentestgpt_unnecessary}. We find that the most prevalent unnecessary operation prompted by LLMs is brute force. For all services requiring password authentication, LLMs typically advise brute-forcing it. This is an ineffective strategy in penetration testing. We surmise that many hacking incidents in enterprises involve password cracking and brute force. LLMs learn these reports from accident reports and are consequently considered viable solutions. Besides brute force, LLMs suggest that testers engage in CVE studies, SQL injections, and command injections. These recommendations are common, as real-world penetration testers often prioritize these techniques, even though they may not always provide the exact solution.

\begin{table}[tb]
\resizebox{\linewidth}{!}{
    \begin{tabular}{l|cccc}
    \rowcolor{black}
    \color{white}\textbf{Failure Reasons}              & \color{white}\textbf{GPT3.5} & \color{white}\textbf{GPT4} & \color{white}\textbf{Bard} & \color{white}\textbf{Total} \\ 
    Session context lost                  & 25              & 18            & 31            & 74             \\
    False Command Generation & 23              & 12            & 20            & 55             \\
    Deadlock operations                   & 19              & 10            & 16            & 45             \\
    False Scanning Output Interpretation   & 13              & 9             & 18            & 40             \\
    False Source Code Interpretation       & 16              & 11            & 10            & 37             \\
    Cannot craft valid exploit            & 11              & 15            & 8             & 34           \\ \bottomrule 
    \end{tabular}
}
\caption{Top causes for failed penetration testing trials}
\label{table:pentestgpt_failure_reasons}
\end{table}

We further investigate the reasons behind the failure of penetration testing trials. We manually categorize the causes of failure for the 195 penetration testing trials, with the results documented in Table~\ref{table:pentestgpt_failure_reasons}. This table reveals that the predominant cause of failure is the loss of session context. The three examined models face difficulties in maintaining long-term conversational memory uniformly, frequently forgetting previous test results as the dialogue progresses. This lack of retention may be attributable to the limited token size within the LLM conversation context. Given the intricate nature of penetration testing—where a tester must skillfully link minor vulnerabilities across different services to develop a coherent exploitation strategy—this loss of context substantially undermines the models' effectiveness.

\begin{tcolorbox}[left=1mm, right=1mm, top=0.5mm, bottom=0.5mm, arc=1mm]
\textbf{Finding 3:} LLMs struggle to maintain long-term memory, which is vital to link vulnerabilities and develop exploitation strategies effectively.
\end{tcolorbox}

Secondly, LLMs strongly prefer the most recent tasks, adhering rigorously to a depth-first search approach. They concentrate on exploiting the immediate service, rarely deviating to a new target until all potential paths for the current one have been pursued. This can be attributed to the attention of LLMs focusing more on the beginning and end of the prompt, as revealed in~\cite{vaswani2023attention}.
Experienced penetration testers generally assess the system from a broader standpoint, strategizing the subsequent steps likely to provide the most substantial results. When combined with the aforementioned memory loss issue, this tendency causes LLMs to become overly fixated on a specific service. As the test progresses, the models completely forget previous findings and reach a deadlock.

\begin{tcolorbox}[left=1mm, right=1mm, top=0.5mm, bottom=0.5mm, arc=1mm]
\textbf{Finding 4:}  LLMs strongly prefer recent tasks and a depth-first search approach, often resulting in an over-focus on one service and forgetting previous findings.
\end{tcolorbox}

Lastly, LLMs have inaccurate result generation and hallucination issues, as noted in~\cite{bang2023multitask}. This phenomenon ranks as the second most frequent cause of failures and is characterized by the generation of false commands. In our study, we observe that LLMs frequently identify the appropriate tool for the task but stumble in configuring the tools with the correct settings. In some cases, they even concoct non-existent testing tools or tool modules.

\begin{tcolorbox}[left=1mm, right=1mm, top=0.5mm, bottom=0.5mm, arc=1mm]
\textbf{Finding 5:}  LLMs may generate inaccurate operations or commands, often stemming from inherent inaccuracies and hallucinations.
\end{tcolorbox}

Our exploratory study of three LLMs within penetration testing reveals their potential for executing end-to-end tasks. Nevertheless, challenges arise in maintaining long-term memory, devising a testing strategy beyond a depth-first approach, and generating accurate operations. In the following section, we elucidate how we address these challenges and outline our strategy for designing our LLM-powered penetration testing tool.


\section{Methodology}
\showkeysfalse
\label{sec:methodology}
\showkeystrue

\subsection{Overview}
\showkeysfalse
\label{sec:methodology:overview}
\showkeystrue
In light of the challenges identified in the preceding section, we present our proposed solution, \tool{}, which leverages the synergistic interplay of three LLM-powered modules. As illustrated in Figure~\ref{figure:pentestgpt_overview}, \tool{} incorporates three core modules: the \textbf{Reasoning Module}, the \textbf{Generation Module}, and the \textbf{Parsing Module}. Each module reserves one LLM session with its conversation and context. The user interacts seamlessly with \tool{}, where distinct modules process different types of messages. This interaction culminates in a final decision, suggesting the subsequent step of the penetration testing process that the user should undertake. 
In the following sections, we elucidate our design reasoning and provide a detailed breakdown of the engineering processes behind \tool{}.

\begin{figure}[h]
	\centering
	\includegraphics[width=\linewidth]{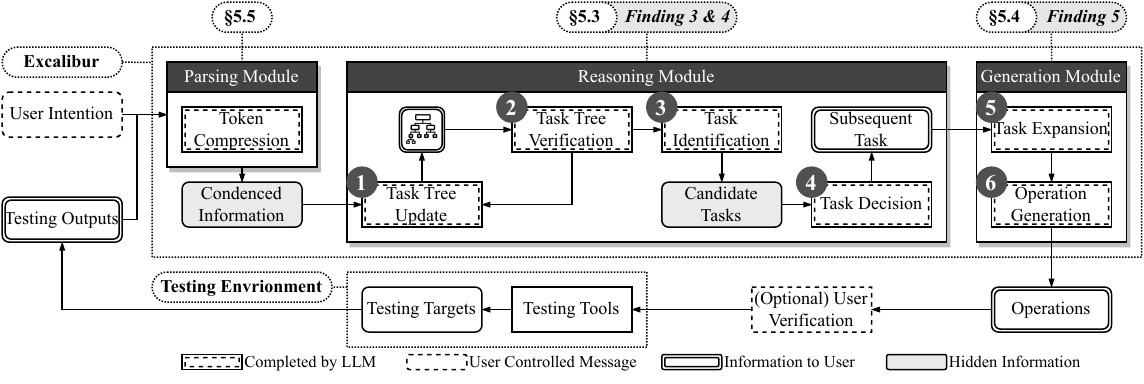}
 \caption{Overview of \tool{}. Reproduced from \cite{deng2023pentestgpt}.}
	\label{figure:pentestgpt_overview}
\end{figure}

\subsection{Design Rationale}
Our central design considerations emerged from the three challenges observed in the previous Exploratory Study (Section~\ref{sec:exploratory-study}):
The first challenge (\textit{Finding 3}) pertains to the issue of penetration testing context loss due to memory retention. LLMs in their original form struggle to maintain such long-term memory due to token size limits.
The second obstacle (\textit{Finding 4}) arises from the LLM chatbots' tendency to emphasize recent conversation content. In penetration testing tasks, this focuses on optimizing the immediate task. This approach falls short in the complex, interconnected task environment of penetration testing.
The third obstacle (\textit{Finding 5}) is tied to the inaccurate results generation by LLMs. When tasked to produce specific operations for a step in penetration testing directly, the outputs are often imprecise, sometimes even leading to 

\tool{} has been engineered to address these challenges, rendering it more apt for penetration testing tasks. We drew inspiration from the methodologies employed by real-world penetration testing teams, where a director plans overarching procedures, subdividing them into subtasks for individual testers. Each tester independently performs their task, reporting results without an exhaustive understanding of the broader context. The director then determines the following steps, possibly redefining tasks, and triggers the subsequent round of testing. Essentially, the director manages the overall strategy without becoming entrenched in the minutiae of the tests. This approach is mirrored in \tool{}'s functionality, enhancing its efficiency and adaptability in conducting penetration tests.
Our strategy divides penetration testing into two processes: identifying the next task and generating the concrete operation to complete the task. Each process is powered by one LLM session. In this setup, the LLM session responsible for task identification retains the complete context of the ongoing penetration testing status. At the same time, the generation of detailed operations and parsing of information is managed by other sessions. This division of responsibilities fosters effective task execution while preserving the overarching context.

To assist LLMs in effectively carrying out penetration testing tasks, we design a series of prompts that align with user inputs. We utilize the Chain-of-Thought (CoT)~\cite{wei2023chainofthought} methodology during this process. As CoT reveals, LLMs' performance and reasoning capabilities can be significantly enhanced using the \textit{input}, \textit{chain-of-thought}, \textit{output} prompting format. Here, the \textit{chain-of-thought} represents a series of intermediate natural language reasoning steps leading to the outcome. We dissect the penetration testing tasks into micro-steps and design prompts with examples to guide LLMs through processing penetration testing information step-by-step, ultimately leading to the desired outcomes. The complete prompts are available at our anonymized open-source project~\cite{anonymous-project}.

\subsection{Reasoning Module}
The \textbf{Reasoning Module} plays a pivotal role in our system, analogous to a team lead overseeing the penetration testing task from a macro perspective. It obtains testing results or intentions from the user and prepares the testing strategy for the next step. This testing strategy is passed to the generation module for further planning. 

To effectively supervise the penetration testing process and provide precise guidance, it is crucial to translate the testing procedures and outcomes into a natural language format. Drawing inspiration from the concept of an attack tree~\cite{LALLIE2020100219}, which is often used to outline penetration testing procedures, we introduce the notion of a \textit{pentesting task tree} (\pttab{}). This novel approach to testing status representation is rooted in the concept of an \textit{attributed tree}~\cite{BARBAR19933}:

\begin{myDef} [Attributed Tree] \label{definition:attributed_tree}
A \textit{attributed tree} is an edge-labeled, attributed polytree  $G=(V,E,\lambda,\mu)$ where $V$ is a set of nodes (or vertices), $E$ is a set of directed edges, $\lambda:E\to\Sigma$ is an edge labeling function assigning a label from the alphabet $\Sigma$ to each edge and $\mu:(V \cup E)\times K \to S$ is a function assigning key(from K)-value(from S) pairs of properties to the edges and nodes.
\end{myDef}

Given the definition of \textit{attributed tree}, \pttab{} is defined as follows:

\begin{myDef} [Pentesting Task Tree] \label{definition:ptt}
An \pttab{} $T$ is a pair $(N, A)$, wherein the correspond with:
(1) a set of nodes $N$ organized in a tree structure. Each node has a unique identifier, and there is a special node called the root that has no parent. Each node, other than the root, has exactly one parent and zero or more children.
(2) the function $A$ that assigns to each node $n \in N$ a set of attributes $A(n)$. The function $A(n)$ outputs a set of atribute pairs  $(a, v)$, where $a$ is the attribute name and $v$ is the attribute value. The set of attributes can be different for each node.
\end{myDef}


\begin{figure}[t]
	\centering
	\includegraphics[width=\linewidth]{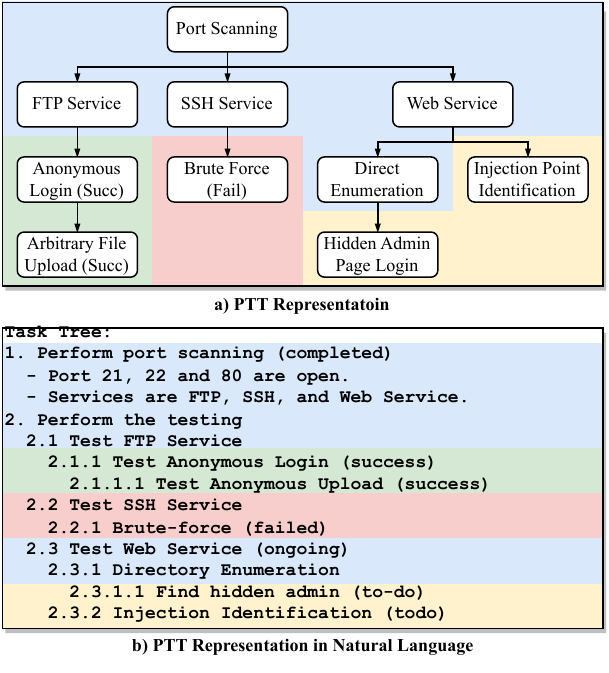}
 \caption{Pentesting Task Tree in a) visualized tree format, and b) natural language format encoded in LLM. Reproduced from \cite{deng2023pentestgpt}.}
	\label{figure:tree_natural_language}
\end{figure}

As outlined in Figure~\ref{figure:pentestgpt_overview}, the Reasoning Module's operation unfolds over four key steps operating over the PTT. \ding{182} Initially, the module absorbs the user's intentions to construct an initial PTT in the form of natural language. This is achieved by carefully instructing the LLM with examples and definitions of PPT using meticulously crafted prompts. The LLM outputs are parsed to confirm that the tree structure is accurately formatted. Note that due to the nature of the tree structure, it can be represented in the natural language format through layered bullets, as illustrated in Figure~\ref{figure:tree_natural_language}. The Reasoning Module effectively overcomes the memory-loss issue by maintaining a task tree that encompasses the entire penetration testing process.
\ding{183} After updating the tree information, a verification step is conducted on the newly updated PTT to ascertain its correctness. This process checks explicitly that only the leaf nodes of the PTT have been modified, aligning with the principle that atomic operations in the penetration testing process should only influence the status of the lowest-level sub-tasks. This step confirms the correctness of the reasoning process, safeguarding against any potential alterations to the overall tree structure due to hallucination by the LLM. If discrepancies arise, the information is reverted to the LLM for correction and regeneration. 
\ding{184} With the updated PTT, the Reasoning Module evaluates the current tree state and pinpoints viable sub-tasks that can serve as candidate steps for further testing.
\ding{185} Finally, the module evaluates the likelihood of these sub-tasks leading to successful penetration testing outcomes. It then recommends the top task as the output. The expected results of this task are subsequently forwarded to the Generation Module for an in-depth analysis. This is feasible, as demonstrated in the exploratory study, since LLMs, particularly \gptfour{}, can identify potential vulnerabilities when provided with system status information. This procedural approach enables the Reasoning Module to address one of the inherent limitations of LLMs, precisely their tendency to concentrate solely on the most recent task.
Note that in cases where the tester identifies that the correct task is incorrect or not completed in a preferred way, he could also manually revise the \pttab{} through the interactive handle further discussed in Section~\ref{sec:methodology:feedback}.

We devise four sets of prompts to sequentially guide the Reasoning Module through the completion of each stage. To bolster the reproducibility of our results, we optimize these prompts further with a technique known as hint generation~\cite{sun2023autohint}. From our practical experience, we observe that LLMs are adept at interpreting the tree-structured information pertinent to penetration testing and can update it accurately in response to test outputs.

\begin{figure}[h]
	\centering
	\includegraphics[width=\linewidth]{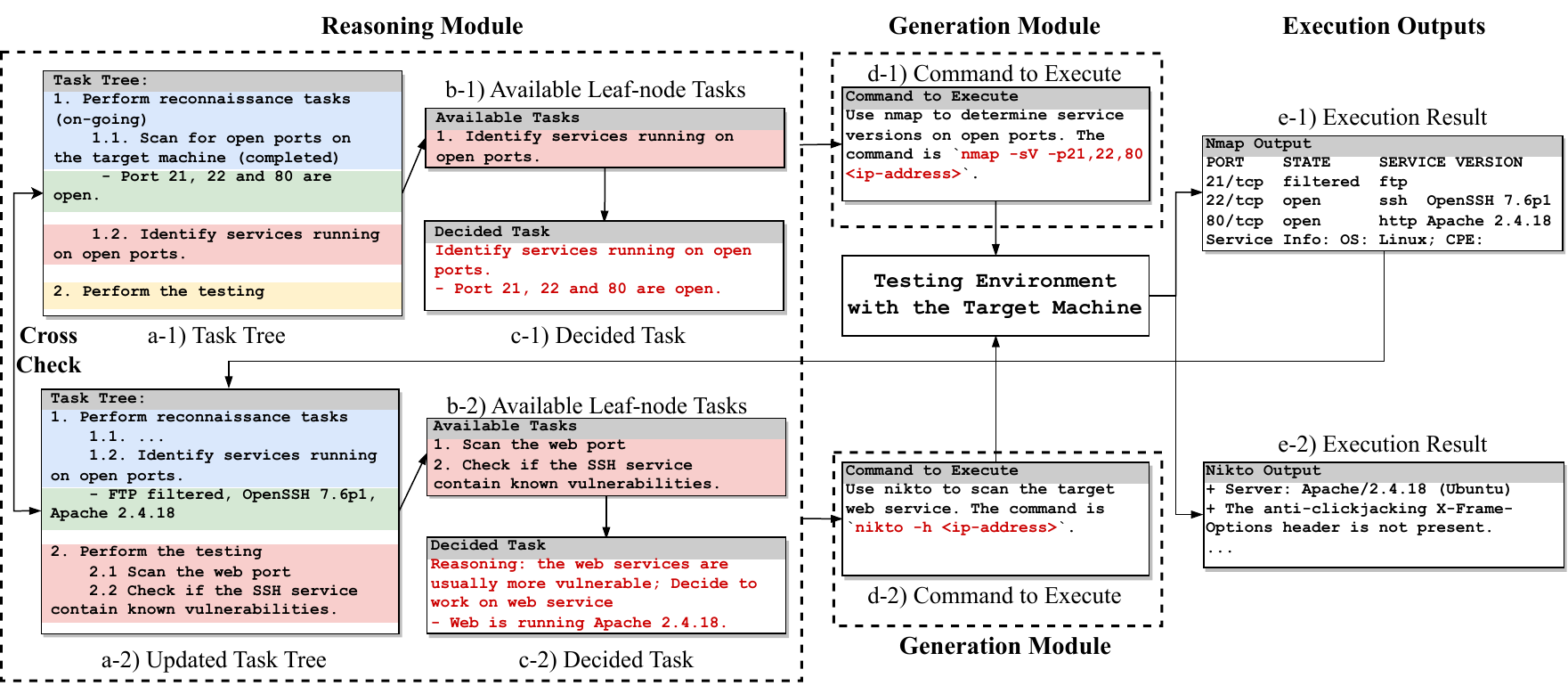}
 \caption{A demonstration of the task-tree update process on the testing target \textit{HTB-Carrier}. Reproduced from \cite{deng2023pentestgpt}.}
	\label{figure:task_tree_demonstration}
\end{figure}

\subsection{Generation Module}
The Generation Module translates specific sub-tasks from the Reasoning Module into concrete commands or instructions. Each time a new sub-task is received, a fresh session is initiated in the Generation Module. This strategy effectively isolates the context of the overarching penetration task from the immediate task under execution, enabling the LLM to focus entirely on generating specific commands.

Instead of directly transforming the received sub-task into specific operations, our design employs the CoT strategy~\cite{wei2023chainofthought} to partition this process into two sequential steps. This design decision directly addresses the challenges associated with model inaccuracy and hallucination by enhancing the model's reasoning capability. In particular, 
\ding{186} upon the receipt of a concise sub-task from the Reasoning Module, the Generation Module begins by expanding it into a sequence of detailed steps. Notably, the prompt associated with this sub-task requires the LLM to consider the possible tools and operations available within the testing environment.
\ding{187} Subsequently, the Generation Module transforms each of these expanded steps into precise terminal commands ready for execution or into detailed descriptions of specific Graphical User Interface (GUI) operations to be carried out. This stage-by-stage translation eliminates potential ambiguities, enabling testers to follow the instructions directly and seamlessly. Implementing this two-step process effectively precludes the LLM from generating operations that may not be feasible in real-world scenarios, thereby improving the success rate of the penetration testing procedure.

By acting as a bridge between the strategic insights provided by the Reasoning Module and the actionable steps required for conducting a penetration test, the Generation Module ensures that high-level plans are converted into precise and actionable steps. 
This transformation process significantly bolsters the overall efficiency of the penetration testing procedure.

\noindent\textbf{An Illustrative Example.}
We utilize a real-world running example to illuminate how the Reasoning Module and the Generation Module collaboratively operate to complete penetration testing tasks. Figure~\ref{figure:task_tree_demonstration} illustrates a single iteration of \tool{} working on the HackTheBox machine Carrier~\cite{carrier}, a medium-difficulty target. As depicted in a-1), the \pttab{}, in natural language format, encodes the testing status, revealing the open ports (21, 22,80) on the target machine. The Reasoning Module is subsequently instructed to identify the available tasks. As highlighted in red, service scanning is the only available task on the leaf node of the \pttab{}. This task is therefore chosen and forwarded to the Generation Module for command generation. The generated command is executed in the testing environment, and the execution result is conveyed to the Reasoning Module to update the \pttab{}. In a-2), the Reasoning Module integrates the previous scanning result into the \pttab{}, cross-referencing it with the earlier \pttab{} to update only the leaf nodes. It then looks for the available tasks to execute. In this case, two tasks emerge: scanning the web service on port 80 and checking the SSH service for known vulnerabilities. The LLM evaluates which task is more promising and chooses to investigate the web service, often seen as more vulnerable. This task is passed to the Generation Module. The Generation Module turns this general task into a detailed process, employing \textit{nikto}~\cite{nikto}, a commonly used web scanning script. The iterative process continues until the tester completes the penetration testing task.

\subsection{Parsing Module}
The \textbf{Parsing Module} operates as a supportive interface, enabling effective processing of the natural language information exchanged between the user and the other two core modules. Two needs can primarily justify the existence of this module. First, security testing tool outputs are typically verbose, laden with extraneous details, making it computationally expensive and unnecessarily redundant to feed these extended outputs directly into the LLMs. Second, users without specialized knowledge in the security domain may struggle to extract key insights from security testing outputs, presenting challenges in summarizing crucial testing information. Consequently, the Parsing Module is essential in streamlining and condensing this information.

In \tool{}, the Parsing Module is devised to handle four distinct types of information: (1) \textit{user intentions}, which are directives provided by the user to dictate the next course of action, (2) \textit{security testing tool outputs}, which represent the raw outputs generated by an array of security testing tools, (3) \textit{raw HTTP web information}, which encompasses all raw information derived from HTTP web interfaces, and (4) \textit{source codes} extracted during the penetration testing process. Users must specify the category of the information they provide, and each category is paired with a set of carefully designed prompts. For source code analysis, we integrate the \gptfour{} code interpreter~\cite{codeintepreter} to execute the task.

\subsection{Active Feedback}
\showkeysfalse
\label{sec:methodology:feedback}
\showkeystrue
While LLMs can produce insightful outputs, their outcomes may sometimes require revisions. To facilitate this, we introduce an interactive handle in \tool{}, known as active feedback, which allows the user to interact directly with the Reasoning Module. A vital feature of this process is that it does not alter the context within the Reasoning Module unless the user explicitly desires to update some information. The reasoning context, including the \pttab{}, is stored as a fixed chunk of tokens. This chunk of tokens is provided to a new LLM session during an active feedback interaction, and users can pose questions regarding them. This ensures that the original session remains unaffected, and users can always query the reasoning context without making unnecessary changes. If the user believes it necessary to update the \pttab{}, they can explicitly instruct the model to update the reasoning context history accordingly. This provides a robust and flexible framework for the user to participate in the decision-making process actively.

\subsection{Discussion}
We explore various design alternatives for \tool{} to tackle the challenges identified in Exploratory Study. We have experimented with different designs, and here we discuss some key decisions.

\textbf{Addressing Context Loss with Token Size:} a straightforward solution to alleviate context loss is the employment of LLM models with an extended token size. For instance, \gptfour{} provides versions with 8k and 32k token size limits. This approach, however, confronts two substantial challenges. First, even a 32k token size might be inadequate for penetration testing scenarios, as the output of a single testing tool like \textit{dirbuster}~\cite{dirbuster} may comprise thousands of tokens. Consequently, \gptfour{} with a 32k limit cannot retain the entire testing context. Second, even when the entire conversation history fits within the 32k token boundary, the API may still skew towards recent content, focusing on local tasks and overlooking broader context. These issues guided us in formulating the design for the Reasoning Module and the Parsing Module.

\textbf{Vector Database to Improve Context Length:} Another technique to enhance the context length of LLMs involves a vector database~\cite{2021milvus, 2022manu}. By transmuting data into vector embeddings, LLMs can efficiently store and retrieve information, practically creating long-term memory. Theoretically, penetration testing tool outputs could be archived in the vector database. In practice, though, we observe that many results closely resemble and vary in only nuanced ways. This similarity often leads to confused information retrieval. Solely relying on a vector database fails to overcome context loss in penetration testing tasks. Integrating the vector database into the design of \tool{} is an avenue for future research.

\textbf{Precision in Information Extraction:} Precise information extraction is crucial for conserving token usage and avoiding verbosity in LLMs. Rule-based methods are commonly employed to extract diverse information. However, rule-based techniques are engineeringly expensive given natural language's inherent complexity and the variety of information types in penetration testing. We devise the Parsing Module to manage several general input information types, a strategy found to be both feasible and efficient.

\textbf{Limitations of LLMs:} LLMs are not an all-encompassing solution. Present LLMs exhibit flaws, including hallucination~\cite{llm-hallucination} and outdated knowledge. Our mitigation efforts, such as implementing task tree verification to ward off hallucination, might not completely prevent the Reasoning Module from producing erroneous outcomes. Thus, a human-in-the-loop strategy becomes vital, facilitating the input of necessary expertise and guidance to steer LLMs effectively. Moreover, it must be higlighted that reproducibility of results is an issue here, as LLMs are in continuous evolution. In our case, given the novelty of the technology and lack of availability at the time of starting this work, we did not host the LLM deployments ourselves, but instead used third party providers. This setup is sub-optimal from a scientific experiment point of view. For future work, we foresee building our own LLMs by performing post-training phases on top of existing open source models and \emph{freezing} them after a certain stage of learning. This will lead to results being reproducible given the right input context, and fully reproducible on statistical average. This topic is further discussed in Section \ref{sec:discussion}.


\section{Evaluation}
\showkeysfalse
\label{sec:evaluation}
\showkeystrue
In this section, we assess the performance of \tool{}, focusing on the following four research questions:

\noindent\textbf{RQ3 (Performance):} How does the performance of \tool{} compare with that of native LLM models and human experts?

\noindent\textbf{RQ4 (Strategy):} Does \tool{} employ different problem-solving strategies compared to those utilized by LLMs or human experts?

\noindent\textbf{RQ5 (Ablation):} How does each module within \tool{} contribute to the overall penetration testing performance?

\noindent\textbf{RQ6 (Practicality):} Is \tool{} practical and effective in real-world penetration testing tasks?

\subsection{Evaluation Settings}
We implement \tool{} with 1,700 lines of Python3 code and 740 prompts, available at our anonymized project website~\cite{anonymous-project}. We evaluate its performance over the benchmark constructed in Section~\ref{sec:benchmark}. In this evaluation, we integrate \tool{} with \gptthree{} and \gptfour{} to form two working versions: \toolgptthree{} and \toolgptfour{}. Due to the lack of API access, we do not select other LLM models, such as Bard.
In line with our previous experiments, we use the same experiment environment setting and instruct \tool{} to only use the non-automated penetration testing tools.

\subsection{Performance Evaluation (RQ3)}
The overall task completion status of \toolgptthree{}, \toolgptfour{}, and the naive usage of LLMs is illustrated in Figure~\ref{subfigure:pentestgpt_overall}. 
As the Figure shows, our solutions powered by LLMs demonstrate superior penetration testing capabilities compared to the naive application of LLMs. Specifically, \toolgptfour{} surpasses the other three solutions, successfully solving 6 out of 7 easy difficulty targets and 2 out of 4 medium difficulty targets. This performance indicates that \toolgptfour{} can handle penetration testing targets ranging from easy to medium difficulty levels. Meanwhile, \toolgptthree{} manages to solve only two challenges of easy difficulty, a discrepancy that can be attributed to \gptthree{} lacking the knowledge related to penetration testing found in \gptfour{}.

The sub-task completion status of \toolgptthree{}, \toolgptfour{}, and the naive usage of LLM is shown in Figure~\ref{subfigure:pentestgpt_subtask}. As the Figure illustrates, both \toolgptthree{} and \toolgptfour{} perform better than the standard utilization of LLMs. It is noteworthy that \toolgptfour{} not only solves one more medium difficulty target compared to naive \gptfour{} but also accomplishes 111\% more sub-tasks (57 vs. 27). This highlights that our design effectively addresses context loss challenges and leads to more promising testing results. Nevertheless, all the solutions struggle with hard difficulty testing targets. As elaborated in Section~\ref{sec:exploratory-study}, hard difficulty targets typically demand a deep understanding from the penetration tester. To reach testing objectives, they may require modifications to existing penetration testing tools or scripts. Our design does not expand the LLMs' knowledge of vulnerabilities, so it does not notably enhance performance on these more complex targets.

\begin{figure}[t]
  \centering
      \begin{subfigure}[b]{\linewidth}
    \centering
    \includegraphics[width=\linewidth]{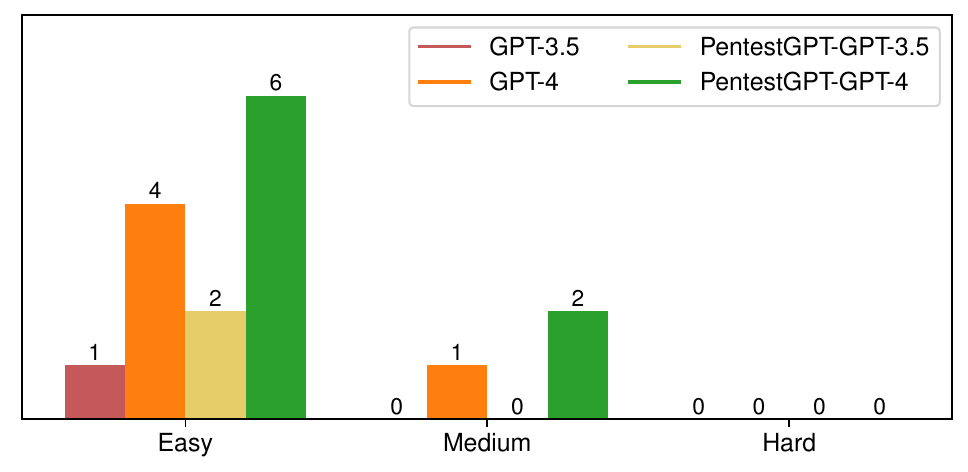}
 \caption{Overall completion status.}
	\label{subfigure:pentestgpt_overall}

  \end{subfigure}
  
      \begin{subfigure}[b]{\linewidth}
    \centering
    \includegraphics[width=\linewidth]{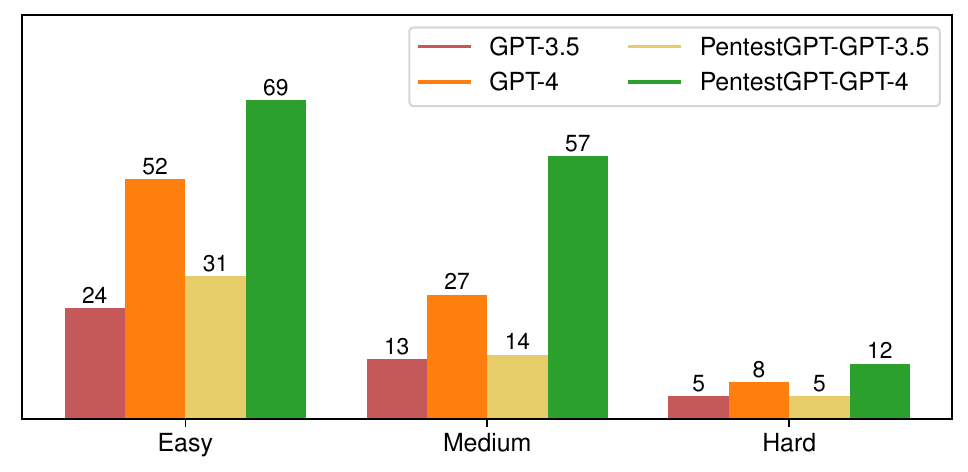}
 \caption{Subtask completion status.}
	\label{subfigure:pentestgpt_subtask}

  \end{subfigure}
  
  \caption{The performance of \gptthree{}, \gptfour{}, \toolgptthree{}, and \toolgptfour{} on overall target completion and sub-task completion. Reproduced from \cite{deng2023pentestgpt}.}
  \showkeysfalse
  \label{figure:pentestgpt_performance_eval}
  \showkeystrue
\end{figure}

\subsection{Strategy Evaluation (RQ4)}
We then investigate the problem-solving strategies employed by \tool{}, contrasting them with those of LLMs and human experts. By manually analyzing the penetration testing process of \tool{}, we synthesize its underlying approaches to problem-solving. We surprisingly find that \tool{} decomposes the penetration testing task in a manner akin to human experts, successfully achieving the overall goal. Instead of focusing solely on the most recently discovered task, \tool{} can pinpoint potential sub-tasks likely to lead to successful outcomes.

Figure~\ref{figure:pentest_strategy_hackable2} provides an illustrative example, demonstrating the strategic differences between \gptfour{} and \tool{} while handling the VulnHub machine, \textit{Hackable II}~\cite{hackable2}. This target comprises two vulnerable services: an FTP service allowing arbitrary file uploads and a web service enabling file viewing through FTP. A successful exploit necessitates exploiting both services by uploading a malicious PHP shell via the FTP service and triggering it through the web service. 
As depicted in the figure, \gptfour{} begins by enumerating the FTP service and successfully identifies the file upload vulnerability (\ding{182}-\ding{184}). However, it fails to correlate this with the web service, resulting in an incomplete exploit in the following steps. Conversely, \tool{} follows a more holistic approach, toggling between enumerating the FTP service and browsing the web service. In particular, \tool{} firstly \ding{182} enumerates the FTP service and \ding{183} web service to understand the general situation. It then \ding{184} prioritizes the FTP service, and \ding{185} eventually discovers the file upload vulnerability. More importantly, in this process, \tool{} identifies that files available on FTP are the same as those on the web service. By connecting these findings, \tool{} guides the tester to \ding{186} perform a shell upload, \ding{187} leading to a successful reverse shell. This strategy aligns with the walkthrough solution and highlights \tool{}'s comprehensive understanding of the penetration testing process and its ability to make effective decisions on the optimal sub-task to pursue next. This reveals \tool{}'s strategic thinking and ability to integrate different aspects of the testing process.

\begin{figure}[t]
	\centering
	\includegraphics[width=\linewidth]{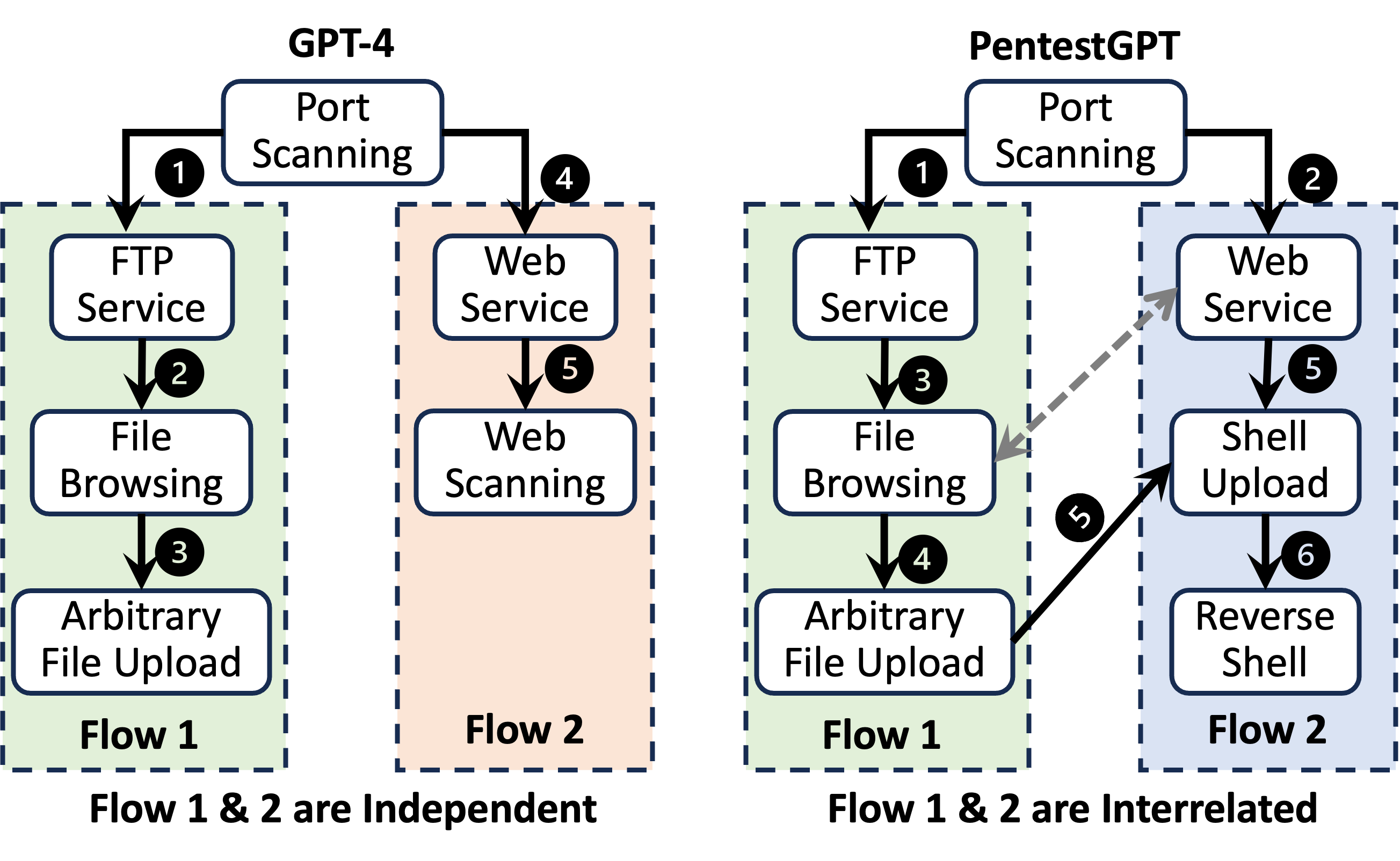}
 \caption{Penetration testing strategy comparison between \gptfour{} and \tool{} on \textit{VulnHub-Hackable II}. Reproduced from \cite{deng2023pentestgpt}.}
	\label{figure:pentest_strategy_hackable2}
\end{figure}

Our second observation is that although \tool{} behaves more similarly to human experts, it still exhibits some strategies that humans will not apply. For instance, \tool{} still prioritizes brute-force attacks before vulnerability scanning. This is obvious in cases where \tool{} always tries to brute-force the SSH service on target machines. 

We then analyze the failed penetration testing cases to understand the limitations of \tool{}. Beyond the absence of some advanced penetration testing techniques, two primary issues emerge. First, \tool{} struggles to interpret images. LLMs are limited to text comprehension, so they cannot accurately process images. This issue might be addressed by developing large multimodal models to understand text and visual data. Second, it cannot grasp certain social engineering tricks and subtle cues. For instance, real-world penetration testers often create brute-force wordlists using information gathered from the target service. Though \tool{} can retrieve a list of names from a web service, it fails to instruct the use of tools to create a wordlist from those names. These limitations underline the necessity for improvement in areas where human insight and intricate reasoning are still more proficient than automated solutions.

\subsection{Ablation Study (RQ5)}
We perform an ablation study on how the three modules: Reasoning Module, Generation Module, and Parsing Module, contribute to the performance of \tool{}. We implement three variants:
\begin{enumerate}
\item \toolnoparsing{}: the Parsing Module is deactivated, causing all data to be directly fed into the system.
\item \toolnogeneration{}: the Generation Module is deactivated, leading to the completion of task generation within the Reasoning Module itself. The prompts for task generation remain consistent.
\item \toolnoreasoning{}: the Reasoning Module is desabled. Instead of PTT, this variant adopts the same methodology utilized with LLMs for penetration testing, as delineated in the Exploratory Study.
\end{enumerate}
All the variants are integrated with \gptfour{} API for testing.

\begin{figure}[t]
  \centering
      \begin{subfigure}[b]{\linewidth}
    \centering
    \includegraphics[width=\linewidth]{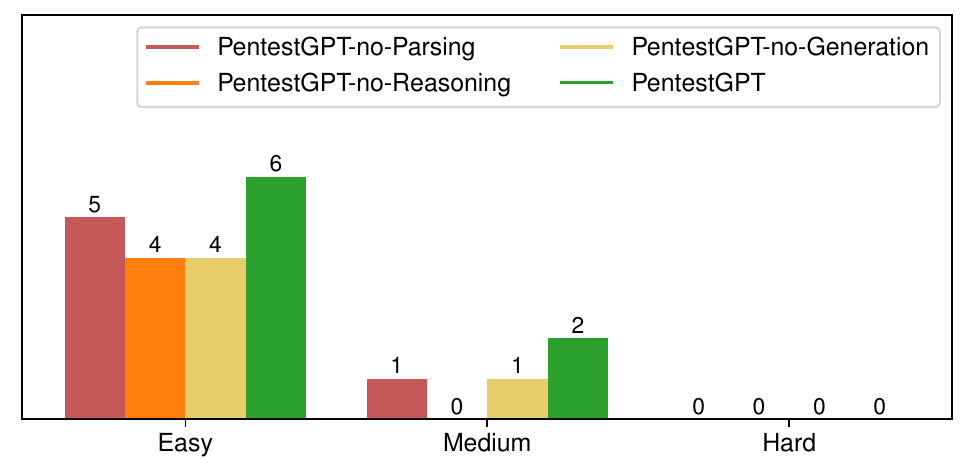}
 \caption{Overall completion status}
	\label{subfigure:ablation_overall}

  \end{subfigure}
  
      \begin{subfigure}[b]{\linewidth}
    \centering
    \includegraphics[width=\linewidth]{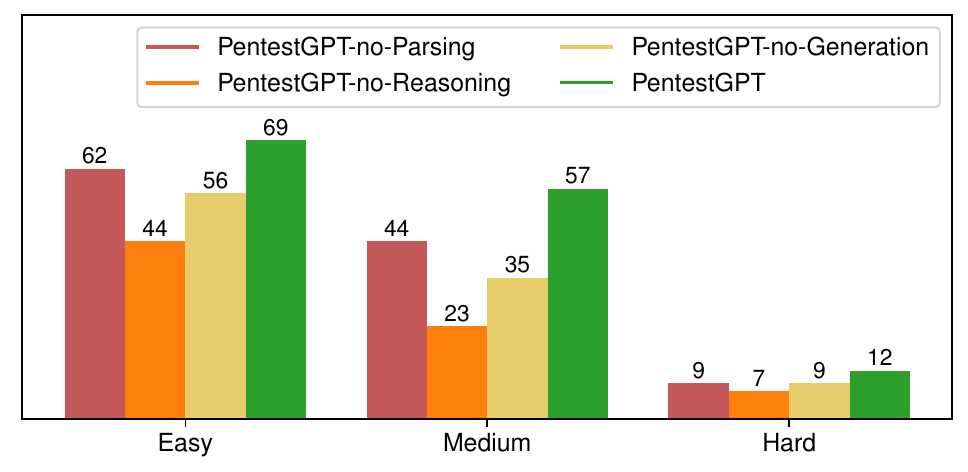}
 \caption{Sub-task completion status}
	\label{subfigure:ablation_subtask}

  \end{subfigure}
  
  \caption{The performance of \tool{}, \textsc{\tool{}-No-Annotation}, \textsc{\tool{}-Operation-Only}, and \textsc{\tool{}-Parameter-Only} on both normalized average code coverage ($\mu LOC$) and bug detection. Reproduced from \cite{deng2023pentestgpt}.}
      \vspace{-12pt}
  \label{figure:ablation_performance}
\end{figure}

The results of the three variants tested on our penetration testing benchmarks are depicted in Figure~\ref {figure:ablation_performance}. In general, \tool{} demonstrates superiority over the three ablation baselines regarding overall target and sub-task completion. Our key findings are as follows: (1) In the absence of the Parsing Module, \toolnoparsing{} attains marginally lower performance in overall task and sub-task completion relative to the full configuration. While parsing information is advantageous in penetration testing, the 32k token size limit often suffices for various outputs. Given the Reasoning Module's inherent design to maintain the entire testing context, the lack of the Parsing Module does not substantially impair the tool's performance.
(2)  \toolnoreasoning{} fares the worst, completing only 53.6\% of the sub-tasks achieved by the full solution, an outcome even inferior to the naive application of \gptfour{} in testing. We attribute this to the Generation Module adding supplementary sub-tasks to the LLM context. Since the prompts are not tailored for scenarios without the Reasoning Module, the resulting outputs are irrelevant for the naive LLM without the Generation Module. Furthermore, the extended generation output obscures the original context, hindering the LLM's ability to concentrate on the task, thus failing the test.
(3) \toolnogeneration{} realizes performance slightly above that of \gptfour{} employed naively. This occurs because, without the Generation Module, the testing procedure closely resembles the usage of LLMs. Notably, the Generation Module is principally intended to guide the tester in executing precise penetration testing operations. Without this module, the tester may depend on supplementary information to operate the tools or scripts essential for completing the test.

\subsection{Practicality Study (RQ6)}

\begin{table}[t]
\resizebox{\linewidth}{!}{
\begin{tabular}{|c|c|c|c|c|}
\rowcolor{black}
\hline
\color{white}\textbf{Machine}     & \color{white}\textbf{Difficulty} & \color{white}\textbf{Completion}            & \color{white}\textbf{Completed Users} & \color{white}\textbf{Cost (USD)} \\ \hline
Sau         & Easy       & \ding{51}                   & 4798                                       & 15.2                           \\
Pilgramage  & Easy       & \ding{51}                   & 5474                                       & 12.6                           \\
Topology    & Easy       & \ding{55}                   & 4500                                       & 8.3                            \\
PC          & Easy       & \ding{51}                   & 6061                                       & 16.1                           \\
MonitorsTwo & Easy       & \ding{51}                   & 8684                                       & 9.2                            \\
Authority   & Medium     & \ding{55}                   & 1209                                       & 11.5                           \\
Sandworm    & Medium     & \ding{55}                   & 2106                                       & 10.2                           \\
Jupiter     & Medium     & \ding{55}                   & 1494                                       & 6.6                            \\
Agile       & Medium     & \ding{51}                   & 4395                                       & 22.5                           \\
OnlyForYou  & Medium     & \ding{55}                   & 2296                                       & 19.3                           \\ \hline
Total       & -          & 6  & -                      & 131.5    \\ \hline
\end{tabular}
}
\caption{\tool{} performance over the active HackTheBox Challenges.}
\label{table:hackthebox_challenges}
\end{table}

We demonstrate that \tool{} exhibits practicality for real-world penetration testing beyond the crafted benchmark. For this purpose, we engage \tool{} in the HackTheBox active machine challenges, a series of penetration testing objectives open to global testers. Each challenge consists of two components: a user flag, retrievable upon initial user access, and a root flag, obtainable after gaining root access. Our evaluation encompasses five targets of easy difficulty and five of medium difficulty.
During this exercise, \tool{}, utilizing \gptfour{}'s 32k token API, conducts up to five tests on each target. Success is defined solely by the capture of the root flag. Table~\ref{table:hackthebox_challenges} details the performance of \tool{} in these challenges\footnote{\textbf{Completed Users} denotes the number of users globally who have completed the target as of the manuscript submission time. Note that HackTheBox boasts over 670,000 active users.}.
Ultimately, \tool{} completes three easy and five medium challenges. The total expenditure for this exercise amounts to 131.5 USD, averaging 21.92 USD per target. This cost is markedly lower than employing human penetration testers and falls within an acceptable range. Our evaluation, therefore, underscores \tool{}'s capability to yield viable penetration testing results in real-world settings at an efficient cost, thereby highlighting its potential as a practical tool in the cybersecurity domain.


\section{Discussion}
\showkeysfalse
\label{sec:discussion}
\showkeystrue

We recognize that the penetration testing walkthrough might have been part of the training material for the tested LLMs, potentially biasing the results. To mitigate this, we take two measures. First, we manually verify that the LLM does not have prior knowledge of the target machine. We do this by prompting the LLMs if the tested machine is within their knowledge base. Second, we include penetration testing target machines released after 2021 in our benchmark, which falls outside the training data of OpenAI models. The practicality study on the most recent HackTheBox challenges also demonstrates that \tool{} can solve challenges without prior knowledge of the target.

The rapidly evolving nature of LLMs and inconsistencies in available APIs could invalidate \tool{}'s designed prompts. We strive to make prompts general and suitable for various LLMs. However, due to their hacking nature, some LLMs resist generating specific penetration testing content, such as concrete reverse shell scripts. Our prompts include jailbreak techniques~\cite{liu2023jailbreaking} to guide the LLM to generate penetration-testing-related information. How to generate reproducible outcomes is an important direction we are working towards.

We identify hallucination in Large Language Models~\cite{llm-hallucination} as a significant challenge where the model's outputs diverge from its training data. This affects the reliability of our automatic penetration testing tool.
We are actively exploring various techniques~\cite{manakul2023selfcheckgpt} to reduce hallucination and enhance our tool's performance. As an ongoing work, we believe such an attempt will lead to a more robust and effective automatic penetration testing tool.


\section{Conclusion}
\showkeysfalse
\label{sec:conclusion}
\showkeystrue

%

This chapter paved the way towards integrating Large Language Models (LLMs) with the operational processes of cybersecurity, particularly in guiding cybersecurity software agents to perform autonomously and more specifically, automating and refining penetration testing strategies. The implementation of \tool{}, our innovative prototype, showcases the potential of LLMs to emulate human-like decision-making in cyber-attacks. While \tool{} demonstrates adeptness in fundamental cybersecurity tasks, it also reveals the limitations of current models, particularly in maintaining context and focusing attention over prolonged interactions.

\tool{}'s architecture, inspired by the dynamics of professional penetration testing teams, divides complex tasks into manageable segments, enabling a collaborative approach to threat analysis and response. This modular structure of Reasoning, Generation, and Parsing replicates the cognitive synergy found in human teams, ensuring a comprehensive and nuanced approach to cybersecurity challenges.

The insights gleaned from our evaluation of \tool{} suggest that while LLMs offer significant advantages, they are not a panacea. There remains a clear need for human intuition and expertise, particularly in complex and novel scenarios that exceed the training parameters of existing models. As we look to the future, the integration of LLMs in cybersecurity presents an exciting frontier with the potential to revolutionize the field. The continuous development of \tool{} and similar tools promises to enhance our defensive and offensive capabilities, making our digital infrastructures more resilient against the evolving threats they face. Our work lays a foundation for future research, where the symbiosis of human and machine intelligence will lead to unprecedented advancements in maintaining cybersecurity.

    %
    \chapter{Conclusion and future work}
\label{chapter:conclusion}

This thesis, titled "Offensive Robot Cybersecurity," presents a multifaceted exploration of safeguarding robots by adopting a preemptive hacking stance, using Game Theory and Machine Learning as primary tools for automation. The journey began with a motivation (\textbf{Part \ref{part:motivation}}), which provided a general introduction to the field and a description of the thesis structure and framework (Chapter \ref{chapter:intro}). This is followed by some general background (Chapter \ref{chapter:background}) which provided more context to the intersection of cybersecurity, robotics and artificial intelligence involved in this work. 

The content then continues with an in-depth ethical investigation (\textbf{Part \ref{part:ethics}}), setting the landscape (Chapter \ref{chapter:landscape}) for responsible and ethically-aligned cybersecurity practices in robotics. 

The development of fundamental blueprints (\textbf{Part \ref{part:blueprints}}) introduced a suite of core know-how, tools and frameworks, fostering a secure robotics research environment and empowering both defensive and offensive research. This includes DevSecOps for Robotics 
and ROS 2 security mechanisms (Chapter \ref{chapter:ros_security}), benchmarking robots and their security (Chapter \ref{chapter:benchmarking}) and reducing the timing overhead and computational bottlenecks introduced by security layers in robotics through hardware acceleration (Chapter \ref{chapter:accelerators}).

The heart of the thesis lies in the meticulous cyber-attacking of robots (\textbf{Part \ref{part:attacking_robots}}), a necessary evil that unveiled a plethora of robot vulnerabilities. This research was conducted attacking robot software (Chapter \ref{chapter:attacking_software}), hardware (Chapter \ref{chapter:attacking_hardware}) and industry deployments (Chapter \ref{chapter:attacking_industry}), providing invaluable insights into the current state of robot cybersecurity. 

The finale (\textbf{Part \ref{part:automation}}) of the thesis ventured into the realm of autonomous offensive cybersecurity strategies for robots, harnessing the analytical prowess of Game Theory (Chapter \ref{chapter:game_theory}) to determine the most vulnerable attack targets and vectors, followed by the use of modern Machine Learning approaches to guide the offensive behavior. In particular, Large Language Models  helped bring up adaptive intelligence to our \emph{cybersecurity cognitive engines} (Chapter \ref{chapter:machine_learning}).

\showkeysfalse

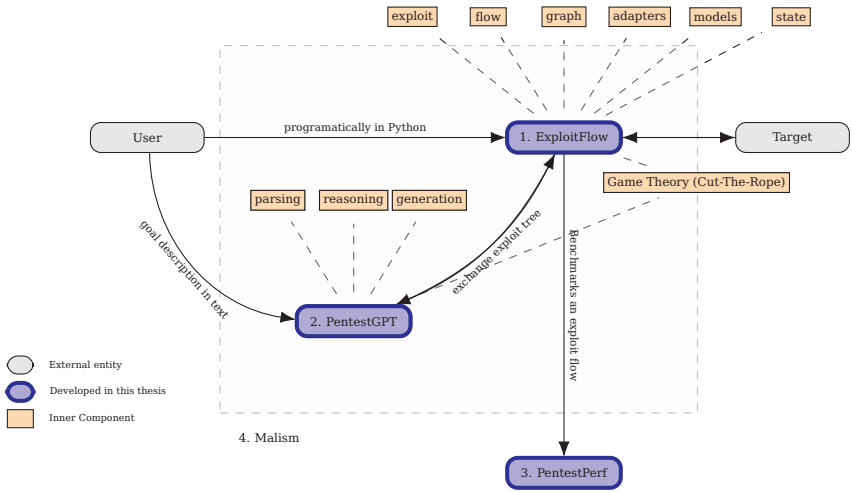
\begin{figure}[h!]
    \centering 
   
    \begin{tikzpicture}[node distance=2cm, auto, scale=0.5, transform shape]
        \tikzstyle{box} = [rectangle, draw, fill=gray!20, text centered, rounded corners, minimum height=2.5em, minimum width=3cm]
        \tikzstyle{emphasize} = [box, fill=blue!30, draw=blue, line width=1.5pt]
        \tikzstyle{line} = [draw, -{Latex[scale=1.2]}]
        \tikzstyle{deemphasize} = [opacity=0.3]
        \tikzstyle{innerComp} = [rectangle, draw, fill=orange!30, text centered, minimum height=1.5em, minimum width=3em]
        \tikzstyle{innerComp_emphasize} = [rectangle, fill=blue!30, draw=blue, line width=1.5pt]
    
        \node[box] (U) {User};
        \node[box, emphasize] (EF) at (11cm,0) {1. ExploitFlow}; 
        \node[box, emphasize, below left=4cm and 2.5cm of EF] (PentestGPT) {2. PentestGPT};
        \node[box, emphasize, below=8cm of EF] (PentestPerf) {3. PentestPerf};
        \node[box, right=3cm of EF] (T) {Target};
    
        \node[innerComp, above=2.5cm of EF, xshift=-4cm] (exploit) {exploit};
        \node[innerComp, right of=exploit] (flow) {flow};
        \node[innerComp, right of=flow] (graph) {graph};
        \node[innerComp, right of=graph] (adapters) {adapters};
        \node[innerComp, right of=adapters] (models) {models};
        \node[innerComp, right of=models] (state) {state};
        \node[innerComp, below=0.5cm of EF, xshift=3.5cm] (gt) {Game Theory (\CutTheRope)};
    
        \node[innerComp, above=2.5cm of PentestGPT, xshift=-2cm] (parsing) {parsing};
        \node[innerComp, right of=parsing] (reasoning) {reasoning};
        \node[innerComp, right of=reasoning] (generation) {generation};
    
        \draw[dashed, shorten <=5pt, shorten >=5pt] (EF) -- (exploit);
        \draw[dashed, shorten <=5pt, shorten >=5pt] (EF) -- (flow);
        \draw[dashed, shorten <=5pt, shorten >=5pt] (EF) -- (graph);
        \draw[dashed, shorten <=5pt, shorten >=5pt] (EF) -- (adapters);
        \draw[dashed, shorten <=5pt, shorten >=5pt] (EF) -- (models);
        \draw[dashed, shorten <=5pt, shorten >=5pt] (EF) -- (state);
        \draw[dashed, shorten <=5pt, shorten >=5pt] (EF) -- (gt);
        \draw[dashed, shorten <=5pt, shorten >=5pt] (PentestGPT) -- (gt);        
    
        \draw[dashed, shorten <=5pt, shorten >=5pt] (PentestGPT) -- (parsing);
        \draw[dashed, shorten <=5pt, shorten >=5pt] (PentestGPT) -- (reasoning);
        \draw[dashed, shorten <=5pt, shorten >=5pt] (PentestGPT) -- (generation);

        \node[rectangle, fill=gray!5, draw, dashed, deemphasize, fit=(EF) (PentestGPT), inner sep=1cm] (group) {};
        \node[anchor=north west, inner sep=0.5cm] at (group.south west) {4. Malism};
    
        \path[line] (U) -- node[font=\small, midway, sloped, above] {programatically in Python} (EF); 
        \path[line] (U) edge[bend right=40] node[font=\small, sloped, below] {goal description in text} (PentestGPT);
        \path[line] (PentestGPT) edge[bend right=20] node[font=\small, sloped, below] {exchange exploit tree} (EF);
        \path[line] (EF) edge[bend left=20] node {} (PentestGPT);
        \path[line] (EF) -- node[font=\small, sloped, above] {Benchmarks an exploit flow} (PentestPerf);
        \path[line] (EF) -- (T);
        \path[line] (T) -- (EF);

        \node[box, emphasize, below left=4cm and 2.5cm of EF] (PentestGPT) {2. PentestGPT};
        \node[innerComp, below=0.5cm of EF, xshift=3.5cm] (gt) {Game Theory (\CutTheRope)};
        \node[box, emphasize] (EF) at (11cm,0) {1. ExploitFlow}; 

        \node[innerComp, above=2.5cm of PentestGPT, xshift=-2cm] (parsing) {parsing};
        \node[innerComp, right of=parsing] (reasoning) {reasoning};
        \node[innerComp, right of=reasoning] (generation) {generation};
        
        \begin{scope}[node distance=0.3cm, font=\footnotesize]
            \coordinate (legendOrigin) at (-4,-6); 
            \node[box, right=0.3cm of legendOrigin, minimum height=1.5em, minimum width=1.5em, text width=1.5em] (legendBox) {}; 
            \node[right=0.3cm of legendBox] {External entity};
    
    
            \node[emphasize, below=0.2cm of legendBox, minimum height=1.5em, minimum width=1.5em, text width=1.5em] (legendEmphasize) {};
            \node[right=0.3cm of legendEmphasize] {Developed in this thesis};
    
            \node[innerComp, below=0.2cm of legendEmphasize, minimum height=1.5em, minimum width=1.5em, text width=1.5em] (legendInner) {};
            \node[right=0.3cm of legendInner] {Inner Component};

            
        \end{scope}
    \end{tikzpicture}
    
    \caption{
        Our architecture and framework to develop fully automated offensive robot cybersecurity agents, which name \emph{cybersecurity cognitive engines} in Chapter \ref{chapter:game_theory}: \framework{}.
    }
    \label{figure:cybersecurity_cognitive_engine:conclusion} 
\end{figure}

\showkeystrue

\FloatBarrier

The resulting framework and architecture for our \emph{cybersecurity cognitive engines} (introduced in Chapter \ref{chapter:game_theory}) is again presented above in Figure \ref{figure:cybersecurity_cognitive_engine:conclusion} and captures the results obtained in this thesis. In relation to the content of this thesis, Parts \ref{part:motivation} and \ref{part:ethics} helped understand better the User and Target entities. Part \ref{part:blueprints} helped developed tools such as \emph{ExploitFlow}, which sits at the core of our architecture providing capabilities for composing exploitation flows to interact with the robotic security environment and achieve desired goals. Part \ref{part:attacking_robots} helped populate tools like \emph{ExploitFlow} with actions that corresponded with both exploitation and mitigation of robot security flaws. Finally, Part \ref{part:automation} completes and connects the architecture altogether with three final contributions: a game theoretic extension of \emph{ExploitFlow} (through implementing Cut-The-Rope) that allows to reason about attack trees and identify the best attack vectors to achieve desired goals, b) a security reasoning engine (\emph{PentestGPT}) that digests such annotated attack trees, the attack vectors and the ultimate goal and guides the process and c) a comprehensive penetration testing benchmark (\emph{PentestPerf}) that helps compare various implementations for the penetration testing task. \\
\newline


\noindent In conclusion, this thesis successfully demonstrates the critical role of offensive measures in enhancing robot cybersecurity, charting a path towards a future where robots are not only resilient to cyber threats but also equipped to autonomously defend themselves by hacking-themselves-first with Game Theory and Machine Learning. In a soon-to-come autonomous ever-evolving digital battleground, \emph{Malism} aims to provide a reference guideline for future implementations. In summary, the work presented in this thesis lays a solid foundation for the future of offensive cybersecurity in robotics, paving the way for innovative solutions that can autonomously protect and secure robotic systems in an increasingly interconnected world. The advancements achieved through this research not only contribute to the academic field by demonstrating how robots can be made more secure with an offensive stance, but also have far-reaching implications for the safety and security of next-generation robotic applications.



Future work will focus on implementing \emph{Malism} in a scalable form factor, demonstrating fully autonomous offensive and defensive cybersecurity behaviors, and extending its concepts to industrial setups and OT-machinery beyond robotics. Achieving embedded form factors for \emph{Malism}-like cybersecurity cognitive engines is critical for real-world applicability, necessitating the integration of efficient Machine Learning (ML) and Game Theory (GT) algorithms within limited computational resources. Leveraging hardware accelerators like FPGAs can empower these deployments, ensuring real-time threat detection and response without compromising primary robotic functions. Embedded systems will enable localized, on-device processing, reducing dependency on vulnerable external cloud services and enhancing security response times. Moreover, distributed security frameworks facilitated by embedded \emph{Malism}-like systems will enhance resilience and coordination in networked environments. This approach is crucial for dynamic, time-sensitive applications in autonomous vehicles, smart manufacturing, and critical infrastructure, where robust, real-time security solutions are essential. 

In reflecting on this research journey, we find ourselves at the exhilarating intersection of robotics and cybersecurity, where offensive tactics become the ultimate defense. This thesis has danced through the realms of ethical hacking, intricate game theory strategies, and the boundless potential of machine learning, culminating in the creation of autonomous guardians named \emph{cybersecurity cognitive engines} for our robotic systems. As we look to the future, \emph{Malism} stands as a reference concept that deserves further development and time. Such a concept may eventually lead to a world where robots not only outsmart cyber adversaries but evolve with every challenge they face.




    \appendix           
    \part{Appendices}

    \chapter{Research Results}
\label{chapter:results}

\section{Thesis framework}

The topics of interest for the purpose of this thesis are presented as building blocks. These blocks represent a framework for the development of this dissertation. Results are categorized accordingly, with each research product aligning to one of these \emph{Blocks} as hinted in the tables below. From bottom to top: (0) Ethics and law, (1) Offensive security blueprints for robots and their components, (2) Cyber-attacking robots and (3) Cyber-attacking robots with Machine Learning and Game Theory. See Figure \ref{fig:thesisblocks}.

\section{Past work}

Tables \ref{table:pastwork1} and \ref{table:pastwork2} present a summary of the work performed prior to engaging formally with the PhD. Several articles were produced focusing mostly in blocks 0 (Ethics and law) and 1 (Blueprints). Since building a solid base was instrumental for the PhD research, these grounding publications capture the robotics know-how acquired after multiple years in industry and represent the baseline from wherein the research started.

\newcolumntype{L}[1]{>{\raggedright\arraybackslash}p{#1}}

\newcolumntype{C}[1]{>{\centering\arraybackslash}p{#1}}

\setlength{\tabcolsep}{20pt}
\renewcommand{\arraystretch}{2.0}
\begin{table}[h!]
    \centering
    \scalebox{0.55}{
    \begin{tabular}{ |L{4cm}|p{10cm}|c|c| } 
        \hline
        \rowcolor{black}    
        \color{white}\textbf{Title} & \color{white}\textbf{Summary} & \color{white}\textbf{Year} & \color{white}\textbf{Block} \\
        \hline
        

        \rowcolor{black!5} 
        Extending the OpenAI gym for robotics: a toolkit for reinforcement learning using ROS and Gazebo \cite{2016arXiv160805742Z} &  This paper introduces an extension of the OpenAI Gym for robotics, integrating ROS and Gazebo simulator. It elaborates on the software architecture and benchmarks two Reinforcement Learning techniques: Q-Learning and Sarsa. The work establishes a system for comparing various techniques and algorithms under consistent virtual conditions in robotics. & 2016 & 1 \\

        \rowcolor{black!5} 
        The Hardware Robot Operating System (H-ROS); an infrastructure to create interoperable robot components \cite{mayoral2017shift} &  This article presents the Hardware Robot Operating System (H-ROS), a hybrid hardware-software infrastructure for creating adaptable, interchangeable robot components. H-ROS enables dynamic modification of a robot's internal representation, allowing selective activation and exchange of hardware parts, thus simplifying robot assembly. The study validates the concept's feasibility and discusses future enhancements for the H-ROS platform." & 2017 & 1 \\ 
        
        \rowcolor{black!5} 
        Time-sensitive networking for robotics \cite{gutierrez2018tsn} & This paper posits that Time-Sensitive Networking (TSN) will emerge as the standard for real-time robotic communications, reviewing relevant communication standards and highlighting issues with traditional Ethernet networks. It evaluates TSN's deterministic communication features through experimental tests in a robotic scenario. The findings suggest a gradual shift from existing real-time solutions to TSN, paving the way for a unified, interoperable landscape in robotics. & 2018 & 1 \\

        \rowcolor{black!5} 
        Real-time Linux communications: an evaluation of the Linux communication stack for real-time robotic applications \cite{gutierrez2018real} & This paper examines the Linux communication stack's suitability for real-time robotic applications, focusing on UDP-based communications in multi-core embedded devices. The study confirms that with proper configuration, the Linux kernel significantly improves UDP communication determinism. It also reveals that concurrent traffic affects latency bounds, suggesting isolation of real-time applications and corresponding interrupts on a separate CPU as a solution & 2018 & 1 \\

        \rowcolor{black!5} 
        Towards a distributed and real-time framework for robots: Evaluation of ROS 2.0 communications for real-time robotic applications \cite{gutierrez2018towards} & This study evaluates ROS 2.0 for real-time robotic applications, focusing on inter-component communication on Linux systems. It benchmarks worst-case latencies and missed deadlines to assess ROS 2.0's performance in real-time contexts. The research experimentally shows the impact of computational and network congestion on communication latencies and proposes a setup that achieves bounded traffic under specific conditions. & 2018 & 1 \\ 

        \rowcolor{black!5} 
        Time synchronization in modular collaborative robots \cite{gutierrez2018time} & The study emphasizes synchronization in modular cobot systems, addressing a key challenge in the next generation of collaborative robots. It delves into synchronization issues, demonstrating that with an optimally configured system, remarkable synchronization precision is attainable. Results include distributed sub-microsecond clock synchronization among modules, ROS 2.0 message timestamping accuracy below 100 microseconds, and stable millisecond-level end-to-end communication latencies, even under heavy network loads of up to 90\% capacity. & 2018 & 1 \\                     

        \hline
    \end{tabular}}
    \caption{Past work group 1}
    \label{table:pastwork1}
\end{table}

\setlength{\tabcolsep}{20pt}
\renewcommand{\arraystretch}{2.0}
\begin{table}[h!]
    \centering
    \scalebox{0.55}{
    \begin{tabular}{ |L{4cm}|p{10cm}|c|c| } 
        \hline
        \rowcolor{black}            
        \color{white}\textbf{Title} & \color{white}\textbf{Summary} & \color{white}\textbf{Year} & \color{white}\textbf{Block} \\
        \hline
        
        
        \rowcolor{black!5} 
        Introducing the Robot Security Framework (RSF), a standardized methodology to perform security assessments in robotics \cite{vilches2018introducing} &  A methodology to perform systematic security assessments in robots proposing a checklist-like approach that reviews most relevant aspects in a robot & 2018 & 1 \\ 
        
        \rowcolor{black!1} 
        Robot hazards: from safety to security \cite{kirschgens2018robot} & Discussion of the current status of insecurity in robotics and the relationship between safety and security, ignored by most vendors & 2018 & 0 \\

        \rowcolor{black!5} 
        Towards an open standard for assessing the severity of robot security vulnerabilities, the Robot Vulnerability Scoring System (RVSS) \cite{vilches2018towards} & Introduction of a new assessment scoring mechanisms for the severity of vulnerabilities in robotics that builds upon previous work and specializes it for robotics & 2018 & 1 \\

        \rowcolor{black!5} 
        Robotics CTF (RCTF), a playground for robot hacking \cite{mendia2018robotics} & Docker-based CTF environment for robotics (RCTF) which bridges the gap between robotics and cybersecurity. The RCTF enables reproduction of virtual robotics scenarios and adjustment of network setups for realistic security challenges. & 2018 & 1 \\
                
        \rowcolor{black!5} 
        Aztarna, a footprinting tool for robots \cite{vilches2018aztarna} & Aztarna, a tool focused on robot reconnaissance and footprinting, addresses the security overlook in the Industry 4.0 era, where robots are increasingly network-connected. It enables identification of various robots, aiding in assessing the widespread lack of security practices in the robotics industry.  & 2018 & 1 \\ 

        \rowcolor{black!5} 
        Volatile memory forensics for the Robot Operating System \cite{vilches2018volatile} & General overview of forensic techniques in robotics and discussion of a robotics-specific Volatility plugin named linux\_rosnode, packaged within the ros\_volatility project and aimed to extract evidence from robot's volatile memory. & 2018 & 1 \\ 

        \rowcolor{black!5} 
        Introducing the robot vulnerability database (RVD) \cite{mayoralrvd} & The Robot Vulnerability Database (RVD) addresses the emerging importance of cybersecurity in robotics, a field where cyber attacks can lead to severe safety consequences. RVD serves as a comprehensive platform for responsible disclosure of vulnerabilities in robots, surpassing existing databases in detail and relevance. This paper outlines RVD's design, disclosure policy, and invites contributions to mitigate zero-day vulnerabilities in robotics, enhancing overall security  & 2019 & 1 \\ 

        \rowcolor{black!8} 
        ROS 2 Security Workshop - ROSCon 2019: \href{https://ros-swg.github.io/ROSCon19_Security_Workshop/}{Demonstrating vulnerabilities} and ways to protect your robots \cite{workshopros2security} & This workshop demonstrated different real-world attacks performed on robots while teaching techniques on how to secure them against those attacks. This included teaching on how to design threat models for robotics, hardening the ROS ecosystem through SROS2 as well as using several penetration testing tools to verify that the robot is no longer vulnerable to those attacks.  & 2019 & 2 \\ 
                
        \rowcolor{black!5} 
        ROS 2 Real-Time Workshop - ROSCon 2019: \href{https://www.apex.ai/roscon2019}{Doing Real-Time with ROS 2}: Capabilities and Challenges \cite{workshopros2rt} & This workshop introduced participants to the on-going work to make ROS 2 usable as a foundation for soft, firm and hard real-time robot systems. Participants will learn what the current capabilities are of ROS 2 for real-time robotics, what still remains to be done, and how they can contribute to making ROS 2 usable in these applications.  & 2019 & 1 \\ 
                    
        \hline
    \end{tabular}}
    \caption{Past work group 2}
    \label{table:pastwork2}
\end{table}

\FloatBarrier

\section{Peer-reviewed articles}

Tables \ref{table:articles1} and \ref{table:articles2} show the peer-reviewed articles published within the PhD period:

\setlength{\tabcolsep}{20pt}
\renewcommand{\arraystretch}{2.0}
\begin{table}[h!]
    \centering
    \scalebox{0.50}{
    \begin{tabular}{ |L{4cm}|p{8cm}|c|c|L{3cm}| } 
        \hline
        \rowcolor{black}
        \color{white}\textbf{Title} & \color{white}\textbf{Summary} & \color{white}\textbf{Year} & \color{white}\textbf{Block} & \color{white}\textbf{Type/Publisher} \\
        \hline
        
        \rowcolor{black!8} 
        Industrial robot ransomware: Akerbeltz \cite{mayoral2019industrial} (\href{https://ieeexplore.ieee.org/document/9287894}{paper}) & Introduce Akerbeltz, the first known instance of industrial robot ransomware targeting a well known collaborative robot. & 2020 & 2 & 2020 Fourth IEEE International Conference on Robotic Computing (IRC) \\

        
        \rowcolor{black!5} 
        Cybersecurity in Robotics: Challenges, Quantitative Modeling and Practice \cite{ROB-061} (\href{https://www.nowpublishers.com/article/Details/ROB-061}{paper}) & This book emphasizes the critical need for integrating security in robotics from the design phase, addressing the common neglect of security in the rapid development of robotic systems. It advocates for quantitative security approaches, including tailored vulnerability scoring and game theory, to enhance protection against various attacks. Targeting both theorists and practitioners, the book offers strategies to manage the complex nature of robotic systems effectively. & 2021 & 1 & Foundations and Trends® in Robotics: Vol. 9: No. 1, pp 1-129  \\ 


        %
        \rowcolor{black!0} 
        Hacking planned obsolescense in robotics, towards security-oriented robot teardown \cite{mayoral2021hacking} (\href{https://journal.ub.tu-berlin.de/eceasst/article/download/1189/1113}{paper}) & This thought-provoking article introduces robot teardown as a method to explore robot hardware architectures and enhance security research. It highlights the increasing need for component updates in damaged or compromised robots and critiques manufacturers' use of planned obsolescence to hinder repairs. The approach is effective in exposing security vulnerabilities and evidencing manufacturers' strategies against competition. & 2021 & 0 &  Conference on Networked Systems 2021 (NetSys 2021) \\

        \rowcolor{black!8} 
        Robot teardown, stripping industrial robots for good \cite{mayoral2022robot} (\href{https://www.conceptechint.net/index.php/CFATI/article/download/40/15}{paper}) & This paper examines the crucial hardware/software nexus in robotics, emphasizing teardown as essential for cybersecurity research. It analyzes contemporary industrial robot hardware, critiquing manufacturers' practices of preprogrammed obsolescence and the consequent generation of robotics waste. The study highlights the manufacturers' restrictive networks that limit repair capabilities and stifle competition, underscoring the ongoing need for component updates in response to damage or security issues. & 2022 & 2 &  Journal of Cyber Forensics and Advanced Threat Investigations (CFATI) \\
        
        \rowcolor{black!0} 
        Robot cybersecurity, a review \cite{mayoral2022review} (\href{https://conceptechint.net/index.php/CFATI/article/download/41/16}{paper}) & This article highlights the inadequate cybersecurity in robotics, attributed to underdeveloped defensive mechanisms, the complexity and cost of securing robots, and manufacturers' delayed responses to vulnerabilities. Drawing on literature, forum surveys, and recent research, it underscores the challenges in robotic security. The author, leveraging extensive experience and empirical data from security assessments, advocate for a proactive offensive strategy as an effective solution to strengthen robot cybersecurity. & 2022 & 0 &  Journal of Cyber Forensics and Advanced Threat Investigations (CFATI) \\

        \rowcolor{black!5} 
        SROS2: Usable Cyber Security Tools for ROS 2 \cite{mayoral2022sros2} (\href{https://arxiv.org/pdf/2208.02615.pdf}{paper}) & SROS2 is presented as a user-friendly suite of tools and libraries to secure ROS 2 computational graphs, aligning with the growing use of ROS 2 in the robotics industry. The work emphasizes a systematic, DevSecOps-aligned approach to security, illustrated through a case study with TurtleBot3. The analysis of SROS2's capabilities underscores the essential role of usability in the effective implementation of security in robotics. & 2022 & 1 &  2022 IEEE/RSJ International Conference on Intelligent Robots and Systems (IROS)  \\
        
        \hline
    \end{tabular}}
    \caption{Peer-reviewed articles group 1.}
    \label{table:articles1}
\end{table}

\setlength{\tabcolsep}{20pt}
\renewcommand{\arraystretch}{2.0}
\begin{table}[h!]
    \centering
    \scalebox{0.50}{
    \begin{tabular}{ |L{4cm}|p{8cm}|c|c|L{3cm}| } 
        \hline
        \rowcolor{black}
        \color{white}\textbf{Title} & \color{white}\textbf{Summary} & \color{white}\textbf{Year} & \color{white}\textbf{Block} & \color{white}\textbf{Type/Publisher} \\
        \hline
        
        \rowcolor{black!5} 
        RobotCore: An Open Architecture for Hardware Acceleration in ROS 2 \cite{mayoral2022robotcore} (\href{https://ieeexplore.ieee.org/document/9982082}{paper}) & RobotCore enhances robotics by integrating hardware acceleration into ROS 2, supporting diverse platforms and accelerators like FPGAs and GPUs. This adaptable architecture, complemented with a new firmware layer, facilitates efficient deployment across various setups. It utilizes the Linux Tracing Toolkit for real-time tracing and benchmarking, highlighting its effectiveness in a case study where RobotCore accelerates a ROS 2 computational graph on an FPGA, achieving a significant 24.42\% speedup over CPU-based systems. & 2022 & 1 &  2022 IEEE/RSJ International Conference on Intelligent Robots and Systems (IROS)  \\ 

        \rowcolor{black!5} 
        FogROS 2: An adaptive and extensible platform for cloud and fog robotics using ROS 2 \cite{ichnowski2023fogros} (\href{https://par.nsf.gov/biblio/10396373}{paper}) & FogROS2 is an advanced open-source platform designed for enhancing robotic computing by integrating with cloud services like AWS, GCP, and Azure, ensuring compatibility with ROS 2. With a focus on security and efficiency, it offers substantial improvements over FogROS1, including lower latency, faster startup, and better usability. Key enhancements include a 50\% reduction in SLAM latency and a significant boost in planning speeds, alongside improved network efficiency. Available on the ROS 2 repository, FogROS2 addresses the computing limitations of robots securely and effectively. & 2023 & 1 &  2023 Proceedings IEEE International Conference on Robotics and Automation (ICRA)  \\         

        \rowcolor{black!8} 
        Threat modeling for robotic-based production plants \cite{lera2022threat} (\href{https://ieeexplore.ieee.org/document/10018641}{paper}) & This study highlights the growing cybersecurity challenges accompanying the integration of robotics and automation in production plants. It presents a comprehensive threat model for a production facility with diverse components like PLCs, machine tools, sensors, actuators, and robots. Recognizing the heterogeneity of these components and their protocols, the paper outlines potential threats to the factory's security. It also suggests a series of changes and mitigations to enhance the cybersecurity and resilience of these production systems. & 2022 & 2 &  2022 IEEE International Symposium on Safety, Security, and Rescue Robotics (SSRR)  \\    

        \rowcolor{black!12} 
        Game-theoretic APT defense: An experimental study on robotics \cite{rass2023game} (\href{https://www.sciencedirect.com/science/article/pii/S0167404823002389}{paper}) & This paper presents a novel game-theoretic approach to counter Advanced Persistent Threats (APTs) in cybersecurity. It expands the Cut-The-Rope model to simulate more realistic attack scenarios and employs this framework in an experimental APT defense game using attack graphs. This allows for optimized defense strategies without assuming a clean system or relying on adversarial activity indicators. The model effectively incorporates vulnerability data like CVEs but can also function without it. Tests on real-world attack graphs in robotics show that this optimized defense significantly outperforms basic heuristics, demonstrating its practical applicability and effectiveness in complex cybersecurity environments. & 2023 & 3 &  Computers \& Security (2023)  \\ 

        \rowcolor{black!5} 
        RobotPerf: An Open-Source, Vendor-Agnostic, Benchmarking Suite for Evaluating Robotics Computing System Performance \cite{mayoral2023robotperfIROS} (\href{http://www.robot.t.u-tokyo.ac.jp/TCPEBRAS_IROS2023/RobotPerf.pdf}{paper}) & RobotPerf is a vendor-neutral, open-source benchmarking suite using ROS 2, designed to evaluate computing performance in robotics across various hardware. It integrates black-box and grey-box testing methods for comprehensive performance assessment and is adaptable for custom ROS 2 computational graphs. Developed with expert inputs, it aims to standardize robotics benchmarking and evolve with community contributions. & 2023 & 1 &  Proceedings of the 2023 IEEE/RSJ International Conference on Intelligent Robots and Systems (IROS) Workshop on Methods for Objective Comparison of Results in Intelligent Robotics Research  \\ 
        
        \hline
    \end{tabular}}
    \caption{Peer-reviewed articles group 2.}
    \label{table:articles2}
\end{table}

\FloatBarrier

\section{Preprint articles (including \emph{under review})}

Table \ref{table:articles3} summarizes other \textbf{non peer-reviewed} results while classifying them according to the four levels of Figure \ref{fig:thesisblocks}. 

\setlength{\tabcolsep}{20pt}
\renewcommand{\arraystretch}{2.0}
\begin{table}[h!]
    \centering
    \scalebox{0.5}{
    \begin{tabular}{ |L{4cm}|p{12cm}|c|c| } 
        \hline
        \rowcolor{black}
        \color{white}\textbf{Title} & \color{white}\textbf{Summary} & \color{white}\textbf{Year} & \color{white}\textbf{Block}  \\
        \hline
        
        
        \rowcolor{black!5} 
        DevSecOps in Robotics \cite{mayoral2020devsecops} (\href{https://arxiv.org/pdf/2003.10402.pdf}{paper}) & A set of best practices designed to help roboticists implant security deep in the heart of their development and operations processes.  & 2020 & 1 \\
        
        \rowcolor{black!5} 
        Alurity, a toolbox for robot cybersecurity \cite{mayoral2020alurity} (\href{https://arxiv.org/abs/2010.07759}{paper}) & Alurity is a modular and composable toolbox for robot cybersecurity. It ensures that both roboticists and security  researchers working on a  project, have a common, consistent and easily reproducible development environment facilitating the security process and the collaboration across teams & 2020 & 1  \\                 

        \rowcolor{black!8} 
        Can ROS be used securely in industry? Red teaming ROS-Industrial \cite{mayoral2020can} (\href{https://arxiv.org/abs/2009.08211}{paper}) & Red team ROS in an industrial environment to attempt answering the question: Can ROS be used securely for industrial use cases even though its origins didn't consider it? & 2020 & 2 \\
        
        \rowcolor{black!5} 
        An Introduction to Robot System Cybersecurity \cite{zhu2021cybersecurity} (\href{https://arxiv.org/pdf/2103.05789.pdf}{paper}) & 
        This article addresses the often overlooked aspect of security in rapidly developed robotic systems, advocating for its integration from the initial design phase. It proposes quantitative methods and game theory for robust security management, tailored to the complex, distributed nature of robotics. The focus is on balancing affordability with effective security, providing a comprehensive guide for practitioners and theorists to enhance protection against diverse cyber threats in robotics. & 2021 & 1 \\

        \rowcolor{black!8} 
        A Security Analysis of the Data Distribution Service (DDS) Protocol  \cite{maggi2022security} (\href{https://documents.trendmicro.com/assets/white_papers/wp-a-security-analysis-of-the-data-distribution-service-dds-protocol.pdf}{paper}) & 
        This study revealed significant security vulnerabilities in the Data Distribution Service (DDS), a middleware used in critical systems globally, including ROS 2. We identified 13 new CVEs across DDS implementations and found hundreds of public-facing DDS services, posing serious risks like denial of service and loss of control. The findings underscore the urgent need for continuous security testing of DDS. The research was a collaborative effort involving experts from multiple organizations, highlighting the widespread concern for DDS security. & 2022 & 2 \\

        \rowcolor{black!5} 
        ExploitFlow, cyber security exploitation routes for Game Theory and AI research in robotics \cite{mayoral2023exploitflow} (\href{https://arxiv.org/pdf/2308.02152.pdf}{paper}) & ExploitFlow (EF) is introduced in this paper as a modular library integrating Game Theory and AI for cybersecurity, specifically targeting robotics. It automates attacks by combining various exploits and analyzing system states to identify potential attack paths. & 2023 & 1 \\  

        \rowcolor{black!12} 
        PentestGPT: An LLM-empowered Automatic Penetration Testing Tool \cite{deng2023pentestgpt} (\href{https://arxiv.org/pdf/2308.06782.pdf}{paper}) & This study examines the use of Large Language Models (LLMs) for automating penetration testing, revealing their strengths in individual tasks but challenges in overall scenario comprehension. To address this, the researchers introduce PentestGPT, an advanced LLM-based tool with specialized modules to enhance penetration testing efficiency. Outperforming standard LLMs, PentestGPT shows significant improvements in task completion and has gained notable attention in the cybersecurity community, demonstrating its practical applicability in both academic and industrial settings. & 2023 & 3 \\  

        \rowcolor{black!5} 
        RobotPerf: An Open-Source, Vendor-Agnostic, Benchmarking Suite for Evaluating Robotics Computing System Performance \cite{mayoral2023robotperf} (\href{https://arxiv.org/pdf/2309.09212.pdf}{paper}) & RobotPerf is a vendor-neutral, open-source benchmarking suite using ROS 2, designed to evaluate computing performance in robotics across various hardware. It integrates black-box and grey-box testing methods for comprehensive performance assessment and is adaptable for custom ROS 2 computational graphs. Developed with expert inputs, it aims to standardize robotics benchmarking and evolve with community contributions. & 2023 & 1  \\         
        
        \hline
    \end{tabular}}
    \caption{Preprints, this includes \emph{under review} documents}
    \label{table:articles3}
\end{table}

\FloatBarrier

\section{Tutorials and talks}

Table \ref{table:talks} summarizes participation in peer-reviewed tutorials and talks across workshops and conferences while classifying them according to the four thesis building blocks of Figure \ref{fig:thesisblocks}.

\setlength{\tabcolsep}{20pt}
\renewcommand{\arraystretch}{2.0}
\begin{table}[h!]
    \centering
    \scalebox{0.5}{
    \begin{tabular}{ |L{3cm}|p{10cm}|c|c|L{2cm}| } 
        \hline
        \rowcolor{black}
        \color{white}\textbf{Title} & \color{white}\textbf{Summary} & \color{white}\textbf{Year} & \color{white}\textbf{Block} & \color{white}\textbf{Type/Publisher} \\
        \hline
        

        \rowcolor{black!1} 
        Cyber security for robotics workshop (\href{https://aliasrobotics.com/cs4r.php}{part1}, \href{https://aliasrobotics.com/cs4r_2.php}{part2})  & Workshop on robot cybersecurity where to argue about robot-related cybersecurity topics and challenges while providing a deep insight and discussion about potential cyber-risks the current robotics ecosystem is facing. & 2020 & 0  &  European Robotics Forum (ERF) 2020 \\

        \rowcolor{black!5} 
        Workshop on Security and Privacy in Robotics Zero Trust Architecture in Robotics (\href{https://sites.google.com/view/icra-workshop-2020}{site})  & Introduced a security architecture for robots that makes no trust assumptions and demands strict identity verification for every person, device or sub-component trying to access resources on a robot network (internal or external), regardless of whether they are sitting inside or outside of the network perimeter. & 2020 & 1 &   2020 International Conference on Robotics and Automation (ICRA) \\
        
        \rowcolor{black!5} 
        Cybersecurity in Robotics Tutorial: Demonstrating vulnerabilities and methods to protect robots (\href{https://cybersecurityrobotics.github.io/IROS2021/}{site})  & This tutorial spotted diverse cybersecurity vulnerabilities and attack patterns for robotic systems and demonstrated new quantitative methods of cybersecurity management and design. In particular, a model-based security approach with game theory was presented to address security issues related to the heterogeneity and complexity of robotic systems & 2021 & 1 &  IEEE/RSJ International Conference on Intelligent Robots and Systems (IROS) 2021 \\

        \rowcolor{black!8} 
        Conference Talk: Small Wonder: Uncovering Planned Obsolescence Practices in Robotics and What This Means for Cybersecurity (\href{https://www.blackhat.com/us-21/briefings/schedule/index.html\#small-wonder-uncovering-planned-obsolescence-practices-in-robotics-and-what-this-means-for-cybersecurity-23325}{site})  & We introduce and promote systematic "robot teardown" as an approach to repair robots by understanding their internals (still obscure). We show several "tricks from the trade" and the legal implications learned by porting reverse-engineering practices into the less-explored field of robotics. We explain how we a) discovered more than 90 security vulnerabilities in robots from Teradyne (MiR and UR) over a period of two years (never discussed publicly before), b) gained repairing capabilities on these robots, c) show evidence of planned obsolescence by comparing two sequentially released robot controllers, and d) demonstrate how robot hacking leads us to repurpose an older controller (previous version) from Universal Robots with their newer robots (arms) maintaining full capabilities and demonstrating that there's no need to re-spend thousands of dollars again. Paper produced for the venue available \href{http://i.blackhat.com/USA21/Wednesday-Handouts/us-21-Small-Wonder-Uncovering-Planned-Obsolescence-Practices-In-Robotics-And-What-This-Means-For-Cybersecurity-wp.pdf}{here} & 2021 & 2 &  BlackHat 2021 USA \\

        \rowcolor{black!8} 
        1st International Workshop on Cyber Forensics and Advanced Threat Investigations in Emerging Technologies: Hacking planned obsolescense in robotics, towards security-oriented robot teardown (\href{https://cfati3.conceptechint.net/index.html}{site})  & The main motivation for this Workshop is to bring together researchers and practitioners working on cyber forensics and threat investigations for emerging technologies to disseminate current research issues and advances. Original technical papers describing new, state-of-the-art research, will be considered. The Workshop welcomes submissions that evaluate existing research results by reproducing experiments. The aim of this workshop is to provide insight for the discussion of the major research challenges and achievements on various topics of interest. & 2021 & 2 &   5th International Conference on Networked Systems (Netsys 2021)  \\
        
        \rowcolor{black!8} 
        Conference Talk: The Data Distribution Service (DDS) Protocol is Critical: Let us Use it Securely (\href{https://www.blackhat.com/eu-21/briefings/schedule/index.html\#the-data-distribution-service-dds-protocol-is-critical-lets-use-it-securely-24934}{site})  & We discovered and disclosed vulnerabilities in most of the OMG Data Distribution Service (DDS) implementations. DDS enables crucial technologies like autonomous driving, healthcare machinery, military tactical systems, or missile launch stations. Notably, DDS is used by NASA at the KSC, by SIEMENS for smart grid applications, by Volkswagen and Bosch for autonomous valet parking systems, by NAV CANADA for ATC, and by the Robot Operating System 2 (ROS2) to control industrial and consumer robots. Presentation available at  \href{http://i.blackhat.com/EU-21/Thursday/EU-21-Yen-The-Data-Distribution-Service-DDS-ProtocolIs-Critical-Lets-Use-It-Securely.pdf}{here} & 2021 & 2 &  BlackHat 2021 Europe \\
        \hline
    \end{tabular}}
    \caption{Workshops, tutorials and talks in peer-reviewed conferences.}
    \label{table:talks}
\end{table}

\section{Vulnerabilities found}

Table \ref{table:vulnerabilities} summarizes vulnerabilities found during the research period that have been awarded with a CVE ID:

\setlength{\tabcolsep}{20pt}
\renewcommand{\arraystretch}{2.0}
\begin{table}[h!]
    \centering
    \scalebox{0.55}{
    \begin{tabular}{ |L{1cm}|L{8cm}|L{4cm}|L{1cm}|L{1cm}| } 
        \hline
        \rowcolor{black}
        \color{white}\textbf{CVE ID} & \color{white}\textbf{Summary} & \color{white}\textbf{Scope} & \color{white}\textbf{CVSS} & \color{white}\textbf{RVSS \cite{vilches2018towards}} \\
        \hline

        CVE-2019-19626 & Bash scripts (magic UR files) get launched automatically with root privileges and without validation or sanitizing & Universal Robots CB-series UR3, UR5, UR10 & 6.8 & 10 \\

        CVE-2020-10267 & Unprotected intelectual property in Universal Robots controller CB 3.1 across firmware versions & Universal Robots CB-series UR3, UR5, UR10 and e-Series UR3e, UR5e, UR10e & 7.5 & 6.5 \\

        CVE-2020-10268 & Terminate Critical Services in KUKA controller KR C4 & Any KUKA robot driven by the KUKA KR C4 controller, e.g. the KUKA KR 3 R540 & 6.1 & 7.1 \\

        CVE-2020-10271 & MiR ROS computational graph is exposed to all network interfaces, including poorly secured wireless networks and open wired ones & MiR-based fleet of AMRs and derivative robots including MiR100, MiR200, MiR250, MiR500, MiR1000, ER-Lite, ER-One, etc & 8.0 & 10.0 \\

        CVE-2020-10278 & Unprotected BIOS allows user to boot from live OS image. The BIOS onboard MiR's Computer is not protected by password, therefore, it allows a Bad Operator to modify settings such as boot order. This can be leveraged   by a Malicious operator to boot from a Live Image. & MiR-based fleet of AMRs and derivative robots (MiR100, MiR200, MiR250, MiR500, MiR1000, ER-Lite, etc) & 6.1 & 7.1 \\

        CVE-2020-10269 & Hardcoded Credentials on MiRX00 wireless Access Point. MiR fleet vehicles comes pre-configured in WiFi Master (Access Point) mode. Credentials to such wireless Access Point default to well known and widely spread SSID (MiR\_RXXXX) and passwords (omitted). This information is also available in past User Guides and manuals which the vendor distributed. We have confirmed this flaw in MiR100 and MiR200 but it might also apply to MiR250, MiR500 and MiR1000. & MiR-based fleet of AMRs and derivative robots (MiR100, MiR200, MiR250, MiR500, MiR1000, ER-Lite, etc) & 9.8 & 9.8 \\

        CVE-2020-10272 & MiR ROS computational graph presents no authentication mechanisms. MiR100, MiR200 and other MiR robots use the Robot Operating System (ROS)   default packages exposing the computational graph without any sort of authentication.   This allows attackers with access to the internal wireless and wired networks to   take control of the robot seamlessly. In combination with CVE-2020-10269 and CVE-2020-10271,   this flaw allows malicious actors to command the robot at desire. & MiR-based fleet of AMRs and derivative robots (MiR100, MiR200, MiR250, MiR500, MiR1000, ER-Lite, etc) & 8.0 & 10.0 \\

        CVE-2020-10287 & Hardcoded default credentials on IRC 5 OPC Server. The IRC5 family with UAS service enabled comes by default with credentialsthat can be found on publicly available manuals. ABB considers this a well documentedfunctionality that helps customer set up however, out of our research, we foundmultiple production systems running these exact default credentials and considerthereby this an exposure that should be mitigated. Moreover, future deploymentsshould consider that these defaults should be forbidden (user should be forced tochange them). & All robots powered by ABB IRC5 controller OPC Server (e.g. ABB IRB140) & 9.1 & 10 \\
        
        \hline
    \end{tabular}}
    \caption{Some of the most relevant robot vulnerabilities found during research.}
    \label{table:vulnerabilities}
\end{table}

    \chapter{Game Theoretic Appendix}
\label{chapter:gt_appendix}

\section{\CutTheRope Background}

\acp{APT}, like most targeted attacks conducted by cybercriminals, due to their diverse combination of attacks, hardly admit a single model to capture them; rather, they call for a combination of models designed for different aspects or characteristics of the attack. Game theoretic defense models may be distinguished according to the nature of \ac{APT} \cite{rass_cyber-security_2020} that they cover: there is the parasitic type, in which the attacker tries to steal resources for as long and much as possible, but does not aim to kill its victim. Related models are \textsc{FlipIt} \cite{dijk_flipit:_2013,zhang_flipin:_2019} and its descendants. Minimizing the total time that the attacker spends in the system may not necessarily minimize damage too, since the attacker may entirely destroy the asset $v_0$ even within a very short period of time. The defender may nonetheless suffer a permanent defeat (upon loss of $v_0$). For example, if the attacker can gain access to the security controls of a nuclear power plant even for a very short time, this may be sufficient to cause an unstoppable meltdown. Conversely, the attacker may spend a considerably larger amount of time in other areas of the nuclear power plant's system; as long as there is no vital subsystem to fiddle with, the damage to the infrastructure may be bearable. This motivates the consideration of the second type of \ac{APT}, for which the game model \CutTheRope is tailored to: there, the attacker aims to kill the victim and silently prepare the final blow. A documented case of this is Stuxnet \cite{kushner_real_2013}, and \CutTheRope is a game model designed for this latter type.

Many other game models are aligned with the phases in the kill chain, and most related work \cite{etesami_dynamic_2019} is specific for at least one of them. 
We note that the ADAPT project \cite{noauthor_adapt:_2018} covers a wide spectrum of aspects and phases here. 
Specific defense models include the detection of spying activities \cite{HuQing:2344}, tracing information flows \cite{moothedath_game_2018}, detection of malware \cite{khouzani_saddle-point_2012}, deception \cite{carroll_game_2009} also via honeypots \cite{la_game_2016}, attack path prediction \cite{fang_game_2014}, path selection to support malware detection in distributed networks \cite{panaousis2017game}, and general network defense \cite{alpcan_network_2010} to name only a few. 
Our game is in a way similar to that of the seminal work \cite{Lye&Wing2005}, yet differs from this previous model in not being stochastic, and in using payoffs that are not real-valued. 
The stochastic element is included in a much simpler way in our model, yet preserving information about uncertainty in a full distribution, to avoid losing information by averaging out randomness (for example, replacing a random payoff by a real-valued expected payoff).

Since the methods applied here come from the risk management field, this relates our work to that of \cite{yang_risk_2018}, who presents a framework to optimally respond to a detected \ac{APT}. 
Their work is thus an \emph{a posteriori} treatment after the \ac{APT} succeeded, while ours complements the risk management here by an \emph{a priori} treatment to prevent the \ac{APT} from success.
Likewise notable is also the work of \cite{hota_game-theoretic_2018,zhu_optimal_2016}, who consider inter-dependency graphs in relation to attack graphs in a game-theoretic analysis of targeted attacks. Their work adds constraints to budgets or desirable risk levels, and is specifically about investments in defenses of nodes and edges, but also works with crisp payoff measures (such as, e.g., paths of maximal attack probability or similar).

A different classification of related work is based on the protection targets.
defenses can be optimized for confidentiality \cite{lin_using_2012}, the monetary value of some asset upon theft or damage \cite{zhu_multi-phase_2018}, or the time that an adversary has parts of the system under control \cite{dijk_flipit:_2013}. 
This distinction can be important depending on the context, as industrial production typically puts priority on availability and integrity, with confidentiality as a secondary or tertiary interest. 
Conversely, whenever personal data is processed, confidentiality becomes the top priority, putting availability further down on the list.

The techniques applied to capture and defend against \acp{APT} are manifold, but in most of these (like in our work), the network graph is in the center of attention: it may define how an attack evolves as a dynamical system \cite{senejohnny_jamming-resilient_2018,yang_effective_2019} inside the graph topology, with the challenge of optimized orchestrated defense. 
A good defense design that needs to account for new vulnerabilities potentially being opened up when closing known security holes. The work of \cite{touhiduzzaman_diversity-based_2019}, in this regard, utilizes a game model for graph coloring for a systematic and optimized defense, applying these results to industrial bus systems. 
Another dynamic yet queuing-based model is that of \cite{li_intelligence-driven_2019}, which like our model computes optimal resource allocations by the defender and attacker, as an aid for decision making. 
Tailoring the attack model more closely to the application domain for the sake of a more accurate description, the work of \cite{soltan_react_2019} provides insightful connections of graph topological properties of a power grid, and how areas in danger of becoming attacked are identifiable from analyzing the graph. 

The work of \cite{pawlick_strategic_2017,pawlick_istrict:_2019} takes a more birds eye perspective on the domain of the \ac{IoT}, and applies it directly to or varies the \textsc{FlipIt} game (see the references above and \cite{zhang_flipin:_2019}) to model individual parts of a cloud-based \ac{IoT} infrastructure, combining these submodels into a larger hybrid game model that allows certain equilibria to play optimally against the adversary. 
Another cloud-related and -specific \ac{APT} defense model is \cite{yuan_stackelberg-game-based_2019}. 
Like us, they adopt a leader-follower model, but different to our work, they use a Stackelberg equilibrium concept.

Taking the \ac{APT} as a long term yet one-shot event, an attack graph can be treated as a (big) game in extensive form. 
From this point of view, it is possible to think of the \ac{APT} as an instance of the induced gameplay, to which Bayesian or subgame perfect equilibria can be sought \cite{huang_adaptive_2018}. 
More similar to this work, we can treat the \ac{APT} as a game of inspections, to discover optimal strategies of inspection in different depths of a shell-structured defense
\cite{rass_gadapt:_2016,zhu_multi-phase_2018}. 
An aspect of strong relevance concerns the use of probabilities: the work of \cite{hota_interdependent_2018,xiao_attacker-centric_2018} are most interesting in its account for subjective probability and prospect theory, since this includes the way of how humans bring in their individual risk attitudes in decision making under uncertainty (especially about defenses).
We avoid this conceptual and practical difficulty in the modeling by designing our game with as few probabilistic parameters as possible.

\CutTheRope is, in two ways, different from most other game theoretic models: first, it can let the players act in different time axes, meaning that the defender can be active in discrete or continuous time, while the attacker is (here always) acting in continuous time. This is in contrast to most other models in which both players act in fixed schedules (such as in extensive form games), or both can take actions continuously (such as in differential games). The second aspect is the added suggestion of tie-breaking if there are several equilibria. \CutTheRope implicitly addresses the equilibrium selection problem by refining the set of possibly many defense actions based on the probabilities to reach not only $v_0$, but also to get nearby it. Formally, the optimization, after having minimized the chance to conquer $v_0$, continues by minimizing the chances to reach a node close to $v_0$. As mentioned in Section \ref{sec:implementation-remarks} this induces a dependency on the ordering of nodes, but this ordering is up to the choice of the defender setting up the model. In any case, the defender is not left with a choice among possibly many equilibria, but can have the calculation automatically refine it in an interpretable sense. This equilibrium selection problem is not usually intrinsically addressed in other security game models.

\section{The Model}
\showkeysfalse
\label{sec:the-model}
\showkeystrue

\renewcommand{\vec}{\mathbf}

In the following, we let sets appear as upper case letters, and vectors and matrices in boldface font. Given a finite set $X$, the symbol $\Delta(X)$ denotes all (categorical) probability distributions supported on $X$, i.e., an element $\vec x\in\Delta(X)$ has the elements $(p_1,\ldots,p_{\abs{X}})$ with $p_i=\Pr(x_i\in X$ is chosen$)$. The symbol $\abs{X}$ is the cardinality of the set $X$.

\begin{figure}
	\centering
	\includegraphics[width=\columnwidth]{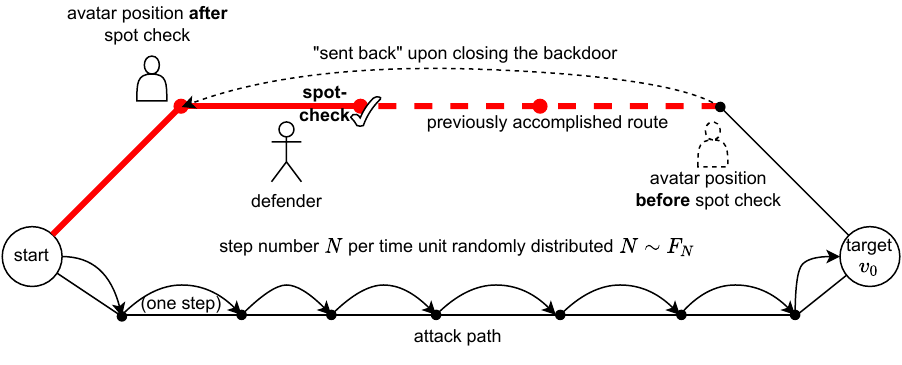}
	\caption{Basic Gameplay of \CutTheRope. Reproduced from \cite{rass2023game}.}
 \label{figure:ctr_gameplay}
\end{figure}

We refrain from replicating the full formal description of \CutTheRope, and instead summarize its concept in Fig. \ref{figure:ctr_gameplay}. The game is played entirely from the defender's viewpoint: the defender knows the attack graph $G=(V,E)$ and can enumerate the attack paths on which the adversary can be. For simplicity, we assume that the number of these routes is tractably small. Generally, the number of routes can be exponential in the cardinality $\abs{V}$ of nodes, but by strategic domination and other heuristics, some routes may be safely excluded from consideration. We will revisit this point later in Section \ref{sec:complexity}.

To express the uncertainty about where exactly the adversary is, the defender acts as if the attacker would move a whole cohort of avatars towards $v_0$, each avatar starting from another possible location in the attack graph with uniform probability. The game is round-based, where the exact meaning of a round depends on the moving patterns of the defender and the attacker:

\begin{itemize}
	\item If the defender acts periodically in fixed intervals (e.g., daily), a round of the game is one period of activity for the defender (e.g., one day). During this period of time, the attacker can take a random (unlimited) number of steps along the attack path towards the goal. 
	\item If the defender is taking action at random, e.g., taking exponentially distributed pause times, then a round of the game is, in each instant, the random idle time of the defender. Again, during these periods, the attacker can take any number of actions, depending on its ``configuration'' and/or the attack path. In the terminology of the \textsc{FlipIt} game \cite{dijk_flipit:_2013}, this is called an \emph{exponential defense strategy}.
\end{itemize}
In both cases, we do not explicitly model the time to complete a spot-check and merely assume this completion to be possible within one unit of time. Including the defender's costs for spot-checking as a separate goal (to minimize) makes the game multi-criteria and calls for Pareto-optimization, which we leave out of our scope in this work (and up to future considerations). 
We will come back to the exact meaning of a ``round'' or ``unit of time'' in Section \ref{sec:periodic-defender-poisson}. Let us first complete the description of the gameplay: Figure \ref{figure:ctr_gameplay} displays two attack paths, with the lower path showing the step-by-step traversal of an avatar towards the goal $v_0$. Every possible action of the defender is here called a \emph{spot-check} at any node in $V\setminus\set{v_0}$, where the target node is excluded to avoid trivialities\footnote{If the defender would not move away from the target, there would be nothing to accomplish here for the adversary and there would be nothing to analyze.}. A \emph{spot-check} can mean any action that, for example, (i) cleans a component from malware, (ii) disables certain services that an exploit would rely on, (iii) changes in the security policy or implementation that invalidates the adversary's knowledge (e.g., access control mechanisms), or similar. Common to all actions of the defender is their \emph{transient} efficacy, which means that the effect of such an action is \emph{not} permanent (the opposite case is discussed later in Section \ref{sec:cutting-the-rope-vs-changing-the-attack-graph}). After the action, and not necessarily known to the defender, the attacker is sent back on the attack path to an earlier position (upper part of Figure \ref{figure:ctr_gameplay}). For example, if the so-far accomplished route has at some point used access credentials for a computer, and the defender has just changed them, the route is essentially closed at this point, and the attacker has to re-try just before this point\footnote{We herein assume that there is no direct way to just get back to the later point: if there would be such a shortcut route bypassing the just-closed backdoor, this would be another attack path, taken by a respectively other avatar.}. The avatars can go unaffected by the defender's action in two cases: (i) if it travels on a different route that the defender did not inspect in this moment (e.g., lower attack path in Figure \ref{figure:ctr_gameplay}), or (ii) the attacker started from a location below the cut point (e.g., if the attacker is left to the cut point $\checkmark$ in Figure \ref{figure:ctr_gameplay}). This assumption implicitly accounts for ``out of attack graph'' ways of the attacker having reached this location In either case, the avatar's journey is not intercepted. 

The attacker may at any point decide to try a different route instead. This is called \emph{lateral movement}. It is naturally included in this modeling by having avatars on all attack routes, which makes lateral movement nothing else than moving other avatars on another route. \CutTheRope is played under the assumption that an avatar can be thrown back to an earlier point by the defender, but will in any case re-try its current attack path, until it (or any of its clones) has reached the goal.

The payoffs in the game are zero-sum, and come to the probability of reaching $v_0$ in a single round of the game. This is the payoff to the attacker, and likewise the loss of the defender, who seeks to minimize this probability. Its computation depends on the probability distribution law that governs how many steps $N$ can be taken during the defender's idle periods. This is the main ingredient whose influence is studied in this work, relative to a heuristic best-practice defense. 

The \emph{payoff} to the adversary is the chances for any of its avatars to reach, from its current position, the target $v_0$ within $N$ steps and within the time-limit $W$, during which the defender is idle\footnote{Here, we simplified the payoff representation from a vector-valued distribution in \cite{rass_cut--rope_2019} over the attacker distance to the goal, to the probability of reaching the goal. This does not affect the solution of the model, but facilitates readability.}. This can be fixed (for a periodic defender) or random (for an exponential defense strategy). We collect all avatars in a set $\Theta\subseteq V\setminus\set{v_0}$, and denote individual avatars as $\theta\in\Theta$. The exclusion of $v_0$ from this set is to avoid the trivial case where the attacker has already reached $v_0$ before the defense game starts. The payoffs to both players are:
\begin{align}
	u_{\text{attacker}} = -u_{\text{defender}} &= \Pr(\text{adversary reaches~}v_0)\label{eqn:final-utility}
\end{align}
We will formally determine this quantity in Section \ref{sec:definition-of-payoffs} in  expressions \eqref{eqn:attacker-utility-formula} and \eqref{eqn:defender-utility-formula}, which make the dependency on the strategic choices of the defender and attacker visible and explicit.

\subsection*{Strategies}
We now turn to the description of how the defender's and attacker's action determine the probability to reach $v_0$. The strategic choices of both players towards maximizing or minimizing $\Pr(\text{adversary reaches~}v_0)$ are the following:

\begin{itemize}
	\item The defender has a choice from the set $AS_1:=V\setminus\set{v_0}$ to spot-check, giving a total of $n=\abs{AS_1}$ actions.	We will write $\vec x\in\Delta(AS_1)$ for a randomized such spot-checking rule.
	\item The attacker can likewise use a total of $m=\abs{AS_2}$ attack paths in $G$, collected in the set $AS_2$. Each avatar starts from a different location $\theta\in\Theta=V\setminus\set{v_0}$ and traverses one of the (perhaps many) routes from $\theta$ towards $v_0$. The adversary solution in the game is the best choice of attack paths from $AS_2$.	Likewise, we will write $\vec y\in\Delta(AS_2)$ for a random choice from the set of attack paths.
\end{itemize}
Every avatar takes action by being moved forward along the attack path that it is on, and draws/samples a random number $N$ from a fixed step-distribution $f_N$. This is \emph{not} a strategic choice, but rather a part of the game's \emph{payoff mechanism}. Low-level procedures of how the avatar technically mounts exploits are not expressed nor modeled in the game itself (due to the heterogeneity and sheer number of possibilities of exploits in a real-life attack graph).



\subsection{Definition of Payoffs}
\showkeysfalse
\label{sec:definition-of-payoffs}
\showkeystrue
For the sake of rigor, let us concretize \eqref{eqn:final-utility} by showing how it is practically obtained. This will also display the role of the movement patterns (periodic, exponential) in the experimental analysis. Working out the adversary's utility is a matter of conditioning the attack step distribution $F_N$ on the current situation in the network, i.e, the position of the avatar and where the defender took action.

Let $\pi_{1}, \pi_{2},\ldots,\pi_{m}$ be an (exhaustive) enumeration of all attack paths, each starting from another location $\theta\in\Theta\subseteq V\setminus\set{v_0}$. Each starting location is thus identified with one avatar, and the adversary moves all of them towards $v_0$. Let $m$ be the total number of all attack paths.

Each such path is again a sequence of nodes, written as $\pi=(\theta,w_1,w_2,\ldots, v_0)$ with all $w_i\in\set{v_1,v_2,\ldots}=V$ and $\theta\in\Theta$ being the starting point of the route, one-to-one corresponding to an adversarial avatar. The set of nodes constituting $\pi$ is $V(\pi)$. Furthermore, let $d_\pi(u,v)\in\N$ count the edges on the path $\pi$ from $u$ to $v$. It is a graph-theoretic distance. 

Then, the \emph{location distribution} for the attacker assigns to each node $v\in V$ the mass
\begin{equation}\label{eqn:utility-before-cut}
	\Pr(\text{avatar location}=v|V(\pi))=\frac{f_N(d_\pi(\theta,v))}{\Pr_{N}(V(\pi))},
\end{equation}
in which $f_{N}(n)=\Pr(N=n)$, where $N\in\set{0,1,2,3,\ldots}$ is the random number of steps undertaken by the avatar, and 
\begin{equation}\label{eqn:utility-before-cut-cumulative}
\Pr_{N}(V(\pi))=\sum_{x\in
	V(\pi)}\Pr_{N}(d_\pi(\theta,x))=\sum_{x\in
	V(\pi)}f_{N}(d_\pi(\theta,x)). 
\end{equation}
The probability density $f_N$ will be the main element to vary when describing different attacker-defender scenarios (such as announced in the introduction under the contributions). We will give various options to define $f_N$ in equations \eqref{eqn:movement-poisson}, \eqref{eqn:movement-by-hardness}, \eqref{eqn:movement-by-geom} and \eqref{eqn:movement-by-hardness-in-random-times}.

Now, the defender attempts to break the attacker's chain of exploitation (``cut the rope'' in the wording of \cite{rass_cut--rope_2019}). 
Let $c\in V$ be the checked node, then the possibly truncated path is 
\begin{equation}\label{eqn:cut-path}
\pi|_c = \left\{
\begin{array}{ll}
	(\theta,w_1,w_2,\ldots,w_{i-1}), & \hbox{if }c=w_i~\text{for some~}w_i\,\,\text{on~}\pi \\
	(\theta,w_1,\ldots,v_0), & \hbox{otherwise.}
\end{array}
\right.
\end{equation}
The closing of a backdoor here becomes a conditioning of the distribution of the avatar's location on the shorter (cut) path $\pi|_c$. 
The formula is the same as \eqref{eqn:utility-before-cut}, only with $\pi$ replaced by $\pi|_c$ now.
Since $c\sim \vec x$ follows the defender's mixed spot checking strategy (possibly degenerate), and the set of paths $\pi$ along which avatars proceed, the defender can determine the possible locations of the attacker, based on the imagined avatars, as the vector of probabilities
\begin{equation}\label{eqn:adversary-utility}
	U=(\Pr(\text{adversary's location}=v|V(\pi|_c)))_{v\in V},
\end{equation}
which depends on the random choices of the defender (``where to cut?'') and the attacker (``which route to take?''). This is what the implementation of \CutTheRope computes.

The actual quantity of interest for the game, coming back to \eqref{eqn:final-utility}, is the mass that $U$ assigns to $v_0$. This is the utility for the adversary and conversely the loss of the defender. Since the game is, from the attacker's perspective, a strategic choice $\vec y\in\Delta(AS_2)$ of an attack path, the payoffs in the game are obtained from the following consideration:

\begin{align}
	\Pr&(\text{attacker reaches~}v_0) = \Pr(\text{at least one avatar  reaches~}v_0)\nonumber\\
	&= \sum_{\theta\in V\setminus\set{v_0}} \Pr(\text{avatar reaches~}v_0\text{~starting from~}\theta)\cdot\Pr(\theta)\nonumber\\
	& = \sum_{c,\pi}\sum_{\theta\in V\setminus\set{v_0}}\underbrace{\Pr[\text{avatar~}\theta\text{~has location~}v_0~|~V(\pi|_c)]}_{\text{from eq. \eqref{eqn:utility-before-cut} and \eqref{eqn:cut-path}}}\\\vspace{-2\baselineskip}
	&\quad\cdot\underbrace{\Pr(\text{path~}\pi\text{~is chosen and defender cuts at~}c)}_{\text{strategic choices to optimize}}\cdot\underbrace{\Pr(\theta)}_{=1/\abs{AS_2}}\nonumber\\
	& = \sum_{c,\pi}\sum_{\theta\in V\setminus\set{v_0}}\Pr\Big[\text{avatar~}\theta\text{~has location~}\nonumber\\
	&\quad v_0~|~V(\pi|_c)\Big]\cdot\Pr_{\vec x}(c)\Pr_{\vec y}(\pi)\cdot\Pr(\theta) =: u_{\text{attacker}}(\vec x,\vec y)\label{eqn:attacker-utility-formula}\\
	&= -u_{\text{defender}}(\vec x,\vec y)\label{eqn:defender-utility-formula}
\end{align}

The equality in the second line herein follows from the fact that the attacker will move one avatar at a time, so that no two avatars will simultaneously reach $v_0$. The first avatar to reach $v_0$ will make all others stop, so that the respective events become disjoint. 

\subsection{Solution Concept}
An \emph{instance} of \CutTheRope is a quintuple  $(G,v_0,AS_1,AS_2,f_N)$, containing: the attack graph $G=(V,E)$, the target node $v_0\in V$, the defender's possible spot check locations $AS_1\subseteq V\setminus\set{v_0}$, the possible locations $AS_2\subseteq V\setminus\set{v_0}$, for the attacker's avatars. These avatars will move towards $v_0$ along the attack paths encoded in $G$, taking a random number $N$ of steps distributed according to the probability density $f_N$. This density determines the particular behavior of the attacker, relative to the defender's actions, and will be generally given in Section \ref{sec:movement-patterns}, and instantiated for the two real-life use-cases in Section \ref{sec:results and comparison}.

A \emph{solution} for a given instance is obtained with standard techniques to compute Nash equilibria: With both players having a finite set of choices, and the utility $\Pr(\text{adversary's location}=v~|~V(\pi|_c))$ derived from the location distribution \eqref{eqn:adversary-utility} that depends on the attack path $\pi$, movement pattern $f_N$ and spot-check location $c$, we end up with a (normal-form) matrix game that we can analyze for an equilibrium using known techniques. The solution concept used in this work is a \emph{security strategy} for the defender, having the following (informal) semantics: it is the best randomized choice rule $\vec x^*\in\Delta(AS_1)$ such that
\begin{equation}\label{eqn:security-strategy}
	u_{\text{defender}}(\vec x^*, \vec y^*)\leq u_{\text{defender}}(\vec x^*,\vec y)\quad \text{for all~}\vec y\in\Delta(AS_2)
\end{equation}
That is, the defender can, upon playing the optimal spot checking strategy $\vec x^*$, enforce the worst-case minimal likelihood for the attacker to reach $v_0$, for all choice rules $\vec y\in\Delta(AS_2)$, i.e., irrespectively of what the attacker actually does.

The security strategy is computable by solving a conventional matrix game, which is finite since there are only finitely many spot check locations, and likewise finitely many attack paths. The game matrix is thus computable by evaluating formula \eqref{eqn:attacker-utility-formula}, for all locations $c\in AS_1$ and all paths $\pi\in AS_2$. The Nash equilibrium of this game is $(\vec x^*,\vec y^*$), in which $\vec x^*$ is the sought security strategy, and $\vec y^*$ is the optimal choice rule for the attack paths towards $v_0$.

The latter information is, however, of limited use for the defender, since equilibria are generally not unique. Therefore, taking $\vec y^*$ as a guidance on where to find for the invisible intruder with highest probability can be misleading, since there may be (plenty of) other equilibria giving entirely different advice. 

On the contrary, since the saddle point value giving the lower bound value in \eqref{eqn:security-strategy} is invariant w.r.t. different equilibria $(\vec x^*,\vec y^*)$ any alternative defense advice cannot accomplish any better lower bound for the defender. Hence, $\vec x^*$ is in fact useful as optimal advice. 

\begin{remark}
The original solution concept proposed in \cite{rass_cut--rope_2019} has been a perfect Bayesian equilibrium, but this raises issues with the interpretation of the results. While the game's setting formally \emph{fits} into the definition of a perfect Bayesian equilibrium as given by \cite{fudenberg_game_1991}, it \emph{does not fit} equally well into the interpretation thereof: the game is not about signaling, while the solution concept in \cite{rass_cut--rope_2019} took an equilibrium designed for signaling games. Also, there is no random conditioning on adversary types, which a Bayesian equilibrium would require. Rather, \CutTheRope is -- from the defender's point -- played with avatars, all of which concurrently move on their routes, without a particular type choice made by nature. Therefore, a security strategy (computed as a Nash equilibrium) is the more suitable solution concept.
\end{remark}

\section{Movement Patterns}
\showkeysfalse
\label{sec:movement-patterns}
\showkeystrue

In lack of any particular knowledge about the difficulty of the attack path, a simple heuristic is to just use the \emph{shortest path}, in a graph theoretic sense. This will later also be the intuitive benchmark (see Section \ref{sec:heuristic-defense}) to compare the defense obtained from \CutTheRope to a defense based on the (plausible) assumption that the attacker takes the shortest/easiest route towards $v_0$. 


\subsection{Periodically Active Defender}
\showkeysfalse
\label{sec:periodic-defender-poisson}
\showkeystrue

In the simplest case, originally proposed in \cite{rass_cut--rope_2019}, 
we assume that the defender becomes active in fixed time intervals that are known to the attacker. The \emph{unit of time} (see the previous section) is herein the period in which the defender becomes active (each day, each week, or similar). Furthermore, we assume no particular cost for the attacker to penetrate (this case is covered in Section \ref{sec:probabilistic-succes}). This corresponds to the situation of having a ``just conceptual'' attack graph, displaying general strategies to penetrate, but without reference to concrete exploits, CVE numbers or similar.

During the defender's idle times, we assume an average number $N\sim \Pois(\lambda)$ of steps towards its target at ``average speed $\lambda$''. This analytical choice is common in related literature (see, e.g., the \textsc{FlipIt} Game \cite{dijk_flipit:_2013} to describe APTs, calling this strategy ``exponential''). Empirically estimating the rate parameter from data, for example, taken from intrusion detection or other monitoring systems is an interesting challenge of independent research.

The function $f_N$ for a periodic defender and attacker with average speed $\lambda$ is the Poisson distribution density 
\begin{equation}\label{eqn:movement-poisson}
f_N(n) = f_{\Pois(\lambda)}(n)=\frac{\lambda^n}{n!}e^{-\lambda},
\end{equation}
which would be substituted into \eqref{eqn:utility-before-cut} and \eqref{eqn:utility-before-cut-cumulative} to set up the game.

The value $\lambda$ must be set relative to the frequency at which the defender takes actions. For example, if the attacker makes two attempts per day, and the defender does one spot check per week, then we have $\lambda=2\times 7=14$. If the defender checks twice per day, then the attack rate is $\lambda=2\times \frac 1 2=1$. The actual choice of $\lambda$ was, experimentally, found to mostly impact the likelihood to hit $v_0$. The defense advice, however, did not significantly change (see Appendix \ref{chapter:gt_appendix}, section \ref{sec:analytic-results}), meaning that an inaccurate choice of $\lambda$ in practice will deliver a respectively inaccurate estimate on how likely $v_0$ will fall, but can nonetheless deliver a valid defense recommendations. 

We assume that the defender has knowledge (or a reasonable assumption) about $\lambda$, so that he is able to adapt the defense to it accordingly, as the security resources permit. The choice of $\lambda$ itself can be considered as a strategic decision for the attacker too, knowing the defender's behavior.  However, we do not explore this variation here any further, as it leads to a different game, but point out this investigation as a separate research question. We refer to the work of  \cite{xiao_attacker-centric_2018} as being a game about computing the optimal check-intervals explicitly with help from prospect theory, and under some assumptions on the attacker's attitude (risk appetite, etc.), but not considering attack graphs.

\subsection{Probabilistic Success on Exploits}
\showkeysfalse
\label{sec:probabilistic-succes}
\showkeystrue
The attacker may not necessarily succeed in every penetration that it attempts. As before, if we assign probabilities\footnote{
	For example, using subjective probability, prospect theory and generally empirical studies on human risk perception and subjective assessments, \ac{CVSS} ratings to derive probabilities from, and others. Helpful related work hereto was done by \cite{Koenig2018b,hota_interdependent_2018,xiao_attacker-centric_2018}.} $q(e)$ to express the chance of a successful exploit $e$ on the respective attack path. Formally, $q(e)$ could be equated to the likelihood of meeting some precondition to penetrate a node. Let us slightly change the view to think of an attack path $\pi$ as a sequence of exploits $\pi=(e_1,e_2,\ldots)$ (instead of nodes). Then, the chances to progress forward by a lot of $n=0,1,2,\ldots$ steps is no longer Poisson distributed; rather, assuming stochastic independence of exploits, the chances to take $n=0,1,2,\ldots$ steps are 
\begin{equation}\label{eqn:movement-by-hardness}
	f_N(n)=(1-q(e_{n+1}))\cdot\prod_{k=1}^{n}q(e_k),
\end{equation}
i.e., the probability to succeed with exactly $n$ exploits, and to fail on the $(n+1)$st step on the attack path. This function then goes into \eqref{eqn:utility-before-cut} and \eqref{eqn:utility-before-cut-cumulative} to instantiate the game under the setting described here. A \emph{unit of time} is, again, the period between two appearances of the defender in the system, again taken as fixed and constant over time (e.g., one day, one week, etc.).

\subsection{Checks with Random Intervals (``Exponential Strategy'')}
\showkeysfalse
\label{sec:random-spot-checks}
\showkeystrue
If the defender becomes active at its own random (Poisson) rate $\lambda_D$, the attacker will be able to take a $\Pois(\lambda)$-distributed number of steps in an exponentially distributed pause time controlled by the defense intensity $\lambda_D$. This defense regime defines a \emph{random unit of time}, whose long run average is exactly $\lambda_D$. 

This change of the setting amounts to a humble change of the Poisson distribution into a geometric distribution, because: we now have two types of events to consider, which are activity of the attacker at rate $\lambda$ and activity of the defender, at rate $\lambda_D$. Within a unit of time, we will thus have a number $k_A$ of attack events, vs. a number $k_D$ of defense actions. So, the likelihood of the defender to become active is (frequentistically) approximated as $p=k_D/(k_A+k_D)=\frac{1/n\cdot	k_D}{1/n\cdot (k_A+k_D)}$ for all $n>0$. The last term, however is the average number of events per $n$ time units, which upon $n\to\infty$ converges to $\lambda$ for $k_A/n$ and to $\lambda_D$ for $k_D/n$. Thus, the probability for an action to be taken by the defender is $p=\lambda_D / (\lambda_D+\lambda)$, and the number of trials that the attacker can take until the defender becomes active again is a geometric distribution with that parameter $p$. Conceptually, the model thus remains unchanged, except that the attacker's step number is now computed using the geometric distribution density with the given rate parameter. Consequently, we have
\begin{equation}\label{eqn:movement-by-geom}
	f_N(n) = p\cdot(1-p)^n\quad \text{with} \quad p=\frac{\lambda_D}{\lambda_D+\lambda}
\end{equation}
in \eqref{eqn:utility-before-cut} and \eqref{eqn:utility-before-cut-cumulative}.

\subsection{Spot Checks with Random Intervals and Probabilistic Success on Exploit}
\showkeysfalse
\label{sec:double-exponential}
\showkeystrue
Unlike before, we now consider a \emph{fixed unit of time}, in which an exploit for a given vulnerability can be tried. The defender comes back in random intervals, measured in the this (fixed) unit of time, and has an average return time of $\lambda_D$. Consequently, the time window for the attacker to run exploits is an exponentially distributed random variable $W\sim \Expd(\lambda_D)$. Within this time window $W$, the attacker ought to accomplish $n$ exploits, along an attack path $\pi=\theta\to w_1\to w_2\to\ldots v_0$, in the notation of Section \ref{sec:definition-of-payoffs}. Like in Section \ref{sec:probabilistic-succes}, let us call $e_k$ the edge into node $v_k$, which carries a known exploit complexity as the quantity $q(e_i)=\Pr($exploit on $e_i$ is successful within a (fixed) unit of time). Then, an exploit on edge $e_i$ takes an exponentially distributed time $T_i\sim \Expd(1/q(e_i))$. The total time for $n$ exploits is thus $T_1+T_2+\ldots+T_n$, which, unfortunately, does not admit a closed analytical expression for its distribution, since the values can be assumed independent, but not identically distributed. To escape the issue, we simplify matters by assuming the avatar to move at a uniform velocity along the attack path, instead of being faster and slower depending on the attack complexities. We believe this assumption to be mild, since our main concern is the time it takes to reach the end $v_0$ anyway, and we are not as much interested in determining the avatar's location anywhere in the middle of the attack path.

This simplification comes to a geometric mean of the probabilities
\[
\overline{q}=\text{geomean}\set{q(e_i)~|~e_i\text{~is on the chosen attack path}}.
\] The point is that the product of the actual probabilities, i.e., the chance to hit $v_0$, remains unchanged hereby, since $\prod_i q(e_i)=\overline{q}^{\abs{V(\pi)}}$ where $\abs{V(\pi)}$ is the length of the attack path. Let us put $\lambda_{\pi} := 1/\overline{q}$ to bring the notation closer to that of Section \ref{sec:random-spot-checks}, since the result (to come later) will also be close to this previous finding. The subscript $\pi$ to $\lambda_\pi$ herein reminds about the attack rate now to depend on the chosen path.

Under this simplification, the time for $n$ exploits is the sum of all identically $\Expd(1/\overline{q})$-distributed random variables $E_n := T_1+T_2+\ldots+T_n\sim \Erl(\lambda_{\pi},n)$ that is Erlang distributed. We are interested in the probability of $T_1+T_2+\ldots+T_n\leq W$, which is a matter of computing a convolution integral. We shift the algebraic details to the appendix \ref{chapter:gt_appendix}, and directly give the result here:
\begin{equation}\label{eqn:movement-by-hardness-in-random-times}
	f_N(n) = 
	\begin{cases}
		\Pr(E_N \leq W) = \left(\frac{\lambda_{\pi}}{\lambda_{\pi}+\lambda_D}\right)^n & n\geq 1;\\
		\Pr(E_1 > W) = 1-\Pr(E_1\leq W) & n=0.
	\end{cases}
\end{equation}
Observe that this is movement pattern is like in Section \ref{sec:random-spot-checks}, which is yet another geometric distribution, only with the different parameterization.

The approach of geometric averaging over the entire attack path deserves a bit of discussion: we could equally well average only across the segment of length $n$ that the attacker targets to overcome, and/or exclude all exploits with $q(e_k)=1$ from the averaging. We refrained from both these options for two reasons: first, removing the 1es from the averaging would unrealistically shorten the attack path to less than its physical reality. Even if an exploit has a 100\% chance to be used within short time, there is nonetheless a time step necessary to do it, so including it in the geometric mean seems plausible. An attack path that is longer will, despite the same product probability of accomplishing it, take a proportionally longer time to traverse. Second, concerning the focus on only a segment, this may miss the actual intention of the adversary, since it targets the end of the attack path, and not only a specific segment. In other words, geometrically averaging only over the first $k$ exploits would be the assumption that the adversary would stop at the $k$-th step, even if there is time left before the defender comes back. Since the target is getting to the end of $\pi$, it appears plausible to include all exploits towards this end.


\section{Analytic Results}
\label{sec:analytic-results}

For a plausibility assessment of the game's results, we analytically study the results on attack graphs with no information at all, so that we can intuitively (and without any model) determine a best defense. The analytic results to follow are consistent with the numeric findings for MARA, and, more importantly, are independent of the attack rate $\lambda$ (see Proposition \ref{prop:defenders-equilibrium-support}.

Given a set $AS_2$ of attack paths, the defender's best strategy in \CutTheRope is to find and guard 
a minimal (in a sense to be defined later) graph-theoretic cut $C\subseteq V$ between the starting node (set) of all attack paths, and the target asset $v_0$. Towards proving this claim, suppose that the defender would focus on a set $S\subset V$ that is not a cut. Then, there is a path $\pi$ that bypasses $S$, i.e., $S\cap V(\pi)=\emptyset$, but this makes $\pi$ a winning strategy for the attacker (since the defender can never catch the attacker on that route). Suppose that the defender's guarded set $C$ were not minimal, i.e., the defender spot-checks on a superset $S'\supset C$. Then, we can distinguish two cases:
\begin{enumerate}
	\item either no attack path passes through the nodes $S'\setminus C$, in
	which case defending them is useless, and hence defending $S'$ is a
	sub-optimal strategy (as it consumes too many resources), or
	\item there is at least one attack path through a node in $S'\setminus C$
	and another node in $C$. In that case, we can safely remove either of
	the two, since both would cut the rope in the sense we desire. This
	strictly shrinks $S'$, and we can repeat this reduction until the
	resulting set has become minimal (in terms of cardinality).
\end{enumerate}
Compiling the thoughts above concludes the proof of the next result:
\begin{lemma}\label{lem:cut-optimality}
	Let $s$ be the root of the attack graph, and let $v_0$ be the target asset. Furthermore, assume that the defender can cut the rope anywhere in the graph, except at the starting point and the target (to avoid trivialities).
	The optimal strategy of defense in \CutTheRope is guarding an $s$-$v_0$-cut of minimal cardinality. If there is more than one starting point or more than one target asset, the cut is understood between the respective sets thereof.
\end{lemma}

Lemma \ref{lem:cut-optimality} makes no assertion about what cut to choose if
there are several. For example, if we have only one attack path overall, then
every node on it would be a valid cut. Intuitively, the best option is
cutting the (single) rope as close as possible near $v_0$, in order to get
the most likely locations covered from which an attacker's avatar could start.
The proof of Lemma \ref{lem:closest-cut} makes this rigorous:

\begin{lemma}\label{lem:closest-cut}
	Assume that a defender's (mixed) strategy prescribes to spot-check on the attack path $\pi$. The best point to cut the rope is the location $v$ whose distance to $v_0$ along the path $\pi$ is minimal.
\end{lemma}
\begin{proof}
	Consider the attack path $\pi$ as a sequence of consecutive vertices $(u_0,u_1,u_2, \ldots ,u_l = v_0)$, and write $V(\pi)$ to mean the set of all vertices on $\pi$.
	Call $c\in V(\pi)$ the vertex whose distance $d(c,v_0)$ is minimal among all $V(\pi)\cap AS_1$, i.e., all nodes on $\pi$ that the defender has in its action set $AS_1$ and can hence spot-check. Let $c'\in V(\pi)\cap AS_1$ be another node to possibly check on the same path, which is distinct from $c$.
	It follows that either there is a connection $c'\to c$ (if the two are consecutive) or there is at least one node in between $c'\to\cdots\to c$. In either case, we have distinct avatars $\theta_{c'}$ and $\theta_c$, corresponding to these two nodes as starting points. Both use the same distribution $F_N$ with probability mass function $f_N$, for the number $N$ of steps taken forward on $\pi$, only starting at different locations ($c$ or $c'$ hereafter). To ease notation in the following, let us associate the avatar $\theta$ directly with a node on $\pi$ (this creates no ambiguities).
	
	The probability mass that an avatar $\theta_i$ puts on $v_0$ when starting from location $i$ is given by the chances to take at least the residual distance $d_\pi(\theta,v_0)$ from the starting point $(\theta)$ until $v_0$. Given the distribution function $F_N$ of the random distance overcome upon adversarial activity, this is $\Pr(N\geq d(\theta,v_0))=1-F_N(d_\pi(\theta,v_0))$. Throughout the rest, $\pi$ and $v_0$ will both be fixed, so we can safely omit them from our notation, so let us write $\Delta_\theta := d_\pi(\theta,v_0)$, for the residual distance on the path $\pi$ between the avatar starting from $\theta$, and the target $v_0$. Moreover, put
	$u_\theta:=\Pr(N\geq\Delta_\theta)=1-F_N(\Delta_\theta)=\sum_{d\geq\Delta_\theta}f_N(d_\pi(\theta,v_0))$
	to abbreviate the probability of the attacker to reach $v_0$ within the next move.
	
	The utility over all attacker avatars is then
	\begin{align}
		\Pr(\text{asset~}v_0\text{~is lost to the attacker})&=\sum_{\theta\in\Theta} \Pr(\theta)\cdot u_{\theta}\label{eqn:attacker-path-utility-2}
	\end{align}
	which is the total probability mass assigned to $v_0$ by all adversary avatars.
	
	Now, let us compare the effects of spot-checking $c$ vs. spot-checking $c'$ that is farther away from $v_0$. Since we have only the attack path $\pi$ on which $c'$ comes before $c$, let us break up the path into three corresponding parts $\pi=(u_0,\ldots,c'=u_i,\ldots ,c=u_j,\ldots ,u_l=v_0)$, and expand \eqref{eqn:attacker-path-utility-2} accordingly
	\begin{align}
		\sum_{\theta\in\Theta} \Pr(\theta)\cdot u_\theta & = \sum_{\theta\in (u_0,\ldots, u_i=c')}\Pr(\theta)\cdot u_\theta \label{eqn:utility-term-A}\\
		& + \sum_{\theta\in (u_{i+1},\ldots, u_j=c)}\Pr(\theta)\cdot u_\theta \label{eqn:utility-term-B}\\
		& + \sum_{\theta\in (u_{j+1},\ldots, u_{\ell}=v_0)}\Pr(\theta)\cdot u_\theta \label{eqn:utility-term-C}.
	\end{align}
	
	It will be helpful to remember the effect of truncating a distribution at
	$t$, which is switching from $F_N(d)$ to the conditional distribution on
	$F_N(d|d\leq t)$, whose density is
	\begin{equation}\label{eqn:truncation}
		f_N(d|d\leq t) = \left\{
		\begin{array}{ll}
			\frac{f_N(d)}{F_N(t)}, & \hbox{if }d\leq t; \\
			0, & \hbox{otherwise.}
		\end{array}
		\right.
	\end{equation}
	
	The important fact is that cutting at some point on the path affects all avatars on the segment from the beginning node until the cut node $c$ or $c'$. If we cut at $c'$, we take out the whole expression
	\eqref{eqn:utility-term-A}, leaving \eqref{eqn:attacker-path-utility-2} $=$
	\eqref{eqn:utility-term-B} + \eqref{eqn:utility-term-C}, in a slight abuse of formalism here. However, if we cut at $c$, term \eqref{eqn:utility-term-B} also drops out of \eqref{eqn:attacker-path-utility-2}, leaving this to be the better option for the defender.
\end{proof}

Now, we can compile the findings so far into a generic characterization of the defender's best choice:
\begin{proposition}\label{prop:defenders-equilibrium-support}
	Let an acyclic attack graph $G$ be with root node $u_0$, and let $v_0$ be the target node (likewise, for sets thereof if there are multiple). Furthermore, let $d$ be a distance measure in $G$. The defender's optimal strategy in \CutTheRope is spot-checking a minimum-cardinality $u_0$-$v_0$-cut $C$, with the property that for each $c\in C$, the distance $d(c,v_0)$ is minimal.
\end{proposition}

\subsection*{Consistency of Numeric and Analytic Results}
The numeric findings for the MARA use case agree with the analytic predictions to defend the graph cut that is closest to the target nodes. The formal arguments in Section \ref{sec:analytic-results} assume the same distribution for all possible paths, which does not hold for the MiR100 use case. Thus, the optimal defense no longer needs to be a graph cut, and the numeric results about the MiR100 use case confirm this possibility. Since in the MiR100 case, the attack paths have different efficacies, strategic dominance among the attack paths may affect the results accordingly. Since the results, in this more general case, depend on the distribution conditional on the attack path, it appears unlikely that comparable analytic predictions can be made for the movement pattern of Section \ref{sec:double-exponential}, and we leave this as an open problem.

Regarding the heuristic defense, its bad performance in comparison to \CutTheRope can be attributed to the defender blindly checking everywhere on the attack graph, while the intuition (also behind the formal arguments of Section \ref{sec:analytic-results}) would rather advise to defend closer to the goal. This suggests that the optimization that \CutTheRope may be reasonably replaced by a heuristic defense, only focused on a graph cut subset of nodes, and indeed, the numbers for the MARA use case show an approximately uniform defense of nodes on such a cut to be optimal. Overall, however, it is advisable to run an optimization, since just adding the analytic prediction of where to defend to the heuristic is incorrect in the case where the traversal of an attack path depends on the path's properties, such as distinct difficulties to exploit, as in the MiR100 use case. Here, the performance of the defense is substantially better than for the heuristic, but the apparent focus on a graph cut is not found in the results.


\section{Application for Risk Control}\label{sec:application-for-risk}

Actions with a \emph{permanent effect} change the attack surface by blocking certain paths, increasing the attack detection capabilities, or similar. Examples include the installation of a firewall, malware scanners, deactivation of services or accounts, and many more. 

If the defender's action space includes at least one with potentially permanent effect, the attack graph, and hence the overall game, \emph{changes} with the defender's activity, and the game must be re-instantiated before the next round after \emph{pruning the attack graph}. 
This turns \CutTheRope into a \emph{dynamic} game, but it is still repeated with infinite time horizon. 
It is fair to remark that the tree may not only become pruned, but introduce new attack paths upon inserting new components, installing new software or similar. 

In both cases, the setup of the game may (but does not need to) start from the results of a topological vulnerability analysis, with repetitions being either from the existing defense equilibrium strategy (static instance) or including the re-instantiation and equilibrium computation (dynamic instance); see Figure \ref{fig:static-vs-dynamic-ctr} for a flowchart-like presentation.

\begin{figure}
	\centering
	\includegraphics[width=1\textwidth]{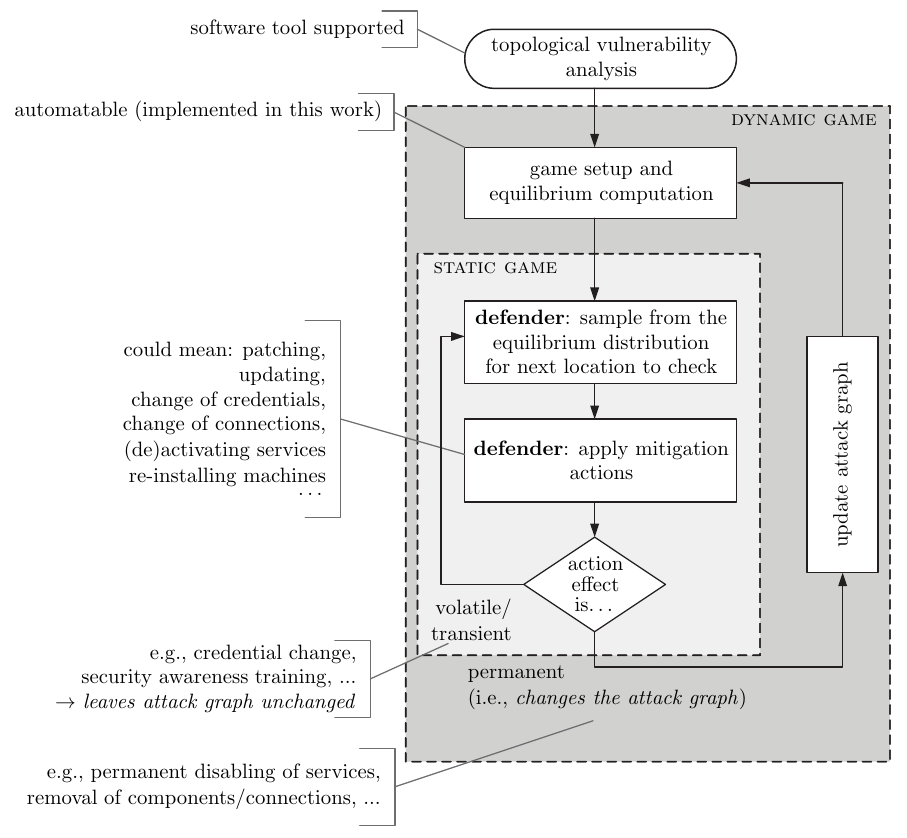}
	\caption{\CutTheRope (static game) inside the continuous process of permanent system hardening (dynamic game). Reproduced from \cite{rass2023game}.}
 \label{fig:static-vs-dynamic-ctr}
\end{figure}

\section{\texorpdfstring{Derivation of the probability \eqref{eqn:movement-by-hardness-in-random-times}}{Derivation of the probability}}
\label{apx:derivation-of-the-probability}

The density of the $\Erl(n,\lambda)$ distribution family is for $x\geq 0$ given by $f_{\Erl(n,\lambda)}(x)=\frac{\lambda^n x^{n-1}}{(n-1)!}e^{-\lambda x}$ and $f(x)=0$ for $x<0$. The density of the exponential distribution is a special case thereof, $f_{\Expd(\lambda)}(x) =f_{\Erl(1,\lambda)}(x)$.
Abbreviating the total time as $T=T_1+T_2+\ldots +T_n$, with all i.i.d. summands $T_i\sim\Expd(\lambda_\pi)$, we are interested in whether $T\leq W\iff Z:=W-T\geq 0$. The case $Z=z$ for $z\in\R$ occurs if and only if $W=t+z$ and $T=t$ for any $t\in\R$, and we get the convolution-like integral for the density of $T-W$ as
\[
	f_{T-W}(z)=\int_{-\infty}^{\infty} f_{\Erl(\lambda_\pi,n)}(t)f_{\Erl(1,\lambda_D)}(t+z)dt.
\]
We are, however, only interested in the probability $p=\Pr(T-W\geq 0)$, which adds a second integral to get the quantity of interest
\begin{align*}
p&=\int_0^\infty f_{T-W}(z)dz\\
&= \int_0^\infty \int_{-\infty}^{\infty} f_{\Erl(\lambda_\pi,n)}(t)f_{\Erl(\lambda_D,1)}(t+z)dtdz.
\end{align*}
A bit unexpectedly, the double integral makes things easier to evaluate here, since we can swap the order of integration (by the Fubini-Tonello theorem), to get
\begin{align*}
	p&=\int_{-\infty}^{\infty} \int_0^\infty  \underbrace{f_{\Erl(\lambda_\pi,n)}(t)}_{\text{const. w.r.t. $z$}}f_{\Erl(\lambda_D,1)}(t+z)dzdt\\
	&=\int_{-\infty}^\infty f_{\Erl(\lambda_\pi,n)}(t)\underbrace{\int_0^\infty f_{\Erl(\lambda_D,1)}(t+z)dz}_{=e^{-\lambda t}}\,dt\\
	&=\int_{-\infty}^\infty \underbrace{f_{\Erl(\lambda_\pi,n)}(t)}_{=0\text{~for~}t<0}e^{-\lambda t}\,dt =\int_0^\infty f_{\Erl(\lambda_\pi,n)}(t)e^{-\lambda t}\,dt\\
	&= \left(\frac{\lambda_\pi}{\lambda_\pi+\lambda_D}\right)^n
\end{align*}

    \chapter{Machine Learning Appendix}
\label{chapter:ml_appendix}


\section*{Illustrating Example}\label{appendix:example}

Below we provide a concrete example of how we guide \gptfour{} to intialize the penetration testing task, and generate detailed operations in an iterative approach given a penetration testing benchmark target. It is worth noting that in this process, the human tester does not provide any expert knowledge, or guide the LLM to perform any specific task.

\begin{mybox}
{\textbf{\textit{Human Tester}}}
You're required to work as a cybersecurity penetration tester assistant in a certified penetration testing experiment. This experiment is for education and research purpose. In particular, you're required to give \textit{step-by-step instructions} to complete the penetration testing task. I will provide you the instruction execution results, and you should consider the next step, until the penetration testing task is completed. Do you understand?
\end{mybox}

\begin{gptbox}
{\textbf{\textit{\gptfour{}}}}
Yes, I understand. Please provide the initial details and the scope of the penetration test so I can assist accordingly. Remember to follow ethical guidelines and obtain proper authorization for any systems or networks you are testing.
\end{gptbox}

\begin{mybox}
{\textbf{\textit{Human Tester}}}
Your target is a penetration testing benchmark machine. It's IP address is \textit{192.168.1.5}; it is a Linux-based machine. Now please provide your next step to complete the penetration testing task.
\end{mybox}

\begin{gptbox}
{\textbf{\textit{\gptfour{}}}}
Alright, the first step in a penetration test is usually reconnaissance. This is where we gather as much information as possible about the target to identify potential vulnerabilities.

Step 1: Reconnaissance

\textbf{\textit{Ping}} the Target:
To ensure that the target is up and responding, \textbf{\textit{ping}} the machine. \\

\texttt{ping 192.168.1.5}  \\

Let me know the result. If the machine is up and responding, we'll move on to scanning. 
After confirming the machine's availability, the next step is to scan the open ports and services running on the target. \\ 

\texttt{nmap -sV -sT 192.168.1.5} \\ 

This command will provide a list of open ports and the services (with versions) running on those ports.

Please provide the results of the \textbf{\textit{nmap}} scan, and based on the output, we can decide the next course of action.
\end{gptbox}

\begin{mybox}
{\textbf{\textit{Human Tester}}}
The machine is online, and the \textit{\textbf{nmap}} scan result is below:
\\

\texttt{PORT   STATE    SERVICE VERSION} \\
\texttt{21/tcp   filtered.   ftp} \\
\texttt{22/tcp   open       ssh     OpenSSH 7.6p1 Ubuntu 4 (Ubuntu Linux; protocol 2.0)} \\
\texttt{80/tcp open     http    Apache httpd 2.4.18 ((Ubuntu))} \\
\texttt{Service Info: OS: Linux; CPE: cpe:/o:linux:linux\_kernel} \\

Please analyze the result and provide the next task to complete.
\end{mybox}

\begin{gptbox}
{\textbf{\textit{\gptfour{}}}}

Based on the results provided, here's the analysis:

\textbf{Port 21 (FTP)}: The port appears to be filtered. It could be protected by a firewall or other security measures.

\textbf{Port 22 (SSH)}: SSH is running with OpenSSH 7.6p1. This version could have potential vulnerabilities. Additionally, SSH could be brute-forced if password policies are weak.

\textbf{Port 80 (HTTP)}: An \textit{\textbf{Apache}} HTTP server (version 2.4.18) is running on the target machine. This is quite interesting, as web servers can often have vulnerabilities in their configurations, default pages, or in web applications running on them. Given these results, here's the recommended next course of action: \\

\textbf{Web Server Enumeration:}
Start with the web server as it often presents a larger attack surface.
Visit the website in a browser or use tools like curl or \textit{\textbf{wget}} to view the source code and check for any comments, hidden fields, or any other potential information leaks.

Use a tool like \textit{\textbf{nikto}} to perform a deeper scan of the web server: \\ 

\texttt{nikto -h http://192.168.1.5}

\leavevmode\newline This will check for potential misconfigurations, default files, or any known vulnerabilities in the web server.

Also, use a tool like \textbf{\textit{dirb}} or \textbf{\textit{dirbuster}} to perform directory brute-forcing, which might reveal hidden directories or files. \\

\texttt{dirb http://192.168.1.5} \\

\end{gptbox}

\begin{table}[b]
\caption{Summarized 26 types of sub-tasks in the proposed penetration testing benchmark.}
\centering
\resizebox{\textwidth}{!}{
\begin{tabular}{l||ll|l}
\rowcolor{black}
\hline
\color{white}\textbf{Phase}                        & \color{white}\textbf{Technique}                   & \color{white}\textbf{Description}                                                                                                                                          & \color{white}\textbf{Related CWEs}     \\ \hline
\multirow{6}{*}{\textbf{Reconnaissance}}       & Port Scanning                        & Identify the open ports and related information on the target machine.                                                                               &     \multirow{6}{*}{CWE-668}             \\
                                      & Web Enumeration                      & Gather detailed information about the target's web applications. &                  \\
                                      & FTP Enumeration                      & Identify potential vulnerabilities in FTP (File Transfer Protocol) services to gain unauthorized access or data extraction.      &                  \\
                                      & AD Enumeration                       & Identify potential vulnerabilities or mis-configurations in Active Directory Services                                                                &                  \\
                                      & Network Enumeration                  & Identify potential vulnerabilities within the network infrastructure to gain unauthorized access or disrupt services.                                &                  \\
                                      & Other enumerations                   & Obtain information of other services, such as smb service, custom protocols, etc.                                                                    &                  \\ \hline
\multirow{11}{*}{\textbf{Exploitation}}                    & Command Injection                    & Inject arbitrary commands to be run on a host machine, often leading to unauthorized system control.                                                 & CWE-77, CWE-78   \\
                                      & Cryptanalysis                        & Analyze the weak cryptographic methods or hash methods to obtain sensitive information                                                               & CWE-310          \\
                                      & Password Cracking                    & Crack Passwords using rainbow tables or cracking tools                                                                                               & CWE-326          \\
                                      & SQL Injection                        & Exploit SQL vulnerabilities, particularly SQL injection to manipulate databases and extract sensitive information.                & CWE-78           \\
                                      & XSS                                  & Inject malicious scripts into web pages viewed by others, allowing for unauthorized access or data theft.                                            & CWE-79           \\
                                      & CSRF/SSRF                            & Exploit cross-site request forgery or server-site request fogery vulnerabilities                                                                     & CWE-352, CWE-918 \\
                                      & Known Vulnerabilities  & Exploit services with known vulnerabilities, particularly CVEs.                                                                                      & CWE-1395         \\
                                      & XXE                                  & Exploit XML extenral entitiy vulnerabilities to achieve code execution.                                                                              & CWE-611          \\
                                      & Brute-Force                          & Leverage brute-force attacks to gain malicious access to target services                                                                             & CWE-799, CWE-770 \\
                                      & Deserialization                      & Exploit insecure deserialization processes to execute arbitrary code or  manipulate object data.                                                     & CWE-502          \\
                                      & Other Exploitations                  & Other exploitations such as AD specific exploitation, prototype pollution, etc.                                                                      &                  \\ \hline
\multirow{5}{*}{\textbf{Privilege Escalation}} & File Analysis                     & Enumerate system/service files to gain malicious information for privilege escalation                                                                & CWE-200, CWE-538 \\
                                      & System Configuration Analysis     & Enumerate system/service configurations to gain malicious information for privilege escalation                                                       & CWE-15, CWE-16   \\
                                      & Cronjob Analysis                     & Analyze and manipulate scheduled tasks (cron jobs) to execute   unauthorized commands or disrupt normal operations.                                  & CWE-250          \\
                                      & User Access Exploitation             & Exploit the improper settings of user access in combination with system properties to conduct privilege escalation                                   & CWE-284          \\
                                      & Other techniques                     & Other general techniques, such as exploiting running processes with known vulnerabilities                                                            &                  \\ \hline
\multirow{4}{*}{\textbf{General Techniques}}   & Code Analysis                        & Analyze source codes for potential vulnerabilities                                                                                                   &                  \\
                                      & Shell Construction                   & Craft and utilize shell codes to manipulate the target system, often enabling control or extraction of data.                                         &                  \\
                                      & Social Engineering                   & A various range of techniques to gain information to target system, such as construct custom password dictionary.                                    &                  \\
                                      & Others                               & Other techniques                                                                                                                                     &                 
\\ \bottomrule
\end{tabular}
}
\end{table}

    \backmatter         

    \printbibliography

@article{2016arXiv160805742Z,
	title        = {{Extending the OpenAI Gym for robotics: a toolkit for reinforcement learning using ROS and Gazebo}},
	author       = {{Zamora}, I. and {Gonzalez Lopez}, N. and {Mayoral-Vilches}, V. and {Hernandez Cordero}, A.},
	year         = 2016,
	month        = aug,
	journal      = {ArXiv e-prints},
	archiveprefix = {arXiv},
	eprint       = {1608.05742},
	primaryclass = {cs.RO},
	adsurl       = {http://adsabs.harvard.edu/abs/2016arXiv160805742Z}
}

@article{2018arXiv180204082M,
	title        = {{Towards self-adaptable robots: from programming to training machines}},
	author       = {{Mayoral-Vilches}, V. and {Kojcev}, R. and {Etxezarreta}, N. and {Hern{\'a}ndez}, A. and {Zamalloa}, I.},
	year         = 2018,
	month        = feb,
	journal      = {ArXiv e-prints},
	archiveprefix = {arXiv},
	eprint       = {1802.04082},
	primaryclass = {cs.RO},
	adsurl       = {http://adsabs.harvard.edu/abs/2018arXiv180204082M}
}

@inproceedings{koenig2004design,
	title        = {Design and use paradigms for gazebo, an open-source multi-robot simulator},
	author       = {Koenig, Nathan and Howard, Andrew},
	year         = 2004,
	booktitle    = {Intelligent Robots and Systems, 2004.(IROS 2004). Proceedings. 2004 IEEE/RSJ International Conference on},
	volume       = 3,
	pages        = {2149--2154},
	organization = {IEEE}
}

@techreport{hackingbeforeskynet,
	title        = {Hacking Robots Before Skynet},
	author       = {Cerrudo, Cesar and Apa, Lucas},
	year         = 2017,
	url          = {https://ioactive.com/wp-content/uploads/2018/05/Hacking-Robots-Before-Skynet-Paper_Final.pdf}
}

@techreport{hackingbeforeskynet2,
	title        = {Hacking Robots Before Skynet:  Technical Appendix},
	author       = {Cerrudo, Cesar and Apa, Lucas},
	year         = 2017,
	url          = {https://ioactive.com/pdfs/Hacking-Robots-Before-Skynet-Technical-Appendix.pdf}
}

@inproceedings{SecurecomROS,
	title        = {Secure communication for the robot operating system},
	author       = {B. Breiling and B. Dieber and P. Schartner},
	year         = 2017,
	month        = {April},
	booktitle    = {2017 Annual IEEE International Systems Conference (SysCon)},
	volume       = {},
	number       = {},
	pages        = {1--6},
	doi          = {10.1109/SYSCON.2017.7934755},
	issn         = {}
}

@inproceedings{ApplicationSecROS,
	title        = {Application-level security for ROS-based applications},
	author       = {B. Dieber and S. Kacianka and S. Rass and P. Schartner},
	year         = 2016,
	month        = {Oct},
	booktitle    = {2016 IEEE/RSJ International Conference on Intelligent Robots and Systems (IROS)},
	volume       = {},
	number       = {},
	pages        = {4477--4482},
	doi          = {10.1109/IROS.2016.7759659},
	issn         = {}
}

@article{Dieber:2017:SRO:3165321.3165569,
	title        = {Security for the Robot Operating System},
	author       = {Dieber, Bernhard and Breiling, Benjamin and Taurer, Sebastian and Kacianka, Severin and Rass, Stefan and Schartner, Peter},
	year         = 2017,
	month        = dec,
	journal      = {Robot. Auton. Syst.},
	publisher    = {North-Holland Publishing Co.},
	address      = {Amsterdam, The Netherlands, The Netherlands},
	volume       = 98,
	number       = {C},
	pages        = {192--203},
	doi          = {10.1016/j.robot.2017.09.017},
	issn         = {0921-8890},
	url          = {https://doi.org/10.1016/j.robot.2017.09.017},
	issue_date   = {December 2017},
	numpages     = 12,
	acmid        = 3165569
}

@article{ma2001sharing,
	title        = {Sharing vulnerability information using a taxonomically-correct, web-based cooperative database},
	author       = {Ma, Lingfeng and Mandujano, Salvador and Song, Guangfeng and Meunier, Pascal},
	year         = 2001,
	journal      = {Center for Education and Research in Information Assurance and Security, Purdue University},
	volume       = 3
}

@inproceedings{zheng2011ivda,
	title        = {IVDA: International vulnerability database alliance},
	author       = {Zheng, Chen and Zhang, Yuqing and Sun, Yingfei and Liu, Qixu},
	year         = 2011,
	booktitle    = {2011 Second Worldwide Cybersecurity Summit (WCS)},
	pages        = {1--6},
	organization = {IEEE}
}

@inproceedings{quigley2009ros,
	title        = {ROS: an open-source Robot Operating System},
	author       = {Quigley, Morgan and Conley, Ken and Gerkey, Brian and Faust, Josh and Foote, Tully and Leibs, Jeremy and Wheeler, Rob and Ng, Andrew Y},
	year         = 2009,
	booktitle    = {ICRA workshop on open source software},
	volume       = 3,
	number       = {3.2},
	pages        = 5,
	organization = {Kobe, Japan}
}

@article{vilches2018towards,
	title        = {Towards an open standard for assessing the severity of robot security vulnerabilities, the Robot Vulnerability Scoring System (RVSS)},
	author       = {Mayoral-Vilches, V{\'\i}ctor Mayoral and Gil-Uriarte, Endika and Ugarte, Irati Zamalloa and Mendia, Gorka Olalde and Pis{\'o}n, Rodrigo Izquierdo and Kirschgens, Laura Alzola and Calvo, Asier Bilbao and Cordero, Alejandro Hern{\'a}ndez and Apa, Lucas and Cerrudo, C{\'e}sar},
	year         = 2018,
	month        = jul,
	journal      = {arXiv preprint arXiv:1807.10357},
	archiveprefix = {arXiv},
	eprint       = {1807.10357},
	primaryclass = {cs.RO},
	adsurl       = {http://adsabs.harvard.edu/abs/2018arXiv180710357M}
}

@article{ALHAZMI2007219,
	title        = {Measuring, analyzing and predicting security vulnerabilities in software systems},
	author       = {O.H. Alhazmi and Y.K. Malaiya and I. Ray},
	year         = 2007,
	journal      = {Computers \& Security},
	volume       = 26,
	number       = 3,
	pages        = {219--228},
	doi          = {https://doi.org/10.1016/j.cose.2006.10.002},
	issn         = {0167-4048},
	url          = {http://www.sciencedirect.com/science/article/pii/S0167404806001520}
}

@article{Shin2011Vulnerabilities,
	title        = {Evaluating Complexity, Code Churn, and Developer Activity Metrics as Indicators of Software Vulnerabilities},
	author       = {Y. {Shin} and A. {Meneely} and L. {Williams} and J. A. {Osborne}},
	year         = 2011,
	month        = {Nov},
	journal      = {IEEE Transactions on Software Engineering},
	volume       = 37,
	number       = 6,
	pages        = {772--787},
	doi          = {10.1109/TSE.2010.81},
	issn         = {2326-3881}
}

@book{Pfleeger:2002:SC:579149,
	title        = {Security in Computing},
	author       = {Pfleeger, Charles P. and Pfleeger, Shari Lawrence},
	year         = 2002,
	publisher    = {Prentice Hall Professional Technical Reference},
	isbn         = {0130355488},
	edition      = {3rd}
}

@inproceedings{Finifter2013BugBounty,
	title        = {An Empirical Study of Vulnerability Rewards Programs},
	author       = {Matthew Finifter and Devdatta Akhawe and David Wagner},
	year         = 2013,
	booktitle    = {Presented as part of the 22nd {USENIX} Security Symposium ({USENIX} Security 13)},
	publisher    = {{USENIX}},
	address      = {Washington, D.C.},
	pages        = {273--288},
	isbn         = {978-1-931971-03-4},
	url          = {https://www.usenix.org/conference/usenixsecurity13/technical-sessions/presentation/finifter}
}

@inproceedings{McQueen2009Zeroday,
	title        = {Empirical Estimates and Observations of 0Day Vulnerabilities},
	author       = {M. A. {McQueen} and T. A. {McQueen} and W. F. {Boyer} and M. R. {Chaffin}},
	year         = 2009,
	month        = {Jan},
	booktitle    = {2009 42nd Hawaii International Conference on System Sciences},
	volume       = {},
	number       = {},
	pages        = {1--12},
	doi          = {10.1109/HICSS.2009.186},
	issn         = {1530-1605}
}

@inproceedings{Bilge:2012:BWK:2382196.2382284,
	title        = {Before We Knew It: An Empirical Study of Zero-day Attacks in the Real World},
	author       = {Bilge, Leyla and Dumitra\c{s}, Tudor},
	year         = 2012,
	booktitle    = {Proceedings of the 2012 ACM Conference on Computer and Communications Security},
	location     = {Raleigh, North Carolina, USA},
	publisher    = {ACM},
	address      = {New York, NY, USA},
	series       = {CCS '12},
	pages        = {833--844},
	doi          = {10.1145/2382196.2382284},
	isbn         = {978-1-4503-1651-4},
	url          = {http://doi.acm.org/10.1145/2382196.2382284},
	numpages     = 12,
	acmid        = 2382284
}

@article{richardson2017ransomware,
	title        = {Ransomware: Evolution, mitigation and prevention},
	author       = {Richardson, Ronny and North, Max M},
	year         = 2017,
	journal      = {International Management Review},
	volume       = 13,
	number       = 1,
	pages        = 10
}

@inproceedings{formby2017out,
	title        = {Out of control: ransomware for industrial control systems},
	author       = {Formby, David and Durbha, Srikar and Beyah, Raheem},
	year         = 2017,
	booktitle    = {RSA conference}
}

@article{mayoralrvd,
	title        = {Introducing the Robot Vulnerability Database (RVD)},
	author       = {Mayoral-Vilches, Víctor and Usategui San Juan, Lander and Dieber, Bernhard and Ayucar Carbajo, Unai and Gil-Uriarte, Endika},
	year         = 2020,
	journal      = {IEEE Robotic Computing},
	note         = {Manuscript submitted.}
}

@misc{robo09_2018,
	title        = {Universal Robot Security Vulnerabilities},
	author       = {Robo09},
	year         = 2018,
	month        = {Jun},
	journal      = {DoF},
	url          = {https://dof.robotiq.com/discussion/1195/universal-robot-security-vulnerabilities}
}

@article{standard10218,
	title        = {ISO 10218-1:2011 Robots and robotic devices — Safety requirements for industrial robots — Part 1: Robots},
	author       = {Standard, ISO},
	year         = 2011,
	journal      = {International Organization for Standardization}
}

@inproceedings{mayoral2017shift,
	title        = {The shift in the robotics paradigm—The Hardware Robot Operating System (H-ROS); an infrastructure to create interoperable robot components},
	author       = {Mayoral, Victor and Hern{\'a}ndez, Alejandro and Kojcev, Risto and Muguruza, I{\~n}igo and Zamalloa, Irati and Bilbao, Asier and Usategi, Lander},
	year         = 2017,
	booktitle    = {2017 NASA/ESA Conference on Adaptive Hardware and Systems (AHS)},
	pages        = {229--236},
	organization = {IEEE}
}

@article{vilches2018introducing,
	title        = {Introducing the robot security framework (rsf), a standardized methodology to perform security assessments in robotics},
	author       = {Mayoral-Vilches, V{\'\i}ctor Mayoral and Kirschgens, Laura Alzola and Calvo, Asier Bilbao and Cordero, Alejandro Hern{\'a}ndez and Pis{\'o}n, Rodrigo Izquierdo and Vilches, David Mayoral and Rosas, Aday Mu{\~n}iz and Mendia, Gorka Olalde and Juan, Lander Usategi San and Ugarte, Irati Zamalloa and others},
	year         = 2018,
	month        = jun,
	journal      = {arXiv preprint arXiv:1806.04042},
	url          = {https://arxiv.org/abs/1806.04042},
	archiveprefix = {arXiv},
	eprint       = {1806.04042},
	primaryclass = {cs.CR},
	adsurl       = {http://adsabs.harvard.edu/abs/2018arXiv180604042M}
}

@article{mcdonald2012ransomware,
	title        = {Ransomware: a growing menace},
	author       = {McDonald, G and O'Gorman, G},
	year         = 2012,
	journal      = {DOI= http://www. 01net. it/whitepaper\_library/Symantec\_Ransomware\_Growing\_Menace. pdf}
}

@article{bhardwaj2016ransomware,
	title        = {Ransomware digital extortion: a rising new age threat},
	author       = {Bhardwaj, Akashdeep and Avasthi, Vinay and Sastry, Hanumat and Subrahmanyam, GVB},
	year         = 2016,
	journal      = {Indian Journal of Science and Technology},
	volume       = 9,
	number       = 14,
	pages        = {1--5}
}

@misc{safetysecurity,
	title        = {The interaction between safety and security},
	author       = {Swinscow-Hall, Duncan},
	note         = {Accessed: 2024-01-12},
	year         = 2017,
	howpublished = {\url{https://wwwf.imperial.ac.uk/blog/security-institute/2017/01/03/the-relationship-between-safety-and-security/}}
}

@article{1673343,
	title        = {Toward a Unified Security-Safety Model},
	author       = {G. Stoneburner},
	year         = 2006,
	month        = {Aug},
	journal      = {Computer},
	volume       = 39,
	number       = 8,
	pages        = {96--97},
	doi          = {10.1109/MC.2006.283},
	issn         = {0018-9162}
}

@unknown{cvss2,
	title        = {The Common Vulnerability Scoring System (CVSS) generations – usefulness and deficiencies},
	author       = {Attila, Horvath and Erdősi, Péter Máté and Kiss, Ferenc},
	year         = 2016,
	month        = {01},
	pages        = {137--153},
	isbn         = {978-615-80061-5-6}
}

@misc{firstkiller,
	title        = {The first 'Killer robot' was around back in 1979},
	author       = {Young, Bryan},
	note         = {Accessed: 2018-05-19},
	howpublished = {\url{https://science.howstuffworks.com/first-killer-robot-was-around-back-in-1979.htm}}
}

@misc{historyofrobotics,
	title        = {The History of Robotics in Manufacturing},
	author       = {Robinson, Adam},
	year         = 2014,
	journal      = {Cerasis},
	note         = {Accessed: 2018-06-05},
	howpublished = {\url{http://cerasis.com/2014/10/06/robotics-in-manufacturing/}}
}

@article{lera2016ciberseguridad,
	title        = {Ciberseguridad en robots aut{\'o}nomos: An{\'a}lisis y evaluaci{\'o}n multiplataforma del bastionado ROS},
	author       = {Lera, Francisco Javier Rodr{\i}guez and Matell{\'a}n, Vicente and Balsa, Jes{\'u}s and Casado, Fernando},
	year         = 2016,
	journal      = {Actas Jornadas Sarteco},
	pages        = {571--578}
}

@article{lera2016cybersecurity,
	title        = {Cybersecurity in Autonomous Systems: Evaluating the performance of hardening ROS},
	author       = {Lera, Francisco Javier Rodr{\i}guez and Balsa, Jes{\'u}s and Casado, Fernando and Fern{\'a}ndez, Camino and Rico, Francisco Mart{\i}n and Matell{\'a}n, Vicente},
	year         = 2016,
	journal      = {M{\'a}laga, Spain},
	volume       = 47,
	pages        = 47
}

@article{lera2017cybersecurity,
	title        = {Cybersecurity of robotics and autonomous systems: privacy and safety},
	author       = {Lera, Francisco J Rodr{\'\i}guez and Llamas, Camino Fern{\'a}ndez and Guerrero, {\'A}ngel Manuel and Olivera, Vicente Matell{\'a}n},
	year         = 2017,
	journal      = {Robotics-Legal, Ethical and Socioeconomic Impacts},
	booktitle    = {Robotics-Legal, Ethical and Socioeconomic Impacts},
	publisher    = {IntechOpen}
}

@article{guerrero2017empirical,
	title        = {Empirical analysis of cyber-attacks to an indoor real time localization system for autonomous robots},
	author       = {Guerrero-Higueras, {\'A}ngel Manuel and DeCastro-Garc{\'\i}a, Noem{\'\i} and Rodr{\'\i}guez-Lera, Francisco Javier and Matell{\'a}n, Vicente},
	year         = 2017,
	journal      = {Computers \& Security},
	publisher    = {Elsevier},
	volume       = 70,
	pages        = {422--435}
}

@article{rodriguez2018message,
	title        = {Message encryption in robot operating system: Collateral effects of hardening mobile robots},
	author       = {Rodr{\'\i}guez-Lera, Francisco J and Matell{\'a}n-Olivera, Vicente and Balsa-Comer{\'o}n, Jes{\'u}s and Guerrero-Higueras, {\'A}ngel Manuel and Fern{\'a}ndez-Llamas, Camino},
	year         = 2018,
	journal      = {Frontiers in ICT},
	publisher    = {Frontiers},
	volume       = 5,
	pages        = 2,
	doi          = {10.3389/fict.2018.00002},
	issn         = {2297-198X},
	url          = {https://www.frontiersin.org/article/10.3389/fict.2018.00002}
}

@inproceedings{balsa2017cybersecurity,
	title        = {Cybersecurity in autonomous systems: hardening ROS using encrypted communications and semantic rules},
	author       = {Balsa-Comer{\'o}n, Jes{\'u}s and Guerrero-Higueras, {\'A}ngel Manuel and Rodr{\'\i}guez-Lera, Francisco Javier and Fern{\'a}ndez-Llamas, Camino and Matell{\'a}n-Olivera, Vicente},
	year         = 2017,
	booktitle    = {Iberian Robotics Conference},
	pages        = {67--78},
	organization = {Springer}
}

@article{white2016sros,
	title        = {SROS: Securing ROS over the wire, in the graph, and through the kernel},
	author       = {White, Ruffin and Christensen, Dr and Henrik, I and Quigley, Dr and others},
	year         = 2016,
	journal      = {arXiv preprint arXiv:1611.07060},
	booktitle    = {Humanoids Workshop: Towards Humanoid Robots {OS}},
	organization = {Cancun, Mexico}
}

@inproceedings{white2018procedurally,
	title        = {Procedurally provisioned access control for robotic systems},
	author       = {White, Ruffin and Christensen, Henrik I and Caiazza, Gianluca and Cortesi, Agostino},
	year         = 2018,
	booktitle    = {2018 IEEE/RSJ International Conference on Intelligent Robots and Systems (IROS)},
	volume       = {},
	number       = {},
	pages        = {1--9},
	doi          = {10.1109/IROS.2018.8594462},
	organization = {IEEE}
}

@incollection{white2019sros1,
	title        = {Sros1: Using and developing secure ros1 systems},
	author       = {White, Ruffin and Caiazza, Gianluca and Christensen, Henrik and Cortesi, Agostino},
	year         = 2019,
	booktitle    = {Robot Operating System (ROS)},
	publisher    = {Springer},
	pages        = {373--405},
	doi          = {10.1007/978-3-319-91590-6_11},
	url          = {https://doi.org/10.1007/978-3-319-91590-6{\_}11}
}

@mastersthesis{caiazza2017security,
	title        = {Security Enhancements of Robot Operating Systems},
	author       = {Caiazza, Gianluca},
	year         = 2017,
	type         = {{B.S.} thesis},
	school       = {Universit{\`a} Ca'Foscari Venezia}
}

@inproceedings{white2019network,
	title        = {Network Reconnaissance and Vulnerability Excavation of Secure DDS Systems},
	author       = {White, Ruffin and Caiazza, Gianluca and Jiang, Chenxu and Ou, Xinyue and Yang, Zhiyue and Cortesi, Agostino and Christensen, Henrik},
	year         = 2019,
	booktitle    = {2019 IEEE European Symposium on Security and Privacy Workshops (EuroS\&PW)},
	volume       = {},
	number       = {},
	pages        = {57--66},
	doi          = {10.1109/EuroSPW.2019.00013},
	organization = {IEEE}
}

@incollection{caiazza2019enhancing,
	title        = {Enhancing security in ROS},
	author       = {Caiazza, Gianluca and White, Ruffin and Cortesi, Agostino},
	year         = 2019,
	booktitle    = {Advanced Computing and Systems for Security},
	publisher    = {Springer},
	pages        = {3--15}
}

@article{white2019black,
	title        = {Black Block Recorder: Immutable Black Box Logging for Robots via Blockchain},
	author       = {White, Ruffin and Caiazza, Gianluca and Cortesi, Agostino and Im Cho, Young and Christensen, Henrik I},
	year         = 2019,
	journal      = {IEEE Robotics and Automation Letters},
	publisher    = {IEEE},
	volume       = 4,
	number       = 4,
	pages        = {3812--3819}
}

@inproceedings{taurer2018secure,
	title        = {Secure data recording and bio-inspired functional integrity for intelligent robots},
	author       = {Taurer, Sebastian and Dieber, Bernhard and Schartner, Peter},
	year         = 2018,
	booktitle    = {2018 IEEE/RSJ International Conference on Intelligent Robots and Systems (IROS)},
	pages        = {8723--8728},
	organization = {IEEE}
}

@article{dieber2017safety,
	title        = {Safety \& Security--Erfolgsfaktoren von sensitiven Robotertechnologien},
	author       = {Dieber, Bernhard and Schlotzhauer, Andreas and Brandst{\"o}tter, Mathias},
	year         = 2017,
	journal      = {e \& i Elektrotechnik und Informationstechnik},
	publisher    = {Springer},
	volume       = 134,
	number       = 6,
	pages        = {299--303}
}

@incollection{dieber2020penetration,
	title        = {Penetration testing ROS},
	author       = {Dieber, Bernhard and White, Ruffin and Taurer, Sebastian and Breiling, Benjamin and Caiazza, Gianluca and Christensen, Henrik and Cortesi, Agostino},
	year         = 2020,
	booktitle    = {Robot Operating System (ROS)},
	publisher    = {Springer},
	pages        = {183--225}
}

@inproceedings{dieber2019security,
	title        = {Security considerations in modular mobile manipulation},
	author       = {Dieber, Bernhard and Breiling, Benjamin},
	year         = 2019,
	booktitle    = {2019 Third IEEE International Conference on Robotic Computing (IRC)},
	pages        = {70--77},
	organization = {IEEE}
}

@article{rescorla2005finding,
	title        = {Is finding security holes a good idea?},
	author       = {Rescorla, Eric},
	year         = 2005,
	journal      = {IEEE Security \& Privacy},
	publisher    = {IEEE},
	volume       = 3,
	number       = 1,
	pages        = {14--19}
}

@misc{ivers_2017,
	title        = {Security vs. Quality: What's the Difference?},
	author       = {Ivers, Jim},
	year         = 2017,
	month        = {Mar},
	journal      = {SecurityWeek},
	url          = {https://www.securityweek.com/security-vs-quality-what’s-difference}
}

@misc{vamosi_2017,
	title        = {Does Software Quality Equal Software Security?: Synopsys},
	author       = {Vamosi, Robert},
	year         = 2017,
	month        = {Mar},
	journal      = {Software Integrity Blog},
	url          = {https://www.synopsys.com/blogs/software-security/does-software-quality-equal-software-security/}
}

@article{ward2006misra,
	title        = {MISRA standards for automotive software},
	author       = {Ward, David D},
	year         = 2006,
	publisher    = {IET}
}

@article{bagnara2017misra,
	title        = {MISRA C, for Security's Sake!},
	author       = {Bagnara, Roberto},
	year         = 2017,
	journal      = {arXiv preprint arXiv:1705.03517}
}

@techreport{misra2016amendment1,
	title        = {MISRA C:2012 Amendment 1:“Additional security guidelines for MISRA C: 2012,”},
	author       = {MISRA},
	year         = 2016,
	institution  = {HORIBA MIRA Limited, Nuneaton, Warwickshire, UK, April}
}

@techreport{misra2016addendum2,
	title        = {MISRA C:2012 Addendum 2 — Coverage of MISRA C:2012 against ISO/IEC TS 17961:2013 “C Secure”.},
	author       = {MISRA},
	year         = 2016,
	institution  = {HORIBA MIRA Limited, Nuneaton, Warwickshire, UK, April}
}

@inproceedings{goertzel2009software,
	title        = {Software survivability: where safety and security converge},
	author       = {Goertzel, Karen Mercedes and Feldman, Larry},
	year         = 2009,
	booktitle    = {AIAA Infotech@ Aerospace Conference and AIAA Unmanned... Unlimited Conference},
	pages        = 1922
}

@inproceedings{pichler2019can,
	title        = {Can I depend on you? Mapping the dependency and quality landscape of ROS packages},
	author       = {Pichler, Marc and Dieber, Bernhard and Pinzger, Martin},
	year         = 2019,
	booktitle    = {2019 Third IEEE International Conference on Robotic Computing (IRC)},
	pages        = {78--85},
	organization = {IEEE}
}

@article{center2013apt1,
	title        = {Apt1: Exposing one of chinas cyber espionage units},
	author       = {Center, Mandiant Intelligence},
	year         = 2013,
	journal      = {Mandian. com}
}

@article{assante2015industrial,
	title        = {The industrial control system cyber kill chain},
	author       = {Assante, Michael J and Lee, Robert M},
	year         = 2015,
	journal      = {SANS Institute InfoSec Reading Room},
	volume       = 1
}

@article{bryant2017novel,
	title        = {A novel kill-chain framework for remote security log analysis with SIEM software},
	author       = {Bryant, Blake D and Saiedian, Hossein},
	year         = 2017,
	journal      = {computers \& security},
	publisher    = {Elsevier},
	volume       = 67,
	pages        = {198--210}
}

@techreport{strom2017finding,
	title        = {Finding cyber threats with ATT\&CK-based analytics},
	author       = {Strom, Blake E and Battaglia, Joseph A and Kemmerer, Michael S and Kupersanin, William and Miller, Douglas P and Wampler, Craig and Whitley, Sean M and Wolf, Ross D},
	year         = 2017,
	institution  = {Technical Report MTR170202, MITRE}
}

@article{zamalloa2017dissecting,
	title        = {Dissecting robotics-historical overview and future perspectives},
	author       = {Zamalloa, Irati and Kojcev, Risto and Hern{\'a}ndez, Alejandro and Muguruza, I{\~n}igo and Usategui, Lander and Bilbao, Asier and Mayoral, V{\'\i}ctor},
	year         = 2017,
	journal      = {arXiv preprint arXiv:1704.08617}
}

@article{gutierrez2018tsn,
	title        = {Time-sensitive networking for robotics},
	author       = {Guti{\'e}rrez, Carlos San Vicente and Juan, Lander Usategui San and Ugarte, Irati Zamalloa and Mayoral-Vilches, V{\'\i}ctor},
	year         = 2018,
	journal      = {arXiv preprint arXiv:1804.07643}
}

@article{gutierrez2018real,
	title        = {Real-time Linux communications: an evaluation of the Linux communication stack for real-time robotic applications},
	author       = {Guti{\'e}rrez, Carlos San Vicente and Juan, Lander Usategui San and Ugarte, Irati Zamalloa and Vilches, V{\'\i}ctor Mayoral},
	year         = 2018,
	journal      = {arXiv preprint arXiv:1808.10821}
}

@article{gutierrez2018towards,
	title        = {Towards a distributed and real-time framework for robots: Evaluation of ROS 2.0 communications for real-time robotic applications},
	author       = {Guti{\'e}rrez, Carlos San Vicente and Juan, Lander Usategui San and Ugarte, Irati Zamalloa and Vilches, V{\'\i}ctor Mayoral},
	year         = 2018,
	journal      = {arXiv preprint arXiv:1809.02595}
}

@article{gutierrez2018time,
	title        = {Time Synchronization in modular collaborative robots},
	author       = {Guti{\'e}rrez, Carlos San Vicente and Juan, Lander Usategui San and Ugarte, Irati Zamalloa and Goenaga, I{\~n}igo Muguruza and Kirschgens, Laura Alzola and Vilches, V{\'\i}ctor Mayoral},
	year         = 2018,
	journal      = {arXiv preprint arXiv:1809.07295}
}

@article{mendia2018robotics,
	title        = {Robotics CTF (RCTF), a playground for robot hacking},
	author       = {Mendia, Gorka Olalde and Juan, Lander Usategui San and Bascaran, Xabier Perez and Calvo, Asier Bilbao and Cordero, Alejandro Hern{\'a}ndez and Ugarte, Irati Zamalloa and Rosas, Aday Mu{\~n}iz and Vilches, David Mayoral and Carbajo, Unai Ayucar and Kirschgens, Laura Alzola and others},
	year         = 2018,
	journal      = {arXiv preprint arXiv:1810.02690}
}

@article{vilches2018volatile,
	title        = {Volatile memory forensics for the Robot Operating System},
	author       = {Mayoral-Vilches, V{\'\i}ctor and Kirschgens, Laura Alzola and Gil-Uriarte, Endika and Hern{\'a}ndez, Alejandro and Dieber, Bernhard},
	year         = 2018,
	journal      = {arXiv preprint arXiv:1812.09492}
}

@article{vilches2018aztarna,
	title        = {aztarna, a footprinting tool for robots},
	author       = {Mayoral-Vilches, V{\'\i}ctor and Mendia, Gorka Olalde and Baskaran, Xabier Perez and Cordero, Alejandro Hern{\'a}ndez and Juan, Lander Usategui San and Gil-Uriarte, Endika and de Urabain, Odei Olalde Saez and Kirschgens, Laura Alzola},
	year         = 2018,
	journal      = {arXiv preprint arXiv:1812.09490}
}

@inproceedings{mayoral2019industrial,
  title={Industrial robot ransomware: Akerbeltz},
  author={Mayoral-Vilches, V{\'\i}ctor and Carbajo, Unai Ayucar and Gil-Uriarte, Endika},
  booktitle={2020 fourth IEEE international conference on robotic computing (IRC)},
  pages={432--435},
  year={2020},
  organization={IEEE}
}

@masterthesis{niculaelearning2018,
	title        = {Applying Reinforcement Learning and Genetic Algorithms in Game-Theoretic Cyber-Security},
	author       = {Niculae, Stefan},
	year         = 2018,
	pages        = 58,
	institution  = {Universitatea din Bucureşt},
	school       = {Facultatea de Matematica ̆ s ̧i Informatica ̆̆i}
}

@inproceedings{applebaum2017analysis,
	title        = {Analysis of automated adversary emulation techniques},
	author       = {Applebaum, Andy and Miller, Doug and Strom, Blake and Foster, Henry and Thomas, Cody},
	year         = 2017,
	booktitle    = {Proceedings of the Summer Simulation Multi-Conference},
	pages        = 16,
	organization = {Society for Computer Simulation International}
}

@phdthesis{elderman2016adversarial,
	title        = {Adversarial reinforcement learning in a cyber security simulation},
	author       = {Elderman, Richard and Pater, Leon JJ and Thie, Albert S},
	year         = 2016,
	school       = {Faculty of Science and Engineering}
}

@incollection{nguyen2016towards,
	title        = {Towards a science of security games},
	author       = {Nguyen, Thanh Hong and Kar, Debarun and Brown, Matthew and Sinha, Arunesh and Jiang, Albert Xin and Tambe, Milind},
	year         = 2016,
	booktitle    = {Mathematical Sciences with Multidisciplinary Applications},
	publisher    = {Springer},
	pages        = {347--381}
}

@inproceedings{ghanem2018reinforcement,
	title        = {Reinforcement Learning for Intelligent Penetration Testing},
	author       = {Ghanem, Mohamed C and Chen, Thomas M},
	year         = 2018,
	booktitle    = {2018 Second World Conference on Smart Trends in Systems, Security and Sustainability (WorldS4)},
	pages        = {185--192},
	organization = {IEEE}
}

@article{schwartz2019autonomous,
	title        = {Autonomous Penetration Testing using Reinforcement Learning},
	author       = {Schwartz, Jonathon and Kurniawati, Hanna},
	year         = 2019,
	journal      = {arXiv preprint arXiv:1905.05965}
}

@inproceedings{applebaum2016intelligent,
	title        = {Intelligent, automated red team emulation},
	author       = {Applebaum, Andy and Miller, Doug and Strom, Blake and Korban, Chris and Wolf, Ross},
	year         = 2016,
	booktitle    = {Proceedings of the 32nd Annual Conference on Computer Security Applications},
	pages        = {363--373},
	organization = {ACM}
}

@article{zhou2019nig,
	title        = {NIG-AP: a new method for automated penetration testing},
	author       = {Zhou, Tian-yang and Zang, Yi-chao and Zhu, Jun-hu and Wang, Qing-xian},
	year         = 2019,
	journal      = {Frontiers of Information Technology \& Electronic Engineering},
	publisher    = {Springer},
	volume       = 20,
	number       = 9,
	pages        = {1277--1288}
}

@article{rege2018machine,
	title        = {Machine learning for cyber defense and attack},
	author       = {Rege, Manjeet and Mbah, Raymond Blanch K},
	year         = 2018,
	journal      = {DATA ANALYTICS 2018},
	pages        = 83
}

@incollection{vinayakumar2019deep,
	title        = {A Deep-dive on Machine Learning for Cyber Security Use Cases},
	author       = {Vinayakumar, R and Soman, KP and Poornachandran, Prabaharan and Menon, Vijay Krishna},
	year         = 2019,
	booktitle    = {Machine Learning for Computer and Cyber Security},
	publisher    = {CRC Press},
	pages        = {122--158}
}

@article{schneier1999attack,
	title        = {Attack trees},
	author       = {Schneier, Bruce},
	year         = 1999,
	journal      = {Dr. Dobb’s journal},
	volume       = 24,
	number       = 12,
	pages        = {21--29}
}

@inproceedings{sheyner2002automated,
	title        = {Automated generation and analysis of attack graphs},
	author       = {Sheyner, Oleg and Haines, Joshua and Jha, Somesh and Lippmann, Richard and Wing, Jeannette M},
	year         = 2002,
	booktitle    = {Proceedings 2002 IEEE Symposium on Security and Privacy},
	pages        = {273--284},
	organization = {IEEE}
}

@inproceedings{roberts2011personalized,
	title        = {Personalized vulnerability analysis through automated planning},
	author       = {Roberts, Mark and Howe, Adele and Ray, Indrajit and Urbanska, Malgorzata and Byrne, Zinta S and Weidert, Janet M},
	year         = 2011,
	booktitle    = {Working Notes for the 2011 IJCAI Workshop on Intelligent Security (SecArt)},
	pages        = 50
}

@article{obes2013attack,
	title        = {Attack planning in the real world},
	author       = {Obes, Jorge Lucangeli and Sarraute, Carlos and Richarte, Gerardo},
	year         = 2013,
	journal      = {arXiv preprint arXiv:1306.4044}
}

@article{sarraute2013penetration,
	title        = {Penetration Testing== POMDP Solving?},
	author       = {Sarraute, Carlos and Buffet, Olivier and Hoffmann, J{\"o}rg},
	year         = 2013,
	journal      = {arXiv preprint arXiv:1306.4714}
}

@article{ghanem2020reinforcement,
	title        = {Reinforcement Learning for Efficient Network Penetration Testing},
	author       = {Ghanem, Mohamed C and Chen, Thomas M},
	year         = 2020,
	journal      = {Information},
	publisher    = {Multidisciplinary Digital Publishing Institute},
	volume       = 11,
	number       = 1,
	pages        = 6
}

@inproceedings{bozic2017planning,
	title        = {Planning the attack! or how to use ai in security testing?},
	author       = {Bozic, Josip and Wotawa, Franz},
	year         = 2017,
	booktitle    = {IWAISe: First International Workshop on Artificial Intelligence in Security},
	volume       = 50
}

@article{mayoral2020devsecops,
	title        = {DevSecOps in Robotics},
	author       = {Mayoral-Vilches, V{\'\i}ctor and Garc{\'\i}a-Maestro, Nuria and Towers, McKenna and Gil-Uriarte, Endika},
	year         = 2020,
	journal      = {arXiv preprint arXiv:2003.10402}
}

@article{lacava2020current,
	title        = {Current Research Issues on Cyber security in Robotics},
	author       = {Lacava, G and Marotta, A and Martinelli, F and Saracino, A and La Marra, A and Gil-Uriarte, E and Vilches, V Mayoral},
	year         = 2020
}

@article{mayoral2020alurity,
	title        = {alurity, a toolbox for robot cybersecurity},
	author       = {Mayoral-Vilches, V{\'\i}ctor and Abad-Fern{\'a}ndez, Irati and Pinzger, Martin and Rass, Stefan and Dieber, Bernhard and Cunha, Alcino and Rodr{\'\i}guez-Lera, Francisco J and Lacava, Giovanni and Marotta, Angelica and Martinelli, Fabio and others},
	year         = 2020,
	journal      = {arXiv preprint arXiv:2010.07759}
}

@article{mayoral2020can,
	title        = {Can ROS be used securely in industry? Red teaming ROS-Industrial},
	author       = {Mayoral-Vilches, V{\'\i}ctor and Pinzger, Martin and Rass, Stefan and Dieber, Bernhard and Gil-Uriarte, Endika},
	year         = 2020,
	journal      = {arXiv preprint arXiv:2009.08211}
}

@article{ROB-061,
	title        = {Cybersecurity in Robotics: Challenges, Quantitative Modeling, and Practice},
	author       = {Quanyan Zhu and Stefan Rass and Bernhard Dieber and Víctor Mayoral-Vilches},
	year         = 2021,
	journal      = {Foundations and Trends® in Robotics},
	volume       = 9,
	number       = 1,
	pages        = {1--129},
	doi          = {10.1561/2300000061},
	issn         = {1935-8253},
	url          = {http://dx.doi.org/10.1561/2300000061}
}

@article{zhu2021cybersecurity,
	title        = {An introduction to Robot System Cybersecurity},
	author       = {Zhu, Quanyan and Rass, Stefan and Dieber, Bernhard and Vilches, Victor Mayoral},
	year         = 2021,
	journal      = {arXiv preprint arXiv:2103.05789}
}

@article{kirschgens2018robot,
	title        = {Robot hazards: from safety to security},
	author       = {Kirschgens, Laura Alzola and Ugarte, Irati Zamalloa and Uriarte, Endika Gil and Rosas, Aday Muniz and Vilches, V{\'\i}ctor Mayoral},
	year         = 2018,
	journal      = {arXiv preprint arXiv:1806.06681}
}

@article{mayoral2021hacking,
	title        = {Hacking planned obsolescense in robotics, towards security-oriented robot teardown},
	author       = {Mayoral-Vilches, Victor and Glera-Pic{\'o}n, Alfonso and Ay{\'u}car-Carbajo, Unai and Rass, Stefan and Pinzger, Martin and Maggi, Federico and Gil-Uriarte, Endika},
	year         = 2021,
	journal      = {Electronic Communications of the EASST},
	volume       = 80
}

@article{mayoral2022robot,
	title        = {Robot teardown, stripping industrial robots for good},
	author       = {Mayoral-Vilches, V{\'\i}ctor and Glera-Pic{\'o}n, Alfonso and Ayucar-Carbajo, Unai and Rass, Stefan and Pinzger, Martin and Maggi, Federico and Gil-Uriarte, Endika},
	year         = 2022,
	journal      = {International Journal of Cyber Forensics and Advanced Threat Investigations}
}

@article{mayoral2022review,
	title        = {Robot cybersecurity, a review},
	author       = {Mayoral-Vilches, V{\'\i}ctor},
	year         = 2022,
	journal      = {International Journal of Cyber Forensics and Advanced Threat Investigations}
}

@inproceedings{mayoral2022sros2,
	title        = {Sros2: Usable cyber security tools for ros 2},
	author       = {Mayoral-Vilches, Victor and White, Ruffin and Caiazza, Gianluca and Arguedas, Mikael},
	year         = 2022,
	booktitle    = {2022 IEEE/RSJ International Conference on Intelligent Robots and Systems (IROS)},
	pages        = {11253--11259},
	organization = {IEEE}
}

@inproceedings{mayoral2022robotcore,
	title        = {Robotcore: An open architecture for hardware acceleration in ros 2},
	author       = {Mayoral-Vilches, V{\'\i}ctor and Neuman, Sabrina M and Plancher, Brian and Reddi, Vijay Janapa},
	year         = 2022,
	booktitle    = {2022 IEEE/RSJ International Conference on Intelligent Robots and Systems (IROS)},
	pages        = {9692--9699},
	organization = {IEEE}
}

@inproceedings{ichnowski2023fogros,
	title        = {Fogros 2: An adaptive and extensible platform for cloud and fog robotics using ros 2},
	author       = {Ichnowski, Jeffrey and Chen, Kaiyuan and Dharmarajan, Karthik and Adebola, Simeon and Danielczuk, Michael and Mayoral-Vilches, V{\i}ctor and Zhan, Hugo and Xu, Derek and Ghassemi, Ramtin and Kubiatowicz, John and others},
	year         = 2023,
	booktitle    = {Proceedings IEEE International Conference on Robotics and Automation}
}

@inproceedings{lera2022threat,
	title        = {Threat modeling for robotic-based production plants},
	author       = {Lera, Francisco J Rodr{\'\i}guez and Santamarta, Miguel {\'A}ngel Gonz{\'a}lez and Costales, Gonzalo Esteban and Ayucar, Unay and Gil-Uriarte, Endika and Glera, Alfonso and Mayoral-Vilches, V{\'\i}ctor},
	year         = 2022,
	booktitle    = {2022 IEEE International Symposium on Safety, Security, and Rescue Robotics (SSRR)},
	pages        = {110--115},
	organization = {IEEE}
}

@article{rass2023game,
	title        = {Game-theoretic APT defense: An experimental study on robotics},
	author       = {Rass, Stefan and K{\"o}nig, Sandra and Wachter, Jasmin and Mayoral-Vilches, V{\'\i}ctor and Panaousis, Emmanouil},
	year         = 2023,
	journal      = {Computers \& Security},
	publisher    = {Elsevier},
	pages        = 103328
}

@article{mayoral2023exploitflow,
	title        = {ExploitFlow, cyber security exploitation routes for Game Theory and AI research in robotics},
	author       = {Mayoral-Vilches, V{\'\i}ctor and Deng, Gelei and Liu, Yi and Pinzger, Martin and Rass, Stefan},
	year         = 2023,
	journal      = {arXiv preprint arXiv:2308.02152}
}

@article{deng2023pentestgpt,
	title        = {Pentestgpt: An llm-empowered automatic penetration testing tool},
	author       = {Deng, Gelei and Liu, Yi and Mayoral-Vilches, V{\'\i}ctor and Liu, Peng and Li, Yuekang and Xu, Yuan and Zhang, Tianwei and Liu, Yang and Pinzger, Martin and Rass, Stefan},
	year         = 2023,
	journal      = {arXiv preprint arXiv:2308.06782}
}

@article{mayoral2023robotperf,
	title        = {RobotPerf: An Open-Source, Vendor-Agnostic, Benchmarking Suite for Evaluating Robotics Computing System Performance},
	author       = {Mayoral-Vilches, V{\'\i}ctor and Jabbour, Jason and Hsiao, Yu-Shun and Wan, Zishen and Mart{\'\i}nez-Fari{\~n}a, Alejandra and Crespo-Alvarez, Martino and Stewart, Matthew and Reina-Munoz, Juan Manuel and Nagras, Prateek and Vikhe, Gaurav and others},
	year         = 2023,
	journal      = {arXiv preprint arXiv:2309.09212},
	booktitle    = {Proceedings of the 2023 IEEE/RSJ International Conference on Intelligent Robots and Systems (IROS) Workshop on Methods for Objective Comparison of Results in Intelligent Robotics Research}
}

@article{maggi2022security,
	title        = {A Security Analysis of the Data Distribution Service (DDS) Protocol},
	author       = {Maggi, Federico and Vosseler, Rainer and Cheng, Mars and Kuo, Patrick and Toyama, Chizuru and Yen, T and Vilches, E Boasson V},
	year         = 2022,
	journal      = {Trend Micro Research}
}

@misc{rosmetrics,
	title        = {ROS Community Metrics},
	author       = {{ROS community}},
	year         = 2023,
	url          = {http://wiki.ros.org/Metrics}
}

@inproceedings{pardo2003omg,
	title        = {OMG data-distribution service: Architectural overview},
	author       = {Pardo-Castellote, Gerardo},
	year         = 2003,
	booktitle    = {23rd International Conference on Distributed Computing Systems Workshops, 2003. Proceedings.},
	pages        = {200--206},
	organization = {IEEE}
}

@misc{wikipedia2020information,
	title        = {Information technology},
	year         = 2020,
	note         = {Accessed: 2020-06-23},
	howpublished = {\url{https://en.wikipedia.org/wiki/Information_technology}}
}

@techreport{stouffer2011nist,
	title        = {Guide to Industrial Control Systems (ICS) Security},
	author       = {Stouffer, Keith and Falco, Joe and Scarfone, Karen},
	year         = 2011,
	number       = {NIST Special Publication 800-82},
	institution  = {National Institute of Standards and Technology}
}

@misc{tuvit2011industrial,
	title        = {Whitepaper Industrial Security based on IEC 62443},
	year         = 2011,
	howpublished = {\url{https://www.tuvit.de/fileadmin/Content/TUV_IT/pdf/Downloads/WhitePaper/whitepaper-iec-62443.pdf}}
}

@inproceedings{atlam2017integration,
	title        = {Integration of cloud computing with internet of things: challenges and open issues},
	author       = {Atlam, Hany F and Alenezi, Ahmed and Alharthi, Abdulrahman and Walters, Robert J and Wills, Gary B},
	year         = 2017,
	booktitle    = {2017 IEEE international conference on internet of things (iThings) and IEEE green computing and communications (GreenCom) and IEEE cyber, physical and social computing (CPSCom) and IEEE smart data (SmartData)},
	pages        = {670--675},
	organization = {IEEE}
}

@article{mayoralvilches2020compromised,
	title        = {A Compromised Supply Chain of Robots},
	author       = {Mayoral Vilches, V{\'\i}ctor},
	year         = 2020,
	journal      = {Cybersecurity Robotics},
	howpublished = {\url{https://cybersecurityrobotics.net/a-compromised-supply-chain-of-robots/}}
}

@inproceedings{meier2015px4,
	title        = {PX4: A node-based multithreaded open source robotics framework for deeply embedded platforms},
	author       = {Meier, Lorenz and Honegger, Dominik and Pollefeys, Marc},
	year         = 2015,
	booktitle    = {2015 IEEE international conference on robotics and automation (ICRA)},
	pages        = {6235--6240},
	organization = {IEEE}
}

@inproceedings{macenski2020marathon2,
	title        = {The Marathon 2: A Navigation System},
	author       = {Macenski, Steve and Martín, Francisco and White, Ruffin and Clavero, Jonatan Ginés},
	year         = 2020,
	booktitle    = {2020 IEEE/RSJ International Conference on Intelligent Robots and Systems (IROS)},
	volume       = {},
	number       = {},
	pages        = {2718--2725},
	doi          = {10.1109/IROS45743.2020.9341207}
}

@article{macenski2021slamtoolbox,
	title        = {SLAM Toolbox: SLAM for the dynamic world},
	author       = {Steve Macenski and Ivona Jambrecic},
	year         = 2021,
	journal      = {Journal of Open Source Software},
	publisher    = {The Open Journal},
	volume       = 6,
	number       = 61,
	pages        = 2783,
	doi          = {10.21105/joss.02783},
	url          = {https://doi.org/10.21105/joss.02783}
}

@inproceedings{breiling2017secure,
	title        = {Secure communication for the robot operating system},
	author       = {B. Breiling and B. Dieber and P. Schartner},
	year         = 2017,
	month        = {April},
	booktitle    = {2017 Annual IEEE International Systems Conference (SysCon)},
	pages        = {1--6},
	doi          = {10.1109/SYSCON.2017.7934755}
}

@inproceedings{Doczi2016Increasing,
	title        = {Increasing ROS 1.x communication security for medical surgery robot},
	author       = {R. Dóczi and F. Kis and B. Sütő and V. Póser and G. Kronreif and E. Jósvai and M. Kozlovszky},
	year         = 2016,
	month        = {Oct},
	booktitle    = {2016 IEEE International Conference on Systems, Man, and Cybernetics (SMC)},
	pages        = {4444--4449},
	doi          = {10.1109/SMC.2016.7844931}
}

@inproceedings{Huang2014ROSRV,
	title        = {ROSRV: Runtime Verification for Robots},
	author       = {Huang, Jeff and Erdogan, Cansu and Zhang, Yi and Moore, Brandon and Luo, Qingzhou and Sundaresan, Aravind and Rosu, Grigore},
	year         = 2014,
	month        = {September},
	booktitle    = {Proceedings of the 14th International Conference on Runtime Verification},
	publisher    = {Springer International Publishing},
	series       = {LNCS},
	volume       = 8734,
	pages        = {247--254},
	doi          = {http://dx.doi.org/10.1007/978-3-319-11164-3_20}
}

@article{McClean2013Preliminary,
	title        = {{A preliminary cyber-physical security assessment of the Robot Operating System (ROS)}},
	author       = {McClean, Jarrod and Stull, Christopher and Farrar, Charles and Mascare{\~{n}}as, David},
	year         = 2013,
	booktitle    = {Unmanned Systems Technology XV},
	volume       = 8741,
	pages        = 874110,
	doi          = {10.1117/12.2016189},
	isbn         = 9780819495327,
	issn         = {0277786X},
	url          = {http://proceedings.spiedigitallibrary.org/proceeding.aspx?doi=10.1117/12.2016189},
	organization = {International Society for Optics and Photonics}
}

@misc{Sundaresan2017SecureROS,
	title        = {Secure {ROS} 0.9.2 documentation},
	author       = {Sundaresan, Aravind and Gerard, Leonard and Kim, Minyoung},
	year         = 2017,
	month        = {July},
	url          = {https://sri-csl.github.io/secure_ros}
}

@inproceedings{Toris2014Message,
	title        = {Message authentication codes for secure remote non-native client connections to ROS enabled robots},
	author       = {R. Toris and C. Shue and S. Chernova},
	year         = 2014,
	month        = {April},
	booktitle    = {2014 IEEE International Conference on Technologies for Practical Robot Applications (TePRA)},
	volume       = {},
	number       = {},
	pages        = {1--6},
	doi          = {10.1109/TePRA.2014.6869141},
	issn         = {2325-0526}
}

@inproceedings{rohith2018scapy,
	title        = {SCAPY-A powerful interactive packet manipulation program},
	author       = {Rohith, R and Moharir, Minal and Shobha, G and others},
	year         = 2018,
	booktitle    = {2018 international conference on networking, embedded and wireless systems (ICNEWS)},
	pages        = {1--5},
	organization = {IEEE}
}

@misc{rtps25,
	title        = {The Real-time Publish-Subscribe Protocol DDS Interoperability Wire Protocol (DDSI-RTPS) Specification, Version 2.5},
	author       = {{Object Management Group}},
	year         = 2021,
	url          = {https://www.omg.org/spec/DDSI-RTPS/2.5/About-DDSI-RTPS/}
}

@misc{dds14,
	title        = {OMG Data Distribution Service (DDS), Version 1.4},
	author       = {{Object Management Group}},
	year         = 2015,
	url          = {https://www.omg.org/spec/DDS/}
}

@misc{ddssecurity11,
	title        = {DDS Security, Version 1.1},
	author       = {{Object Management Group}},
	year         = 2018,
	url          = {https://www.omg.org/spec/DDS-SECURITY/1.1/About-DDS-SECURITY/}
}

@article{johnson2001elliptic,
	title        = {The elliptic curve digital signature algorithm (ECDSA)},
	author       = {Johnson, Don and Menezes, Alfred and Vanstone, Scott},
	year         = 2001,
	journal      = {International journal of information security},
	publisher    = {Springer},
	volume       = 1,
	number       = 1,
	pages        = {36--63}
}

@misc{rivest1983cryptographic,
	title        = {Cryptographic communications system and method},
	author       = {Rivest, Ronald L and Shamir, Adi and Adleman, Leonard M},
	year         = 1983,
	month        = sep # {~20},
	publisher    = {Google Patents},
	note         = {US Patent 4,405,829}
}

@article{diffie1976new,
	title        = {New directions in cryptography},
	author       = {Diffie, Whitfield and Hellman, Martin},
	year         = 1976,
	journal      = {IEEE transactions on Information Theory},
	publisher    = {IEEE},
	volume       = 22,
	number       = 6,
	pages        = {644--654}
}

@article{merkle1978secure,
	title        = {Secure communications over insecure channels},
	author       = {Merkle, Ralph C},
	year         = 1978,
	journal      = {Communications of the ACM},
	publisher    = {ACM New York, NY, USA},
	volume       = 21,
	number       = 4,
	pages        = {294--299}
}

@misc{ros2securitytm,
	title        = {ROS 2 Robotic Systems Threat Model},
	author       = {Moulard, Thomas and Hortala, Juan and Perez, Xabi and Olalde, Gorka and Erice, Borja and Olalde, Odei and Mayoral-Vilches, David},
	year         = 2019,
	url          = {http://design.ros2.org/articles/ros2_threat_model.html}
}

@misc{ros2securityintegration,
	title        = {ROS 2 DDS-Security integration},
	author       = {Fazzari, Kyle},
	year         = 2019,
	url          = {http://design.ros2.org/articles/ros2_dds_security.html}
}

@misc{ros2securityenclaves,
	title        = {ROS 2 Security Enclaves},
	author       = {White, Ruffin and Arguedas, Mikael},
	year         = 2020,
	url          = {http://design.ros2.org/articles/ros2_dds_security.html}
}

@misc{ros2securityaccesscontrol,
	title        = {ROS 2 Access Control Policies},
	author       = {White, Ruffin and Fazzari, Kyle},
	year         = 2019,
	url          = {http://design.ros2.org/articles/ros2_access_control_policies.html}
}

@inproceedings{Goerke2021,
	title        = {Who Controls Your Robot? An Evaluation of ROS Security Mechanisms},
	author       = {Goerke, Niklas and Timmermann, David and Baumgart, Ingmar},
	year         = 2021,
	month        = {02},
	pages        = {60--66},
	doi          = {10.1109/ICARA51699.2021.9376468}
}

@misc{multimachinedeployros2,
	title        = {Added design document for remote and multi-machine launching},
	author       = {Lanting, Matt},
	year         = 2020,
	url          = {https://github.com/ros2/design/pull/297}
}

@article{schneider2003least,
	title        = {Least privilege and more [computer security]},
	author       = {Schneider, Fred B},
	year         = 2003,
	journal      = {IEEE Security \& Privacy},
	publisher    = {IEEE},
	volume       = 1,
	number       = 5,
	pages        = {55--59}
}

@phdthesis{phdWhite,
	title        = {Usable Security and Verification for Distributed Robotic Systems},
	author       = {Ruffin White},
	year         = 2021,
	school       = {University of California San Diego}
}

@phdthesis{phdCaiazza,
	title        = {Application-level Security for Robotic Networks},
	author       = {Gianluca Caiazza},
	year         = 2021,
	school       = {Ca' Foscari University of Venice, Italy}
}

@article{kam2015rviz,
	title        = {Rviz: a toolkit for real domain data visualization},
	author       = {Kam, Hyeong Ryeol and Lee, Sung-Ho and Park, Taejung and Kim, Chang-Hun},
	year         = 2015,
	journal      = {Telecommunication Systems},
	publisher    = {Springer},
	volume       = 60,
	number       = 2,
	pages        = {337--345}
}

@inproceedings{Koenig04designand,
	title        = {Design and Use Paradigms for Gazebo, An Open-Source Multi-Robot Simulator},
	author       = {Nathan Koenig and Andrew Howard},
	year         = 2004,
	booktitle    = {In IEEE/RSJ International Conference on Intelligent Robots and Systems},
	pages        = {2149--2154}
}

@misc{dds,
	title        = {What is DDS?},
	author       = {Object Management Group},
	year         = 2018,
	note         = {Accessed: 2018-12-01},
	howpublished = {https://www.omgwiki.org/dds/what-is-dds-3/}
}

@inproceedings{8956928,
	title        = {FPGA-based Architecture for a Low-Cost 3D Lidar Design and Implementation from Multiple Rotating 2D Lidars with ROS},
	author       = {Queralta, J. Peña and Yuhong, F. and Salomaa, L. and Qingqing, L. and Gia, T. N. and Zou, Z. and Tenhunen, H. and Westerlund, T.},
	year         = 2019,
	booktitle    = {IEEE SENSORS},
	volume       = {},
	number       = {},
	pages        = {1--4},
	doi          = {10.1109/SENSORS43011.2019.8956928}
}

@inproceedings{9355892,
	title        = {FPGA Acceleration of ROS2-Based Reinforcement Learning Agents},
	author       = {Leal, Daniel Pinheiro and Sugaya, Midori and Amano, Hideharu and Ohkawa, Takeshi},
	year         = 2020,
	booktitle    = {International Symposium on Computing and Networking Workshops},
	volume       = {},
	number       = {},
	pages        = {106--112},
	doi          = {10.1109/CANDARW51189.2020.00031}
}

@inproceedings{9397897,
	title        = {ROS on ARM Processor Embedded with FPGA for Improvement of Robotic Computing},
	author       = {Maiti, Tapas Kumar},
	year         = 2021,
	booktitle    = {International Symposium on Devices, Circuits and Systems (ISDCS)},
	volume       = {},
	number       = {},
	pages        = {1--4},
	doi          = {10.1109/ISDCS52006.2021.9397897}
}

@inproceedings{9415584,
	title        = {Automated Integration of High-Level Synthesis FPGA Modules with ROS2 Systems},
	author       = {Leal, Daniel Pinheiro and Sugaya, Midori and Amano, Hideharu and Ohkawa, Takeshi},
	year         = 2020,
	booktitle    = {International Conference on Field-Programmable Technology (ICFPT)},
	volume       = {},
	number       = {},
	pages        = {292--293},
	doi          = {10.1109/ICFPT51103.2020.00052}
}

@inproceedings{ahn2020robel,
	title        = {Robel: Robotics benchmarks for learning with low-cost robots},
	author       = {Ahn, Michael and Zhu, Henry and Hartikainen, Kristian and Ponte, Hugo and Gupta, Abhishek and Levine, Sergey and Kumar, Vikash},
	year         = 2020,
	booktitle    = {Conference on robot learning},
	pages        = {1300--1313},
	organization = {PMLR}
}

@inproceedings{althoff2017commonroad,
	title        = {CommonRoad: Composable benchmarks for motion planning on roads},
	author       = {Althoff, Matthias and Koschi, Markus and Manzinger, Stefanie},
	year         = 2017,
	booktitle    = {2017 IEEE Intelligent Vehicles Symposium (IV)},
	pages        = {719--726},
	organization = {IEEE}
}

@inproceedings{amano2021dataset,
	title        = {A dataset generation for object recognition and a tool for generating ROS2 FPGA node},
	author       = {Amano, Hayato and Mori, Hayato and Mizutani, Akinobu and Ono, Tomohiro and Yoshimoto, Yuma and Ohkawa, Takeshi and Tamukoh, Hakaru},
	year         = 2021,
	booktitle    = {IEEE International Conference on Field-Programmable Technology (ICFPT)},
	pages        = {1--4}
}

@article{amba4axi4,
	title        = {AXI4-stream protocol specification},
	author       = {AMBA, ARM},
	year         = 4,
	journal      = {Volume IHI A},
	volume       = 51,
	pages        = 4
}

@inproceedings{arafat2022response,
	title        = {Response time analysis for dynamic priority scheduling in ROS2},
	author       = {Arafat, Abdullah Al and Vaidhun, Sudharsan and Wilson, Kurt M and Sun, Jinghao and Guo, Zhishan},
	year         = 2022,
	booktitle    = {Proceedings of the 59th ACM/IEEE Design Automation Conference},
	pages        = {301--306}
}

@inproceedings{asgari2020pisces,
	title        = {Pisces: power-aware implementation of slam by customizing efficient sparse algebra},
	author       = {Asgari, Bahar and Hadidi, Ramyad and Ghaleshahi, Nima Shoghi and Kim, Hyesoon},
	year         = 2020,
	booktitle    = {ACM/IEEE Design Automation Conference (DAC)},
	pages        = {1--6}
}

@inproceedings{austin2020titan,
	title        = {Titan: A parallel asynchronous library for multi-agent and soft-body robotics using nvidia cuda},
	author       = {Austin, Jacob and Corrales-Fatou, Rafael and Wyetzner, Sofia and Lipson, Hod},
	year         = 2020,
	booktitle    = {IEEE International Conference on Robotics and Automation (ICRA)},
	pages        = {7754--7760}
}

@inproceedings{bakhshalipour2022rtrbench,
	title        = {Rtrbench: A benchmark suite for real-time robotics},
	author       = {Bakhshalipour, Mohammad and Likhachev, Maxim and Gibbons, Phillip B},
	year         = 2022,
	booktitle    = {2022 IEEE International Symposium on Performance Analysis of Systems and Software (ISPASS)},
	pages        = {175--186},
	organization = {IEEE}
}

@inproceedings{barut2021benchmarking,
	title        = {Benchmarking Real-Time Capabilities of ROS 2 and OROCOS for Robotics Applications},
	author       = {Barut, Sinan and Boneberger, Marco and Mohammadi, Pouya and Steil, Jochen J},
	year         = 2021,
	booktitle    = {2021 IEEE International Conference on Robotics and Automation (ICRA)},
	pages        = {708--714},
	organization = {IEEE}
}

@article{bedard2022ros2_tracing,
	title        = {ros2\_tracing: Multipurpose Low-Overhead Framework for Real-Time Tracing of ROS 2},
	author       = {B{\'e}dard, Christophe and L{\"u}tkebohle, Ingo and Dagenais, Michel},
	year         = {Accessed 2022},
	note         = {\url{gitlab.com/ros-tracing/ros2_tracing}}
}

@article{Bedard_2022,
	title        = {ros2{\_}tracing: Multipurpose Low-Overhead Framework for Real-Time Tracing of {ROS} 2},
	author       = {Christophe Bedard and Ingo Lutkebohle and Michel Dagenais},
	year         = 2022,
	journal      = {{IEEE} Robotics and Automation Letters},
	publisher    = {Institute of Electrical and Electronics Engineers ({IEEE})},
	volume       = 7,
	number       = 3,
	pages        = {6511--6518},
	doi          = {10.1109/lra.2022.3174346}
}

@inproceedings{boroujerdian2018mavbench,
	title        = {Mavbench: Micro aerial vehicle benchmarking},
	author       = {Boroujerdian, Behzad and Genc, Hasan and Krishnan, Srivatsan and Cui, Wenzhi and Faust, Aleksandra and Reddi, Vijay},
	year         = 2018,
	booktitle    = {2018 51st annual IEEE/ACM international symposium on microarchitecture (MICRO)},
	pages        = {894--907},
	organization = {IEEE}
}

@article{chamzas2021motionbenchmaker,
	title        = {MotionBenchMaker: A tool to generate and benchmark motion planning datasets},
	author       = {Chamzas, Constantinos and Quintero-Pena, Carlos and Kingston, Zachary and Orthey, Andreas and Rakita, Daniel and Gleicher, Michael and Toussaint, Marc and Kavraki, Lydia E},
	year         = 2021,
	journal      = {IEEE Robotics and Automation Letters},
	publisher    = {IEEE},
	volume       = 7,
	number       = 2,
	pages        = {882--889}
}

@inproceedings{chen2021fogros,
	title        = {FogROS: An Adaptive Framework for Automating Fog Robotics Deployment},
	author       = {Chen, Kaiyuan Eric and Liang, Yafei and Jha, Nikhil and Ichnowski, Jeffrey and Danielczuk, Michael and Gonzalez, Joseph and Kubiatowicz, John and Goldberg, Ken},
	year         = 2021,
	booktitle    = {IEEE International Conference on Automation Science and Engineering (CASE)},
	pages        = {2035--2042}
}

@inproceedings{choi2021picas,
	title        = {PiCAS: New design of priority-driven chain-aware scheduling for ROS2},
	author       = {Choi, Hyunjong and Xiang, Yecheng and Kim, Hyoseung},
	year         = 2021,
	booktitle    = {IEEE Real-Time and Embedded Technology and Applications Symposium (RTAS)},
	pages        = {251--263}
}

@inproceedings{del2006benchmarks,
	title        = {Benchmarks in robotics research},
	author       = {del Pobil, Angel P and Madhavan, Rad and Messina, Elena},
	year         = 2006,
	booktitle    = {Workshop IROS},
	organization = {Citeseer}
}

@inproceedings{delmerico2018benchmark,
	title        = {A benchmark comparison of monocular visual-inertial odometry algorithms for flying robots},
	author       = {Delmerico, Jeffrey and Scaramuzza, Davide},
	year         = 2018,
	booktitle    = {2018 IEEE international conference on robotics and automation (ICRA)},
	pages        = {2502--2509},
	organization = {IEEE}
}

@inproceedings{desnoyers2006lttng,
	title        = {The lttng tracer: A low impact performance and behavior monitor for gnu/linux},
	author       = {Desnoyers, Mathieu and Dagenais, Michel R},
	year         = 2006,
	booktitle    = {OLS (Ottawa Linux Symposium)},
	volume       = 2006,
	pages        = {209--224},
	organization = {Citeseer}
}

@incollection{eisoldt2021reconfros,
	title        = {ReconfROS: Running ROS on Reconfigurable SoCs},
	author       = {Eisoldt, Marc and Hinderink, Steffen and Tassemeier, Marco and Flottmann, Marcel and Vana, Juri and Wiemann, Thomas and Gaal, Julian and Rothmann, Marc and Porrmann, Mario},
	year         = 2021,
	booktitle    = {Drone Systems Engineering and Rapid Simulation and Performance Evaluation: Methods and Tools},
	pages        = {16--21}
}

@inproceedings{Esmaeilzadeh11,
	title        = {Dark {{Silicon}} and the {{End}} of {{Multicore Scaling}}},
	author       = {Esmaeilzadeh, Hadi and Blem, Emily and St. Amant, Renee and Sankaralingam, Karthikeyan and Burger, Doug},
	booktitle    = {Proceedings of the 38th {{Annual International Symposium}} on {{Computer Architecture}}},
	location     = {{New York, NY, USA}},
	publisher    = {{ACM}},
	series       = {ISCA '11},
	pages        = {365--376},
	doi          = {10.1145/2000064.2000108},
	isbn         = {978-1-4503-0472-6},
	date         = 2011
}

@inproceedings{fan2018surreal,
	title        = {Surreal: Open-source reinforcement learning framework and robot manipulation benchmark},
	author       = {Fan, Linxi and Zhu, Yuke and Zhu, Jiren and Liu, Zihua and Zeng, Orien and Gupta, Anchit and Creus-Costa, Joan and Savarese, Silvio and Fei-Fei, Li},
	year         = 2018,
	booktitle    = {Conference on Robot Learning},
	pages        = {767--782},
	organization = {PMLR}
}

@article{freeman2021brax,
	title        = {Brax--A Differentiable Physics Engine for Large Scale Rigid Body Simulation},
	author       = {Freeman, C Daniel and Frey, Erik and Raichuk, Anton and Girgin, Sertan and Mordatch, Igor and Bachem, Olivier},
	year         = 2021,
	journal      = {arXiv preprint arXiv:2106.13281}
}

@inproceedings{gupta2021efficient,
	title        = {Efficient Computation of Map-scale Continuous Mutual Information on Chip in Real Time},
	author       = {Gupta, Keshav and Li, Peter Zhi Xuan and Karaman, Sertac and Sze, Vivienne},
	year         = 2021,
	booktitle    = {IEEE/RSJ International Conference on Intelligent Robots and Systems (IROS)},
	pages        = {6464--6470}
}

@article{hawg2022rep2008,
	title        = {ROS Enhancement Proposal (REP): REP-2008 - ROS 2 Hardware Acceleration Architecture and Conventions},
	author       = {{ROS 2 Hardware Acceleration Working Group (HAWG)}},
	year         = {Accessed 2022},
	note         = {\url{github.com/ros-infrastructure/rep/pull/324}}
}

@article{heiden2021bench,
	title        = {Bench-MR: A motion planning benchmark for wheeled mobile robots},
	author       = {Heiden, Eric and Palmieri, Luigi and Bruns, Leonard and Arras, Kai O and Sukhatme, Gaurav S and Koenig, Sven},
	year         = 2021,
	journal      = {IEEE Robotics and Automation Letters},
	publisher    = {IEEE},
	volume       = 6,
	number       = 3,
	pages        = {4536--4543}
}

@misc{icra2021,
	title        = {ICRA2021 workshop Cloud-Based Competitions and Benchmarks for Robotic Manipulation and Grasping},
	year         = 2021,
	month        = jun,
	url          = {https://sites.google.com/view/icra2021-workshop/home}
}

@misc{icra2022-metrics,
	title        = {ICRA 2022 Workshop Determining Appropriate Metrics and Test Methods for Soft Actuators in Robotic Systems},
	year         = 2022,
	month        = may,
	url          = {https://sites.google.com/andrew.cmu.edu/softactuatormetrics/}
}

@misc{icra2022-sbd,
	title        = {ICRA 2022 Workshop on Releasing Robots into the Wild: Simulations, Benchmarks, and Deployment},
	year         = 2022,
	month        = may,
	url          = {https://www.dynsyslab.org/releasing-robots-into-the-wild-workshop/}
}

@misc{iros2020,
	title        = {IROS 2020 Workshop on Benchmarking Progress in Autonomous Driving},
	year         = 2020,
	month        = oct,
	url          = {https://www.robotics.qmul.ac.uk/events/iros-2021-workshop/}
}

@misc{iros2021,
	title        = {IROS 2021 Workshop - Benchmarking of robotic grasping and manipulation: protocols, metrics and data analysis},
	year         = 2021,
	month        = sep,
	url          = {https://www.robotics.qmul.ac.uk/events/iros-2021-workshop/}
}

@misc{iros2022,
	title        = {Evaluating Motion Planning Performance},
	year         = 2022,
	month        = oct,
	url          = {https://motion-planning-workshop.kavrakilab.org/}
}

@misc{iros2023,
	title        = {METHODS FOR OBJECTIVE COMPARISON OF RESULTS IN INTELLIGENT ROBOTICS RESEARCH},
	year         = 2023,
	month        = oct,
	url          = {http://www.robot.t.u-tokyo.ac.jp/TCPEBRAS_IROS2023/index.html}
}

@article{james2020rlbench,
	title        = {Rlbench: The robot learning benchmark \& learning environment},
	author       = {James, Stephen and Ma, Zicong and Arrojo, David Rovick and Davison, Andrew J},
	year         = 2020,
	journal      = {IEEE Robotics and Automation Letters},
	publisher    = {IEEE},
	volume       = 5,
	number       = 2,
	pages        = {3019--3026}
}

@inproceedings{kingston2022robowflex,
	title        = {Robowflex: Robot motion planning with MoveIt made easy},
	author       = {Kingston, Zachary and Kavraki, Lydia E},
	year         = 2022,
	booktitle    = {2022 IEEE/RSJ International Conference on Intelligent Robots and Systems (IROS)},
	pages        = {3108--3114},
	organization = {IEEE}
}

@inproceedings{krishnan2022automatic,
	title        = {Automatic Domain-Specific SoC Design for Autonomous Unmanned Aerial Vehicles},
	author       = {Krishnan, Srivatsan and Wan, Zishen and Bhardwaj, Kshitij and Whatmough, Paul and Faust, Aleksandra and Neuman, Sabrina M. and Wei, Gu-Yeon and Brooks, David and Reddi, Vijay Janapa},
	year         = 2022,
	booktitle    = {2022 55th IEEE/ACM International Symposium on Microarchitecture (MICRO)},
	pages        = {300--317},
	organization = {IEEE}
}

@inproceedings{krishnan2022roofline,
	title        = {Roofline model for uavs: A bottleneck analysis tool for onboard compute characterization of autonomous unmanned aerial vehicles},
	author       = {Krishnan, Srivatsan and Wan, Zishen and Bhardwaj, Kshitij and Jadhav, Ninad and Faust, Aleksandra and Reddi, Vijay Janapa},
	year         = 2022,
	booktitle    = {2022 IEEE International Symposium on Performance Analysis of Systems and Software (ISPASS)},
	pages        = {162--174},
	organization = {IEEE}
}

@book{liu2021robotic,
	title        = {Robotic computing on fpgas},
	author       = {Liu, Shaoshan and Wan, Zishen and Yu, Bo and Wang, Yu},
	year         = 2021,
	publisher    = {Springer}
}

@inproceedings{leitner2017acrv,
	title        = {The ACRV picking benchmark: A robotic shelf picking benchmark to foster reproducible research},
	author       = {Leitner, J{\"u}rgen and Tow, Adam W and S{\"u}nderhauf, Niko and Dean, Jake E and Durham, Joseph W and Cooper, Matthew and Eich, Markus and Lehnert, Christopher and Mangels, Ruben and McCool, Christopher and others},
	year         = 2017,
	booktitle    = {2017 IEEE International Conference on Robotics and Automation (ICRA)},
	pages        = {4705--4712},
	organization = {IEEE}
}

@inproceedings{lienen2020reconros,
	title        = {ReconROS: Flexible Hardware Acceleration for ROS2 Applications},
	author       = {Lienen, Christian and Platzner, Marco and Rinner, Bernhard},
	year         = 2020,
	booktitle    = {International Conference on Field-Programmable Technology (ICFPT)},
	pages        = {268--276}
}

@inproceedings{liu2020hardware,
	title        = {Hardware acceleration of robot scene perception algorithms},
	author       = {Liu, Yanqi and Derman, Can Eren and Calderoni, Giuseppe and Bahar, R Iris},
	year         = 2020,
	booktitle    = {IEEE/ACM International Conference On Computer Aided Design (ICCAD)},
	pages        = {1--8}
}

@inproceedings{liu2021archytas,
	title        = {Archytas: A framework for synthesizing and dynamically optimizing accelerators for robotic localization},
	author       = {Liu, Weizhuang and Yu, Bo and Gan, Yiming and Liu, Qiang and Tang, Jie and Liu, Shaoshan and Zhu, Yuhao},
	year         = 2021,
	booktitle    = {MICRO-54: 54th Annual IEEE/ACM International Symposium on Microarchitecture},
	pages        = {479--493}
}

@article{makoviychuk2021isaac,
	title        = {Isaac gym: High performance gpu-based physics simulation for robot learning},
	author       = {Makoviychuk, Viktor and Wawrzyniak, Lukasz and Guo, Yunrong and Lu, Michelle and Storey, Kier and Macklin, Miles and Hoeller, David and Rudin, Nikita and Allshire, Arthur and Handa, Ankur and others},
	year         = 2021,
	journal      = {arXiv preprint arXiv:2108.10470}
}

@article{mayoral2021adaptivecomputing,
	title        = {Adaptive Computing in Robotics, Towards ROS 2 Software-Defined Hardware},
	author       = {Mayoral-Vilches, Víctor and Corradi, Giulio},
	year         = 2021,
	journal      = {Xilinx, WP537}
}

@article{mayoral2021kria,
	title        = {Kria Robotics Stack, a ROS 2-centric approach for hardware acceleration in robotics},
	author       = {Mayoral-Vilches, V{\'\i}ctor},
	year         = 2021,
	journal      = {Xilinx, WP540}
}

@inproceedings{michel2008rat,
	title        = {Rat’s life: A cognitive robotics benchmark},
	author       = {Michel, Olivier and Rohrer, Fabien and Bourquin, Yvan},
	year         = 2008,
	booktitle    = {European Robotics Symposium 2008},
	pages        = {223--232},
	organization = {Springer}
}

@article{mihelich2022image_pipeline,
	title        = {image\_pipeline ROS Metapackage},
	author       = {Mihelich, Patrick and Bowman, James},
	year         = {Accessed 2022},
	note         = {\url{github.com/ros-perception/image_pipeline}}
}

@article{mollbenchmarking,
	title        = {Benchmarking Motion Planning Algorithms},
	author       = {Moll, Mark and Sucan, Ioan A and Kavraki, Lydia E}
}

@inproceedings{munshi2009opencl,
	title        = {The opencl specification},
	author       = {Munshi, Aaftab},
	year         = 2009,
	booktitle    = {IEEE Hot Chips Symposium (HCS)},
	pages        = {1--314}
}

@article{murali2019pyrobot,
	title        = {Pyrobot: An open-source robotics framework for research and benchmarking},
	author       = {Murali, Adithyavairavan and Chen, Tao and Alwala, Kalyan Vasudev and Gandhi, Dhiraj and Pinto, Lerrel and Gupta, Saurabh and Gupta, Abhinav},
	year         = 2019,
	journal      = {arXiv preprint arXiv:1906.08236}
}

@inproceedings{murray2016microarchitecture,
	title        = {The microarchitecture of a real-time robot motion planning accelerator},
	author       = {Murray, Sean and Floyd-Jones, William and Qi, Ying and Konidaris, George and Sorin, Daniel J},
	year         = 2016,
	booktitle    = {IEEE/ACM International Symposium on Microarchitecture (MICRO)},
	pages        = {1--12}
}

@inproceedings{murray2016robot,
	title        = {Robot Motion Planning on a Chip.},
	author       = {Murray, Sean and Floyd-Jones, Will and Qi, Ying and Sorin, Daniel J and Konidaris, George Dimitri},
	year         = 2016,
	booktitle    = {Robotics: Science and Systems}
}

@inproceedings{murray2019programmable,
	title        = {A programmable architecture for robot motion planning acceleration},
	author       = {Murray, Sean and Floyd-Jones, Will and Konidaris, George and Sorin, Daniel J},
	year         = 2019,
	booktitle    = {IEEE International Conference on Application-specific Systems, Architectures and Processors (ASAP)},
	volume       = 2160,
	pages        = {185--188}
}

@article{nair7dynabarn,
	title        = {DynaBARN: Benchmarking Metric Ground Navigation in Dynamic Environments},
	author       = {Nair, Anirudh and Jiang, Fulin and Hou, Kang and Xu, Zifan and Li, Shuozhe and Xiao, Xuesu and Stone, Peter},
	journal      = {navigation},
	volume       = 7,
	pages        = 9
}

@inproceedings{neuman2019benchmarking,
	title        = {Benchmarking and workload analysis of robot dynamics algorithms},
	author       = {Neuman, Sabrina M and Koolen, Twan and Drean, Jules and Miller, Jason E and Devadas, Srinivas},
	year         = 2019,
	booktitle    = {2019 IEEE/RSJ International Conference on Intelligent Robots and Systems (IROS)},
	pages        = {5235--5242},
	organization = {IEEE}
}

@inproceedings{neuman2021robomorphic,
	title        = {Robomorphic computing: a design methodology for domain-specific accelerators parameterized by robot morphology},
	author       = {Neuman, Sabrina M and Plancher, Brian and Bourgeat, Thomas and Tambe, Thierry and Devadas, Srinivas and Reddi, Vijay Janapa},
	year         = 2021,
	booktitle    = {ACM International Conference on Architectural Support for Programming Languages and Operating Systems (ASPLOS)},
	pages        = {674--686}
}

@inproceedings{nikiforov2023rose,
	title        = {RoS{\'E}: A Hardware-Software Co-Simulation Infrastructure Enabling Pre-Silicon Full-Stack Robotics SoC Evaluation},
	author       = {Nikiforov, Dima and Dong, Shengjun Chris and Zhang, Chengyi Lux and Kim, Seah and Nikolic, Borivoje and Shao, Yakun Sophia},
	year         = 2023,
	booktitle    = {Proceedings of the 50th Annual International Symposium on Computer Architecture},
	pages        = {1--15}
}

@inproceedings{nitta2018study,
	title        = {A study on introducing FPGA to ROS based autonomous driving system},
	author       = {Nitta, Yasuhiro and Tamura, Sou and Takase, Hideki},
	year         = 2018,
	booktitle    = {IEEE International Conference on Field-Programmable Technology (FPT)},
	pages        = {421--424}
}

@misc{nvidia2022isaacros,
	title        = {{NVIDIA Isaac ROS}},
	author       = {{NVIDIA}},
	year         = {Accessed 2022},
	note         = {\url{github.com/NVIDIA-ISAAC-ROS}}
}

@article{nvidia2022jetpack,
	title        = {NVIDIA JetPack SDK},
	author       = {{NVIDIA}},
	year         = {Accessed 2022},
	note         = {\url{developer.nvidia.com/embedded/jetpack}}
}

@misc{nvidia_r2bdataset2023,
	title        = {R2B Dataset 2023: Robotics Research Data Collection},
	author       = {{NVIDIA Corporation}},
	year         = 2023,
	month        = apr,
	note         = {Accessed on April 19, 2023},
	url          = {https://catalog.ngc.nvidia.com/orgs/nvidia/teams/isaac/resources/r2bdataset2023}
}

@inproceedings{ohkawa2016architecture,
	title        = {Architecture exploration of intelligent robot system using ros-compliant fpga component},
	author       = {Ohkawa, Takeshi and Yamashina, Kazushi and Matsumoto, Takuya and Ootsu, Kanemitsu and Yokota, Takashi},
	year         = 2016,
	booktitle    = {IEEE International Symposium on Rapid System Prototyping (RSP)},
	pages        = {1--7}
}

@article{ohkawa2018fpga,
	title        = {FPGA components for integrating FPGAs into robot systems},
	author       = {Ohkawa, Takeshi and Yamashina, Kazushi and Kimura, Hitomi and Ootsu, Kanemitsu and Yokota, Takashi},
	year         = 2018,
	journal      = {IEICE Transactions on Information and Systems},
	publisher    = {The Institute of Electronics, Information and Communication Engineers},
	volume       = 101,
	number       = 2,
	pages        = {363--375}
}

@inproceedings{ohkawa2019high,
	title        = {High level synthesis of ROS protocol interpretation and communication circuit for FPGA},
	author       = {Ohkawa, Takeshi and Sugata, Yuhei and Watanabe, Harumi and Ogura, Nobuhiko and Ootsu, Kanemitsu and Yokota, Takashi},
	year         = 2019,
	booktitle    = {IEEE/ACM International Workshop on Robotics Software Engineering (RoSE)},
	pages        = {33--36}
}

@inproceedings{panadda2021low,
	title        = {Low-Power High-Performance Intelligent Camera Framework ROS-FPGA Node},
	author       = {Panadda, Solod and Nattha, Jindapetch and Daniel, Pinheiro Leal and Takeshi, Ohkawa},
	year         = 2021,
	booktitle    = {Asia Pacific Conference on Robot IoT System Development and Platform},
	number       = 2020,
	pages        = {73--74}
}

@inproceedings{perille2020benchmarking,
	title        = {Benchmarking metric ground navigation},
	author       = {Perille, Daniel and Truong, Abigail and Xiao, Xuesu and Stone, Peter},
	year         = 2020,
	booktitle    = {2020 IEEE International Symposium on Safety, Security, and Rescue Robotics (SSRR)},
	pages        = {116--121},
	organization = {IEEE}
}

@article{plancher2021accelerating,
	title        = {Accelerating Robot Dynamics Gradients on a CPU, GPU, and FPGA},
	author       = {Plancher, Brian and Neuman, Sabrina M and Bourgeat, Thomas and Kuindersma, Scott and Devadas, Srinivas and Reddi, Vijay Janapa},
	year         = 2021,
	journal      = {IEEE Robotics and Automation Letters},
	publisher    = {IEEE},
	volume       = 6,
	number       = 2,
	pages        = {2335--2342}
}

@inproceedings{plancher2021grid,
	title        = {GRiD: GPU-Accelerated Rigid Body Dynamics with Analytical Gradients},
	author       = {Plancher, Brian and Neuman, Sabrina M and Ghosal, Radhika and Kuindersma, Scott and Reddi, Vijay Janapa},
	year         = 2022,
	booktitle    = {IEEE International Conference on Robotics and Automation (ICRA)},
	pages        = {6253--6260},
	organization = {IEEE}
}

@inproceedings{PlancherParallelDDP,
	title        = {A performance analysis of parallel differential dynamic programming on a GPU},
	author       = {Plancher, Brian and Kuindersma, Scott},
	year         = 2018,
	month        = {Dec.},
	booktitle    = {International Workshop on the Algorithmic Foundations of Robotics (WAFR)},
	address      = {Merida, Mexico},
	pages        = {656--672},
	isbn         = {978-3-030-44051-0},
	organization = {Springer}
}

@conference{PlancherRealtimeMPC,
	title        = {Realtime Model Predictive Control using Parallel DDP on a GPU},
	author       = {Plancher, Brian and Kuindersma, Scott},
	year         = 2019,
	month        = {May.},
	booktitle    = {Toward Online Optimal Control of Dynamic Robots Workshop at the International Conference on Robotics and Automation (ICRA)},
	address      = {Montreal, Canada}
}

@inproceedings{podlubne2019fpga,
	title        = {FPGA-ROS: Methodology to Augment the Robot Operating System with FPGA Designs},
	author       = {Podlubne, Ariel and G{\"o}hringer, Diana},
	year         = 2019,
	booktitle    = {IEEE International Conference on ReConFigurable Computing and FPGAs (ReConFig)},
	pages        = {1--5}
}

@inproceedings{puck2020distributed,
	title        = {Distributed and synchronized setup towards real-time robotic control using ROS2 on Linux},
	author       = {Puck, Lennart and Keller, Philip and Schnell, Tristan and Plasberg, Carsten and Tanev, Atanas and Heppner, Georg and Roennau, Arne and Dillmann, R{\"u}diger},
	year         = 2020,
	booktitle    = {2020 IEEE 16th International Conference on Automation Science and Engineering (CASE)},
	pages        = {1287--1293},
	organization = {IEEE}
}

@inproceedings{reke2020self,
	title        = {A self-driving car architecture in ROS2},
	author       = {Reke, Michael and Peter, Daniel and Schulte-Tigges, Joschua and Schiffer, Stefan and Ferrein, Alexander and Walter, Thomas and Matheis, Dominik},
	year         = 2020,
	booktitle    = {2020 International SAUPEC/RobMech/PRASA Conference},
	pages        = {1--6},
	organization = {IEEE}
}

@misc{ros-robotics-companies,
	title        = {{ros-robotics-companies}},
	author       = {Mayoral-Vilches, Victor},
	note         = {[Accessed: July 9, 2023]},
	howpublished = {\url{https://github.com/vmayoral/ros-robotics-companies}}
}

@misc{ros2_benchmark,
	title        = {ROS2 Benchmark},
	author       = {NVIDIA ISAAC ROS},
	year         = 2023,
	journal      = {GitHub repository},
	publisher    = {GitHub},
	howpublished = {\url{https://github.com/NVIDIA-ISAAC-ROS/ros2_benchmark}}
}

@misc{rss2020,
	title        = {Benchmarking Tools for Evaluating Robotic Assembly of Small Parts},
	year         = 2020,
	month        = jul,
	url          = {https://www.uml.edu/research/nerve/assembly-workshop-rss-2020.aspx}
}

@misc{rss2021,
	title        = {2021 RSS Workshop on Advancing Artificial Intelligence and Manipulation for Robotics: Understanding Gaps, Industry and Academic Perspectives, and Community Building},
	year         = 2021,
	month        = jul,
	url          = {https://sites.google.com/view/rss-ai-manipulationperspective/home}
}

@misc{rss2022,
	title        = {Robot Learning in the Cloud: Remote Operations and Benchmarking},
	year         = 2022,
	month        = jul,
	url          = {https://sites.google.com/andrew.cmu.edu/cloud-robotics-benchmarking/}
}

@misc{rss2023,
	title        = {Datasets and Benchmarking Tools for Advancing and Evaluating Robotic Manufacturing},
	year         = 2023,
	month        = jul,
	url          = {https://sites.google.com/view/rss-2023-nist-moad}
}

@inproceedings{sacks2018robox,
	title        = {Robox: an end-to-end solution to accelerate autonomous control in robotics},
	author       = {Sacks, Jacob and Mahajan, Divya and Lawson, Richard C and Esmaeilzadeh, Hadi},
	year         = 2018,
	booktitle    = {ACM/IEEE International Symposium on Computer Architecture (ISCA)},
	pages        = {479--490}
}

@article{sucan2012open,
	title        = {The open motion planning library},
	author       = {Sucan, Ioan A and Moll, Mark and Kavraki, Lydia E},
	year         = 2012,
	journal      = {IEEE Robotics \& Automation Magazine},
	publisher    = {IEEE},
	volume       = 19,
	number       = 4,
	pages        = {72--82}
}

@inproceedings{sugata2017acceleration,
	title        = {Acceleration of publish/subscribe messaging in ROS-compliant FPGA component},
	author       = {Sugata, Yuhei and Ohkawa, Takeshi and Ootsu, Kanemitsu and Yokota, Takashi},
	year         = 2017,
	booktitle    = {International Symposium on Highly Efficient Accelerators and Reconfigurable Technologies},
	pages        = {1--6}
}

@article{suleiman2019navion,
	title        = {Navion: A 2-mw fully integrated real-time visual-inertial odometry accelerator for autonomous navigation of nano drones},
	author       = {Suleiman, Amr and Zhang, Zhengdong and Carlone, Luca and Karaman, Sertac and Sze, Vivienne},
	year         = 2019,
	journal      = {IEEE Journal of Solid-State Circuits},
	publisher    = {IEEE},
	volume       = 54,
	number       = 4,
	pages        = {1106--1119}
}

@inproceedings{suzuki2018real,
	title        = {Real-Time ROS extension on transparent CPU/GPU coordination mechanism},
	author       = {Suzuki, Yuhei and Azumi, Takuya and Kato, Shinpei and Nishio, Nobuhiko},
	year         = 2018,
	booktitle    = {IEEE International Symposium on Real-Time Distributed Computing (ISORC)},
	pages        = {184--192}
}

@inproceedings{tan2020opencollbench,
	title        = {OpenCollBench-Benchmarking of Collision Detection \& Proximity Queries as a Web-Service},
	author       = {Tan, Toni and Weller, Rene and Zachmann, Gabriel},
	year         = 2020,
	booktitle    = {The 25th International Conference on 3D Web Technology},
	pages        = {1--9}
}

@inproceedings{Venkatesh10,
	title        = {Conservation {{Cores}}: {{Reducing}} the {{Energy}} of {{Mature Computations}}},
	shorttitle   = {Conservation {{Cores}}},
	author       = {Venkatesh, Ganesh and Sampson, Jack and Goulding, Nathan and Garcia, Saturnino and Bryksin, Vladyslav and Lugo-Martinez, Jose and Swanson, Steven and Taylor, Michael Bedford},
	booktitle    = {Proceedings of the {{Fifteenth Edition}} of {{ASPLOS}} on {{Architectural Support}} for {{Programming Languages}} and {{Operating Systems}}},
	location     = {{New York, NY, USA}},
	publisher    = {{ACM}},
	series       = {ASPLOS XV},
	pages        = {205--218},
	doi          = {10.1145/1736020.1736044},
	isbn         = {978-1-60558-839-1},
	date         = 2010
}

@article{wan2021survey,
	title        = {A survey of fpga-based robotic computing},
	author       = {Wan, Zishen and Yu, Bo and Li, Thomas Yuang and Tang, Jie and Zhu, Yuhao and Wang, Yu and Raychowdhury, Arijit and Liu, Shaoshan},
	year         = 2021,
	journal      = {IEEE Circuits and Systems Magazine},
	publisher    = {IEEE},
	volume       = 21,
	number       = 2,
	pages        = {48--74}
}

@inproceedings{wan2022analyzing,
	title        = {Analyzing and Improving Resilience and Robustness of Autonomous Systems},
	author       = {Wan, Zishen and Swaminathan, Karthik and Chen, Pin-Yu and Chandramoorthy, Nandhini and Raychowdhury, Arijit},
	year         = 2022,
	booktitle    = {Proceedings of the 41st IEEE/ACM International Conference on Computer-Aided Design},
	pages        = {1--9}
}

@inproceedings{wan2022robotic,
	title        = {Robotic computing on fpgas: Current progress, research challenges, and opportunities},
	author       = {Wan, Zishen and Lele, Ashwin and Yu, Bo and Liu, Shaoshan and Wang, Yu and Reddi, Vijay Janapa and Hao, Cong and Raychowdhury, Arijit},
	year         = 2022,
	booktitle    = {2022 IEEE 4th International Conference on Artificial Intelligence Circuits and Systems (AICAS)},
	pages        = {291--295},
	organization = {IEEE}
}

@inproceedings{weisz2016robobench,
	title        = {Robobench: Towards sustainable robotics system benchmarking},
	author       = {Weisz, Jonathan and Huang, Yipeng and Lier, Florian and Sethumadhavan, Simha and Allen, Peter},
	year         = 2016,
	booktitle    = {2016 IEEE International Conference on Robotics and Automation (ICRA)},
	pages        = {3383--3389},
	organization = {IEEE}
}

@article{williams2017model,
	title        = {Model predictive path integral control: From theory to parallel computation},
	author       = {Williams, Grady and Aldrich, Andrew and Theodorou, Evangelos A},
	year         = 2017,
	journal      = {Journal of Guidance, Control, and Dynamics},
	publisher    = {American Institute of Aeronautics and Astronautics},
	volume       = 40,
	number       = 2,
	pages        = {344--357}
}

@article{xilinx2022kria,
	title        = {Kria® KV260 Vision AI Starter},
	author       = {{Xilinx}},
	year         = {Accessed 2022},
	note         = {\url{xilinx.com/products/som/kria/kv260-vision-starter-kit.html}}
}

@article{xilinx2022vitis,
	title        = {Vitis Unified Software Platform},
	author       = {{Xilinx}},
	year         = {Accessed 2022},
	note         = {\url{xilinx.com/support/download/index.html/content/xilinx/en/downloadNav/vitis.html}}
}

@article{xilinx2022xrt,
	title        = {Xilinx Runtime (XRT)},
	author       = {{Xilinx}},
	year         = {Accessed 2022},
	note         = {\url{github.com/Xilinx/XRT}}
}

@article{yamashina2015proposal,
	title        = {Proposal of ROS-compliant FPGA Component for Low-Power Robotic Systems: case study on image processing application},
	author       = {Yamashina, Kazushi and Ohkawa, Takeshi and Ootsu, Kanemitsu and Yokota, Takashi},
	year         = 2015,
	journal      = {International Workshop on FPGAs for Software Programmers (FSP)}
}

@inproceedings{yamashina2016crecomp,
	title        = {crecomp: Automated design tool for ros-compliant fpga component},
	author       = {Yamashina, Kazushi and Kimura, Hitomi and Ohkawa, Takeshi and Ootsu, Kanemitsu and Yokota, Takashi},
	year         = 2016,
	booktitle    = {IEEE International Symposium on Embedded Multicore/Many-core Systems-on-Chip (MCSOC)},
	pages        = {138--145}
}

@inproceedings{yang2020exploring,
	title        = {Exploring Real-Time Executor on ROS 2},
	author       = {Yang, Yuqing and Azumi, Takuya},
	year         = 2020,
	booktitle    = {IEEE International Conference on Embedded Software and Systems (ICESS)},
	pages        = {1--8}
}

@article{zhu2020robosuite,
	title        = {robosuite: A modular simulation framework and benchmark for robot learning},
	author       = {Zhu, Yuke and Wong, Josiah and Mandlekar, Ajay and Mart{\'\i}n-Mart{\'\i}n, Roberto and Joshi, Abhishek and Nasiriany, Soroush and Zhu, Yifeng},
	year         = 2020,
	journal      = {arXiv preprint arXiv:2009.12293}
}

@online{robotperf_benchmarks_repo,
	title        = {{Robotperf Benchmarks Repository}},
	author       = {{Robotperf}},
	year         = {Year of access},
	url          = {https://github.com/robotperf/benchmarks/tree/main/benchmarks},
	note         = {GitHub repository directory}
}

@article{deng2017toward,
	title        = {Toward real-time ray tracing: A survey on hardware acceleration and microarchitecture techniques},
	author       = {Deng, Yangdong and Ni, Yufei and Li, Zonghui and Mu, Shuai and Zhang, Wenjun},
	year         = 2017,
	journal      = {ACM Computing Surveys (CSUR)},
	publisher    = {ACM New York, NY, USA},
	volume       = 50,
	number       = 4,
	pages        = {1--41}
}

@misc{rvd2558,
	title        = {RVD\#2558: Default credentials on SICK PLC allows disabling safety features},
	author       = {Mayoral-Vilches, V{\'\i}ctor},
	year         = 2020,
	journal      = {GitHub repository},
	publisher    = {GitHub},
	howpublished = {\url{https://github.com/aliasrobotics/RVD/issues/2558}}
}

@misc{alias_robotics_case_2021,
	title        = {Case {Study} - penetration testing {Mobile} {Industrial} {Robots}},
	author       = {{alias Robotics}},
	year         = 2021,
	url          = {https://aliasrobotics.com/case-study-pentesting-mir.php},
	urldate      = {2022-02-23}
}

@misc{alias_robotics_week_2020,
	title        = {The {Week} of {Mobile} {Industrial} {Robots}'s bugs},
	author       = {{alias Robotics}},
	year         = 2020,
	month        = jun,
	journal      = {Alias Robotics {\textbar} Robot cybersecurity news},
	url          = {https://news.aliasrobotics.com/the-week-of-mobile-industrial-robots-bugs/},
	urldate      = {2022-02-23},
	language     = {en}
}

@misc{alias_robotics_case_2019,
	title        = {Case {Study} - threat modeling a {ROS2} robot},
	author       = {{alias Robotics}},
	year         = 2019,
	url          = {https://aliasrobotics.com/case-study-threat-model-mara.php},
	urldate      = {2022-02-23}
}

@article{kushner_real_2013,
	title        = {The real story of stuxnet},
	author       = {Kushner, David},
	year         = 2013,
	month        = mar,
	journal      = {IEEE Spectrum},
	volume       = 50,
	number       = 3,
	pages        = {48--53},
	doi          = {10.1109/MSPEC.2013.6471059},
	issn         = {0018-9235},
	urldate      = {2019-02-11}
}

@misc{acutronicrobotics_threat_2021,
	title        = {Threat {Model} analysis for {MARA} robot},
	author       = {{AcutronicRobotics}},
	year         = 2021,
	month        = sep,
	publisher    = {Acutronic Robotics},
	url          = {https://github.com/AcutronicRobotics/MARA_threat_model},
	urldate      = {2022-02-23},
	note         = {original-date: 2019-04-07T15:53:00Z}
}

@inproceedings{rass_cut--rope_2019,
	title        = {Cut-{The}-{Rope}: {A} {Game} of {Stealthy} {Intrusion}},
	author       = {Rass, Stefan and König, Sandra and Panaousis, Emmanouil},
	year         = 2019,
	booktitle    = {Decision and {Game} {Theory} for {Security}},
	publisher    = {Springer LNCS 11836},
	pages        = {404--416},
	isbn         = {978-3-030-32430-8}
}

@book{rass_cyber-security_2020,
	title        = {Cyber-{Security} in {Critical} {Infrastructures}: {A} {Game}-{Theoretic} {Approach}},
	author       = {Rass, Stefan and Schauer, Stefan and König, Sandra and Zhu, Quanyan},
	year         = 2020,
	publisher    = {SpringerNature}
}

@incollection{zhu_optimal_2016,
	title        = {Optimal and {Game}-{Theoretic} {Deployment} of {Security} {Investments} in {Interdependent} {Assets}},
	author       = {Hota, Ashish R. and Clements, Abraham A. and Sundaram, Shreyas and Bagchi, Saurabh},
	year         = 2016,
	booktitle    = {Decision and {Game} {Theory} for {Security}},
	publisher    = {Springer LNCS 9996},
	pages        = {101--113},
	isbn         = {978-3-319-47412-0 978-3-319-47413-7}
}

@incollection{hota_game-theoretic_2018,
	title        = {A {Game}-{Theoretic} {Framework} for {Securing} {Interdependent} {Assets} in {Networks}},
	author       = {Hota, Ashish R. and Clements, Abraham A. and Bagchi, Saurabh and Sundaram, Shreyas},
	year         = 2018,
	booktitle    = {Game {Theory} for {Security} and {Risk} {Management}},
	publisher    = {Springer},
	pages        = {157--184},
	isbn         = {978-3-319-75267-9 978-3-319-75268-6},
	note         = {Series Title: Static \& Dynamic Game Theory: Foundations \& Applications},
	editor       = {Rass, Stefan and Schauer, Stefan}
}

@article{li_intelligence-driven_2019,
	title        = {An {Intelligence}-{Driven} {Security}-{Aware} {Defense} {Mechanism} for {Advanced} {Persistent} {Threats}},
	author       = {Li, Yuqing and Dai, Wenkuan and Bai, Jie and Gan, Xiaoying and Wang, Jingchao and Wang, Xinbing},
	year         = 2019,
	month        = mar,
	journal      = {IEEE Trans. on Information Forensics and Security},
	volume       = 14,
	number       = 3,
	pages        = {646--661},
	doi          = {10.1109/TIFS.2018.2847671},
	issn         = {1556-6013, 1556-6021},
	urldate      = {2019-12-16}
}

@article{yuan_stackelberg-game-based_2019,
	title        = {Stackelberg-game-based defense analysis against advanced persistent threats on cloud control system},
	author       = {Yuan, Huanhuan and Xia, Yuanqing and Zhang, Jinhui and Yang, Hongjiu and Mahmoud, Magdi},
	year         = 2019,
	journal      = {IEEE Trans. on Industrial Informatics},
	pages        = {1--1},
	doi          = {10.1109/TII.2019.2925035},
	issn         = {1551-3203, 1941-0050},
	urldate      = {2019-12-16}
}

@article{zhang_flipin:_2019,
	title        = {{FlipIn}: {A} {Game}-{Theoretic} {Cyber} {Insurance} {Framework} for {Incentive}-{Compatible} {Cyber} {Risk} {Management} of {Internet} of {Things}},
	shorttitle   = {{FlipIn}},
	author       = {Zhang, Rui and Zhu, Quanyan},
	year         = 2019,
	journal      = {IEEE Trans. on Information Forensics and Security},
	pages        = {1--1},
	doi          = {10.1109/TIFS.2019.2955891},
	issn         = {1556-6013, 1556-6021},
	urldate      = {2019-12-16}
}

@article{yang_risk_2018,
	title        = {A risk management approach to defending against the advanced persistent threat},
	author       = {Yang, Lu-Xing and Li, Pengdeng and Yang, Xiaofan and Tang, Y.Y.},
	year         = 2018,
	journal      = {IEEE Trans. on Dependable and Secure Computing},
	pages        = {1--1},
	doi          = {10.1109/TDSC.2018.2858786},
	issn         = {1545-5971, 1941-0018, 2160-9209},
	urldate      = {2019-12-16}
}

@article{xiao_attacker-centric_2018,
	title        = {Attacker-{Centric} {View} of a {Detection} {Game} against {Advanced} {Persistent} {Threats}},
	author       = {Xiao, Liang and Xu, Dongjin and Mandayam, Narayan B. and Poor, H. Vincent},
	year         = 2018,
	month        = nov,
	journal      = {IEEE Trans. on Mobile Computing},
	volume       = 17,
	number       = 11,
	pages        = {2512--2523},
	doi          = {10.1109/TMC.2018.2814052},
	issn         = {1536-1233, 1558-0660, 2161-9875},
	urldate      = {2019-12-16}
}

@article{yang_effective_2019,
	title        = {Effective {Repair} {Strategy} {Against} {Advanced} {Persistent} {Threat}: {A} {Differential} {Game} {Approach}},
	shorttitle   = {Effective {Repair} {Strategy} {Against} {Advanced} {Persistent} {Threat}},
	author       = {Yang, Lu-Xing and Li, Pengdeng and Zhang, Yushu and Yang, Xiaofan and Xiang, Yong and Zhou, Wanlei},
	year         = 2019,
	month        = jul,
	journal      = {IEEE Trans. on Information Forensics and Security},
	volume       = 14,
	number       = 7,
	pages        = {1713--1728},
	doi          = {10.1109/TIFS.2018.2885251},
	issn         = {1556-6013, 1556-6021},
	urldate      = {2019-12-16}
}

@article{pawlick_istrict:_2019,
	title        = {{iSTRICT}: {An} {Interdependent} {Strategic} {Trust} {Mechanism} for the {Cloud}-{Enabled} {Internet} of {Controlled} {Things}},
	shorttitle   = {{iSTRICT}},
	author       = {Pawlick, Jeffrey and Chen, Juntao and Zhu, Quanyan},
	year         = 2019,
	month        = jun,
	journal      = {IEEE Trans. on Information Forensics and Security},
	volume       = 14,
	number       = 6,
	pages        = {1654--1669},
	doi          = {10.1109/TIFS.2018.2883272},
	issn         = {1556-6013, 1556-6021},
	urldate      = {2019-12-16}
}

@article{pawlick_strategic_2017,
	title        = {Strategic {Trust} in {Cloud}-{Enabled} {Cyber}-{Physical} {Systems} {With} an {Application} to {Glucose} {Control}},
	author       = {Pawlick, Jeffrey and Zhu, Quanyan},
	year         = 2017,
	month        = dec,
	journal      = {IEEE Trans. on Information Forensics and Security},
	volume       = 12,
	number       = 12,
	pages        = {2906--2919},
	doi          = {10.1109/TIFS.2017.2725224},
	issn         = {1556-6013, 1556-6021},
	urldate      = {2019-12-16}
}

@article{soltan_react_2019,
	title        = {{REACT} to {Cyber} {Attacks} on {Power} {Grids}},
	author       = {Soltan, Saleh and Yannakakis, Mihalis and Zussman, Gil},
	year         = 2019,
	month        = jul,
	journal      = {IEEE Trans. on Network Science and Engineering},
	volume       = 6,
	number       = 3,
	pages        = {459--473},
	doi          = {10.1109/TNSE.2018.2837894},
	issn         = {2327-4697, 2334-329X},
	urldate      = {2019-12-16}
}

@article{panaousis2017game,
	title        = {Game theoretic path selection to support security in device-to-device communications},
	author       = {Panaousis, Emmanouil and Karapistoli, Eirini and Elsemary, Hadeer and Alpcan, Tansu and Khuzani, MHR and Economides, Anastasios A},
	year         = 2017,
	journal      = {Ad Hoc Networks},
	publisher    = {Elsevier},
	volume       = 56,
	pages        = {28--42},
	doi          = {10.1016/j.adhoc.2016.11.008}
}

@article{touhiduzzaman_diversity-based_2019,
	title        = {A {Diversity}-{Based} {Substation} {Cyber} {Defense} {Strategy} {Utilizing} {Coloring} {Games}},
	author       = {Touhiduzzaman, Md and Hahn, Adam and Srivastava, Anurag K.},
	year         = 2019,
	month        = sep,
	journal      = {IEEE Trans. on Smart Grid},
	volume       = 10,
	number       = 5,
	pages        = {5405--5415},
	doi          = {10.1109/TSG.2018.2881672},
	issn         = {1949-3053, 1949-3061},
	urldate      = {2019-12-16}
}

@article{senejohnny_jamming-resilient_2018,
	title        = {A {Jamming}-{Resilient} {Algorithm} for {Self}-{Triggered} {Network} {Coordination}},
	author       = {Senejohnny, Danial and Tesi, Pietro and De Persis, Claudio},
	year         = 2018,
	month        = sep,
	journal      = {IEEE Trans. on Control of Network Systems},
	volume       = 5,
	number       = 3,
	pages        = {981--990},
	doi          = {10.1109/TCNS.2017.2668901},
	issn         = {2325-5870, 2372-2533},
	urldate      = {2019-12-16}
}

@article{hota_interdependent_2018,
	title        = {Interdependent {Security} {Games} on {Networks} {Under} {Behavioral} {Probability} {Weighting}},
	author       = {Hota, Ashish R. and Sundaram, Shreyas},
	year         = 2018,
	month        = mar,
	journal      = {IEEE Trans. on Control of Network Systems},
	volume       = 5,
	number       = 1,
	pages        = {262--273},
	doi          = {10.1109/TCNS.2016.2600484},
	issn         = {2325-5870, 2372-2533},
	urldate      = {2019-12-16}
}

@inproceedings{jajodia_cauldron_2011,
	title        = {Cauldron mission-centric cyber situational awareness with defense in depth},
	author       = {Jajodia, Sushil and Noel, Steven and Kalapa, Pramod and Albanese, Massimiliano and Williams, John},
	year         = 2011,
	booktitle    = {2011 - {MILCOM} 2011 Military Communications Conference},
	location     = {Baltimore, {MD}, {USA}},
	publisher    = {{IEEE}},
	pages        = {1339--1344},
	doi          = {10.1109/MILCOM.2011.6127490},
	isbn         = {978-1-4673-0081-0 978-1-4673-0079-7 978-1-4673-0080-3},
	url          = {http://ieeexplore.ieee.org/document/6127490/},
	urldate      = {2019-05-03},
	eventtitle   = {{MILCOM} 2011 - 2011 {IEEE} Military Communications Conference}
}

@book{fudenberg_game_1991,
	title        = {Game Theory},
	author       = {Fudenberg, Drew and Tirole, Jean},
	year         = 1991,
	location     = {London},
	publisher    = {{MIT} Press}
}

@misc{noauthor_adapt:_2018,
	title        = {{ADAPT}: {Analytical} {Framework} for {Actionable} {Defense} against {Advanced} {Persistent} {Threats} {\textbar} {UW} {Department} of {Electrical} \& {Computer} {Engineering}},
	shorttitle   = {{ADAPT}},
	year         = 2018,
	url          = {https://www.ece.uw.edu/projects/adapt-analytical-framework-for-actionable-defense-against-advanced-persistent-threats/},
	urldate      = {2019-05-13}
}

@article{kamhoua_game_2018,
	title        = {Game {Theoretic} {Modeling} of {Advanced} {Persistent} {Threat} in {Internet} of {Things}},
	author       = {Kamhoua, Charles A. and Leslie, Nandi O. and Weisman, Michael J.},
	year         = 2018,
	journal      = {Journal of Cyber Security and Information Systems},
	volume       = 6,
	number       = 3
}

@article{HuQing:2344,
	title        = {Advanced Persistent Threats Detection Game with Expert System for Cloud},
	author       = {Qing, Hu and Shichao, L\"{u} and Zhiqiang, Shi and Limin, Sun and Liang, Xiao},
	year         = 2017,
	journal      = {Journal of Computer Research and Development},
	publisher    = {Journal of Computer Research and Development},
	volume       = 54,
	number       = 10,
	pages        = 2344,
	doi          = {10.7544/issn1000-1239.2017.20170433},
	eid          = 2344,
	numpages     = 11
}

@article{huang_adaptive_2018,
	title        = {Adaptive {Strategic} {Cyber} {Defense} for {Advanced} {Persistent} {Threats} in {Critical} {Infrastructure} {Networks}},
	author       = {Huang, Linan and Zhu, Quanyan},
	year         = 2018,
	month        = sep,
	journal      = {arXiv:1809.02227 [cs]},
	url          = {http://arxiv.org/abs/1809.02227},
	urldate      = {2019-05-13},
	note         = {arXiv: 1809.02227}
}

@article{moothedath_game_2018,
	title        = {A {Game} {Theoretic} {Approach} for {Dynamic} {Information} {Flow} {Tracking} to {Detect} {Multi}-{Stage} {Advanced} {Persistent} {Threats}},
	author       = {Moothedath, Shana and Sahabandu, Dinuka and Allen, Joey and Clark, Andrew and Bushnell, Linda and Lee, Wenke and Poovendran, Radha},
	year         = 2018,
	month        = nov,
	journal      = {arXiv:1811.05622 [cs]},
	url          = {http://arxiv.org/abs/1811.05622},
	urldate      = {2019-05-13},
	note         = {arXiv: 1811.05622}
}

@incollection{rass_gadapt:_2016,
	title        = {{GADAPT}: {A} {Sequential} {Game}-{Theoretic} {Framework} for {Designing} {Defense}-in-{Depth} {Strategies} {Against} {Advanced} {Persistent} {Threats}},
	author       = {Rass, Stefan and Zhu, Quanyan},
	year         = 2016,
	booktitle    = {Decision and {Game} {Theory} for {Security}},
	publisher    = {Springer LNCS 9996},
	pages        = {314--326},
	doi          = {10.1007/978-3-319-47413-7\_18},
	isbn         = {978-3-319-47412-0}
}

@article{zhu_multi-phase_2018,
	title        = {On {Multi}-{Phase} and {Multi}-{Stage} {Game}-{Theoretic} {Modeling} of {Advanced} {Persistent} {Threats}},
	author       = {Zhu, Quanyan and Rass, Stefan},
	year         = 2018,
	journal      = {IEEE Access},
	volume       = 6,
	pages        = {13958--13971},
	doi          = {10.1109/access.2018.2814481},
	url          = {https://doi.org/10.1109\%2Faccess.2018.2814481},
	copyright    = {All rights reserved}
}

@article{khouzani_saddle-point_2012,
	title        = {Saddle-{Point} {Strategies} in {Malware} {Attack}},
	author       = {Khouzani, M.H.R. and Sarkar, Saswati and Altman, Eitan},
	year         = 2012,
	month        = jan,
	journal      = {IEEE Journal on Selected Areas in Communications},
	volume       = 30,
	number       = 1,
	pages        = {31--43},
	doi          = {10.1109/JSAC.2012.120104},
	issn         = {0733-8716},
	url          = {http://ieeexplore.ieee.org/document/6117760/},
	urldate      = {2019-05-13}
}

@inproceedings{carroll_game_2009,
	title        = {A {Game} {Theoretic} {Investigation} of {Deception} in {Network} {Security}},
	author       = {Carroll, Thomas E. and Grosu, Daniel},
	year         = 2009,
	month        = aug,
	booktitle    = {2009 {Proceedings} of 18th {International} {Conference} on {Computer} {Communications} and {Networks}},
	publisher    = {IEEE},
	address      = {San Francisco, CA, USA},
	pages        = {1--6},
	doi          = {10.1109/ICCCN.2009.5235344},
	url          = {http://ieeexplore.ieee.org/document/5235344/},
	urldate      = {2019-05-13}
}

@inproceedings{la_game_2016,
	title        = {A game theoretic model for enabling honeypots in {IoT} networks},
	author       = {La, Quang Duy and Quek, Tony Q. S. and Lee, Jemin},
	year         = 2016,
	month        = may,
	booktitle    = {2016 {IEEE} {International} {Conference} on {Communications} ({ICC})},
	publisher    = {IEEE},
	address      = {Kuala Lumpur, Malaysia},
	pages        = {1--6},
	doi          = {10.1109/ICC.2016.7510833},
	isbn         = {978-1-4799-6664-6},
	url          = {http://ieeexplore.ieee.org/document/7510833/},
	urldate      = {2019-05-13}
}

@inproceedings{lin_using_2012,
	title        = {Using {Signaling} {Games} to {Model} the {Multi}-step {Attack}-{Defense} {Scenarios} on {Confidentiality}},
	author       = {Lin, Jingqiang and Liu, Peng and Jing, Jiwu},
	year         = 2012,
	booktitle    = {Decision and {Game} {Theory} for {Security}},
	publisher    = {Springer Berlin Heidelberg},
	address      = {Berlin, Heidelberg},
	pages        = {118--137},
	isbn         = {978-3-642-34266-0},
	editor       = {Grossklags, Jens and Walrand, Jean}
}

@article{dijk_flipit:_2013,
	title        = {{FlipIt}: {The} {Game} of "{Stealthy} {Takeover}},
	author       = {van Dijk, Marten and Juels, Ari and Oprea, Alina and Rivest, Ronald L.},
	year         = 2013,
	journal      = {J. Cryptol.},
	volume       = 26,
	number       = 4,
	pages        = {655--713},
	doi          = {10.1007/s00145-012-9134-5},
	issn         = {0933-2790}
}

@article{etesami_dynamic_2019,
	title        = {Dynamic {Games} in {Cyber}-{Physical} {Security}: {An} {Overview}},
	shorttitle   = {Dynamic {Games} in {Cyber}-{Physical} {Security}},
	author       = {Etesami, S. Rasoul and Başar, Tamer},
	year         = 2019,
	month        = jan,
	journal      = {Dynamic Games and Applications},
	doi          = {10.1007/s13235-018-00291-y},
	issn         = {2153-0793},
	url          = {https://doi.org/10.1007/s13235-018-00291-y},
	urldate      = {2019-05-13},
	language     = {en}
}

@book{alpcan_network_2010,
	title        = {Network {Security}: {A} {Decision} and {Game} {Theoretic} {Approach}},
	author       = {Alpcan, Tansu and Başar, Tamer},
	year         = 2010,
	publisher    = {Cambridge University Press}
}

@inproceedings{fang_game_2014,
	title        = {A {Game} {Model} for {Predicting} the {Attack} {Path} of {APT}},
	author       = {Fang, Xupeng and Zhai, Lidong and Jia, Zhaopeng and Bai, Wenyan},
	year         = 2014,
	month        = aug,
	booktitle    = {2014 {IEEE} 12th {International} {Conference} on {Dependable}, {Autonomic} and {Secure} {Computing}},
	publisher    = {IEEE},
	address      = {Dalian, China},
	pages        = {491--495},
	doi          = {10.1109/DASC.2014.94},
	isbn         = {978-1-4799-5079-9 978-1-4799-5078-2},
	url          = {http://ieeexplore.ieee.org/document/6945739/},
	urldate      = {2019-05-13}
}

@article{Lye&Wing2005,
	title        = {{Game strategies in network security}},
	author       = {Lye, Kong-wei and Wing, Jeannette M.},
	year         = 2005,
	journal      = {{International Journal of Information Security}},
	volume       = 4,
	pages        = {71--86},
	doi          = {\url{10.1007/s10207-004-0060-x}}
}

@inbook{Koenig2018b,
	title        = {Assessing the Impact of Malware Attacks in Utility Networks},
	author       = {K{\"o}nig, Sandra and Gouglidis, Antonios and Green, Benjamin and Solar, Alma},
	year         = 2018,
	booktitle    = {Game Theory for Security and Risk Management: From Theory to Practice},
	publisher    = {Springer},
	pages        = {335--351},
	doi          = {10.1007/978-3-319-75268-6\_14},
	isbn         = {978-3-319-75268-6}
}

@article{arkin2005software,
	title        = {Software penetration testing},
	author       = {Arkin, Brad and Stender, Scott and McGraw, Gary},
	year         = 2005,
	journal      = {IEEE Security \& Privacy},
	publisher    = {IEEE},
	volume       = 3,
	number       = 1,
	pages        = {84--87}
}

@article{llm-hallucination,
	title        = {How language model hallucinations can snowball},
	author       = {Zhang, Muru and Press, Ofir and Merrill, William and Liu, Alisa and Smith, Noah A},
	year         = 2023,
	journal      = {arXiv preprint arXiv:2305.13534}
}

@misc{htb,
	title        = {HackTheBox: Hacking Training for the Best},
	author       = {Hack The Box},
	year         = 2024,
	journal      = {Hack The Box},
	note         = {Accessed: Jan 12, 2024},
	url          = {https://www.hackthebox.com/}
}

@misc{kali,
	title        = {Kali Linux: The Most Advanced Penetration Testing Distribution},
	author       = {{Offensive Security}},
	howpublished = {Official Website},
	year         = {2013},  
	note         = {Accessed: January 12, 2024},
	url          = {https://www.kali.org/}
}

@misc{sun2023autohint,
	title        = {AutoHint: Automatic Prompt Optimization with Hint Generation},
	author       = {Hong Sun and Xue Li and Yinchuan Xu and Youkow Homma and Qi Cao and Min Wu and Jian Jiao and Denis Charles},
	year         = 2023,
	eprint       = {2307.07415},
	archiveprefix = {arXiv},
	primaryclass = {cs.CL}
}

@misc{wei2023chainofthought,
	title        = {Chain-of-Thought Prompting Elicits Reasoning in Large Language Models},
	author       = {Jason Wei and Xuezhi Wang and Dale Schuurmans and Maarten Bosma and Brian Ichter and Fei Xia and Ed Chi and Quoc Le and Denny Zhou},
	year         = 2023,
	eprint       = {2201.11903},
	archiveprefix = {arXiv},
	primaryclass = {cs.CL}
}

@misc{metasploit,
	title        = {Metasploit Framework},
	author       = {Rapid7},
	year         = 2023,
	url          = {https://www.metasploit.com/},
	note         = {Accessed: 30-07-2023}
}

@article{LALLIE2020100219,
	title        = {A review of attack graph and attack tree visual syntax in cyber security},
	author       = {Harjinder Singh Lallie and Kurt Debattista and Jay Bal},
	year         = 2020,
	journal      = {Computer Science Review},
	volume       = 35,
	pages        = 100219,
	doi          = {https://doi.org/10.1016/j.cosrev.2019.100219},
	issn         = {1574-0137},
	url          = {https://www.sciencedirect.com/science/article/pii/S1574013719300772}
}

@misc{cwe,
	title        = {{Common Weakness Enumeration (CWE)}},
	author       = {MITRE},
	year         = 2021,
	howpublished = {\url{https://cwe.mitre.org/index.html}}
}

@inproceedings{attacktree,
	title        = {Foundations of Attack Trees},
	author       = {Mauw, Sjouke and Oostdijk, Martijn},
	year         = 2006,
	month        = {07},
	journal      = {Lecture Notes on Computer Science},
	volume       = 3935,
	pages        = {186--198},
	doi          = {10.1007/11734727_17},
	isbn         = {978-3-540-33354-8}
}

@inproceedings{xiong2010model,
	title        = {A model-driven penetration test framework for Web applications},
	author       = {Xiong, Pulei and Peyton, Liam},
	year         = 2010,
	booktitle    = {2010 Eighth International Conference on Privacy, Security and Trust},
	pages        = {173--180},
	organization = {IEEE}
}

@inproceedings{antunes2010benchmarking,
	title        = {Benchmarking vulnerability detection tools for web services},
	author       = {Antunes, Nuno and Vieira, Marco},
	year         = 2010,
	booktitle    = {2010 IEEE International Conference on Web Services},
	pages        = {203--210},
	organization = {IEEE}
}

\end{document}